\documentclass[english]{article}
\usepackage[T1]{fontenc}
\usepackage[latin9]{inputenc}
\usepackage{geometry}
\usepackage{babel}
\usepackage{array}
\usepackage{verbatim}
\usepackage{bm}
\usepackage{multirow}
\usepackage{amsmath}
\usepackage{amssymb}
\usepackage{graphicx}
\usepackage{hyperref}

\allowdisplaybreaks
\makeatletter

\providecommand{\tabularnewline}{\\}

\usepackage{cite}\usepackage{amsthm}\usepackage{dsfont}\usepackage{array}\usepackage{mathrsfs}\usepackage{comment}\onecolumn

\usepackage{color}\usepackage{babel}

\usepackage{amsmath,amssymb}
\newcommand{\conj}{\mathsf{H}}

\usepackage{babel}
\usepackage{algorithm}
\usepackage{algorithmic}
\usepackage{arydshln,enumitem,letltxmacro}

\setlist[itemize]{leftmargin=1em}
\setlist[enumerate]{leftmargin=1em}

\let\hat\widehat
\let\tilde\widetilde

\newcommand{\ba}{\bm{a}}

\newcommand{\bh}{\bm{h}}

\newcommand{\bx}{\bm{x}}

\newcommand{\bX}{\bm{X}}
\newcommand{\bY}{\bm{Y}}

\newcommand{\cA}{\mathcal{A}}
\newcommand{\cB}{\mathcal{B}}

\newcommand{\cE}{\mathcal{E}}

\newcommand{\cH}{\mathcal{H}}

\newcommand{\cO}{\mathcal{O}}
\newcommand{\cP}{\mathcal{P}}

\newcommand{\cS}{{\mathcal{S}}}

\newcommand{\CC}{\mathbb{C}}
\newcommand{\EE}{\mathbb{E}}

\newcommand{\PP}{\mathbb{P}}

\newcommand{\RR}{\mathbb{R}}

\newcommand{\argmin}{\mathop{\mathrm{argmin}}}

\DeclareMathOperator{\Var}{{\mathsf{Var}}}

\DeclareMathOperator{\ind}{\mathds{1}}  

\definecolor{yxc}{RGB}{255,0,0}
\definecolor{yjc}{RGB}{125,0,0}
\definecolor{cm}{RGB}{0,0,200}
\definecolor{kzw}{RGB}{0,150,0}

\newcommand{\vertiii}[1]{{\left\vert\kern-0.25ex\left\vert\kern-0.25ex\left\vert #1 
    \right\vert\kern-0.25ex\right\vert\kern-0.25ex\right\vert}}

\makeatother

\begin{document}
\theoremstyle{plain} \newtheorem{lemma}{\textbf{Lemma}} \newtheorem{prop}{\textbf{Proposition}}\newtheorem{theorem}{\textbf{Theorem}}\setcounter{theorem}{0}
\newtheorem{corollary}{\textbf{Corollary}} \newtheorem{assumption}{\textbf{Assumption}}
\newtheorem{example}{\textbf{Example}} \newtheorem{definition}{\textbf{Definition}}
\newtheorem{fact}{\textbf{Fact}} \theoremstyle{definition}

\theoremstyle{remark}\newtheorem{remark}{\textbf{Remark}}\newtheorem{condition}{Condition}

\title{Implicit Regularization in Nonconvex Statistical Estimation: \\
 Gradient Descent Converges Linearly for Phase Retrieval, \\Matrix Completion, and Blind
Deconvolution}

\author
{
	Cong Ma\thanks{Department of Operations Research and Financial Engineering, Princeton University, Princeton, NJ 08544, USA; Email:
		\texttt{\{congm, kaizheng\}@princeton.edu}.}
	\and Kaizheng Wang\footnotemark[1]
	\and Yuejie Chi\thanks{Department of Electrical and Computer Engineering, Carnegie Mellon University, Pittsburgh, PA 15213, USA; Email:
		\texttt{yuejiechi@cmu.edu}. }
	\and Yuxin Chen\thanks{Department of Electrical Engineering, Princeton University, Princeton, NJ 08544, USA; Email:
		\texttt{yuxin.chen@princeton.edu}.}
}

\date{November 2017; \quad Revised July 2019}
\maketitle
\begin{abstract}
Recent years have seen a flurry of activities in designing provably efficient nonconvex procedures for solving statistical estimation problems.
Due to the highly nonconvex nature of the empirical loss,
 state-of-the-art procedures often require proper regularization
(e.g.~trimming, regularized cost, projection) in order to guarantee fast
convergence. For vanilla procedures such as gradient descent, however, prior theory either
recommends highly conservative learning rates to avoid overshooting, or completely lacks performance guarantees. 

This paper uncovers a striking phenomenon in nonconvex optimization: even in the absence of explicit regularization, gradient descent enforces proper regularization implicitly under various 
statistical models. 
In fact, gradient descent follows a trajectory staying within a basin that enjoys nice geometry, consisting of points incoherent with the sampling mechanism.
This ``implicit regularization'' feature allows gradient descent to proceed in a far more aggressive fashion without overshooting, which in turn results in substantial computational savings. Focusing on three fundamental statistical estimation problems, i.e.~phase retrieval, low-rank matrix completion, and blind deconvolution, we establish that gradient descent achieves near-optimal statistical and computational guarantees without explicit regularization. In particular, by marrying statistical modeling with generic optimization theory, we develop a general recipe for analyzing the trajectories of iterative algorithms via a leave-one-out perturbation argument. As a byproduct, for noisy matrix completion, we demonstrate that gradient descent achieves near-optimal error control --- measured entrywise and by the spectral norm --- which might be of independent interest.

\end{abstract}
\date{}

%
%
\tableofcontents{}


\section{Introduction}

\subsection{Nonlinear systems and empirical loss minimization}\label{sec:intro_nonlinear}

A wide spectrum of science and engineering applications call for solutions
to a nonlinear system of equations. Imagine we have collected a set
of data points $\bm{y}=\{y_{j}\}_{1\leq j \leq m}$, generated by a nonlinear sensing
system,
\[
y_j \,\approx\,\mathcal{A}_j \big(\bm{x}^{\star}\big), \quad 1\leq j\leq m, 
\]
where $\bm{x}^{\star}$ is the unknown object of interest, and
the $\mathcal{A}_j$'s are certain nonlinear maps known \emph{a priori}. Can
we reconstruct the underlying object $\bm{x}^{\star}$
in a faithful yet efficient manner? Problems of this kind abound in
information and statistical science, prominent examples including
low-rank matrix recovery \cite{KesMonSew2010,ExactMC09}, robust principal
component analysis \cite{chandrasekaran2011rank,CanLiMaWri09}, phase retrieval \cite{candes2012phaselift,jaganathan2015phase}, neural
networks \cite{soltanolkotabi2017theoretical,zhong2017recovery},
to name just a few. 

In principle, it is possible to attempt reconstruction by searching for
a solution that minimizes the empirical loss, namely,
\begin{equation}
\text{minimize}_{\bm{x}}\quad f(\bm{x})=  \sum_{j=1}^m \big| y_j -\mathcal{A}_j(\bm{x}) \big|^{2}.\label{eq:empirical-loss-min}
\end{equation}
Unfortunately, this empirical loss minimization problem
is, in many cases, nonconvex, making it NP-hard in general. This issue of non-convexity
 comes up in, for example, several representative problems that epitomize
the structures of nonlinear systems encountered in practice.\footnote{Here, we choose different pre-constants in front of the empirical loss in order to be consistent with the literature of the respective problems. In addition, we only introduce the problem in the noiseless case for simplicity of presentation.} 
\begin{itemize}
\item \textbf{Phase retrieval\,/\,solving quadratic systems of equations.}
Imagine we are asked to recover an unknown object $\bm{x}^{\star}\in\mathbb{R}^{n}$,
but are only given the square modulus of certain linear measurements about the
object, with all sign\,/\,phase information of the measurements missing. This arises,
for example, in X-ray crystallography \cite{candes2013phase}, and in
latent-variable models where the hidden variables are captured by
the missing signs \cite{chen2014convex}. To fix ideas, assume we
would like to solve for $\bm{x}^{\star}\in\mathbb{R}^n $ in the following quadratic
system of $m$ equations
\[
y_{j}=\big(\bm{a}_{j}^{\top}\bm{x}^{\star}\big)^{2},\qquad1\leq j\leq m,
\]
where $\{\bm{a}_{j}\}_{1\leq j \leq m}$ are the known design vectors. One
strategy is thus to solve the following problem
\begin{equation}
\text{minimize}_{\bm{x}\in\mathbb{R}^{n}}\quad f(\bm{x})=\frac{1}{4m}\sum_{j=1}^{m}\Big[y_{j}-\big(\bm{a}_{j}^{\top}\bm{x}\big)^{2}\Big]^{2}.\label{eq:PR-empirical-risk}
\end{equation}
\item \textbf{Low-rank matrix completion.} In many scenarios such as collaborative
filtering, we wish to make predictions about all entries
of an (approximately) low-rank matrix $\bm{M}^{\star}\in\mathbb{R}^{n\times n}$
(e.g.~a matrix consisting of users' ratings about many movies), yet
only a highly incomplete subset of the entries are revealed to us \cite{ExactMC09}. For clarity of presentation, assume $\bm{M}^{\star}$
		to be rank-$r$ ($r\ll n$) and positive semidefinite (PSD), i.e.~$\bm{M}^{\star}=\bm{X}^{\star}\bm{X}^{\star\top}$
with $\bm{X}^{\star}\in\mathbb{R}^{n\times r}$, and suppose we
have only seen the entries
\[
	Y_{j,k} = M^{\star}_{j,k} = ( \bm{X}^{\star} \bm{X}^{\star\top} )_{j,k} ,\qquad(j,k)\in\Omega
\]
within some index subset $\Omega$ of cardinality $m$. These entries can be viewed as nonlinear measurements about the low-rank factor $\bm{X}^{\star}$. The task of completing
		the true matrix $\bm{M}^{\star}$ can then be cast as solving
\begin{equation}
\text{minimize}_{\bm{X}\in\mathbb{R}^{n\times r}}\quad f(\bm{X})
	= \frac{n^2}{4m}
	\sum_{(j,k)\in\Omega}\left( Y_{j,k}-\bm{e}_{j}^{\top}\bm{X}\bm{X}^{\top}\bm{e}_{k}\right)^{2},\label{eq:MC-empirical-risk}
\end{equation}
		where the $\bm{e}_{j}$'s stand for the canonical basis vectors in $\mathbb{R}^n$. 

\item \textbf{Blind deconvolution\,/\,solving bilinear systems of equations.} Imagine we are interested in estimating two signals of interest $\bm{h}^{\star},\bm{x}^{\star}\in\mathbb{C}^{K}$,
but only get to collect  a few bilinear measurements about them. This
problem arises from mathematical modeling of blind deconvolution \cite{ahmed2014blind,DBLP:journals/corr/LiLSW16}, which frequently arises in astronomy, imaging,
communications, etc. The goal is to recover two signals from their convolution.
Put more formally, suppose we have acquired $m$ bilinear measurements
taking the following form
\[
y_{j}=\bm{b}_{j}^{\conj}\bm{h}^{\star}\bm{x}^{\star\conj}\bm{a}_{j},\qquad1\leq j\leq m,
\]
where $\bm{a}_{j},\bm{b}_{j}\in\mathbb{C}^{K}$ are distinct design
		vectors (e.g.~Fourier and/or random design vectors) known {\em a priori}, and $\bm{b}_{j}^{\conj}$ denotes the conjugate transpose of $\bm{b}_{j}$. In order to reconstruct the underlying signals, one asks
for solutions to the following problem
\[
\text{minimize}_{\bm{h},\bm{x}\in\mathbb{C}^{K}}\quad f(\bm{h},\bm{x})=  \sum_{j=1}^{m}\big|y_{j}-\bm{b}_{j}^{\conj}\bm{h}\bm{x}^{\conj}\bm{a}_{j}\big|^{2}.
\]
\end{itemize}

\subsection{Nonconvex optimization via regularized gradient descent} \label{sec:intro_rgd}


First-order methods have been a popular heuristic in practice for solving nonconvex problems including (\ref{eq:empirical-loss-min}). For instance, a widely adopted procedure is gradient descent, which follows the update rule 
\begin{equation}
\bm{x}^{t+1}=\bm{x}^{t}-\eta_{t}\nabla f\big(\bm{x}^{t}\big),\qquad t\geq0,\label{eq:GD-general}
\end{equation}
where $\eta_{t}$ is the learning rate (or step size) and $\bm{x}^{0}$
is some proper initial guess. Given that it only performs a single
gradient calculation $\nabla f(\cdot)$ per iteration (which typically can be completed within near-linear time), this paradigm emerges as a candidate
for solving large-scale problems. The concern is: whether $\bm{x}^{t}$ converges to the global solution and, if so, how long it takes for convergence, especially since  (\ref{eq:empirical-loss-min}) is highly nonconvex. 

Fortunately, despite the worst-case hardness, appealing convergence properties have been discovered
in various statistical estimation problems; the blessing being that the
statistical models help rule out ill-behaved instances. For the average
case, the empirical loss often enjoys benign geometry, in a {\em local} region  (or at least along certain directions)
surrounding the global optimum. In light of this, an effective nonconvex iterative method typically consists of two stages:
\begin{enumerate}[leftmargin=10mm]
\item a carefully-designed initialization scheme (e.g.~spectral method);
\item an iterative refinement procedure (e.g.~gradient descent). 
\end{enumerate}
This strategy has
recently spurred a great deal of interest, owing to its promise of
achieving computational efficiency and statistical accuracy at
once for a growing list of problems (e.g.~\cite{KesMonSew2010,jain2013low,chen2015fast,sun2016guaranteed,candes2014wirtinger,ChenCandes15solving,DBLP:journals/corr/LiLSW16,li2017blind}). However, rather than directly applying gradient descent \eqref{eq:GD-general}, existing theory often suggests enforcing proper regularization. 
Such explicit regularization enables improved computational convergence by properly ``stabilizing'' the
search directions. 
The following regularization schemes, among others, have been suggested
 to obtain or improve computational guarantees. We refer to these algorithms collectively as {\em Regularized Gradient Descent}.
\begin{itemize}
\item \emph{Trimming\,/\,truncation}, which discards/truncates a subset of
the gradient components when forming the descent direction. For instance,
when solving quadratic systems of equations, one can modify the gradient
descent update rule as
\begin{equation}
\bm{x}^{t+1}=\bm{x}^{t}-\eta_{t}\mathcal{T}\left(\nabla f\big(\bm{x}^{t}\big)\right),\label{eq:GD-general-1}
\end{equation}
where $\mathcal{T}$ is an operator that effectively drops samples bearing
too much influence on the search direction. This strategy \cite{ChenCandes15solving,zhang2016provable,wang2017solving}
has been shown to enable exact recovery with linear-time computational
complexity and optimal sample complexity. 
\item \emph{Regularized loss}, which attempts to optimize a regularized
empirical risk 
\begin{equation}
\bm{x}^{t+1}= \bm{x}^{t}-\eta_{t} \left(\nabla f\big(\bm{x}^{t}\big)+\nabla R\big(\bm{x}^{t}\big)\right),\label{eq:GD-general-1-2}
\end{equation}
where $R(\bm{x})$ stands for an additional penalty term in the empirical loss. For example,
in low-rank matrix completion 
$R(\cdot)$ imposes penalty based on the $\ell_{2}$ row norm \cite{KesMonSew2010,sun2016guaranteed} as well as the Frobenius norm \cite{sun2016guaranteed} of the
decision matrix, while in blind deconvolution,
it penalizes the $\ell_{2}$ norm as well as certain component-wise
incoherence measure of the decision vectors \cite{DBLP:journals/corr/LiLSW16,huang2017blind,ling2017regularized}. 
\item \emph{Projection}, which projects the iterates onto certain sets based
on prior knowledge, that is,
\begin{equation}
\bm{x}^{t+1}=\mathcal{P}\left(\bm{x}^{t}-\eta_{t}\nabla f\big(\bm{x}^{t}\big)\right),\label{eq:GD-general-1-1}
\end{equation}
where $\mathcal{P}$ is a certain projection operator used to enforce,
for example, incoherence properties. This strategy has been employed
in both low-rank matrix completion \cite{chen2015fast,zheng2016convergence} and
blind deconvolution \cite{DBLP:journals/corr/LiLSW16}. 
\end{itemize}

Equipped with such regularization procedures, existing works uncover
appealing computational and statistical properties under various statistical models.
Table \ref{tab:Performance-guarantees-GD} summarizes the performance
guarantees derived in the prior literature; for simplicity, only orderwise
results are provided. 
\begin{remark}
There is another role of regularization commonly studied in the literature, which exploits prior knowledge about the structure of the unknown object, such as sparsity to prevent overfitting and improve statistical generalization ability. This is, however, not the focal point of this paper, since we are primarily pursuing solutions
to (\ref{eq:empirical-loss-min}) without imposing additional structures. 
\end{remark}
\begin{table}
\caption{Prior theory for gradient descent (with spectral initialization)\label{tab:Performance-guarantees-GD}}

\centering

\begin{tabular}{c|c|c|c|c|c|c}
\hline 
 & \multicolumn{3}{c|}{Vanilla gradient descent } & \multicolumn{3}{c}{Regularized gradient descent}\tabularnewline
\hline 
\hline 
\multirow{2}{*}{} & sample & iteration & step & sample & iteration & type of\tabularnewline
 & complexity & complexity & size & complexity & complexity & regularization\tabularnewline
\hline 
Phase & \multirow{2}{*}{$n\log n$} & \multirow{2}{*}{$n\log\frac{1}{\epsilon}$} & \multirow{2}{*}{$\frac{1}{n}$} & \multirow{2}{*}{$n$} & \multirow{2}{*}{$\log\frac{1}{\epsilon}$} & trimming\tabularnewline
retrieval &  &  &  &  &  & \cite{ChenCandes15solving,zhang2016provable}\tabularnewline
\hline 
 & \multirow{4}{*}{n/a} & \multirow{4}{*}{n/a} & \multirow{4}{*}{n/a} & \multirow{2}{*}{$nr^{7}$} & \multirow{2}{*}{$\frac{n}{r}\log\frac{1}{\epsilon}$} & regularized loss\tabularnewline
Matrix &  &  &  &  &  & \cite{sun2016guaranteed}\tabularnewline
\cline{5-7} 
completion &  &  &  & \multirow{2}{*}{$nr^{2}$} & \multirow{2}{*}{$r^{2}\log\frac{1}{\epsilon}$} & projection \tabularnewline
 &  &  &  &  &  & \cite{chen2015fast,zheng2016convergence}\tabularnewline
\hline 
Blind  & \multirow{2}{*}{n/a} & \multirow{2}{*}{n/a} & \multirow{2}{*}{n/a} & \multirow{2}{*}{$K\text{poly}\log m$} & \multirow{2}{*}{$m\log\frac{1}{\epsilon}$} & regularized loss \&\tabularnewline
deconvolution &  &  &  &  &  & projection \cite{DBLP:journals/corr/LiLSW16}\tabularnewline
\hline 
\end{tabular}
\end{table}

\subsection{Regularization-free procedures?}


The regularized gradient descent algorithms, while exhibiting appealing performance, usually introduce more algorithmic parameters that need to be carefully tuned based on the assumed statistical models.
In contrast, vanilla gradient descent (cf.~(\ref{eq:GD-general}))
--- which is perhaps the very first method that comes into mind and requires minimal tuning parameters ---
is far less understood (cf.~Table~\ref{tab:Performance-guarantees-GD}). Take matrix completion and blind deconvolution
as examples: to the best of our knowledge, there is currently no theoretical
guarantee derived for vanilla gradient descent.

The situation
is better for phase retrieval: the local convergence of  vanilla
gradient descent, also known as Wirtinger flow (WF), has been investigated in
\cite{candes2014wirtinger,white2015local}. Under i.i.d.~Gaussian
design and with near-optimal sample complexity, WF (combined
with spectral initialization) provably achieves $\epsilon$-accuracy
(in a relative sense) within $O\big(n\log\left({1} / {\varepsilon}\right)\big)$
iterations.
Nevertheless, the computational guarantee is significantly outperformed
by the regularized version (called truncated Wirtinger flow \cite{ChenCandes15solving}),
which only requires $O\big(\log\left({1} / {\varepsilon}\right)\big)$ iterations
to converge with similar per-iteration cost. On closer inspection, the high computational
cost of WF is largely due to the vanishingly small step size $\eta_{t}=O\big({1} / ({n\|\bm{x}^{\star}\|_{2}^{2}})\big)$
--- and hence slow movement --- suggested by the theory \cite{candes2014wirtinger}.
While this is already the largest possible step size allowed in 
the theory published in \cite{candes2014wirtinger}, it is considerably
more conservative than the choice $\eta_{t}=O\big({1} / {\|\bm{x}^{\star}\|_{2}^{2}}\big)$
 theoretically justified for the regularized version \cite{ChenCandes15solving,zhang2016provable}. 

The lack of understanding and suboptimal results about vanilla
gradient descent raise a very natural question: \emph{are regularization-free
iterative algorithms inherently suboptimal when solving nonconvex statistical estimation
problems of this kind?} 

\subsection{Numerical surprise of unregularized gradient descent}

\begin{figure}[t]
\centering

\begin{tabular}{ccc}
\includegraphics[width=0.31\textwidth]{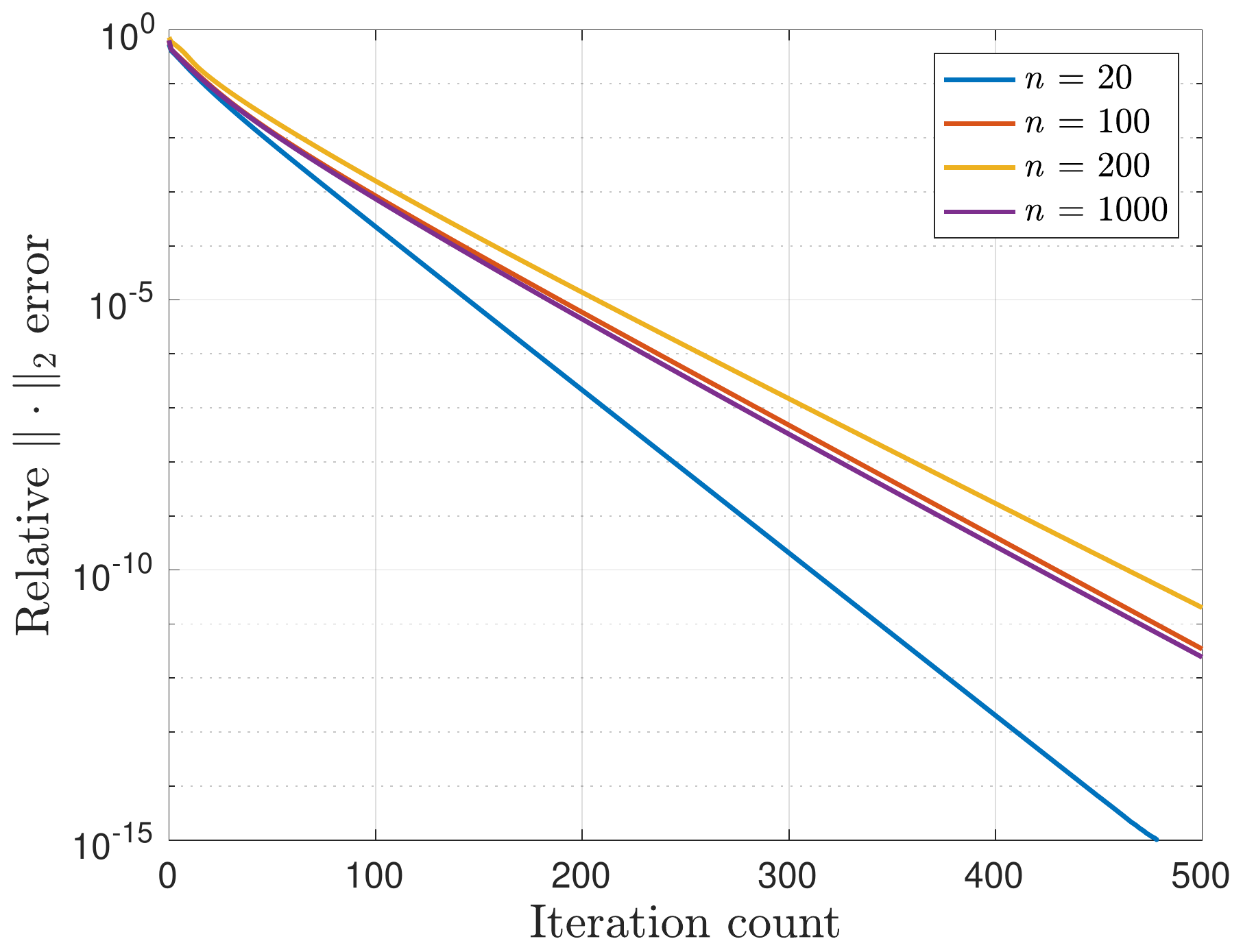} & \includegraphics[width=0.31\textwidth]{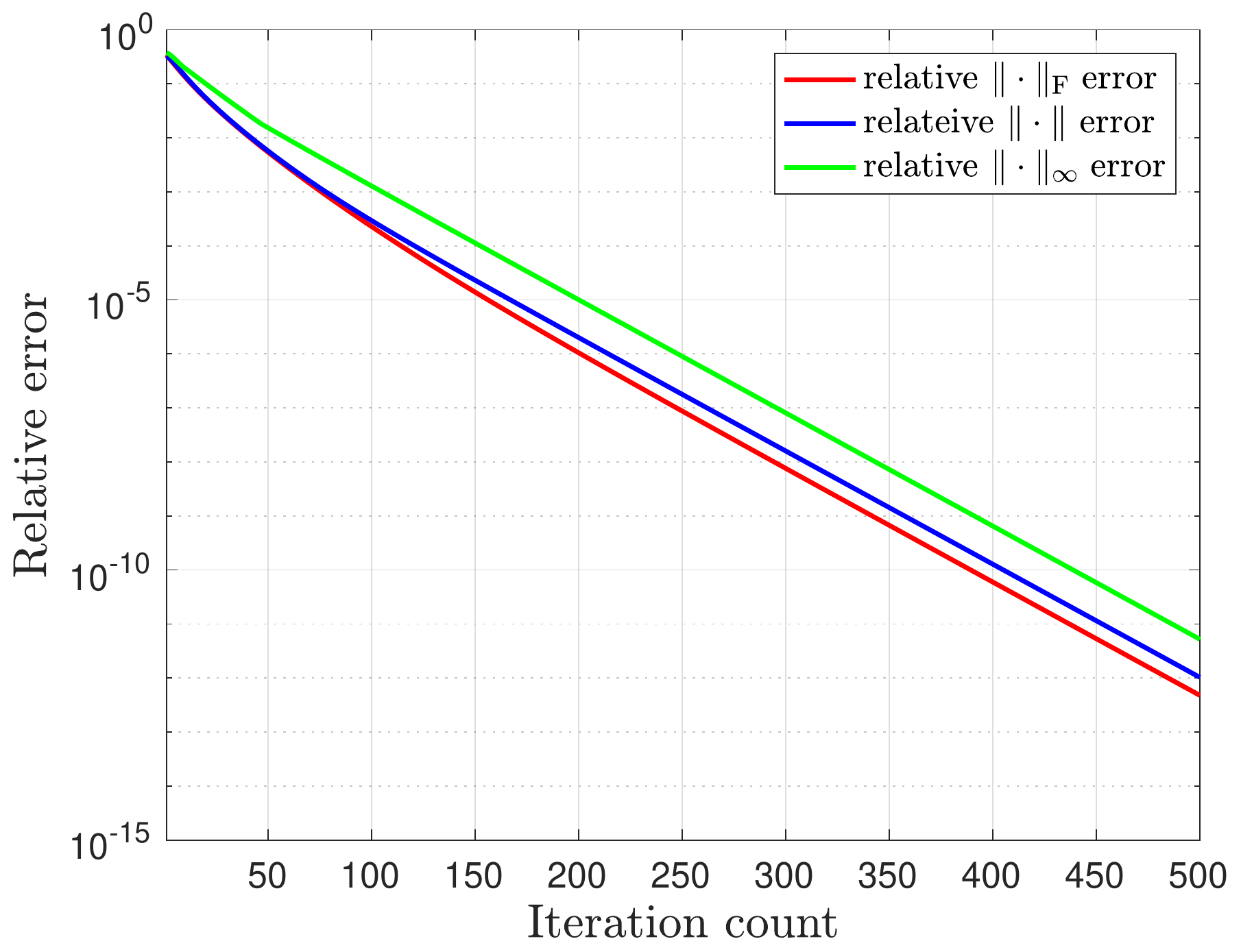} & \includegraphics[width=0.31\textwidth]{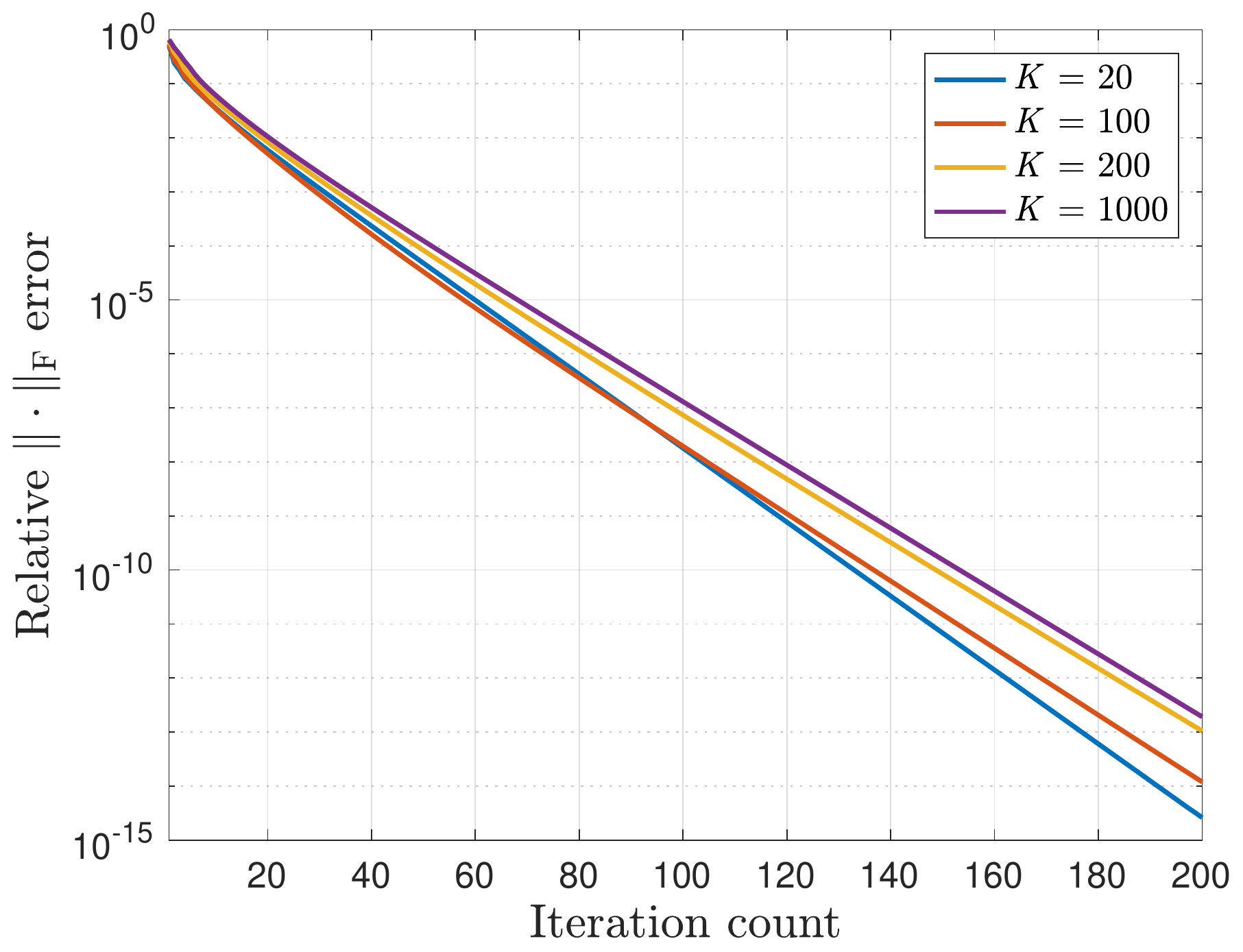}\tabularnewline
(a) phase retrieval & (b) matrix completion & (c) blind deconvolution\tabularnewline
\end{tabular}

\caption{(a) Relative $\ell_{2}$ error of $\bm{x}^{t}$ (modulo the global
phase) vs.~iteration count for phase retrieval under i.i.d.~Gaussian
design, where $m=10n$ and $\eta_{t}=0.1$. (b) Relative error of
$\bm{X}^{t}\bm{X}^{t\top}$ (measured by $\left\Vert \cdot\right\Vert _{\mathrm{F}},\left\Vert \cdot\right\Vert ,\left\Vert \cdot\right\Vert _{\infty}$)
vs.~iteration count for matrix completion, where $n=1000$, $r=10$,
$p=0.1$, and $\eta_{t}=0.2$. (c) Relative error of $\bm{h}^{t}\bm{x}^{t\,\conj}$
(measured by $\left\Vert \cdot\right\Vert _{\mathrm{F}}$) vs.~iteration count for blind deconvolution,
where $m=10K$ and $\eta_{t}=0.5$. \label{fig:WF-stepsize}}
\end{figure}

To answer the preceding question, it is perhaps best to first collect
some numerical evidence. In what follows, we test the performance of vanilla gradient descent for phase retrieval, matrix completion, and blind deconvolution, using a {\em constant} step size. For all of these experiments, the initial guess is
obtained by means of the standard spectral method. Our numerical findings
are as follows:
\begin{itemize}
\item \emph{Phase retrieval. }For each $n$, set $m=10n$, take $\bm{x}^{\star}\in\mathbb{R}^{n}$
to be a random vector with unit norm, and generate the design vectors
$\bm{a}_{j}\overset{\text{i.i.d.}}{\sim}\mathcal{N}(\bm{0},\bm{I}_{n})$,
$1\leq j\leq m$. Figure \ref{fig:WF-stepsize}(a) illustrates the
relative $\ell_{2}$ error $\min\{\|\bm{x}^{t}-\bm{x}^{\star}\|_{2},\|\bm{x}^{t}+\bm{x}^{\star}\|_{2}\}/\|\bm{x}^{\star}\|_{2}$
(modulo the unrecoverable global phase) vs.~the iteration count.
The results are shown for $n=20,100,200,1000$, with the step size
taken to be $\eta_{t}=0.1 $ in all settings.

\item \emph{Matrix completion}. Generate a random PSD matrix $\bm{M}^{\star}\in \RR^{n \times n}$
with dimension $n=1000$, rank $r=10$, and all nonzero eigenvalues equal to one. Each entry of $\bm{M}^{\star}$
is observed independently with probability $p=0.1$. Figure \ref{fig:WF-stepsize}(b) plots
the relative error $\vertiii{\bm{X}^{t}\bm{X}^{t\top}-\bm{M}^{\star}}/\vertiii{\bm{M}^{\star}}$ vs.~the iteration count,
where $\vertiii{\cdot}$ can either be the Frobenius norm $\left\Vert \cdot\right\Vert _{\mathrm{F}}$,
the spectral norm $\|\cdot\|$, or the entrywise $\ell_{\infty}$
		norm $\|\cdot\|_{\infty}$. Here, we pick the step
size as $\eta_{t}=0.2$.
		
\item \emph{Blind deconvolution. }For each $K\in\left\{ 20,100,200,1000\right\} $
and $m=10K$, generate the design vectors $\bm{a}_{j}\overset{\text{i.i.d.}}{\sim}\mathcal{N}(\bm{0},\frac{1}{2}\bm{I}_{K})+i\mathcal{N}(\bm{0},\frac{1}{2}\bm{I}_{K})$
		for $1\leq j\leq m$ independently,\footnote{Here and throughout, $i$ represents the imaginary unit.} and the $\bm{b}_{j}$'s are drawn
from a partial Discrete Fourier Transform (DFT) matrix (to be described in Section \ref{sec:main-blind-deconvolution}).
The underlying signals $\bm{h}^{\star},\bm{x}^{\star}\in\mathbb{C}^{K}$
are produced as random vectors with unit norm. Figure \ref{fig:WF-stepsize}(c) plots the relative
error $\|\bm{h}^{t}\bm{x}^{t\conj}-\bm{h}^{\star}\bm{x}^{\star\conj}\|_{\mathrm{F}}/\|\bm{h}^{\star}\bm{x}^{\star\conj}\|_{\mathrm{F}}$
vs.~the iteration count, with the step size taken
to be $\eta_{t} = 0.5$ in all settings. 
\end{itemize}
In all of these numerical experiments, vanilla gradient descent enjoys remarkable
linear convergence, always yielding an accuracy of $10^{-5}$ (in
a relative sense) within around 200 iterations. In particular, for
the phase retrieval problem, the step size is taken to be $\eta_{t}=0.1 $
although we vary the problem size from $n=20$ to $n=1000$. The
consequence is that the convergence rates experience little changes
when the problem sizes vary. In comparison, the theory published in
\cite{candes2014wirtinger} seems overly pessimistic, as it suggests
a diminishing step size inversely proportional to $n$ and, as a result, an iteration
complexity that worsens as the problem size grows.

In addition, it has been empirically observed in prior literature \cite{ChenCandes15solving,zhang2017reshaped,DBLP:journals/corr/LiLSW16} that vanilla gradient descent performs comparably with the regularized counterpart for phase retrieval and blind deconvolution. To complete the picture, we further conduct experiments on matrix completion. In particular, we follow the experimental setup for matrix completion used above. We vary $p$ from $0.01$ to $0.1$ with 51 logarithmically spaced points. For each $p$, we apply vanilla gradient descent,  projected gradient descent \cite{chen2015fast} and gradient descent with additional regularization terms \cite{sun2016guaranteed} with step size $\eta = 0.2$ to $50$ randomly generated instances. Successful recovery is declared if $\|\bm{X}^{t}\bm{X}^{t\top} - \bm{M}^{\star}\|_{\mathrm{F}} / \|\bm{M}^{\star}\|_{\mathrm{F}} \leq 10^{-5}$ in $10^{4}$ iterations. Figure \ref{fig:mc_phase} reports the success rate vs.~the sampling rate. As can be seen, the phase transition of vanilla GD and that of GD with regularized cost are almost identical, whereas projected GD performs slightly better than the other two.

\begin{figure}
\centering
	\includegraphics[scale = 0.4]{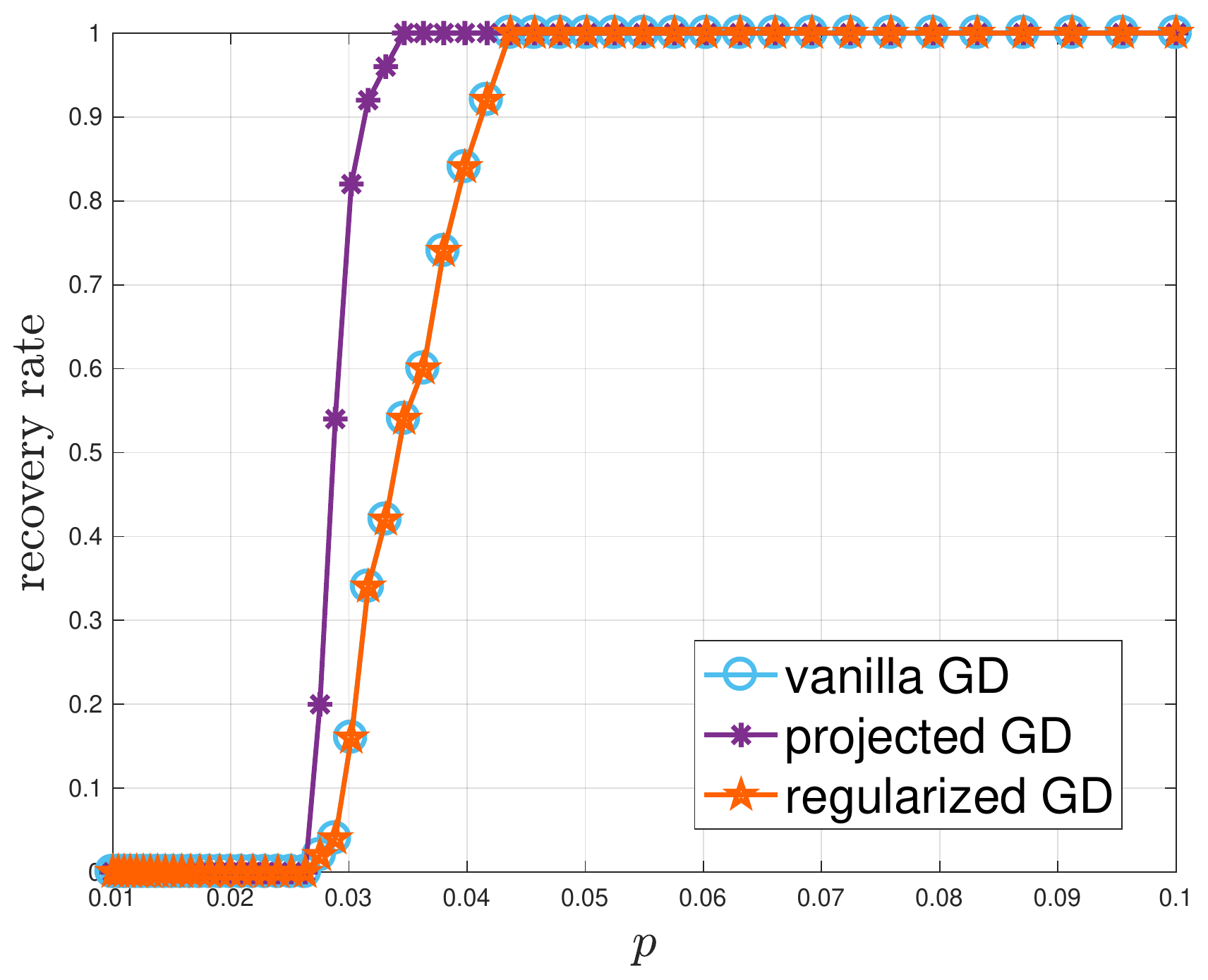}
	\caption{Success rate vs.~sampling rate $p$ over $50$ Monte Carlo trials for matrix completion with $n = 1000$ and $r=10$. \label{fig:mc_phase}}
\end{figure}
In short, the above empirical results are surprisingly positive yet
puzzling. Why was the computational efficiency of vanilla gradient descent unexplained or 
substantially underestimated in prior theory? 

\subsection{This paper}

The main contribution of this paper is towards demystifying the ``unreasonable''
effectiveness of regularization-free nonconvex iterative methods. As asserted
in previous work, regularized gradient descent succeeds by properly
enforcing/promoting certain incoherence conditions throughout the
execution of the algorithm. In contrast, we discover that
{ \setlist{rightmargin=\leftmargin} 
\begin{itemize}
\item[] \emph{Vanilla gradient descent automatically forces the iterates to
stay incoherent with the measurement mechanism, thus implicitly regularizing
the search directions.}
\end{itemize}
}

This ``implicit regularization'' phenomenon is of fundamental importance,
suggesting that vanilla gradient descent proceeds as if it were properly regularized.
This explains the remarkably favorable performance of unregularized gradient descent
in practice. Focusing on the three representative problems mentioned in Section~\ref{sec:intro_nonlinear}, our theory guarantees both statistical and computational efficiency of vanilla gradient descent under random designs and spectral initialization. With near-optimal sample complexity, to attain $\epsilon$-accuracy,
\begin{itemize}
\item \textbf{Phase retrieval (informal)}: vanilla gradient descent converges in $O\big(\log n\log\frac{1}{\epsilon}\big)$ iterations; 
\item \textbf{Matrix completion (informal)}: vanilla gradient descent converges in $O\big(\log\frac{1}{\epsilon}\big)$ iterations;
\item \textbf{Blind deconvolution (informal)}: vanilla gradient descent converges in  $O\big(\log\frac{1}{\epsilon}\big)$ iterations. 
\end{itemize}
In words, gradient descent provably achieves (nearly) linear convergence in all of these examples. Throughout this paper, an algorithm is said to {\em converge (nearly) linearly} to $\bm{x}^{\star}$ in the noiseless case if the iterates $\{\bm{x}^t\}$ obey
\begin{equation*}
\mathrm{dist}(\bm{x}^{t+1}, \bm{x}^{\star}) \leq (1-c)\, \mathrm{dist}(\bm{x}^{t}, \bm{x}^{\star}), \qquad \forall t \geq 0
\end{equation*}
for some $0<c\leq 1$ that is (almost) independent of the problem size. Here, $\mathrm{dist}(\cdot,\cdot)$ can be any appropriate discrepancy measure. 

As a byproduct
of our theory, gradient descent also provably controls the {\em entrywise}
empirical risk uniformly across all iterations; for instance, this
implies that vanilla gradient descent controls entrywise estimation error for the
matrix completion task. Precise statements of these results are deferred
to Section \ref{sec:main-results} and are briefly summarized in Table \ref{tab:theory-vanilla-GD}. 

Notably, our study of implicit regularization suggests that the behavior of {\em nonconvex optimization} algorithms for statistical estimation needs to be examined in the context of {\em statistical models}, which induces an objective function as a finite sum. Our proof is accomplished via a
leave-one-out perturbation argument, which is inherently tied to
statistical models and leverages  homogeneity across samples. Altogether, this allows us to localize benign landscapes for optimization and characterize finer dynamics not accounted for in generic gradient descent theory.

\begin{table}[t]

\caption{Prior theory vs.~our theory for vanilla gradient descent (with spectral
initialization)\label{tab:theory-vanilla-GD} }

\centering

\begin{tabular}{c|c|c|c|c|c|c}
\hline 
 & \multicolumn{3}{c|}{Prior theory} & \multicolumn{3}{c}{Our theory} \tabularnewline 
\hline 
\hline 
\multirow{2}{*}{} & sample  & iteration  & step  & sample  & iteration  & step \tabularnewline
 & complexity & complexity & size & complexity & complexity & size\tabularnewline
\hline 
Phase retrieval & $n\log n$ & $n\log\left({1} / {\varepsilon}\right)$ & $1/n$ & $n\log n$ & $\log n\log\left({1} / {\varepsilon}\right)$ & $1/\log n$\tabularnewline[\doublerulesep]
\hline 
Matrix completion & n/a & n/a & n/a & $nr^{3}\text{poly}\log n$ & $\log\left({1} / {\varepsilon}\right)$ & 1\tabularnewline[\doublerulesep]
\hline 
Blind deconvolution & n/a & n/a & n/a & $K\text{poly}\log m$ & $\log\left({1} / {\varepsilon}\right)$ & 1\tabularnewline[\doublerulesep]
\hline 
\end{tabular}
\end{table}

\subsection{Notations\label{sec:notations}}
Before continuing, we introduce several notations used throughout the paper. First of all, boldfaced symbols are reserved for vectors and matrices. 
For any vector $\bm{v}$, we use $\|\bm{v}\|_2$ to denote its Euclidean norm. For any matrix $\bm{A}$, we use $\sigma_{j}(\bm{A})$ and $\lambda_{j}(\bm{A})$ to denote its $j$th largest singular value and eigenvalue, respectively, and let $\bm{A}_{j,\cdot}$ and $\bm{A}_{\cdot,j}$ denote its $j$th row and $j$th column, respectively. In addition, $\|\bm{A}\|$, $\|\bm{A}\|_{\mathrm{F}}$, $\|\bm{A}\|_{2,\infty}$, and $\|\bm{A}\|_{\infty}$  stand for the spectral norm (i.e.~the largest singular value), the Frobenius norm, the $\ell_2/\ell_{\infty}$ norm (i.e.~the largest $\ell_2$ norm of the rows), and the entrywise $\ell_{\infty}$ norm (the largest magnitude of all entries) of a matrix $\bm{A}$. Also, $\bm{A}^{\top}$, $\bm{A}^\conj$ and $\overline{\bm{A}}$ denote the transpose, the conjugate transpose, and the entrywise conjugate of $\bm{A}$, respectively. $\bm{I}_{n}$ denotes the identity matrix with dimension $n\times n$. The notation $\cO^{n\times r}$ represents the set of all $n\times r$ orthonormal matrices. The notation $[n]$ refers to the set $\{1,\cdots, n\}$. Also, we use $\text{Re}(x)$
to denote the real part of a complex number $x$. 
Throughout the paper, we use the terms ``samples'' and ``measurements'' interchangeably. 

Additionally, the standard notation $f(n)=O\left(g(n)\right)$ or
$f(n)\lesssim g(n)$ means that there exists a constant $c>0$ such
that $\left|f(n)\right|\leq c|g(n)|$,  $f(n)\gtrsim g(n)$ means that there exists a constant $c>0$ such
that $|f(n)|\geq c\left|g(n)\right|$, 
and $f(n)\asymp g(n)$ means that there exist constants $c_{1},c_{2}>0$
such that $c_{1}|g(n)|\leq|f(n)|\leq c_{2}|g(n)|$.  Also, $f(n)\gg g(n)$ means that there exists some large enough constant $c>0$ such that $|f(n)|\geq c\left|g(n)\right|$. Similarly, $f(n)\ll g(n)$ means that there exists some sufficiently small constant $c>0$ such that $|f(n)|\leq c\left|g(n)\right|$.

\section{Implicit regularization -- a case study\label{sec:GD-theory}}

To reveal reasons behind the effectiveness of vanilla gradient descent, we first examine existing theory of gradient descent and identify the geometric properties that enable linear convergence. We then develop an understanding as to why prior theory is conservative, and describe the phenomenon of implicit regularization that helps explain the effectiveness of vanilla gradient descent. To facilitate discussion, we will use the problem of solving random quadratic systems (phase retrieval) and Wirtinger flow as a case study, but our diagnosis applies more generally, as will be seen in later sections. 

\subsection{Gradient descent theory revisited\label{subsec:Strong-convexity-smoothness}}

In the convex optimization literature, there are two standard conditions
about the objective function --- strong convexity and smoothness
--- that allow for linear convergence of gradient descent. 

\begin{definition}[Strong convexity]A twice continuously differentiable
function $f:\mathbb{R}^{n}\mapsto\mathbb{R}$ is said to be $\alpha$-strongly
convex for $\alpha > 0$ if 
\[
\nabla^{2}f(\bm{x})\succeq\alpha\bm{I}_{n},\qquad\forall\bm{x}\in \mathbb{R}^n.
\]

\end{definition}\begin{definition}[Smoothness]A twice continuously
differentiable function $f:\mathbb{R}^{n}\mapsto\mathbb{R}$ is said
to be $\beta$-smooth for $\beta > 0$ if 
\[
\left\Vert \nabla^{2}f(\bm{x})\right\Vert \leq\beta,\qquad\forall\bm{x}\in\mathbb{R}^n.
\]
\end{definition}

It is well known that for an unconstrained optimization problem, if
the objective function $f$ is both $\alpha$-strongly convex and $\beta$-smooth,
then vanilla gradient descent (\ref{eq:GD-general}) enjoys $\ell_{2}$ error
contraction \cite[Theorem 3.12]{bubeck2015convex}, namely,
\begin{equation}
\big\|\bm{x}^{t+1}-\bm{x}^{\star}\|_{2}\leq\left(1-\frac{2}{\beta/\alpha+1}\right)\big\|\bm{x}^{t}-\bm{x}^{\star}\big\|_{2},\quad
\mbox{and}\quad\big\|\bm{x}^{t}-\bm{x}^{\star}\|_{2}\leq\left(1-\frac{2}{\beta/\alpha+1}\right)^{t}\big\|\bm{x}^{0}-\bm{x}^{\star}\big\|_{2},\quad t\geq0, \label{eq:GD-classical-theory}
\end{equation}
as long as the step size is chosen as $\eta_{t}=2/(\alpha+\beta)$.
Here, $\bm{x}^{\star}$ denotes  the global minimum. This
immediately reveals the iteration complexity for gradient descent: the number
of iterations taken to attain $\epsilon$-accuracy (in a relative sense) is bounded by
\[
O\left(\frac{\beta}{\alpha}\log\frac{1}{\epsilon}\right).
\]
In other words, the iteration complexity is dictated by and scales
linearly with the condition number --- the ratio $\beta/\alpha$
of smoothness to strong convexity parameters. 

Moving beyond convex optimization, one can easily extend the above theory
to {\em nonconvex} problems with {\em local} strong convexity and smoothness.
More precisely, suppose the objective function $f$ satisfies
\[
\nabla^{2}f(\bm{x})\succeq\alpha\bm{I}\qquad\text{and}\qquad\big\|\nabla^{2}f(\bm{x})\big\|\leq\beta\]
over a local $\ell_{2}$ ball surrounding the global minimum $\bm{x}^{\star}$:
\begin{equation}
\mathcal{B}_{\delta}(\bm{x}) :=	\big\{ \bm{x} \mid \|\bm{x}-\bm{x}^{\star}\|_{2}\leq\delta\|\bm{x}^{\star}\|_{2} \big\}. \label{eq:l2-ball}
\end{equation}
Then the contraction result (\ref{eq:GD-classical-theory})
continues to hold, as long as the algorithm is seeded with an initial
point that falls inside $\mathcal{B}_{\delta}(\bm{x})$.

\subsection{Local geometry for solving random quadratic systems}

To invoke generic gradient descent theory, it is critical to characterize the local strong convexity and smoothness properties of the loss function. Take the  problem of solving random quadratic systems (phase retrieval) as an example.
Consider the i.i.d.~Gaussian design in which 
$\bm{a}_{j}\overset{\mathrm{i.i.d.}}{\sim}\mathcal{N}(\bm{0},\bm{I}_{n})$, $1\leq j\leq m$,
and suppose without loss of generality that the underlying signal
obeys $\|\bm{x}^{\star}\|_{2}=1$. It is well known that $\bm{x}^{\star}$
is the unique minimizer --- up to global phase --- of (\ref{eq:PR-empirical-risk})
under this statistical model, provided that the ratio $m/n$ of equations
to unknowns is sufficiently large. The Hessian of the loss function $f(\bm{x})$ is given by
\begin{equation} 
\nabla^{2}f\left(\bm{x}\right)  =\frac{1}{m}\sum_{j=1}^{m}\left[3\left(\bm{a}_{j}^{\top}\bm{x}\right)^{2}-y_{j}\right]\bm{a}_{j}\bm{a}_{j}^{\top}.\label{eq:hessian-WF-intro}
\end{equation}

\begin{itemize}
\item \emph{Population-level analysis}. Consider the case with an infinite number
of equations or samples, i.e.~$m\rightarrow\infty$, where $\nabla^{2}f(\bm{x})$ converges to its expectation. Simple calculation yields
that
\[
\mathbb{E}\big[\nabla^{2}f(\bm{x})\big]=3\left(\|\bm{x}\|_{2}^{2}\bm{I}_{n}+2\bm{x}\bm{x}^{\top}\right)-\left(\bm{I}_{n}+2\bm{x}^{\star}\bm{x}^{\star\top}\right).
\]
It it straightforward to verify that for any sufficiently small constant
$\delta>0$, one has the crude bound
\[
\bm{I}_{n} \,\preceq\, \mathbb{E}\big[\nabla^{2}f(\bm{x})\big] \,\preceq\, 10\bm{I}_{n},\qquad\forall \bm{x}\in\mathcal{B}_{\delta}(\bm{x}) : \big\|\bm{x}-\bm{x}^{\star}\big\|_{2}\leq\delta\big\|\bm{x}^{\star}\big\|_{2},
\]
meaning that $f$ is $1$-strongly convex and $10$-smooth within
a local ball around $\bm{x}^{\star}$. As a consequence, when we have infinite samples and an initial guess $\bm{x}^{0}$ such that $\|\bm{x}^{0}-\bm{x}^{\star}\|_{2}\leq\delta\big\|\bm{x}^{\star}\big\|_{2}$,
 vanilla gradient descent with a constant step size converges to the global minimum within logarithmic iterations. 
\item \emph{Finite-sample regime with $m\asymp n\log n$.} Now that $f$
exhibits favorable landscape in the population level, one thus hopes
that the fluctuation can be well-controlled so that the nice geometry
		carries over to the finite-sample regime. In the regime where $m\asymp n\log n$ (which is the regime considered in \cite{candes2014wirtinger}), the local strong convexity is still preserved, in the sense that
\[
\nabla^{2}f(\bm{x}) \,\succeq\, \left({1} / {2}\right)\cdot\bm{I}_{n},\qquad\forall\bm{x}:\text{ }\big\|\bm{x}-\bm{x}^{\star}\big\|_{2}\leq\delta\big\|\bm{x}^{\star}\big\|_{2}
\]
occurs with high probability, provided that $\delta>0$ is sufficiently
small (see  \cite{soltanolkotabi2014algorithms,white2015local} and Lemma \ref{lemma:wf_hessian}).
The smoothness parameter, however, is not well-controlled. In fact,
it can be as large as (up to logarithmic factors)\footnote{To demonstrate this, take $\bm{x}=\bm{x}^{\star}+\left(\delta / {\|\bm{a}_{1}\|_{2}}\right)\cdot {\bm{a}_{1}}$ in \eqref{eq:hessian-WF-intro},
one can easily verify that, with high probability, $\big\|\nabla^{2}f(\bm{x})\big\|\geq \left|3(\bm{a}_{1}^{\top}\bm{x})^{2}-y_{1}\right| \big\|\bm{a}_{1}\bm{a}_{1}^{\top}\big\| / m-O(1)\gtrsim {\delta^{2}n^{2}} / {m}\asymp\delta^{2}{n} / {\log n}$. }
\[
\big\|\nabla^{2}f(\bm{x})\big\|\,\lesssim\,n
\]
even when we restrict attention to the local $\ell_{2}$ ball \eqref{eq:l2-ball}
with $\delta>0$ being a fixed small constant. This means that the
condition number $\beta/\alpha$ (defined in Section \ref{subsec:Strong-convexity-smoothness})
may scale as $O(n)$, leading to the
step size recommendation 
\[
\eta_t\,\asymp\,1/n,
\]
and, as a consequence, a high iteration complexity $O\big(n\log (1 / {\epsilon} )\big)$.
This underpins the analysis in \cite{candes2014wirtinger}. 
\end{itemize}

In summary, the geometric properties of the loss function --- even in the
local $\ell_{2}$ ball centering around the global minimum --- is
not as favorable as one anticipates, in particular in view of its population counterpart. A direct application of 
generic gradient descent theory leads to an overly conservative step size and a pessimistic convergence rate, unless the number of samples is enormously larger than the number of unknowns. 

\begin{remark} Notably, due to Gaussian designs, the phase retrieval problem enjoys more favorable geometry compared to other nonconvex problems. In matrix completion and blind deconvolution, the Hessian matrices are rank-deficient even at the population level. In such cases, the above discussions need to be adjusted, e.g. strong convexity is only possible when we restrict attention to certain directions. 
\end{remark}

\subsection{Which region enjoys nicer geometry?}

Interestingly, our theory identifies a local region surrounding $\bm{x}^{\star}$ with a large diameter that enjoys much nicer geometry. This region does not mimic an $\ell_{2}$ ball, but rather, the
intersection of an $\ell_{2}$ ball and a polytope. We term it the {\em region of incoherence and contraction} (RIC). For phase retrieval, the RIC includes all points $\bm{x}\in\mathbb{R}^{n}$ obeying
\begin{subequations}
\begin{align}
 \big\|\bm{x}-\bm{x}^{\star}\big\|_{2} & \leq\delta\big\|\bm{x}^{\star}\big\|_{2}\qquad\text{and}\label{eq:L2-condition-local}\\
\max_{1\leq j\leq m}\big|\bm{a}_{j}^{\top}\big(\bm{x}-\bm{x}^{\star}\big)\big| & \lesssim\sqrt{\log n}\,\big\|\bm{x}^{\star}\big\|_{2}, \label{eq:incoherence-condition-local}
\end{align}
\end{subequations}
where $\delta>0$ is some small numerical constant.
As will be formalized
in Lemma \ref{lemma:wf_hessian}, with high probability the Hessian
matrix satisfies
\[
\left({1}/{2}\right)\cdot\bm{I}_{n}\,\preceq\,\nabla^{2}f(\bm{x}) \,\preceq \,O(\log n)\cdot\bm{I}_n
\]
simultaneously for $\bm{x}$ in the RIC. In words, the Hessian
matrix is nearly well-conditioned (with the condition number bounded
by $O(\log n)$), as long as (i) the iterate is not very far from
the global minimizer (cf.~\eqref{eq:L2-condition-local}), and (ii) the iterate remains incoherent\footnote{If $\bm{x}$ is aligned with (and hence very coherent with) one vector $\bm{a}_{j}$, then with
high probability one has $\big|\bm{a}_{j}^{\top}\big(\bm{x}-\bm{x}^{\star}\big)|\gtrsim\big|\bm{a}_{j}^{\top}\bm{x}| \asymp \sqrt{n}\|\bm{x}\|_{2}$,
which is significantly larger than $\sqrt{\log n}\|\bm{x}\|_{2}$. } with respect to the sensing vectors (cf.~\eqref{eq:incoherence-condition-local}). Another way
to interpret the incoherence condition (\ref{eq:incoherence-condition-local})
is that the empirical risk needs to be well-controlled uniformly across
all samples. See Figure \ref{fig:L2-incoherence-GD}(a) for an illustration
of the above region. 

\begin{figure}[ht]
\centering

\begin{tabular}{ccc}
\includegraphics[height=0.2\textwidth]{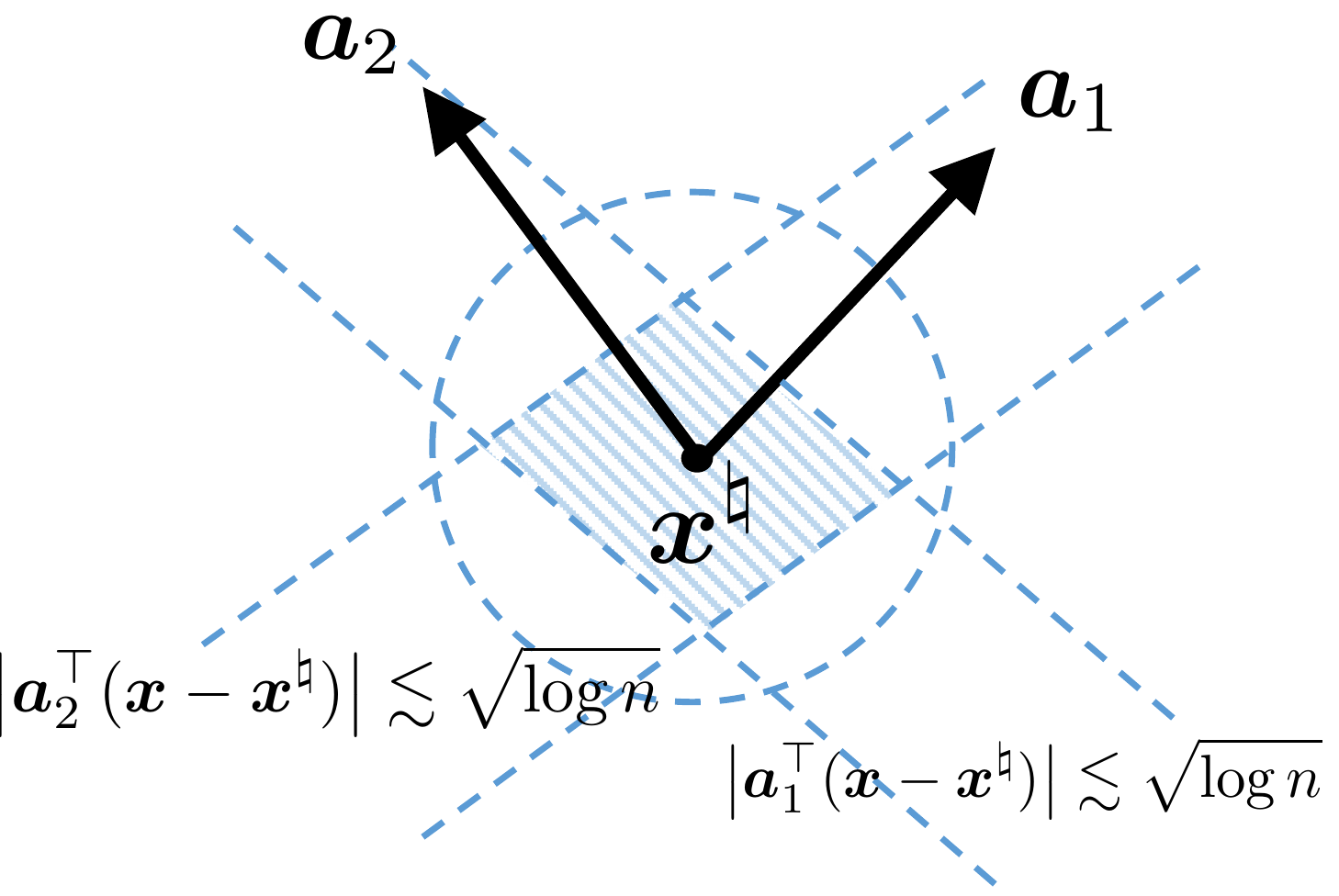} & \includegraphics[height=0.2\textwidth]{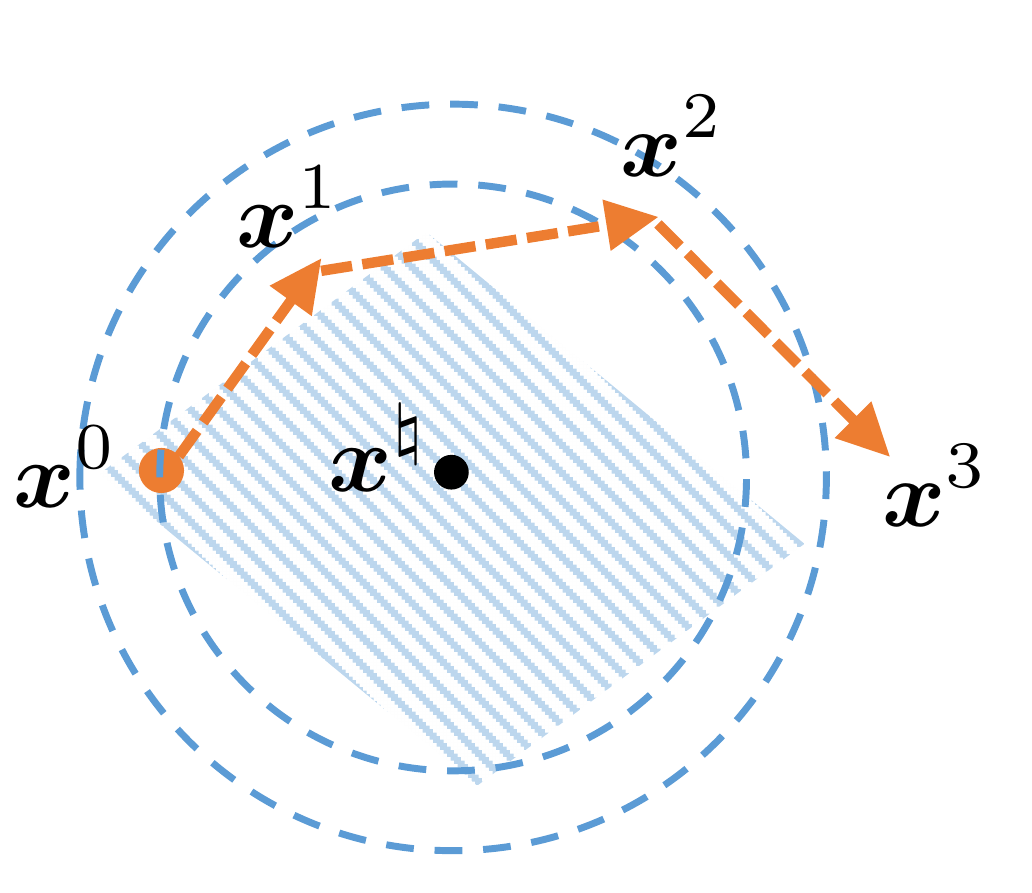} & \includegraphics[height=0.2\textwidth]{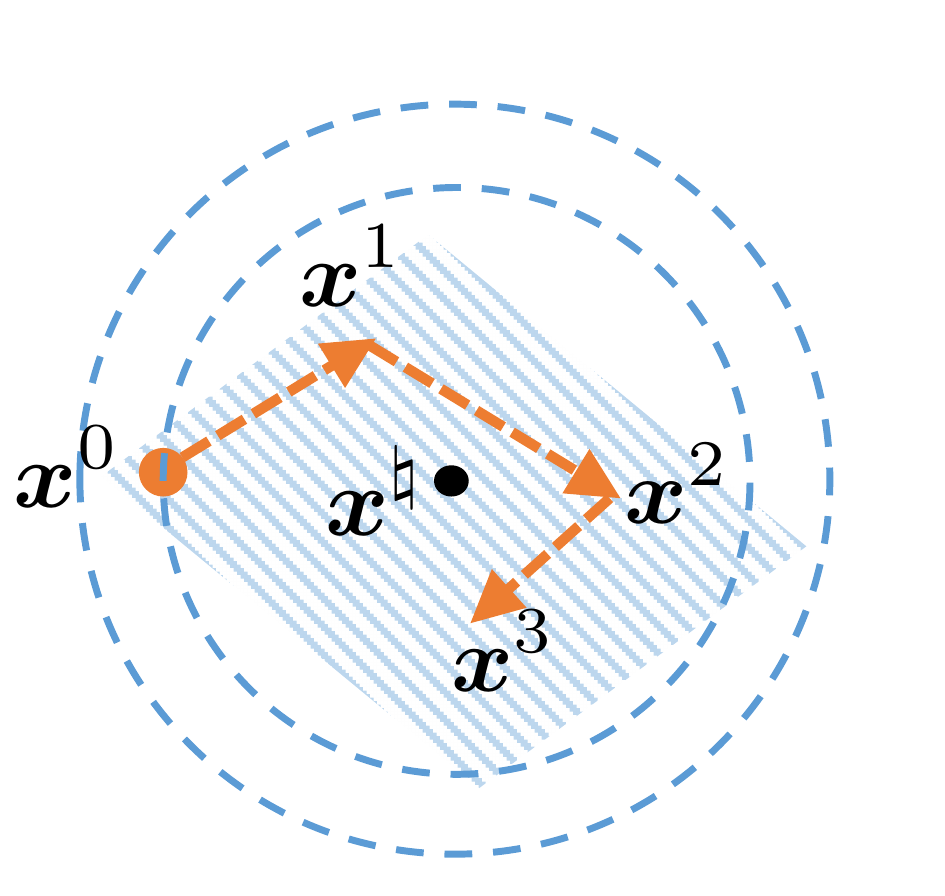}\tabularnewline
(a)  & (b)  & (c) \tabularnewline
\end{tabular}

\caption{(a) The shaded region is an illustration of the incoherence region, which
satisfies $\big|\bm{a}_{j}^{\top}(\bm{x}-\bm{x}^{\star})\big|\lesssim\sqrt{\log n}$
for all points $\bm{x}$ in the region. (b) When $\bm{x}^{0}$ resides
in the desired region, we know that $\bm{x}^{1}$ remains within the
$\ell_{2}$ ball but might fall out of the incoherence region (the
shaded region). Once $\bm{x}^{1}$ leaves the incoherence region,
we lose control and may overshoot. (c) Our theory reveals that with high
probability, all iterates will stay within the incoherence region,
enabling fast convergence. \label{fig:L2-incoherence-GD}}
\end{figure}

The following observation is thus immediate: one can safely adopt
a far more aggressive step size (as large as $\eta_{t}=O(1/\log n)$)
to achieve acceleration, as long as the iterates stay within the RIC. This, however, fails to be guaranteed by generic gradient
descent theory. To be more precise, if the current iterate $\bm{x}^{t}$
falls within the desired region, then in view of (\ref{eq:GD-classical-theory}),
we can ensure $\ell_{2}$ error contraction after one iteration, namely,
\[
\|\bm{x}^{t+1}-\bm{x}^{\star}\|_{2}\leq\|\bm{x}^{t}-\bm{x}^{\star}\|_{2}
\]
and hence $\bm{x}^{t+1}$ stays within the local $\ell_{2}$ ball and hence satisfies \eqref{eq:L2-condition-local}. However, it is not immediately obvious that $\bm{x}^{t+1}$ would still stay incoherent with the sensing vectors and satisfy \eqref{eq:incoherence-condition-local}. If
$\bm{x}^{t+1}$ leaves the RIC, it no longer enjoys the benign local geometry of the loss function, and the algorithm has to slow
down in order to avoid overshooting. See Figure \ref{fig:L2-incoherence-GD}(b)
for a visual illustration. In fact, in almost all regularized gradient descent algorithms
mentioned in Section~\ref{sec:intro_rgd}, one of the main purposes of the proposed regularization procedures is to enforce
such incoherence constraints. 

\subsection{Implicit regularization}

However, is regularization really necessary for the iterates to stay within the RIC? To answer this question, we plot in Figure \ref{fig:numerics-incoherence}(a) (resp.~Figure \ref{fig:numerics-incoherence}(b))
the incoherence measure $\frac{\max_{j}\left|\bm{a}_{j}^{\top}\bm{x}^{t}\right|}{\sqrt{\log n}\left\Vert \bm{x}^{\star}\right\Vert _{2}}$ (resp.~$\frac{\max_{j}\left|\bm{a}_{j}^{\top}(\bm{x}^{t}-\bm{x}^\star)\right|}{\sqrt{\log n}\left\Vert \bm{x}^{\star}\right\Vert _{2}}$)
vs.~the iteration count in a typical Monte Carlo trial, generated in the same way as for Figure \ref{fig:WF-stepsize}(a). Interestingly,  the incoherence measure remains bounded by
$2$ for all iterations $t>1$. This important observation suggests that
one may adopt a substantially more aggressive step size throughout the whole algorithm.
\begin{figure}[ht]
\centering

\begin{tabular}{cc}
\includegraphics[height=0.27\textwidth]{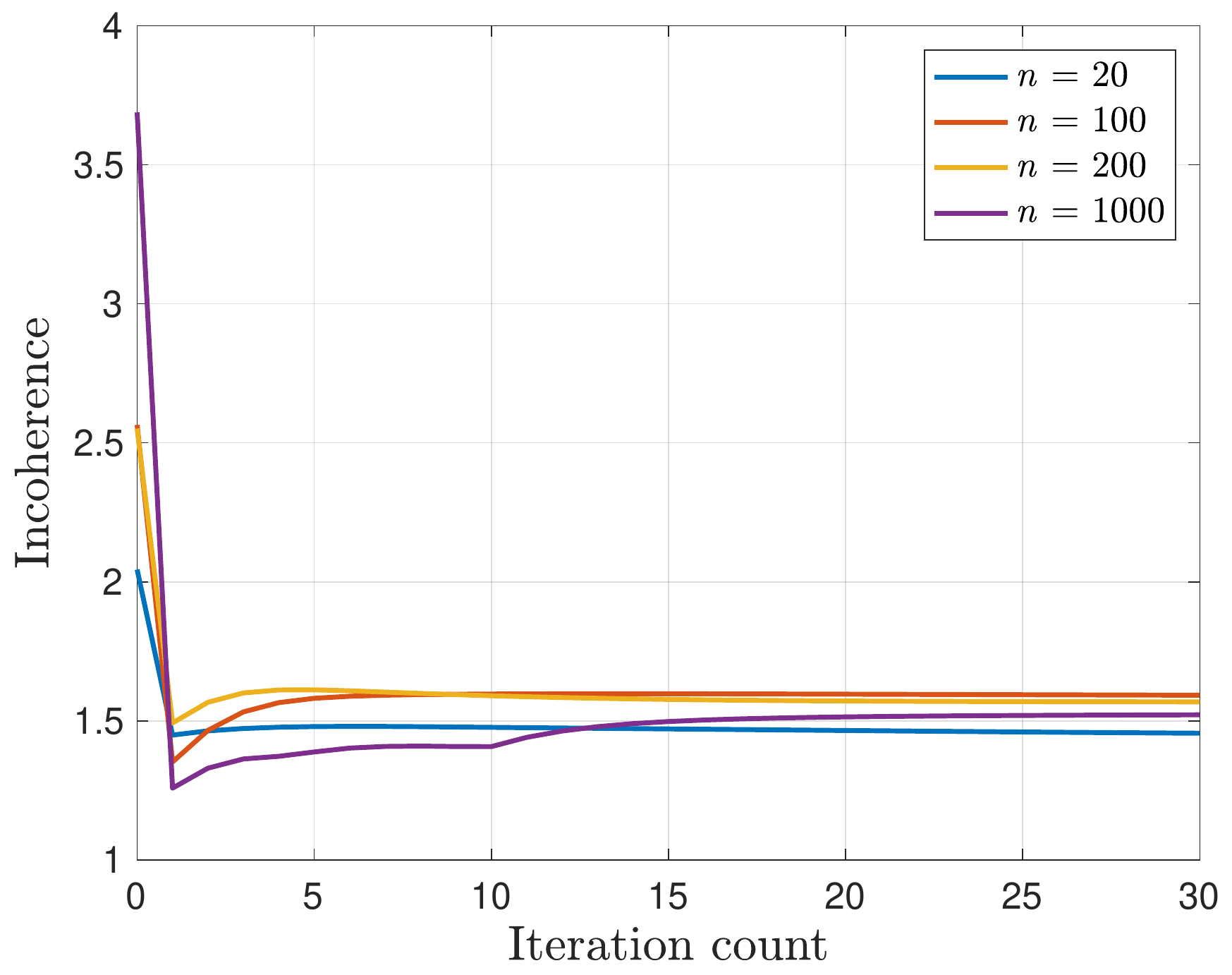}  & \includegraphics[height=0.27\textwidth]{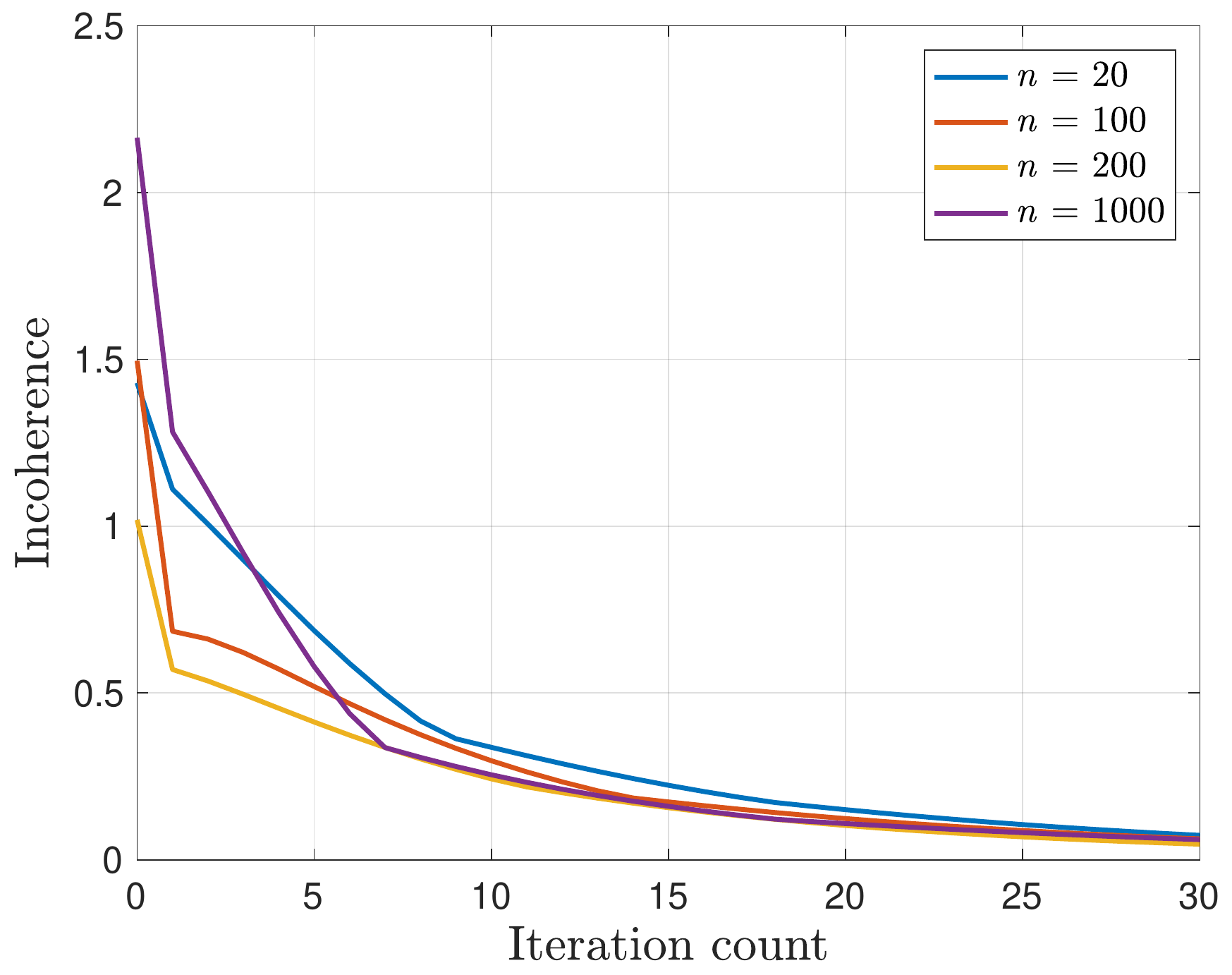}\tabularnewline
(a)  & (b) \tabularnewline
\end{tabular}

\caption{The incoherence measure $\frac{\max_{1\leq j\leq m}\left|\bm{a}_{j}^{\top}\bm{x}^{t}\right|}{\sqrt{\log n}\left\Vert \bm{x}^{\star}\right\Vert _{2}}$ (in (a)) and $\frac{\max_{1\leq j\leq m}\left|\bm{a}_{j}^{\top}(\bm{x}^{t}-\bm{x}^\star)\right|}{\sqrt{\log n}\left\Vert \bm{x}^{\star}\right\Vert _{2}}$ (in (b))
of the gradient iterates vs.~iteration count for the phase retrieval
problem. The results are shown for $n\in\left\{ 20,100,200,1000\right\} $
and $m=10n$, with the step size taken to be $\eta_{t}=0.1$.
The problem instances are generated in the same way as in Figure \ref{fig:WF-stepsize}(a).
\label{fig:numerics-incoherence}}
\end{figure} 

The main objective of this paper is thus to provide a theoretical validation of the above empirical observation. As we will demonstrate shortly, with high probability all iterates along the execution of the algorithm (as well as the spectral initialization) are provably constrained within the RIC, implying fast convergence of vanilla gradient descent (cf.~Figure \ref{fig:L2-incoherence-GD}(c)). The fact that the iterates stay incoherent with the measurement mechanism automatically, without explicit enforcement, is termed ``implicit regularization''.

\begin{figure}[t]
\centering
\begin{tabular}{ccc}
	\includegraphics[width=0.58\textwidth]{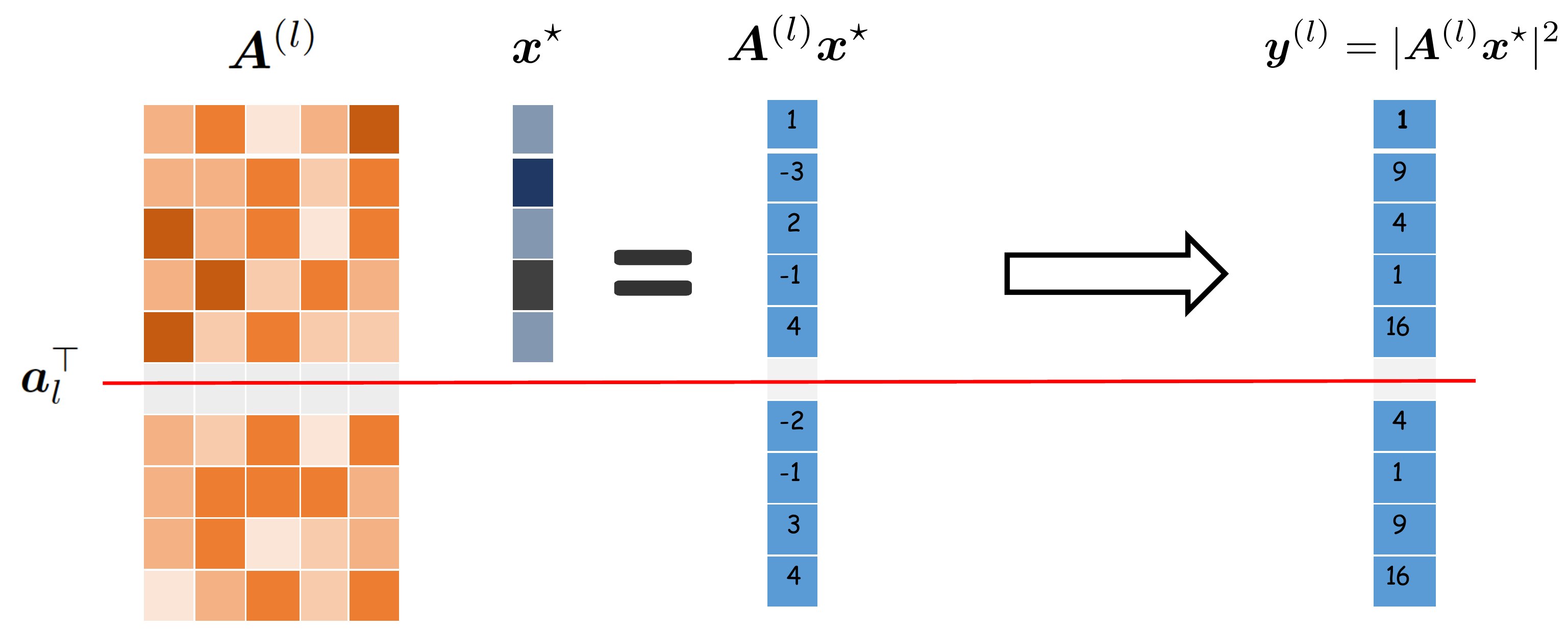} & & \includegraphics[width=0.32\textwidth]{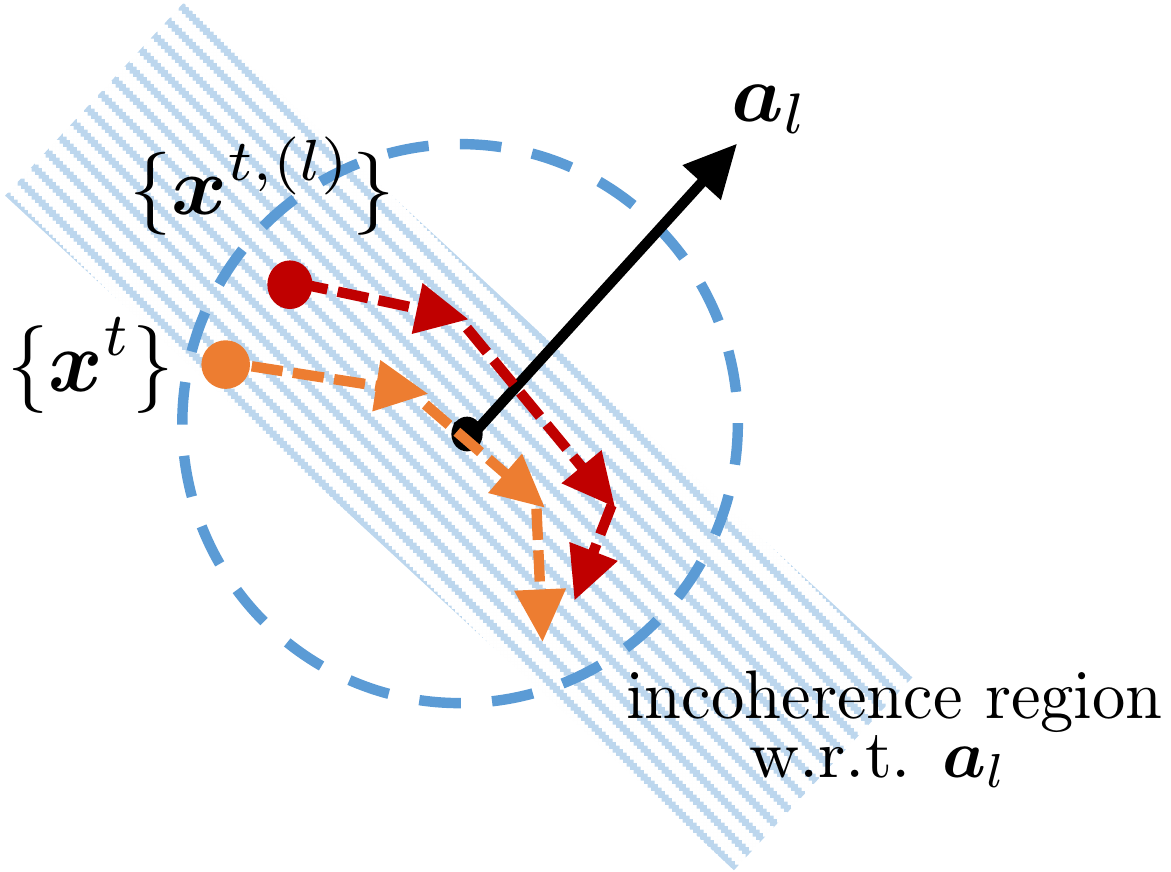}  \tabularnewline
	(a)&  & (b) \tabularnewline
\end{tabular}

	\caption{Illustration of the leave-one-out sequence w.r.t.~$\bm{a}_{l}$. (a) The sequence $\{\bm{x}^{t,(l)}\}_{t\geq 0}$ is constructed without using the $l$th sample. 
(b) Since the auxiliary sequence
$\{\bm{x}^{t,(l)}\}$ is constructed without using $\bm{a}_{l}$,
the leave-one-out iterates stay within the incoherence region w.r.t.~$\bm{a}_{l}$
with high probability. Meanwhile, $\{\bm{x}^{t}\}$ and $\{\bm{x}^{t,(l)}\}$
are expected to remain close as their construction differ only in
a single sample. \label{fig:LLO-sequence}}
\end{figure}

\subsection{A glimpse of the analysis: a leave-one-out trick}

In order to rigorously establish \eqref{eq:incoherence-condition-local} for all iterates, the current paper develops a powerful mechanism based on the leave-one-out perturbation argument, a trick rooted and widely used in probability and random matrix theory. Note that the iterate $\bm{x}^t$ is statistically dependent with the design vectors $\{\bm{a}_j\}$. Under such circumstances, one often resorts to generic bounds like the Cauchy-Schwarz inequality, which would not yield a desirable estimate. To address this issue, we introduce a sequence of auxiliary iterates $\{\bm{x}^{t,(l)}\}$ for each $1\leq l\leq m$ (for analytical purposes only), obtained by running vanilla gradient descent using all but the $l$th sample. 
As one can expect, such auxiliary trajectories serve as extremely good surrogates of  $\{\bm{x}^t\}$ in the sense that
\begin{align}
\bm{x}^{t} \,\approx\, \bm{x}^{t,\left(l\right)}, \qquad 1\leq l\leq m, \quad t\geq 0, \label{eq:induction-loop-WF-1}
\end{align}
 since their constructions only differ by a single sample. 
Most importantly, since $\bm{x}^{t,(l)}$ is independent with the $l$th design vector, it is much easier to control its incoherence w.r.t.~$\ba_l$ to the desired level:
\begin{align}  \label{eq:induction_LOO_incoherence}
\big|\bm{a}_{l}^{\top}\big(\bm{x}^{t,(l)}-\bm{x}^{\star}\big)\big| \lesssim \sqrt{\log n}\,\big\|\bm{x}^{\star}\big\|_{2}.
 \end{align}
Combining \eqref{eq:induction-loop-WF-1} and \eqref{eq:induction_LOO_incoherence} then leads to \eqref{eq:incoherence-condition-local}. See Figure \ref{fig:LLO-sequence} for a graphical illustration of this argument. Notably, this technique is very general and  applicable to many other problems. We invite the readers to Section \ref{sec:A-general-recipe} for more details.


\section{Main results\label{sec:main-results}}


This section formalizes the implicit regularization phenomenon underlying unregularized gradient descent, and  presents its consequences, namely near-optimal statistical and computational guarantees for phase retrieval, matrix completion, and blind deconvolution. Note that the discrepancy measure $\text{dist}\left(\cdot, \cdot\right)$ may vary from problem to problem.  

\subsection{Phase retrieval}

Suppose the $m$ quadratic equations
\begin{equation}\label{eq:pr_model}
y_{j}=\big(\bm{a}_{j}^{\top}\bm{x}^{\star}\big)^{2},\qquad j=1,2,\ldots,m
\end{equation}
are collected using random design vectors, namely, $\bm{a}_{j}\overset{\mathrm{i.i.d.}}{\sim}\mathcal{N}(\bm{0},\bm{I}_{n})$,
and the nonconvex problem to solve is
\begin{equation}
\text{minimize}_{\bm{x}\in\RR^{n}}\quad f(\bm{x}):=\frac{1}{4m}\sum_{j=1}^{m}\left[\left(\bm{a}_{j}^{\top}\bm{x}\right)^{2}-y_{j}\right]^{2}.\label{eq:loss-WF-PR}
\end{equation}
The Wirtinger flow (WF) algorithm, first introduced in \cite{candes2014wirtinger},
is a combination of spectral initialization and vanilla gradient descent;
see Algorithm \ref{alg:wf}. 

\begin{algorithm}
\caption{Wirtinger flow\,/\,gradient descent for phase retrieval}

\label{alg:wf}\begin{algorithmic}

\STATE \textbf{{Input}}: $\{\bm{a}_{j}\}_{1\leq j\leq m}$ and
$\{y_{j}\}_{1\leq j\leq m}$.

\STATE \textbf{{Spectral initialization}}: Let $\lambda_{1}\left(\bm{Y}\right)$
and $\tilde{\bm{x}}^{0}$ be the leading eigenvalue and eigenvector
of
\begin{equation}
\bm{Y}=\frac{1}{m}\sum_{j=1}^{m}y_{j}\bm{a}_{j}\bm{a}_{j}^{\top},\label{eq:init-WF}
\end{equation}
respectively, and set $\bm{x}^{0}=\sqrt{\lambda_{1}\left(\bm{Y}\right)/3}\;\tilde{\bm{x}}^{0}$.

\STATE \textbf{{Gradient updates}}: \textbf{for} $t=0,1,2,\ldots,T-1$
\textbf{do}

\STATE
\begin{equation}
\bm{x}^{t+1}=\bm{x}^{t}-\eta_{t}\nabla f\left(\bm{x}^{t}\right).
\label{eq:gradient_update-WF}
\end{equation}

\end{algorithmic} 
\end{algorithm}

Recognizing that the global phase\,/\,sign is unrecoverable from quadratic measurements,
we introduce the $\ell_{2}$ distance modulo the global phase as follows
\begin{equation}
\mathrm{dist}(\bm{x},\bm{x}^{\star}):=\min\left\{ \|\bm{x}-\bm{x}^{\star}\|_{2},\|\bm{x}+\bm{x}^{\star}\|_{2}\right\} .\label{eq:dist-WF}
\end{equation} 
Our finding is summarized in the following theorem.

\begin{theorem}
\label{thm:WF}
Let $\bm{x}^{\star}\in\mathbb{R}^{n}$
be a fixed vector. Suppose $\bm{a}_{j}\overset{\mathrm{i.i.d.}}{\sim}\mathcal{N}\left(\bm{0},\bm{I}_{n}\right)$
for each $1\leq j\leq m$ and $m\geq c_{0}n\log n$ for some sufficiently
large constant $c_{0}>0$. Assume the step size obeys $\eta_{t}\equiv\eta={c_{1}} / \left({\log n}\cdot\|\bm{x}_0\|_{2}^{2}\right)$
for any sufficiently small constant $c_{1}>0$. Then there exist some
absolute constants $0<\varepsilon<1$ and $c_{2}>0$ such that with
probability at least $1-O\left(mn^{-5}\right)$, Algorithm \ref{alg:wf} satisfies that for all $t\geq0$, 
\begin{subequations}
\label{eq:WF-guarantee}
\begin{align}
\mathrm{dist}(\bm{x}^{t},\bm{x}^{\star}) & \leq\varepsilon(1-\eta \| \bm{x}^{\star} \|_2^2 /2)^{t}\|\bm{x}^{\star}\|_{2},\label{eq:WF-theoretical-guarantee}\\
\max_{1\leq j\leq m}\big|\bm{a}_{j}^{\top}\big(\bm{x}^{t}-\bm{x}^{\star}\big)\big| & \leq c_{2}\sqrt{\log n}\|\bm{x}^{\star}\|_{2}.\label{eq:implicit-reg-WF}
\end{align}
\end{subequations}
\end{theorem}

Theorem \ref{thm:WF} reveals a few intriguing properties of Algorithm \ref{alg:wf}. 
\begin{itemize}
\item \textbf{Implicit regularization:} Theorem \ref{thm:WF} asserts that the incoherence properties are satisfied
throughout the execution of the algorithm (see \eqref{eq:implicit-reg-WF}),
which formally justifies the implicit regularization feature we hypothesized.
\item \textbf{Near-constant step size:} Consider the case where $\| \bm{x}^{\star} \|_2 =1$. Theorem \ref{thm:WF} establishes near-linear convergence
of WF with a substantially more aggressive step size $\eta\asymp 1/\log n$. Compared with the choice $\eta\lesssim 1/n$ admissible
in \cite[Theorem 3.3]{candes2014wirtinger}, Theorem \ref{thm:WF}
allows WF\,/\,GD to attain $\epsilon$-accuracy within $O(\log n\log(1/\epsilon))$ iterations. The
resulting computational complexity of the algorithm is 
\[
O\left(mn\log n\log\frac{1}{\epsilon}\right),
\]
which significantly improves upon the result $O\big(mn^{2}\log\left({1} / {\epsilon}\right)\big)$
		derived in \cite{candes2014wirtinger}. As a side note, if the sample size further increases to $m\asymp n\log^2 n$, then a constant step size $\eta \asymp 1$ is also feasible, resulting in an iteration complexity $\log(1/\epsilon)$. This follows since with high probability, the entire trajectory resides within a more refined incoherence region $\max_{j}\big|\bm{a}_{j}^{\top}\big(\bm{x}^{t}-\bm{x}^{\star}\big)\big| \lesssim \|\bm{x}^{\star}\|_{2}$. We omit the details here.  
\item \textbf{Incoherence of spectral initialization:} We have also demonstrated in Theorem \ref{thm:WF} that the initial guess $\bm{x}^{0}$ falls within the RIC and is hence nearly orthogonal to all design vectors. This provides a finer characterization of spectral initialization, in comparison to prior theory that focuses primarily on the $\ell_2$ accuracy \cite{netrapalli2013phase,candes2014wirtinger}. We expect our leave-one-out analysis to accommodate other variants of spectral initialization studied in the literature \cite{ChenCandes15solving,cai2016optimal,wang2017solving,lu2017phase,mondelli2017fundamental}.

\end{itemize}

\begin{remark}
	As it turns out, a carefully designed initialization is not pivotal in enabling fast convergence.  In fact, randomly initialized gradient descent  provably attains $\varepsilon$-accuracy in $O(\log n + \log \tfrac{1}{\varepsilon})$ iterations; see \cite{chen2018gradient} for details.  
\end{remark}

\subsection{Low-rank matrix completion\label{sec:main-mc}}

Let $\bm{M}^{\star}\in\mathbb{R}^{n\times n}$ be a positive semidefinite matrix\footnote{Here, we assume $\bm{M}^{\star}$ to be positive semidefinite to simplify the presentation, but note that our analysis easily extends to asymmetric low-rank matrices. } 
with rank $r$, and suppose its eigendecomposition is 
\begin{equation}\label{eq:M_svd}
\bm{M}^{\star}=\bm{U}^{\star}\bm{\Sigma}^{\star}\bm{U}^{\star\top},
\end{equation}
where $\bm{U}^{\star}\in\mathbb{R}^{n\times r}$ consists of orthonormal columns, and $\bm{\Sigma}^{\star}$
is an $r\times r$ diagonal matrix with eigenvalues in a descending
order, i.e.~$\sigma_{\max}=\sigma_{1}\geq\cdots\geq\sigma_{r}=\sigma_{\min}>0$.
Throughout this paper, we assume the condition number $\kappa:=\sigma_{\max}/\sigma_{\min}$
is bounded by a fixed constant, independent of the problem size (i.e.~$n$ and $r$). Denoting $\bm{X}^{\star}=\bm{U}^{\star}(\bm{\Sigma}^{\star})^{1/2}$
allows us to factorize $\bm{M}^{\star}$ as 
\begin{equation}
\bm{M}^{\star}=\bm{X}^{\star}\bm{X}^{\star\top}.\label{eq:M-X-MC}
\end{equation}
Consider a random sampling model such that each entry of $\bm{M}^{\star}$
is observed independently with probability $0<p\leq 1$, i.e.~for $1\leq j\leq k\leq n$,
\begin{equation}
Y_{j,k}=\begin{cases}
M_{j,k}^{\star}+E_{j,k},\quad & \text{with probability }p,\\
0, & \text{else},
\end{cases}\label{eq:sampling-model-MC}
\end{equation}
where the entries of $\bm{E}=[ E_{j,k} ] _{1\leq j\leq k\leq n}$ are independent
sub-Gaussian noise with sub-Gaussian norm $\sigma$ (see \cite[Definition 5.7]{Vershynin2012}). We denote by $\Omega$ the set of locations being sampled, and $\mathcal{P}_\Omega(\bY)$ represents the projection of $\bm{Y}$ onto the set of matrices supported in $\Omega$. We note here that the sampling rate $p$, if not known, can be faithfully estimated by the sample proportion $| \Omega |/n^2$.

To fix ideas, we consider the following nonconvex optimization problem
\begin{equation}
\text{minimize}_{\bm{X}\in\RR^{n\times r}}\quad f\left(\bm{X}\right):=\frac{1}{4p}\sum_{(j,k)\in\Omega}\big(\bm{e}_{j}^{\top}\bm{X}\bm{X}^{\top}\bm{e}_{k}-Y_{j,k}\big)^{2}.\label{eq:minimization-MC}
\end{equation}
The vanilla gradient descent algorithm (with spectral initialization)
is summarized in Algorithm \ref{alg:gd-mc}.

\begin{algorithm}
\caption{Vanilla gradient descent for matrix completion (with spectral initialization)}

\label{alg:gd-mc}\begin{algorithmic}

\STATE \textbf{{Input}}: $\bm{Y}=\left[Y_{j,k}\right]_{1\leq j,k\leq n}$, $r$, $p$.

\STATE \textbf{{Spectral initialization}}: Let $\bm{U}^{0}\bm{\Sigma}^{0}\bm{U}^{0\top}$
be the rank-$r$ eigendecomposition of
\[
\bm{M}^{0}:=\frac{1}{p} \mathcal{P}_{\Omega}(\bY)=\frac{1}{p}\mathcal{P}_{\Omega}\left(\bm{M}^{\star}+\bm{E}\right),
\]
and set $\bm{X}^{0}=\bm{U}^{0}\left(\bm{\Sigma}^{0}\right)^{1/2}$.

\STATE \textbf{{Gradient updates}}: \textbf{for} $t=0,1,2,\ldots,T-1$
\textbf{do}

\STATE 
\begin{equation}
\bm{X}^{t+1}=\bm{X}^{t}-\eta_{t}\nabla f\left(\bm{X}^{t}\right).\label{eq:gradient_update-MC}
\end{equation}

\end{algorithmic} 
\end{algorithm}
Before proceeding to the main theorem, we first introduce a standard incoherence
parameter required for matrix completion \cite{ExactMC09}.

\begin{definition}[Incoherence for matrix completion]\label{def:mc-incoherence}A rank-$r$ matrix $\bm{M}^{\star}$
with eigendecomposition $\bm{M}^{\star}=\bm{U}^{\star}\bm{\Sigma}^{\star}\bm{U}^{\star\top}$
is said to be $\mu$-incoherent if 
\begin{equation}
\left\Vert \bm{U}^{\star}\right\Vert _{2,\infty}\leq\sqrt{\frac{\mu}{n}}\left\Vert \bm{U}^{\star}\right\Vert _{\mathrm{F}}=\sqrt{\frac{\mu r}{n}}. 
	\label{eq:incoherence-U-MC}
\end{equation}
\end{definition}In addition, recognizing that $\bm{X}^{\star}$
is identifiable only up to orthogonal transformation, we define the optimal transform from the $t$th iterate $\bm{X}^t$ to $\bm{X}^{\star}$ as
\begin{equation}
\hat{\bm{H}}^{t}:=\argmin_{\bm{R}\in\cO^{r\times r}}\left\Vert \bm{X}^{t}\bm{R}-\bm{X}^{\star}\right\Vert _{\mathrm{F}},\label{eq:rotation-hat-h-MC}
\end{equation}
where $\mathcal{O}^{r\times r}$ is the set of $r\times r$ orthonormal
matrices. With these definitions in place, we have the following theorem.


\begin{theorem}
\label{thm:main-MC}
Let $\bm{M}^{\star}$
be a rank $r$, $\mu$-incoherent PSD matrix, and its condition number $\kappa$ is a fixed constant. Suppose
the sample size satisfies $n^{2}p\geq C\mu^{3}r^{3}n\log^{3}n$ for
some sufficiently large constant $C>0$, and the noise satisfies 
\begin{equation}
\sigma\sqrt{\frac{n}{p}}\ll\frac{\sigma_{\min}}{\sqrt{\kappa^3 \mu r\log^{3}n}}.\label{eq:mc-noise-condition}
\end{equation}
With probability at least $1-O\left(n^{-3}\right)$, the iterates
of Algorithm \ref{alg:gd-mc} satisfy \begin{subequations} \label{eq:induction_original_MC_thm}
\begin{align}
\big\|\bm{X}^{t}\hat{\bm{H}^{t}}-\bm{X}^{\star}\big\|_{\mathrm{F}} & \leq\left(C_{4}\rho^{t}\mu r\frac{1}{\sqrt{np}}+C_{1}\frac{\sigma}{\sigma_{\min}}\sqrt{\frac{n}{p}}\right)\big\|\bm{X}^{\star}\big\|_{\mathrm{F}},\label{eq:induction_original_ell_2-MC_thm}\\
\big\|\bm{X}^{t}\hat{\bm{H}^{t}}-\bm{X}^{\star}\big\|_{2,\infty} & \leq\left(C_{5}\rho^{t}\mu r\sqrt{\frac{\log n}{np}}+C_{8}\frac{\sigma}{\sigma_{\min}}\sqrt{\frac{n\log n}{p}}\right)\big\|\bm{X}^{\star}\big\|_{2,\infty},\label{eq:induction_original_ell_infty-MC_thm}\\
\big\|\bm{X}^{t}\hat{\bm{H}^{t}}-\bm{X}^{\star}\big\| & \leq\left(C_{9}\rho^{t}\mu r\frac{1}{\sqrt{np}}+C_{10}\frac{\sigma}{\sigma_{\min}}\sqrt{\frac{n}{p}}\right)\big\|\bm{X}^{\star}\big\|\label{eq:induction_original_operator-MC_thm}
\end{align}
\end{subequations}for all $0\leq t\leq T=O(n^{5})$, where $C_{1}$,
$C_{4}$, $C_{5}$, $C_{8}$, $C_{9}$ and $C_{10}$ are some absolute
positive constants and $1-\left({\sigma_{\min}} / {5}\right)\cdot\eta \leq \rho <1$, provided
that $0<\eta_{t}\equiv\eta\leq{2} / \left({25\kappa\sigma_{\max}}\right)$. \end{theorem}

Theorem \ref{thm:main-MC} provides the first theoretical guarantee of unregularized gradient descent for matrix completion, demonstrating near-optimal statistical accuracy and computational complexity. 

\begin{itemize}
\item \textbf{Implicit regularization:} In Theorem \ref{thm:main-MC}, we bound the $\ell_{2}/\ell_{\infty}$ error of the iterates in a uniform manner via \eqref{eq:induction_original_ell_infty-MC_thm}. Note that  $\big\|\bm{X}-\bm{X}^{\star}\big\|_{2,\infty}=\max_{j}\big\|\bm{e}_{j}^{\top}\big(\bm{X}-\bm{X}^{\star}\big)\big\|_{2}$,
which implies the iterates remain incoherent with the sensing vectors throughout and have small incoherence parameters (cf.~\eqref{eq:incoherence-U-MC}). 
		In comparison, prior works either include a penalty term on $\{\|\bm{e}_{j}^{\top}\bm{X}\|_2\}_{1\leq j\leq n}$ \cite{KesMonSew2010,sun2016guaranteed} and/or  $\|\bm{X}\|_{\mathrm{F}}$ \cite{sun2016guaranteed} to encourage an incoherent and/or low-norm solution, or add an extra  projection operation to enforce incoherence \cite{chen2015fast,zheng2016convergence}. 		
		Our results demonstrate that such explicit regularization  is unnecessary. 

\item \textbf{Constant step size:}
Without loss of generality we may assume that $\sigma_{\max} = \| \bm{M}^{\star} \| = O( 1 )$, which can be done by choosing proper scaling of $\bm{M}^{\star}$. Hence we have a constant step size $\eta_t  \asymp 1$. 
Actually it is more convenient to consider the scale invariant parameter $\rho$: Theorem \ref{thm:main-MC} guarantees linear convergence of the vanilla gradient descent at a constant rate $\rho$. Remarkably,
the convergence occurs with respect to three different unitarily invariant
norms: the Frobenius norm $\|\cdot\|_{\mathrm{F}}$, the $\ell_{2}/\ell_{\infty}$
norm $\|\cdot\|_{2,\infty}$, and the spectral norm $\|\cdot\|$. As far as we know,  the latter two are established for the first time.
Note that our result even improves upon that for regularized gradient descent; see Table~\ref{tab:Performance-guarantees-GD}.

\item \textbf{Near-optimal sample complexity:} When the rank $r=O(1)$, vanilla
gradient descent succeeds under a near-optimal sample complexity $n^{2}p\gtrsim n\mathrm{poly}\log n$,
which is statistically optimal up to some logarithmic factor. 
\item \textbf{Near-minimal Euclidean error:} In view of \eqref{eq:induction_original_ell_2-MC_thm}, as $t$ increases, the Euclidean error of vanilla GD converges to
\begin{equation}
\big\|\bm{X}^{t}\hat{\bm{H}}^{t}-\bm{X}^{\star}\big\|_{\mathrm{F}}\lesssim\frac{\sigma}{\sigma_{\min}}\sqrt{\frac{n}{p}}\big\|\bm{X}^{\star}\big\|_{\mathrm{F}}, 
	\label{eq:Euclidean-error-MC_thm}
\end{equation}
which coincides with the theoretical guarantee in \cite[Corollary 1]{chen2015fast} and matches the minimax lower bound established in \cite{Negahban2012rscMC, MR2906869}.

\item \textbf{Near-optimal entrywise error:} The $\ell_{2}/\ell_{\infty}$ error
bound (\ref{eq:induction_original_ell_infty-MC_thm}) immediately yields
		entrywise control of the empirical risk. Specifically, as soon as $t$ is sufficiently large (so that the first term in \eqref{eq:induction_original_ell_infty-MC_thm} is negligible), we have
\begin{align*}
\big\|\bm{X}^{t}\bm{X}^{t\top}-\bm{M}^{\star}\big\|_{\infty} & \leq\big\|\bm{X}^{t}\hat{\bm{H}}^{t}\big(\bm{X}^{t}\hat{\bm{H}}^{t}-\bm{X}^{\star}\big)^{\top}\big\|_{\infty}+\big\|\big(\bm{X}^{t}\hat{\bm{H}}^{t}-\bm{X}^{\star}\big)\bm{X}^{\star\top}\big\|_{\infty}\\
 & \leq \big\Vert  \bm{X}^{t}\hat{\bm{H}}^{t}\ \big\Vert _{2,\infty}\big\|\bm{X}^{t}\hat{\bm{H}}^{t}-\bm{X}^{\star}\big\|_{2,\infty}+\big\|\bm{X}^{t}\hat{\bm{H}}^{t}-\bm{X}^{\star}\big\|_{2,\infty}\big\| \bm{X}^{\star}\big\|_{2,\infty} \\
 & \lesssim\frac{\sigma}{\sigma_{\min}}\sqrt{\frac{n\log n}{p}}\left\Vert \bm{M}^{\star}\right\Vert _{\infty},
\end{align*}
where the last line follows from (\ref{eq:induction_original_ell_infty-MC_thm})
as well as the facts that $\|\bm{X}^{t}\hat{\bm{H}}^{t}-\bm{X}^{\star}\|_{2,\infty}\leq\| \bm{X}^{\star}\| _{2,\infty}$
and $\| \bm{M}^{\star}\| _{\infty}=\| \bm{X}^{\star}\| _{2,\infty}^{2}$.
Compared with the Euclidean loss (\ref{eq:Euclidean-error-MC_thm}), this
implies that when $r=O(1)$, the entrywise error of $\bm{X}^{t}\bm{X}^{t\top}$
is uniformly spread out across all entries.
As far as we know, this is the first result that reveals near-optimal
entrywise error control for noisy matrix completion using nonconvex optimization, without resorting to sample splitting.

\end{itemize}
\begin{remark}
Theorem \ref{thm:main-MC} remains valid if the total number $T$ of iterations
obeys $T=n^{O(1)}$. In the noiseless case where $\sigma = 0$, the
theory allows arbitrarily large $T$. \end{remark}

Finally, we report the empirical statistical accuracy of vanilla gradient descent in the presence of noise. 
Figure \ref{fig:mc-noisy} displays the squared relative error of vanilla gradient descent as a
function of the signal-to-noise ratio (SNR), where the SNR is defined to be
\begin{equation}\label{eq:SNR}
	\mathsf{SNR} \,:=\, \frac{\sum_{(j,k)\in \Omega} \big(M^{\star}_{j,k}\big)^2}{\sum_{(j,k)\in \Omega} \Var\left(E_{j,k}\right)} \,\approx\, \frac{ \|\bm{M}^{\star}\|_{\mathrm{F}}^2}{n^2 \sigma^2} ,
\end{equation}
and the relative error is measured in terms of the square of the metrics as in \eqref{eq:induction_original_MC_thm} as well as the squared entrywise prediction error. 
Both the relative error and the SNR are shown on a dB scale (i.e.~$10\log_{10}(\text{SNR})$ and $10\log_{10}(\text{squared relative error})$ are plotted). The results are averaged over 20 independent trials. As one can see from the plot, the squared relative error scales inversely proportional to the SNR, which is consistent with our theory.\footnote{Note that when $\bm{M}^{\star}$ is well-conditioned and when $r=O(1)$, one can easily check that 
$\mathsf{SNR}\approx\left({\|\bm{M}^{\star}\|_{\mathrm{F}}^{2}}\right) / \left({n^{2}\sigma^{2}}\right)\asymp{\sigma_{\min}^{2}} / ({n^{2}\sigma^{2}})$, and our theory says that the squared relative error bound is  proportional to $\sigma^2 / \sigma_{\min}^2$.}

\begin{figure}[t]
\centering

\includegraphics[width=0.4\textwidth]{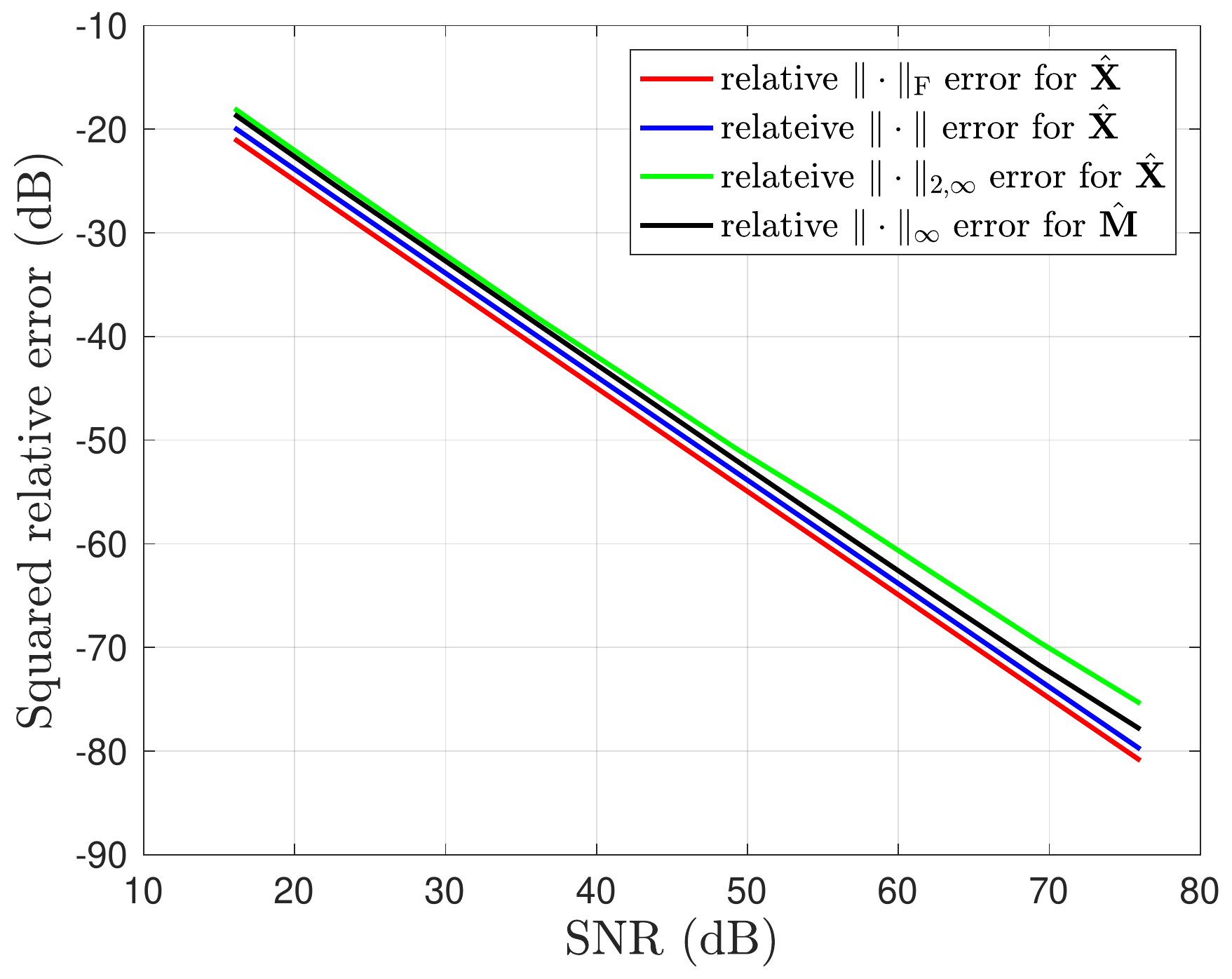}

	\caption{Squared relative error of the estimate $\hat{\bm{X}}$ (measured by $\left\Vert \cdot\right\Vert _{\mathrm{F}},\left\Vert \cdot\right\Vert ,\left\Vert \cdot\right\Vert _{2,\infty}$ modulo global transformation) and $\hat{\bm{M}}=\hat{\bm{X}}\hat{\bm{X}}^{\top}$ (measured by $\left\Vert \cdot\right\Vert _{\infty}$) vs.~SNR for noisy matrix completion, where $n=500$, $r=10$,
$p=0.1$, and $\eta_{t}=0.2$. Here $\hat{\bm{X}}$ denotes the estimate returned by Algorithm \ref{alg:gd-mc} after convergence. The results are averaged over 20 independent Monte Carlo trials. \label{fig:mc-noisy}}
\end{figure}


\subsection{Blind deconvolution}
\label{sec:main-blind-deconvolution}

Suppose we have collected $m$ bilinear measurements 
\begin{equation}\label{eq:bd_model}
y_{j}=\bm{b}_{j}^{\conj}\bm{h}^{\star}\bm{x}^{\star\conj}\bm{a}_{j},\qquad1\leq j\leq m,
\end{equation}
where $\bm{a}_{j}$ follows a complex Gaussian distribution, i.e.~$\bm{a}_{j}\overset{\text{i.i.d.}}{\sim}\mathcal{N}\left(\bm{0},\frac{1}{2}\bm{I}_{K}\right)+i\mathcal{N}\left(\bm{0},\frac{1}{2}\bm{I}_{K}\right)$ for $1\leq j \leq m$,
and $\bm{B} :=\left[\bm{b}_1,\cdots,\bm{b}_m\right]^\conj \in\mathbb{C}^{m\times K}$ is formed by the first $K$ columns of a unitary discrete Fourier
transform (DFT) matrix $\bm{F}\in\mathbb{C}^{m\times m}$ obeying
$\bm{F}\bm{F}^{\conj}=\bm{I}_m$ (see Appendix~\ref{sec:fourier} for a brief introduction to DFT matrices). This setup models blind deconvolution, where the two signals under convolution belong to known low-dimensional subspaces of dimension $K$ \cite{ahmed2014blind}\footnote{For simplicity, we have set the dimensions of the two subspaces equal, and it is straightforward to extend our results to the case of unequal subspace dimensions.}. In particular, the partial DFT matrix $\bm{B}$ plays an important role in image blind deblurring. In this subsection, we consider solving the following
nonconvex optimization problem
\begin{equation}
\text{minimize}_{\bm{h},\bm{x}\in\mathbb{C}^{K}}\quad f\left(\bm{h},\bm{x}\right)=\sum_{j=1}^{m}\left|\bm{b}_{j}^{\conj} \bm{h}\bm{x}^{\conj}\bm{a}_{j} -y_j \right|^{2}.\label{eq:loss-BD}
\end{equation}
The  (Wirtinger) gradient descent algorithm (with spectral initialization) is summarized in Algorithm
\ref{alg:gd-BD}; here, $\nabla_{\bm{h}} f(\bm{h},\bm{x})$ and $\nabla_{\bm{x}} f(\bm{h},\bm{x})$  stand for the Wirtinger
gradient and are given in (\ref{eq:grad-h-BD}) and (\ref{eq:grad-x-BD}), respectively; see \cite[Section 6]{candes2014wirtinger} for a brief introduction to Wirtinger calculus. 

It is self-evident that $\bm{h}^{\star}$ and $\bm{x}^{\star}$
are only identifiable up to global scaling, that is, for any nonzero
$\alpha\in\CC$, 
\[
\bm{h}^{\star}\bm{x}^{\star\conj}=\frac{1}{\overline{\alpha}}\bm{h}^{\star}\left(\alpha\bm{x}^{\star}\right)^{\conj}.
\]
In light of this, we will measure the discrepancy between 
\begin{equation}\label{eq:def_zz}
\bm{z}:=\begin{bmatrix}
\bm{h} \\
\bm{x}
\end{bmatrix}\in\mathbb{C}^{2K}\qquad \text{and} \qquad\bm{z}^{\star}:=\begin{bmatrix}
\bm{h}^{\star} \\
\bm{x}^{\star}
\end{bmatrix}\in\mathbb{C}^{2K}
\end{equation}
via the following function 
\begin{equation}
\mathrm{dist}\left(\bm{z},\bm{z}^{\star}\right):= \min_{\alpha\in\mathbb{C}} \sqrt{\left\Vert \frac{1}{\overline{\alpha}}\bm{h}-\bm{h}^{\star}\right\Vert _{2}^{2}+\left\Vert \alpha\bm{x}-\bm{x}^{\star}\right\Vert _{2}^{2}}.\label{eq:defn-dist-BD}
\end{equation}
\begin{algorithm}
\caption{Vanilla gradient descent for blind deconvolution (with spectral initialization)}

\label{alg:gd-BD}\begin{algorithmic}

\STATE \textbf{{Input}}: $\left\{ \bm{a}_{j}\right\} _{1\leq j\leq m},\left\{ \bm{b}_{j}\right\} _{1\leq j\leq m}$
and $\left\{ y_{j}\right\} _{1\leq j\leq m}$.

\STATE \textbf{{Spectral initialization}}: Let $\sigma_{1}(\bm{M})$,
$\check{\bm{h}}^{0}$ and $\check{\bm{x}}^{0}$  be the leading
singular value, left and right singular vectors of 
\[
\bm{M}:=\sum_{j=1}^{m}y_{j}\bm{b}_{j}\bm{a}_{j}^{\conj},
\]
respectively. Set $\bm{h}^{0}=\sqrt{\sigma_{1}(\bm{M})}\;\check{\bm{h}}^{0}$ and
 $\bm{x}^{0}=\sqrt{\sigma_{1}(\bm{M})}\;\check{\bm{x}}^{0}$. 

\STATE \textbf{{Gradient updates}}: \textbf{for} $t=0,1,2,\ldots,T-1$ 
\textbf{do}

\STATE 
	\begin{equation}
\left[\begin{array}{c}
\bm{h}^{t+1}\\
\bm{x}^{t+1}
\end{array}\right]=\left[\begin{array}{c}
\bm{h}^{t}\\
\bm{x}^{t}
\end{array}\right]-\eta\left[\begin{array}{c}
\frac{1}{\|\bm{x}^{t}\|_2^{2}}\nabla_{\bm{h}}f\big(\bm{h}^{t},\bm{x}^{t}\big)\\
\frac{1}{\|\bm{h}^{t}\|_2^{2}}\nabla_{\bm{x}}f\big(\bm{h}^{t},\bm{x}^{t}\big)
\end{array}\right] .
\label{eq:gradient_update_BD}
	\end{equation}
\end{algorithmic} 
\end{algorithm}

Before proceeding, we need to introduce the incoherence
parameter \cite{ahmed2014blind,DBLP:journals/corr/LiLSW16}, which is crucial for blind deconvolution, whose role is similar to the incoherence parameter (cf.~Definition~\ref{def:mc-incoherence}) in matrix completion.

\begin{definition}[Incoherence for blind deconvolution]Let the incoherence parameter $\mu$ of $\bh^{\star}$ be the smallest
number such that 
\begin{equation}
\max_{1\leq j\leq m}\left|\bm{b}_{j}^{\conj}\bm{h}^{\star}\right|\leq\frac{\mu}{\sqrt{m}}\left\Vert \bm{h}^{\star}\right\Vert _{2}.\label{eq:incoherence-BD}
\end{equation}
\end{definition}
The incoherence parameter describes the spectral flatness of the signal $\bm{h}^{\star}$. With this definition in place, we have the following theorem, where for identifiability we assume that $\left\Vert\bm{h}^{\star}\right\Vert_{2}=\left\Vert\bm{x}^{\star}\right\Vert_{2}$.

\begin{theorem}\label{thm:main-BD}
Suppose the number of measurements
obeys $m\geq C\mu^{2}K\log^{9}m$ for some sufficiently large constant
$C>0$, and suppose the step size $\eta>0$ is taken to be some sufficiently
small constant. Then there exist constants $c_{1},c_{2},C_{1},C_{3},C_{4}>0$ such that with probability exceeding $1-c_{1}m{}^{-5}-c_{1}me^{-c_{2}K}$,
the iterates in Algorithm \ref{alg:gd-BD} satisfy \begin{subequations} \label{eq:BD_thm}
\begin{align}
\mathrm{dist}\left(\bm{z}^{t},\bm{z}^{\star}\right) & \leq C_{1}\left(1-\frac{\eta}{16}\right)^{t}\frac{1}{\log^{2}m}\left\Vert\bm{z}^{\star}\right\Vert_{2}\label{eq:BD-thm-ell-2},\\
\max_{1\leq l\leq m}\left|\bm{a}_{l}^{\conj}\left(\alpha^{t}\bm{x}^{t}-\bm{x}^{\star}\right)\right| & \leq C_{3}\frac{1}{\log^{1.5}m}\left\Vert\bm{x}^{\star}\right\Vert_{2}\label{eq:BD-thm-incoherence-a},\\
\max_{1\leq l\leq m}\left|\bm{b}_{l}^{\conj}\frac{1}{\overline{\alpha^{t}}}\bm{h}^{t}\right| & \leq C_{4}\frac{\mu}{\sqrt{m}}\log^{2}m\left\Vert\bm{h}^{\star}\right\Vert_{2}\label{eq:BD-thm-incoherence-b}
\end{align}
\end{subequations}
for all $t\geq0$. Here, we denote $\alpha^t$ as the {\em alignment parameter},
\begin{equation}
\alpha^{t}:=\arg\min_{\alpha\in\CC}\left\Vert \frac{1}{\overline{\alpha}}\bm{h}^{t}-\bm{h}^{\star}\right\Vert _{2}^{2}+\left\Vert \alpha\bm{x}^{t}-\bm{x}^{\star}\right\Vert _{2}^{2}.\label{eq:defn-alphat}
\end{equation}
\end{theorem}

Theorem \ref{thm:main-BD} provides the first theoretical guarantee of unregularized gradient descent for blind deconvolution at a near-optimal statistical and computational complexity. A few remarks are in order. 
\begin{itemize}
	\item \textbf{Implicit regularization:} Theorem \ref{thm:main-BD} reveals that the unregularized gradient descent iterates remain incoherent with the sampling mechanism (see (\ref{eq:BD-thm-incoherence-a}) and (\ref{eq:BD-thm-incoherence-b})). Recall that prior works operate upon a regularized cost function with an additional penalty term that regularizes the global scaling $\{\|\bm{h}\|_2,\|\bm{x}\|_2\}$ and the incoherence $\{|\bm{b}_j^\conj \bm{h}|\}_{1\leq j\leq m}$ 
\cite{DBLP:journals/corr/LiLSW16,huang2017blind,ling2017regularized}. 
In comparison, our theorem implies that it is unnecessary to regularize either the incoherence or the scaling ambiguity, which is somewhat surprising. This justifies the use of regularization-free (Wirtinger) gradient descent for blind deconvolution. 
\item \textbf{Constant step size:} Compared to the step size $\eta_t \lesssim 1/m$ suggested in \cite{DBLP:journals/corr/LiLSW16} for regularized gradient descent, our theory admits a substantially more aggressive step size (i.e.~$\eta_t\asymp 1$) even without regularization. Similar to phase retrieval, the computational efficiency is boosted by a factor of $m$, attaining $\epsilon$-accuracy within $O\left(\log(1/\epsilon)\right)$ iterations (vs.~$O\left(m\log(1/\epsilon)\right)$ iterations in prior theory). 
\item \textbf{Near-optimal sample complexity:} It is demonstrated that vanilla gradient descent succeeds at a near-optimal sample complexity up to logarithmic factors, although our requirement is slightly worse than \cite{DBLP:journals/corr/LiLSW16} which uses explicit regularization. Notably, even under the sample complexity herein, the iteration complexity given in \cite{DBLP:journals/corr/LiLSW16} is still $O\left(m/\mathrm{poly}\log(m)\right)$. 
\item \textbf{Incoherence of spectral initialization:} As in phase retrieval, Theorem~\ref{thm:main-BD} demonstrates that the estimates returned by the spectral method are incoherent with respect to both $\{\bm{a}_j\}$ and $\{\bm{b}_j\}$. In contrast, \cite{DBLP:journals/corr/LiLSW16} recommends a projection operation (via a linear program) to enforce incoherence of the initial estimates, which is dispensable according to our theory.  
\item \textbf{Contraction in $\left\Vert\cdot\right\Vert_{\mathrm{F}}$:} It is easy to check that the Frobenius norm error satisfies $\left\Vert \bm{h}^{t}\bm{x}^{t\conj}-\bm{h}^{\star}\bm{x}^{\star\conj}\right\Vert _{\mathrm{F}}\lesssim \mathrm{dist}\left(\bm{z}^{t},\bm{z}^{\star}\right)$, and therefore Theorem~\ref{thm:main-BD} corroborates the empirical results shown in Figure \ref{fig:WF-stepsize}(c).

\end{itemize}


\section{Related work}

Solving nonlinear systems of equations has received much attention
in the past decade. Rather than directly attacking the nonconvex formulation, convex relaxation  lifts the object of interest into a higher dimensional space and then attempts recovery via semidefinite programming (e.g.~\cite{RecFazPar07,candes2012phaselift,ExactMC09,ahmed2014blind}). This has enjoyed great success in both theory and practice. Despite appealing statistical guarantees, semidefinite programming is in general prohibitively expensive when processing large-scale datasets.

Nonconvex approaches, on the other end, have been under extensive study in the last few years, due to their computational advantages. There is a growing list of statistical estimation problems for which nonconvex approaches are guaranteed to find global optimal solutions, including but not limited to phase retrieval \cite{netrapalli2013phase,candes2014wirtinger,ChenCandes15solving}, low-rank matrix sensing and completion \cite{tu2016low,bhojanapalli2016global, chen2015fast,zheng2015convergent,ge2016matrix}, blind deconvolution and self-calibration \cite{DBLP:journals/corr/LiLSW16,ling2017regularized,li2017blind,lee2017blind}, dictionary learning \cite{sun2017complete}, tensor decomposition \cite{ge2017optimization}, joint alignment \cite{chen2016projected}, learning shallow neural networks \cite{soltanolkotabi2017theoretical,zhong2017recovery}, robust subspace learning \cite{netrapalli2014non,maunu2017well,lerman2014fast,cherapanamjeri2017thresholding}. In several problems \cite{sun2016geometric,sun2017complete,ge2017optimization,ge2016matrix,li2016symmetry,li2016nonconvex,mei2016landscape,maunu2017well,davis2017nonsmooth}, it is further suggested that the optimization landscape is benign under sufficiently large sample complexity, in the sense that all local minima are globally optimal, and hence nonconvex iterative algorithms become promising in solving such problems. See \cite{chi2018nonconvex} for a recent overview. Below we review the three problems studied in this paper in more details. Some state-of-the-art results are summarized in Table~\ref{tab:Performance-guarantees-GD}.
\begin{itemize}
\item \emph{Phase retrieval.} Cand\`es et al.~proposed \emph{PhaseLift} \cite{candes2012phaselift} to solve the quadratic systems
of equations based on convex programming. Specifically, it lifts the
decision variable $\bm{x}^{\star}$ into a rank-one matrix $\bm{X}^{\star}=\bm{x}^{\star}\bm{x}^{\star\top}$
and translates the quadratic constraints of $\bm{x}^{\star}$ in \eqref{eq:pr_model} into linear constraints of $\bX^{\star}$. By dropping the rank constraint, the problem
becomes convex \cite{candes2012phaselift,candes2012solving,chen2013exact,cai2015rop,tropp2015convex}. Another convex program PhaseMax \cite{goldstein2016phasemax,bahmani2016phase,hand2016elementary,dhifallah2017phase} operates in the natural parameter space via linear programming, provided that an anchor vector is available. 
		On the other hand, alternating minimization \cite{netrapalli2013phase} with sample splitting has been shown to enjoy much better computational guarantee. In contrast, Wirtinger Flow \cite{candes2014wirtinger} provides the first global convergence result for nonconvex methods without sample splitting, whose statistical and computational guarantees are later improved by  \cite{ChenCandes15solving} via an adaptive truncation strategy. Several other variants of WF are also proposed \cite{cai2016optimal,kolte2016phase,soltanolkotabi2017structured}, among which an amplitude-based loss function has been investigated \cite{wang2017solving,zhang2017reshaped,wang2016sparse,wang2017solving2}. In particular, \cite{zhang2017reshaped} demonstrates that the amplitude-based loss function has a better curvature, and vanilla gradient descent can indeed converge with a constant step size at the order-wise optimal sample complexity. A small sample of other nonconvex  phase retrieval methods include \cite{shechtman2013gespar,schniter2015compressive,chi2016kaczmarz,chen2015phase,duchi2017solving,gao2016phase,wei2015solving,bendory2017non,tan2017phase,cai2017fast,qu2017convolutional}, which are beyond the scope of this paper.


\item \emph{Matrix completion}. Nuclear norm minimization was studied in \cite{ExactMC09} as a convex relaxation paradigm to solve the matrix completion problem.
Under certain incoherence conditions imposed upon the ground truth matrix, exact
recovery is guaranteed under near-optimal sample complexity \cite{CanTao10,Gross2011recovering,recht2011simpler,chen2015incoherence,davenport2016overview}. Concurrently, several works \cite{KesMonSew2010,Se2010Noisy,lee2010admira,jain2013low,hardt2013provable,hastie2015matrix,zhao2015nonconvex,jain2015fast,tanner2016low,jin2016provable,wei2016guarantees,zhao2015nonconvex} 
tackled the matrix completion problem via nonconvex approaches. In particular, the seminal work by Keshavan et al.~\cite{KesMonSew2010,Se2010Noisy} pioneered the two-stage approach that is widely adopted by later works. Sun and Luo \cite{sun2016guaranteed} demonstrated the convergence of gradient descent type methods for noiseless matrix completion with a regularized nonconvex loss function.
Instead of penalizing the loss function, \cite{chen2015fast,zheng2016convergence} employed projection to enforce the incoherence condition throughout the execution of the algorithm. To the best of our knowledge, no rigorous guarantees have been established for matrix completion without explicit regularization. A notable exception is \cite{jin2016provable}, which uses unregularized stochastic gradient descent for matrix completion in the online setting. However, the analysis is performed with fresh samples in each iteration. Our work closes the gap and makes the first contribution towards understanding implicit regularization in gradient descent without sample splitting. In addition, entrywise eigenvector perturbation has been studied by \cite{jain2015fast,abbe2017entrywise,chen2018asymmetry} in order to analyze the spectral algorithms for matrix completion, which helps us establish theoretical guarantees for the spectral initialization step.  Finally, it has recently been shown that the analysis of nonconvex gradient descent in turn yields near-optimal statistical guarantees for convex relaxation in the context of noisy matrix completion; see
\cite{chen2019noisy,chen2019inference}.

\item \emph{Blind deconvolution}. In \cite{ahmed2014blind}, Ahmed et al.~first proposed to invoke similar lifting ideas for blind deconvolution, which translates the bilinear measurements \eqref{eq:bd_model} into a system of linear measurements of a rank-one matrix $\bX^{\star}=\bm{h}^{\star}\bm{x}^{\star\conj}$. Near-optimal performance guarantees have been established for convex relaxation \cite{ahmed2014blind}. Under the same model, Li et al.~\cite{DBLP:journals/corr/LiLSW16} proposed a regularized gradient descent algorithm that directly optimizes the nonconvex loss function \eqref{eq:loss-BD} with a few regularization terms that account for scaling ambiguity and incoherence. In \cite{huang2017blind}, a Riemannian steepest descent method is developed that removes the regularization for scaling ambiguity, although they still need to regularize for incoherence. In \cite{aghasi2017branchhull}, a linear program is proposed but requires exact knowledge of the signs of the signals. Blind deconvolution has also been studied for other models -- interested readers may refer to \cite{chi2016guaranteed,ling2017regularized,lee2017blind,ling2015self,lee2016fast,zhang2017global,wang2016blind}.

\end{itemize}

On the other hand, our analysis framework is based on a leave-one-out perturbation argument. This technique has been widely used to analyze high-dimensional problems with random designs, including but not limited to robust M-estimation \cite{el2013robust,el2015impact}, statistical inference for sparse regression \cite{javanmard2015biasing},  likelihood ratio test in logistic regression \cite{sur2017likelihood}, phase synchronization \cite{zhong2017near,abbe2017entrywise}, ranking from pairwise comparisons \cite{chen2017spectral}, community recovery \cite{abbe2017entrywise}, and covariance sketching \cite{li2018nonconvex}. In particular, this technique results in tight performance guarantees for  the generalized power method \cite{zhong2017near}, the spectral method \cite{abbe2017entrywise,chen2017spectral}, and convex programming approaches \cite{el2015impact,zhong2017near,sur2017likelihood,chen2017spectral}, however it has not been applied to analyze nonconvex optimization algorithms. 

Finally, we note that the notion of implicit regularization --- broadly defined --- arises in settings far beyond the models and algorithms considered herein. For instance, it has been conjectured that in matrix factorization, over-parameterized stochastic gradient descent effectively enforces certain norm constraints, allowing it to converge to a minimal-norm solution as long as it starts from the origin \cite{gunasekar2017implicit}.  The stochastic gradient methods have also been shown to implicitly enforce Tikhonov regularization in several statistical learning settings  \cite{lin2016generalization}. More broadly, this phenomenon seems crucial in enabling efficient training of deep neural networks \cite{zhang2016understanding,soudry2017implicit}.

\section{A general recipe for trajectory analysis\label{sec:A-general-recipe}}

In this section, we sketch a general recipe for establishing performance guarantees of  gradient descent, which conveys the key idea for proving the main results of this paper. The main challenge is to demonstrate that appropriate incoherence conditions are preserved throughout the trajectory of the algorithm. This requires exploiting statistical independence of the samples in a careful manner, in conjunction with generic optimization theory. Central to our approach is a leave-one-out perturbation argument, which allows to decouple the statistical dependency while controlling the component-wise incoherence measures.

\medskip

\fbox{\begin{minipage}[t]{0.95\columnwidth}%
\begin{center}

\textbf{General Recipe (a leave-one-out analysis)}

\medskip

\begin{minipage}{0.9\textwidth}

\begin{itemize}[leftmargin=3em]
\item[{\bf Step 1:}] characterize restricted strong convexity and smoothness of $f$, and identify the region of incoherence
and contraction (RIC). 

\item[{\bf Step 2:}] introduce leave-one-out sequences $\{\bm{X}^{t,(l)}\}$
and $\{\bm{H}^{t,(l)}\}$ for each $l$, where $\{\bm{X}^{t,(l)}\}$
		(resp.~$\{\bm{H}^{t,(l)}\}$) is independent of any sample involving
		$\bm{\phi}_{l}$ (resp.~$\bm{\psi}_{l}$); 

\item[{\bf Step 3:}] establish the incoherence condition for $\{\bm{X}^t\}$ and $\{\bm{H}^t\}$ via induction. Suppose the iterates satisfy the claimed conditions in the $t$th iteration:
	\vspace{-0.5em}
\begin{itemize}
[leftmargin=0em]
\item[(a)] show, via restricted strong convexity, that the true iterates $(\bm{X}^{t+1},\bm{H}^{t+1})$
and the leave-one-out version $(\bm{X}^{t+1,(l)},\bm{H}^{t+1,(l)})$ are
exceedingly close;

\item[(b)] use statistical independence to show that $\bm{X}^{t+1,(l)}-\bm{X}^{\star}$
(resp.~$\bm{H}^{t+1,(l)}-\bm{H}^{\star}$) is incoherent w.r.t.~$\bm{\phi}_{l}$ (resp.~$\bm{\psi}_{l}$), namely, $\|\bm{\phi}_{l}^{\conj} ( \bm{X}^{t+1,(l)} -\bm{X}^{\star} )\|_{2}$
and $\|\bm{\psi}_{l}^{\conj} ( \bm{H}^{t+1,(l)} - \bm{H}^{\star} )\|_2$ are both well-controlled;

\item[(c)] combine the bounds to establish the desired incoherence condition concerning
$
\max\limits_{l}
\|\bm{\phi}_{l}^{\conj} ( \bm{X}^{t+1} -\bm{X}^{\star})\|_{2}$
and $\max\limits_{l} \|\bm{\psi}_{l}^{\conj} ( \bm{H}^{t+1}
-\bm{H}^{\star})
\|_{2}$.

\end{itemize}

\end{itemize}

\end{minipage}

\end{center}%
\end{minipage}}

\subsection{General model}
Consider the following problem where the samples are collected in a bilinear/quadratic form as
\begin{equation}\label{eq:general_formula}
y_{j}=\bm{\psi}_{j}^{\conj} \bm{H}^{\star}\bm{X}^{\star \conj} \bm{\phi}_{j} , \qquad1\leq j\leq m,
\end{equation}
where the objects of interest $\bm{H}^{\star}, \bm{X}^{\star}\in\mathbb{C}^{n\times r}$ or $\mathbb{R}^{n\times r}$
might be vectors or tall matrices taking either real or complex values. The design vectors $\left\{ \bm{\psi}_{j}\right\} $ and $\{\bm{\phi}_{j}\}$ are
 in either $\mathbb{C}^n$ or $\mathbb{R}^n$, and  can be either random or deterministic. This model is quite general and entails all three examples in this paper as special cases:
\begin{itemize}
\item \emph{Phase retrieval}: $\bm{H}^{\star}=\bm{X}^{\star}=\bm{x}^{\star}\in\mathbb{R}^{n}$,
and $\bm{\psi}_{j}=\bm{\phi}_{j}=\bm{a}_{j}$;
\item \emph{Matrix completion}: $\bm{H}^{\star}=\bm{X}^{\star}\in\mathbb{R}^{n\times r}$
and $\bm{\psi}_{j},\bm{\phi}_{j}\in\{\bm{e}_{1},\cdots,\bm{e}_{n}\}$;
\item \emph{Blind deconvolution}: $\bm{H}^{\star}=\bm{h}^{\star}\in\mathbb{C}^{K}$,
$\bm{X}^{\star}=\bm{x}^{\star}\in\mathbb{C}^{K}$, $\bm{\phi}_{j}=\bm{a}_{j},$
and $\bm{\psi}_{j}=\bm{b}_{j}$.
\end{itemize}
For this setting, the empirical loss function is given by
\[
f(\bm{Z}) := f(\bm{H},\bm{X})=\frac{1}{m}\sum_{j=1}^{m}\Big| \bm{\psi}_{j}^{\conj} \bm{H}\bm{X}^{\conj} \bm{\phi}_{j}-y_{j}\Big|^{2},
\]
where we denote $\bm{Z}=(\bm{H},\bm{X})$. To minimize $f(\bm{Z})$, we  proceed with vanilla gradient descent 
\[
\bm{Z}^{t+1} = \bm{Z}^{t}-\eta \nabla f\big(\bm{Z}^{t}\big), \qquad \forall t\geq 0
\]
following a standard spectral initialization, where $\eta$ is the step size. As a remark, for complex-valued problems, the gradient (resp.~Hessian) should be understood as the Wirtinger gradient (resp.~Hessian). 

It is clear from \eqref{eq:general_formula} that $\bm{Z}^{\star} = (\bm{H}^{\star},\bm{X}^{\star})$ can only be recovered up to certain global ambiguity.  For clarity of presentation, we assume in this section that such ambiguity has already been taken care of via proper global transformation.

\subsection{Outline of the recipe}
We are now positioned to outline the general recipe, which entails the following steps. 
\begin{itemize}[leftmargin=1em]
\item \textbf{Step 1: characterizing local geometry in the RIC.} Our first step is to characterize a region $\mathcal{R}$ --- which we term as the {\em region of incoherence
		and contraction} (RIC) --- such that the Hessian matrix $\nabla^{2}f(\bm{Z})$ obeys strong convexity and smoothness,
\begin{equation}
\bm{0} \,\prec\, \alpha\bm{I} \,\preceq\, \nabla^{2}f(\bm{Z}) \,\preceq\, \beta\bm{I},\qquad\forall\bm{Z}\in\mathcal{R},\label{eq:strong-convexity-recipe}
\end{equation}
or at least along certain directions (i.e.~restricted strong convexity and smoothness), where $\beta/\alpha$ scales slowly (or even remains bounded) with
the problem size. As revealed by optimization theory, this geometric property \eqref{eq:strong-convexity-recipe} immediately implies linear convergence with the contraction rate $1- O(\alpha/\beta)$ for a properly chosen step size $\eta$, as long as all iterates stay within the RIC.

A natural question then arises: what does the RIC $\mathcal{R}$ look like? As it turns out, the RIC typically contains all points
such that the $\ell_2$ error $\|\bm{Z} - \bm{Z}^{\star}\|_{\mathrm{F}}$ is not too large and
\begin{align}
	\label{eq:defn-incoherence-general}
	(\textbf{incoherence}) \qquad\max_{j} {\big\|\bm{\phi}_{j}^{\conj}(\bm{X}-\bm{X}^{\star})\big\|_2}~~\text{and}~~
	\max_{j} {\big\|\bm{\psi}_{j}^{\conj}(\bm{H}-\bm{H}^{\star})\big\|_2} 
	\text{ are well-controlled}.
\end{align}
In the three examples, the above incoherence condition translates to:
\begin{itemize}
\item \emph{Phase retrieval}: $\max_{j} {\big|\bm{a}_{j}^{\top}(\bm{x}-\bm{x}^{\star})\big|}$ is well-controlled;
\item \emph{Matrix completion}: $\big\|\bX - \bX^{\star}\big\|_{2,\infty}$ is well-controlled;
\item \emph{Blind deconvolution}: $\max_{j} {\big|\bm{a}_{j}^{\top}(\bm{x}-\bm{x}^{\star})\big|}$ and $\max_{j} {\big|\bm{b}_{j}^{\top}(\bm{h}-\bm{h}^{\star})\big|}$ are well-controlled.
\end{itemize}

%
\item \textbf{Step 2: introducing the leave-one-out sequences.} To justify that no iterates leave the RIC, we rely on the construction of auxiliary sequences. Specifically,  
for each
		$l$, produce an auxiliary sequence $\{\bm{Z}^{t,(l)}=(\bm{X}^{t,(l)},\bm{H}^{t,(l)})\}$ such that $\bm{X}^{t,(l)}$ (resp.~$\bm{H}^{t,(l)}$) is independent of any sample involving $\bm{\phi}_l$ (resp.~$\bm{\psi}_l$). As an example, suppose that the $\bm{\phi}_l$'s and the $\bm{\psi}_l$'s are independently and randomly generated. Then for each $l$, one can 
consider a leave-one-out loss function
\[
	f^{(l)}(\bm{Z}) : = \frac{1}{m}\sum_{j: j\neq l}\Big| \bm{\psi}_{j}^{\conj} \bm{H}\bm{X}^{\conj} \bm{\phi}_{j}-y_{j}\Big|^{2}
\]
that discards the $l$th sample. One further generates 
$\{\bm{Z}^{t,(l)}\}$ by running vanilla gradient descent
w.r.t.~this auxiliary loss function,  with a spectral initialization that similarly discards the $l$th sample. 
Note that this procedure is only introduced to facilitate  analysis and is never implemented in practice.

\item \textbf{Step 3: establishing the incoherence condition.} 
	We are now ready to establish the incoherence condition with the assistance of the auxiliary sequences. Usually the proof proceeds by induction, where our goal is to show that the next iterate remains within the RIC, given that the current one does.
\begin{itemize}[leftmargin=1em]
\item \textbf{Step 3(a): proximity between the original and the leave-one-out
iterates. }As one can anticipate, $\{\bm{Z}^{t}\}$ and $\{\bm{Z}^{t,(l)}\}$
remain ``glued'' to each other along the whole trajectory, since their constructions differ by only a single sample. In fact, as long as the initial estimates stay sufficiently
close, their gaps will never explode. To intuitively see why, use the fact $\nabla f(\bm{Z}^{t})\approx\nabla f^{(l)}(\bm{Z}^{t})$ to discover that
\begin{align*}
\bm{Z}^{t+1}-\bm{Z}^{t+1,(l)} & =\bm{Z}^{t}-\eta\nabla f(\bm{Z}^{t})-\big(\bm{Z}^{t,(l)}-\eta\nabla f^{(l)}\big(\bm{Z}^{t,(l)}\big)\big)\\
 & \approx\bm{Z}^{t}-\bm{Z}^{t,(l)}-\eta\nabla^{2}f(\bm{Z}^{t})\big(\bm{Z}^{t}-\bm{Z}^{t,(l)}\big),
\end{align*}
which together with the strong convexity condition implies $\ell_{2}$ contraction
\[
\big\|\bm{Z}^{t+1}-\bm{Z}^{t+1,(l)} \big\|_{\mathrm{F}} \approx
		\Big\|\big(\bm{I}-\eta\nabla^{2}f(\bm{Z}^{t})\big)\big(\bm{Z}^{t}-\bm{Z}^{t,(l)}\big)\Big\|_{\mathrm{F}}\leq \big\|\bm{Z}^{t}-\bm{Z}^{t,(l)} \big\|_{2}.
\]
Indeed, (restricted) strong  convexity is crucial in controlling the size of leave-one-out perturbations. 
\item \textbf{Step 3(b): incoherence condition of the leave-one-out iterates. }The
fact that $\bm{Z}^{t+1}$ and $\bm{Z}^{t+1,(l)}$ are exceedingly close
motivates us to control the incoherence of $\bm{Z}^{t+1,(l)} - \bm{Z}^{\star}$ instead, for $1\leq l\leq m$.
		By construction, $\bm{X}^{t+1,(l)}$ (resp.~$\bm{H}^{t+1,(l)}$) is statistically {\em independent} of any sample involving
		the design vector $\bm{\phi}_{l}$ (resp.~$\bm{\psi}_
		l$), a fact that typically leads to a more friendly analysis for controlling $\left\| \bm{\phi}_{l}^{\conj}\big(\bm{X}^{t+1,(l)}-\bm{X}^{\star}\big) \right\|_2$ and $\left\| \bm{\psi}_{l}^{\conj}\big(\bm{H}^{t+1,(l)}-\bm{H}^{\star}\big) \right\|_2$. 
%
\item \textbf{Step 3(c): combining the bounds. }With these results in
place, apply the triangle inequality to obtain 
\begin{align*}
	\big \|\bm{\phi}_{l}^{\conj}\big(\bm{X}^{t+1}-\bm{X}^{\star}\big)\big \|_2 & \leq \big \|\bm{\phi}_{l}\|_2 \big\| \bm{X}^{t+1}-\bm{X}^{t+1,(l)} \big \|_{\mathrm{F}} + \big \|\bm{\phi}_{l}^{\conj}\big(\bm{X}^{t+1,(l)}-\bm{X}^{\star}\big)\big \|_2, 
\end{align*}
where the first term is controlled in Step 3(a) and the second term is controlled in Step 3(b).  The term $\big \|\bm{\psi}_{l}^{\conj}\big(\bm{H}^{t+1}-\bm{H}^{\star}\big)\big \|_2$ can be bounded similarly. By choosing the bounds properly, this establishes the incoherence condition
for all $1\leq l\leq m$ as desired. 
\end{itemize}
\end{itemize}


\section{Analysis for phase retrieval\label{sec:Analysis-of-WF}}
In this section, we instantiate the general recipe presented in Section~\ref{sec:A-general-recipe} to phase retrieval and prove Theorem~\ref{thm:WF}. Similar to the Section 7.1 in \cite{candes2014wirtinger}, we are going to use $\eta_t = c_1/( \log n \cdot \| \bm{x}^{\star} \|_2^2 )$ instead of $c_1/( \log n \cdot \| \bm{x}_0 \|_2^2 )$ as the step size for analysis. This is because with high probability, $\| \bm{x}_0 \|_2$ and $\| \bm{x}^{\star} \|_2$ are rather close in the relative sense. Without loss of generality, we assume throughout this section that $\big\|\bm{x}^{\star}\big\|_{2}=1$
and
\begin{equation}
\mathrm{dist}(\bm{x}^{0},\bm{x}^{\star})=\|\bm{x}^{0}-\bm{x}^{\star}\|_{2}\leq\|\bm{x}^{0}+\bm{x}^{\star}\|_{2}.\label{eq:assumption-x0-rotation-WF}
\end{equation}
In addition, the gradient and the Hessian of $f(\cdot)$ for this
problem (see \eqref{eq:loss-WF-PR}) are given respectively by 
\begin{align}
\nabla f\left(\bm{x}\right) & =\frac{1}{m}\sum_{j=1}^{m}\left[\left(\bm{a}_{j}^{\top}\bm{x}\right)^{2}-y_{j}\right]\left(\bm{a}_{j}^{\top}\bm{x}\right)\bm{a}_{j},\label{eq:gradient-WF}\\
\nabla^{2}f\left(\bm{x}\right) & =\frac{1}{m}\sum_{j=1}^{m}\left[3\left(\bm{a}_{j}^{\top}\bm{x}\right)^{2}-y_{j}\right]\bm{a}_{j}\bm{a}_{j}^{\top},\label{eq:hessian-WF}
\end{align}
which are useful throughout the proof. 

\subsection{Step 1: characterizing local geometry in the RIC\label{subsec:Proof-outline-of-thm-WF}}

\subsubsection{Local geometry}

We start by characterizing the region that enjoys both strong convexity
and the desired level of smoothness. This is supplied in the following
lemma, which plays a crucial role in the subsequent analysis.
\begin{lemma}[Restricted
strong convexity and smoothness for phase retrieval]\label{lemma:wf_hessian}
Fix any sufficiently small constant $C_{1}>0$ and any sufficiently
large constant $C_{2}>0$, and suppose the sample complexity obeys
$m\geq c_{0}n\log n$ for some sufficiently large constant $c_{0}>0$.
With probability at least $1-O(mn^{-10})$, 
\[
\nabla^{2}f\left(\bm{x}\right)\succeq\left({1} / {2}\right) \cdot \bm{I}_{n}
\]
holds simultaneously for all $\bm{x}\in\mathbb{R}^{n}$ satisfying
$\left\Vert \bm{x}-\bm{x}^{\star}\right\Vert _{2}\leq2C_{1}$; 
and 
\[
\nabla^{2}f\left(\bm{x}\right)\preceq\left(5C_{2}\left(10+C_{2}\right)\log n\right)\cdot\bm{I}_{n}
\]
holds simultaneously for all $\bm{x}\in\mathbb{R}^{n}$  obeying
\begin{subequations}
\label{eq:WF-hessian-condition} 
\begin{align} \left\Vert \bm{x} -\bm{x}^{\star}\right\Vert _{2} & \leq 2C_{1},\label{eq:WF-induction-L2error-hessian}\\
\max_{1\leq j\leq m}\left|\bm{a}_{j}^{\top}\left(\bm{x} -\bm{x}^{\star}\right)\right| & \leq C_{2}\sqrt{\log n}.\label{eq:WF-induction-incoherence-t-hessian}
\end{align}
\end{subequations}

\end{lemma}\begin{proof}See Appendix \ref{subsec:Proof-of-Lemma-wf-hessian}.\end{proof}
	In words, Lemma \ref{lemma:wf_hessian} reveals that the Hessian matrix
is positive definite and (almost) well-conditioned, if one restricts
attention to the set of points that are (i) not far away from the truth (cf.~(\ref{eq:WF-induction-L2error-hessian}))
and (ii) incoherent with respect to the measurement vectors $\left\{ \bm{a}_{j}\right\} _{1\leq j\leq m}$ (cf.~(\ref{eq:WF-induction-incoherence-t-hessian})). 

\subsubsection{Error contraction }

As we point out before, the nice local geometry enables $\ell_{2}$ contraction, which we
formalize below.

\begin{lemma}
\label{lem:error-contraction-Hessian-WF}
There exists an event that does not depend on $t$ and has probability $1-O(mn^{-10})$, such that when it happens and $\bm{x}^t$ obeys the conditions \eqref{eq:WF-hessian-condition}, one has
\begin{align}
\left\Vert \bm{x}^{t+1}-\bm{x}^{\star}\right\Vert _{2} & \leq\left(1-\eta/2\right)\left\Vert \bm{x}^{t}-\bm{x}^{\star}\right\Vert _{2}
\end{align}
provided that the step size satisfies $0<\eta\leq1/ \left[5C_{2}\left(10+C_{2}\right)\log n\right]$.
\end{lemma}
\begin{proof}
This proof applies the standard argument
when establishing the $\ell_{2}$ error contraction of gradient descent
for strongly convex and smooth functions. See Appendix \ref{subsec:Proof-of-Lemma-contraction-WF}.
\end{proof}

With the help of Lemma \ref{lem:error-contraction-Hessian-WF}, we can turn the proof of 
Theorem \ref{thm:WF} into ensuring that the trajectory $\left\{\bm{x}^t\right\}_{0\leq t \leq n}$ lies in the RIC specified by \eqref{eq:WF-induction}.\footnote{Here, we deliberately change $2C_{1}$ in \eqref{eq:WF-induction-L2error-hessian} to $C_{1}$ in the definition of the RIC \eqref{eq:WF-induction-L2error} to ensure the correctness of the analysis.} 
This is formally stated in the next lemma.  
\begin{lemma}\label{lemma:WF-t-iteration-enough}
	Suppose for all $0\leq t\leq T_{0}:=n$, the trajectory $\left\{\bm{x}^t\right\}$
falls within the region of incoherence and contraction (termed the RIC), namely, 
\begin{subequations}
\label{eq:WF-induction} 
\begin{align}
\left\Vert \bm{x}^{t}-\bm{x}^{\star}\right\Vert _{2} & \leq C_{1},\label{eq:WF-induction-L2error}\\
\max_{1\leq l\leq m}\left|\bm{a}_{l}^{\top}\left(\bm{x}^{t}-\bm{x}^{\star}\right)\right| & \leq C_{2}\sqrt{\log n},\label{eq:WF-induction-incoherence-t}
\end{align}
then the claims in Theorem \ref{thm:WF} hold true. Here and throughout this section, $C_{1},C_{2}>0$ are two absolute constants as specified in Lemma~\ref{lemma:wf_hessian}.
\end{subequations}
\end{lemma}
\begin{proof}See Appendix \ref{subsec:Proof-of-Lemma-WF-t-iteration}. \end{proof}

\subsection{Step 2: introducing the leave-one-out sequences}
In comparison to the $\ell_{2}$ error bound \eqref{eq:WF-induction-L2error} that captures
the overall loss, the incoherence hypothesis \eqref{eq:WF-induction-incoherence-t}
--- which concerns sample-wise control of the empirical risk ---
is more complicated to establish. This is partly due to the statistical dependence
between $\bm{x}^{t}$ and the sampling vectors $\{\ba_{l}\}$. 
As described
in the general recipe, the key idea is the introduction of a \emph{leave-one-out}
version of the WF iterates, which removes a single measurement from
consideration.

To be precise, for each
$1\leq l\leq m$, we define the leave-one-out empirical loss function
as 
\begin{equation}
f^{\left(l\right)}(\bm{x}):=\frac{1}{4m}\sum_{j:j\neq l}\left[\left(\bm{a}_{j}^{\top}\bm{x}\right)^{2}-y_{j}\right]^{2},\label{eq:leave-one-out-f-WF}
\end{equation}
and the auxiliary trajectory $\left\{ \bm{x}^{t,(l)}\right\} _{t\geq0}$
is constructed by running WF w.r.t.~$f^{(l)}(\bm{x})$. In addition, the
spectral initialization $\bm{x}^{0,\left(l\right)}$ is computed based
on the rescaled leading eigenvector of the leave-one-out data matrix
\begin{equation}
\bm{Y}^{\left(l\right)}:=\frac{1}{m}\sum_{j:j\neq l}y_{j}\bm{a}_{j}\bm{a}_{j}^{\top}.\label{eq:leave-one-out-Y-WF}
\end{equation}
Clearly, the entire sequence $\left\{ \bm{x}^{t,(l)}\right\} _{t\geq0}$
is independent of the $l$th sampling vector $\bm{a}_{l}$. This auxiliary
procedure is formally described in Algorithm \ref{alg:wf-LOO}. 
\begin{algorithm}[ht]
\caption{The $l$th leave-one-out sequence for phase retrieval}

\label{alg:wf-LOO}\begin{algorithmic}

\STATE \textbf{{Input}}: $\{\bm{a}_{j}\}_{1\leq j\leq m,j\neq l}$
and $\{y_{j}\}_{1\leq j\leq m,j\neq l}$.

\STATE \textbf{{Spectral initialization}}: let $\lambda_{1}\big(\bm{Y}^{\left(l\right)}\big)$
and $\tilde{\bm{x}}^{0,\left(l\right)}$ be the leading eigenvalue
and eigenvector of 
\[
\bm{Y}^{(l)}=\frac{1}{m}\sum_{j:j\neq l}y_{j}\bm{a}_{j}\bm{a}_{j}^{\top},
\]
respectively, and set 
\[
\bm{x}^{0,\left(l\right)}=\begin{cases}
\sqrt{\lambda_{1}\left(\bm{Y}^{\left(l\right)}\right) / 3}\;\tilde{\bm{x}}^{0,\left(l\right)}, & \text{if }\big\|\tilde{\bm{x}}^{0,(l)}-\bm{x}^{\star}\big\|_{2}\leq\big\|\tilde{\bm{x}}^{0,(l)}+\bm{x}^{\star}\big\|_{2},\\
-\sqrt{\lambda_{1}\left(\bm{Y}^{\left(l\right)}\right) / 3}\;\tilde{\bm{x}}^{0,\left(l\right)},\quad & \text{else}.
\end{cases}
\]

\STATE \textbf{{Gradient updates}}: \textbf{for} $t=0,1,2,\ldots,T-1$
\textbf{do}

\STATE 
\begin{equation}
\bm{x}^{t+1,(l)}=\bm{x}^{t,(l)}-\eta_{t}\nabla f^{(l)}\big(\bm{x}^{t,(l)}\big).\label{eq:gradient-update-leave-WF}
\end{equation}

\end{algorithmic} 
\end{algorithm}


\subsection{Step 3: establishing the incoherence condition by induction}

As revealed by Lemma \ref{lemma:WF-t-iteration-enough}, it suffices to prove that the iterates $\{\bm{x}^t\}_{0\leq t \leq T_{0}}$ satisfies (\ref{eq:WF-induction}) with high probability. Our proof will be inductive in nature. For the sake of clarity, we list all the induction hypotheses: 
\begin{subequations}
\label{eq:WF-induction-step-3} 
\begin{align} \left\Vert \bm{x}^{t}-\bm{x}^{\star}\right\Vert _{2} & \leq C_{1},\label{eq:WF-induction-L2error-step-3}\\
\max_{1\leq l\leq m}\big\|\bm{x}^{t}-\bm{x}^{t,\left(l\right)}\big\|_{2} & \leq C_{3}\sqrt{\frac{\log n}{n}}\label{eq:induction-loop-WF-step-3}\\
\max_{1\leq j\leq m}\left|\bm{a}_{j}^{\top}\left(\bm{x}^{t}-\bm{x}^{\star}\right)\right| & \leq C_{2}\sqrt{\log n}.\label{eq:WF-induction-incoherence-t-step-3}
\end{align}
\end{subequations}
Here $C_{3} >0$ is some universal constant. 
For any $t\geq0$, define $\mathcal{E}_{t}$ to be the event where the conditions in (51) hold for 
	the $t$-th iteration. 
	According to Lemma \ref{lem:error-contraction-Hessian-WF}, there exists some event $\mathcal{E}$ with probability $1-O(mn^{-10})$ such that on $\mathcal{E}_t \cap \mathcal{E}$ one has
\begin{equation}\label{eq:induction-t+1-l2-wf}
\left\Vert \bm{x}^{t+1}-\bm{x}^{\star}\right\Vert _{2}  \leq C_{1}.
\end{equation}

This subsection is devoted to establishing \eqref{eq:induction-loop-WF-step-3} and \eqref{eq:WF-induction-incoherence-t-step-3} for the $(t+1)$th iteration,  assuming that (\ref{eq:WF-induction-step-3}) holds true up to the $t$th iteration. We defer the justification of the base case (i.e. initialization at $t=0$) to Section \ref{subsec:base-WF}.

\begin{itemize}
\item \textbf{Step 3(a): proximity between the original and the leave-one-out
iterates.} The leave-one-out sequence $\{\bx^{t,(l)}\}$ behaves similarly
to the true WF iterates $\{\bx^{t}\}$ while maintaining statistical independence
with $\ba_{l}$, a key fact that allows us to control the incoherence
of $l$th leave-one-out sequence w.r.t.~$\bm{a}_{l}$. We will formally quantify the gap between $\bx^{t+1}$ and $\bx^{t+1,(l)}$ in the following lemma, which establishes the induction in \eqref{eq:induction-loop-WF-step-3}.

\begin{lemma} \label{lemma:WF-loop-inductive}Suppose that the sample size obeys $m\geq Cn\log n$ for some sufficiently
	large constant $C>0$ and that the stepsize obeys $0<\eta<1/[5C_{2}(10+C_{2})\log n]$.
Then on some event $\mathcal{E}_{t+1,1}\subseteq \mathcal{E}_{t}$ obeying
	$\mathbb{P}(\mathcal{E}_{t}\cap \mathcal{E}_{t+1,1}^{c} ) = O( mn^{-10})$,
	one has 
	\begin{equation}
	\max_{1\leq l\leq m}\left\Vert \bm{x}^{t+1}-\bm{x}^{t+1,(l)}\right\Vert _{2}\leq C_{3}\sqrt{\frac{\log n}{n}}. \label{eq:induction-loop-WF}
	\end{equation} 
\end{lemma}

\begin{proof}The proof relies
heavily on the restricted strong convexity (see Lemma~\ref{lemma:wf_hessian}) and is deferred to Appendix
\ref{subsec:Proof-outline-proximity-LOO-WF}. \end{proof}
\item \textbf{Step 3(b): incoherence of the leave-one-out iterates.} By
construction, $\bm{x}^{t+1,(l)}$ is statistically independent of
the sampling vector $\bm{a}_{l}$. One can thus invoke the standard
Gaussian concentration results and the union bound to derive that on an event $\mathcal{E}_{t+1,2} \subseteq \mathcal{E}_t$ obeying
	$\mathbb{P}(\mathcal{E}_{t}\cap \mathcal{E}_{t+1,2}^{c} ) = O( mn^{-10})$,
\begin{align}
\max_{1\leq l\leq m}\left|\bm{a}_{l}^{\top}\big(\bm{x}^{t+1,\left(l\right)}-\bm{x}^{\star}\big)\right| & \leq5\sqrt{\log n}\big\|\bm{x}^{t+1,\left(l\right)}-\bm{x}^{\star}\big\|_{2}\nonumber \\
 & \overset{(\text{i})}{\leq}5\sqrt{\log n}\left(\big\|\bm{x}^{t+1,\left(l\right)}-\bm{x}^{t+1}\big\|_{2}+\left\Vert \bm{x}^{t+1}-\bm{x}^{\star}\right\Vert _{2}\right)\nonumber \\
 & \overset{(\mathrm{ii})}{\leq}5\sqrt{\log n}\left(C_{3}\sqrt{\frac{\log n}{n}}+C_{1}\right)\nonumber \\
 & \leq C_{4}\sqrt{\log n}\label{eq:LOOWF-induction-incoherence-t}
\end{align}
holds for some constant $C_{4}\geq 6 C_{1}>0$ and $n$ sufficiently large. Here, (i) comes
from the triangle inequality, and (ii) arises from the proximity bound
(\ref{eq:induction-loop-WF}) and the condition (\ref{eq:induction-t+1-l2-wf}). 
\item \textbf{Step 3(c): combining the bounds.} We are now prepared to establish
\eqref{eq:WF-induction-incoherence-t-step-3} for the ($t+1$)th iteration. Specifically, 
\begin{align}
\max_{1\leq l\leq m}\left|\bm{a}_{l}^{\top}\left(\bm{x}^{t+1}-\bm{x}^{\star}\right)\right| & \leq\max_{1\leq l\leq m}\left|\bm{a}_{l}^{\top}\big(\bm{x}^{t+1}-\bm{x}^{t+1,\left(l\right)}\big)\right|+\max_{1\leq l\leq m}\left|\bm{a}_{l}^{\top}\big(\bm{x}^{t+1,\left(l\right)}-\bm{x}^{\star}\big)\right|\nonumber \\
 & \overset{\left(\text{i}\right)}{\leq}\max_{1\leq l\leq m}\|\bm{a}_{l}\|_{2}\big\|\bm{x}^{t+1}-\bm{x}^{t+1,\left(l\right)}\big\|_{2}+C_{4}\sqrt{\log n}\nonumber \\
 & \overset{\left(\text{ii}\right)}{\leq}\sqrt{6n}\cdot C_{3}\sqrt{\frac{\log n}{n}}+C_{4}\sqrt{\log n}\leq C_{2}\sqrt{\log n},\label{eq:LOOWF-induction-incoherence-t-0}
\end{align}
where (i) follows from the Cauchy-Schwarz inequality and \eqref{eq:LOOWF-induction-incoherence-t},
the inequality (ii) is a consequence of \eqref{eq:induction-loop-WF}
and (\ref{eq:max-al-norm}), and the last inequality holds as long
as $C_{2}/(C_{3}+C_{4})$ is sufficiently large. From the deduction above we easily get $\mathbb{P}(\mathcal{E}_{t}\cap \mathcal{E}_{t+1}^{c} ) = O( mn^{-10}) $.
\end{itemize}
Using mathematical induction and the union bound, we establish (\ref{eq:WF-induction-step-3})
for all $t\leq T_{0}=n$ with  high probability. This in turn
concludes the proof of Theorem \ref{thm:WF}, as long as the hypotheses are valid for the base case.  

\subsection{The base case: spectral initialization}\label{subsec:base-WF}
In the end, we return to verify the induction hypotheses for the base case ($t=0$), i.e.~the spectral initialization obeys (\ref{eq:WF-induction-step-3}). The following lemma justifies (\ref{eq:WF-induction-L2error-step-3}) by choosing $\delta$ sufficiently small.
\begin{lemma}
\label{lemma:wf_L2-base}
Fix any small constant
$\delta>0$, and suppose $m>c_{0}n\log n$ for some large constant
$c_{0}>0$. Consider the two vectors $\bm{x}^{0}$ and $\tilde{\bm{x}}^{0}$
as defined in Algorithm \ref{alg:wf}, and suppose without loss of
generality that (\ref{eq:assumption-x0-rotation-WF}) holds. Then
with probability exceeding $1-O(n^{-10})$, one has
\begin{equation}
\left\Vert \bm{Y}-\EE\left[\bm{Y}\right]\right\Vert \leq\delta,\label{eq:Y-EY-delta}
\end{equation}
\begin{equation}
\|\bm{x}^{0}-\bm{x}^{\star}\|_{2}\leq2\delta\qquad\text{and}\qquad\big\|\tilde{\bm{x}}^{0}-\bm{x}^{\star}\big\|_{2}\leq\sqrt{2}\delta.\label{eq:WF_init}
\end{equation}
\end{lemma}
\begin{proof}
This result follows directly from the Davis-Kahan
sin$\Theta$ theorem. See Appendix \ref{subsec:Proof-of-Lemma-wf_L2-base}.
\end{proof}
We then move on to justifying \eqref{eq:induction-loop-WF-step-3}, the proximity between the original and leave-one-out iterates for $t=0$.
\begin{lemma}\label{lemma:wf_loop-base}
Suppose $m>c_{0}n\log n$ for some large constant
$c_{0}>0$. Then
with probability at least ${1-O(mn^{-10})}$, one has
\begin{equation}
\max_{1\leq l\leq m}\big\|\bm{x}^{0}-\bm{x}^{0,\left(l\right)}\big\|_{2}  \leq C_{3}\sqrt{\frac{\log n}{n}}.
\end{equation}

\end{lemma}
\begin{proof}This is also a consequence of the Davis-Kahan
sin$\Theta$ theorem.  See Appendix \ref{subsec:Proof-of-Lemma-wf_loop-base}.\end{proof}
The final claim (\ref{eq:WF-induction-incoherence-t-step-3}) can be proved using the same argument as in deriving (\ref{eq:LOOWF-induction-incoherence-t-0}), and hence is omitted.



\section{Analysis for matrix completion\label{sec:Analysis-MC}}

In this section, we instantiate the general recipe presented in Section~\ref{sec:A-general-recipe} to matrix completion and prove Theorem~\ref{thm:main-MC}. Before continuing, we first gather a few useful facts regarding the loss function in (\ref{eq:minimization-MC}). The gradient of it
is given by 
\begin{align}
\nabla f\left(\bm{X}\right) & =\frac{1}{p}\mathcal{P}_{\Omega}\left[\bm{X}\bm{X}^{\top}-\left(\bm{M}^{\star}+\bm{E}\right)\right]\bm{X}.\label{eq:gradient-MC-formula}
\end{align}
We define the expected gradient (with respect to the sampling
set $\Omega$) to be 
\[
\nabla F\left(\bm{X}\right)=\left[\bm{X}\bm{X}^{\top}-\left(\bm{M}^{\star}+\bm{E}\right)\right]\bm{X}
\]
and also the (expected) gradient without
noise to be 
\begin{align}
\nabla f_{\mathrm{clean}}\left(\bm{X}\right)=\frac{1}{p}\mathcal{P}_{\Omega}\left(\bm{X}\bm{X}^{\top}-\bm{M}^{\star}\right)\bm{X}\qquad\text{and}\qquad\nabla F_{\mathrm{clean}}\left(\bm{X}\right)=\left(\bm{X}\bm{X}^{\top}-\bm{M}^{\star}\right)\bm{X}\label{eq:gradient-MC-formula-clean}.
\end{align}
In addition, we need the Hessian $\nabla^{2}f_{\mathrm{clean}}(\bm{X})$, which is represented
by an $nr\times nr$ matrix. Simple calculations reveal that
for any $\bm{V}\in\RR^{n\times r}$, 
\begin{equation}
\mathrm{vec}\left(\bm{V}\right)^{\top}\nabla^{2}f_{\mathrm{clean}}\left(\bm{X}\right)\mathrm{vec}\left(\bm{V}\right)=\frac{1}{2p}\left\Vert \mathcal{P}_{\Omega}\left(\bm{V}\bm{X}^{\top}+\bm{X}\bm{V}^{\top}\right)\right\Vert _{\mathrm{F}}^{2}+\frac{1}{p}\left\langle \mathcal{P}_{\Omega}\left(\bm{X}\bm{X}^{\top}-\bm{M}^{\star}\right),\bm{V}\bm{V}^{\top}\right\rangle ,\label{eq:hessian_quadratic-MC}
\end{equation}
where $\mathrm{vec}(\bm{V})\in\mathbb{R}^{nr}$ denotes the vectorization
of $\bm{V}$.

\subsection{Step 1: characterizing local geometry in the RIC}

\subsubsection{Local geometry}

The first step is to characterize the region where the empirical loss function enjoys
restricted strong convexity and smoothness in an appropriate sense. 
This is formally stated in the following lemma.

\begin{lemma}[Restricted strong convexity and smoothness for matrix
completion]\label{lemma:hessian-MC}Suppose that the sample size
obeys $n^{2}p\geq C\kappa^{2}\mu rn\log n$ for some sufficiently large
constant $C>0$. Then with probability at least $1-O\left(n^{-10}\right)$,
the Hessian $\nabla^{2}f_{\mathrm{clean}}(\bm{X})$ as defined in \eqref{eq:hessian_quadratic-MC} obeys 
\begin{align}
\mathrm{vec}\left(\bm{V}\right)^{\top}\nabla^{2}f_{\mathrm{clean}}\left(\bm{X}\right)\mathrm{vec}\left(\bm{V}\right) & \geq\frac{\sigma_{\min}}{2}\left\Vert \bm{V}\right\Vert _{\mathrm{F}}^{2}\qquad\text{and}\qquad\left\Vert \nabla^{2}f_{\mathrm{clean}}\left(\bm{X}\right)\right\Vert \leq\frac{5}{2}\sigma_{\max}\label{eq:strong-convexity-smoothness-MC}
\end{align}
for all $\bm{X}$ and $\bm{V} = \bm{Y}\bm{H}_{Y}-\bm{Z}$, with $\bm{H}_{Y}:=\arg\min_{\bm{R}\in\cO^{r\times r}}\left\Vert \bm{Y}\bm{R}-\bm{Z}\right\Vert _{\mathrm{F}}$, satisfying:
\begin{subequations}
\begin{align}
 \left\Vert \bm{X}-\bm{X}^{\star}\right\Vert _{2,\infty} & \leq\epsilon\left\Vert \bm{X}^{\star}\right\Vert _{2,\infty}, \label{eq:MC_incoherence_neighborhood} \\
\|\bm{Z}-\bm{X}^{\star}\|  & \leq\delta\|\bm{X}^{\star}\|,  \label{eq:MC_spectral_neighborhood} 
\end{align}
\end{subequations} 
where $\epsilon\ll1/\sqrt{\kappa^{3}\mu r\log^{2}n}$ and $\delta\ll1 / \kappa$.

\end{lemma}\begin{proof}See Appendix \ref{subsec:Proof-of-Lemma-hessian-MC}.
\end{proof}

Lemma \ref{lemma:hessian-MC} reveals that the Hessian matrix
is well-conditioned in a neighborhood close to $\bm{X}^{\star}$ that remains incoherent measured in the $\ell_2/\ell_\infty$ norm (cf.~\eqref{eq:MC_incoherence_neighborhood}), and along directions that point towards points which are not far away from the truth in the spectral norm (cf.~\eqref{eq:MC_spectral_neighborhood}).

\begin{remark}
The second condition \eqref{eq:MC_spectral_neighborhood} is
characterized using the spectral norm $\|\cdot\|$, while in previous
works this is typically presented in the Frobenius norm $\|\cdot\|_{\mathrm{F}}$. It is also worth noting that the Hessian matrix --- even in the infinite-sample and noiseless case --- is rank-deficient and cannot be positive definite. As a
result, we resort to the form of strong convexity by restricting attention to certain directions (see the conditions on $\bm{V}$).
\end{remark}


\subsubsection{Error contraction\label{subsec:Error-contraction-MC}}

Our goal is to demonstrate the error bounds (\ref{eq:induction_original_MC_thm})
measured in three different norms. Notably, as long as the iterates
satisfy \eqref{eq:induction_original_MC_thm} at the $t$th iteration,
then $\|\bm{X}^{t}\hat{\bm{H}}^{t}-\bm{X}^{\star}\|_{2,\infty}$
is sufficiently small. Under our sample complexity assumption,
$\bm{X}^{t}\hat{\bm{H}}^{t}$ satisfies the $\ell_{2}/\ell_{\infty}$
condition \eqref{eq:MC_incoherence_neighborhood} required in Lemma~\ref{lemma:hessian-MC}. Consequently,
we can invoke Lemma~\ref{lemma:hessian-MC} to arrive at the following
error contraction result. 

\begin{lemma}[Contraction w.r.t. the Frobenius norm]\label{lemma:induction-fro-MC}
Suppose that $n^{2}p\geq C\kappa^{3}\mu^{3}r^{3}n\log^{3}n$ for some sufficiently
large constant $C>0$, the noise satisfies \eqref{eq:mc-noise-condition}. There exists an event that does not depend on $t$ and has probability $1-O(n^{-10})$, such that when it happens and \eqref{eq:induction_original_ell_2-MC_thm}, \eqref{eq:induction_original_ell_infty-MC_thm} hold for the $t$th iteration, one has
\[
\big\|\bm{X}^{t+1}\hat{\bm{H}}^{t+1}-\bm{X}^{\star}\big\|_{\mathrm{F}}\leq C_{4}\rho^{t+1}\mu r\frac{1}{\sqrt{np}}\left\Vert \bm{X}^{\star}\right\Vert _{\mathrm{F}}+C_{1}\frac{\sigma}{\sigma_{\min}}\sqrt{\frac{n}{p}}\left\Vert \bm{X}^{\star}\right\Vert _{\mathrm{F}}
\]
provided that $0<\eta\leq{2} / ({25\kappa\sigma_{\max}})$, $1-\left({\sigma_{\min}} / {4}\right)\cdot\eta \leq \rho <1$,
and $C_{1}$ is sufficiently large.
\end{lemma} \begin{proof}The
proof is built upon Lemma~\ref{lemma:hessian-MC}. See Appendix~\ref{sec:proof-induction-fro-MC}.
\end{proof}


Further, if the current iterate satisfies all three conditions in \eqref{eq:induction_original_MC_thm}, then we can derive a stronger sense of error contraction, namely, contraction
{in terms of }the spectral norm. 

\begin{lemma}[Contraction w.r.t.~the spectral norm] 
	\label{lemma:operator-norm-contraction-MC}
Suppose $n^{2}p\geq C\kappa^{3}\mu^{3}r^{3}n\log^{3}n$ for some sufficiently
large constant $C>0$
and the noise satisfies \eqref{eq:mc-noise-condition}. There exists an event that does not depend on $t$ and has probability $1-O(n^{-10})$, such that when it happens and \eqref{eq:induction_original_MC_thm} holds for the $t$th iteration, one has
\begin{equation}
\big\|\bm{X}^{t+1}\hat{\bm{H}}^{t+1}-\bm{X}^{\star}\big\|\leq C_{9}\rho^{t+1}\mu r\frac{1}{\sqrt{np}}\left\Vert \bm{X}^{\star}\right\Vert +C_{10}\frac{\sigma}{\sigma_{\min}}\sqrt{\frac{n}{p}}\left\Vert \bm{X}^{\star}\right\Vert \label{eq:induction-op-MC}
\end{equation}
provided that $0<\eta\leq {1} / \left({2\sigma_{\max}}\right)$
and $1-\left({\sigma_{\min}} / {3}\right)\cdot\eta \leq \rho <1$.\end{lemma}
\begin{proof}The key observation is this: the iterate that proceeds
according to the population-level gradient reduces the error w.r.t.~$\|\cdot\|$,
namely,
\[
\big\|\bm{X}^{t}\hat{\bm{H}}^{t}-\eta\nabla F_{\mathrm{clean}}\big(\bm{X}^{t}\hat{\bm{H}}^{t}\big)-\bm{X}^{\star}\big\|<\big\|\bm{X}^{t}\hat{\bm{H}}^{t}-\bm{X}^{\star}\big\|,
\]
as long as $\bm{X}^{t}\hat{\bm{H}}^{t}$ is sufficiently close to the truth. Notably,
the orthonormal matrix $\hat{\bm{H}}^{t}$ is still chosen to be the
one that minimizes the $\|\cdot\|_{\mathrm{F}}$ distance (as opposed
to $\|\cdot\|$), which yields a symmetry property $\bm{X}^{\star\top}\bm{X}^{t}\hat{\bm{H}}^{t}=\big(\bm{X}^{t}\hat{\bm{H}}^{t}\big)^{\top}\bm{X}^{\star}$,
crucial for our analysis. See Appendix \ref{subsec:Proof-of-Lemma-operator-contraction-MC}
for details. \end{proof}

\subsection{Step 2: introducing the leave-one-out sequences}


In order to establish the incoherence properties \eqref{eq:induction_original_ell_infty-MC_thm}
for the entire trajectory, which is difficult to deal with directly due to the complicated statistical dependence, we introduce a collection of \emph{leave-one-out}
versions of $\left\{ \bm{X}^{t}\right\} _{t\geq0}$, denoted by $\left\{ \bm{X}^{t,(l)}\right\} _{t\geq0}$
for each $1\leq l\leq n$. Specifically, $\left\{ \bm{X}^{t,(l)}\right\} _{t\geq0}$
is the iterates of gradient descent operating on the auxiliary loss function
\begin{equation}
f^{(l)}\left(\bm{X}\right):=\frac{1}{4p}\left\Vert \mathcal{P}_{\Omega^{-l}}\left[\bm{X}\bm{X}^{\top}-\left(\bm{M}^{\star}+\bm{E}\right)\right]\right\Vert _{\mathrm{F}}^{2}+\frac{1}{4}\left\Vert \mathcal{P}_{l}\left(\bm{X}\bm{X}^{\top}-\bm{M}^{\star}\right)\right\Vert _{\mathrm{F}}^{2}.\label{eq:loo-loss-MC}
\end{equation}
Here, $\mathcal{P}_{\Omega_{l}}$ (resp.~$\mathcal{P}_{\Omega^{-l}}$
and $\mathcal{P}_{l}$) represents the orthogonal projection onto
the subspace of matrices which vanish outside of the index set $\Omega_{l}:=\left\{ \left(i,j\right)\in\Omega\mid i=l\text{ or }j=l\right\} $
(resp.~ $\Omega^{-l}:=\left\{ \left(i,j\right)\in\Omega\mid i\neq l,j\neq l\right\} $
and $\left\{ \left(i,j\right)\mid i=l\text{ or }j=l\right\} $);
that is, for any matrix $\bm{M}$, 
\begin{equation}
\left[\mathcal{P}_{\Omega_{l}}\left(\bm{M}\right)\right]_{i,j}=\begin{cases}
M_{i,j}, & \text{if }\left(i=l\text{ or }j=l\right)\text{ and }(i,j)\in\Omega,\\
0, & \text{else},
\end{cases}\label{eq:projection-Omega-l-1}
\end{equation}
\begin{equation}
\left[\mathcal{P}_{\Omega^{-l}}\left(\bm{M}\right)\right]_{i,j}=\begin{cases}
M_{i,j}, & \text{if }i\neq l\text{ and }j\neq l\text{ and }(i,j)\in\Omega\\
0, & \text{else}
\end{cases}\quad\text{and}\quad\left[\mathcal{P}_{l}\left(\bm{M}\right)\right]_{i,j}=\begin{cases}
0, & \text{if }i\neq l\text{ and }j\neq l,\\
M_{i,j}, & \text{if }i=l\text{ or }j=l.
\end{cases}\label{eq:projection-Omega-l}
\end{equation}
The gradient of the leave-one-out loss function \eqref{eq:loo-loss-MC}
is given by 
\begin{align}
\nabla f^{\left(l\right)}\left(\bm{X}\right) & =\frac{1}{p}\mathcal{P}_{\Omega^{-l}}\left[\bm{X}\bm{X}^{\top}-\left(\bm{M}^{\star}+\bm{E}\right)\right]\bm{X}+\mathcal{P}_{l}\left(\bm{X}\bm{X}^{\top}-\bm{M}^{\star}\right)\bm{X}.\label{eq:loo-gradient-MC}
\end{align}
The full algorithm to obtain the leave-one-out sequence $\{\bm{X}^{t,(l)}\}_{t\geq0}$
(including spectral initialization) is summarized in Algorithm \ref{alg:leave-one-out-gd-MC}.
\begin{algorithm}[htp]
\caption{The $l$th leave-one-out sequence for matrix completion}

\label{alg:leave-one-out-gd-MC}\begin{algorithmic}

\STATE \textbf{{Input}}: $\bm{Y}=\left[Y_{i,j}\right]_{1\leq i,j\leq n}, \bm{M}^\star_{\cdot,l},\bm{M}^\star_{l,\cdot}, r, p$. 

\STATE \textbf{{Spectral initialization}}: Let $\bm{U}^{0,(l)}\bm{\Sigma}^{\left(l\right)}\bm{U}^{0,\left(l\right)\top}$
be the top-$r$ eigendecomposition of 
\[
\bm{M}^{\left(l\right)}:=\frac{1}{p}\mathcal{P}_{\Omega^{-l}}\left(\bm{Y}\right)+\mathcal{P}_{l}\left(\bm{M}^{\star}\right)=\frac{1}{p}\mathcal{P}_{\Omega^{-l}}\left(\bm{M}^{\star}+\bm{E}\right)+\mathcal{P}_{l}\left(\bm{M}^{\star}\right)
\]
with $\mathcal{P}_{\Omega^{-l}}$ and $\mathcal{P}_{l}$ defined in \eqref{eq:projection-Omega-l},
and set $\bm{X}^{0,\left(l\right)}=\bm{U}^{0,\left(l\right)}\big(\bm{\Sigma}^{\left(l\right)}\big)^{1/2}$.

\STATE \textbf{{Gradient updates}}: \textbf{for} $t=0,1,2,\ldots,T-1$
\textbf{do}

\STATE 
\begin{equation}
\bm{X}^{t+1,(l)}=\bm{X}^{t,(l)}-\eta_{t}\nabla f^{(l)}\big(\bm{X}^{t,(l)}\big).\label{eq:loo-gradient_update-MC}
\end{equation}

\end{algorithmic} 
\end{algorithm}

\begin{remark}Rather than simply dropping all samples in the $l$th
row/column, we replace the $l$th row/column with their respective
population means. In other words, the leave-one-out gradient forms
an unbiased surrogate for the true gradient, which is particularly
important in ensuring high estimation accuracy. \end{remark}


\subsection{Step 3: establishing the incoherence condition by induction\label{subsec:Step-3-MC}}

We will continue the proof of Theorem \ref{thm:main-MC} in an inductive manner. As seen in Section \ref{subsec:Error-contraction-MC},
the induction hypotheses \eqref{eq:induction_original_ell_2-MC_thm}
and \eqref{eq:induction_original_operator-MC_thm} hold for the $(t+1)$th iteration as long as \eqref{eq:induction_original_MC_thm} holds at the $t$th iteration. Therefore,
we are left with proving the incoherence hypothesis \eqref{eq:induction_original_ell_infty-MC_thm} for all $0\leq t\leq T=O(n^{5})$.
 For clarity
of analysis, it is crucial to maintain a list of induction hypotheses,
which includes a few more hypotheses that complement \eqref{eq:induction_original_MC_thm}, and is given below.
\begin{subequations} \label{eq:induction_original_MC} 
\begin{align}
\big\|\bm{X}^{t}\hat{\bm{H}}^{t}-\bm{X}^{\star}\big\|_{\mathrm{F}} & \leq\left(C_{4}\rho^{t}\mu r\frac{1}{\sqrt{np}}+C_{1}\frac{\sigma}{\sigma_{\min}}\sqrt{\frac{n}{p}}\right)\left\Vert \bm{X}^{\star}\right\Vert _{\mathrm{F}},\label{eq:induction_original_ell_2-MC}\\
\big\|\bm{X}^{t}\hat{\bm{H}}^{t}-\bm{X}^{\star}\big\|_{2,\infty} & \leq\left(C_{5}\rho^{t}\mu r\sqrt{\frac{\log n}{np}}+C_{8}\frac{\sigma}{\sigma_{\min}}\sqrt{\frac{n\log n}{p}}\right)\left\Vert \bm{X}^{\star}\right\Vert _{2,\infty},\label{eq:induction_original_ell_infty-MC}\\
\big\|\bm{X}^{t}\hat{\bm{H}}^{t}-\bm{X}^{\star}\big\| & \leq\left(C_{9}\rho^{t}\mu r\frac{1}{\sqrt{np}}+C_{10}\frac{\sigma}{\sigma_{\min}}\sqrt{\frac{n}{p}}\right)\left\Vert \bm{X}^{\star}\right\Vert ,\label{eq:induction_original_operator-MC}\\
\max_{1\leq l \leq n}\big\|\bm{X}^{t}\hat{\bm{H}}^{t}-\bm{X}^{t,\left(l\right)}\bm{R}^{t,\left(l\right)}\big\|_{\mathrm{F}} & \leq\left(C_{3}\rho^{t}\mu r\sqrt{\frac{\log n}{np}}+C_{7}\frac{\sigma}{\sigma_{\min}}\sqrt{\frac{n\log n}{p}}\right)\left\Vert \bm{X}^{\star}\right\Vert _{2,\infty},\label{eq:induction_ell_2_diff-MC}\\
\max_{1\leq l \leq n}\big\|\big(\bm{X}^{t,\left(l\right)}\hat{\bm{H}}^{t,\left(l\right)}-\bm{X}^{\star}\big)_{l,\cdot}\big\|_{2} & \leq\left(C_{2}\rho^{t}\mu r\frac{1}{\sqrt{np}}+C_{6}\frac{\sigma}{\sigma_{\min}}\sqrt{\frac{n\log n}{p}}\right)\left\Vert \bm{X}^{\star}\right\Vert _{2,\infty} \label{eq:induction_loo-ell_infty-MC}
\end{align} 
\end{subequations}
hold for some absolute constants $0<\rho<1$ and $C_{1},\cdots,C_{10}>0$.
Here, $\hat{\bm{H}}^{t,\left(l\right)}$ and $\bm{R}^{t,(l)}$ are
orthonormal matrices defined by 
\begin{align}
\hat{\bm{H}}^{t,\left(l\right)} & :=\arg\min_{\bm{R}\in\cO^{r\times r}}\left\Vert \bm{X}^{t,\left(l\right)}\bm{R}-\bm{X}^{\star}\right\Vert _{\mathrm{F}},\label{eq:rotation-hat-h-l}\\
\bm{R}^{t,\left(l\right)} & :=\arg\min_{\bm{R}\in\cO^{r\times r}}\big\|\bm{X}^{t,\left(l\right)}\bm{R}-\bm{X}^{t}\hat{\bm{H}}^{t}\big\|_{\mathrm{F}}.\label{eq:rotation_r_t_l}
\end{align}
Clearly, the first three hypotheses \eqref{eq:induction_original_ell_2-MC}-\eqref{eq:induction_original_operator-MC}
constitute the conclusion of Theorem \ref{thm:main-MC}, i.e.~\eqref{eq:induction_original_MC_thm}. The last
two hypotheses \eqref{eq:induction_ell_2_diff-MC} and \eqref{eq:induction_loo-ell_infty-MC}
are auxiliary properties connecting the true iterates and the auxiliary leave-one-out
sequences. Moreover, we summarize below several immediate consequences of \eqref{eq:induction_original_MC},
which will be useful throughout.

\begin{lemma}\label{lemma:MC-hypothesis-consequence} Suppose $n^{2}p \geq C\kappa^{3}\mu^{2}r^{2}n\log n$ for some sufficiently large constant $C>0$ and the noise satisfies \eqref{eq:mc-noise-condition}.
Under the hypotheses
\eqref{eq:induction_original_MC}, one has \begin{subequations}\label{eq:auxiliary-hypotheses-MC}
\begin{align}
\left\Vert \bm{X}^{t}\hat{\bm{H}}^{t}-\bm{X}^{t,(l)}\hat{\bm{H}}^{t,\left(l\right)}\right\Vert _{\mathrm{F}} & \leq5\kappa\left\Vert \bm{X}^{t}\hat{\bm{H}}^{t}-\bm{X}^{t,(l)}\bm{R}^{t,\left(l\right)}\right\Vert _{\mathrm{F}}, \label{eq:combine-fro} \\
\big\|\bm{X}^{t,\left(l\right)}\hat{\bm{H}}^{t,\left(l\right)}-\bm{X}^{\star}\big\|_{\mathrm{F}} & \leq\left\Vert \bm{X}^{t,\left(l\right)}\bm{R}^{t,\left(l\right)}-\bm{X}^{\star}\right\Vert _{\mathrm{F}}\leq\left\{ 2C_{4}\rho^{t}\mu r\frac{1}{\sqrt{np}}+2C_{1}\frac{\sigma}{\sigma_{\min}}\sqrt{\frac{n}{p}}\right\} \left\Vert \bm{X}^{\star}\right\Vert _{\mathrm{F}},\label{eq:implication-finer-fro-hat-t-l}\\
\big\|\bm{X}^{t,\left(l\right)}\bm{R}^{t,\left(l\right)}-\bm{X}^{\star}\big\|_{2,\infty} & \leq\left\{ \left(C_{3}+C_{5}\right)\rho^{t}\mu r\sqrt{\frac{\log n}{np}}+\left(C_{8}+C_{7}\right)\frac{\sigma}{\sigma_{\min}}\sqrt{\frac{n\log n}{p}}\right\} \left\Vert \bm{X}^{\star}\right\Vert _{2,\infty},\label{eq:implication-finer-2-inf-R-t-l}\\
\big\|\bm{X}^{t,\left(l\right)}\hat{\bm{H}}^{t,\left(l\right)}-\bm{X}^{\star}\big\| & \leq\left\{ 2C_{9}\rho^{t}\mu r\frac{1}{\sqrt{np}}+2C_{10}\frac{\sigma}{\sigma_{\min}}\sqrt{\frac{n}{p}}\right\} \left\Vert \bm{X}^{\star}\right\Vert .\label{eq:implication-finer-op-hat-t-l}
\end{align}
In particular, \eqref{eq:combine-fro} follows from hypotheses \eqref{eq:induction_original_operator-MC} and \eqref{eq:induction_ell_2_diff-MC}.
\end{subequations}

\end{lemma}\begin{proof}See Appendix \ref{subsec:Proof-of-Lemma-MC-hypothesis-consequence}.\end{proof}

In the sequel, we follow the general recipe outlined in Section~\ref{sec:A-general-recipe} to establish the induction hypotheses. We only need to establish \eqref{eq:induction_original_ell_infty-MC}, \eqref{eq:induction_ell_2_diff-MC} and \eqref{eq:induction_loo-ell_infty-MC} for the $(t+1)$th iteration, since \eqref{eq:induction_original_ell_2-MC} and \eqref{eq:induction_original_operator-MC} have been established in Section~\ref{subsec:Error-contraction-MC}. Specifically, we resort to the leave-one-out iterates by showing
that: first, the true and the auxiliary iterates remain exceedingly close
throughout; second, the $l$th leave-one-out sequence stays incoherent
with $\bm{e}_{l}$ due to statistical independence. 
\begin{itemize}
\item \textbf{Step 3(a): proximity between the original and the leave-one-out
iterates.} We demonstrate that $\bm{X}^{t+1}$ is well approximated
by $\bm{X}^{t+1,(l)}$, up to proper orthonormal transforms. This is precisely
the induction hypothesis \eqref{eq:induction_ell_2_diff-MC} for the $(t+1)$th iteration.\begin{lemma}\label{lemma:loop-MC}Suppose
the sample complexity satisfies $n^{2}p\geq C \kappa^{4}\mu^{3}r^{3}n\log^{3}n$ for some sufficiently large constant $C>0$ 
and the noise satisfies \eqref{eq:mc-noise-condition}. Let $\mathcal{E}_t$ be the event where the
hypotheses in \eqref{eq:induction_original_MC} hold for the $t$th iteration. Then on some event $\mathcal{E}_{t+1,1} \subseteq \mathcal{E}_t$ obeying $\PP ( \mathcal{E}_t \cap \mathcal{E}_{t+1,1}^c ) = O(n^{-10})$, we have 
\begin{equation}
\left\Vert \bm{X}^{t+1}\hat{\bm{H}}^{t+1}-\bm{X}^{t+1,(l)}\bm{R}^{t+1,\left(l\right)}\right\Vert _{\mathrm{F}}\leq C_{3}\rho^{t+1}\mu r\sqrt{\frac{\log n}{np}}\left\Vert \bm{X}^{\star}\right\Vert _{2,\infty}+C_{7}\frac{\sigma}{\sigma_{\min}}\sqrt{\frac{n\log n}{p}}\left\Vert \bm{X}^{\star}\right\Vert _{2,\infty}\label{eq:induction-loop-MC}
\end{equation}
provided that $0<\eta\leq{2} / {\left(25\kappa\sigma_{\max}\right)}$, $1-\left({\sigma_{\min}} / {5}\right)\cdot\eta \leq \rho <1$ and $C_{7} > 0$ is sufficiently large. 
\end{lemma}
\begin{proof}
The fact that this difference is well-controlled relies heavily on the benign geometric property of the Hessian revealed by Lemma \ref{lemma:hessian-MC}.
Two important remarks are in order: (1) both points $\bm{X}^{t}\hat{\bm{H}}^{t}$
and $\bm{X}^{t,(l)}\bm{R}^{t,\left(l\right)}$ satisfy \eqref{eq:MC_incoherence_neighborhood}; (2) the difference
$\bm{X}^{t}\hat{\bm{H}}^{t}-\bm{X}^{t,(l)}\bm{R}^{t,\left(l\right)}$ forms a valid direction for restricted strong convexity. These two properties together allow
us to invoke Lemma \ref{lemma:hessian-MC}. See Appendix \ref{subsec:Proof-of-Lemma-loop-MC}.\end{proof}
\item \textbf{Step 3(b): incoherence of the leave-one-out iterates. }Given
that $\bm{X}^{t+1,(l)}$ is sufficiently close to $\bm{X}^{t+1}$,
we turn our attention to establishing the incoherence of this surrogate $\bm{X}^{t+1,(l)}$
w.r.t.~$\bm{e}_{l}$. This amounts to proving the induction hypothesis \eqref{eq:induction_loo-ell_infty-MC}
for the $(t+1)$th iteration. 
\begin{lemma}
\label{lemma:looe-MC}
Suppose the sample
complexity meets $n^{2}p\geq C \kappa^{3}\mu^{3}r^{3}n\log^{3}n$ for some sufficiently large constant $C>0$ and
the noise satisfies \eqref{eq:mc-noise-condition}. 
Let $\mathcal{E}_t$ be the event where the
hypotheses in \eqref{eq:induction_original_MC} hold for the $t$th iteration.
Then on some event $\mathcal{E}_{t+1,2} \subseteq \mathcal{E}_t$ obeying $\PP ( \mathcal{E}_t \cap \mathcal{E}_{t+1,2}^c ) = O(n^{-10})$, we have 
\begin{equation}
\left\Vert \big(\bm{X}^{t+1,\left(l\right)}\hat{\bm{H}}^{t+1,\left(l\right)}-\bm{X}^{\star}\big)_{l,\cdot}\right\Vert _{2}\leq C_{2}\rho^{t+1}\mu r\frac{1}{\sqrt{np}}\left\Vert \bm{X}^{\star}\right\Vert _{2,\infty}+C_{6}\frac{\sigma}{\sigma_{\min}}\sqrt{\frac{n\log n}{p}}\left\Vert \bm{X}^{\star}\right\Vert _{2,\infty}\label{eq:induction-looe-MC}
\end{equation}
so long as $0<\eta\leq 1/{\sigma_{\max}}$,
$1-\left({\sigma_{\min}} / {3}\right)\cdot\eta \leq \rho <1$,  $ C_{2} \gg \kappa C_{9} $ and $ C_{6} \gg \kappa C_{10} / \sqrt{\log n} $.
\end{lemma}
\begin{proof}The key observation
is that $\bm{X}^{t+1,(l)}$ is statistically independent from any
sample in the $l$th row/column of the matrix. Since there are an
order of $np$ samples in each row/column, we obtain enough information
that helps establish the desired incoherence property. See Appendix
\ref{subsec:Proof-of-Lemma-looe-MC}.
\end{proof}
\item \textbf{Step 3(c): combining the bounds.} The inequalities \eqref{eq:induction_ell_2_diff-MC} and \eqref{eq:induction_loo-ell_infty-MC} taken collectively allow us
to establish the induction hypothesis \eqref{eq:induction_original_ell_infty-MC}. Specifically,
for every $1\leq l\leq n$, write
\begin{align*}
\big(\bm{X}^{t+1}\hat{\bm{H}}^{t+1}-\bm{X}^{\star}\big)_{l,\cdot} & =\big(\bm{X}^{t+1}\hat{\bm{H}}^{t+1}-\bm{X}^{t+1,(l)}\hat{\bm{H}}^{t+1,\left(l\right)}\big)_{l,\cdot}+\big(\bm{X}^{t+1,\left(l\right)}\hat{\bm{H}}^{t+1,\left(l\right)}-\bm{X}^{\star}\big)_{l,\cdot},
\end{align*}
and the triangle inequality gives
\begin{align}
\big\|\big(\bm{X}^{t+1}\hat{\bm{H}}^{t+1}-\bm{X}^{\star}\big)_{l,\cdot}\big\|_{2} & \leq\big\|\bm{X}^{t+1}\hat{\bm{H}}^{t+1}-\bm{X}^{t+1,(l)}\hat{\bm{H}}^{t+1,\left(l\right)}\big\|_{\mathrm{F}}+\big\|\big(\bm{X}^{t+1,\left(l\right)}\hat{\bm{H}}^{t+1,\left(l\right)}-\bm{X}^{\star}\big)_{l,\cdot}\big\|_{2}.\label{eq:triangle_ell_infty}
\end{align}
The second term has already been bounded by \eqref{eq:induction-looe-MC}.
Since we have established the induction hypotheses \eqref{eq:induction_original_operator-MC} and \eqref{eq:induction_ell_2_diff-MC} for the $(t+1)$th iteration, the first term can be bounded by (\ref{eq:combine-fro}) for the $(t+1)$th iteration, i.e. 
\[
\left\Vert \bm{X}^{t+1}\hat{\bm{H}}^{t+1}-\bm{X}^{t+1,(l)}\hat{\bm{H}}^{t+1,\left(l\right)}\right\Vert _{\mathrm{F}} \leq5\kappa\left\Vert \bm{X}^{t+1}\hat{\bm{H}}^{t+1}-\bm{X}^{t+1,(l)}\bm{R}^{t+1,\left(l\right)}\right\Vert _{\mathrm{F}}.
\]
Plugging the above inequality, \eqref{eq:induction-loop-MC} and \eqref{eq:induction-looe-MC}
into (\ref{eq:triangle_ell_infty}), we have 
\begin{align*}
\left\Vert \bm{X}^{t+1}\hat{\bm{H}}^{t+1}-\bm{X}^{\star}\right\Vert _{2,\infty} & \leq5\kappa\left(C_{3}\rho^{t+1}\mu r\sqrt{\frac{\log n}{np}}\left\Vert \bm{X}^{\star}\right\Vert _{2,\infty}+\frac{C_{7}}{\sigma_{\min}}\sigma\sqrt{\frac{n\log n}{p}}\left\Vert \bm{X}^{\star}\right\Vert _{2,\infty}\right)\\
 & \quad+C_{2}\rho^{t+1}\mu r\frac{1}{\sqrt{np}}\left\Vert \bm{X}^{\star}\right\Vert _{2,\infty}+\frac{C_{6}}{\sigma_{\min}}\sigma\sqrt{\frac{n\log n}{p}}\left\Vert \bm{X}^{\star}\right\Vert _{2,\infty}\\
 & \leq C_{5}\rho^{t+1}\mu r\sqrt{\frac{\log n}{np}}\left\Vert \bm{X}^{\star}\right\Vert _{2,\infty}+\frac{C_{8}}{\sigma_{\min}}\sigma\sqrt{\frac{n\log n}{p}}\left\Vert \bm{X}^{\star}\right\Vert _{2,\infty}
\end{align*}
as long as $C_{5}/(\kappa C_{3}+C_{2})$ and $C_{8}/(\kappa C_{7}+C_{6})$ are sufficiently large. This establishes the induction hypothesis \eqref{eq:induction_original_ell_infty-MC}. From the deduction above we see $\mathcal{E}_t \cap \mathcal{E}_{t+1}^c = O(n^{-10})$ and thus finish the proof. 
\end{itemize}

\subsection{The base case: spectral initialization}


Finally, we return to check the base case, namely, we aim to show
that the spectral initialization satisfies the induction hypotheses
\eqref{eq:induction_original_ell_2-MC}-\eqref{eq:induction_loo-ell_infty-MC} for $t=0$.
This is accomplished via the following lemma. 
\begin{lemma}
\label{lemma:spectral-MC}
Suppose the sample size obeys $n^{2}p\geq C \mu^{2}r^{2}n\log n$ for some sufficiently large constant $C>0$, the noise
satisfies (\ref{eq:mc-noise-condition}), and $\kappa=\sigma_{\max}/\sigma_{\min}\asymp 1$. Then with probability at
least $1-O\left(n^{-10}\right)$, the claims in \eqref{eq:induction_original_ell_2-MC}-\eqref{eq:induction_loo-ell_infty-MC}
hold simultaneously for $t=0$. 
\end{lemma}
\begin{proof}This follows by invoking
the Davis-Kahan sin$\Theta$ theorem \cite{davis1970rotation} as
well as the entrywise eigenvector perturbation analysis in \cite{abbe2017entrywise}.
We defer the proof to Appendix \ref{subsec:Proof-of-Lemma-spectral-MC}.
\end{proof}



\section{Analysis for blind deconvolution\label{sec:Blind-deconvolution}}

In this section, we instantiate the general recipe presented in Section
\ref{sec:A-general-recipe} to blind deconvolution and prove Theorem \ref{thm:main-BD}. Without loss of generality, we assume throughout that $\left\Vert \bm{h}^{\star}\right\Vert _{2}=\left\Vert \bm{x}^{\star}\right\Vert _{2}=1$.

Before presenting the analysis, we first gather some simple facts
about the empirical loss function in (\ref{eq:loss-BD}). Recall the definition of $\bm{z}$ in \eqref{eq:def_zz}, and for notational simplicity, we write $f\left(\bm{z}\right)=f(\bm{h},\bm{x})$. Since $\bm{z}$ is complex-valued, we need to resort to Wirtinger calculus; see \cite[Section 6]{candes2014wirtinger} for
a brief introduction. The Wirtinger gradient
of (\ref{eq:loss-BD}) with respect to $\bm{h}$ and $\bm{x}$ are
given respectively by 
\begin{align}
\nabla_{\bm{h}}f\left(\bm{z}\right)= \nabla_{\bm{h}}f\left(\bm{h},\bm{x}\right) & =\sum_{j=1}^{m}\left(\bm{b}_{j}^{\conj}\bm{h}\bm{x}^{\conj}\bm{a}_{j}-y_{j}\right)\bm{b}_{j}\bm{a}_{j}^{\conj}\bm{x};\label{eq:grad-h-BD}\\
\nabla_{\bm{x}}f\left(\bm{z}\right)= \nabla_{\bm{x}}f\left(\bm{h},\bm{x}\right) & =\sum_{j=1}^{m}\overline{(\bm{b}_{j}^{\conj}\bm{h}\bm{x}^{\conj}\bm{a}_{j}-y_{j})}\bm{a}_{j}\bm{b}_{j}^{\conj}\bm{h}.\label{eq:grad-x-BD}
\end{align}
It is worth noting that the formal Wirtinger gradient contains $\nabla_{\overline{\bm{h}}}f\left(\bm{h},\bm{x}\right)$
and $\nabla_{\overline{\bm{x}}}f\left(\bm{h},\bm{x}\right)$ as well.
Nevertheless, since $f\left(\bm{h},\bm{x}\right)$ is a real-valued function, the
following identities always hold 
\[
\nabla_{\bm{h}}f\left(\bm{h},\bm{x}\right)=\overline{\nabla_{\overline{\bm{h}}}f\left(\bm{h},\bm{x}\right)}\qquad\text{and}\qquad\nabla_{\bm{x}}f\left(\bm{h},\bm{x}\right)=\overline{\nabla_{\overline{\bm{x}}}f\left(\bm{h},\bm{x}\right)}.
\]
In light of these observations, one often omits the gradient with respect
to the conjugates; correspondingly, the gradient update rule \eqref{eq:gradient_update_BD} can be written as 
\begin{subequations}\label{eq:gradient-update-Bd-explicit}
\begin{align}
\bm{h}^{t+1} & =\bm{h}^{t}-\frac{\eta}{\left\Vert \bm{x}^{t}\right\Vert _{2}^{2}}\sum_{j=1}^{m}\left(\bm{b}_{j}^{\conj}\bm{h}^{t}\bm{x}^{t\conj}\bm{a}_{j}-y_{j}\right)\bm{b}_{j}\bm{a}_{j}^{\conj}\bm{x}^{t},\label{eq:gradient-update-h-Bd}\\
\bm{x}^{t+1} & =\bm{x}^{t}-\frac{\eta}{\left\Vert \bm{h}^{t}\right\Vert _{2}^{2}}\sum_{j=1}^{m}\overline{(\bm{b}_{j}^{\conj}\bm{h}^{t}\bm{x}^{t\conj}\bm{a}_{j}-y_{j})}\bm{a}_{j}\bm{b}_{j}^{\conj}\bm{h}^{t}.\label{eq:gradient-update-x-BD}
\end{align}
\end{subequations}

We can also compute the Wirtinger Hessian of $f(\bm{z})$ as follows, 
\begin{equation}
\nabla^{2}f\left(\bm{z}\right)=\left[\begin{array}{cc}
\bm{A} & \bm{B}\\
\bm{B}^{\conj} & \overline{\bm{A}}
\end{array}\right],\label{eq:hessian-BD}
\end{equation}
where 
\[
\bm{A}=\left[\begin{array}{cc}
\sum_{j=1}^{m}\left|\bm{a}_{j}^{\conj}\bm{x}\right|^{2}\bm{b}_{j}\bm{b}_{j}^{\conj} & \sum_{j=1}^{m}\left(\bm{b}_{j}^{\conj}\bm{h}\bm{x}^{\conj}\bm{a}_{j}-y_{j}\right)\bm{b}_{j}\bm{a}_{j}^{\conj}\\
\sum_{j=1}^{m}\left[\left(\bm{b}_{j}^{\conj}\bm{h}\bm{x}^{\conj}\bm{a}_{j}-y_{j}\right)\bm{b}_{j}\bm{a}_{j}^{\conj}\right]^{\conj} & \sum_{j=1}^{m}\left|\bm{b}_{j}^{\conj}\bm{h}\right|^{2}\bm{a}_{j}\bm{a}_{j}^{\conj}
\end{array}\right]\in\CC^{2K\times2K};
\]
\[
\bm{B}=\left[\begin{array}{cc}
\bm{0} & \sum_{j=1}^{m}\bm{b}_{j}\bm{b}_{j}^{\conj}\bm{h}\left(\bm{a}_{j}\bm{a}_{j}^{\conj}\bm{x}\right)^{\top}\\
\sum_{j=1}^{m}\bm{a}_{j}\bm{a}_{j}^{\conj}\bm{x}\left(\bm{b}_{j}\bm{b}_{j}^{\conj}\bm{h}\right)^{\top} & \bm{0}
\end{array}\right]\in\CC^{2K\times2K}.
\]

Last but not least, we say $\left(\bm{h}_{1},\bm{x}_{1}\right)$ is aligned with $\left(\bm{h}_{2},\bm{x}_{2}\right)$, if the following holds,
\[
\left\Vert \bm{h}_{1}-\bm{h}_{2}\right\Vert _{2}^{2}+\left\Vert \bm{x}_{1}-\bm{x}_{2}\right\Vert _{2}^{2}=\min_{\alpha\in\CC}\left\{ \left\Vert \frac{1}{\overline{\alpha}}\bm{h}_{1}-\bm{h}_{2}\right\Vert _{2}^{2}+\left\Vert \alpha\bm{x}_{1}-\bm{x}_{2}\right\Vert _{2}^{2}\right\} .
\]
To simplify notations, define $\tilde{\bm{z}}^t$ as
\begin{equation}\label{eq:tilde-notaion-BD}
\tilde{\bm{z}}^t = \begin{bmatrix}
\tilde{\bm{h}}^{t} \\
\tilde{\bm{x}}^{t}
\end{bmatrix} :=\begin{bmatrix}
\frac{1}{\overline{\alpha^{t}}}\bm{h}^{t}\\
\alpha^{t}\bm{x}^{t}
\end{bmatrix}
\end{equation}
with the alignment parameter $\alpha^{t}$ given in \eqref{eq:defn-alphat}. Then we can see that $\tilde{\bm{z}}^t$ is aligned with $\bm{z}^{\star}$ and 
\[
\text{dist}\left(\bm{z}^{t},\bm{z}^{\star}\right)=\text{dist}\left(\tilde{\bm{z}}^{t},\bm{z}^{\star}\right)=\left\Vert \tilde{\bm{z}}^t  - \bm{z}^{\star} \right\Vert _{2}.
\]

\subsection{Step 1: characterizing local geometry in the RIC}

\subsubsection{Local geometry}

The first step is to characterize the region of incoherence and contraction (RIC), where the empirical loss function enjoys restricted strong convexity and smoothness properties. To this end, we have the following lemma.

\begin{lemma}[Restricted strong convexity and smoothness for blind
deconvolution]\label{lemma:hessian-bd}
Let $c>0$ be a sufficiently
small constant and $$\delta = c/\log^2 m.$$ Suppose the sample size satisfies $m\geq c_0 \mu^{2}K\log^{9}m$ for some sufficiently large constant $c_0 >0$.
Then with probability $1-O\left(m^{-10} + e^{-K} \log m \right)$, the Wirtinger
Hessian $\nabla^{2}f\left(\bm{z}\right)$ obeys 
\[
\bm{u}^{\conj}\left[\bm{D}\nabla^{2}f\left(\bm{z}\right)+\nabla^{2}f\left(\bm{z}\right)\bm{D}\right]\bm{u}\geq({1}/{4})\cdot\left\Vert \bm{u}\right\Vert _{2}^{2}\qquad\text{and}\qquad\left\Vert \nabla^{2}f\left(\bm{z}\right)\right\Vert \leq3
\]
simultaneously for all
\[
\bm{z}=\left[\begin{array}{c}
\bm{h}\\
\bm{x}
\end{array}\right]\qquad\text{and}\qquad\bm{u}
=\left[\begin{array}{c}
\bm{h}_{1}-\bm{h}_{2}\\
\bm{x}_{1}-\bm{x}_{2}\\
\overline{\bm{h}_{1}-\bm{h}_{2}}\\
\overline{\bm{x}_{1}-\bm{x}_{2}}
\end{array}\right]\qquad\text{and}\qquad\bm{D}=\left[\begin{array}{cccc}
\gamma_{1}\bm{I}_{K}\\
 & \gamma_{2}\bm{I}_{K}\\
 &  & \gamma_{1}\bm{I}_{K}\\
 &  &  & \gamma_{2}\bm{I}_{K}
\end{array}\right],
\]
where $\bm{z}$ satisfies 
\begin{subequations}\label{eq:condition_z}
\begin{align} 
\max\left\{ \left\Vert \bm{h}-\bm{h}^{\star}\right\Vert _{2},\left\Vert \bm{x}-\bm{x}^{\star}\right\Vert _{2}\right\} & \leq\delta; \label{eq:condition_l2_ball} \\ 
 \max_{1\leq j\leq m}\left|\bm{a}_{j}^{\conj}\left(\bm{x}-\bm{x}^{\star}\right)\right|& \leq2C_{3}\frac{1}{\log^{3/2}m}; \label{eq:condition_incoherence_a}\\
\max_{1\leq j\leq m}\left|\bm{b}_{j}^{\conj}\bm{h}\right| &\leq2C_{4}\frac{\mu}{\sqrt{m}}\log^{2}m;\label{eq:condition_incoherence_b}
 \end{align}
 \end{subequations}
$\left(\bm{h}_{1},\bm{x}_{1}\right)$ is aligned with $\left(\bm{h}_{2},\bm{x}_{2}\right)$, and they satisfy 
\begin{align} \label{eq:condition_u}
 \max\left\{ \left\Vert \bm{h}_{1}-\bm{h}^{\star}\right\Vert _{2},\left\Vert \bm{h}_{2}-\bm{h}^{\star}\right\Vert _{2},\left\Vert \bm{x}_{1}-\bm{x}^{\star}\right\Vert _{2},\left\Vert \bm{x}_{2}-\bm{x}^{\star}\right\Vert _{2}\right\} \leq\delta ;
\end{align}
and finally, $\bm{D}$ satisfies for $\gamma_{1},\gamma_{2}\in\RR$,
\begin{equation}\label{eq:condition_D}
\max\left\{ \left|\gamma_{1}-1\right|,\left|\gamma_{2}-1\right|\right\} \leq\delta.
\end{equation} 
Here, $C_{3},C_{4}>0$ are numerical constants.

\end{lemma} \begin{proof}See Appendix \ref{subsec:Proof-of-Lemma-hessian}.\end{proof}

Lemma \ref{lemma:hessian-bd} characterizes the restricted strong
convexity and smoothness of the loss function used in blind deconvolution. To the best of our knowledge, this provides the first characterization
regarding geometric properties of the Hessian matrix for blind deconvolution. A few interpretations are in order. 
\begin{itemize}
\item The conditions \eqref{eq:condition_z} specify the region of incoherence and contraction
(RIC). In particular, \eqref{eq:condition_l2_ball} specifies a neighborhood that is close to the ground truth in $\ell_2$ norm, and \eqref{eq:condition_incoherence_a} and \eqref{eq:condition_incoherence_b} specify the incoherence region with respect to the sensing vectors $\{\bm{a}_j\}$ and $\{\bm{b}_j\}$, respectively.
\item Similar to matrix completion, the Hessian matrix is rank-deficient
even at the population level. Consequently, we resort to a restricted
form of strong convexity by focusing on certain directions.
More specifically, these directions can be viewed as the difference
between two pre-aligned points that are not far from the truth, which is characterized by \eqref{eq:condition_u}. 
\item Finally, the diagonal matrix $\bm{D}$ accounts for scaling factors that are not too far from $1$ (see \eqref{eq:condition_D}), which allows us to account for different step sizes employed for $\bm{h}$ and $\bm{x}$. 
\end{itemize}

\subsubsection{Error contraction}

The restricted strong convexity and smoothness allow us to establish the 
contraction of the error measured in terms of $\text{dist}(\cdot, \bm{z}^{\star})$ as defined in (\ref{eq:defn-dist-BD}) as long as the iterates stay in the RIC. \begin{lemma}\label{lemma:inductive-ell-2}Suppose
the number of measurements satisfies $m\geq C \mu^{2}K\log^{9}m$ for some sufficiently large constant $C>0$, and
the step size $\eta>0$ is some sufficiently small constant. There exists an event that does not depend on $t$ and has probability $1-O\left(m^{-10} + e^{-K} \log m \right)$, such that when it happens and
\begin{subequations}
	\label{eq:contraction-BD}
\begin{align}
\mathrm{dist}\left(\bm{z}^{t},\bm{z}^{\star}\right) & \leq\xi,\label{eq:induction-dist-ell-2-BD}\\
\max_{1\leq j\leq m}\left|\bm{a}_{j}^{\conj}\left(\tilde{\bm{x}}^{t}-\bm{x}^{\star}\right)\right| & \leq C_{3}\frac{1}{\log^{1.5}m},\label{eq:induction-incoherence-a-BD}\\
\max_{1\leq j\leq m}\left|\bm{b}_{j}^{\conj}\tilde{\bm{h}}^{t}\right| & \leq C_{4}\frac{\mu}{\sqrt{m}}\log^{2}m\label{eq:induction-incoherence-b-BD}
\end{align}
\end{subequations}
hold for some constants $C_{3},C_{4}>0$, one has
\[
\mathrm{dist}\left(\bm{z}^{t+1},\bm{z}^{\star}\right)\leq(1-\eta/16)\,\mathrm{dist}\left(\bm{z}^{t},\bm{z}^{\star}\right).
\]
Here, $\tilde{\bm{h}}^{t}$
and $\tilde{\bm{x}}^{t}$ are defined in (\ref{eq:tilde-notaion-BD}),
and $\xi \ll 1/\log^2 m$. 
\end{lemma}

\begin{proof}See
Appendix \ref{subsec:Proof-of-Lemma-inductive-ell-2}.\end{proof}

As a result, if $\bm{z}^{t}$ satisfies the condition (\ref{eq:contraction-BD})
for all $0\leq t\leq T$, then
\[
\mathrm{dist}\left(\bm{z}^{t},\bm{z}^{\star}\right)\leq\rho\,\mathrm{dist}\left(\bm{z}^{t-1},\bm{z}^{\star}\right)\leq\rho^{t}\mathrm{dist}\left(\bm{z}^{0},\bm{z}^{\star}\right)\leq\rho^{t}c_{1},\qquad0<t\leq T,
\]
where $\rho:= 1- \eta/16$. Furthermore, similar to the case of phase retrieval (i.e.~Lemma \ref{lemma:WF-t-iteration-enough}),
as soon as we demonstrate that the conditions (\ref{eq:contraction-BD})
hold for all $0\leq t\leq m$, then Theorem \ref{thm:main-BD} holds
true. The proof of this claim is exactly the same as for Lemma \ref{lemma:WF-t-iteration-enough},
and is thus omitted for conciseness. In what follows, we focus on
establishing (\ref{eq:contraction-BD}) for all $0\leq t\leq m$. 

Before concluding this subsection, we make note of another important
result that concerns the alignment parameter $\alpha^{t}$, which will be useful in the subsequent analysis. 
Specifically, the alignment parameter sequence $\{\alpha^{t}\}$ converges linearly to a constant whose magnitude is fairly close to 1, as long as the two initial vectors $\bm{h}^0$ and $\bm{x}^0$ have similar $\ell_2$ norms and are close to the truth. Given that $\alpha^t$ determines the global scaling of the iterates, this reveals rapid convergence of both $\|\bm{h}^t\|_2$ and $\|\bm{x}^t\|_2$, which
explains why there is no need to impose  extra terms to regularize the $\ell_2$ norm  as employed in \cite{DBLP:journals/corr/LiLSW16,huang2017blind}.

\begin{lemma}
\label{lemma:alpha-t+1-BD}
When $m > 1$ is sufficiently large, the following two claims hold true. 
\begin{itemize}
\item If $\big||\alpha^{t}|-1\big|\leq 1/2$ and $\mathrm{dist}(\bm{z}^{t},\bm{z}^{\star})\leq C_{1}/\log^{2}m$,
then 
\[
\left|\frac{\alpha^{t+1}}{\alpha^{t}}-1\right|\leq c\,\mathrm{dist}(\bm{z}^{t},\bm{z}^{\star})\leq\frac{cC_{1}}{\log^{2}m}
\]
for some absolute constant $c>0$; 
\item If $\big||\alpha^{0}|-1\big|\leq 1/4$ and $\mathrm{dist}(\bm{z}^{s},\bm{z}^{\star})\leq C_{1}(1-\eta/16)^{s}/\log^{2}m$
for all $0\leq s\leq t$, then one has 
\[
\big||\alpha^{s+1}|-1\big|\leq 1/2,\qquad0\leq s\leq t.
\]
\end{itemize}
\end{lemma}\begin{proof}See Appendix \ref{subsec:Proof-of-Lemma-inductive-ell-2}.\end{proof}
The initial condition $\big||\alpha^{0}|-1\big|<1/4$ will be guaranteed to hold with high probability by Lemma \ref{lemma:spectral-BD-L2}.


\subsection{Step 2: introducing the leave-one-out sequences}

As demonstrated by the assumptions in Lemma \ref{lemma:inductive-ell-2},
the key is to show that the whole trajectory lies in the region specified
by (\ref{eq:induction-dist-ell-2-BD})-(\ref{eq:induction-incoherence-b-BD}).
Once again, the difficulty lies in the statistical dependency between
the iterates $\left\{ \bm{z}^{t}\right\} $ and the measurement vectors
$\left\{ \bm{a}_{j}\right\} $. We follow the general recipe and introduce
the \emph{leave-one-out} sequences, denoted by $\left\{ \bm{h}^{t,\left(l\right)},\bm{x}^{t,(l)}\right\} _{t\geq0}$
for each $1\leq l\leq m$. Specifically, $\left\{ \bm{h}^{t,\left(l\right)},\bm{x}^{t,(l)}\right\} _{t\geq0}$
is the gradient sequence operating on the loss function

\begin{equation}
f^{(l)}\left(\bm{h},\bm{x}\right):=\sum_{j:j\neq l}\left|\bm{b}_{j}^{\conj}\left(\bm{h}\bm{x}^{\conj}-\bm{h}^{\star}\bm{x}^{\star\conj}\right)\bm{a}_{j}\right|^{2}.\label{eq:loo-loss-BD}
\end{equation}
The whole sequence is constructed by running gradient descent with spectral initialization on the leave-one-out loss (\ref{eq:loo-loss-BD}). The
precise description is supplied in Algorithm \ref{alg:leave-one-out-gd-BD}.

For notational simplicity, we denote $\bm{z}^{t,(l)}=\left[\begin{array}{c}
\bm{h}^{t,(l)}\\
\bm{x}^{t,(l)}
\end{array}\right]$ and use $f(\bm{z}^{t,(l)})=f(\bm{h}^{t,(l)},\bm{x}^{t,(l)})$ interchangeably. Define similarly the alignment parameters
\begin{equation}
\alpha^{t,\left(l\right)}:=\arg\min_{\alpha\in\CC}\left\Vert \frac{1}{\overline{\alpha}}\bm{h}^{t,\left(l\right)}-\bm{h}^{\star}\right\Vert _{2}^{2}+\big\|\alpha\bm{x}^{t,\left(l\right)}-\bm{x}^{\star}\big\|_{2}^{2},\label{eq:defn-alpha-tl-BD}
\end{equation}
and denote $\tilde{\bm{z}}^{t,(l)}=\left[\begin{array}{c}
\tilde{\bm{h}}^{t,\left(l\right)}\\
\tilde{\bm{x}}^{t,\left(l\right)}
\end{array}\right]$ where
\begin{equation}
\tilde{\bm{h}}^{t,\left(l\right)}=\frac{1}{\overline{\alpha^{t,\left(l\right)}}}\bm{h}^{t,\left(l\right)}\qquad\text{and}\qquad\tilde{\bm{x}}^{t,\left(l\right)}=\alpha^{t,\left(l\right)}\bm{x}^{t,\left(l\right)}.\label{eq:defn-htl-xtl-BD}
\end{equation}

\begin{algorithm}
\caption{The $l$th leave-one-out sequence for blind deconvolution}

\label{alg:leave-one-out-gd-BD}\begin{algorithmic}

\STATE \textbf{{Input}}: $\left\{ \bm{a}_{j}\right\} _{1\leq j\leq m,j\neq l},\left\{ \bm{b}_{j}\right\} _{1\leq j\leq m,j\neq l}$
and $\left\{ y_{j}\right\} _{1\leq j\leq m,j\neq l}$.

\STATE \textbf{{Spectral initialization}}: Let $\sigma_{1}(\bm{M}^{(l)})$,
$\check{\bm{h}}^{0,\left(l\right)}$ and $\check{\bm{x}}^{0,\left(l\right)}$ be the leading singular value, left and right singular vectors
of 
\[
\bm{M}^{\left(l\right)}:=\sum_{j:j\neq l}y_{j}\bm{b}_{j}\bm{a}_{j}^{\conj},
\]
respectively. Set $\bm{h}^{0,(l)}=\sqrt{\sigma_{1}(\bm{M}^{(l)})}\;\check{\bm{h}}^{0,\left(l\right)}$
and $\bm{x}^{0,(l)}=\sqrt{\sigma_{1}(\bm{M}^{(l)})}\;\check{\bm{x}}^{0,\left(l\right)}$.

\STATE \textbf{{Gradient updates}}: \textbf{for} $t=0,1,2,\ldots,T-1$
\textbf{do }

\STATE 
\begin{equation}
\left[\begin{array}{c}
\bm{h}^{t+1,(l)}\\
\bm{x}^{t+1,(l)}
\end{array}\right]=\left[\begin{array}{c}
\bm{h}^{t,(l)}\\
\bm{x}^{t,(l)}
\end{array}\right]-\eta\left[\begin{array}{c}
\frac{1}{\|\bm{x}^{t,(l)}\|_{2}^{2}}\nabla_{\bm{h}}f^{(l)}\big(\bm{h}^{t,(l)},\bm{x}^{t,(l)}\big)\\
\frac{1}{\|\bm{h}^{t,(l)}\|_{2}^{2}}\nabla_{\bm{x}}f^{(l)}\big(\bm{h}^{t,(l)},\bm{x}^{t,(l)}\big)
\end{array}\right].\label{eq:loo-gradient_update}
\end{equation}

\end{algorithmic} 
\end{algorithm}

\subsection{Step 3: establishing the incoherence condition by induction}

As usual, we continue the proof in an inductive manner. For clarity of presentation,
we list below the set of induction hypotheses underlying our analysis:
  \begin{subequations}\label{subeq:hypotheses-BD}
\begin{align}
\text{dist}\left(\bm{z}^{t},\bm{z}^{\star}\right) & \leq C_{1}\frac{1}{\log^{2}m},\label{eq:L2-hypothesis-BD}\\
\max_{1\leq l\leq m}\text{dist}\big(\bm{z}^{t,\left(l\right)},\tilde{\bm{z}}^{t}\big) & \leq C_{2}\frac{\mu}{\sqrt{m}}\sqrt{\frac{\mu^{2}K\log^{9}m}{m}},\label{eq:LOO-perturb-hypothesis-BD}\\
\max_{1\leq l\leq m}\big|\bm{a}_{l}^{\conj}\big(\tilde{\bm{x}}^{t}-\bm{x}^{\star}\big)\big| & \leq C_{3}\frac{1}{\log^{1.5}m},\label{eq:incoherence-hypothesis-ax-BD}\\
\max_{1\leq l\leq m}\big|\bm{b}_{l}^{\conj}\tilde{\bm{h}}^{t}\big| & \leq C_{4}\frac{\mu}{\sqrt{m}}\log^{2}m,\label{eq:incoherence-hypothesis-bh-BD}
\end{align}
\end{subequations}where $\tilde{\bm{h}}^{t},\; \tilde{\bm{x}}^{t}$ and $\tilde{\bm{z}}^{t}$
are defined in (\ref{eq:tilde-notaion-BD}). Here, $C_{1},C_{3}>0$
are some sufficiently small constants, while $C_{2},C_{4}>0$ are
some sufficiently large constants. We aim to show that if these hypotheses (\ref{subeq:hypotheses-BD}) 
hold up to the $t$th iteration, then the same would hold for the ($t+1$)th iteration with exceedingly high probability (e.g. $1-O(m^{-10})$). The first hypothesis \eqref{eq:L2-hypothesis-BD} has already been established in Lemma \ref{lemma:inductive-ell-2},
and hence the rest of this section focuses on establishing the remaining
three. To justify the incoherence hypotheses (\ref{eq:incoherence-hypothesis-ax-BD}) and
(\ref{eq:incoherence-hypothesis-bh-BD}) 
for the ($t+1$)th iteration, we need to leverage the nice properties of the leave-one-out sequences, and establish (\ref{eq:LOO-perturb-hypothesis-BD}) first. In the sequel, we follow the steps suggested in the general
recipe. 
\begin{itemize}
\item \textbf{Step 3(a): proximity between the original and the leave-one-out
iterates.} We first justify the hypothesis (\ref{eq:LOO-perturb-hypothesis-BD})
for the ($t+1$)th iteration via the following lemma.
\begin{lemma}\label{lemma:inductive-loop}Suppose
the sample complexity obeys $m\geq C \mu^{2}K\log^{9}m$ for some sufficiently large constant $C>0$. Let $\mathcal{E}_t$ be the event where the hypotheses
(\ref{eq:L2-hypothesis-BD})-(\ref{eq:incoherence-hypothesis-bh-BD})
hold for the $t$th iteration. Then on an event $\mathcal{E}_{t+1,1} \subseteq \mathcal{E}_t$ obeying $\PP(\mathcal{E}_t \cap \mathcal{E}_{t+1,1}^c ) = O(m^{-10} + me^{-cK} )$ for some constant $c>0$, one has
\begin{align*}
\max_{1\leq l\leq m}\mathrm{dist}\big(\bm{z}^{t+1,\left(l\right)},\tilde{\bm{z}}^{t+1}\big) & \leq C_{2}\frac{\mu}{\sqrt{m}}\sqrt{\frac{\mu^{2}K\log^{9}m}{m}}\\
\text{and}\qquad\max_{1\leq l\leq m}\big\|\tilde{\bm{z}}^{t+1,\left(l\right)}-\tilde{\bm{z}}^{t+1}\big\|_{2} & \lesssim C_{2}\frac{\mu}{\sqrt{m}}\sqrt{\frac{\mu^{2}K\log^{9}m}{m}},
\end{align*}
provided that the step size $\eta>0$ is some sufficiently small constant. \end{lemma}\begin{proof}As
usual, this result follows from the restricted strong convexity, which
forces the distance between the two sequences of interest to be contractive.
See Appendix \ref{subsec:Proof-of-Lemma-inductive-loop}.\end{proof} 
\item \textbf{Step 3(b): incoherence of the leave-one-out iterate $\bm{x}^{t+1,(l)}$
w.r.t.~$\bm{a}_{l}$. }Next, we show that the leave-one-out iterate $\tilde{\bm{x}}^{t+1,\left(l\right)}$
\textemdash{} which is independent of $\bm{a}_{l}$ \textemdash{}
is incoherent w.r.t.~$\bm{a}_{l}$ in the sense that 
\begin{equation}
\left|\bm{a}_{l}^{\conj}\big(\tilde{\bm{x}}^{t+1,\left(l\right)}-\bm{x}^{\star}\big)\right|\leq10C_{1}\frac{1}{\log^{3/2}m}\label{eq:a-x-incoherence-BD-l}
\end{equation}
with probability exceeding $1-O\left( m^{-10} + e^{-K} \log m \right)$. To see why,
use the statistical independence and the standard Gaussian concentration
inequality to show that 
\[
\max_{1\leq l\leq m}\left|\bm{a}_{l}^{\conj}\big(\tilde{\bm{x}}^{t+1,(l)}-\bm{x}^{\star}\big)\right|\leq5\sqrt{\log m}\max_{1\leq l\leq m}\big\|\tilde{\bm{x}}^{t+1,(l)}-\bm{x}^{\star}\big\|_{2}
\]
with probability exceeding $1-O(m^{-10})$. It then follows from the
triangle inequality that 
\begin{align*}
\big\|\tilde{\bm{x}}^{t+1,(l)}-\bm{x}^{\star}\big\|_{2} & \leq\big\|\tilde{\bm{x}}^{t+1,(l)}-\tilde{\bm{x}}^{t+1}\big\|_{2}+\left\Vert \tilde{\bm{x}}^{t+1}-\bm{x}^{\star}\right\Vert _{2}\\
 & \overset{\left(\text{i}\right)}{\leq}CC_{2}\frac{\mu}{\sqrt{m}}\sqrt{\frac{\mu^{2}K\log^{9}m}{m}}+C_{1}\frac{1}{\log^{2}m}\\
 & \overset{\left(\text{ii}\right)}{\leq}2C_{1}\frac{1}{\log^{2}m},
\end{align*}
where (i) follows from Lemmas \ref{lemma:inductive-ell-2} and \ref{lemma:inductive-loop},
and (ii) holds as soon as $m / ( \mu^{2}\sqrt{K}\log^{13/2}m )$ is sufficiently large. Combining
the preceding two bounds establishes (\ref{eq:a-x-incoherence-BD-l}). 
\item \textbf{Step 3(c): combining the bounds to show incoherence of $\bm{x}^{t+1}$
w.r.t.~$\left\{ \bm{a}_{l}\right\} $.} The above bounds immediately
allow us to conclude that 
\[
\max_{1\leq l\leq m}\left|\bm{a}_{l}^{\conj}\big(\tilde{\bm{x}}^{t+1}-\bm{x}^{\star}\big)\right|\leq C_{3}\frac{1}{\log^{3/2}m}
\]
with probability at least $1-O\left(m^{-10} + e^{-K}\log m \right)$, which is exactly
the hypothesis (\ref{eq:incoherence-hypothesis-ax-BD}) for the ($t+1$)th iteration.
Specifically, for each $1\leq l\leq m$, the triangle inequality yields
\begin{align*}
\left|\bm{a}_{l}^{\conj}\big(\tilde{\bm{x}}^{t+1}-\bm{x}^{\star}\big)\right| & \leq\left|\bm{a}_{l}^{\conj}\big(\tilde{\bm{x}}^{t+1}-\tilde{\bm{x}}^{t+1,(l)}\big)\right|+\left|\bm{a}_{l}^{\conj}\big(\tilde{\bm{x}}^{t+1,(l)}-\bm{x}^{\star}\big)\right|\\
 & \overset{\left(\text{i}\right)}{\leq}\left\Vert \bm{a}_{l}\right\Vert _{2}\left\Vert \tilde{\bm{x}}^{t+1}-\tilde{\bm{x}}^{t+1,(l)}\right\Vert _{2}+\left|\bm{a}_{l}^{\conj}\big(\tilde{\bm{x}}^{t+1,(l)}-\bm{x}^{\star}\big)\right|\\
 & \overset{\left(\text{ii}\right)}{\leq}3\sqrt{K}\cdot CC_{2}\frac{\mu}{\sqrt{m}}\sqrt{\frac{\mu^{2}K\log^{9}m}{m}}+10C_{1}\frac{1}{\log^{3/2}m}\\
 & \overset{\left(\text{iii}\right)}{\leq}C_{3}\frac{1}{\log^{3/2}m}.
\end{align*}
Here (i) follows from Cauchy-Schwarz, (ii) is a consequence of (\ref{eq:max_gaussian-2}),
Lemma \ref{lemma:inductive-loop} and the bound (\ref{eq:a-x-incoherence-BD-l}),
and the last inequality holds as long as $m / ( \mu^{2}K\log^{6}m )$ is sufficiently large
and $C_{3}\geq11C_{1}$. 
\item \textbf{Step 3(d): incoherence of $\bm{h}^{t+1}$
w.r.t.~$\{\bm{b}_{l}\}$. }It remains to justify that $\bm{h}^{t+1}$
is also incoherent w.r.t.~its associated design vectors $\{\bm{b}_{l}\}$.
This proof of this step, however, is much more involved and challenging,
due to the deterministic nature of the $\bm{b}_{l}$'s. As a result,
we would need to ``propagate'' the randomness brought about by $\{\bm{a}_{l}\}$
to $\bm{h}^{t+1}$ in order to facilitate the analysis. The result is
summarized as follows. \begin{lemma}\label{lemma:incoherence-b}
Suppose that the sample complexity obeys
 $m\geq C \mu^{2}K\log^{9}m$ for some sufficiently large constant $C>0$. 
Let $\mathcal{E}_t$ be the event where the hypotheses
(\ref{eq:L2-hypothesis-BD})-(\ref{eq:incoherence-hypothesis-bh-BD})
hold for the $t$th iteration. Then on an event $\mathcal{E}_{t+1,2} \subseteq \mathcal{E}_t$ obeying $\PP(\mathcal{E}_t \cap \mathcal{E}_{t+1,2}^c ) = O(m^{-10})$, one has
\[
\max_{1\leq l\leq m}\left|\bm{b}_{l}^{\conj}\tilde{\bm{h}}^{t+1}\right|\leq C_{4}\frac{\mu}{\sqrt{m}}\log^{2}m
\]
as long as $C_{4}$ is sufficiently large, and $\eta>0$ is taken
to be some sufficiently small constant. \end{lemma}
\begin{proof}The key idea  is to divide $\{1,\cdots,m\}$ into consecutive
bins each of size $\mathrm{poly}\log(m)$, and to exploit the randomness
(namely, the randomness from $\bm{a}_{l}$) within each bin.
This binning idea is crucial in ensuring that the incoherence measure
of interest does not blow up as $t$ increases. See Appendix \ref{subsec:Proof-of-Lemma-incoherence-b}.
\end{proof} 
\end{itemize}
With these steps in place, we conclude the proof of Theorem \ref{thm:main-BD}
via induction and the union bound.


\subsection{The base case: spectral initialization}

In order to finish the induction steps, we still need to justify the
induction hypotheses for the base cases, namely, we need to show that
the spectral initializations $\bm{z}^{0}$ and $\left\{ \bm{z}^{0,\left(l\right)}\right\}_{1\leq l \leq m} $
satisfy the induction hypotheses (\ref{subeq:hypotheses-BD}) at
$t=0$. 

To start with, the initializations are sufficiently close to the truth
when measured by the $\ell_{2}$ norm, as summarized by the following
lemma.

\begin{lemma}\label{lemma:spectral-BD-L2}Fix any small constant $\xi>0$. Suppose the sample size obeys $m\geq{C\mu^{2}K\log^{2}m}/{\xi^{2}}$ for some sufficiently large constant $C>0$. Then with probability at least
$1-O(m^{-10})$, we have  
\begin{align}
\min_{\alpha\in\mathbb{C},|\alpha|=1}\left\{ \big\|\alpha\bm{h}^{0}-\bm{h}^{\star}\big\|_{2}+\big\|\alpha\bm{x}^{0}-\bm{x}^{\star}\big\|_{2}\right\}  & \leq\xi\qquad\text{and}\label{eq:spectral-L2-h0-1}\\
\min_{\alpha\in\mathbb{C},|\alpha|=1}\left\{ \big\|\alpha\bm{h}^{0,(l)}-\bm{h}^{\star}\big\|_{2}+\big\|\alpha\bm{x}^{0,(l)}-\bm{x}^{\star}\big\|_{2}\right\}  & \leq\xi,\qquad1\leq l\leq m,\label{eq:spectral-L2-h0-1l}
\end{align}
and $\left||\alpha_0|-1\right| \leq 1/4$.
\end{lemma}

\begin{proof}This follows from Wedin's sin$\Theta$ theorem \cite{wedin1972perturbation} and \cite[Lemma 5.20]{DBLP:journals/corr/LiLSW16}.
See Appendix \ref{sec:proof-of-lemma:spectral-BD-L2}.\end{proof}

From the definition of $\mathrm{dist}(\cdot,\cdot)$ (cf.~(\ref{eq:defn-dist-BD})),
we immediately have 
\begin{align}
\mathrm{dist}\big(\bm{z}^{0},\bm{z}^{\star}\big) & =\min_{\alpha\in\mathbb{C}}\sqrt{\left\Vert \frac{1}{\overline{\alpha}}\bm{h}-\bm{h}^{\star}\right\Vert _{2}^{2}+\left\Vert \alpha\bm{x}-\bm{x}^{\star}\right\Vert _{2}^{2}}\overset{\left(\text{i}\right)}{\leq}\min_{\alpha\in\mathbb{C}}\left\{ \left\Vert \frac{1}{\overline{\alpha}}\bm{h}-\bm{h}^{\star}\right\Vert _{2}+\left\Vert \alpha\bm{x}-\bm{x}^{\star}\right\Vert _{2}\right\} \nonumber \\
 & \overset{\left(\text{ii}\right)}{\leq}\min_{\alpha\in\mathbb{C},|\alpha|=1}\left\{ \big\|\alpha\bm{h}^{0}-\bm{h}^{\star}\big\|_{2}+\big\|\alpha\bm{x}^{0}-\bm{x}^{\star}\big\|_{2}\right\} \overset{\left(\text{iii}\right)}{\leq}C_{1}\frac{1}{\log^{2}m},\label{eq:spectral-BD-dist}
\end{align}
as long as $ m\geq C\mu^{2}K\log^{6}m$
for some sufficiently large constant $C>0$. Here (i) follows from
the elementary inequality that $a^{2}+b^{2}\leq\left(a+b\right)^{2}$
for positive $a$ and $b$, (ii) holds since the feasible set of the
latter one is strictly smaller, and (iii) follows directly from Lemma
\ref{lemma:spectral-BD-L2}. This finishes the proof of (\ref{eq:L2-hypothesis-BD})
for $t=0$. Similarly, with high probability we have 
\begin{align}
\mathrm{dist}\big(\bm{z}^{0,\left(l\right)},\bm{z}^{\star}\big) & \leq\min_{\alpha\in\mathbb{C},|\alpha|=1}\left\{ \big\|\alpha\bm{h}^{0,(l)}-\bm{h}^{\star}\big\|_{2}+\big\|\alpha\bm{x}^{0,(l)}-\bm{x}^{\star}\big\|_{2}\right\} \lesssim\frac{1}{\log^{2}m},\quad1\leq l\leq m.\label{eq:spectral-BD-dist-l}
\end{align}

Next, when properly aligned, the true initial estimate $\bm{z}^0$
and the leave-one-out estimate $\bm{z}^{0,(l)}$
are expected to be sufficiently close, as claimed by the following
lemma. Along the way, we show that $\bm{h}^{0}$ is incoherent w.r.t.~the
sampling vectors $\left\{ \bm{b}_{l}\right\}$. This
establishes (\ref{eq:LOO-perturb-hypothesis-BD}) and (\ref{eq:incoherence-hypothesis-bh-BD})
for $t=0$.

\begin{lemma}\label{lemma:perturbabtion-init-BD}Suppose that $m\geq C \mu^{2}K\log^{3}m$ for some sufficiently large constant $C>0$.
Then with probability at least $1-O(m^{-10})$, one has
\begin{equation}
\max_{1\leq l \leq m}\mathrm{dist}\big(\bm{z}^{0,\left(l\right)},\tilde{\bm{z}}^{0}\big)\leq C_{2}\frac{\mu}{\sqrt{m}}\sqrt{\frac{\mu^{2}K\log^{5}m}{m}}\label{eq:base-loop-BD}
\end{equation}
and
\begin{align}
\max_{1\leq l \leq m}\big|\bm{b}_{l}^{\conj}\tilde{\bm{h}}^{0}\big| & \leq C_{4}\frac{\mu\log^{2}m}{\sqrt{m}}.\label{eq:base-incoherence-bh-BD}
\end{align}
\end{lemma}\begin{proof}The key is to establish that $\mathrm{dist}\big(\bm{z}^{0,\left(l\right)},\tilde{\bm{z}}^{0}\big)$
can be upper bounded by some linear scaling of $\big|\bm{b}_{l}^{\conj}\tilde{\bm{h}}^{0}\big|$,
and vice versa. This allows us to derive bounds simultaneously for
both quantities. See Appendix \ref{subsec:Proof-of-Lemma-perturbation-init-BD}.\end{proof}

Finally, we establish (\ref{eq:incoherence-hypothesis-ax-BD}) regarding
the incoherence of $\bm{x}^{0}$ with respect to the design vectors
$\left\{ \bm{a}_{l}\right\}$.

\begin{lemma}\label{lemma:BD-init-al-x0}Suppose that $m\geq C \mu^{2}K\log^{6}m$ for some sufficiently large constant $C>0$.
Then with probability exceeding $1-O(m^{-10})$, we have
\begin{align*}
\max_{1\leq l\leq m}\left|\bm{a}_{l}^{\conj}\big(\tilde{\bm{x}}^{0}-\bm{x}^{\star}\big)\right| & \leq C_{3}\frac{1}{\log^{1.5}m}.
\end{align*}

\end{lemma}

\begin{proof}See Appendix \ref{subsec:Proof-of-Lemma-BD-init-al-x0}.\end{proof}

\section{Discussions\label{sec:Discussion}}

This paper showcases an important phenomenon in nonconvex optimization: even without explicit enforcement of
regularization, the vanilla form of gradient descent effectively achieves implicit regularization for a large family of statistical estimation problems. 
We believe this phenomenon arises in problems far beyond the three cases studied herein, and our results are initial steps towards understanding this fundamental phenomenon. There are numerous avenues open for future investigation, 
and we point out a few of them.

\begin{itemize}
\item \emph{Improving sample complexity.} 
In the current paper, the required sample complexity $O\left(\mu^{3}r^{3}n\log^{3}n\right)$ for matrix completion is sub-optimal  when the rank $r$ of the underlying matrix is large. While this allows us to achieve a dimension-free iteration complexity, it is slightly higher than the sample complexity derived for regularized gradient descent in \cite{chen2015fast}.  
We expect our results continue to hold under lower sample complexity $O\left(\mu^{2}r^{2}n\log n\right)$, but it calls for a more refined analysis (e.g.~a generic chaining argument).
\item \emph{Leave-one-out tricks for more general designs.} So far our focus is on independent designs, including the i.i.d.~Gaussian design adopted in phase retrieval and partially in blind deconvolution, as well as the independent sampling mechanism in matrix completion. Such independence property creates some sort of ``statistical homogeneity'', for which the leave-one-out argument works beautifully. 
It remains unclear how to generalize such leave-one-out tricks for more general designs (e.g.~more general sampling patterns in matrix completion and more structured Fourier designs in phase retrieval and blind deconvolution). In fact, the readers can already get a flavor of this issue in the analysis of blind deconvolution, where the Fourier design vectors require much more delicate treatments than purely Gaussian designs. 

\item \emph{Uniform stability.} The leave-one-out perturbation argument is established upon a basic fact: when we exclude one sample from consideration, the resulting estimates/predictions do not deviate much from the original ones. This leave-one-out stability bears similarity to the notion of uniform stability studied in statistical learning theory \cite{bousquet2002stability}. We expect our analysis framework to be helpful for analyzing other learning algorithms that are uniformly stable.


\item \emph{Other  iterative methods and other loss functions.} The focus of the current paper has been the analysis of vanilla GD tailored to the natural squared loss. This is by no means to advocate GD as the top-performing algorithm in practice; rather, we are using this simple algorithm to isolate some seemingly pervasive  phenomena (i.e.~implicit regularization) that generic optimization theory fails to account for. The simplicity of vanilla GD makes it an ideal object to initiate such discussions. That being said, practitioners should definitely explore as many algorithmic alternatives as possible before settling on a particular algorithm. Take phase retrieval for example: 
	iterative methods other than GD and\,/\,or algorithms tailored to other loss functions have been proposed in the nonconvex optimization literature, including but not limited to alternating minimization, block coordinate descent, and sub-gradient methods and prox-linear methods tailed to non-smooth losses.  It would be interesting to develop a full theoretical understanding of a broader class of iterative algorithms, and to conduct a careful comparison regarding which loss functions  lead to the most desirable practical performance. 

\item \emph{Connections to deep learning?} We have focused on nonlinear systems that are bilinear or quadratic in this paper. Deep learning formulations/architectures, highly nonlinear, are notorious for their daunting nonconvex geometry. However, iterative methods including stochastic gradient descent have enjoyed enormous practical success in learning neural networks (e.g.~\cite{zhong2017recovery,soltanolkotabi2017theoretical,fan2019selective}), even when the architecture is significantly over-parameterized without explicit regularization. We hope the message conveyed in this paper for several simple statistical models can shed  light on why simple forms of gradient descent and variants work so well in learning complicated neural networks.
\end{itemize}

\noindent Finally, while the present paper provides a general recipe for problem-specific analyses of nonconvex algorithms, we acknowledge that a unified theory of this kind has yet to be developed. As a consequence, each problem requires delicate and somewhat lengthy analyses of its own.  It would certainly be helpful if one could single out a few stylized structural properties\,/\,elements (like sparsity and incoherence in compressed sensing \cite{CandesPlan2011RIPless}) that enable near-optimal performance guarantees through an over-arching method of analysis; with this in place,  one would not need to start each  problem from scratch.  Having said that, we believe that our current theory elucidates  a few ingredients (e.g.~the region of incoherence and leave-one-out stability) that might serve as crucial building blocks for such a general  theory. We invite the interested readers to contribute towards this path forward.

\section*{Acknowledgements}

Y.~Chen is supported in part by the AFOSR YIP award FA9550-19-1-0030, by the ARO  grant W911NF-18-1-0303, by the ONR grant N00014-19-1-2120, by the NSF grants CCF-1907661 and IIS-1900140, and by the Princeton SEAS innovation award. 
Y.~Chi is supported in part by the grants AFOSR FA9550-15-1-0205, ONR N00014-18-1-2142 and N00014-19-1-2404, ARO W911NF-18-1-0303, NSF CCF-1826519, ECCS-1818571, CCF-1806154.  Y.~Chen would like to thank Yudong Chen for inspiring discussions about matrix completion.

\bibliographystyle{alphaabbr}
\bibliography{../bibfileNonconvex}

\clearpage{}

\appendix


\section{Proofs for phase retrieval} \label{sec:proof-WF}
Before proceeding, we gather a few simple facts. The standard
concentration inequality for $\chi^{2}$ random variables together with the
union bound reveals that 
the sampling vectors $\{\bm{a}_{j}\}$ obey 
\begin{equation}
\max_{1\leq j\leq m}\left\Vert \bm{a}_{j}\right\Vert _{2}\leq\sqrt{6n}\label{eq:max-al-norm}
\end{equation}
with probability at least $1-O(me^{-1.5n})$. In addition, standard
Gaussian concentration inequalities give
\begin{equation}
\max_{1\leq j\leq m}\left|\bm{a}_{j}^{\top}\bm{x}^{\star}\right|\leq5\sqrt{\log n}\label{eq:aj-xnatural-incoherence-WF}
\end{equation}
with probability exceeding $1-O(mn^{-10})$. 

\subsection{Proof of Lemma \ref{lemma:wf_hessian}\label{subsec:Proof-of-Lemma-wf-hessian}}

We start with the smoothness bound, namely, $\nabla^{2}f(\bm{x})\preceq O(\log n)\cdot\bm{I}_{n}$. It suffices to prove the upper bound ${\left\Vert \nabla^{2}f\left(\bm{x}\right)\right\Vert \lesssim\log n}$.
To this end, we first decompose the Hessian (cf.~\eqref{eq:hessian-WF})
into three components as follows:
\begin{align*}
\nabla^{2}f\left(\bm{x}\right) & =\underbrace{\frac{3}{m}\sum_{j=1}^{m}\left[\left(\bm{a}_{j}^{\top}\bm{x}\right)^{2}-\left(\bm{a}_{j}^{\top}\bm{x}^{\star}\right)^{2}\right]\bm{a}_{j}\bm{a}_{j}^{\top}}_{:=\bm{\Lambda}_{1}}+\underbrace{\frac{2}{m}\sum_{j=1}^{m}\left(\bm{a}_{j}^{\top}\bm{x}^{\star}\right)^{2}\bm{a}_{j}\bm{a}_{j}^{\top}-2\left(\bm{I}_{n}+2\bm{x}^{\star}\bm{x}^{\star\top}\right)}_{:=\bm{\Lambda}_{2}}+\underbrace{2\vphantom{\sum_{j=1}^{m}}\left(\bm{I}_{n}+2\bm{x}^{\star}\bm{x}^{\star\top}\right)}_{:=\bm{\Lambda}_{3}},
\end{align*}
where we have used $y_j=(\bm{a}_j^{\top}\bm{x}^{\star})^2$. 
In the sequel, we control the three terms $\bm{\Lambda}_{1}$,
$\bm{\Lambda}_{2}$ and $\bm{\Lambda}_{3}$ in reverse order. 
\begin{itemize}
\item The third term $\bm{\Lambda}_{3}$ can be easily bounded by 
\[
\left\Vert \bm{\Lambda}_{3}\right\Vert \leq2\left(\left\Vert \bm{I}_{n}\right\Vert +2\left\Vert \bm{x}^{\star}\bm{x}^{\star\top}\right\Vert \right)=6.
\]
\item The second term $\bm{\Lambda}_{2}$ can be controlled by means of
Lemma \ref{lemma:spectral-fix-WF}: 
\[
\left\Vert \bm{\Lambda}_{2}\right\Vert \leq2\delta
\]
for an arbitrarily small constant $\delta>0$, as long as $m\geq c_{0}n\log n$
for $c_{0}$ sufficiently large. 
\item It thus remains to control $\bm{\Lambda}_{1}$. Towards this we discover
that 
\begin{align}
\left\Vert \bm{\Lambda}_{1}\right\Vert  & \leq\left\Vert \frac{3}{m}\sum_{j=1}^{m}\left|\bm{a}_{j}^{\top}\left(\bm{x}-\bm{x}^{\star}\right)\right|\left|\bm{a}_{j}^{\top}\left(\bm{x}+\bm{x}^{\star}\right)\right|\bm{a}_{j}\bm{a}_{j}^{\top}\right\Vert .\label{eq:Lambda1-UB-PR}
\end{align}
Under the assumption $\max_{1\leq j\leq m}\left|\bm{a}_{j}^{\top}\left(\bm{x}-\bm{x}^{\star}\right)\right|\leq C_{2}\sqrt{\log n}$
and the fact (\ref{eq:aj-xnatural-incoherence-WF}), we can also
obtain 
\[
\max_{1\leq j\leq m}\left|\bm{a}_{j}^{\top}\left(\bm{x}+\bm{x}^{\star}\right)\right|\leq2\max_{1\leq j\leq m}\left|\bm{a}_{j}^{\top}\bm{x}^{\star}\right|+\max_{1\leq j\leq m}\left|\bm{a}_{j}^{\top}\left(\bm{x}-\bm{x}^{\star}\right)\right|\leq\left(10+C_{2}\right)\sqrt{\log n}.
\]
Substitution into (\ref{eq:Lambda1-UB-PR}) leads to 
\begin{align*}
\left\Vert \bm{\Lambda}_{1}\right\Vert  & \leq3C_{2}\left(10+C_{2}\right)\log n\cdot\Bigg\|\frac{1}{m}\sum_{j=1}^{m}\bm{a}_{j}\bm{a}_{j}^{\top}\Bigg\|\leq4C_{2}\left(10+C_{2}\right)\log n,
\end{align*}
where the last inequality is a direct consequence of Lemma \ref{lemma:identity_concentration-WF}. 
\end{itemize}
Combining the above bounds on $\bm{\Lambda}_{1}$, $\bm{\Lambda}_{2}$
and $\bm{\Lambda}_{3}$ yields 
\begin{align*}
\left\Vert \nabla^{2}f\left(\bm{x}\right)\right\Vert  & \leq\left\Vert \bm{\Lambda}_{1}\right\Vert +\left\Vert \bm{\Lambda}_{2}\right\Vert +\left\Vert \bm{\Lambda}_{3}\right\Vert \leq4C_{2}\left(10+C_{2}\right)\log n+2\delta+6\leq5C_{2}\left(10+C_{2}\right)\log n,
\end{align*}
as long as $n$ is sufficiently large. This establishes the claimed
smoothness property.

Next we move on to the strong convexity lower bound. Picking a constant $C>0$ and enforcing proper truncation, we get 
\begin{align*}
\nabla^{2}f\left(\bm{x}\right) & =\frac{1}{m}\sum_{j=1}^{m}\left[3\left(\bm{a}_{j}^{\top}\bm{x}\right)^{2}-y_{j}\right]\bm{a}_{j}\bm{a}_{j}^{\top}
 \succeq\underbrace{\frac{3}{m}\sum_{j=1}^{m}\left(\bm{a}_{j}^{\top}\bm{x}\right)^{2}\ind_{\left\{ \left|\bm{a}_{j}^{\top}\bm{x}\right|\leq C\right\} }\bm{a}_{j}\bm{a}_{j}^{\top}}_{:=\bm{\Lambda}_{4}}-\underbrace{\frac{1}{m}\sum_{j=1}^{m}\left(\bm{a}_{j}^{\top}\bm{x}^{\star}\right)^{2}\bm{a}_{j}\bm{a}_{j}^{\top}}_{:=\bm{\Lambda}_{5}}.
\end{align*}
We begin with the simpler term $\bm{\Lambda}_{5}$. Lemma \ref{lemma:spectral-fix-WF}
implies that with probability at least $1-O(n^{-10})$, 
\[
\left\Vert \bm{\Lambda}_{5}-\left(\bm{I}_{n}+2\bm{x}^{\star}\bm{x}^{\star\top}\right)\right\Vert \leq\delta
\]
holds for any small constant $\delta>0$, as long as $m/(n\log n)$
is sufficiently large. This reveals that 
\[
\bm{\Lambda}_{5}\preceq\left(1+\delta\right)\cdot\bm{I}_{n}+2\bm{x}^{\star}\bm{x}^{\star\top}.
\]
To bound $\bm{\Lambda}_{4}$, invoke Lemma
\ref{lemma:universal-chen-candes} to conclude that with probability
at least $1-c_{3}e^{-c_{2}m}$ (for some constants $c_{2},c_{3}>0$),
\[
\left\Vert \bm{\Lambda}_{4}-3\left(\beta_{1}\bm{x}\bm{x}^{\top}+\beta_{2}\|\bm{x}\|_{2}^{2}\bm{I}_{n}\right)\right\Vert \leq\delta\|\bm{x}\|_{2}^{2}
\]
for any small constant $\delta>0$, provided that $m/n$ is sufficiently
large. Here, 
\[
\beta_{1}:=\EE\left[\xi^{4}\ind_{\{|\xi|\leq C\}}\right]-\EE\left[\xi^{2}\ind_{|\xi|\leq C}\right]\quad\text{and}\quad\beta_{2}:=\EE\left[\xi^{2}\ind_{|\xi|\leq C}\right],
\]
where the expectation is taken with respect to $\xi\sim\mathcal{N}(0,1)$. By the assumption $\left\Vert \bm{x}-\bm{x}^{\star}\right\Vert _{2}\leq2C_{1}$, one has
\[
\left\Vert \bm{x}\right\Vert _{2}\leq1+2C_{1},\quad\left|\left\Vert \bm{x}\right\Vert _{2}^{2}-\|\bm{x}^{\star}\|_{2}^{2}\right|\leq2C_{1}\left(4C_{1}+1\right),\quad\left\Vert \bm{x}^{\star}\bm{x}^{\star\top}-\bm{x}\bm{x}^{\top}\right\Vert \leq6C_{1}\left(4C_{1}+1\right),
\]
which leads to
\begin{align*}
\left\Vert \bm{\Lambda}_{4}-3\left(\beta_{1}\bm{x}^{\star}\bm{x}^{\star\top}+\beta_{2}\bm{I}_{n}\right)\right\Vert  & \leq\left\Vert \bm{\Lambda}_{4}-3\left(\beta_{1}\bm{x}\bm{x}^{\top}+\beta_{2}\|\bm{x}\|_{2}^{2}\bm{I}_{n}\right)\right\Vert +3\left\Vert \left(\beta_{1}\bm{x}^{\star}\bm{x}^{\star\top}+\beta_{2}\bm{I}_{n}\right)-\left(\beta_{1}\bm{x}\bm{x}^{\top}+\beta_{2}\|\bm{x}\|_{2}^{2}\bm{I}_{n}\right)\right\Vert \\
 & \leq\delta\|\bm{x}\|_{2}^{2}+3\beta_{1}\left\Vert \bm{x}^{\star}\bm{x}^{\star\top}-\bm{x}\bm{x}^{\top}\right\Vert +3\beta_{2}\left\Vert \bm{I}_{n}-\|\bm{x}\|_{2}^{2}\bm{I}_{n}\right\Vert \\
 & \leq\delta\left(1+2C_{1}\right)^{2}+18\beta_{1}C_{1}\left(4C_{1}+1\right)+6\beta_{2}C_{1}\left(4C_{1}+1\right).
\end{align*}
This further implies 
\[
\bm{\Lambda}_{4}\succeq3\left(\beta_{1}\bm{x}^{\star}\bm{x}^{\star\top}+\beta_{2}\bm{I}_{n}\right)-\left[\delta\left(1+2C_{1}\right)^{2}+18\beta_{1}C_{1}\left(4C_{1}+1\right)+6\beta_{2}C_{1}\left(4C_{1}+1\right)\right]\bm{I}_{n}.
\]
Recognizing that $\beta_{1}$ (resp.~$\beta_{2}$) approaches 2 (resp.~1)
as $C$ grows, we can thus take $C_{1}$ small enough and $C$ large
enough to guarantee that 
\[
\bm{\Lambda}_{4}\succeq5\bm{x}^{\star}\bm{x}^{\star\top}+2\bm{I}_{n}.
\]
Putting the preceding two bounds on $\bm{\Lambda}_{4}$ and $\bm{\Lambda}_{5}$
together yields 
\begin{align*}
\nabla^{2}f\left(\bm{x}\right) & \succeq5\bm{x}^{\star}\bm{x}^{\star\top}+2\bm{I}_{n}-\left[\left(1+\delta\right)\cdot\bm{I}_{n}+2\bm{x}^{\star}\bm{x}^{\star\top}\right]\succeq\left({1}/{2}\right)\cdot\bm{I}_{n}
\end{align*}
as claimed.

\subsection{Proof of Lemma \ref{lem:error-contraction-Hessian-WF}\label{subsec:Proof-of-Lemma-contraction-WF}}

Using the update rule (cf.~(\ref{eq:gradient_update-WF})) as well
as the fundamental theorem of calculus \cite[Chapter XIII, Theorem 4.2]{lang1993real},
we get 
\begin{align*}
\bm{x}^{t+1}-\bm{x}^{\star} & =\bm{x}^{t}-\eta\nabla f\left(\bm{x}^{t}\right)-\left[\bm{x}^{\star}-\eta\nabla f\left(\bm{x}^{\star}\right)\right]=\left[\bm{I}_{n}-\eta\int_{0}^{1}\nabla^{2}f\left(\bm{x}\left(\tau\right)\right)\mathrm{d}\tau\right]\left(\bm{x}^{t}-\bm{x}^{\star}\right),
\end{align*}
where we denote $\bm{x}\left(\tau\right)=\bm{x}^{\star}+\tau(\bm{x}^{t}-\bm{x}^{\star})$, $0\leq \tau\leq 1$.
Here, the first equality makes use of the fact that $\nabla f(\bm{x}^{\star})=\bm{0}$. Under the condition \eqref{eq:WF-hessian-condition}, it is self-evident that
for all $0\leq\tau\leq1$, 
\[
\left\Vert \bm{x}\left(\tau\right)-\bm{x}^{\star}\right\Vert _{2}=\|\tau(\bm{x}^{t}-\bm{x}^{\star})\|_{2}\leq 2C_{1}\qquad\text{and}
\]
\[
\max_{1\leq l\leq m}\left|\bm{a}_{l}^{\top}\left(\bm{x}(\tau)-\bm{x}^{\star}\right)\right|\leq\max_{1\leq l\leq m}\left|\bm{a}_{l}^{\top}\tau\left(\bm{x}^{t}-\bm{x}^{\star}\right)\right|\leq C_{2}\sqrt{\log n}.
\]
This means that for all $0\leq\tau\leq1$, 
\[
\left({1} / {2}\right)\cdot\bm{I}_{n}\preceq\nabla^{2}f\left(\bm{x}(\tau)\right)\preceq\left[5C_{2}\left(10+C_{2}\right)\log n\right]\cdot\bm{I}_{n}
\]
in view of Lemma \ref{lemma:wf_hessian}. Picking $\eta\leq{1} / \left[5C_{2}\left(10+C_{2}\right)\log n\right]$
(and hence $\|\eta\nabla^{2}f(\bm{x}(\tau))\|\leq1$), one sees that
\[
\bm{0}\preceq\bm{I}_{n}-\eta\int_{0}^{1}\nabla^{2}f\left(\bm{x}\left(\tau\right)\right)\mathrm{d}\tau\preceq\left(1-\eta/2\right)\cdot\bm{I}_{n},
\]
which immediately yields 
\begin{align*}
\left\Vert \bm{x}^{t+1}-\bm{x}^{\star}\right\Vert _{2} & \leq\left\Vert \bm{I}_{n}-\eta\int_{0}^{1}\nabla^{2}f\left(\bm{x}\left(\tau\right)\right)\mathrm{d}\tau\right\Vert \cdot\left\Vert \bm{x}^{t}-\bm{x}^{\star}\right\Vert _{2}\leq\left(1-\eta/2\right)\left\Vert \bm{x}^{t}-\bm{x}^{\star}\right\Vert _{2}.
\end{align*}

\subsection{Proof of Lemma \ref{lemma:WF-t-iteration-enough}\label{subsec:Proof-of-Lemma-WF-t-iteration}}
We start with proving \eqref{eq:WF-theoretical-guarantee}. 
For all $0\leq t\leq T_{0}$, invoke Lemma \ref{lem:error-contraction-Hessian-WF} recursively with the conditions \eqref{eq:WF-induction} to reach
\begin{equation}
\left\Vert \bm{x}^{t}-\bm{x}^{\star}\right\Vert _{2}\leq(1-\eta/2)^{t}\big\|\bm{x}^{0}-\bm{x}^{\star}\big\|_{2}\leq C_{1}(1-\eta/2)^{t}\left\Vert \bm{x}^{\star}\right\Vert _{2}.
\end{equation}
This finishes the proof of \eqref{eq:WF-theoretical-guarantee} for $0\leq t\leq T_{0}$ and also reveals that
\begin{equation}
	\left\Vert \bm{x}^{T_{0}}-\bm{x}^{\star}\right\Vert _{2}\leq C_{1}(1-\eta/2)^{T_{0}}\left\Vert \bm{x}^{\star}\right\Vert _{2} \ll \frac{1}{n}\left\Vert \bm{x}^{\star}\right\Vert _{2}, 
\end{equation}
provided that $\eta\asymp1/\log n$. Applying the Cauchy-Schwarz inequality and the fact (\ref{eq:max-al-norm})
indicate that
\[
\max_{1\leq l\leq m}\left|\bm{a}_{l}^{\top}\left(\bm{x}^{T_0}-\bm{x}^{\star}\right)\right|\leq\max_{1\leq l\leq m}\|\bm{a}_{l}\|_{2}\|\bm{x}^{T_0}-\bm{x}^{\star}\|_{2}\leq\sqrt{6n}\cdot\frac{1}{n}\|\bm{x}^{\star}\|_{2}\ll C_{2}\sqrt{\log n},
\]
leading to the satisfaction of \eqref{eq:WF-hessian-condition}. Therefore, invoking Lemma~\ref{lem:error-contraction-Hessian-WF} yields
\[
\left\Vert \bm{x}^{T_{0}+1}-\bm{x}^{\star}\right\Vert _{2}\leq\left(1-\eta/2\right)\left\Vert \bm{x}^{T_{0}}-\bm{x}^{\star}\right\Vert _{2}\ll\frac{1}{n}\|\bm{x}^{\star}\|_{2}.
\]
One can then repeat this argument to arrive at for all $t > T_{0}$
\begin{equation}
\left\Vert \bm{x}^{t}-\bm{x}^{\star}\right\Vert _{2}\leq\left(1-\eta/2\right)^{t}\big\|\bm{x}^{0}-\bm{x}^{\star}\big\|_{2}\leq C_{1}\left(1-\eta/2\right)^{t}\left\Vert \bm{x}^{\star}\right\Vert _{2}\ll\frac{1}{n}\|\bm{x}^{\star}\|_{2}.\label{eq:T+1-WF}
\end{equation}

We are left with \eqref{eq:implicit-reg-WF}. It is self-evident that the iterates from $0\leq t\leq T_{0}$ satisfy \eqref{eq:implicit-reg-WF} by assumptions. For $t>T_{0}$, we can use the Cauchy-Schhwarz inequality to obtain 
\begin{equation*}
	\max_{1\leq j\leq m}\big|\bm{a}_{j}^{\top}\big(\bm{x}^{t}-\bm{x}^{\star}\big)\big| \leq \max_{1\leq j\leq m}\left\Vert\bm{a}_j\right\Vert_{2}\left\Vert \bm{x}^{t}-\bm{x}^{\star}\right\Vert _{2}\ll \sqrt{n}\cdot \frac{1}{n}\leq C_{2}\sqrt{\log n},
\end{equation*}
where the penultimate relation uses the conditions \eqref{eq:max-al-norm} and \eqref{eq:T+1-WF}.

\subsection{Proof of Lemma \ref{lemma:WF-loop-inductive}\label{subsec:Proof-outline-proximity-LOO-WF}}

First, going through the same derivation as in (\ref{eq:LOOWF-induction-incoherence-t}) and (\ref{eq:LOOWF-induction-incoherence-t-0}) will result in
\begin{align}
\max_{1\leq l\leq m}\left|\bm{a}_{l}^{\top}\big(\bm{x}^{t,\left(l\right)}-\bm{x}^{\star}\big)\right|\leq C_{4}\sqrt{\log n}\label{eq:incoherent-t-WF}
\end{align}
for some $C_{4}<C_{2}$, which will be helpful for our analysis. 

We use the gradient update rules once again to decompose
\begin{align*}
\bm{x}^{t+1}-\bm{x}^{t+1,\left(l\right)} & =\bm{x}^{t}-\eta\nabla f\left(\bm{x}^{t}\right)-\left[\bm{x}^{t,\left(l\right)}-\eta\nabla f^{\left(l\right)}\big(\bm{x}^{t,\left(l\right)}\big)\right]\\
 & =\bm{x}^{t}-\eta\nabla f\left(\bm{x}^{t}\right)-\left[\bm{x}^{t,\left(l\right)}-\eta\nabla f\big(\bm{x}^{t,\left(l\right)}\big)\right]-\eta\left[\nabla f\big(\bm{x}^{t,\left(l\right)}\big)-\nabla f^{\left(l\right)}\big(\bm{x}^{t,\left(l\right)}\big)\right]\\
 & =\underbrace{\bm{x}^{t}-\bm{x}^{t,\left(l\right)}-\eta\left[\nabla f\left(\bm{x}^{t}\right)-\nabla f\big(\bm{x}^{t,\left(l\right)}\big)\right]}_{:=\bm{\nu}_{1}^{\left(l\right)}}\underset{:=\bm{\nu}_{2}^{(l)}}{-\underbrace{\eta\frac{1}{m}\left[\big(\bm{a}_{l}^{\top}\bm{x}^{t,\left(l\right)}\big)^{2}-\big(\bm{a}_{l}^{\top}\bm{x}^{\star}\big)^{2}\right]\big(\bm{a}_{l}^{\top}\bm{x}^{t,\left(l\right)}\big)\bm{a}_{l}},}
\end{align*}
where the last line comes from the definition of $\nabla f\left(\cdot\right)$ and $\nabla f^{(l)}\left(\cdot\right)$. 
\begin{enumerate}
\item We first control the term $\bm{\nu}_{2}^{\left(l\right)}$, which
is easier to deal with. Specifically, 
\begin{align*}
\|\bm{\nu}_{2}^{\left(l\right)}\|_{2} & \leq\eta\frac{\|\bm{a}_{l}\|_{2}}{m}\left|\big(\bm{a}_{l}^{\top}\bm{x}^{t,\left(l\right)}\big)^{2}-\big(\bm{a}_{l}^{\top}\bm{x}^{\star}\big)^{2}\right|\left|\bm{a}_{l}^{\top}\bm{x}^{t,\left(l\right)}\right|\\
 & \overset{(\text{i})}{\lesssim}C_{4}(C_{4}+5)(C_{4}+10)\eta\frac{n\log n}{m}\sqrt{\frac{\log n}{n}}\overset{(\text{ii})}{\leq}c\eta \sqrt{\frac{\log n}{n}},
\end{align*}
for any small constant $c>0$. Here (i) follows since
		\eqref{eq:max-al-norm} and, in view of (\ref{eq:aj-xnatural-incoherence-WF}) and (\ref{eq:incoherent-t-WF}),
\begin{align*}
\left|\big(\bm{a}_{l}^{\top}\bm{x}^{t,\left(l\right)}\big)^{2}-\big(\bm{a}_{l}^{\top}\bm{x}^{\star}\big)^{2}\right|&\leq\left|\bm{a}_{l}^{\top}\big(\bm{x}^{t,\left(l\right)}-\bm{x}^{\star}\big)\right|\left(\left|\bm{a}_{l}^{\top}\big(\bm{x}^{t,\left(l\right)}-\bm{x}^{\star}\big)\right|+2\left|\bm{a}_{l}^{\top}\bm{x}^{\star}\right|\right)\leq C_{4} (C_{4}+10)\log n, \\
\text{and}\qquad \left|\bm{a}_{l}^{\top}\bm{x}^{t,\left(l\right)}\right | &\leq \left|\bm{a}_{l}^{\top}\big(\bm{x}^{t,\left(l\right)}-\bm{x}^{\star}\big)\right| + \left|\bm{a}_{l}^{\top}\bm{x}^{\star}\right |\leq (C_{4}+5)\sqrt{\log n}.
\end{align*}
And (ii) holds as long as $m\gg n\log n$.
\item For the term $\bm{\nu}_{1}^{\left(l\right)}$, the fundamental theorem of calculus \cite[Chapter XIII, Theorem 4.2]{lang1993real} tells
us that 
\[
\bm{\nu}_{1}^{\left(l\right)}=\left[\bm{I}_{n}-\eta\int_{0}^{1}\nabla^{2}f\left(\bm{x}\left(\tau\right)\right)\mathrm{d}\tau\right]\big(\bm{x}^{t}-\bm{x}^{t,\left(l\right)}\big),
\]
where we abuse the notation and denote $\bm{x}\left(\tau\right)=\bm{x}^{t,\left(l\right)}+\tau(\bm{x}^{t}-\bm{x}^{t,\left(l\right)})$.
By the induction hypotheses (\ref{eq:WF-induction-step-3}) and the condition (\ref{eq:incoherent-t-WF}), one can verify that
\begin{equation}
\big\|\bm{x}\left(\tau\right)-\bm{x}^{\star}\big\|_{2}\leq\tau\big\|\bm{x}^{t}-\bm{x}^{\star}\big\|_{2}+(1-\tau)\big\|\bm{x}^{t,(l)}-\bm{x}^{\star}\big\|_{2}\leq2C_{1}\qquad\text{and}\label{eq:WF-loo-condition-ell-2}
\end{equation}
\[
\max_{1\leq l\leq m}\left|\bm{a}_{l}^{\top}\big(\bm{x}\left(\tau\right)-\bm{x}^{\star}\big)\right|\leq\tau\max_{1\leq l\leq m}\left|\bm{a}_{l}^{\top}\big(\bm{x}^{t}-\bm{x}^{\star}\big)\right|+(1-\tau)\max_{1\leq l\leq m}\left|\bm{a}_{l}^{\top}\big(\bm{x}^{t,(l)}-\bm{x}^{\star}\big)\right|\leq C_{2}\sqrt{\log n}
\]
for all $0\leq\tau\leq1$, as long as $C_{4}\leq C_{2}$. The second
line follows directly from (\ref{eq:incoherent-t-WF}). To see why
(\ref{eq:WF-loo-condition-ell-2}) holds, we note that
\[
\big\|\bm{x}^{t,(l)}-\bm{x}^{\star}\big\|_{2}\leq\big\|\bm{x}^{t,(l)}-\bm{x}^{t}\big\|_{2}+\big\|\bm{x}^{t}-\bm{x}^{\star}\big\|_{2}\leq C_{3}\sqrt{\frac{\log n}{n}}+C_{1},
\]
where the second inequality follows from the induction hypotheses
(\ref{eq:induction-loop-WF-step-3}) and (\ref{eq:WF-induction-L2error-step-3}).
This combined with (\ref{eq:WF-induction-L2error-step-3}) gives
\begin{align*}
\left\Vert \bm{x}\left(\tau\right)-\bm{x}^{\star}\right\Vert _{2} \leq\tau C_{1}+\left(1-\tau\right)\left(C_{3}\sqrt{\frac{\log n}{n}}+C_{1}\right)\leq2C_{1}
\end{align*}
as long as $n$ is large enough, thus justifying (\ref{eq:WF-loo-condition-ell-2}).
Hence by Lemma \ref{lemma:wf_hessian}, $\nabla^{2}f\left(\bm{x}\left(\tau\right)\right)$
is positive definite and almost well-conditioned. By choosing $0<\eta\leq {1} / \left[{5C_{2}\left(10+C_{2}\right)\log n}\right]$,
we get 
\begin{align*}
\big\|\bm{\nu}_{1}^{\left(l\right)}\big\|_{2} & \leq\left(1-\eta/2\right)\big\|\bm{x}^{t}-\bm{x}^{t,\left(l\right)}\big\|_{2}.
\end{align*}
\item Combine the preceding bounds on $\bm{\nu}_{1}^{\left(l\right)}$ and
$\bm{\nu}_{2}^{\left(l\right)}$ as well as the induction bound (\ref{eq:induction-loop-WF-step-3})
to arrive at 
\begin{align}
\big\|\bm{x}^{t+1}-\bm{x}^{t+1,\left(l\right)}\big\|_{2} & \leq\left(1-\eta/2\right)\big\|\bm{x}^{t}-\bm{x}^{t,\left(l\right)}\big\|_{2}+c\eta\sqrt{\frac{\log n}{n}}\leq C_{3}\sqrt{\frac{\log n}{n}}.\label{eq:t1-UB-WF}
\end{align}
This establishes (\ref{eq:induction-loop-WF})
for the $(t+1)$th iteration. 
\end{enumerate}

\subsection{Proof of Lemma \ref{lemma:wf_L2-base}\label{subsec:Proof-of-Lemma-wf_L2-base}}
In view of the assumption (\ref{eq:assumption-x0-rotation-WF}) that $\left\Vert \bm{x}^{0}-\bm{x}^{\star}\right\Vert _{2}\leq\left\Vert \bm{x}^{0}+\bm{x}^{\star}\right\Vert _{2}$ and the fact that $\bm{x}^{0} =\sqrt{\lambda_{1}\left(\bm{Y}\right)/3}\;\tilde{\bm{x}}^{0}$ for some $\lambda_{1}\left(\bm{Y}\right)>0$ (which we will verify below), it is straightforward to see that 
\[
\big\|\tilde{\bm{x}}^{0}-\bm{x}^{\star}\big\|_{2}\leq\big\|\tilde{\bm{x}}^{0}+\bm{x}^{\star}\big\|_{2}.
\]
One can then invoke the Davis-Kahan sin$\Theta$ theorem \cite[Corollary 1]{MR3371006}
to obtain 
\[
\big\|\tilde{\bm{x}}^{0}-\bm{x}^{\star}\big\|_{2}\leq2\sqrt{2}\frac{\left\Vert \bm{Y}-\EE\left[\bm{Y}\right]\right\Vert }{\lambda_{1}\left(\EE\left[\bm{Y}\right]\right)-\lambda_{2}\left(\EE\left[\bm{Y}\right]\right)}.
\]
Note that \eqref{eq:Y-EY-delta} --- $\|\bm{Y}-\mathbb{E}[\bm{Y}]\|\leq\delta$
--- is a direct consequence of Lemma~\ref{lemma:spectral-fix-WF}.
Additionally, the fact that $\EE\left[\bm{Y}\right]=\bm{I}+2\bm{x}^{\star}\bm{x}^{\star\top}$
gives $\lambda_{1}\left(\EE\left[\bm{Y}\right]\right)=3$, $\lambda_{2}\left(\EE\left[\bm{Y}\right]\right)=1$, and $\lambda_{1}\left(\EE\left[\bm{Y}\right]\right)-\lambda_{2}\left(\EE\left[\bm{Y}\right]\right)=2$. Combining this spectral gap and the inequality $\|\bm{Y}-\mathbb{E}[\bm{Y}]\|\leq\delta$,
we arrive at
\[
\big\|\tilde{\bm{x}}^{0}-\bm{x}^{\star}\big\|_{2}\leq\sqrt{2}\delta.
\]

To connect this bound with $\bm{x}^{0}$, we need to take into account
the scaling factor $\sqrt{\lambda_{1}\left(\bm{Y}\right)/3}$. To
this end, it follows from Weyl's inequality and (\ref{eq:Y-EY-delta})
that 
\[
\left|\lambda_{1}\left(\bm{Y}\right)-3\right|=\left|\lambda_{1}\left(\bm{Y}\right)-\lambda_{1}\left(\EE\left[\bm{Y}\right]\right)\right|\leq\left\Vert \bm{Y}-\EE\left[\bm{Y}\right]\right\Vert \leq\delta
\]
and, as a consequence, $\lambda_{1}\left(\bm{Y}\right)\geq3-\delta>0$
when $\delta\leq1$. This further implies that 
\begin{equation}
\left|\sqrt{\frac{\lambda_{1}\left(\bm{Y}\right)}{3}}-1\right|=\left|\frac{\frac{\lambda_{1}\left(\bm{Y}\right)}{3}-1}{\sqrt{\frac{\lambda_{1}\left(\bm{Y}\right)}{3}}+1}\right|\leq\left|\frac{\lambda_{1}\left(\bm{Y}\right)}{3}-1\right|\leq\frac{1}{3}\delta,\label{eq:eigen-diff-PR}
\end{equation}
where we have used the elementary identity $\sqrt{a}-\sqrt{b}=\left(a-b\right)/(\sqrt{a}+\sqrt{b})$.
With these bounds in place, we can use the triangle inequality to get
\begin{align*}
\left\Vert \bm{x}^{0}-\bm{x}^{\star}\right\Vert _{2} & =\left\Vert \sqrt{\frac{\lambda_{1}\left(\bm{Y}\right)}{3}}\tilde{\bm{x}}^{0}-\bm{x}^{\star}\right\Vert _{2}=\left\Vert \sqrt{\frac{\lambda_{1}\left(\bm{Y}\right)}{3}}\tilde{\bm{x}}^{0}-\tilde{\bm{x}}^{0}+\tilde{\bm{x}}^{0}-\bm{x}^{\star}\right\Vert _{2}
 \nonumber\\
 & \leq\left|\sqrt{\frac{\lambda_{1}\left(\bm{Y}\right)}{3}}-1\right|+\left\Vert \tilde{\bm{x}}^{0}-\bm{x}^{\star}\right\Vert _{2}\\
 & \leq\frac{1}{3}\delta+\sqrt{2}\delta \,\leq\, 2\delta.
\end{align*}


\subsection{Proof of Lemma \ref{lemma:wf_loop-base}\label{subsec:Proof-of-Lemma-wf_loop-base}}
To begin with, repeating the same argument as in Lemma~\ref{lemma:wf_L2-base}
(which we omit here for conciseness), we see that for any fixed constant
$\delta>0$, 
\begin{equation}
\left\Vert \bm{Y}^{(l)}-\EE\left[\bm{Y}^{(l)}\right]\right\Vert \leq\delta,\quad\|\bm{x}^{0,(l)}-\bm{x}^{\star}\|_{2}\leq2\delta,\quad\big\|\tilde{\bm{x}}^{0,(l)}-\bm{x}^{\star}\big\|_{2}\leq\sqrt{2}\delta,\qquad1\leq l\leq m\label{eq:WFLOO_init}
\end{equation}
holds with probability at least $1-O(mn^{-10})$ as long as $m\gg n \log n$. The $\ell_{2}$
bound on $\|\bm{x}^{0}-\bm{x}^{0,\left(l\right)}\|_{2}$ is derived
as follows. 
\begin{enumerate}
\item We start by controlling $\big\|\tilde{\bm{x}}^{0}-\tilde{\bm{x}}^{0,\left(l\right)}\big\|_{2}$.
Combining \eqref{eq:WF_init} and \eqref{eq:WFLOO_init} yields
\[
\big\|\tilde{\bm{x}}^{0}-\tilde{\bm{x}}^{0,\left(l\right)}\big\|_{2}\leq\big\|\tilde{\bm{x}}^{0}-\bm{x}^{\star}\big\|_{2}+\big\|\tilde{\bm{x}}^{0,\left(l\right)}-\bm{x}^{\star}\big\|_{2}\leq2\sqrt{2}\delta.
\]
For $\delta$ sufficiently small, this implies that $\big\|\tilde{\bm{x}}^{0}-\tilde{\bm{x}}^{0,\left(l\right)}\big\|_{2}\leq\big\|\tilde{\bm{x}}^{0}+\tilde{\bm{x}}^{0,\left(l\right)}\big\|_{2}$,
and hence the Davis-Kahan sin$\Theta$ theorem \cite{davis1970rotation}
gives
\begin{align}
\big\|\tilde{\bm{x}}^{0}-\tilde{\bm{x}}^{0,\left(l\right)}\big\|_{2} & \leq\frac{\left\Vert \big(\bm{Y}-\bm{Y}^{(l)}\big)\tilde{\bm{x}}^{0,\left(l\right)}\right\Vert _{2}}{\lambda_{1}\left(\bm{Y}\right)-\lambda_{2}\left(\bm{Y}^{\left(l\right)}\right)}\leq\big\Vert \big(\bm{Y}-\bm{Y}^{(l)}\big)\tilde{\bm{x}}^{0,\left(l\right)}\big\Vert _{2}.\label{eq:xtilde0-xtilde0l-WF}
\end{align}
Here, the second inequality uses Weyl's inequality:
\begin{align*}
\lambda_{1}\big(\bm{Y}\big)-\lambda_{2}\big(\bm{Y}^{\left(l\right)}\big) & \geq\lambda_{1}(\mathbb{E}[\bm{Y}])-\big\|\bm{Y}-\mathbb{E}[\bm{Y}]\big\|-\lambda_{2}(\mathbb{E}[\bm{Y}^{(l)}])-\big\|\bm{Y}^{\left(l\right)}-\mathbb{E}[\bm{Y}^{(l)}]\big\|\\
 & \geq3-\delta-1-\delta\geq1,
\end{align*}
with the proviso that $\delta \leq 1/2$. 
\item We now connect $\|\bm{x}^{0}-\bm{x}^{0,(l)}\|_{2}$ with $\|\tilde{\bm{x}}^{0}-\tilde{\bm{x}}^{0,(l)}\|_{2}$.
Applying the Weyl's inequality and (\ref{eq:Y-EY-delta}) yields
\begin{equation}
\left|\lambda_{1}\left(\bm{Y}\right)-3\right|\leq\|\bm{Y}-\mathbb{E}[\bm{Y}]\|\leq\delta\qquad\Longrightarrow\qquad\lambda_{1}(\bm{Y})\in\left[3-\delta,3+\delta\right]\subseteq[2,4]\label{eq:lambda1-Y-Yl-UB}
\end{equation}
and, similarly, $\lambda_{1}(\bm{Y}^{(l)}),\|\bm{Y}\|,\|\bm{Y}^{(l)}\|\in\left[2,4\right]$.
Invoke Lemma \ref{lemma:eigenvalue-difference} to arrive at 
\begin{align}
	\frac{1}{\sqrt{3}}\big\Vert \bm{x}^{0}-\bm{x}^{0,(l)}\big\Vert _{2} & \leq\frac{\big\|\big(\bm{Y}-\bm{Y}^{(l)}\big) \tilde{\bm{x}}^{0,(l)}\big\|_{2}}{2\sqrt{2}}+\left(2+\frac{4}{\sqrt{2}}\right)\big\Vert \tilde{\bm{x}}^{0}-\tilde{\bm{x}}^{0,(l)}\big\Vert _{2} \nonumber\\
	&\leq6\big\|\big(\bm{Y}-\bm{Y}^{(l)}\big) \tilde{\bm{x}}^{0,(l)}\big\|_{2},\label{eq:PR-loop}
\end{align}
where the last inequality comes from (\ref{eq:xtilde0-xtilde0l-WF}).
\item Everything then boils down to controlling $\left\Vert \left(\bm{Y}-\bm{Y}^{(l)}\right)\tilde{\bm{x}}^{0,\left(l\right)}\right\Vert _{2}$.
Towards this we observe that 
\begin{align}
\max_{1\leq l\leq m}\big\Vert \big(\bm{Y}-\bm{Y}^{(l)}\big)\tilde{\bm{x}}^{0,\left(l\right)}\big\Vert _{2} & =\max_{1\leq l\leq m}\frac{1}{m}\left\Vert \left(\bm{a}_{l}^{\top}\bm{x}^{\star}\right)^{2}\bm{a}_{l}\bm{a}_{l}^{\top}\tilde{\bm{x}}^{0,\left(l\right)}\right\Vert _{2}\nonumber \\
 & \leq\max_{1\leq l\leq m}\frac{\left(\bm{a}_{l}^{\top}\bm{x}^{\star}\right)^{2}\big|\bm{a}_{l}^{\top}\tilde{\bm{x}}^{0,\left(l\right)}\big|\big\|\bm{a}_{l}\big\|_{2}}{m}\nonumber \\
 & \overset{\left(\text{i}\right)}{\lesssim}\frac{\log n\cdot\sqrt{\log n}\cdot\sqrt{n}}{m} \nonumber\\
 & \asymp\sqrt{\frac{\log n}{n}}\cdot\frac{n\log n}{m}.\label{eq:PR-loop-davis-kahan}
\end{align}
The inequality (i) makes use of the fact $\max_{l}\left|\bm{a}_{l}^{\top}\bm{x}^{\star}\right|\leq5\sqrt{\log n}$
(cf.~(\ref{eq:aj-xnatural-incoherence-WF})), the bound $\max_{l}\|\bm{a}_{l}\|_{2}\leq6\sqrt{n}$
(cf.~(\ref{eq:max-al-norm})), and $\max_{l}\left|\bm{a}_{l}^{\top}\tilde{\bm{x}}^{0,\left(l\right)}\right|\leq5\sqrt{\log n}$
(due to statistical independence and standard Gaussian concentration).
As long as $m/(n\log n)$ is sufficiently large, substituting the
above bound (\ref{eq:PR-loop-davis-kahan}) into (\ref{eq:PR-loop}) leads us to conclude that 
\begin{equation}
\max_{1\leq l\leq m}\big\|\bm{x}^{0}-\bm{x}^{0,\left(l\right)}\big\|_{2}\leq C_{3}\sqrt{\frac{\log n}{n}}\label{eq:x0-x0l-gap-WF}
\end{equation}
for any constant $C_{3}>0$. 
\end{enumerate}


\section{Proofs for matrix completion}

\label{sec:proof-MC}

Before proceeding to the proofs, let us record an immediate consequence
of the incoherence property (\ref{eq:incoherence-U-MC}):
\begin{align}
\left\Vert \bm{X}^{\star}\right\Vert _{2,\infty} & \leq\sqrt{\frac{\kappa\mu}{n}}\left\Vert \bm{X}^{\star}\right\Vert _{\mathrm{F}}\leq\sqrt{\frac{\kappa\mu r}{n}}\left\Vert \bm{X}^{\star}\right\Vert .\label{eq:incoherence-X-MC}
\end{align}
where $\kappa=\sigma_{\max}/\sigma_{\min}$ is the condition number
of $\bm{M}^{\star}$. This follows since 
\begin{align*}
\left\Vert \bm{X}^{\star}\right\Vert _{2,\infty} & =\left\Vert \bm{U}^{\star}\big(\bm{\Sigma}^{\star}\big)^{1/2}\right\Vert _{2,\infty}\leq\left\Vert \bm{U}^{\star}\right\Vert _{2,\infty}\big\|\big(\bm{\Sigma}^{\star}\big)^{1/2}\big\|\\
 & \leq\sqrt{\frac{\mu}{n}}\left\Vert \bm{U}^{\star}\right\Vert _{\mathrm{F}}\big\|\big(\bm{\Sigma}^{\star}\big)^{1/2}\big\|\leq\sqrt{\frac{\mu}{n}}\left\Vert \bm{U}^{\star}\right\Vert _{\mathrm{F}}\sqrt{\kappa\sigma_{\min}}\\
 & \leq\sqrt{\frac{\kappa\mu}{n}}\left\Vert \bm{X}^{\star}\right\Vert _{\mathrm{F}}\leq\sqrt{\frac{\kappa\mu r}{n}}\left\Vert \bm{X}^{\star}\right\Vert .
\end{align*}

Unless otherwise specified, we use the indicator variable $\delta_{j,k}$ to denote whether the entry in the location $(j,k)$ is included in $\Omega$. Under our model,  $\delta_{j,k}$ is a Bernoulli random variable with mean $p$.

\subsection{Proof of Lemma \ref{lemma:hessian-MC}\label{subsec:Proof-of-Lemma-hessian-MC}}

By the expression of the Hessian in \eqref{eq:hessian_quadratic-MC},
one can decompose 
\begin{align*}
 & \mathrm{vec}\left(\bm{V}\right)^{\top}\nabla^{2}f_{\mathrm{clean}}\left(\bm{X}\right)\mathrm{vec}\left(\bm{V}\right)=\frac{1}{2p}\left\Vert \mathcal{P}_{\Omega}\left(\bm{V}\bm{X}^{\top}+\bm{X}\bm{V}^{\top}\right)\right\Vert _{\mathrm{F}}^{2}+\frac{1}{p}\left\langle \mathcal{P}_{\Omega}\left(\bm{X}\bm{X}^{\top}-\bm{M}^{\star}\right),\bm{V}\bm{V}^{\top}\right\rangle \\
 & \quad=\underbrace{\frac{1}{2p}\left\Vert \mathcal{P}_{\Omega}\left(\bm{V}\bm{X}^{\top}+\bm{X}\bm{V}^{\top}\right)\right\Vert _{\mathrm{F}}^{2}-\frac{1}{2p}\left\Vert \mathcal{P}_{\Omega}\left(\bm{V}\bm{X}^{\star\top}+\bm{X}^{\star}\bm{V}^{\top}\right)\right\Vert _{\mathrm{F}}^{2}}_{:=\alpha_{1}}+\underbrace{\frac{1}{p}\left\langle \mathcal{P}_{\Omega}\left(\bm{X}\bm{X}^{\top}-\bm{M}^{\star}\right),\bm{V}\bm{V}^{\top}\right\rangle }_{:=\alpha_{2}}\\
 & \quad\quad+\underbrace{\frac{1}{2p}\left\Vert \mathcal{P}_{\Omega}\left(\bm{V}\bm{X}^{\star\top}+\bm{X}^{\star}\bm{V}^{\top}\right)\right\Vert _{\mathrm{F}}^{2}-\frac{1}{2}\left\Vert \bm{V}\bm{X}^{\star\top}+\bm{X}^{\star}\bm{V}^{\top}\right\Vert _{\mathrm{F}}^{2}}_{:=\alpha_{3}}+\underbrace{\frac{1}{2}\left\Vert \bm{V}\bm{X}^{\star\top}+\bm{X}^{\star}\bm{V}^{\top}\right\Vert _{\mathrm{F}}^{2}}_{:=\alpha_{4}}.
\end{align*}
The basic idea 
is to demonstrate that: $\left(\text{1}\right)$ $\alpha_{4}$ is
bounded both from above and from below, and $\left(\text{2}\right)$
the first three terms are sufficiently small in size compared to $\alpha_{4}$. 
\begin{enumerate}
\item We start by controlling $\alpha_{4}$. It is immediate to derive the
following upper bound 
\begin{align*}
\alpha_{4} & \leq\left\Vert \bm{V}\bm{X}^{\star\top}\right\Vert _{\mathrm{F}}^{2}+\left\Vert \bm{X}^{\star}\bm{V}^{\top}\right\Vert _{\mathrm{F}}^{2}\leq2\|\bm{X}^{\star}\|^{2}\left\Vert \bm{V}\right\Vert _{\mathrm{F}}^{2}=2\sigma_{\max}\left\Vert \bm{V}\right\Vert _{\mathrm{F}}^{2}.
\end{align*}
When it comes to the lower bound, one discovers that 
\begin{align}
\alpha_{4} & =\frac{1}{2}\left\{ \left\Vert \bm{V}\bm{X}^{\star\top}\right\Vert _{\mathrm{F}}^{2}+\left\Vert \bm{X}^{\star}\bm{V}^{\top}\right\Vert _{\mathrm{F}}^{2}+2\mathrm{Tr}\left(\bm{X}^{\star\top}\bm{V}\bm{X}^{\star\top}\bm{V}\right)\right\} \nonumber \\
 & \geq\sigma_{\min}\left\Vert \bm{V}\right\Vert _{\mathrm{F}}^{2}+\mathrm{Tr}\left[\left(\bm{Z}+\bm{X}^{\star}-\bm{Z}\right)^{\top}\bm{V}\left(\bm{Z}+\bm{X}^{\star}-\bm{Z}\right)^{\top}\bm{V}\right]\nonumber \\
 & \geq\sigma_{\min}\left\Vert \bm{V}\right\Vert _{\mathrm{F}}^{2}+\mathrm{Tr}\left(\bm{Z}^{\top}\bm{V}\bm{Z}^{\top}\bm{V}\right)-2\left\Vert \bm{Z}-\bm{X}^{\star}\right\Vert \left\Vert \bm{Z}\right\Vert \left\Vert \bm{V}\right\Vert _{\mathrm{F}}^{2}-\left\Vert \bm{Z}-\bm{X}^{\star}\right\Vert ^{2}\left\Vert \bm{V}\right\Vert _{\mathrm{F}}^{2}\nonumber \\
 & \geq\left(\sigma_{\min}-5\delta\sigma_{\max}\right)\left\Vert \bm{V}\right\Vert _{\mathrm{F}}^{2}+\mathrm{Tr}\left(\bm{Z}^{\top}\bm{V}\bm{Z}^{\top}\bm{V}\right),\label{eq:alpha_4-MC}
\end{align}
where the last line comes from the assumptions that 
\[
\left\Vert \bm{Z}-\bm{X}^{\star}\right\Vert \leq\delta\left\Vert \bm{X}^{\star}\right\Vert \leq\left\Vert \bm{X}^{\star}\right\Vert \qquad\text{and}\qquad\left\Vert \bm{Z}\right\Vert \leq\left\Vert \bm{Z}-\bm{X}^{\star}\right\Vert +\left\Vert \bm{X}^{\star}\right\Vert \leq2\left\Vert \bm{X}^{\star}\right\Vert .
\]
With our assumption $\bm{V}=\bm{Y}\bm{H}_{Y}-\bm{Z}$ in mind, it
comes down to controlling 
\[
\mathrm{Tr}\left(\bm{Z}^{\top}\bm{V}\bm{Z}^{\top}\bm{V}\right)=\mathrm{Tr}\left[\bm{Z}^{\top}\left(\bm{Y}\bm{H}_{Y}-\bm{Z}\right)\bm{Z}^{\top}\left(\bm{Y}\bm{H}_{Y}-\bm{Z}\right)\right].
\]
From the definition of $\bm{H}_{Y}$, we see from Lemma \ref{lemma:opp}
that $\bm{Z}^{\top}\bm{Y}\bm{H}_{Y}$ (and hence $\bm{Z}^{\top}\left(\bm{Y}\bm{H}_{Y}-\bm{Z}\right)$)
is a symmetric matrix, which implies that 
\[
\text{Tr}\left[\bm{Z}^{\top}\left(\bm{Y}\bm{H}_{Y}-\bm{Z}\right)\bm{Z}^{\top}\left(\bm{Y}\bm{H}_{Y}-\bm{Z}\right)\right]\geq0.
\]
Substitution into (\ref{eq:alpha_4-MC}) gives 
\[
\alpha_{4}\geq\left(\sigma_{\min}-5\delta\sigma_{\max}\right)\left\Vert \bm{V}\right\Vert _{\mathrm{F}}^{2}\geq\frac{9}{10}\sigma_{\min}\left\Vert \bm{V}\right\Vert _{\mathrm{F}}^{2},
\]
provided that $\kappa\delta\leq 1/50.$ 
\item For $\alpha_{1}$, we consider the following quantity 
\begin{align*}
\big\|\mathcal{P}_{\Omega}\left(\bm{V}\bm{X}^{\top}+\bm{X}\bm{V}^{\top}\right)\big\|_{\mathrm{F}}^{2} & =\left\langle \cP_{\Omega}\left(\bm{V}\bm{X}^{\top}\right),\cP_{\Omega}\left(\bm{V}\bm{X}^{\top}\right)\right\rangle +\left\langle \cP_{\Omega}\left(\bm{V}\bm{X}^{\top}\right),\cP_{\Omega}\left(\bm{X}\bm{V}^{\top}\right)\right\rangle \\
 & \quad+\left\langle \cP_{\Omega}\left(\bm{X}\bm{V}^{\top}\right),\cP_{\Omega}\left(\bm{V}\bm{X}^{\top}\right)\right\rangle +\left\langle \cP_{\Omega}\left(\bm{X}\bm{V}^{\top}\right),\cP_{\Omega}\left(\bm{X}\bm{V}^{\top}\right)\right\rangle \\
 & =2\left\langle \cP_{\Omega}\left(\bm{V}\bm{X}^{\top}\right),\cP_{\Omega}\left(\bm{V}\bm{X}^{\top}\right)\right\rangle +2\left\langle \cP_{\Omega}\left(\bm{V}\bm{X}^{\top}\right),\cP_{\Omega}\left(\bm{X}\bm{V}^{\top}\right)\right\rangle .
\end{align*}
Similar decomposition can be performed on $\big\|\mathcal{P}_{\Omega}\left(\bm{V}\bm{X}^{\star\top}+\bm{X}^{\star}\bm{V}^{\top}\right)\big\|_{\mathrm{F}}^{2}$
as well. These identities yield 
\begin{align*}
\alpha_{1} & =\underbrace{\frac{1}{p}\left[\left\langle \cP_{\Omega}\big(\bm{V}\bm{X}^{\top}\big),\cP_{\Omega}\left(\bm{V}\bm{X}^{\top}\right)\right\rangle -\left\langle \cP_{\Omega}\left(\bm{V}\bm{X}^{\star\top}\right),\cP_{\Omega}\left(\bm{V}\bm{X}^{\star\top}\right)\right\rangle \right]}_{:=\beta_{1}}\\
 & \quad+\underbrace{\frac{1}{p}\left[\left\langle \cP_{\Omega}\big(\bm{V}\bm{X}^{\top}\big),\cP_{\Omega}\left(\bm{X}\bm{V}^{\top}\right)\right\rangle -\left\langle \cP_{\Omega}\left(\bm{V}\bm{X}^{\star\top}\right),\cP_{\Omega}\left(\bm{X}^{\star}\bm{V}^{\top}\right)\right\rangle \right]}_{:=\beta_{2}}.
\end{align*}
For $\beta_{2}$, one has 
\begin{align*}
\beta_{2} & =\frac{1}{p}\left\langle \cP_{\Omega}\left(\bm{V}\left(\bm{X}-\bm{X}^{\star}\right)^{\top}\right),\cP_{\Omega}\left(\left(\bm{X}-\bm{X}^{\star}\right)\bm{V}^{\top}\right)\right\rangle \\
 & \quad+\frac{1}{p}\left\langle \cP_{\Omega}\left(\bm{V}\left(\bm{X}-\bm{X}^{\star}\right)^{\top}\right),\cP_{\Omega}\left(\bm{X}^{\star}\bm{V}^{\top}\right)\right\rangle +\frac{1}{p}\left\langle \cP_{\Omega}\left(\bm{V}\bm{X}^{\star\top}\right),\cP_{\Omega}\left(\left(\bm{X}-\bm{X}^{\star}\right)\bm{V}^{\top}\right)\right\rangle 
\end{align*}
which together with the inequality $\left|\langle\bm{A},\bm{B}\rangle\right|\leq\|\bm{A}\|_{\mathrm{F}}\|\bm{B}\|_{\mathrm{F}}$
gives 
\begin{equation}
\left|\beta_{2}\right|\leq\frac{1}{p}\left\Vert \cP_{\Omega}\left(\bm{V}\left(\bm{X}-\bm{X}^{\star}\right)^{\top}\right)\right\Vert _{\mathrm{F}}^{2}+\frac{2}{p}\left\Vert \cP_{\Omega}\left(\bm{V}\left(\bm{X}-\bm{X}^{\star}\right)^{\top}\right)\right\Vert _{\mathrm{F}}\left\Vert \cP_{\Omega}\left(\bm{X}^{\star}\bm{V}^{\top}\right)\right\Vert _{\mathrm{F}}.\label{eq:beta2-UB-1-MC}
\end{equation}
This then calls for upper bounds on the following two terms 
\[
\frac{1}{\sqrt{p}}\left\Vert \cP_{\Omega}\left(\bm{V}\left(\bm{X}-\bm{X}^{\star}\right)^{\top}\right)\right\Vert _{\mathrm{F}}\qquad\text{and}\qquad\frac{1}{\sqrt{p}}\left\Vert \cP_{\Omega}\left(\bm{X}^{\star}\bm{V}^{\top}\right)\right\Vert _{\mathrm{F}}.
\]
The injectivity of $\mathcal{P}_{\Omega}$ (cf.~\cite[Section 4.2]{ExactMC09}
or Lemma \ref{lemma:injectivity})\textemdash when restricted to the
tangent space of $\bm{M}^{\star}$\textemdash gives: for any fixed constant $\gamma>0$, 
\[
\frac{1}{\sqrt{p}}\left\Vert \cP_{\Omega}\left(\bm{X}^{\star}\bm{V}^{\top}\right)\right\Vert _{\mathrm{F}}\leq\left(1+\gamma\right)\left\Vert \bm{X}^{\star}\bm{V}^{\top}\right\Vert _{\mathrm{F}}\leq\left(1+\gamma\right)\left\Vert \bm{X}^{\star}\right\Vert \left\Vert \bm{V}\right\Vert _{\mathrm{F}}
\]
with probability at least $1-O\left(n^{-10}\right)$, provided that $n^{2}p / (\mu nr\log n)$ is sufficiently large.
In addition, 
\begin{align*}
\frac{1}{p}\left\Vert \cP_{\Omega}\left(\bm{V}\left(\bm{X}-\bm{X}^{\star}\right)^{\top}\right)\right\Vert _{\mathrm{F}}^{2} & =\frac{1}{p}\sum_{1\leq j,k\leq n}\delta_{j,k}\left[\bm{V}_{j,\cdot}\left(\bm{X}_{k,\cdot}-\bm{X}_{k,\cdot}^{\star}\right)^{\top}\right]^{2}\\
 & =\sum_{1\leq j\leq n}\bm{V}_{j,\cdot}\left[\frac{1}{p}\sum_{1\leq k\leq n}\delta_{j,k}\left(\bm{X}_{k,\cdot}-\bm{X}_{k,\cdot}^{\star}\right)^{\top}\left(\bm{X}_{k,\cdot}-\bm{X}_{k,\cdot}^{\star}\right)\right]\bm{V}_{j,\cdot}^{\top}\\
 & \leq\max_{1\leq j\leq n}\left\Vert \frac{1}{p}\sum_{1\leq k\leq n}\delta_{j,k}\left(\bm{X}_{k,\cdot}-\bm{X}_{k,\cdot}^{\star}\right)^{\top}\left(\bm{X}_{k,\cdot}-\bm{X}_{k,\cdot}^{\star}\right)\right\Vert \left\Vert \bm{V}\right\Vert _{\mathrm{F}}^{2}\\
 & \leq\left\{ \frac{1}{p}\max_{1\leq j\leq n}\sum_{1\leq k\leq n}\delta_{j,k}\right\} \left\{ \max_{1\leq k\leq n}\left\Vert \bm{X}_{k,\cdot}-\bm{X}_{k,\cdot}^{\star}\right\Vert _{2}^{2}\right\} \left\Vert \bm{V}\right\Vert _{\mathrm{F}}^{2}\\
 & \leq\left(1+\gamma\right)n\left\Vert \bm{X}-\bm{X}^{\star}\right\Vert _{2,\infty}^{2}\left\Vert \bm{V}\right\Vert _{\mathrm{F}}^{2},
\end{align*}
with probability exceeding $1-O\left(n^{-10}\right)$, 
which holds as long as $np/\log n$ is sufficiently large. Taken collectively,
the above bounds yield that for any small constant $\gamma>0$, 
\begin{align*}
\left|\beta_{2}\right| & \leq\left(1+\gamma\right)n\left\Vert \bm{X}-\bm{X}^{\star}\right\Vert _{2,\infty}^{2}\left\Vert \bm{V}\right\Vert _{\mathrm{F}}^{2}+2\sqrt{\left(1+\gamma\right)n\left\Vert \bm{X}-\bm{X}^{\star}\right\Vert _{2,\infty}^{2}\left\Vert \bm{V}\right\Vert _{\mathrm{F}}^{2}\cdot\left(1+\gamma\right)^{2}\left\Vert \bm{X}^{\star}\right\Vert ^{2}\left\Vert \bm{V}\right\Vert _{\mathrm{F}}^{2}}\\
 & \lesssim\left(\epsilon^{2}n\left\Vert \bm{X}^{\star}\right\Vert _{2,\infty}^{2}+\epsilon\sqrt{n}\left\Vert \bm{X}^{\star}\right\Vert _{2,\infty}\left\Vert \bm{X}^{\star}\right\Vert \right)\left\Vert \bm{V}\right\Vert _{\mathrm{F}}^{2},
\end{align*}
where the last inequality makes use of the assumption $\|\bm{X}-\bm{X}^{\star}\|_{2,\infty}\leq\epsilon\|\bm{X}^{\star}\|_{2,\infty}$.
The same analysis can be repeated to control $\beta_{1}$. Altogether, we obtain 
\begin{align*}
\left|\alpha_{1}\right|  \leq\left|\beta_{1}\right|+\left|\beta_{2}\right| & \lesssim\left(n\epsilon^{2}\left\Vert \bm{X}^{\star}\right\Vert _{2,\infty}^{2}+\sqrt{n}\epsilon\left\Vert \bm{X}^{\star}\right\Vert _{2,\infty}\left\Vert \bm{X}^{\star}\right\Vert \right)\left\Vert \bm{V}\right\Vert _{\mathrm{F}}^{2}\\
 & \overset{\left(\text{i}\right)}{\leq}\left(n\epsilon^{2}\frac{\kappa\mu r}{n}+\sqrt{n}\epsilon\sqrt{\frac{\kappa\mu r}{n}}\right)\sigma_{\max}\left\Vert \bm{V}\right\Vert _{\mathrm{F}}^{2}\overset{\left(\text{ii}\right)}{\leq}\frac{1}{10}\sigma_{\min}\left\Vert \bm{V}\right\Vert _{\mathrm{F}}^{2},
\end{align*}
where (i) utilizes the incoherence condition (\ref{eq:incoherence-X-MC})
and (ii) holds with the proviso that $\epsilon\sqrt{\kappa^{3}\mu r}\ll1$.
\item To bound $\alpha_{2}$, apply the Cauchy-Schwarz inequality to get
\[
\left|\alpha_{2}\right|=\left|\left\langle \bm{V},\text{ }\frac{1}{p}\mathcal{P}_{\Omega}\left(\bm{X}\bm{X}^{\top}-\bm{M}^{\star}\right)\bm{V}\right\rangle \right|\leq\left\Vert \frac{1}{p}\mathcal{P}_{\Omega}\left(\bm{X}\bm{X}^{\top}-\bm{M}^{\star}\right)\right\Vert \left\Vert \bm{V}\right\Vert _{\mathrm{F}}^{2}.
\]
In view of Lemma \ref{lemma:energy-MC}, with probability at least $1-O\left(n^{-10}\right)$, 
\begin{align*}
\left\Vert \frac{1}{p}\mathcal{P}_{\Omega}\left(\bm{X}\bm{X}^{\top}-\bm{M}^{\star}\right)\right\Vert  & \leq2n\epsilon^{2}\left\Vert \bm{X}^{\star}\right\Vert _{2,\infty}^{2}+4\epsilon\sqrt{n}\log n\left\Vert \bm{X}^{\star}\right\Vert _{2,\infty}\left\Vert \bm{X}^{\star}\right\Vert \\
 & \leq\left(2n\epsilon^{2}\frac{\kappa\mu r}{n}+4\epsilon\sqrt{n}\log n\sqrt{\frac{\kappa\mu r}{n}}\right)\sigma_{\max} \leq\frac{1}{10}\sigma_{\min}
\end{align*}
as soon as $\epsilon\sqrt{\kappa^{3}\mu r}\log n\ll1$, where we utilize the incoherence condition (\ref{eq:incoherence-X-MC}). This in turn
implies that 
\[
\left|\alpha_{2}\right|\leq\frac{1}{10}\sigma_{\min}\left\Vert \bm{V}\right\Vert _{\mathrm{F}}^{2}.
\]
Notably, this bound holds uniformly over all $\bm{X}$ satisfying
the condition in Lemma \ref{lemma:hessian-MC}, regardless of the
statistical dependence between $\bm{X}$ and the sampling set $\Omega$. 
\item The last term $\alpha_{3}$ can also be controlled using the injectivity
of $\mathcal{P}_{\Omega}$ when restricted to the tangent space of
$\bm{M}^{\star}$. Specifically, it follows from the bounds in
\cite[Section 4.2]{ExactMC09} or Lemma \ref{lemma:injectivity} that
\[
\left|\alpha_{3}\right|\leq\gamma\left\Vert \bm{V}\bm{X}^{\star\top}+\bm{X}^{\star}\bm{V}^{\top}\right\Vert _{\mathrm{F}}^{2}\leq4\gamma\sigma_{\max}\left\Vert \bm{V}\right\Vert _{\mathrm{F}}^{2}\leq\frac{1}{10}\sigma_{\min}\left\Vert \bm{V}\right\Vert _{\mathrm{F}}^{2}
\]
for any $\gamma>0$ such that  $\kappa \gamma$ is a small constant, as soon as $n^{2}p\gg \kappa^2\mu  rn\log n$. 
\item Taking all the preceding bounds collectively yields 
\begin{align*}
 \mathrm{vec}\left(\bm{V}\right)^{\top}\nabla^{2}f_{\mathrm{clean}}\left(\bm{X}\right)\mathrm{vec}\left(\bm{V}\right)&\geq\alpha_{4}-\left|\alpha_{1}\right|-\left|\alpha_{2}\right|-\left|\alpha_{3}\right|\\
 & \geq\left(\frac{9}{10}-\frac{3}{10}\right)\sigma_{\min}\left\Vert \bm{V}\right\Vert _{\mathrm{F}}^{2}\geq\frac{1}{2}\sigma_{\min}\left\Vert \bm{V}\right\Vert _{\mathrm{F}}^{2}
\end{align*}
for all $\bm{V}$ satisfying our assumptions, and 
\begin{align*}
 \left|\mathrm{vec}\left(\bm{V}\right)^{\top}\nabla^{2}f_{\mathrm{clean}}\left(\bm{X}\right)\mathrm{vec}\left(\bm{V}\right)\right|&\leq\alpha_{4}+\left|\alpha_{1}\right|+\left|\alpha_{2}\right|+\left|\alpha_{3}\right|\\
 & \leq\left(2\sigma_{\max}+\frac{3}{10}\sigma_{\min}\right)\left\Vert \bm{V}\right\Vert _{\mathrm{F}}^{2}\leq\frac{5}{2}\sigma_{\max}\left\Vert \bm{V}\right\Vert _{\mathrm{F}}^{2}
\end{align*}
for all $\bm{V}$. Since this upper bound holds uniformly over all
$\bm{V}$, we conclude that 
\[
\left\Vert \nabla^{2}f_{\mathrm{clean}}\left(\bm{X}\right)\right\Vert \leq\frac{5}{2}\sigma_{\max}
\]
as claimed. 
\end{enumerate}

\subsection{Proof of Lemma~\ref{lemma:induction-fro-MC}}

\label{sec:proof-induction-fro-MC}

Given that $\hat{\bm{H}}^{t+1}$ is chosen to minimize the error in
terms of the Frobenius norm (cf.~(\ref{eq:rotation-hat-h-MC})), we
have 
\begin{align}
 & \left\Vert \bm{X}^{t+1}\hat{\bm{H}}^{t+1}-\bm{X}^{\star}\right\Vert _{\mathrm{F}}\leq\left\Vert \bm{X}^{t+1}\hat{\bm{H}}^{t}-\bm{X}^{\star}\right\Vert _{\mathrm{F}}=\left\Vert \left[\bm{X}^{t}-\eta\nabla f\left(\bm{X}^{t}\right)\right]\hat{\bm{H}}^{t}-\bm{X}^{\star}\right\Vert _{\mathrm{F}}\nonumber \\
 & \quad\overset{\text{(i)}}{=}\left\Vert \bm{X}^{t}\hat{\bm{H}}^{t}-\eta\nabla f\big(\bm{X}^{t}\hat{\bm{H}}^{t}\big)-\bm{X}^{\star}\right\Vert _{\mathrm{F}}\nonumber\\
 & \quad\overset{\text{(ii)}}{=}\left\Vert \bm{X}^{t}\hat{\bm{H}}^{t}-\eta\left[\nabla f_{\mathrm{clean}}\big(\bm{X}^{t}\hat{\bm{H}}^{t}\big)-\frac{1}{p}\cP_{\Omega}\left(\bm{E}\right)\bm{X}^{t}\hat{\bm{H}}^{t}\right]-\bm{X}^{\star}\right\Vert_{\mathrm{F}} \nonumber \\
 & \quad\leq\underbrace{\left\Vert \bm{X}^{t}\hat{\bm{H}}^{t}-\eta\nabla f_{\mathrm{clean}}\big(\bm{X}^{t}\hat{\bm{H}}^{t}\big)-\left(\bm{X}^{\star}-\eta\nabla f_{\mathrm{clean}}\big(\bm{X}^{\star}\big)\right)\right\Vert _{\mathrm{F}}}_{:=\alpha_{1}}+\underbrace{\eta\left\Vert \frac{1}{p}\cP_{\Omega}\left(\bm{E}\right)\bm{X}^{t}\hat{\bm{H}}^{t}\right\Vert _{\mathrm{F}}}_{:=\alpha_{2}},\label{eq:defn-I1-I2-MC}
\end{align}
where (i) follows from the identity $\nabla f(\bm{X}^{t}\bm{R})=\nabla f\left(\bm{X}^{t}\right)\bm{R}$
for any orthonormal matrix $\bm{R}\in\cO^{r\times r}$, (ii) arises from the definitions of $\nabla f\left(\bm{X}\right)$ and $\nabla f_{\mathrm{clean}}\left(\bm{X}\right)$ (see~(\ref{eq:gradient-MC-formula}) and (\ref{eq:gradient-MC-formula-clean}), respectively), and the last
inequality (\ref{eq:defn-I1-I2-MC}) utilizes the triangle inequality and the fact that $\nabla f_{\mathrm{clean}}(\bm{X}^{\star})=\bm{0}$.
It thus suffices to control $\alpha_{1}$ and $\alpha_{2}$. 
\begin{enumerate}
\item For the second term $\alpha_{2}$ in (\ref{eq:defn-I1-I2-MC}), it is easy to see
that with probability at least $1-O\left(n^{-10}\right)$,
\[
\alpha_{2}\leq\eta\left\Vert \frac{1}{p}\cP_{\Omega}\left(\bm{E}\right)\right\Vert \left\Vert \bm{X}^{t}\hat{\bm{H}}^{t}\right\Vert _{\mathrm{F}}\leq2\eta\left\Vert \frac{1}{p}\cP_{\Omega}\left(\bm{E}\right)\right\Vert \left\Vert \bm{X}^{\star}\right\Vert _{\mathrm{F}}\leq2\eta C\sigma\sqrt{\frac{n}{p}}\|\bm{X}^{\star}\|_{\mathrm{F}}
\]
for some absolute constant $C > 0$. 
Here, the second inequality holds because $\big\Vert \bm{X}^{t}\hat{\bm{H}}^{t}\big\Vert _{\mathrm{F}} \leq \big\Vert \bm{X}^{t}\hat{\bm{H}}^{t}-\bm{X}^{\star}\big\Vert _{\mathrm{F}}+\left\Vert \bm{X}^{\star}\right\Vert _{\mathrm{F}} \leq 2\left\Vert \bm{X}^{\star}\right\Vert _{\mathrm{F}}$,
following the hypothesis (\ref{eq:induction_original_ell_2-MC_thm})
together with our assumptions on the noise and the sample complexity.
The last inequality makes use of Lemma \ref{lemma:noise-spectral}. 
\item For the first term $\alpha_{1}$ in (\ref{eq:defn-I1-I2-MC}), the fundamental theorem of calculus \cite[Chapter XIII, Theorem 4.2]{lang1993real} reveals
\begin{align}
 & \mathrm{vec}\left[\bm{X}^{t}\hat{\bm{H}}^{t}-\eta\nabla f_{\mathrm{clean}}\big(\bm{X}^{t}\hat{\bm{H}}^{t}\big)-\left(\bm{X}^{\star}-\eta\nabla f_{\mathrm{clean}}\big(\bm{X}^{\star}\big)\right)\right]\nonumber \\
 & \quad=\mathrm{vec}\left[\bm{X}^{t}\hat{\bm{H}}^{t}-\bm{X}^{\star}\right]-\eta\cdot \mathrm{vec}\left[\nabla f_{\mathrm{clean}}\big(\bm{X}^{t}\hat{\bm{H}}^{t}\big)-\nabla f_{\mathrm{clean}}\left(\bm{X}^{\star}\right)\right]\nonumber \\
 & \quad=\Bigg(\bm{I}_{nr}-\eta\underset{:=\bm{A}}{\underbrace{\int_{0}^{1}\nabla^{2}f_{\mathrm{clean}}\left(\bm{X}(\tau)\right)\mathrm{d}\tau}}\Bigg)\mathrm{vec}\left(\bm{X}^{t}\hat{\bm{H}}^{t}-\bm{X}^{\star}\right),\label{eq:gradient-MVT-MC}
\end{align}
where we denote $\bm{X}(\tau):=\bm{X}^{\star}+\tau(\bm{X}^{t}\hat{\bm{H}}^{t}-\bm{X}^{\star})$.
Taking the squared Euclidean norm of both sides of the equality (\ref{eq:gradient-MVT-MC})
leads to 
\begin{align}
\left(\alpha_{1}\right)^{2} & =\mathrm{vec}\big(\bm{X}^{t}\hat{\bm{H}}^{t}-\bm{X}^{\star}\big)^{\top}\left(\bm{I}_{nr}-\eta\bm{A}\right)^{2}\mathrm{vec}\big(\bm{X}^{t}\hat{\bm{H}}^{t}-\bm{X}^{\star}\big)\nonumber \\
 & =\mathrm{vec}\big(\bm{X}^{t}\hat{\bm{H}}^{t}-\bm{X}^{\star}\big)^{\top}\left(\bm{I}_{nr}-2\eta\bm{A}+\eta^{2}\bm{A}^{2}\right)\mathrm{vec}\big(\bm{X}^{t}\hat{\bm{H}}^{t}-\bm{X}^{\star}\big)\nonumber \\
 & \leq\left\Vert \bm{X}^{t}\hat{\bm{H}}^{t}-\bm{X}^{\star}\right\Vert _{\mathrm{F}}^{2}+\eta^{2}\left\Vert \bm{A}\right\Vert ^{2}\left\Vert \bm{X}^{t}\hat{\bm{H}}^{t}-\bm{X}^{\star}\right\Vert _{\mathrm{F}}^{2} -2\eta\;\mathrm{vec}\big(\bm{X}^{t}\hat{\bm{H}}^{t}-\bm{X}^{\star}\big)^{\top}\bm{A}\text{ }\mathrm{vec}\big(\bm{X}^{t}\hat{\bm{H}}^{t}-\bm{X}^{\star}\big),\label{eq:v^A^2v-MC}
\end{align}
where in (\ref{eq:v^A^2v-MC}) we have used the fact that 
\[
\mathrm{vec}\big(\bm{X}^{t}\hat{\bm{H}}^{t}-\bm{X}^{\star}\big)^{\top}\bm{A}^{2}\mathrm{vec}\big(\bm{X}^{t}\hat{\bm{H}}^{t}-\bm{X}^{\star}\big)\leq\left\Vert \bm{A}\right\Vert ^{2}\left\Vert \mathrm{vec}\big(\bm{X}^{t}\hat{\bm{H}}^{t}-\bm{X}^{\star}\big)\right\Vert _{2}^{2}=\left\Vert \bm{A}\right\Vert ^{2}\left\Vert \bm{X}^{t}\hat{\bm{H}}^{t}-\bm{X}^{\star}\right\Vert _{\mathrm{F}}^{2}.
\]
Based on the condition (\ref{eq:induction_original_ell_infty-MC_thm}), it is easily
seen that $\forall\tau\in[0,1]$,
\begin{align*}
\left\Vert \bm{X}\left(\tau\right)-\bm{X}^{\star}\right\Vert _{2,\infty}&\leq\left(C_{5}\mu r\sqrt{\frac{\log n}{np}}+\frac{C_{8}}{\sigma_{\min}}\sigma\sqrt{\frac{n\log n}{p}}\right)\left\Vert \bm{X}^{\star}\right\Vert _{2,\infty}.
\end{align*}
Taking $\bm{X}=\bm{X}\left(\tau\right),\bm{Y}=\bm{X}^{t}$
and $\bm{Z}=\bm{X}^{\star}$ in Lemma~\ref{lemma:hessian-MC}, one can easily verify the assumptions therein given our sample size condition $n^2 p \gg \kappa^{3} \mu^{3} r^{3} n \log^{3} n$ and the noise condition (\ref{eq:mc-noise-condition}).  As a result,  
\[
\mathrm{vec}\big(\bm{X}^{t}\hat{\bm{H}}^{t}-\bm{X}^{\star}\big)^{\top}\bm{A}\text{ }\mathrm{vec}\big(\bm{X}^{t}\hat{\bm{H}}^{t}-\bm{X}^{\star}\big)\geq\frac{\sigma_{\min}}{2}\big\|\bm{X}^{t}\hat{\bm{H}}^{t}-\bm{X}^{\star}\big\|_{\mathrm{F}}^{2}
\qquad\text{and}\qquad\|\bm{A}\| \leq\frac{5}{2}\sigma_{\max}.
\]
Substituting these two inequalities into (\ref{eq:v^A^2v-MC}) yields
\begin{align*}
\left(\alpha_{1}\right)^{2} \leq\left(1+\frac{25}{4}\eta^{2}\sigma_{\max}^{2}-\sigma_{\min}\eta\right)\big\|\bm{X}^{t}\hat{\bm{H}}^{t}-\bm{X}^{\star}\big\|_{\mathrm{F}}^{2}\leq\left(1-\frac{\sigma_{\min}}{2}\eta\right)\big\|\bm{X}^{t}\hat{\bm{H}}^{t}-\bm{X}^{\star}\big\|_{\mathrm{F}}^{2}
\end{align*}
as long as $0<\eta\leq({2\sigma_{\min}}) / ({25\sigma_{\max}^{2}})$,
which further implies that 
\begin{align*}
\alpha_{1} & \leq\left(1-\frac{\sigma_{\min}}{4}\eta\right)\big\|\bm{X}^{t}\hat{\bm{H}}^{t}-\bm{X}^{\star}\big\|_{\mathrm{F}}.
\end{align*}
\item Combining the preceding bounds on both $\alpha_{1}$ and $\alpha_{2}$ and making
use of the hypothesis (\ref{eq:induction_original_ell_2-MC_thm}), we
have
\begin{align*}
 & \left\Vert \bm{X}^{t+1}\hat{\bm{H}}^{t+1}-\bm{X}^{\star}\right\Vert _{\mathrm{F}}\leq\left(1-\frac{\sigma_{\min}}{4}\eta\right)\left\Vert \bm{X}^{t}\hat{\bm{H}}^{t}-\bm{X}^{\star}\right\Vert _{\mathrm{F}}+2\eta C\sigma\sqrt{\frac{n}{p}}\left\Vert \bm{X}^{\star}\right\Vert _{\mathrm{F}}\\
 & \quad\leq\left(1-\frac{\sigma_{\min}}{4}\eta\right)\left(C_{4}\rho^{t}\mu r\frac{1}{\sqrt{np}}\left\Vert \bm{X}^{\star}\right\Vert _{\mathrm{F}}+C_{1}\frac{\sigma}{\sigma_{\min}}\sqrt{\frac{n}{p}}\left\Vert \bm{X}^{\star}\right\Vert _{\mathrm{F}}\right)+2\eta C\sigma\sqrt{\frac{n}{p}}\left\Vert \bm{X}^{\star}\right\Vert _{\mathrm{F}}\\
 & \quad\leq\left(1-\frac{\sigma_{\min}}{4}\eta\right)C_{4}\rho^{t}\mu r\frac{1}{\sqrt{np}}\left\Vert \bm{X}^{\star}\right\Vert _{\mathrm{F}}+\left[\left(1-\frac{\sigma_{\min}}{4}\eta\right)\frac{C_{1}}{\sigma_{\min}}+2\eta C\right]\sigma\sqrt{\frac{n}{p}}\left\Vert \bm{X}^{\star}\right\Vert _{\mathrm{F}}\\
 & \quad\leq C_{4}\rho^{t+1}\mu r\frac{1}{\sqrt{np}}\left\Vert \bm{X}^{\star}\right\Vert _{\mathrm{F}}+C_{1}\frac{\sigma}{\sigma_{\min}}\sqrt{\frac{n}{p}}\left\Vert \bm{X}^{\star}\right\Vert _{\mathrm{F}}
\end{align*}
as long as $0<\eta\leq({2\sigma_{\min}}) / ({25\sigma_{\max}^{2}})$, $1-\left({\sigma_{\min}} / {4}\right)\cdot\eta \leq \rho <1$
and $C_{1}$ is sufficiently large. This completes the proof of the contraction with respect to the Frobenius norm. 
\end{enumerate}

\subsection{Proof of Lemma \ref{lemma:operator-norm-contraction-MC}\label{subsec:Proof-of-Lemma-operator-contraction-MC}}

To facilitate  analysis, we construct an auxiliary matrix defined
as follows 
\begin{equation}
\tilde{\bm{X}}^{t+1}:=\bm{X}^{t}\hat{\bm{H}}^{t}-\eta\frac{1}{p}\cP_{\Omega}\left[\bm{X}^{t}\bm{X}^{t\top}-\left(\bm{M}^{\star}+\bm{E}\right)\right]\bm{X}^{\star}.\label{eq:op-aux}
\end{equation}
With this auxiliary matrix in place, we invoke the triangle inequality
to bound 
\begin{align}
\label{eq:defn-alpha1-alpha2-1}
\big\|\bm{X}^{t+1}\hat{\bm{H}}^{t+1}-\bm{X}^{\star}\big\| & \leq\underbrace{\big\|\bm{X}^{t+1}\hat{\bm{H}}^{t+1}-\tilde{\bm{X}}^{t+1}\big\|}_{:=\alpha_{1}}+\underbrace{\big\|\tilde{\bm{X}}^{t+1}-\bm{X}^{\star}\big\|}_{:=\alpha_{2}}.
\end{align}
\begin{enumerate}
\item We start with the second term $\alpha_{2}$ and show that the auxiliary
matrix $\tilde{\bm{X}}^{t+1}$ is also not far from the truth. The
definition of $\tilde{\bm{X}}^{t+1}$ allows one to express 
\begin{align}
\alpha_{2} & =\left\Vert \bm{X}^{t}\hat{\bm{H}}^{t}-\eta\frac{1}{p}\cP_{\Omega}\left[\bm{X}^{t}\bm{X}^{t\top}-\left(\bm{M}^{\star}+\bm{E}\right)\right]\bm{X}^{\star}-\bm{X}^{\star}\right\Vert \nonumber \\
 & \leq\eta\left\Vert \frac{1}{p}\cP_{\Omega}\left(\bm{E}\right)\right\Vert \left\Vert \bm{X}^{\star}\right\Vert +\left\Vert \bm{X}^{t}\hat{\bm{H}}^{t}-\eta\frac{1}{p}\cP_{\Omega}\left(\bm{X}^{t}\bm{X}^{t\top}-\bm{X}^{\star}\bm{X}^{\star\top}\right)\bm{X}^{\star}-\bm{X}^{\star}\right\Vert \\
 & \leq\eta\left\Vert \frac{1}{p}\cP_{\Omega}\left(\bm{E}\right)\right\Vert \left\Vert \bm{X}^{\star}\right\Vert +\underbrace{\left\Vert \bm{X}^{t}\hat{\bm{H}}^{t}-\eta\left(\bm{X}^{t}\bm{X}^{t\top}-\bm{X}^{\star}\bm{X}^{\star\top}\right)\bm{X}^{\star}-\bm{X}^{\star}\right\Vert }_{:=\beta_{1}}\nonumber \\
 & \quad+\underbrace{\eta\left\Vert \frac{1}{p}\cP_{\Omega}\left(\bm{X}^{t}\bm{X}^{t\top}-\bm{X}^{\star}\bm{X}^{\star\top}\right)\bm{X}^{\star}-\left(\bm{X}^{t}\bm{X}^{t\top}-\bm{X}^{\star}\bm{X}^{\star\top}\right)\bm{X}^{\star}\right\Vert }_{:=\beta_{2}},\label{eq:alpha-2-beta-MC}
\end{align}
where we have used the triangle inequality to separate the population-level
component (i.e.~$\beta_{1}$), the perturbation (i.e.~$\beta_{2}$),
and the noise component. In what follows, we will denote 
\[
\bm{\Delta}^{t}:=\bm{X}^{t}\hat{\bm{H}}^{t}-\bm{X}^{\star}
\]
which, by Lemma \ref{lemma:opp}, satisfies the following symmetry
property 
\begin{equation}
\hat{\bm{H}}^{t\top}\bm{X}^{t\top}\bm{X}^{\star}=\bm{X}^{\star\top}\bm{X}^{t}\hat{\bm{H}}^{t}\qquad\Longrightarrow\qquad\bm{\Delta}^{t\top}\bm{X}^{\star}=\bm{X}^{\star\top}\bm{\Delta}^{t}.\label{eq:symmetry-Deltat}
\end{equation}

\begin{enumerate}
\item The population-level component $\beta_1$ is easier to control. Specifically,
we first simplify its expression as 
\begin{align*}
\beta_{1} & =\left\Vert \bm{\Delta}^{t}-\eta\left(\bm{\Delta}^{t}\bm{\Delta}^{t\top}+\bm{\Delta}^{t}\bm{X}^{\star\top}+\bm{X}^{\star}\bm{\Delta}^{t\top}\right)\bm{X}^{\star}\right\Vert \\
 & \leq\underbrace{\left\Vert \bm{\Delta}^{t}-\eta\left(\bm{\Delta}^{t}\bm{X}^{\star\top}+\bm{X}^{\star}\bm{\Delta}^{t\top}\right)\bm{X}^{\star}\right\Vert }_{:=\gamma_{1}}+\underbrace{\eta\left\Vert \bm{\Delta}^{t}\bm{\Delta}^{t\top}\bm{X}^{\star}\right\Vert }_{:=\gamma_{2}}.
\end{align*}
The leading term $\gamma_{1}$ can be upper bounded by 
\begin{align*}
\gamma_{1} & =\left\Vert \bm{\Delta}^{t}-\eta\bm{\Delta}^{t}\bm{\Sigma}^{\star}-\eta\bm{X}^{\star}\bm{\Delta}^{t\top}\bm{X}^{\star}\right\Vert \nonumber =\left\Vert \bm{\Delta}^{t}-\eta\bm{\Delta}^{t}\bm{\Sigma}^{\star}-\eta\bm{X}^{\star}\bm{X}^{\star\top}\bm{\Delta}^{t}\right\Vert\\
 & =\left\Vert \frac{1}{2}\bm{\Delta}^{t}\left(\bm{I}_{r}-2\eta\bm{\Sigma}^{\star}\right)+\frac{1}{2}\left(\bm{I}_{n}-2\eta\bm{M}^{\star}\right)\bm{\Delta}^{t}\right\Vert \leq 
\frac{1}{2} \left(
  \left\Vert \bm{I}_{r}-2\eta\bm{\Sigma}^{\star}\right\Vert
 +
\left\Vert \bm{I}_{n}-2\eta\bm{M}^{\star}\right\Vert
 \right)
 \left\Vert \bm{\Delta}^{t}\right\Vert
\end{align*}
where the second identity follows from the symmetry property
(\ref{eq:symmetry-Deltat}). By choosing $\eta\leq {1} / ({2\sigma_{\max}})$, one has $\bm{0} \preceq \bm{I}_{r}-2\eta\bm{\Sigma}^{\star} \preceq \left(1-2\eta\sigma_{\min}\right)\bm{I}_{r}$ and $\bm{0} \preceq \bm{I}_{n}-2\eta\bm{M}^{\star} \preceq \bm{I}_{n}$, and further 
one can ensure
\begin{align}
\gamma_{1} & \leq 
\frac{1}{2}\left[ \left(1-2\eta\sigma_{\min}\right) + 1\right]\left\Vert \bm{\Delta}^{t}\right\Vert
=
\left(1-\eta\sigma_{\min}\right)\left\Vert \bm{\Delta}^{t}\right\Vert .\label{eq:op-gamma-1}
\end{align}
Next, regarding the higher order term $\gamma_{2}$, we can easily
obtain 
\begin{align}
\gamma_{2} & \leq\eta\left\Vert \bm{\Delta}^{t}\right\Vert ^{2}\left\Vert \bm{X}^{\star}\right\Vert .\label{eq:op-gamma-2}
\end{align}
The bounds (\ref{eq:op-gamma-1}) and (\ref{eq:op-gamma-2}) taken
collectively give 
\begin{equation}
\beta_{1}\leq\left(1-\eta\sigma_{\min}\right)\left\Vert \bm{\Delta}^{t}\right\Vert +\eta\left\Vert \bm{\Delta}^{t}\right\Vert ^{2}\left\Vert \bm{X}^{\star}\right\Vert .\label{eq:op-beta-1}
\end{equation}
\item We now turn to the perturbation part $\beta_{2}$ by showing that
\begin{align}
\frac{1}{\eta}\beta_{2} & =\left\Vert \frac{1}{p}\cP_{\Omega}\left(\bm{\Delta}^{t}\bm{\Delta}^{t\top}+\bm{\Delta}^{t}\bm{X}^{\star\top}+\bm{X}^{\star}\bm{\Delta}^{t\top}\right)\bm{X}^{\star}-\left[\bm{\Delta}^{t}\bm{\Delta}^{t\top}+\bm{\Delta}^{t}\bm{X}^{\star\top}+\bm{X}^{\star}\bm{\Delta}^{t\top}\right]\bm{X}^{\star}\right\Vert \nonumber \\
 & \leq\underbrace{\left\Vert \frac{1}{p}\cP_{\Omega}\left(\bm{\Delta}^{t}\bm{X}^{\star\top}\right)\bm{X}^{\star}-\left(\bm{\Delta}^{t}\bm{X}^{\star\top}\right)\bm{X}^{\star}\right\Vert _{\mathrm{F}}}_{:=\theta_{1}}+\underbrace{\left\Vert \frac{1}{p}\cP_{\Omega}\left(\bm{X}^{\star}\bm{\Delta}^{t\top}\right)\bm{X}^{\star}-\left(\bm{X}^{\star}\bm{\Delta}^{t\top}\right)\bm{X}^{\star}\right\Vert _{\mathrm{F}}}_{:=\theta_{2}}\nonumber \\
 & \quad+\underbrace{\left\Vert \frac{1}{p}\cP_{\Omega}\left(\bm{\Delta}^{t}\bm{\Delta}^{t\top}\right)\bm{X}^{\star}-\left(\bm{\Delta}^{t}\bm{\Delta}^{t\top}\right)\bm{X}^{\star}\right\Vert _{\mathrm{F}}}_{:=\theta_{3}},\label{eq:beta2-theta123}
\end{align}
where the last inequality holds due to the triangle inequality as
well as the fact that $\|\bm{A}\|\leq\|\bm{A}\|_{\mathrm{F}}$. In
the sequel, we shall bound the three terms separately. 
\begin{itemize}
\item For the first term $\theta_{1}$ in \eqref{eq:beta2-theta123}, the $l$th row of $\frac{1}{p}\cP_{\Omega}\left(\bm{\Delta}^{t}\bm{X}^{\star\top}\right)\bm{X}^{\star}-\left(\bm{\Delta}^{t}\bm{X}^{\star\top}\right)\bm{X}^{\star}$
is given by 
\begin{align*}
\frac{1}{p}\sum_{j=1}^{n}\left(\delta_{l,j}-p\right)\bm{\Delta}_{l,\cdot}^{t}\bm{X}_{j,\cdot}^{\star\top}\bm{X}_{j,\cdot}^{\star} & =\bm{\Delta}_{l,\cdot}^{t}\left[\frac{1}{p}\sum_{j=1}^{n}\left(\delta_{l,j}-p\right)\bm{X}_{j,\cdot}^{\star\top}\bm{X}_{j,\cdot}^{\star}\right]
\end{align*}
where, as usual, $\delta_{l,j}=\ind_{\left\{ (l,j)\in\Omega\right\} }$.
Lemma \ref{lemma:bernoulli-spectral-norm} together with the union
bound reveals that 
\begin{align*}
\left\Vert \frac{1}{p}\sum_{j=1}^{n}\left(\delta_{l,j}-p\right)\bm{X}_{j,\cdot}^{\star\top}\bm{X}_{j,\cdot}^{\star}\right\Vert  & \lesssim\frac{1}{p}\left(\sqrt{p\left\Vert \bm{X}^{\star}\right\Vert _{2,\infty}^{2}\left\Vert \bm{X}^{\star}\right\Vert ^{2}\log n}+\left\Vert \bm{X}^{\star}\right\Vert _{2,\infty}^{2}\log n\right)\\
 & \asymp\sqrt{\frac{\|\bm{X}^{\star}\|_{2,\infty}^{2}\sigma_{\max}\log n}{p}}+\frac{\|\bm{X}^{\star}\|_{2,\infty}^{2}\log n}{p}
\end{align*}
for all $1\leq l\leq n$ with high probability. This gives 
\begin{align*}
\left\Vert \bm{\Delta}_{l,\cdot}^{t}\left[\frac{1}{p}\sum_{j=1}^{n}\left(\delta_{l,j}-p\right)\bm{X}_{j,\cdot}^{\star\top}\bm{X}_{j,\cdot}^{\star}\right]\right\Vert _{2} & \leq\left\Vert \bm{\Delta}_{l,\cdot}^{t}\right\Vert _{2}\left\Vert \frac{1}{p}\sum_{j}\left(\delta_{l,j}-p\right)\bm{X}_{j,\cdot}^{\star\top}\bm{X}_{j,\cdot}^{\star}\right\Vert \\
 & \lesssim\left\Vert \bm{\Delta}_{l,\cdot}^{t}\right\Vert _{2}\left\{ \sqrt{\frac{\|\bm{X}^{\star}\|_{2,\infty}^{2}\sigma_{\max}\log n}{p}}+\frac{\|\bm{X}^{\star}\|_{2,\infty}^{2}\log n}{p}\right\} ,
\end{align*}
which further reveals that 
\begin{align*}
\theta_{1}  =\sqrt{\sum_{l=1}^{n}\left\Vert \frac{1}{p}\sum_{j}\left(\delta_{l,j}-p\right)\bm{\Delta}_{l,\cdot}^{t}\bm{X}_{j,\cdot}^{\star\top}\bm{X}_{j,\cdot}^{\star}\right\Vert _{2}^{2}}  & \lesssim\left\Vert \bm{\Delta}^{t}\right\Vert _{\mathrm{F}}\left\{ \sqrt{\frac{\|\bm{X}^{\star}\|_{2,\infty}^{2}\sigma_{\max}\log n}{p}}+\frac{\|\bm{X}^{\star}\|_{2,\infty}^{2}\log n}{p}\right\} \\
 & \overset{\left(\text{i}\right)}{\lesssim}\left\Vert \bm{\Delta}^{t}\right\Vert \left\{ \sqrt{\frac{\|\bm{X}^{\star}\|_{2,\infty}^{2}r\sigma_{\max}\log n}{p}}+\frac{\sqrt{r}\|\bm{X}^{\star}\|_{2,\infty}^{2}\log n}{p}\right\} \\
 & \overset{\left(\text{ii}\right)}{\lesssim}\left\Vert \bm{\Delta}^{t}\right\Vert \left\{ \sqrt{\frac{\kappa\mu r^{2}\log n}{np}}+\frac{\kappa\mu r^{3/2}\log n}{np}\right\} \sigma_{\max}\\
 & \overset{\left(\text{iii}\right)}{\leq}\gamma\sigma_{\min}\left\Vert \bm{\Delta}^{t}\right\Vert ,
\end{align*}
for arbitrarily small $\gamma>0$. Here, (i) follows from $\left\Vert \bm{\Delta}^{t}\right\Vert _{\mathrm{F}}\leq\sqrt{r}\left\Vert \bm{\Delta}^{t}\right\Vert $,
(ii) holds owing to the incoherence condition (\ref{eq:incoherence-X-MC}),
and (iii) follows as long as $n^{2}p\gg\kappa^{3}\mu r^{2} n \log n$.

\item For the second term $\theta_{2}$ in (\ref{eq:beta2-theta123}), denote
\[
\bm{A}=\cP_{\Omega}\left(\bm{X}^{\star}\bm{\Delta}^{t\top}\right)\bm{X}^{\star}-p\left(\bm{X}^{\star}\bm{\Delta}^{t\top}\right)\bm{X}^{\star},
\]
whose $l$th row is given by 
\begin{align}
\bm{A}_{l,\cdot}=\bm{X}_{l,\cdot}^{\star}\sum_{j=1}^{n}\left(\delta_{l,j}-p\right)\bm{\Delta}_{j,\cdot}^{t\top}\bm{X}_{j,\cdot}^{\star}\label{eq:lemma-22-A-l}.
\end{align}
Recalling the induction hypotheses (\ref{eq:induction_original_ell_infty-MC_thm})
and (\ref{eq:induction_original_operator-MC_thm}), we define 
\begin{align}
\left\Vert \bm{\Delta}^{t}\right\Vert _{2,\infty} & \leq C_{5}\rho^{t}\mu r\sqrt{\frac{\log n}{np}}\left\Vert \bm{X}^{\star}\right\Vert _{2,\infty}+C_{8}\frac{\sigma}{\sigma_{\min}}\sqrt{\frac{n\log n}{p}}\left\Vert \bm{X}^{\star}\right\Vert _{2,\infty}:=\xi\label{eq:defn-xi-Deltat}\\
\left\Vert \bm{\Delta}^{t}\right\Vert  & \leq C_{9}\rho^{t}\mu r\frac{1}{\sqrt{np}}\left\Vert \bm{X}^{\star}\right\Vert +C_{10}\frac{\sigma}{\sigma_{\min}}\sqrt{\frac{n}{p}}\left\Vert \bm{X}^{\star}\right\Vert :=\psi.\label{eq:defn-psi-Deltat}
\end{align}
With these two definitions in place, we now introduce a ``truncation level'' 
\begin{equation}
\omega:=2p\xi\sigma_{\max}\label{eq:op-truncation-level}
\end{equation}
that allows us to bound $\theta_{2}$ in terms of the following two
terms 
\[
\theta_{2}=\frac{1}{p}\sqrt{\sum_{l=1}^{n}\left\Vert \bm{A}_{l,\cdot}\right\Vert _{2}^{2}}\leq\frac{1}{p}\underbrace{\sqrt{\sum_{l=1}^{n}\left\Vert \bm{A}_{l,\cdot}\right\Vert _{2}^{2}\ind_{\left\{ \left\Vert \bm{A}_{l,\cdot}\right\Vert _{2}\leq\omega\right\} }}}_{:=\phi_{1}}+\frac{1}{p}\underbrace{\sqrt{\sum_{l=1}^{n}\left\Vert \bm{A}_{l,\cdot}\right\Vert _{2}^{2}\ind_{\left\{ \left\Vert \bm{A}_{l,\cdot}\right\Vert _{2}\geq\omega\right\} }}}_{:=\phi_{2}}.
\]

We will apply different strategies when upper bounding the terms $\phi_{1}$
and $\phi_{2}$, with their bounds given in the following two lemmas under the induction hypotheses (\ref{eq:induction_original_ell_infty-MC_thm})
and (\ref{eq:induction_original_operator-MC_thm}). 

\begin{lemma}\label{lemma:phi_1-MC}
	Under the conditions in Lemma \ref{lemma:operator-norm-contraction-MC},
there exist some constants $c,C>0$ such that with probability
exceeding $1-c\exp(-Cnr\log n)$, 
\begin{equation}
\phi_{1}\lesssim \xi\sqrt{p\sigma_{\max}\|\bm{X}^{\star}\|_{2,\infty}^{2}nr\log^{2}n}\label{eq:operator-phi-1}
\end{equation}
holds simultaneously for all $\bm{\Delta}^{t}$ obeying (\ref{eq:defn-xi-Deltat})
and (\ref{eq:defn-psi-Deltat}). Here, $\xi$ is defined in (\ref{eq:defn-xi-Deltat}).
\end{lemma} 

\begin{lemma}\label{lemma:phi-2-MC}
	Under the conditions in Lemma \ref{lemma:operator-norm-contraction-MC}, with probability at
least $1-O\left(n^{-10}\right)$, 
\begin{equation}
\phi_{2}\lesssim \xi\sqrt{\kappa\mu r^{2}p\log^{2}n}\left\Vert \bm{X}^{\star}\right\Vert ^{2}\label{eq:operator-phi-2}
\end{equation}
holds simultaneously for all $\bm{\Delta}^{t}$ obeying (\ref{eq:defn-xi-Deltat})
and (\ref{eq:defn-psi-Deltat}). Here, $\xi$ is defined in (\ref{eq:defn-xi-Deltat}).
\end{lemma}

The bounds (\ref{eq:operator-phi-1}) and (\ref{eq:operator-phi-2}) together with the incoherence condition (\ref{eq:incoherence-X-MC})
 yield 
\begin{align*}
\theta_{2} & \lesssim\frac{1}{p}\xi\sqrt{p\sigma_{\max}\|\bm{X}^{\star}\|_{2,\infty}^{2}nr\log^{2}n}+\frac{1}{p}\xi\sqrt{\kappa\mu r^{2}p\log^{2}n}\left\Vert \bm{X}^{\star}\right\Vert ^{2}\lesssim\sqrt{\frac{\kappa\mu r^{2}\log^{2}n}{p}}\xi\sigma_{\max}.
\end{align*}
\item Next, we assert that the third term $\theta_{3}$ in (\ref{eq:beta2-theta123})
has the same upper bound as $\theta_{2}$. The proof follows by repeating
the same argument used in bounding $\theta_{2}$, and is hence
omitted. 
\end{itemize}
Take the previous three bounds on $\theta_{1}$, $\theta_{2}$ and
$\theta_{3}$ together to arrive at 
\begin{align*}
\beta_{2}\leq \eta\left( |\theta_{1}|+|\theta_{2}|+|\theta_{3}| \right) & \leq\eta\gamma\sigma_{\min}\left\Vert \bm{\Delta}^{t}\right\Vert +\tilde{C}\eta\sqrt{\frac{\kappa\mu r^{2}\log^{2}n}{p}}\xi\sigma_{\max}
\end{align*}
for some constant $\tilde{C}>0$. 
\item Substituting the preceding bounds on $\beta_{1}$ and $\beta_{2}$ into
(\ref{eq:alpha-2-beta-MC}), we reach 
\begin{align}
\alpha_{2} & \overset{\left(\text{i}\right)}{\leq}\left(1-\eta\sigma_{\min}+\eta\gamma\sigma_{\min}+\eta\left\Vert \bm{\Delta}^{t}\right\Vert \left\Vert \bm{X}^{\star}\right\Vert \right)\left\Vert \bm{\Delta}^{t}\right\Vert +\eta\left\Vert \frac{1}{p}\cP_{\Omega}\left(\bm{E}\right)\right\Vert \left\Vert \bm{X}^{\star}\right\Vert \nonumber \\
 & \quad+\tilde{C}\eta\sqrt{\frac{\kappa\mu r^{2}\log^{2}n}{p}}\sigma_{\max}\left(C_{5}\rho^{t}\mu r\sqrt{\frac{\log n}{np}}\left\Vert \bm{X}^{\star}\right\Vert _{2,\infty}+C_{8}\frac{\sigma}{\sigma_{\min}}\sqrt{\frac{n\log n}{p}}\left\Vert \bm{X}^{\star}\right\Vert _{2,\infty}\right)\nonumber \\
 & \overset{\left(\text{ii}\right)}{\leq}\left(1-\frac{\sigma_{\min}}{2}\eta\right)\left\Vert \bm{\Delta}^{t}\right\Vert +\eta\left\Vert \frac{1}{p}\cP_{\Omega}\left(\bm{E}\right)\right\Vert \left\Vert \bm{X}^{\star}\right\Vert \nonumber \\
 & \quad+\tilde{C}\eta\sqrt{\frac{\kappa\mu r^{2}\log^{2}n}{p}}\sigma_{\max}\left(C_{5}\rho^{t}\mu r\sqrt{\frac{\log n}{np}}\left\Vert \bm{X}^{\star}\right\Vert _{2,\infty}+C_{8}\frac{\sigma}{\sigma_{\min}}\sqrt{\frac{n\log n}{p}}\left\Vert \bm{X}^{\star}\right\Vert _{2,\infty}\right)\nonumber \\
 & \overset{\left(\text{iii}\right)}{\leq}\left(1-\frac{\sigma_{\min}}{2}\eta\right)\left\Vert \bm{\Delta}^{t}\right\Vert +C\eta\sigma\sqrt{\frac{n}{p}}\left\Vert \bm{X}^{\star}\right\Vert \nonumber \\
 & \quad+\tilde{C}\eta\sqrt{\frac{\kappa^{2}\mu^{2}r^{3}\log^{3}n}{np}}\sigma_{\max}\left(C_{5}\rho^{t}\mu r\sqrt{\frac{1}{np}}+C_{8}\frac{\sigma}{\sigma_{\min}}\sqrt{\frac{n}{p}}\right)\left\Vert \bm{X}^{\star}\right\Vert\label{eq:op-alpha-2}
\end{align}
for some constant $C>0$. 
Here, (i) uses the definition of $\xi$ (cf.~(\ref{eq:defn-xi-Deltat})),
(ii) holds if $\gamma$ is small enough and $\left\Vert \bm{\Delta}^{t}\right\Vert \left\Vert \bm{X}^{\star}\right\Vert \ll\sigma_{\min}$,
and (iii) follows from Lemma \ref{lemma:noise-spectral} as well as
the incoherence condition (\ref{eq:incoherence-X-MC}). An immediate
consequence of (\ref{eq:op-alpha-2}) is that under the sample size
condition and the noise condition of this lemma, one has 
\begin{equation}
\big\Vert \tilde{\bm{X}}^{t+1}-\bm{X}^{\star} \big\Vert \left\Vert \bm{X}^{\star}\right\Vert \leq \sigma_{\min} /2 \label{eq:op-psd}
\end{equation}
if $0<\eta\leq1/\sigma_{\max}$. 
\end{enumerate}

\item We then move on to the first term $\alpha_{1}$ in \eqref{eq:defn-alpha1-alpha2-1}, which can be rewritten
as 
\[
\alpha_{1}=\big\|\bm{X}^{t+1}\hat{\bm{H}}^{t}\bm{R}_{1}-\tilde{\bm{X}}^{t+1} \big\|,
\]
with 
\begin{align}
\bm{R}_{1} & =\big(\hat{\bm{H}}^{t}\big)^{-1}\hat{\bm{H}}^{t+1}:=\arg\min_{\bm{R}\in\cO^{r\times r}}\big\|\bm{X}^{t+1}\hat{\bm{H}}^{t}\bm{R}-\bm{X}^{\star}\big\|_{\mathrm{F}}.\label{eq:R1-minimizer}
\end{align}

\begin{enumerate}
\item First, we claim that $\tilde{\bm{X}}^{t+1}$ satisfies
\begin{equation}
\bm{I}_{r}=\arg\min_{\bm{R}\in\cO^{r\times r}}\big\|\tilde{\bm{X}}^{t+1}\bm{R}-\bm{X}^{\star}\big\|_{\mathrm{F}},\label{eq:R2-minimizer}
\end{equation}
meaning that $\tilde{\bm{X}}^{t+1}$ is already rotated to the direction that is most ``aligned''
with $\bm{X}^{\star}$. This important property eases the analysis. 
In fact, in view of Lemma \ref{lemma:opp}, \eqref{eq:R2-minimizer}
follows if one can show that $\bm{X}^{\star\top}\tilde{\bm{X}}^{t+1}$
is symmetric and positive semidefinite. First of all, it follows from
Lemma \ref{lemma:opp} that $\bm{X}^{\star\top}\bm{X}^{t}\hat{\bm{H}}^{t}$
is symmetric and, hence, by definition, 
\[
\bm{X}^{\star\top}\tilde{\bm{X}}^{t+1}=\bm{X}^{\star\top}\bm{X}^{t}\hat{\bm{H}}^{t}-\frac{\eta}{p}\bm{X}^{\star\top}\cP_{\Omega}\left[\bm{X}^{t}\bm{X}^{t\top}-\left(\bm{M}^{\star}+\bm{E}\right)\right]\bm{X}^{\star}
\]
is also symmetric. Additionally, 
\[
\big\|\bm{X}^{\star\top}\tilde{\bm{X}}^{t+1}-\bm{M}^{\star}\big\|\leq\big\|\tilde{\bm{X}}^{t+1}-\bm{X}^{\star}\big\|\left\Vert \bm{X}^{\star}\right\Vert \leq \sigma_{\min} / 2,
\]
where the second inequality holds according to (\ref{eq:op-psd}). Weyl's inequality guarantees that
\[
\bm{X}^{\star\top}\tilde{\bm{X}}^{t+1}\succeq \frac{1}{2}\sigma_{\min} \bm{I}_r,
\]
thus justifying \eqref{eq:R2-minimizer} via Lemma \ref{lemma:opp}. 
\item With (\ref{eq:R1-minimizer}) and (\ref{eq:R2-minimizer}) in place,
we resort to Lemma \ref{lemma:align-diff-MC} to establish the bound.
Specifically, take $\bm{X}_{1}=\tilde{\bm{X}}^{t+1}$ and $\bm{X}_{2}=\bm{X}^{t+1}\hat{\bm{H}}^{t}$,
and it comes from (\ref{eq:op-psd}) that 
\[
  \left\Vert \bm{X}_{1}-\bm{X}^{\star}\right\Vert \left\Vert \bm{X}^{\star}\right\Vert \leq \sigma_{\min}/2.
\]
Moreover, we have 
\[
\left\Vert \bm{X}_{1}-\bm{X}_{2}\right\Vert \left\Vert \bm{X}^{\star}\right\Vert =\big\|\bm{X}^{t+1}\hat{\bm{H}}^{t}-\tilde{\bm{X}}^{t+1}\big\|\big\|\bm{X}^{\star}\big\|,
\]
in which 
\begin{align*}
\bm{X}^{t+1}\hat{\bm{H}}^{t}-\tilde{\bm{X}}^{t+1} & =\left(\bm{X}^{t}-\eta\frac{1}{p}\cP_{\Omega}\left[\bm{X}^{t}\bm{X}^{t\top}-\left(\bm{M}^{\star}+\bm{E}\right)\right]\bm{X}^{t}\right)\hat{\bm{H}}^{t}\\
 & \qquad-\left[\bm{X}^{t}\hat{\bm{H}}^{t}-\eta\frac{1}{p}\cP_{\Omega}\left[\bm{X}^{t}\bm{X}^{t\top}-\left(\bm{M}^{\star}+\bm{E}\right)\right]\bm{X}^{\star}\right]\\
 & =-\eta\frac{1}{p}\cP_{\Omega}\left[\bm{X}^{t}\bm{X}^{t\top}-\left(\bm{M}^{\star}+\bm{E}\right)\right]\left(\bm{X}^{t}\hat{\bm{H}}^{t}-\bm{X}^{\star}\right).
\end{align*}
This allows one to derive 
\begin{align}
\big\|\bm{X}^{t+1}\hat{\bm{H}}^{t}-\tilde{\bm{X}}^{t+1}\big\| & \leq\eta\left\Vert \frac{1}{p}\cP_{\Omega}\left[\bm{X}^{t}\bm{X}^{t\top}-\bm{M}^{\star}\right]\left(\bm{X}^{t}\hat{\bm{H}}^{t}-\bm{X}^{\star}\right)\right\Vert +\eta\left\Vert \frac{1}{p}\cP_{\Omega}(\bm{E})\left(\bm{X}^{t}\hat{\bm{H}}^{t}-\bm{X}^{\star}\right)\right\Vert \nonumber \\
 & \leq \eta \left(2n\left\Vert \bm{\Delta}^{t}\right\Vert _{2,\infty}^{2}+4\sqrt{n}\log n\left\Vert \bm{\Delta}^{t}\right\Vert _{2,\infty}\left\Vert \bm{X}^{\star}\right\Vert +C\sigma\sqrt{\frac{n}{p}}\right)\left\Vert \bm{\Delta}^{t}\right\Vert \label{eq:Xt+1_Xtilde}
\end{align}
for some absolute constant $C > 0$. Here the last inequality follows from Lemma \ref{lemma:noise-spectral} and Lemma \ref{lemma:energy-MC}. As a consequence, 
\begin{align*}
\left\Vert \bm{X}_{1}-\bm{X}_{2}\right\Vert \left\Vert \bm{X}^{\star}\right\Vert \leq\eta\left(2n\left\Vert \bm{\Delta}^{t}\right\Vert _{2,\infty}^{2}+4\sqrt{n}\log n\left\Vert \bm{\Delta}^{t}\right\Vert _{2,\infty}\left\Vert \bm{X}^{\star}\right\Vert +C\sigma\sqrt{\frac{n}{p}}\right)\left\Vert \bm{\Delta}^{t}\right\Vert \left\Vert \bm{X}^{\star}\right\Vert.
\end{align*}
Under our sample size condition and the noise condition (\ref{eq:mc-noise-condition}) and the induction hypotheses \eqref{eq:induction_original_MC_thm}, one can show 
\[
\left\Vert \bm{X}_{1}-\bm{X}_{2}\right\Vert \left\Vert \bm{X}^{\star}\right\Vert \leq \sigma_{\min} / 4.
\]
Apply Lemma \ref{lemma:align-diff-MC} and (\ref{eq:Xt+1_Xtilde})
to reach 
\begin{align*}
\alpha_{1} & \leq5\kappa\big\|\bm{X}^{t+1}\hat{\bm{H}}^{t}-\tilde{\bm{X}}^{t+1}\big\|\\
 & \leq5\kappa\eta\left(2n\left\Vert \bm{\Delta}^{t}\right\Vert _{2,\infty}^{2}+2\sqrt{n}\log n\left\Vert \bm{\Delta}^{t}\right\Vert _{2,\infty}\left\Vert \bm{X}^{\star}\right\Vert +C\sigma\sqrt{\frac{n}{p}}\right)\left\Vert \bm{\Delta}^{t}\right\Vert .
\end{align*}
\end{enumerate}
\item Combining the above bounds on $\alpha_{1}$ and $\alpha_{2}$, we
arrive at 
\begin{align*}
\big\|\bm{X}^{t+1}\hat{\bm{H}}^{t+1}-\bm{X}^{\star}\big\| & \leq\left(1-\frac{\sigma_{\min}}{2}\eta\right)\left\Vert \bm{\Delta}^{t}\right\Vert +\eta C\sigma\sqrt{\frac{n}{p}}\left\Vert \bm{X}^{\star}\right\Vert \\
 & \quad+\tilde{C}\eta\sqrt{\frac{\kappa^{2}\mu^{2}r^{3}\log^{3}n}{np}}\sigma_{\max}\left(C_{5}\rho^{t}\mu r\sqrt{\frac{1}{np}}+\frac{C_{8}}{\sigma_{\min}}\sigma\sqrt{\frac{n}{p}}\right)\left\Vert \bm{X}^{\star}\right\Vert \\
 & \quad+5\eta\kappa\left(2n\left\Vert \bm{\Delta}^{t}\right\Vert _{2,\infty}^{2}+2\sqrt{n}\log n\left\Vert \bm{\Delta}^{t}\right\Vert _{2,\infty}\left\Vert \bm{X}^{\star}\right\Vert +C\sigma\sqrt{\frac{n}{p}}\right)\left\Vert \bm{\Delta}^{t}\right\Vert \\
 & \leq C_{9}\rho^{t+1}\mu r\frac{1}{\sqrt{np}}\left\Vert \bm{X}^{\star}\right\Vert +C_{10}\frac{\sigma}{\sigma_{\min}}\sqrt{\frac{n}{p}}\left\Vert \bm{X}^{\star}\right\Vert ,
\end{align*}
with the proviso that $\rho\geq1-({\sigma_{\min}} / {3}) \cdot \eta$, $\kappa$
is a constant, and $n^{2}p\gg\mu^{3}r^{3}n\log^{3}n$.
\end{enumerate}

\subsubsection{Proof of Lemma \ref{lemma:phi_1-MC}}
In what follows, we first assume that the $\delta_{j,k}$'s are independent, and then use the standard decoupling trick to extend the result to symmetric sampling case (i.e.~$\delta_{j,k} = \delta_{k,j}$).

To begin with,
we justify the concentration bound for any $\bm{\Delta}^{t}$
independent of $\Omega$, followed by the standard covering argument
that extends the bound to all $\bm{\Delta}^{t}$. For any $\bm{\Delta}^{t}$
independent of $\Omega$, one has 
\begin{align*}
B & :=\max_{1\leq j\leq n}\left\Vert \bm{X}_{l,\cdot}^{\star}\left(\delta_{l,j}-p\right)\bm{\Delta}_{j,\cdot}^{t\top}\bm{X}_{j,\cdot}^{\star}\right\Vert _{2}\leq\left\Vert \bm{X}^{\star}\right\Vert _{2,\infty}^{2}\xi\\
\text{and}\qquad V & :=\left\Vert \EE\left[\sum_{j=1}^{n}\left(\delta_{l,j}-p\right)^{2}\bm{X}_{l,\cdot}^{\star}\bm{\Delta}_{j,\cdot}^{t\top}\bm{X}_{j,\cdot}^{\star}\left(\bm{X}_{l,\cdot}^{\star}\bm{\Delta}_{j,\cdot}^{t\top}\bm{X}_{j,\cdot}^{\star}\right)^{\top}\right]\right\Vert \\
 & \leq p\left\Vert \bm{X}_{l,\cdot}^{\star}\right\Vert _{2}^{2}\left\Vert \bm{X}^{\star}\right\Vert _{2,\infty}^{2}\left\Vert \sum_{j=1}^{n}\bm{\Delta}_{j,\cdot}^{t\top}\bm{\Delta}_{j,\cdot}^{t}\right\Vert \\
 & \leq p\left\Vert \bm{X}_{l,\cdot}^{\star}\right\Vert _{2}^{2}\left\Vert \bm{X}^{\star}\right\Vert _{2,\infty}^{2}\psi^{2}\\
 & \leq2p\left\Vert \bm{X}^{\star}\right\Vert _{2,\infty}^{2}\xi^{2}\sigma_{\max},
\end{align*}
where $\xi$ and $\psi$ are defined respectively in (\ref{eq:defn-xi-Deltat})
and (\ref{eq:defn-psi-Deltat}). Here, the last line makes use of
the fact that 
\begin{equation}
\left\Vert \bm{X}^{\star}\right\Vert _{2,\infty}\psi\ll\xi\left\Vert \bm{X}^{\star}\right\Vert =\xi\sqrt{\sigma_{\max}},\label{eq:psi-xi-inequality}
\end{equation}
as long as $n$ is sufficiently large. 
Apply the matrix Bernstein inequality \cite[Theorem 6.1.1]{Tropp:2015:IMC:2802188.2802189}
to get 
\begin{align*}
\mathbb{P}\left\{ \left\Vert \bm{A}_{l,\cdot}\right\Vert _{2}\geq t\right\}  & \leq2r\exp\left(-\frac{ct^{2}}{2p\xi^{2}\sigma_{\max}\left\Vert \bm{X}^{\star}\right\Vert _{2,\infty}^{2}+ t\cdot \left\Vert \bm{X}^{\star}\right\Vert _{2,\infty}^{2} \xi}\right)\\
 & \leq2r\exp\left(-\frac{ct^{2}}{4p\xi^{2}\sigma_{\max}\left\Vert \bm{X}^{\star}\right\Vert _{2,\infty}^{2}}\right)
\end{align*}
for some constant $c>0$, provided that 
\[
t\leq2p\sigma_{\max}\xi.
\]
This upper bound on $t$ is exactly the truncation level $\omega$
we introduce in (\ref{eq:op-truncation-level}). With this in mind,
we can easily verify that 
\[
\left\Vert \bm{A}_{l,\cdot}\right\Vert _{2}\ind_{\left\{ \left\Vert \bm{A}_{l,\cdot}\right\Vert _{2}\leq\omega\right\} }
\]
is a sub-Gaussian random variable with variance proxy not exceeding
$O\left(p\xi^{2}\sigma_{\max}\left\Vert \bm{X}^{\star}\right\Vert _{2,\infty}^{2}\log r\right)$.
Therefore, invoking the concentration bounds for quadratic functions
\cite[Theorem 2.1]{MR2994877} yields that for some constants $C_{0}, C>0$, with probability at least
$1-C_{0}e^{-Cnr\log n}$, 
\begin{align*}
\phi_{1}^{2}=\sum_{l=1}^{n}\left\Vert \bm{A}_{l,\cdot}\right\Vert _{2}^{2}\ind_{\left\{ \left\Vert \bm{A}_{l,\cdot}\right\Vert _{2}\leq\omega\right\} } & \lesssim p\xi^{2}\sigma_{\max}\|\bm{X}^{\star}\|_{2,\infty}^{2}nr\log^{2}n.
\end{align*}

Now that we have established an upper bound on any fixed matrix $\bm{\Delta}^{t}$
(which holds with exponentially high probability), we can proceed
to invoke the standard epsilon-net argument to establish a uniform
bound over all feasible $\bm{\Delta}^{t}$. This argument is fairly
standard, and is thus omitted; see \cite[Section 2.3.1]{Tao2012RMT}
or the proof of Lemma \ref{lemma:op-uniform-spectral-MC}. In conclusion,
we have that with probability exceeding $1-C_{0}e^{-\frac{1}{2}Cnr\log n}$,
\begin{align}
\phi_{1}=\sqrt{\sum_{l=1}^{n}\left\Vert \bm{A}_{l,\cdot}\right\Vert _{2}^{2}\ind_{\left\{ \left\Vert \bm{A}_{l,\cdot}\right\Vert _{2}\leq\omega\right\} }}\lesssim\sqrt{p\xi^{2}\sigma_{\max}\|\bm{X}^{\star}\|_{2,\infty}^{2}nr\log^{2}n}\label{eq:lemma-22-prev}
\end{align}
holds simultaneously for all $\bm{\Delta}^{t}\in\RR^{n\times r}$
obeying the conditions of the lemma. 

In the end, we comment on how to extend the bound to the symmetric sampling pattern where~$\delta_{j,k} = \delta_{k,j}$. Recall from (\ref{eq:lemma-22-A-l}) that the diagonal element $\delta_{l,l}$ cannot change the $\ell_{2}$ norm of $\bm{A}_{l,\cdot}$ by more than $\left\Vert\bm{X}^{\star}\right\Vert_{2,\infty}^2 \xi$. As a result, changing all the diagonals $\{\delta_{l,l}\}$ cannot change the quantity of interest (i.e.~$\phi_{1}$) by more than $\sqrt{n}\left\Vert\bm{X}^{\star}\right\Vert_{2,\infty}^2 \xi$. This is smaller than the right hand side of (\ref{eq:lemma-22-prev}) under our incoherence and sample size conditions. Hence from now on we ignore the effect of $\{\delta_{l,l}\}$ and focus on  off-diagonal terms. The proof then follows from the same argument as in \cite[Theorem D.2]{ge2016matrix}. More specifically, we can employ the construction of Bernoulli random variables introduced therein to demonstrate that the upper bound in (\ref{eq:lemma-22-prev}) still holds if the indicator $\delta_{i,j}$ is replaced by $(\tau_{i,j} + \tau_{i,j}') / 2$, where $\tau_{i,j}$ and $\tau_{i,j}'$ are independent copies of the symmetric Bernoulli random variables. Recognizing that $\sup_{\bm{\Delta}^{t}}\phi_{1}$ is a norm of the Bernoulli random variables $\tau_{i,j}$, one can repeat the decoupling argument in \cite[Claim D.3]{ge2016matrix} to finish the proof. We omit the details here for brevity.

\subsubsection{Proof of Lemma \ref{lemma:phi-2-MC}}
Observe from (\ref{eq:lemma-22-A-l}) that 
\begin{align}
\left\Vert \bm{A}_{l,\cdot}\right\Vert _{2} & \leq\left\Vert \bm{X}^{\star}\right\Vert _{2,\infty}\left\Vert \sum_{j=1}^{n}\left(\delta_{l,j}-p\right)\bm{\Delta}_{j,\cdot}^{t\top}\bm{X}_{j,\cdot}^{\star}\right\Vert \label{eq:AL-UB1}\\
 & \leq\left\Vert \bm{X}^{\star}\right\Vert _{2,\infty}\left(\left\Vert \sum_{j=1}^{n}\delta_{l,j}\bm{\Delta}_{j,\cdot}^{t\top}\bm{X}_{j,\cdot}^{\star}\right\Vert +p\left\Vert \bm{\Delta}^{t}\right\Vert \left\Vert \bm{X}^{\star}\right\Vert \right)\nonumber \\
 & \leq\left\Vert \bm{X}^{\star}\right\Vert _{2,\infty}\left(\left\Vert \left[\delta_{l,1}\bm{\Delta}_{1,\cdot}^{t\top},\cdots,\delta_{l,n}\bm{\Delta}_{n,\cdot}^{t\top}\right]\right\Vert \left\Vert \left[\begin{array}{c}
\delta_{l,1}\bm{X}_{1,\cdot}^{\star}\\
\vdots\\
\delta_{l,n}\bm{X}_{n,\cdot}^{\star}
\end{array}\right]\right\Vert +p\psi\left\Vert \bm{X}^{\star}\right\Vert \right)\nonumber \\
 & {\leq}\left\Vert \bm{X}^{\star}\right\Vert _{2,\infty}\left(\left\Vert \bm{G}_{l}\left(\bm{\Delta}^{t}\right)\right\Vert \cdot1.2\sqrt{p}\left\Vert \bm{X}^{\star}\right\Vert +p\psi\left\Vert \bm{X}^{\star}\right\Vert \right),\label{eq:AL-UB2}
\end{align}
where $\psi$ is as defined in \eqref{eq:defn-psi-Deltat} and $\bm{G}_{l}\left(\cdot\right)$ is as defined in Lemma  \ref{lemma:bernoulli-spectral-norm}. Here, the last inequality follows from Lemma \ref{lemma:bernoulli-spectral-norm},
namely, for some constant $C>0$, the following holds with probability at least $1-O(n^{-10})$
\begin{align}
\left\Vert \left[\begin{array}{c}
\delta_{l,1}\bm{X}_{1,\cdot}^{\star} \notag \\
\vdots\\
\delta_{l,n}\bm{X}_{n,\cdot}^{\star}
\end{array}\right]\right\Vert  & \leq\left(p\left\Vert \bm{X}^{\star}\right\Vert ^{2}+ C \sqrt{p\|\bm{X}^{\star}\|_{2,\infty}^{2}\left\Vert \bm{X}^{\star}\right\Vert ^{2}\log n}+ C\|\bm{X}^{\star}\|_{2,\infty}^{2}\log n\right)^{\frac{1}{2}} \notag \\
 & \leq\left(p+ C\sqrt{p\frac{\kappa\mu r}{n}\log n}+ C \frac{\kappa\mu r\log n}{n}\right)^{\frac{1}{2}}\left\Vert \bm{X}^{\star}\right\Vert \leq1.2\sqrt{p}\left\Vert \bm{X}^{\star}\right\Vert, \label{1119-ineq-regularity-event}
\end{align}
where we also use the incoherence condition (\ref{eq:incoherence-X-MC}) and the sample complexity condition $n^2 p\gg\kappa\mu r n \log n$. Hence,
the event 
\[
\left\Vert \bm{A}_{l,\cdot}\right\Vert _{2}\geq\omega=2p\sigma_{\max}\xi
\]
together with (\ref{eq:AL-UB1}) and (\ref{eq:AL-UB2}) necessarily
implies that 
\[
\left\Vert \sum_{j=1}^{n}\left(\delta_{l,j}-p\right)\bm{\Delta}_{j,\cdot}^{t\top}\bm{X}_{j,\cdot}^{\star}\right\Vert \geq2p\sigma_{\max}\frac{\xi}{\left\Vert \bm{X}^{\star}\right\Vert _{2,\infty}}\qquad\text{and}
\]
\begin{align*}
\left\Vert \bm{G}_{l}\left(\bm{\Delta}^{t}\right)\right\Vert  & \geq\frac{\frac{2p\sigma_{\max}\xi}{\left\Vert \bm{X}^{\star}\right\Vert \left\Vert \bm{X}^{\star}\right\Vert _{2,\infty}}-p\psi}{1.2\sqrt{p}}\geq\frac{\frac{2\sqrt{p}\left\Vert \bm{X}^{\star}\right\Vert \xi}{\left\Vert \bm{X}^{\star}\right\Vert _{2,\infty}}-\sqrt{p}\psi}{1.2}\geq1.5\sqrt{p}\frac{\xi}{\left\Vert \bm{X}^{\star}\right\Vert _{2,\infty}}\left\Vert \bm{X}^{\star}\right\Vert,
\end{align*}
where the last inequality follows from the bound (\ref{eq:psi-xi-inequality}).
As a result, with probability at least $1-O(n^{-10})$ (i.e. when (\ref{1119-ineq-regularity-event}) holds for all $l$'s) we can upper bound $\phi_{2}$ by 
\[
\phi_{2}=\sqrt{\sum_{l=1}^{n}\left\Vert \bm{A}_{l,\cdot}\right\Vert _{2}^{2}\ind_{\left\{ \left\Vert \bm{A}_{l,\cdot}\right\Vert _{2}\geq\omega\right\} }}\leq\sqrt{\sum_{l=1}^{n}\left\Vert \bm{A}_{l,\cdot}\right\Vert _{2}^{2}\ind_{\left\{ \left\Vert \bm{G}_{l}\left(\bm{\Delta}^{t}\right)\right\Vert \geq\frac{1.5\sqrt{p}\xi\sqrt{\sigma_{\max}}}{\left\Vert \bm{X}^{\star}\right\Vert _{2,\infty}}\right\} }},
\]
where the indicator functions are now specified with respect to $\left\Vert \bm{G}_{l}\left(\bm{\Delta}^{t}\right)\right\Vert $. 

Next, we divide into multiple cases based on the size of $\left\Vert \bm{G}_{l}\left(\bm{\Delta}^{t}\right)\right\Vert $.
By Lemma \ref{lemma:op-uniform-spectral-MC}, for some constants $c_{1}, c_{2} > 0$, with probability at
least $1-c_{1}\exp\left(-c_{2}nr\log n\right)$,  
\begin{equation}
\sum_{l=1}^{n}\ind_{\left\{ \left\Vert \bm{G}_{l}\left(\bm{\Delta}^{t}\right)\right\Vert \geq 4\sqrt{p}\psi+\sqrt{2^{k}r}\xi\right\} }\leq\frac{\alpha n}{2^{k-3}}\label{eq:G-bound}
\end{equation}
for any $k\geq0$ and any $\alpha\gtrsim\log n$. We claim that it
suffices to consider the set of sufficiently large $k$ obeying 
\begin{equation}
\sqrt{2^{k}r}\xi\geq 4\sqrt{p}\psi\quad\text{or equivalently}\quad k\geq\log\frac{16p\psi^{2}}{r\xi^{2}};\label{eq:k-lower-bound}
\end{equation}
otherwise we can use (\ref{eq:psi-xi-inequality}) to obtain 
\[
4\sqrt{p}\psi+\sqrt{2^{k}r}\xi\leq 8\sqrt{p}\psi\ll1.5\sqrt{p}\frac{\xi}{\left\Vert \bm{X}^{\star}\right\Vert _{2,\infty}}\left\Vert \bm{X}^{\star}\right\Vert ,
\]
which contradicts the event $\left\Vert \bm{A}_{l,\cdot}\right\Vert _{2}\geq\omega$.
Consequently, we divide all indices into the following sets 
\begin{equation}
S_{k}=\left\{ 1\leq l\leq n: \left\Vert \bm{G}_{l}\left(\bm{\Delta}^{t}\right)\right\Vert \in\big(\sqrt{2^{k}r}\xi,\sqrt{2^{k+1}r}\xi\big]\right\} \label{eq:constraint-Sk}
\end{equation}
defined for each integer $k$ obeying (\ref{eq:k-lower-bound}). Under
the condition (\ref{eq:k-lower-bound}), it follows from (\ref{eq:G-bound})
that 
\[
\sum_{l=1}^{n}\ind_{\left\{ \left\Vert \bm{G}_{l}\left(\bm{\Delta}^{t}\right)\right\Vert \geq\sqrt{2^{k+2}r}\xi\right\} }\leq\sum_{l=1}^{n}\ind_{\left\{ \left\Vert \bm{G}_{l}\left(\bm{\Delta}^{t}\right)\right\Vert \geq 4\sqrt{p}\psi+\sqrt{2^{k}r}\xi\right\} }\leq\frac{\alpha n}{2^{k-3}},
\]
meaning that the cardinality of $S_{k}$ satisfies 
\[
\left|S_{k+2}\right|\leq\frac{\alpha n}{2^{k-3}} \qquad \text{or} \qquad \left|S_{k}\right|\leq\frac{\alpha n}{2^{k-5}}
\]
which decays exponentially fast as $k$ increases. Therefore, when
restricting attention to the set of indices within $S_{k}$, we can
obtain 
\begin{align*}
\sqrt{\sum_{l\in S_{k}}\left\Vert \bm{A}_{l,\cdot}\right\Vert _{2}^{2}} & \overset{\left(\text{i}\right)}{\leq}\sqrt{|S_{k}|\cdot\left\Vert \bm{X}^{\star}\right\Vert _{2,\infty}^{2}\left(1.2\sqrt{2^{k+1}r}\xi\sqrt{p}\left\Vert \bm{X}^{\star}\right\Vert +p\psi\left\Vert \bm{X}^{\star}\right\Vert \right)^{2}}\\
 & \leq\sqrt{\frac{\alpha n}{2^{k-5}}}\left\Vert \bm{X}^{\star}\right\Vert _{2,\infty}\left(2\sqrt{2^{k+1}r}\xi\sqrt{p}\left\Vert \bm{X}^{\star}\right\Vert +p\psi\left\Vert \bm{X}^{\star}\right\Vert \right)\\
 & \overset{\left(\text{ii}\right)}{\leq}4\sqrt{\frac{\alpha n}{2^{k-5}}}\left\Vert \bm{X}^{\star}\right\Vert _{2,\infty}\sqrt{2^{k+1}r}\xi\sqrt{p}\left\Vert \bm{X}^{\star}\right\Vert \\
 & \overset{\left(\text{iii}\right)}{\leq} 32\sqrt{\alpha\kappa\mu r^{2}p}\xi\left\Vert \bm{X}^{\star}\right\Vert ^{2},
\end{align*}
where (i) follows from the bound (\ref{eq:AL-UB2}) and the constraint
(\ref{eq:constraint-Sk}) in $S_{k}$, (ii) is a consequence of
(\ref{eq:k-lower-bound}) and (iii) uses the incoherence condition (\ref{eq:incoherence-X-MC}).

Now that we have developed an upper bound with respect to each $S_{k}$,
we can add them up to yield the final upper bound. Note that there
are in total no more than $O\left(\log n\right)$ different sets,
i.e.~$S_{k}=\emptyset$ if $k\geq c_{1}\log n$ for $c_{1}$ sufficiently
large. This arises since 
\[
\|\bm{G}_{l}(\bm{\Delta}^{t})\|\leq\|\bm{\Delta}^{t}\|_{\mathrm{F}}\leq\sqrt{n}\|\bm{\Delta}^{t}\|_{2,\infty}\leq\sqrt{n}\xi\leq\sqrt{n}\sqrt{r}\xi
\]
and hence 
\[
\ind_{\left\{ \left\Vert \bm{G}_{l}\left(\bm{\Delta}^{t}\right)\right\Vert \geq 4\sqrt{p}\psi+\sqrt{2^{k}r}\xi\right\} }=0\qquad\text{and}\qquad S_{k}=\emptyset
\]
if $k/ \log n$ is sufficiently large. One can thus conclude that 
\begin{align*}
\phi_{2}^{2} & \leq\sum_{k=\log\frac{16p\psi^{2}}{r\xi^{2}}}^{c_{1}\log n}\sum_{l\in S_{k}}\left\Vert \bm{A}_{l,\cdot}\right\Vert _{2}^{2}\lesssim\left(\sqrt{\alpha\kappa\mu r^{2}p}\xi\left\Vert \bm{X}^{\star}\right\Vert ^{2}\right)^{2}\cdot\log n,
\end{align*}
leading to $\phi_{2} \lesssim \xi\sqrt{\alpha\kappa\mu r^{2}p\log n}\left\Vert \bm{X}^{\star}\right\Vert ^{2}$. The proof is finished by taking $\alpha=c\log n$ for some sufficiently
large constant $c>0$. 

\subsection{Proof of Lemma \ref{lemma:MC-hypothesis-consequence}\label{subsec:Proof-of-Lemma-MC-hypothesis-consequence}}


\begin{enumerate}
\item  To obtain \eqref{eq:combine-fro}, we invoke Lemma \ref{lemma:align-diff-MC}.
Setting
$\bm{X}_{1}=\bm{X}^{t}\hat{\bm{H}}^{t}$ and $\bm{X}_{2}=\bm{X}^{t,(l)}\bm{R}^{t,\left(l\right)}$, we get
\[
\left\Vert \bm{X}_{1}-\bm{X}^{\star}\right\Vert \left\Vert \bm{X}^{\star}\right\Vert \overset{\left(\text{i}\right)}{\leq}C_{9}\rho^{t}\mu r\frac{1}{\sqrt{np}}\sigma_{\max}+\frac{C_{10}}{\sigma_{\min}}\sigma\sqrt{\frac{n\log n}{p}}\sigma_{\max}\overset{\left(\text{ii}\right)}{\leq}\frac{1}{2}\sigma_{\min},
\]
where (i) follows from (\ref{eq:induction_original_operator-MC})
and (ii) holds as long as $n^{2}p\gg\kappa^{2}\mu^{2}r^{2}n$ and the noise satisfies (\ref{eq:mc-noise-condition}). In
addition, 
\begin{align*}
\left\Vert \bm{X}_{1}-\bm{X}_{2}\right\Vert \left\Vert \bm{X}^{\star}\right\Vert  & \leq\left\Vert \bm{X}_{1}-\bm{X}_{2}\right\Vert _{\mathrm{F}}\left\Vert \bm{X}^{\star}\right\Vert \\
 & \overset{\left(\text{i}\right)}{\leq}\left(C_{3}\rho^{t}\mu r\sqrt{\frac{\log n}{np}}\left\Vert \bm{X}^{\star}\right\Vert _{2,\infty}+\frac{C_{7}}{\sigma_{\min}}\sigma\sqrt{\frac{n\log n}{p}}\left\Vert \bm{X}^{\star}\right\Vert _{2,\infty}\right)\left\Vert \bm{X}^{\star}\right\Vert \\
 & \overset{\left(\text{ii}\right)}{\leq}C_{3}\rho^{t}\mu r\sqrt{\frac{\log n}{np}}\sigma_{\max}+\frac{C_{7}}{\sigma_{\min}}\sigma\sqrt{\frac{n\log n}{p}}\sigma_{\max}\\
 & \overset{\left(\text{iii}\right)}{\leq}\frac{1}{2}\sigma_{\min},
\end{align*}
where (i) utilizes (\ref{eq:induction_ell_2_diff-MC}), (ii) follows
since $\left\Vert \bm{X}^{\star}\right\Vert _{2,\infty}\leq\left\Vert \bm{X}^{\star}\right\Vert $,
and (iii) holds if $n^{2}p\gg\kappa^{2}\mu^{2}r^{2}n\log n$ and the noise satisfies (\ref{eq:mc-noise-condition}).
With these in place, Lemma \ref{lemma:align-diff-MC} immediately
yields (\ref{eq:combine-fro}). 

\item The first inequality in (\ref{eq:implication-finer-fro-hat-t-l})
follows directly from the definition of $\hat{\bm{H}}^{t,\left(l\right)}$. The second inequality is concerned with the estimation error of $\bm{X}^{t,\left(l\right)}\bm{R}^{t,\left(l\right)}$
with respect to the Frobenius norm. Combining (\ref{eq:induction_original_ell_2-MC}), (\ref{eq:induction_ell_2_diff-MC}) and the triangle inequality
yields 
\begin{align}
 & \left\Vert \bm{X}^{t,\left(l\right)}\bm{R}^{t,\left(l\right)}-\bm{X}^{\star}\right\Vert _{\mathrm{F}}\leq\left\Vert \bm{X}^{t}\hat{\bm{H}}^{t}-\bm{X}^{\star}\right\Vert _{\mathrm{F}}+\left\Vert \bm{X}^{t}\hat{\bm{H}}^{t}-\bm{X}^{t,\left(l\right)}\bm{R}^{t,\left(l\right)}\right\Vert _{\mathrm{F}}\nonumber \\
 & \quad\leq C_{4}\rho^{t}\mu r\frac{1}{\sqrt{np}}\left\Vert \bm{X}^{\star}\right\Vert _{\mathrm{F}}+\frac{C_{1}\sigma}{\sigma_{\min}}\sqrt{\frac{n}{p}}\left\Vert \bm{X}^{\star}\right\Vert _{\mathrm{F}}+C_{3}\rho^{t}\mu r\sqrt{\frac{\log n}{np}}\left\Vert \bm{X}^{\star}\right\Vert _{2,\infty}+\frac{C_{7}\sigma}{\sigma_{\min}}\sqrt{\frac{n\log n}{p}}\left\Vert \bm{X}^{\star}\right\Vert _{2,\infty}\nonumber \\
 & \quad\leq C_{4}\rho^{t}\mu r\frac{1}{\sqrt{np}}\left\Vert \bm{X}^{\star}\right\Vert _{\mathrm{F}}+\frac{C_{1}\sigma}{\sigma_{\min}}\sqrt{\frac{n}{p}}\left\Vert \bm{X}^{\star}\right\Vert _{\mathrm{F}}+C_{3}\rho^{t}\mu r\sqrt{\frac{\log n}{np}}\sqrt{\frac{\kappa\mu}{n}}\left\Vert \bm{X}^{\star}\right\Vert _{\mathrm{F}}+\frac{C_{7}\sigma}{\sigma_{\min}}\sqrt{\frac{n\log n}{p}}\sqrt{\frac{\kappa\mu}{n}}\left\Vert \bm{X}^{\star}\right\Vert _{\mathrm{F}}\nonumber \\
 & \quad\leq2C_{4}\rho^{t}\mu r\frac{1}{\sqrt{np}}\left\Vert \bm{X}^{\star}\right\Vert _{\mathrm{F}}+\frac{2C_{1}\sigma}{\sigma_{\min}}\sqrt{\frac{n}{p}}\left\Vert \bm{X}^{\star}\right\Vert _{\mathrm{F}},\label{eq:implication-finer-fro-R-t-l}
\end{align}
where the last step holds true as long as $n\gg\kappa\mu\log n$.

\item To obtain \eqref{eq:implication-finer-2-inf-R-t-l}, we use (\ref{eq:induction_ell_2_diff-MC})
and (\ref{eq:induction_original_ell_infty-MC}) to get 
\begin{align*}
 & \left\Vert \bm{X}^{t,\left(l\right)}\bm{R}^{t,\left(l\right)}-\bm{X}^{\star}\right\Vert _{2,\infty}\leq\left\Vert \bm{X}^{t}\hat{\bm{H}}^{t}-\bm{X}^{\star}\right\Vert _{2,\infty}+\left\Vert \bm{X}^{t}\hat{\bm{H}}^{t}-\bm{X}^{t,\left(l\right)}\bm{R}^{t,\left(l\right)}\right\Vert _{\mathrm{F}}\\
 & \quad\leq C_{5}\rho^{t}\mu r\sqrt{\frac{\log n}{np}}\left\Vert \bm{X}^{\star}\right\Vert _{2,\infty}+\frac{C_{8}\sigma}{\sigma_{\min}}\sqrt{\frac{n\log n}{p}}\left\Vert \bm{X}^{\star}\right\Vert _{2,\infty}+C_{3}\rho^{t}\mu r\sqrt{\frac{\log n}{np}}\left\Vert \bm{X}^{\star}\right\Vert _{2,\infty}+\frac{C_{7}\sigma}{\sigma_{\min}}\sqrt{\frac{n\log n}{p}}\left\Vert \bm{X}^{\star}\right\Vert _{2,\infty}\\
 & \quad\leq\left(C_{3}+C_{5}\right)\rho^{t}\mu r\sqrt{\frac{\log n}{np}}\left\Vert \bm{X}^{\star}\right\Vert _{2,\infty}+\frac{C_{8}+C_{7}}{\sigma_{\min}}\sigma\sqrt{\frac{n\log n}{p}}\left\Vert \bm{X}^{\star}\right\Vert _{2,\infty}.
\end{align*}

\item Finally, to obtain \eqref{eq:implication-finer-op-hat-t-l}, one can take the triangle
inequality 
\begin{align*}
\left\Vert \bm{X}^{t,\left(l\right)}\hat{\bm{H}}^{t,\left(l\right)}-\bm{X}^{\star}\right\Vert  & \leq\left\Vert \bm{X}^{t,\left(l\right)}\hat{\bm{H}}^{t,\left(l\right)}-\bm{X}^{t}\hat{\bm{H}}^{t}\right\Vert _{\mathrm{F}}+\left\Vert \bm{X}^{t}\hat{\bm{H}}^{t}-\bm{X}^{\star}\right\Vert \\
 & \leq5\kappa\left\Vert \bm{X}^{t}\hat{\bm{H}}^{t}-\bm{X}^{t,\left(l\right)}\bm{R}^{t,\left(l\right)}\right\Vert _{\mathrm{F}}+\left\Vert \bm{X}^{t}\hat{\bm{H}}^{t}-\bm{X}^{\star}\right\Vert ,
\end{align*}
where the second line follows from \eqref{eq:combine-fro}. Combine
(\ref{eq:induction_ell_2_diff-MC}) and (\ref{eq:induction_original_operator-MC})
to yield 
\begin{align*}
 & \left\Vert \bm{X}^{t,\left(l\right)}\hat{\bm{H}}^{t,\left(l\right)}-\bm{X}^{\star}\right\Vert \\
 & \quad\leq5\kappa\left(C_{3}\rho^{t}\mu r\sqrt{\frac{\log n}{np}}\left\Vert \bm{X}^{\star}\right\Vert _{2,\infty}+\frac{C_{7}}{\sigma_{\min}}\sigma\sqrt{\frac{n\log n}{p}}\left\Vert \bm{X}^{\star}\right\Vert _{2,\infty}\right)+C_{9}\rho^{t}\mu r\frac{1}{\sqrt{np}}\left\Vert \bm{X}^{\star}\right\Vert +\frac{C_{10}}{\sigma_{\min}}\sigma\sqrt{\frac{n}{p}}\left\Vert \bm{X}^{\star}\right\Vert \\
 & \quad\leq5\kappa\sqrt{\frac{\kappa\mu r}{n}}\left\Vert \bm{X}^{\star}\right\Vert \left(C_{3}\rho^{t}\mu r\sqrt{\frac{\log n}{np}}+\frac{C_{7}}{\sigma_{\min}}\sigma\sqrt{\frac{n\log n}{p}}\right)+C_{9}\rho^{t}\mu r\frac{1}{\sqrt{np}}\left\Vert \bm{X}^{\star}\right\Vert +\frac{C_{10}\sigma}{\sigma_{\min}}\sqrt{\frac{n}{p}}\left\Vert \bm{X}^{\star}\right\Vert \\
 & \quad\leq2C_{9}\rho^{t}\mu r\frac{1}{\sqrt{np}}\left\Vert \bm{X}^{\star}\right\Vert +\frac{2C_{10}\sigma}{\sigma_{\min}}\sqrt{\frac{n}{p}}\left\Vert \bm{X}^{\star}\right\Vert ,
\end{align*}
where the second inequality uses the incoherence of $\bm{X}^{\star}$
(cf. (\ref{eq:incoherence-X-MC})) and the last inequality holds as long
as $n\gg\kappa^{3}\mu r\log n$.
\end{enumerate}

\subsection{Proof of Lemma \ref{lemma:loop-MC}\label{subsec:Proof-of-Lemma-loop-MC}}

From the definition of $\bm{R}^{t+1,\left(l\right)}$ (see (\ref{eq:rotation_r_t_l})),
we must have 
\begin{align*}
\left\Vert \bm{X}^{t+1}\hat{\bm{H}}^{t+1}-\bm{X}^{t+1,(l)}\bm{R}^{t+1,(l)}\right\Vert _{\mathrm{F}} & \leq\left\Vert \bm{X}^{t+1}\hat{\bm{H}}^{t}-\bm{X}^{t+1,(l)}\bm{R}^{t,\left(l\right)}\right\Vert _{\mathrm{F}}.
\end{align*}
The gradient update rules in (\ref{eq:gradient_update-MC}) and (\ref{eq:loo-gradient_update-MC}) allow one to express 
\begin{align*}
\bm{X}^{t+1}\hat{\bm{H}}^{t}-\bm{X}^{t+1,(l)}\bm{R}^{t,\left(l\right)} & =\left[\bm{X}^{t}-\eta\nabla f\left(\bm{X}^{t}\right)\right]\hat{\bm{H}}^{t}-\left[\bm{X}^{t,\left(l\right)}-\eta\nabla f^{(l)}\big(\bm{X}^{t,(l)}\big)\right]\bm{R}^{t,\left(l\right)}\\
 & =\bm{X}^{t}\hat{\bm{H}}^{t}-\eta\nabla f\big(\bm{X}^{t}\hat{\bm{H}}^{t}\big)-\left[\bm{X}^{t,\left(l\right)}\bm{R}^{t,\left(l\right)}-\eta\nabla f^{\left(l\right)}\big(\bm{X}^{t,\left(l\right)}\bm{R}^{t,\left(l\right)}\big)\right]\\
 & =\big(\bm{X}^{t}\hat{\bm{H}}^{t}-\bm{X}^{t,\left(l\right)}\bm{R}^{t,\left(l\right)}\big)-\eta\left[\nabla f(\bm{X}^{t}\hat{\bm{H}}^{t})-\nabla f\big(\bm{X}^{t,\left(l\right)}\bm{R}^{t,\left(l\right)})\right]\\
 & \quad-\eta\left[\nabla f\big(\bm{X}^{t,\left(l\right)}\bm{R}^{t,\left(l\right)}\big)-\nabla f^{\left(l\right)}\big(\bm{X}^{t,\left(l\right)}\bm{R}^{t,\left(l\right)}\big)\right],
\end{align*}
where we have again used the fact that $\nabla f\left(\bm{X}^{t}\right)\bm{R}=\nabla f(\bm{X}^{t}\bm{R})$
for any orthonormal matrix $\bm{R}\in\cO^{r\times r}$ (similarly for
$\nabla f^{(l)}\big(\bm{X}^{t,(l)}\big)$). Relate the right-hand
side of the above equation with $\nabla f_{\mathrm{clean}}\left(\bm{X}\right)$ to reach
\begin{align}
\bm{X}^{t+1}\hat{\bm{H}}^{t}-\bm{X}^{t+1,(l)}\bm{R}^{t,\left(l\right)} & =\underbrace{\big(\bm{X}^{t}\hat{\bm{H}}^{t}-\bm{X}^{t,\left(l\right)}\bm{R}^{t,\left(l\right)}\big)-\eta\left[\nabla f_{\mathrm{clean}}\big(\bm{X}^{t}\hat{\bm{H}}^{t}\big)-\nabla f_{\mathrm{clean}}\big(\bm{X}^{t,\left(l\right)}\bm{R}^{t,\left(l\right)}\big)\right]}_{:=\bm{B}_{1}^{\left(l\right)}}\nonumber \\
 & \quad-\underbrace{\eta\left[\frac{1}{p}\mathcal{P}_{\Omega_{l}}\left(\bm{X}^{t,\left(l\right)}\bm{X}^{t,\left(l\right)\top}-\bm{M}^{\star}\right)-\mathcal{P}_{l}\left(\bm{X}^{t,\left(l\right)}\bm{X}^{t,\left(l\right)\top}-\bm{M}^{\star}\right)\right]\bm{X}^{t,\left(l\right)}\bm{R}^{t,\left(l\right)}}_{:=\bm{B}_{2}^{(l)}}\nonumber \\
 & \quad+\underbrace{\eta\frac{1}{p}\cP_{\Omega}\left(\bm{E}\right)\left(\bm{X}^{t}\hat{\bm{H}}^{t}-\bm{X}^{t,\left(l\right)}\bm{R}^{t,\left(l\right)}\right)}_{:=\bm{B}_{3}^{\left(l\right)}}+\underbrace{\eta\frac{1}{p}\mathcal{P}_{\Omega_{l}}\left(\bm{E}\right)\bm{X}^{t,\left(l\right)}\bm{R}^{t,\left(l\right)}}_{:=\bm{B}_{4}^{\left(l\right)}},\label{eq:LOO-difference-defn-MC}
\end{align}
where we have used the following relationship between $\nabla f^{(l)}\left(\bm{X}\right)$
and $\nabla f\left(\bm{X}\right)$: 
\begin{equation}
\nabla f^{(l)}\left(\bm{X}\right)=\nabla f\left(\bm{X}\right)-\frac{1}{p}\mathcal{P}_{\Omega_{l}}\left[\bm{X}\bm{X}^{\top}-\left(\bm{M}^{\star}+\bm{E}\right)\right]\bm{X}+\mathcal{P}_{l}\left(\bm{X}\bm{X}^{\top}-\bm{M}^{\star}\right)\bm{X}\label{eq:gradient-and-loo-gradient-MC}
\end{equation}
for all $\bm{X}\in\RR^{n\times r}$ with $\mathcal{P}_{\Omega_{l}}$
and $\mathcal{P}_{l}$ defined respectively in (\ref{eq:projection-Omega-l-1})
and (\ref{eq:projection-Omega-l}). In the sequel, we control the
four terms in  reverse order. 
\begin{enumerate}
\item The last term $\bm{B}_{4}^{\left(l\right)}$ is controlled via the
following lemma. \begin{lemma}\label{lemma:loop-remainder-4-MC}Suppose
that the sample size obeys $n^2 p>C\mu^2 r^2 n \log^2 n$ for some sufficiently large
constant $C>0$. Then with probability at least $1-O\left(n^{-10}\right)$, the matrix $\bm{B}_{4}^{(l)}$ as defined in (\ref{eq:LOO-difference-defn-MC})
satisfies 
\[
\left\Vert \bm{B}_{4}^{\left(l\right)}\right\Vert _{\mathrm{F}}\lesssim\eta\sigma\sqrt{\frac{n\log n}{p}}\left\Vert \bm{X}^{\star}\right\Vert _{2,\infty}.
\]
\end{lemma} 
\item The third term $\bm{B}_{3}^{\left(l\right)}$ can be bounded as
follows 
\[
\left\Vert \bm{B}_{3}^{(l)}\right\Vert _{\mathrm{F}}\leq\eta\left\Vert \frac{1}{p}\cP_{\Omega}\left(\bm{E}\right)\right\Vert \left\Vert \bm{X}^{t}\hat{\bm{H}}^{t}-\bm{X}^{t,\left(l\right)}\bm{R}^{t,\left(l\right)}\right\Vert _{\mathrm{F}}\lesssim\eta\sigma\sqrt{\frac{n}{p}}\left\Vert \bm{X}^{t}\hat{\bm{H}}^{t}-\bm{X}^{t,\left(l\right)}\bm{R}^{t,\left(l\right)}\right\Vert _{\mathrm{F}},
\]
where the second inequality comes from Lemma \ref{lemma:noise-spectral}. 
\item For the second term $\bm{B}_{2}^{(l)}$, we have the following lemma.
\begin{lemma}\label{lemma:ell_2_remainder-1-MC}Suppose that the
sample size obeys $n^{2}p\gg\mu^{2}r^{2}n\log n$. Then with probability exceeding 
$1-O\left(n^{-10}\right)$, the matrix $\bm{B}_{2}^{(l)}$ as defined in (\ref{eq:LOO-difference-defn-MC})
satisfies 
\begin{equation}
\left\Vert \bm{B}_{2}^{(l)}\right\Vert _{\mathrm{F}}\lesssim\eta\sqrt{\frac{\kappa^{2}\mu^{2}r^{2}\log n}{np}}\left\Vert \bm{X}^{t,\left(l\right)}\bm{R}^{t,\left(l\right)}-\bm{X}^{\star}\right\Vert _{2,\infty}\sigma_{\max}.\label{eq:nu2-UB-1-MC}
\end{equation}
\end{lemma} 
\item Regarding the first term $\bm{B}_{1}^{(l)}$, apply the fundamental theorem of calculus \cite[Chapter XIII, Theorem 4.2]{lang1993real}
to get 
\begin{equation}
\mathrm{vec}\big(\bm{B}_{1}^{(l)}\big)=\left(\bm{I}_{nr}-\eta\int_{0}^{1}\nabla^{2}f_{\mathrm{clean}}\left(\bm{X}(\tau)\right)\mathrm{d}\tau\right)\mathrm{vec}\left(\bm{X}^{t}\hat{\bm{H}}^{t}-\bm{X}^{t,\left(l\right)}\bm{R}^{t,\left(l\right)}\right),\label{eq:nu1-Taylor-1-MC}
\end{equation}
where we abuse the notation and denote $\bm{X}(\tau):=\bm{X}^{t,\left(l\right)}\bm{R}^{t,\left(l\right)}+\tau\left(\bm{X}^{t}\hat{\bm{H}}^{t}-\bm{X}^{t,\left(l\right)}\bm{R}^{t,\left(l\right)}\right)$.
Going through the same derivations as in the proof of Lemma~\ref{lemma:induction-fro-MC}
(see Appendix \ref{sec:proof-induction-fro-MC}), we get 
\begin{equation}
\big\|\bm{B}_{1}^{(l)}\big\|_{\mathrm{F}}\leq\left(1-\frac{\sigma_{\min}}{4}\eta\right)\left\Vert \bm{X}^{t}\hat{\bm{H}}^{t}-\bm{X}^{t,\left(l\right)}\bm{R}^{t,\left(l\right)}\right\Vert _{\mathrm{F}}\label{eq:nu1-Fnorm-bound}
\end{equation}
with the proviso that $0<\eta\leq({2\sigma_{\min}}) / ({25\sigma_{\max}^{2}})$. 
\end{enumerate}
Applying the triangle inequality to (\ref{eq:LOO-difference-defn-MC}) and invoking the preceding four bounds, we arrive at 
\begin{align*}
 & \left\Vert \bm{X}^{t+1}\hat{\bm{H}}^{t+1}-\bm{X}^{t+1,(l)}\bm{R}^{t+1,\left(l\right)}\right\Vert _{\mathrm{F}}\\
 & \quad\leq\left(1-\frac{\sigma_{\min}}{4}\eta\right)\left\Vert \bm{X}^{t}\hat{\bm{H}}^{t}-\bm{X}^{t,\left(l\right)}\bm{R}^{t,\left(l\right)}\right\Vert _{\mathrm{F}}+\tilde{C}\eta\sqrt{\frac{\kappa^{2}\mu^{2}r^{2}\log n}{np}}\left\Vert \bm{X}^{t,\left(l\right)}\bm{R}^{t,\left(l\right)}-\bm{X}^{\star}\right\Vert _{2,\infty}\sigma_{\max}\\
 & \quad\quad\quad+\tilde{C}\eta\sigma\sqrt{\frac{n}{p}}\left\Vert \bm{X}^{t}\hat{\bm{H}}^{t}-\bm{X}^{t,\left(l\right)}\bm{R}^{t,\left(l\right)}\right\Vert _{\mathrm{F}}+\tilde{C}\eta\sigma\sqrt{\frac{n\log n}{p}}\left\Vert \bm{X}^{\star}\right\Vert _{2,\infty}\\
 & \quad=\left(1-\frac{\sigma_{\min}}{4}\eta+ \tilde{C}\eta\sigma\sqrt{\frac{n}{p}}\right)\left\Vert \bm{X}^{t}\hat{\bm{H}}^{t}-\bm{X}^{t,\left(l\right)}\bm{R}^{t,\left(l\right)}\right\Vert _{\mathrm{F}}+\tilde{C}\eta\sqrt{\frac{\kappa^{2}\mu^{2}r^{2}\log n}{np}}\left\Vert \bm{X}^{t,\left(l\right)}\bm{R}^{t,\left(l\right)}-\bm{X}^{\star}\right\Vert _{2,\infty}\sigma_{\max}\\
 & \quad\quad\quad+\tilde{C}\eta\sigma\sqrt{\frac{n\log n}{p}}\left\Vert \bm{X}^{\star}\right\Vert _{2,\infty}\\
 & \quad\leq\left(1-\frac{2\sigma_{\min}}{9}\eta\right)\left\Vert \bm{X}^{t}\hat{\bm{H}}^{t}-\bm{X}^{t,\left(l\right)}\bm{R}^{t,\left(l\right)}\right\Vert _{\mathrm{F}}+\tilde{C}\eta\sqrt{\frac{\kappa^{2}\mu^{2}r^{2}\log n}{np}}\left\Vert \bm{X}^{t,\left(l\right)}\bm{R}^{t,\left(l\right)}-\bm{X}^{\star}\right\Vert _{2,\infty}\sigma_{\max}\\
 & \quad\quad\quad+\tilde{C}\eta\sigma\sqrt{\frac{n\log n}{p}}\left\Vert \bm{X}^{\star}\right\Vert _{2,\infty}
\end{align*}
for some absolute constant $\tilde{C}>0$. Here the last inequality holds as long
as $\sigma\sqrt{{n} / {p}}\ll\sigma_{\min}$, which is satisfied under our noise condition (\ref{eq:mc-noise-condition}). This taken collectively 
with the hypotheses (\ref{eq:induction_ell_2_diff-MC}) and (\ref{eq:implication-finer-2-inf-R-t-l})
leads to 
\begin{align*}
 & \left\Vert \bm{X}^{t+1}\hat{\bm{H}}^{t+1}-\bm{X}^{t+1,(l)}\bm{R}^{t+1,\left(l\right)}\right\Vert _{\mathrm{F}}\\
 & \quad\leq\left(1-\frac{2\sigma_{\min}}{9}\eta\right)\left(C_{3}\rho^{t}\mu r\sqrt{\frac{\log n}{np}}\left\Vert \bm{X}^{\star}\right\Vert _{2,\infty}+C_{7}\frac{\sigma}{\sigma_{\min}}\sqrt{\frac{n\log n}{p}}\left\Vert \bm{X}^{\star}\right\Vert _{2,\infty}\right)\\
 & \quad\quad+\tilde{C}\eta\sqrt{\frac{\kappa^{2}\mu^{2}r^{2}\log n}{np}}\left[\left(C_{3}+C_{5}\right)\rho^{t}\mu r\sqrt{\frac{\log n}{np}}+(C_{8}+C_{7})\frac{\sigma}{\sigma_{\min}}\sqrt{\frac{n\log n}{p}}\right]\left\Vert \bm{X}^{\star}\right\Vert _{2,\infty}\sigma_{\max}\\
 & \quad\quad+\tilde{C}\eta\sigma\sqrt{\frac{n\log n}{p}}\left\Vert \bm{X}^{\star}\right\Vert _{2,\infty}\\
 & \quad\leq\left(1-\frac{\sigma_{\min}}{5}\eta\right)C_{3}\rho^{t}\mu r\sqrt{\frac{\log n}{np}}\left\Vert \bm{X}^{\star}\right\Vert _{2,\infty}+C_{7}\frac{\sigma}{\sigma_{\min}}\sqrt{\frac{n\log n}{p}}\left\Vert \bm{X}^{\star}\right\Vert _{2,\infty}
\end{align*}
as long as $C_{7}>0$ is sufficiently large, where we have used the
sample complexity assumption $n^{2}p\gg\kappa^{4}\mu^{2}r^{2} n \log n$
and the step size $0<\eta\leq {1} / ({2\sigma_{\max}}) \leq {1} / ({2\sigma_{\min}})$.
This finishes the proof.

\subsubsection{Proof of Lemma \ref{lemma:loop-remainder-4-MC}} 
By
the unitary invariance of the Frobenius norm, one has 
\[
\left\Vert \bm{B}_{4}^{\left(l\right)}\right\Vert _{\mathrm{F}}=\frac{\eta}{p}\left\Vert \cP_{\Omega_{l}}\left(\bm{E}\right)\bm{X}^{t,\left(l\right)}\right\Vert _{\mathrm{F}},
\]
where all nonzero entries of the matrix $\cP_{\Omega_{l}}\left(\bm{E}\right)$
reside in the $l$th row/column. Decouple the effects of the $l$th
row and the $l$th column of $\cP_{\Omega_{l}}\left(\bm{E}\right)$
to reach 
\begin{equation}
\frac{p}{\eta}\left\Vert \bm{B}_{4}^{\left(l\right)}\right\Vert _{\mathrm{F}}\leq\Bigg\Vert \sum_{j=1}^{n}\underbrace{\delta_{l,j}E_{l,j}\bm{X}_{j,\cdot}^{t,\left(l\right)}}_{:=\bm{u}_{j}}\Bigg\Vert _{2}+\underbrace{\Bigg\Vert \sum_{j:j\neq l}\delta_{l,j}E_{l,j}\bm{X}_{l,\cdot}^{t,\left(l\right)}\Bigg\Vert _{2}}_{:=\alpha},\label{eq:alpha-u-defn}
\end{equation}
where $\delta_{l,j}:=\ind_{\left\{ (l,j)\in\Omega\right\} }$ indicates
whether the $(l,j)$-th entry is observed. Since $\bm{X}^{t,\left(l\right)}$
is independent of $\{\delta_{l,j}\}_{1\leq j \leq n}$ and $\{{E}_{l,j}\}_{1\leq j \leq n}$,
we can treat the first term as a sum of independent vectors $\{\bm{u}_{j}\}$. It is
easy to verify that 
\[
\Big\|\|\bm{u}_{j}\|_{2}\Big\|_{\psi_{1}}\leq\left\Vert \bm{X}^{t,\left(l\right)}\right\Vert _{2,\infty}\left\Vert \delta_{l,j}E_{l,j}\right\Vert _{\psi_{1}}\lesssim\sigma\left\Vert \bm{X}^{t,\left(l\right)}\right\Vert _{2,\infty},
\]
where $\|\cdot\|_{\psi_{1}}$ denotes the sub-exponential norm \cite[Section A.1]{Koltchinskii2011oracle}. Further, one can calculate 
\begin{align*}
V & :=\left\Vert \EE\left[\sum_{j=1}^{n}\left(\delta_{l,j}E_{l,j}\right)^{2}\bm{X}_{j,\cdot}^{t,\left(l\right)}\bm{X}_{j,\cdot}^{t,\left(l\right)\top}\right]\right\Vert \lesssim p\sigma^{2}\left\Vert \EE\left[\sum_{j=1}^{n}\bm{X}_{j,\cdot}^{t,\left(l\right)}\bm{X}_{j,\cdot}^{t,\left(l\right)\top}\right]\right\Vert  =p\sigma^{2}\left\Vert \bm{X}^{t,\left(l\right)}\right\Vert _{\mathrm{F}}^{2}.
\end{align*}
Invoke the matrix Bernstein inequality \cite[Theorem 2.7]{Koltchinskii2011oracle} to discover that with probability at least $1-O\left(n^{-10}\right)$,
\begin{align*}
\Big\Vert \sum\nolimits_{j=1}^{n}\bm{u}_{j}\Big\Vert _{2} & \lesssim\sqrt{V\log n}+\Big\|\|\bm{u}_{j}\|_2\Big\|_{\psi_{1}}\log^2 n\\
 & \lesssim\sqrt{p\sigma^{2}\left\Vert \bm{X}^{t,\left(l\right)}\right\Vert _{\mathrm{F}}^{2}\log n}+\sigma\big\Vert \bm{X}^{t,\left(l\right)}\big\Vert _{2,\infty}\log^2 n\\
 & \lesssim\sigma\sqrt{np\log n}\big\Vert \bm{X}^{t,\left(l\right)}\big\Vert _{2,\infty}+\sigma\big\Vert \bm{X}^{t,\left(l\right)}\big\Vert _{2,\infty}\log^2 n\\
 & \lesssim\sigma\sqrt{np\log n}\big\Vert \bm{X}^{t,\left(l\right)}\big\Vert _{2,\infty},
\end{align*}
where the third inequality follows from $\left\Vert \bm{X}^{t,\left(l\right)}\right\Vert _{\mathrm{F}}^{2}\leq n\left\Vert \bm{X}^{t,\left(l\right)}\right\Vert _{2,\infty}^{2}$,
and the last inequality holds as long as $np\gg\log^2 n$.

Additionally, the remaining term $\alpha$ in (\ref{eq:alpha-u-defn})
can be controlled using the same argument, giving rise to
\[
\alpha\lesssim\sigma\sqrt{np\log n}\big\|\bm{X}^{t,\left(l\right)}\big\|_{2,\infty}.
\]
We then complete the proof by observing that
\begin{equation}
\big\|\bm{X}^{t,\left(l\right)}\big\|_{2,\infty}=\big\|\bm{X}^{t,\left(l\right)}\bm{R}^{t,\left(l\right)}\big\|_{2,\infty}\leq\big\|\bm{X}^{t,\left(l\right)}\bm{R}^{t,\left(l\right)}-\bm{X}^{\star}\big\|_{2,\infty}+\big\|\bm{X}^{\star}\big\|_{2,\infty}\leq2\big\|\bm{X}^{\star}\big\|_{2,\infty}, \label{eq:Xtl-2-inf-norm}
\end{equation}
where the last inequality follows by combining (\ref{eq:implication-finer-2-inf-R-t-l}),
the sample complexity condition $n^{2}p\gg\mu^{2}r^{2}n\log n$, and
the noise condition (\ref{eq:mc-noise-condition}).

\subsubsection{Proof of Lemma \ref{lemma:ell_2_remainder-1-MC}}
For
notational simplicity, we denote  
\begin{equation}\label{eq:def_C}
\bm{C}:=\bm{X}^{t,(l)}\bm{X}^{t,(l)\top}-\bm{M}^{\star} = \bm{X}^{t,(l)}\bm{X}^{t,(l)\top}-\bm{X}^{\star}\bm{X}^{\star\top}.
\end{equation}
Since the Frobenius norm is unitarily invariant, we have 
\[
\left\Vert \bm{B}_{2}^{\left(l\right)}\right\Vert _{\mathrm{F}}=\eta\Bigg\|\underset{:=\bm{W}}{\underbrace{\left[\frac{1}{p}\mathcal{P}_{\Omega_{l}}\left(\bm{C}\right)-\mathcal{P}_{l}\left(\bm{C}\right)\right]}}\bm{X}^{t,(l)}\Bigg\|_{\mathrm{F}}.
\]
Again, all nonzero entries of the matrix $\bm{W}$ reside in its $l$th
row/column. We can deal with the $l$th row and the $l$th column
of $\bm{W}$ separately as follows 
\begin{align*}
\frac{p}{\eta}\left\Vert \bm{B}_{2}^{\left(l\right)}\right\Vert _{\mathrm{F}} & \leq\Bigg\Vert \sum_{j=1}^{n}\left(\delta_{l,j}-p\right)C_{l,j}\bm{X}_{j,\cdot}^{t,\left(l\right)}\Bigg\Vert _{2}+\sqrt{\sum_{j:j\neq l}\left(\delta_{l,j}-p\right)^{2}}\left\Vert \bm{C}\right\Vert _{\infty}\big\Vert \bm{X}_{l,\cdot}^{t,(l)}\big\Vert _{2}\\
 & \lesssim \Bigg\Vert \sum_{j=1}^{n}\left(\delta_{l,j}-p\right)C_{l,j}\bm{X}_{j,\cdot}^{t,\left(l\right)}\Bigg\Vert _{2}+\sqrt{np}\left\Vert \bm{C}\right\Vert _{\infty}\big\Vert \bm{X}_{l,\cdot}^{t,(l)}\big\Vert _{2},
\end{align*}
where $\delta_{l,j}:=\ind_{\left\{ (l,j)\in\Omega\right\} }$ and
the second line relies on the fact that $\sum_{j:j\neq l}\left(\delta_{l,j}-p\right)^{2}\asymp np$.
It follows that 
\begin{align*}
L & :=\max_{1\leq j\leq n}\left\Vert \left(\delta_{l,j}-p\right)C_{l,j}\bm{X}_{j,\cdot}^{t,\left(l\right)}\right\Vert _{2}\leq\left\Vert \bm{C}\right\Vert _{\infty}\left\Vert \bm{X}^{t,\left(l\right)}\right\Vert _{2,\infty} \overset{\left(\text{i}\right)}{\leq}2\left\Vert \bm{C}\right\Vert _{\infty}\left\Vert \bm{X}^{\star}\right\Vert _{2,\infty},\\
V & :=\Bigg\Vert \sum_{j=1}^{n}\EE\big[\left(\delta_{l,j}-p\right)^{2} \big] C_{l,j}^{2}\bm{X}_{j,\cdot}^{t,\left(l\right)}\bm{X}_{j,\cdot}^{t,\left(l\right)\top}\Bigg\Vert \leq p\|\bm{C}\|_{\infty}^{2}\Bigg\Vert \sum_{j=1}^{n}\bm{X}_{j,\cdot}^{t,\left(l\right)}\bm{X}_{j,\cdot}^{t,\left(l\right)\top}\Bigg\Vert \\
 & =p\left\Vert \bm{C}\right\Vert _{\infty}^{2}\left\Vert \bm{X}^{t,\left(l\right)}\right\Vert _{\mathrm{F}}^{2}\overset{\left(\text{ii}\right)}{\leq}4p\left\Vert \bm{C}\right\Vert _{\infty}^{2}\left\Vert \bm{X}^{\star}\right\Vert _{\mathrm{F}}^{2}.
\end{align*}
Here, (i) is a consequence of \eqref{eq:Xtl-2-inf-norm}. In addition, (ii) follows from
\[
\big\Vert \bm{X}^{t,\left(l\right)}\big\Vert _{\mathrm{F}}=\left\Vert \bm{X}^{t,\left(l\right)}\bm{R}^{t,\left(l\right)}\right\Vert _{\mathrm{F}}\leq\left\Vert \bm{X}^{t,\left(l\right)}\bm{R}^{t,\left(l\right)}-\bm{X}^{\star}\right\Vert _{\mathrm{F}}+\left\Vert \bm{X}^{\star}\right\Vert _{\mathrm{F}}\leq2\left\Vert \bm{X}^{\star}\right\Vert _{\mathrm{F}},
\]
where the last inequality comes from (\ref{eq:implication-finer-fro-hat-t-l}),
the sample complexity condition $n^{2}p\gg\mu^{2}r^{2}n\log n$, and
the noise condition (\ref{eq:mc-noise-condition}). The matrix Bernstein
inequality \cite[Theorem 6.1.1]{Tropp:2015:IMC:2802188.2802189} reveals
that 
\begin{align*}
\Bigg\Vert \sum_{j=1}^{n}\left(\delta_{l,j}-p\right)C_{l,j}\bm{X}_{j,\cdot}^{t,\left(l\right)}\Bigg\Vert _{2} & \lesssim\sqrt{V\log n}+L\log n\lesssim\sqrt{p\left\Vert \bm{C}\right\Vert _{\infty}^{2}\left\Vert \bm{X}^{\star}\right\Vert _{\mathrm{F}}^{2}\log n}+\left\Vert \bm{C}\right\Vert _{\infty}\left\Vert \bm{X}^{\star}\right\Vert _{2,\infty}\log n
\end{align*}
with probability exceeding $1-O\left(n^{-10}\right)$, and as a result, 
\begin{align}
\frac{p}{\eta}\big\Vert \bm{B}_{2}^{\left(l\right)}\big\Vert _{\mathrm{F}} \lesssim\sqrt{p\log n}\left\Vert \bm{C}\right\Vert _{\infty}\left\Vert \bm{X}^{\star}\right\Vert _{\mathrm{F}}+\sqrt{np}\left\Vert \bm{C}\right\Vert _{\infty}\left\Vert \bm{X}^{\star}\right\Vert _{2,\infty}\label{eq:nu2-bound-MC}
\end{align}
as soon as $np\gg\log n$.

To finish up, we make the observation that 
\begin{align}
\left\Vert \bm{C}\right\Vert _{\infty} & =\left\Vert \bm{X}^{t,(l)}\bm{R}^{t,\left(l\right)}\big(\bm{X}^{t,(l)}\bm{R}^{t,\left(l\right)}\big)^{\top}-\bm{X}^{\star}\bm{X}^{\star\top}\right\Vert _{\infty}\nonumber \\
 & \leq\left\Vert \big(\bm{X}^{t,(l)}\bm{R}^{t,\left(l\right)}-\bm{X}^{\star}\big)\big(\bm{X}^{t,(l)}\bm{R}^{t,\left(l\right)}\big)^{\top}\right\Vert _{\infty}+\left\Vert \bm{X}^{\star}\big(\bm{X}^{t,(l)}\bm{R}^{t,\left(l\right)}-\bm{X}^{\star}\big)^{\top}-\bm{X}^{\star}\bm{X}^{\star\top}\right\Vert _{\infty}\nonumber \\
 & \leq\left\Vert \bm{X}^{t,(l)}\bm{R}^{t,\left(l\right)}-\bm{X}^{\star}\right\Vert _{2,\infty}\left\Vert \bm{X}^{t,(l)}\bm{R}^{t,\left(l\right)}\right\Vert _{2,\infty}+\left\Vert \bm{X}^{\star}\right\Vert _{2,\infty}\left\Vert \bm{X}^{t,(l)}\bm{R}^{t,\left(l\right)}-\bm{X}^{\star}\right\Vert _{2,\infty}\nonumber \\
 & \leq3\left\Vert \bm{X}^{t,(l)}\bm{R}^{t,\left(l\right)}-\bm{X}^{\star}\right\Vert _{2,\infty}\left\Vert \bm{X}^{\star}\right\Vert _{2,\infty},\label{eq:Xt-M-inf-MC}
\end{align}
where the last line arises from (\ref{eq:Xtl-2-inf-norm}). This combined
with (\ref{eq:nu2-bound-MC}) gives 
\begin{align*}
\left\Vert \bm{B}_{2}^{(l)}\right\Vert _{\mathrm{F}} & \lesssim\eta\sqrt{\frac{\log n}{p}}\left\Vert \bm{C}\right\Vert _{\infty}\left\Vert \bm{X}^{\star}\right\Vert _{\mathrm{F}}+\eta\sqrt{\frac{n}{p}}\left\Vert \bm{C}\right\Vert _{\infty}\left\Vert \bm{X}^{\star}\right\Vert _{2,\infty}\\
 & \overset{(\text{i})}{\lesssim}\eta\sqrt{\frac{\log n}{p}}\left\Vert \bm{X}^{t,(l)}\bm{R}^{t,\left(l\right)}-\bm{X}^{\star}\right\Vert _{2,\infty}\left\Vert \bm{X}^{\star}\right\Vert _{2,\infty}\left\Vert \bm{X}^{\star}\right\Vert _{\mathrm{F}}+\eta\sqrt{\frac{n}{p}}\left\Vert \bm{X}^{t,(l)}\bm{R}^{t,\left(l\right)}-\bm{X}^{\star}\right\Vert _{2,\infty}\left\Vert \bm{X}^{\star}\right\Vert _{2,\infty}^{2}\\
 & \overset{(\text{ii})}{\lesssim}\eta\sqrt{\frac{\log n}{p}}\left\Vert \bm{X}^{t,\left(l\right)}\bm{R}^{t,\left(l\right)}-\bm{X}^{\star}\right\Vert _{2,\infty}\sqrt{\frac{\kappa\mu r^{2}}{n}}\sigma_{\max}+\eta\sqrt{\frac{n}{p}}\left\Vert \bm{X}^{t,\left(l\right)}\bm{R}^{t,\left(l\right)}-\bm{X}^{\star}\right\Vert _{2,\infty}\frac{\kappa\mu r}{n}\sigma_{\max}\\
 & \lesssim\eta\sqrt{\frac{\kappa^{2}\mu^{2}r^{2}\log n}{np}}\left\Vert \bm{X}^{t,\left(l\right)}\bm{R}^{t,\left(l\right)}-\bm{X}^{\star}\right\Vert _{2,\infty}\sigma_{\max},
\end{align*}
where (i) comes from (\ref{eq:Xt-M-inf-MC}), and (ii) makes use of
the incoherence condition (\ref{eq:incoherence-X-MC}).

\subsection{Proof of Lemma \ref{lemma:looe-MC}\label{subsec:Proof-of-Lemma-looe-MC}}

We first introduce an auxiliary matrix 
\begin{equation}
\tilde{\bm{X}}^{t+1,\left(l\right)}:=\bm{X}^{t,\left(l\right)}\hat{\bm{H}}^{t,\left(l\right)}-\eta\left[\frac{1}{p}\cP_{\Omega^{-l}}\left[\bm{X}^{t,\left(l\right)}\bm{X}^{t,\left(l\right)\top}-\left(\bm{M}^{\star}+\bm{E}\right)\right]+\cP_{l}\left(\bm{X}^{t,\left(l\right)}\bm{X}^{t,\left(l\right)\top}-\bm{M}^{\star}\right)\right]\bm{X}^{\star}.\label{eq:looe_aux}
\end{equation}
With this in place, we can use the triangle inequality to obtain 
\begin{equation}
\left\Vert \big(\bm{X}^{t+1,\left(l\right)}\hat{\bm{H}}^{t+1,\left(l\right)}-\bm{X}^{\star}\big)_{l,\cdot}\right\Vert _{2}\leq\underbrace{\left\Vert \big(\bm{X}^{t+1,\left(l\right)}\hat{\bm{H}}^{t+1,\left(l\right)}-\tilde{\bm{X}}^{t+1,\left(l\right)}\big)_{l,\cdot}\right\Vert _{2}}_{:=\alpha_{1}}+\underbrace{\left\Vert \big(\tilde{\bm{X}}^{t+1,\left(l\right)}-\bm{X}^{\star}\big)_{l,\cdot}\right\Vert _{2}}_{:=\alpha_{2}}.\label{eq:defn-alpha1-alpha2}
\end{equation}
In what follows, we bound the two terms $\alpha_{1}$ and $\alpha_{2}$
separately. 
\begin{enumerate}
\item Regarding the second term $\alpha_{2}$ of (\ref{eq:defn-alpha1-alpha2}),
we see from the definition of $\tilde{\bm{X}}^{t+1,\left(l\right)}$
(see~(\ref{eq:looe_aux})) that 
\begin{equation}
\big(\tilde{\bm{X}}^{t+1,\left(l\right)}-\bm{X}^{\star}\big)_{l,\cdot}=\left[\bm{X}^{t,\left(l\right)}\hat{\bm{H}}^{t,\left(l\right)}-\eta\big(\bm{X}^{t,\left(l\right)}\bm{X}^{t,\left(l\right)\top}-\bm{X}^{\star}\bm{X}^{\star\top}\big)\bm{X}^{\star}-\bm{X}^{\star}\right]_{l,\cdot},\label{eq:looe-alpha-2-first-step}
\end{equation}
where we also utilize the definitions of $\cP_{\Omega^{-l}}$ and $\cP_{l}$
in (\ref{eq:projection-Omega-l}).
For notational convenience, we denote 
\begin{align}
\bm{\Delta}^{t,\left(l\right)}:=\bm{X}^{t,\left(l\right)}\hat{\bm{H}}^{t,\left(l\right)}-\bm{X}^{\star}.\label{eq:lemma-looe-delta-defn}
\end{align}
This allows us to rewrite (\ref{eq:looe-alpha-2-first-step}) as 
\begin{align*}
\left(\tilde{\bm{X}}^{t+1,\left(l\right)}-\bm{X}^{\star}\right)_{l,\cdot} & =\bm{\Delta}_{l,\cdot}^{t,\left(l\right)}-\eta\left[\left(\bm{\Delta}^{t,\left(l\right)}\bm{X}^{\star\top}+\bm{X}^{\star}\bm{\Delta}^{t,\left(l\right)\top}\right)\bm{X}^{\star}\right]_{l,\cdot}-\eta\left[\bm{\Delta}^{t,\left(l\right)}\bm{\Delta}^{t,\left(l\right)\top}\bm{X}^{\star}\right]_{l,\cdot}\\
 & =\bm{\Delta}_{l,\cdot}^{t,\left(l\right)}-\eta\bm{\Delta}_{l,\cdot}^{t,\left(l\right)}\bm{\Sigma}^{\star}-\eta\bm{X}_{l,\cdot}^{\star}\bm{\Delta}^{t,\left(l\right)\top}\bm{X}^{\star}-\eta\bm{\Delta}_{l,\cdot}^{t,\left(l\right)}\bm{\Delta}^{t,\left(l\right)\top}\bm{X}^{\star},
\end{align*}
which further implies that 
\begin{align*}
\alpha_{2} & \leq\left\Vert \bm{\Delta}_{l,\cdot}^{t,\left(l\right)}-\eta\bm{\Delta}_{l,\cdot}^{t,\left(l\right)}\bm{\Sigma}^{\star}\right\Vert _{2}+\eta\left\Vert \bm{X}_{l,\cdot}^{\star}\bm{\Delta}^{t,\left(l\right)\top}\bm{X}^{\star}\right\Vert _{2}+\eta\left\Vert \bm{\Delta}_{l,\cdot}^{t,\left(l\right)}\bm{\Delta}^{t,\left(l\right)\top}\bm{X}^{\star}\right\Vert _{2}\\
 & \leq\left\Vert \bm{\Delta}_{l,\cdot}^{t,\left(l\right)}\right\Vert _{2}\left\Vert \bm{I}_{r}-\eta\bm{\Sigma}^{\star}\right\Vert +\eta\left\Vert \bm{X}^{\star}\right\Vert _{2,\infty}\left\Vert \bm{\Delta}^{t,\left(l\right)}\right\Vert \left\Vert \bm{X}^{\star}\right\Vert +\eta\left\Vert \bm{\Delta}_{l,\cdot}^{t,\left(l\right)}\right\Vert _{2}\left\Vert \bm{\Delta}^{t,\left(l\right)}\right\Vert \left\Vert \bm{X}^{\star}\right\Vert \\
 & \leq\left\Vert \bm{\Delta}_{l,\cdot}^{t,\left(l\right)}\right\Vert _{2}\left\Vert \bm{I}_{r}-\eta\bm{\Sigma}^{\star}\right\Vert +2\eta\left\Vert \bm{X}^{\star}\right\Vert _{2,\infty}\left\Vert \bm{\Delta}^{t,\left(l\right)}\right\Vert \left\Vert \bm{X}^{\star}\right\Vert .
\end{align*}
Here, the last line follows from the fact that $\left\Vert \bm{\Delta}_{l,\cdot}^{t,\left(l\right)}\right\Vert _{2} \leq \left\Vert \bm{X}^{\star}\right\Vert _{2,\infty}$. To see this, one can use the induction hypothesis (\ref{eq:induction_loo-ell_infty-MC}) to get 
\begin{equation}
\left\Vert \bm{\Delta}_{l,\cdot}^{t,\left(l\right)}\right\Vert _{2}\leq C_{2}\rho^{t}\mu r\frac{1}{\sqrt{np}}\left\Vert \bm{X}^{\star}\right\Vert _{2,\infty}+C_{6}\frac{\sigma}{\sigma_{\min}}\sqrt{\frac{n\log n}{p}}\left\Vert \bm{X}^{\star}\right\Vert _{2,\infty}\ll\left\Vert \bm{X}^{\star}\right\Vert _{2,\infty}\label{eq:Delta_l-tl-UB}
\end{equation}
as long as $np\gg\mu^{2}r^{2}$ and $\sigma\sqrt{\left(n\log n \right)/ p}\ll\sigma_{\min}$.
By taking $0<\eta\leq1 / {\sigma_{\max}}$, we have $\bm{0}\preceq \bm{I}_{r}-\eta\bm{\Sigma}^{\star} \preceq \left(1-\eta \sigma_{\min} \right)\bm{I}_{r}$, and hence can obtain 
\begin{align}
\alpha_{2} & \leq\left(1-\eta\sigma_{\min}\right)\left\Vert \bm{\Delta}_{l,\cdot}^{t,\left(l\right)}\right\Vert _{2}+2\eta\left\Vert \bm{X}^{\star}\right\Vert _{2,\infty}\left\Vert \bm{\Delta}^{t,\left(l\right)}\right\Vert \left\Vert \bm{X}^{\star}\right\Vert .\label{eq:looe_alpha_2}
\end{align}
An immediate consequence of the above two inequalities and (\ref{eq:implication-finer-op-hat-t-l}) is
\begin{equation}
\alpha_{2}\leq\|\bm{X}^{\star}\|_{2,\infty}.\label{eq:alpha2-UB-tl}
\end{equation}
\item The first term $\alpha_{1}$ of (\ref{eq:defn-alpha1-alpha2}) can
be equivalently written as 
\[
\alpha_{1}=\left\Vert \big(\bm{X}^{t+1,\left(l\right)}\hat{\bm{H}}^{t,\left(l\right)}\bm{R}_{1}-\tilde{\bm{X}}^{t+1,\left(l\right)} \big)_{l,\cdot}\right\Vert _{2},
\]
where 
\begin{align}
\bm{R}_{1} & =\big(\hat{\bm{H}}^{t,\left(l\right)}\big)^{-1}\hat{\bm{H}}^{t+1,\left(l\right)}:=\arg\min_{\bm{R}\in\cO^{r\times r}}\left\Vert \bm{X}^{t+1,\left(l\right)}\hat{\bm{H}}^{t,\left(l\right)}\bm{R}-\bm{X}^{\star}\right\Vert _{\mathrm{F}},\nonumber 
\end{align}
Simple algebra yields 
\begin{align*}
\alpha_{1} & \leq\left\Vert \left(\bm{X}^{t+1,\left(l\right)}\hat{\bm{H}}^{t,\left(l\right)}-\tilde{\bm{X}}^{t+1,\left(l\right)}\right)_{l,\cdot}\bm{R}_{1}\right\Vert _{2}+\left\Vert \tilde{\bm{X}}_{l,\cdot}^{t+1,\left(l\right)}\right\Vert _{2}\left\Vert \bm{R}_{1}- \bm{I}_{r} \right\Vert \\
 & \leq\underbrace{\left\Vert \left(\bm{X}^{t+1,\left(l\right)}\hat{\bm{H}}^{t,\left(l\right)}-\tilde{\bm{X}}^{t+1,\left(l\right)}\right)_{l,\cdot}\right\Vert _{2}}_{:=\beta_{1}}+2\left\Vert \bm{X}^{\star}\right\Vert _{2,\infty}\underbrace{\left\Vert \bm{R}_{1}-\bm{I}_{r} \right\Vert }_{:=\beta_{2}}.
\end{align*}
Here, to bound the the second term we have used 
\[
\left\Vert \tilde{\bm{X}}_{l,\cdot}^{t+1,\left(l\right)}\right\Vert _{2}\leq\left\Vert \tilde{\bm{X}}_{l,\cdot}^{t+1,\left(l\right)}-\bm{X}_{l,\cdot}^{\star}\right\Vert _{2}+\left\Vert \bm{X}_{l,\cdot}^{\star}\right\Vert _{2}=\alpha_{2}+\left\Vert \bm{X}_{l,\cdot}^{\star}\right\Vert _{2}\leq2\left\Vert \bm{X}^{\star}\right\Vert _{2,\infty},
\]
where the last inequality follows from (\ref{eq:alpha2-UB-tl}). It
remains to upper bound $\beta_{1}$ and $\beta_{2}$. For both $\beta_1$ and $\beta_2$, a central quantity to control is $\bm{X}^{t+1,\left(l\right)}\hat{\bm{H}}^{t,\left(l\right)}-\tilde{\bm{X}}^{t+1,\left(l\right)}$.
By the definition of $\tilde{\bm{X}}^{t+1,\left(l\right)}$ in (\ref{eq:looe_aux})
and the gradient update rule for $\bm{X}^{t+1,\left(l\right)}$ (see (\ref{eq:loo-gradient_update-MC})), one
has 
\begin{align}
 &\quad \bm{X}^{t+1,\left(l\right)}\hat{\bm{H}}^{t,\left(l\right)}-\tilde{\bm{X}}^{t+1,\left(l\right)}\nonumber \\
 & =\left\{ \bm{X}^{t,\left(l\right)}\hat{\bm{H}}^{t,\left(l\right)}-\eta\left[\frac{1}{p}\cP_{\Omega^{-l}}\left[\bm{X}^{t,\left(l\right)}\bm{X}^{t,\left(l\right)\top}-\left(\bm{M}^{\star}+\bm{E}\right)\right]+\cP_{l}\left(\bm{X}^{t,\left(l\right)}\bm{X}^{t,\left(l\right)\top}-\bm{M}^{\star}\right)\right]\bm{X}^{t,\left(l\right)}\hat{\bm{H}}^{t,\left(l\right)}\right\} \nonumber \\
 & \quad-\left\{ \bm{X}^{t,\left(l\right)}\hat{\bm{H}}^{t,\left(l\right)}-\eta\left[\frac{1}{p}\cP_{\Omega^{-l}}\left[\bm{X}^{t,\left(l\right)}\bm{X}^{t,\left(l\right)\top}-\left(\bm{M}^{\star}+\bm{E}\right)\right]+\cP_{l}\left(\bm{X}^{t,\left(l\right)}\bm{X}^{t,\left(l\right)\top}-\bm{M}^{\star}\right)\right]\bm{X}^{\star}\right\} \nonumber \\
 & =-\eta\left[\frac{1}{p}\cP_{\Omega^{-l}}\left(\bm{X}^{t,\left(l\right)}\bm{X}^{t,\left(l\right)\top}-\bm{X}^{\star}\bm{X}^{\star\top}\right)+\cP_{l}\left(\bm{X}^{t,\left(l\right)}\bm{X}^{t,\left(l\right)\top}-\bm{X}^{\star}\bm{X}^{\star\top}\right)\right]\bm{\Delta}^{t,\left(l\right)}+\frac{\eta}{p}\cP_{\Omega^{-l}}\left(\bm{E}\right)\bm{\Delta}^{t,\left(l\right)}.\label{eq:looe_prev}
\end{align}
It is easy to verify that 
\[
\frac{1}{p} \left\Vert \cP_{\Omega^{-l}}\left(\bm{E}\right)\right\Vert \overset{\text{(i)}}{\leq} \frac{1}{p} \left\Vert \cP_{\Omega}\left(\bm{E}\right)\right\Vert \overset{\text{(ii)}}{\lesssim}\sigma\sqrt{\frac{n}{p}}\overset{\text{(iii)}}{\leq}\frac{\delta}{2}\sigma_{\min}
\]
for $\delta>0$ sufficiently small. Here, (i) uses the elementary fact that the spectral norm of a submatrix is no more than that of the matrix itself, (ii) arises from Lemma \ref{lemma:noise-spectral} and (iii) is a consequence of the noise condition (\ref{eq:mc-noise-condition}). Therefore, in
order to control (\ref{eq:looe_prev}), we need to upper bound the
following quantity 
\begin{equation}
\gamma:=\left\Vert \frac{1}{p}\cP_{\Omega^{-l}}\left(\bm{X}^{t,\left(l\right)}\bm{X}^{t,\left(l\right)\top}-\bm{X}^{\star}\bm{X}^{\star\top}\right)+\cP_{l}\left(\bm{X}^{t,\left(l\right)}\bm{X}^{t,\left(l\right)\top}-\bm{X}^{\star}\bm{X}^{\star\top}\right)\right\Vert .\label{eq:defn-gamma-MC}
\end{equation}
To this end, we make the observation that 
\begin{align}
\gamma & \leq\underbrace{\left\Vert \frac{1}{p}\cP_{\Omega}\left(\bm{X}^{t,\left(l\right)}\bm{X}^{t,\left(l\right)\top}-\bm{X}^{\star}\bm{X}^{\star\top}\right)\right\Vert }_{:=\gamma_{1}}\nonumber \\
 & \quad+\underbrace{\left\Vert \frac{1}{p}\cP_{\Omega_{l}}\left(\bm{X}^{t,\left(l\right)}\bm{X}^{t,\left(l\right)\top}-\bm{X}^{\star}\bm{X}^{\star\top}\right)-\cP_{l}\left(\bm{X}^{t,\left(l\right)}\bm{X}^{t,\left(l\right)\top}-\bm{X}^{\star}\bm{X}^{\star\top}\right)\right\Vert }_{:=\gamma_{2}},\label{eq:defn-gamma-MC-12}
\end{align}
where $\cP_{\Omega_{l}}$ is defined in (\ref{eq:projection-Omega-l-1}). 
An application of Lemma \ref{lemma:energy-MC} reveals that 
\begin{align*}
\gamma_{1} & \leq2n\left\Vert \bm{X}^{t,\left(l\right)}\bm{R}^{t,\left(l\right)}-\bm{X}^{\star}\right\Vert _{2,\infty}^{2}+4\sqrt{n}\log n\left\Vert \bm{X}^{t,\left(l\right)}\bm{R}^{t,\left(l\right)}-\bm{X}^{\star}\right\Vert _{2,\infty}\left\Vert \bm{X}^{\star}\right\Vert,
\end{align*}
where $\bm{R}^{t,\left(l\right)} \in \cO^{r \times r}$ is defined in (\ref{eq:rotation_r_t_l}). 
Let $\bm{C}=\bm{X}^{t,\left(l\right)}\bm{X}^{t,\left(l\right)\top}-\bm{X}^{\star}\bm{X}^{\star\top}$ as in \eqref{eq:def_C}, and one can bound the other term $\gamma_{2}$ by taking advantage of
the triangle inequality and the symmetry property: 
\begin{align*}
\gamma_{2} & \leq\frac{2}{p}\sqrt{\sum_{j=1}^{n}\left(\delta_{l,j}-p\right)^{2}C_{l,j}^{2}}\overset{\left(\text{i}\right)}{\lesssim}\sqrt{\frac{n}{p}}\left\Vert \bm{C}\right\Vert _{\infty} \overset{\left(\text{ii}\right)}{\lesssim}\sqrt{\frac{n}{p}}\left\Vert \bm{X}^{t,\left(l\right)}\bm{R}^{t,\left(l\right)}-\bm{X}^{\star}\right\Vert _{2,\infty}\left\Vert \bm{X}^{\star}\right\Vert _{2,\infty},
\end{align*}
where (i) comes from the standard Chernoff bound $\sum_{j=1}^{n}\left(\delta_{l,j}-p\right)^{2}\asymp np$,
and in (ii) we utilize the bound established in \eqref{eq:Xt-M-inf-MC}.
The previous two bounds taken collectively give 
\begin{align}
\gamma & \leq2n\left\Vert \bm{X}^{t,\left(l\right)}\bm{R}^{t,\left(l\right)}-\bm{X}^{\star}\right\Vert _{2,\infty}^{2}+4\sqrt{n}\log n\left\Vert \bm{X}^{t,\left(l\right)}\bm{R}^{t,\left(l\right)}-\bm{X}^{\star}\right\Vert _{2,\infty}\left\Vert \bm{X}^{\star}\right\Vert \nonumber \\
 & \quad+\tilde{C}\sqrt{\frac{n}{p}}\left\Vert \bm{X}^{t,\left(l\right)}\bm{R}^{t,\left(l\right)}-\bm{X}^{\star}\right\Vert _{2,\infty}\left\Vert \bm{X}^{\star}\right\Vert _{2,\infty} \leq \frac{\delta}{2}\sigma_{\min}\label{eq:gamma-UB-MC}
\end{align}
for some constant $\tilde{C}>0$ and $\delta>0$ sufficiently small. The last inequality follows from (\ref{eq:implication-finer-2-inf-R-t-l}), the incoherence condition (\ref{eq:incoherence-X-MC}) and our sample size condition. 
In summary, we obtain 
\begin{align}
\left\Vert \bm{X}^{t+1,\left(l\right)}\hat{\bm{H}}^{t,\left(l\right)}-\tilde{\bm{X}}^{t+1,\left(l\right)}\right\Vert  & \leq\eta\left(\gamma+\left\Vert \frac{1}{p}\cP_{\Omega^{-l}}\left(\bm{E}\right)\right\Vert \right)\left\Vert \bm{\Delta}^{t,\left(l\right)}\right\Vert \leq\eta\delta\sigma_{\min}\left\Vert \bm{\Delta}^{t,\left(l\right)}\right\Vert ,\label{eq:looe-1-1-1-1}
\end{align}
for $\delta>0$ sufficiently small. With the estimate (\ref{eq:looe-1-1-1-1}) in place, we can continue
our derivation on $\beta_{1}$ and $\beta_{2}$. 
\begin{enumerate}
\item With regard to $\beta_{1}$, in view of (\ref{eq:looe_prev}) we can obtain
\begin{align}
\beta_{1} & \overset{(\text{i})}{=}\eta\left\Vert \left(\bm{X}^{t,\left(l\right)}\bm{X}^{t,\left(l\right)\top}-\bm{X}^{\star}\bm{X}^{\star\top}\right)_{l,\cdot}\bm{\Delta}^{t,\left(l\right)}\right\Vert _{2}\nonumber \\
 & \leq\eta\left\Vert \left(\bm{X}^{t,\left(l\right)}\bm{X}^{t,\left(l\right)\top}-\bm{X}^{\star}\bm{X}^{\star\top}\right)_{l,\cdot}\right\Vert _{2}\left\Vert \bm{\Delta}^{t,\left(l\right)}\right\Vert \nonumber \\
 & \overset{(\text{ii})}{=}\eta\left\Vert \left[\bm{\Delta}^{t,\left(l\right)}\left(\bm{X}^{t,\left(l\right)}\hat{\bm{H}}^{t,\left(l\right)}\right)^{\top}+\bm{X}^{\star}\bm{\Delta}^{t,\left(l\right)\top}\right]_{l,\cdot}\right\Vert _{2}\left\Vert \bm{\Delta}^{t,\left(l\right)}\right\Vert \nonumber \\
 & \leq\eta\left(\left\Vert \bm{\Delta}_{l,\cdot}^{t,\left(l\right)}\right\Vert _{2}\left\Vert \bm{X}^{t,\left(l\right)}\right\Vert +\left\Vert \bm{X}_{l,\cdot}^{\star}\right\Vert _{2}\left\Vert \bm{\Delta}^{t,\left(l\right)}\right\Vert \right)\left\Vert \bm{\Delta}^{t,\left(l\right)}\right\Vert \nonumber \\
 & \leq\eta\left\Vert \bm{\Delta}_{l,\cdot}^{t,\left(l\right)}\right\Vert _{2}\left\Vert \bm{X}^{t,\left(l\right)}\right\Vert \left\Vert \bm{\Delta}^{t,\left(l\right)}\right\Vert +\eta\left\Vert \bm{X}_{l,\cdot}^{\star}\right\Vert _{2}\left\Vert \bm{\Delta}^{t,\left(l\right)}\right\Vert ^{2},\label{eq:looe-beta_1}
\end{align}
where (i) follows from the definitions of $\cP_{\Omega^{-l}}$ and
$\cP_{l}$ (see (\ref{eq:projection-Omega-l}) and note that all entries in the $l$th row of $\cP_{\Omega^{-l}}(\cdot)$
are identically zero), and the identity (ii) is due to the definition of $\bm{\Delta}^{t,\left(l\right)}$ in (\ref{eq:lemma-looe-delta-defn}). 
\item For $\beta_{2}$, we first claim that
\begin{equation}
\bm{I}_{r}:=\arg\min_{\bm{R}\in\cO^{r\times r}}\left\Vert \tilde{\bm{X}}^{t+1,\left(l\right)}\bm{R}-\bm{X}^{\star}\right\Vert _{\mathrm{F}},\label{eq:loo-rotation-proof}
\end{equation}
whose justification follows similar reasonings as that of \eqref{eq:R2-minimizer}, and is therefore omitted. In particular, it gives rise to the facts that $\bm{X}^{\star\top}\tilde{\bm{X}}^{t+1,(l)}$ is symmetric and
\begin{equation}
\big(\tilde{\bm{X}}^{t+1,\left(l\right)}\big)^{\top}\bm{X}^{\star}\succeq\frac{1}{2} \sigma_{\min}\bm{I}_{r}.\label{eq:looe-2}
\end{equation}

We are now ready to invoke Lemma \ref{lemma:rotation-diff} to bound $\beta_{2}$. We abuse the notation and denote $\bm{C}:=\big(\tilde{\bm{X}}^{t+1,\left(l\right)}\big)^{\top}\bm{X}^{\star}$
and $\bm{E}:=\big(\bm{X}^{t+1,\left(l\right)}\hat{\bm{H}}^{t,\left(l\right)}-\tilde{\bm{X}}^{t+1,\left(l\right)}\big)^{\top}\bm{X}^{\star}$. We have
\[
\left\Vert \bm{E}\right\Vert \leq\frac{1}{2}\sigma_{\min}\leq\sigma_{r}\left(\bm{C}\right).
\]
The first inequality arises from (\ref{eq:looe-1-1-1-1}), namely,
\begin{align*}
\left\Vert \bm{E}\right\Vert  & \leq\left\Vert \bm{X}^{t+1,\left(l\right)}\hat{\bm{H}}^{t,\left(l\right)}-\tilde{\bm{X}}^{t+1,\left(l\right)}\right\Vert \left\Vert \bm{X}^{\star}\right\Vert \leq\eta\delta\sigma_{\min}\left\Vert \bm{\Delta}^{t,\left(l\right)}\right\Vert \left\Vert \bm{X}^{\star}\right\Vert \\
 & \overset{\left(\text{i}\right)}{\leq}\eta\delta\sigma_{\min}\left\Vert \bm{X}^{\star}\right\Vert ^{2}\overset{\left(\text{ii}\right)}{\leq}\frac{1}{2}\sigma_{\min},
\end{align*}
where (i) holds since $\left\Vert \bm{\Delta}^{t,\left(l\right)}\right\Vert \leq\left\Vert \bm{X}^{\star}\right\Vert $
and (ii) holds true for $\delta$ sufficiently small and $\eta\leq 1/{\sigma_{\max}}$.
Invoke Lemma \ref{lemma:rotation-diff} to obtain 
\begin{align}
\beta_{2}=\left\Vert \bm{R}_{1}-\bm{I}_r \right\Vert  & \leq\frac{2}{\sigma_{r-1}\left(\bm{C}\right)+\sigma_{r}\left(\bm{C}\right)}\left\Vert \bm{E}\right\Vert \nonumber \\
 & \leq\frac{2}{\sigma_{\min}}\left\Vert \bm{X}^{t+1,\left(l\right)}\hat{\bm{H}}^{t,\left(l\right)}-\tilde{\bm{X}}^{t+1,\left(l\right)}\right\Vert \left\Vert \bm{X}^{\star}\right\Vert \label{eq:looe-beta-3}\\
 & \leq2\delta\eta\left\Vert \bm{\Delta}^{t,\left(l\right)}\right\Vert \left\Vert \bm{X}^{\star}\right\Vert ,\label{eq:looe-bate_2}
\end{align}
where (\ref{eq:looe-beta-3}) follows since $\sigma_{r-1}\left(\bm{C}\right)\geq\sigma_{r}\left(\bm{C}\right)\geq \sigma_{\min}/2$ from \eqref{eq:looe-2},
and the last line comes from (\ref{eq:looe-1-1-1-1}).

\item Putting the previous bounds (\ref{eq:looe-beta_1}) and (\ref{eq:looe-bate_2})
together yields 
\begin{align}
\alpha_{1} & \leq\eta\left\Vert \bm{\Delta}_{l,\cdot}^{t,\left(l\right)}\right\Vert _{2}\left\Vert \bm{X}^{t,\left(l\right)}\right\Vert \left\Vert \bm{\Delta}^{t,\left(l\right)}\right\Vert +\eta\left\Vert \bm{X}_{l,\cdot}^{\star}\right\Vert _{2}\left\Vert \bm{\Delta}^{t,\left(l\right)}\right\Vert ^{2}+4\delta\eta\left\Vert \bm{X}^{\star}\right\Vert _{2,\infty}\left\Vert \bm{\Delta}^{t,\left(l\right)}\right\Vert \left\Vert \bm{X}^{\star}\right\Vert .\label{eq:looe_alpha_1}
\end{align}

\end{enumerate}
\item Combine (\ref{eq:defn-alpha1-alpha2}), (\ref{eq:looe_alpha_2}) and
(\ref{eq:looe_alpha_1}) to reach
\begin{align*}
 & \left\Vert \left(\bm{X}^{t+1,\left(l\right)}\hat{\bm{H}}^{t+1,\left(l\right)}-\bm{X}^{\star}\right)_{l,\cdot}\right\Vert _{2}
\leq
\left(1-\eta\sigma_{\min}\right)\left\Vert \bm{\Delta}_{l,\cdot}^{t,\left(l\right)}\right\Vert_2 +2\eta\left\Vert \bm{X}^{\star}\right\Vert _{2,\infty}\left\Vert \bm{\Delta}^{t,\left(l\right)}\right\Vert \left\Vert \bm{X}^{\star}\right\Vert \\
 & \quad\quad\qquad+\eta\left\Vert \bm{\Delta}_{l,\cdot}^{t,\left(l\right)}\right\Vert _{2}\left\Vert \bm{X}^{t,\left(l\right)}\right\Vert \left\Vert \bm{\Delta}^{t,\left(l\right)}\right\Vert +\eta\left\Vert \bm{X}_{l,\cdot}^{\star}\right\Vert _{2}\left\Vert \bm{\Delta}^{t,\left(l\right)}\right\Vert ^{2}+4\delta\eta\left\Vert \bm{X}^{\star}\right\Vert _{2,\infty}\left\Vert \bm{\Delta}^{t,\left(l\right)}\right\Vert \left\Vert \bm{X}^{\star}\right\Vert \\
 & \quad\overset{\left(\text{i}\right)}{\leq}\left(1-\eta\sigma_{\min}+\eta\left\Vert \bm{X}^{t,\left(l\right)}\right\Vert \left\Vert \bm{\Delta}^{t,\left(l\right)}\right\Vert \right)\left\Vert \bm{\Delta}_{l,\cdot}^{t,\left(l\right)}\right\Vert_2 +4\eta\left\Vert \bm{X}^{\star}\right\Vert _{2,\infty}\left\Vert \bm{\Delta}^{t,\left(l\right)}\right\Vert \left\Vert \bm{X}^{\star}\right\Vert \\
 & \quad\overset{\left(\text{ii}\right)}{\leq}\left(1-\frac{\sigma_{\min}}{2}\eta\right)\left(C_{2}\rho^{t}\mu r\frac{1}{\sqrt{np}}+\frac{C_{6}}{\sigma_{\min}}\sigma\sqrt{\frac{n\log n}{p}}\right)\left\Vert \bm{X}^{\star}\right\Vert _{2,\infty}\\
 & \quad\quad+4\eta\left\Vert \bm{X}^{\star}\right\Vert \left\Vert \bm{X}^{\star}\right\Vert _{2,\infty}\left(2C_{9}\rho^{t}\mu r\frac{1}{\sqrt{np}}\left\Vert \bm{X}^{\star}\right\Vert +\frac{2C_{10}}{\sigma_{\min}}\sigma\sqrt{\frac{n}{p}}\left\Vert \bm{X}^{\star}\right\Vert \right)\\
 & \quad\overset{\left(\text{iii}\right)}{\leq}C_{2}\rho^{t+1}\mu r\frac{1}{\sqrt{np}}\left\Vert \bm{X}^{\star}\right\Vert _{2,\infty}+\frac{C_{6}}{\sigma_{\min}}\sigma\sqrt{\frac{n\log n}{p}}\left\Vert \bm{X}^{\star}\right\Vert _{2,\infty}.
\end{align*}
Here, (i) follows since $\left\Vert \bm{\Delta}^{t,\left(l\right)}\right\Vert \leq\left\Vert \bm{X}^{\star}\right\Vert $
and $\delta$ is sufficiently small, (ii) invokes the hypotheses (\ref{eq:induction_loo-ell_infty-MC})
and (\ref{eq:implication-finer-op-hat-t-l}) and recognizes that 
\[
\left\Vert \bm{X}^{t,\left(l\right)}\right\Vert \left\Vert \bm{\Delta}^{t,\left(l\right)}\right\Vert \leq2\left\Vert \bm{X}^{\star}\right\Vert \left(2C_{9}\mu r\frac{1}{\sqrt{np}}\left\Vert \bm{X}^{\star}\right\Vert +\frac{2C_{10}}{\sigma_{\min}}\sigma\sqrt{\frac{n\log n}{np}}\left\Vert \bm{X}^{\star}\right\Vert \right)\leq\frac{\sigma_{\min}}{2}
\]
holds under the sample size and noise condition, while $\left(\text{iii}\right)$
is valid as long as $1- \left({\sigma_{\min}} / {3}\right)\cdot \eta \leq \rho < 1$, $C_{2}\gg\kappa C_{9}$ and $C_{6}\gg\kappa C_{10} / \sqrt{\log n}$. 
\end{enumerate}

\subsection{Proof of Lemma \ref{lemma:spectral-MC}\label{subsec:Proof-of-Lemma-spectral-MC}}

For notational convenience, we define the following two orthonormal matrices
\[
\bm{Q}:=\arg\min_{\bm{R}\in\cO^{r\times r}}\left\Vert \bm{U}^{0}\bm{R}-\bm{U}^{\star}\right\Vert _{\mathrm{F}}\qquad\text{and}\qquad\bm{Q}^{\left(l\right)}:=\arg\min_{\bm{R}\in\cO^{r\times r}}\big\|\bm{U}^{0,\left(l\right)}\bm{R}-\bm{U}^{\star}\big\|_{\mathrm{F}}.
\]
The problem of finding $\hat{\bm{H}}^{t}$ (see (\ref{eq:rotation-hat-h-MC}))
is called the \emph{orthogonal Procrustes problem} \cite{MR0501589}.
It is well-known that the minimizer $\hat{\bm{H}}^{t}$ always exists
and is given by 
\[
\hat{\bm{H}}^{t}=\mathrm{sgn}\left(\bm{X}^{t\top}\bm{X}^{\star}\right).
\]
Here, the sign matrix $\mathrm{sgn}(\bm{B})$ is defined as 
\begin{equation}
\mathrm{sgn}(\bm{B}):=\bm{U}\bm{V}^{\top}\label{eq:sign-matrix-defn}
\end{equation}
for any matrix $\bm{B}$ with singular value decomposition $\bm{B}=\bm{U}\bm{\Sigma}\bm{V}^{\top}$, where the columns of $\bm{U}$ and $\bm{V}$ are left and right singular vectors, respectively.

Before proceeding, we make note of the following  perturbation
bounds on $\bm{M}^{0}$ and $\bm{M}^{(l)}$ (as defined in Algorithm
\ref{alg:gd-mc} and Algorithm \ref{alg:leave-one-out-gd-MC}, respectively):
\begin{align}
\left\Vert \bm{M}^{0}-\bm{M}^{\star}\right\Vert  & \overset{\text{(i)}}{\leq}\left\Vert \frac{1}{p}\cP_{\Omega}\left(\bm{M}^{\star}\right)-\bm{M}^{\star}\right\Vert +\left\Vert \frac{1}{p}\cP_{\Omega}\left(\bm{E}\right)\right\Vert \nonumber\\
&\overset{\text{(ii)}}{\leq} C\sqrt{\frac{n}{p}}\left\Vert \bm{M}^{\star}\right\Vert_{\infty} +C\sigma\sqrt{\frac{n}{p}} =C\sqrt{\frac{n}{p}}\left\Vert \bm{X}^{\star}\right\Vert_{2,\infty}^{2} +C\frac{\sigma}{\sqrt{\sigma_{\min}}}\sqrt{\frac{n}{p}}\sqrt{\sigma_{\min}}\nonumber \\
&\overset{\text{(iii)}}{\leq} C\left\{ \mu r\sqrt{\frac{1}{np}}\sqrt{\sigma_{\max}}+\frac{\sigma}{\sqrt{\sigma_{\min}}}\sqrt{\frac{n}{p}}\right\} \left\Vert \bm{X}^{\star}\right\Vert \overset{\text{(iv)}}{\ll}\sigma_{\min}, \label{eq:M0-spectral-error}
\end{align}
for some universal constant $C >0$. 
Here, (i) arises from the triangle inequality, (ii) utilizes Lemma \ref{lemma:montanari} and 
Lemma \ref{lemma:noise-spectral}, (iii) follows from the incoherence condition (\ref{eq:incoherence-X-MC}) and (iv) holds under our sample complexity assumption that $n^2 p \gg \mu^2 r^2 n$ and the noise condition (\ref{eq:mc-noise-condition}).
Similarly, we have
\begin{align}
\left\Vert \bm{M}^{(l)}-\bm{M}^{\star}\right\Vert  & \lesssim\left\{ \mu r\sqrt{\frac{1}{np}}\sqrt{\sigma_{\max}}+\frac{\sigma}{\sqrt{\sigma_{\min}}}\sqrt{\frac{n}{p}}\right\} \left\Vert \bm{X}^{\star}\right\Vert \ll\sigma_{\min}.\label{eq:Ml-spectral-error}
\end{align}
Combine Weyl's inequality, (\ref{eq:M0-spectral-error}) and (\ref{eq:Ml-spectral-error}) to obtain
\begin{align}
	\left\Vert\bm{\Sigma}^{0} - \bm{\Sigma}^{\star}\right\Vert \leq \left\Vert \bm{M}^{0}-\bm{M}^{\star}\right\Vert \ll \sigma_{\min}\quad\text{and}\quad\left\Vert\bm{\Sigma}^{(l)} - \bm{\Sigma}^{\star}\right\Vert \leq \left\Vert \bm{M}^{(l)}-\bm{M}^{\star}\right\Vert \ll \sigma_{\min},
\end{align}
which further implies
\begin{align}
	\frac{1}{2}\sigma_{\min} \leq \sigma_{r}\left(\bm{\Sigma}^{0} \right) \leq \sigma_{1}\left(\bm{\Sigma}^{0} \right) \leq 2\sigma_{\max}\quad\text{and}\quad\frac{1}{2}\sigma_{\min} \leq \sigma_{r}\left(\bm{\Sigma}^{(l)} \right) \leq \sigma_{1}\left(\bm{\Sigma}^{(l)} \right) \leq 2\sigma_{\max}.\label{eq:init-sigma-spectrum}
\end{align}

We start by proving (\ref{eq:induction_original_ell_2-MC}), (\ref{eq:induction_original_ell_infty-MC})
and (\ref{eq:induction_original_operator-MC}). The key decomposition
we need is the following
\begin{equation}
\bm{X}^{0}\hat{\bm{H}}^{0}-\bm{X}^{\star}=\bm{U}^{0}\left(\bm{\Sigma}^{0}\right)^{1/2}\big(\hat{\bm{H}}^{0}-\bm{Q}\big)+\bm{U}^{0}\left[\left(\bm{\Sigma}^{0}\right)^{1/2}\bm{Q}-\bm{Q}\left(\bm{\Sigma}^{\star}\right)^{1/2}\right]+\left(\bm{U}^{0}\bm{Q}-\bm{U}^{\star}\right)\left(\bm{\Sigma}^{\star}\right)^{1/2}.\label{eq:spectral-MC-decomposation}
\end{equation}

\begin{enumerate}
\item For the spectral norm error bound in (\ref{eq:induction_original_operator-MC}),
the triangle inequality together with (\ref{eq:spectral-MC-decomposation}) yields 
\[
\left\Vert \bm{X}^{0}\hat{\bm{H}}^{0}-\bm{X}^{\star}\right\Vert \leq \left\|\left(\bm{\Sigma}^{0}\right)^{1/2}\right\| \left\|\hat{\bm{H}}^{0}-\bm{Q}\right\|+\left\Vert \left(\bm{\Sigma}^{0}\right)^{1/2}\bm{Q}-\bm{Q}\left(\bm{\Sigma}^{\star}\right)^{1/2}\right\Vert +\sqrt{\sigma_{\max}}\left\Vert \bm{U}^{0}\bm{Q}-\bm{U}^{\star}\right\Vert,
\]
where we have also used the fact that $\|\bm{U}^{0}\|=1$. 
Recognizing that $\left\Vert \bm{M}^{0}-\bm{M}^{\star}\right\Vert \ll\sigma_{\min}$
(see~(\ref{eq:M0-spectral-error})) and the assumption $\sigma_{\max}/\sigma_{\min}\lesssim1$, we can apply Lemma \ref{lemma:rotation-U-R-diff-MC},
Lemma \ref{lemma:interchange-MC} and Lemma \ref{lemma:davis-kahan-operator-norm}
to obtain
\begin{subequations} \label{eq:properties-init-MC}
\begin{equation}
\big\|\hat{\bm{H}}^{0}-\bm{Q}\big\|\lesssim\frac{1}{\sigma_{\min}}\left\Vert \bm{M}^{0}-\bm{M}^{\star}\right\Vert ,\label{eq:spectral-1-1}
\end{equation}
\begin{equation}
\left\Vert \left(\bm{\Sigma}^{0}\right)^{1/2}\bm{Q}-\bm{Q}\left(\bm{\Sigma}^{\star}\right)^{1/2}\right\Vert \lesssim\frac{1}{\sqrt{\sigma_{\min}}}\left\Vert \bm{M}^{0}-\bm{M}^{\star}\right\Vert ,\label{eq:spectral-1-2}
\end{equation}
\begin{equation}
\left\Vert \bm{U}^{0}\bm{Q}-\bm{U}^{\star}\right\Vert \lesssim\frac{1}{\sigma_{\min}}\left\Vert \bm{M}^{0}-\bm{M}^{\star}\right\Vert .
\end{equation}
\end{subequations}
 These taken collectively imply the advertised upper bound
\begin{align*}
\big\|\bm{X}^{0}\hat{\bm{H}}^{0}-\bm{X}^{\star}\big\| & \lesssim\sqrt{\sigma_{\max}}\frac{1}{\sigma_{\min}}\left\Vert \bm{M}^{0}-\bm{M}^{\star}\right\Vert +\frac{1}{\sqrt{\sigma_{\min}}}\left\Vert \bm{M}^{0}-\bm{M}^{\star}\right\Vert \lesssim\frac{1}{\sqrt{\sigma_{\min}}}\left\Vert \bm{M}^{0}-\bm{M}^{\star}\right\Vert \\
 & \lesssim\left\{ \mu r\sqrt{\frac{1}{np}}\sqrt{\frac{\sigma_{\max}}{\sigma_{\min}}}+\frac{\sigma}{{\sigma_{\min}}}\sqrt{\frac{n}{p}}\right\} \left\Vert \bm{X}^{\star}\right\Vert,
\end{align*}
where we also utilize the fact that $\big\|\left(\bm{\Sigma}^{0}\right)^{1/2}\big\|\leq\sqrt{2\sigma_{\max}}$ (see~(\ref{eq:init-sigma-spectrum})) and the bounded condition number assumption, i.e.~$\sigma_{\max}/\sigma_{\min}\lesssim1$. This finishes the proof of (\ref{eq:induction_original_operator-MC}). 

\item With regard to the Frobenius norm bound in (\ref{eq:induction_original_ell_2-MC}), one has
\begin{align*}
	\left\Vert \bm{X}^{0}\hat{\bm{H}}^{0}-\bm{X}^{\star}\right\Vert _{\mathrm{F}}& \leq\sqrt{r}\big\|\bm{X}^{0}\hat{\bm{H}}^{0}-\bm{X}^{\star}\big\|\\
	& \overset{\text{(i)}}{\lesssim}\left\{ \mu r\sqrt{\frac{1}{np}}+\frac{\sigma}{{\sigma_{\min}}}\sqrt{\frac{n}{p}}\right\}\sqrt{r} \left\Vert \bm{X}^{\star}\right\Vert =\left\{ \mu r\sqrt{\frac{1}{np}}+\frac{\sigma}{{\sigma_{\min}}}\sqrt{\frac{n}{p}}\right\}\sqrt{r} \frac{\sqrt{\sigma_{\max}}}{\sqrt{\sigma_{\min}}}\sqrt{\sigma_{\min}} \\
	& \overset{\text{(ii)}}{\lesssim} \left\{ \mu r\sqrt{\frac{1}{np}}+\frac{\sigma}{{\sigma_{\min}}}\sqrt{\frac{n}{p}}\right\}\sqrt{r} \left\Vert \bm{X}^{\star}\right\Vert_{\mathrm{F}}.
\end{align*}
Here (i) arises from (\ref{eq:induction_original_operator-MC}) and (ii) holds true since $\sigma_{\max}/\sigma_{\min}\asymp1$ and $\sqrt{r}\sqrt{\sigma_{\min}}\leq\left\Vert \bm{X}^{\star}\right\Vert_{\mathrm{F}}$, thus completing the proof of (\ref{eq:induction_original_ell_2-MC}).
\item The proof of (\ref{eq:induction_original_ell_infty-MC}) follows from
similar arguments as used in proving (\ref{eq:induction_original_operator-MC}). Combine (\ref{eq:spectral-MC-decomposation})
and the triangle inequality to reach 
\begin{align*}
\left\Vert \bm{X}^{0}\hat{\bm{H}}^{0}-\bm{X}^{\star}\right\Vert _{2,\infty} & \leq\left\Vert \bm{U}^{0}\right\Vert _{2,\infty}\left\{ \left\Vert \left(\bm{\Sigma}^{0}\right)^{1/2}\right\Vert \left\Vert \hat{\bm{H}}^{0}-\bm{Q}\right\Vert +\left\Vert \left(\bm{\Sigma}^{0}\right)^{1/2}\bm{Q}-\bm{Q}\left(\bm{\Sigma}^{\star}\right)^{1/2}\right\Vert \right\} \\
 & \qquad+\sqrt{\sigma_{\max}}\left\Vert \bm{U}^{0}\bm{Q}-\bm{U}^{\star}\right\Vert _{2,\infty}.
\end{align*}
Plugging in the estimates (\ref{eq:M0-spectral-error}), (\ref{eq:init-sigma-spectrum}), (\ref{eq:spectral-1-1}) and (\ref{eq:spectral-1-2}) results in 
\[
\left\Vert \bm{X}^{0}\hat{\bm{H}}^{0}-\bm{X}^{\star}\right\Vert _{2,\infty}\lesssim\left\{ \mu r\sqrt{\frac{1}{np}}+\frac{\sigma}{\sigma_{\min}}\sqrt{\frac{n}{p}}\right\} \left\Vert \bm{X}^{\star}\right\Vert \left\Vert \bm{U}^{0}\right\Vert _{2,\infty}+\sqrt{\sigma_{\max}}\left\Vert \bm{U}^{0}\bm{Q}-\bm{U}^{\star}\right\Vert _{2,\infty}.
\]
It remains to study the component-wise error of $\bm{U}^{0}$. To
this end, it has already been shown in \cite[Lemma 14]{abbe2017entrywise}
that 
\begin{equation}
\left\Vert \bm{U}^{0}\bm{Q}-\bm{U}^{\star}\right\Vert _{2,\infty}\lesssim\left(\mu r\sqrt{\frac{1}{np}}+\frac{\sigma}{\sigma_{\min}}\sqrt{\frac{n}{p}}\right)\left\Vert \bm{U}^{\star}\right\Vert _{2,\infty}\quad\text{and}\quad\left\Vert \bm{U}^{0}\right\Vert _{2,\infty}\lesssim\left\Vert \bm{U}^{\star}\right\Vert _{2,\infty}
\end{equation}
under our assumptions. These combined with the previous inequality give
\[
\left\Vert \bm{X}^{0}\hat{\bm{H}}^{0}-\bm{X}^{\star}\right\Vert _{2,\infty}\lesssim\left\{ \mu r\sqrt{\frac{1}{np}}+\frac{\sigma}{\sigma_{\min}}\sqrt{\frac{n}{p}}\right\} \sqrt{\sigma_{\max}}\left\Vert \bm{U}^{\star}\right\Vert _{2,\infty}\lesssim\left\{ \mu r\sqrt{\frac{1}{np}}+\frac{\sigma}{\sigma_{\min}}\sqrt{\frac{n}{p}}\right\} \left\Vert \bm{X}^{\star}\right\Vert _{2,\infty},
\]
where the last relation is due to the observation that
\[
 \sqrt{\sigma_{\max}}\left\Vert \bm{U}^{\star}\right\Vert _{2,\infty} \lesssim \sqrt{\sigma_{\min}}\left\Vert \bm{U}^{\star}\right\Vert _{2,\infty} \leq\left\Vert \bm{X}^{\star}\right\Vert _{2,\infty}.
\]
\item We now move on to proving (\ref{eq:induction_loo-ell_infty-MC}).
Recall that $\bm{Q}^{(l)}=\arg\min_{\bm{R}\in\cO^{r\times r}}\left\Vert \bm{U}^{0,\left(l\right)}\bm{R}-\bm{U}^{\star}\right\Vert _{\mathrm{F}}$.
By the triangle inequality, 
\begin{align}
\left\Vert \big(\bm{X}^{0,\left(l\right)}\hat{\bm{H}}^{0,\left(l\right)}-\bm{X}^{\star}\big)_{l,\cdot}\right\Vert _{2} & \leq\left\Vert \bm{X}_{l,\cdot}^{0,\left(l\right)}\big(\hat{\bm{H}}^{0,\left(l\right)}-\bm{Q}^{(l)}\big)\right\Vert _{2}+\left\Vert \big(\bm{X}^{0,\left(l\right)}\bm{Q}^{(l)}-\bm{X}^{\star}\big)_{l,\cdot}\right\Vert _{2}\nonumber \\
 & \leq\left\Vert \bm{X}_{l,\cdot}^{0,\left(l\right)}\right\Vert _{2}\big\|\hat{\bm{H}}^{0,\left(l\right)}-\bm{Q}^{(l)}\big\|+\left\Vert \big(\bm{X}^{0,\left(l\right)}\bm{Q}^{(l)}-\bm{X}^{\star}\big)_{l,\cdot}\right\Vert _{2}.\label{eq:spectral-loo-triangle}
\end{align}
Note that $\bm{X}_{l,\cdot}^{\star}=\bm{M}_{l,\cdot}^{\star}\bm{U}^{\star}\left(\bm{\Sigma}^{\star}\right)^{-1/2}$
and, by construction of $\bm{M}^{(l)}$, 
\[
\bm{X}_{l,\cdot}^{0,\left(l\right)}=\bm{M}_{l,\cdot}^{\left(l\right)}\bm{U}^{0,\left(l\right)}\big(\bm{\Sigma}^{\left(l\right)}\big)^{-1/2}=\bm{M}_{l,\cdot}^{\star}\bm{U}^{0,\left(l\right)}\big(\bm{\Sigma}^{\left(l\right)}\big)^{-1/2}.
\]
We can thus decompose
\begin{align*}
\left(\bm{X}^{0,\left(l\right)}\bm{Q}^{(l)}-\bm{X}^{\star}\right)_{l,\cdot} & =\bm{M}_{l,\cdot}^{\star}\left\{ \bm{U}^{0,\left(l\right)}\left[\big(\bm{\Sigma}^{\left(l\right)}\big)^{-1/2}\bm{Q}^{(l)}-\bm{Q}^{(l)}\left(\bm{\Sigma}^{\star}\right)^{-1/2}\right]+\left(\bm{U}^{0,\left(l\right)}\bm{Q}^{(l)}-\bm{U}^{\star}\right)\left(\bm{\Sigma}^{\star}\right)^{-1/2}\right\} ,
\end{align*}
which further implies  that
\begin{equation}\label{eq:bound_loo_MC_lrow}
\left\Vert \big(\bm{X}^{0,\left(l\right)}\bm{Q}^{(l)}-\bm{X}^{\star}\big)_{l,\cdot}\right\Vert _{2}\leq\left\Vert \bm{M}^{\star}\right\Vert _{2,\infty}\left\{ \left\Vert \big(\bm{\Sigma}^{\left(l\right)}\big)^{-1/2}\bm{Q}^{(l)}-\bm{Q}^{(l)}\left(\bm{\Sigma}^{\star}\right)^{-1/2}\right\Vert +\frac{1}{\sqrt{\sigma_{\min}}}\left\Vert \bm{U}^{0,\left(l\right)}\bm{Q}^{(l)}-\bm{U}^{\star}\right\Vert \right\} .
\end{equation}
In order to control this, we first see that 
\begin{align*}
\left\Vert \big(\bm{\Sigma}^{\left(l\right)}\big)^{-1/2}\bm{Q}^{(l)}-\bm{Q}^{(l)}\left(\bm{\Sigma}^{\star}\right)^{-1/2}\right\Vert  & =\left\Vert \big(\bm{\Sigma}^{\left(l\right)}\big)^{-1/2}\left[\bm{Q}^{(l)}\left(\bm{\Sigma}^{\star}\right)^{1/2}-\big(\bm{\Sigma}^{\left(l\right)}\big)^{1/2}\bm{Q}^{(l)}\right]\left(\bm{\Sigma}^{\star}\right)^{-1/2}\right\Vert \\
 & \lesssim\frac{1}{\sigma_{\min}}\left\Vert \bm{Q}^{(l)}\left(\bm{\Sigma}^{\star}\right)^{1/2}-\big(\bm{\Sigma}^{\left(l\right)}\big)^{-1/2}\bm{Q}^{(l)}\right\Vert \\
 & \lesssim\frac{1}{\sigma_{\min}^{3/2}}\left\Vert \bm{M}^{\left(l\right)}-\bm{M}^{\star}\right\Vert ,
\end{align*}
where the penultimate inequality uses (\ref{eq:init-sigma-spectrum}) and the last inequality arises from Lemma \ref{lemma:interchange-MC}. Additionally, Lemma \ref{lemma:davis-kahan-operator-norm}
gives
\[
\left\Vert \bm{U}^{0,\left(l\right)}\bm{Q}^{(l)}-\bm{U}^{\star}\right\Vert \lesssim\frac{1}{\sigma_{\min}}\left\Vert \bm{M}^{\left(l\right)}-\bm{M}^{\star}\right\Vert .
\]
Plugging the previous two bounds into \eqref{eq:bound_loo_MC_lrow},
we reach 
\begin{align*}
\left\Vert \big(\bm{X}^{0,\left(l\right)}\bm{Q}^{(l)}-\bm{X}^{\star}\big)_{l,\cdot}\right\Vert _{2} & \lesssim\frac{1}{\sigma_{\min}^{3/2}}\left\Vert \bm{M}^{\left(l\right)}-\bm{M}^{\star}\right\Vert \left\Vert \bm{M}^{\star}\right\Vert _{2,\infty} \lesssim\left\{ \mu r\sqrt{\frac{1}{np}}+\frac{\sigma}{\sigma_{\min}}\sqrt{\frac{n}{p}}\right\} \left\Vert \bm{X}^{\star}\right\Vert _{2,\infty}.
\end{align*}
where the last relation follows from $\left\Vert \bm{M}^{\star}\right\Vert _{2,\infty}=\left\Vert \bm{X}^{\star}\bm{X}^{\star\top}\right\Vert _{2,\infty}\leq\sqrt{\sigma_{\max}}\left\Vert \bm{X}^{\star}\right\Vert _{2,\infty}$ and the estimate (\ref{eq:Ml-spectral-error}). Note that this also implies that $\left\Vert \bm{X}_{l,\cdot}^{0,\left(l\right)}\right\Vert _{2}\leq2\left\Vert \bm{X}^{\star}\right\Vert _{2,\infty}$. To see this, one has by the unitary invariance of $\left\Vert\left(\cdot\right)_{l,\cdot}\right\Vert_{2}$,
\[
	\left\Vert\bm{X}_{l,\cdot}^{0,\left(l\right)}\right\Vert_{2} = \left\Vert\bm{X}_{l,\cdot}^{0,\left(l\right)}\bm{Q}^{(l)}\right\Vert_{2} \leq  \left\Vert \big(\bm{X}^{0,\left(l\right)}\bm{Q}^{(l)}-\bm{X}^{\star}\big)_{l,\cdot}\right\Vert _{2} + \left\Vert\bm{X}^{\star}_{l,\cdot}\right\Vert_{2} \leq 2\left\Vert \bm{X}^{\star}\right\Vert _{2,\infty}. 
\]
Substituting the above bounds back to (\ref{eq:spectral-loo-triangle}) yields
in 
\begin{align*}
\left\Vert \big(\bm{X}^{0,\left(l\right)}\hat{\bm{H}}^{0,\left(l\right)}-\bm{X}^{\star}\big)_{l,\cdot}\right\Vert _{2} & \lesssim\left\Vert \bm{X}^{\star}\right\Vert _{2,\infty}\left\Vert \hat{\bm{H}}^{0,\left(l\right)}-\bm{Q}^{(l)}\right\Vert +\left\{ \mu r\sqrt{\frac{1}{np}}+\frac{\sigma}{\sigma_{\min}}\sqrt{\frac{n}{p}}\right\} \left\Vert \bm{X}^{\star}\right\Vert _{2,\infty}\\
 & \lesssim\left\{ \mu r\sqrt{\frac{1}{np}}+\frac{\sigma}{\sigma_{\min}}\sqrt{\frac{n}{p}}\right\} \left\Vert \bm{X}^{\star}\right\Vert _{2,\infty},
\end{align*}
where the second line relies on Lemma \ref{lemma:rotation-U-R-diff-MC}, the bound \eqref{eq:Ml-spectral-error}, and the condition $\sigma_{\max}/\sigma_{\min}\asymp 1$.
This establishes (\ref{eq:induction_loo-ell_infty-MC}). 
\item Our final step is to justify (\ref{eq:induction_ell_2_diff-MC}).
Define $\bm{B}:=\arg\min_{\bm{R}\in\cO^{r\times r}}\left\Vert \bm{U}^{0,\left(l\right)}\bm{R}-\bm{U}^{0}\right\Vert _{\mathrm{F}}$.
From the definition of $\bm{R}^{0,\left(l\right)}$ (cf.~\eqref{eq:rotation_r_t_l}), one has 
\[
\left\Vert \bm{X}^{0}\hat{\bm{H}}^{0}-\bm{X}^{0,\left(l\right)}\bm{R}^{0,\left(l\right)}\right\Vert _{\mathrm{F}}\leq\left\Vert \bm{X}^{0,\left(l\right)}\bm{B}-\bm{X}^{0}\right\Vert _{\mathrm{F}}.
\]
Recognizing that
\[
\bm{X}^{0,\left(l\right)}\bm{B}-\bm{X}^{0}=\bm{U}^{0,\left(l\right)}\left[\big(\bm{\Sigma}^{\left(l\right)}\big)^{1/2}\bm{B}-\bm{B}\left(\bm{\Sigma}^{0}\right)^{1/2}\right]+\left(\bm{U}^{0,\left(l\right)}\bm{B}-\bm{U}^{0}\right)\left(\bm{\Sigma}^{0}\right)^{1/2},
\]
we can use the triangle inequality to bound
\[
\left\Vert \bm{X}^{0,\left(l\right)}\bm{B}-\bm{X}^{0}\right\Vert _{\mathrm{F}}\leq\left\Vert \big(\bm{\Sigma}^{\left(l\right)}\big)^{1/2}\bm{B}-\bm{B}\left(\bm{\Sigma}^{0}\right)^{1/2}\right\Vert _{\mathrm{F}}+\left\Vert \bm{U}^{0,\left(l\right)}\bm{B}-\bm{U}^{0}\right\Vert _{\mathrm{F}}\left\Vert \left(\bm{\Sigma}^{0}\right)^{1/2}\right\Vert .
\]
In view of Lemma \ref{lemma:interchange-MC} and the bounds (\ref{eq:M0-spectral-error}) and (\ref{eq:Ml-spectral-error}), one has
\[
\left\Vert \big(\bm{\Sigma}^{\left(l\right)}\big)^{-1/2}\bm{B}-\bm{B}\bm{\Sigma}^{1/2}\right\Vert _{\mathrm{F}}\lesssim\frac{1}{\sqrt{\sigma_{\min}}}\left\Vert \big(\bm{M}^{0}-\bm{M}^{\left(l\right)}\big)\bm{U}^{0,\left(l\right)}\right\Vert _{\mathrm{F}}.
\]
From Davis-Kahan's sin$\Theta$ theorem \cite{davis1970rotation} we see that
\[
\left\Vert \bm{U}^{0,\left(l\right)}\bm{B}-\bm{U}^{0}\right\Vert _{\mathrm{F}}\lesssim\frac{1}{\sigma_{\min}}\left\Vert \big(\bm{M}^{0}-\bm{M}^{\left(l\right)}\big)\bm{U}^{0,\left(l\right)}\right\Vert _{\mathrm{F}}.
\]
These estimates taken together with (\ref{eq:init-sigma-spectrum}) give 
\[
\left\Vert \bm{X}^{0,\left(l\right)}\bm{B}-\bm{X}^{0}\right\Vert _{\mathrm{F}}\lesssim\frac{1}{\sqrt{\sigma_{\min}}}\left\Vert \big(\bm{M}^{0}-\bm{M}^{\left(l\right)}\big)\bm{U}^{0,\left(l\right)}\right\Vert _{\mathrm{F}}.
\]
It then boils down to controlling $\left\Vert \left(\bm{M}^{0}-\bm{M}^{\left(l\right)}\right)\bm{U}^{0,\left(l\right)}\right\Vert _{\mathrm{F}}$.
Quantities of this type have showed up multiple times already,
and hence we omit the proof details for conciseness (see Appendix \ref{subsec:Proof-of-Lemma-loop-MC}). With probability at least $1-O\left(n^{-10}\right)$, 
\[
\left\Vert \big(\bm{M}^{0}-\bm{M}^{\left(l\right)}\big)\bm{U}^{0,\left(l\right)}\right\Vert _{\mathrm{F}}\lesssim\left\{ \mu r\sqrt{\frac{\log n}{np}}\sigma_{\max}+\sigma\sqrt{\frac{n\log n}{p}}\right\} \left\Vert \bm{U}^{0,\left(l\right)}\right\Vert _{2,\infty}.
\]
If one further has
\begin{equation}
\left\Vert \bm{U}^{0,\left(l\right)}\right\Vert _{2,\infty}\lesssim\left\Vert \bm{U}^{\star}\right\Vert _{2,\infty}\lesssim\frac{1}{\sqrt{\sigma_{\min}}}\left\Vert \bm{X}^{\star}\right\Vert _{2,\infty},\label{eq:claim-U0l}
\end{equation}
then taking the previous bounds collectively establishes the desired
bound
\[
\left\Vert \bm{X}^{0}\hat{\bm{H}}^{0}-\bm{X}^{0,\left(l\right)}\bm{R}^{0,\left(l\right)}\right\Vert _{\mathrm{F}}\lesssim\left\{ \mu r\sqrt{\frac{\log n}{np}}+\frac{\sigma}{\sigma_{\min}}\sqrt{\frac{n\log n}{p}}\right\} \left\Vert \bm{X}^{\star}\right\Vert _{2,\infty}.
\]
\begin{proof}[Proof of Claim (\ref{eq:claim-U0l})] Denote by
$\bm{M}^{(l),\text{zero}}$ the matrix derived by zeroing out the
$l$th row/column of $\bm{M}^{\left(l\right)}$, and $\bm{U}^{(l),\text{zero}}
\in\mathbb{R}^{n\times r}$ containing the leading $r$ eigenvectors of
$\bm{M}^{(l),\text{zero}}$. On the one hand, \cite[Lemma 4 and Lemma 14]{abbe2017entrywise}
demonstrate that 
\[
\max\limits _{1\leq l\leq n}\|\bm{U}^{(l),\text{zero}}\|_{2,\infty}\lesssim\|\bm{U}^{\star}\|_{2,\infty}.
\]
On the other hand, by the Davis-Kahan $\sin\Theta$ theorem \cite{davis1970rotation}
we obtain
\begin{align}
\label{eq:Ul-Uzero-L2}
\left\Vert \bm{U}^{0,\left(l\right)}\text{sgn}\left(\bm{U}^{0,\left(l\right)\top}\bm{U}^{\left(l\right),\text{zero}}\right)-\bm{U}^{\left(l\right),\text{zero}}\right\Vert _{\mathrm{F}}\lesssim\frac{1}{\sigma_{\min}}\left\Vert \left(\bm{M}^{\left(l\right)}-\bm{M}^{\left(l\right),\text{zero}}\right)\bm{U}^{\left(l\right),\text{zero}}\right\Vert _{\mathrm{F}},
\end{align}
where $\mathrm{sgn}(\bm{A})$ denotes the sign matrix of $\bm{A}$.
For any $j\neq l$, one has
\[
\left(\bm{M}^{\left(l\right)}-\bm{M}^{\left(l\right),\text{zero}}\right)_{j,\cdot}\bm{U}^{\left(l\right),\text{zero}}=\left(\bm{M}^{\left(l\right)}-\bm{M}^{\left(l\right),\text{zero}}\right)_{j,l}\bm{U}_{l,\cdot}^{\left(l\right),\text{zero}}=\bm{0}_{1\times r},
\]
since the $l$th row of $\bm{U}_{l,\cdot}^{\left(l\right),\text{zero}}$
is identically zero by construction. In addition,
\[
\left\Vert \left(\bm{M}^{\left(l\right)}-\bm{M}^{\left(l\right),\text{zero}}\right)_{l,\cdot}\bm{U}^{\left(l\right),\text{zero}}\right\Vert _{2}=\left\Vert \bm{M}_{l,\cdot}^{\star}\bm{U}^{\left(l\right),\text{zero}}\right\Vert _{2}\leq\left\Vert \bm{M}^{\star}\right\Vert _{2,\infty}\leq
\sigma_{\max}\left\Vert \bm{U}^{\star}\right\Vert _{2,\infty}.
\]
As a consequence, one has
\[
	\left\Vert \left(\bm{M}^{\left(l\right)}-\bm{M}^{\left(l\right),\text{zero}}\right) \bm{U}^{\left(l\right),\text{zero}}\right\Vert _{\mathrm{F}}  =
	\left\Vert \left(\bm{M}^{\left(l\right)}-\bm{M}^{\left(l\right),\text{zero}}\right)_{l,\cdot}\bm{U}^{\left(l\right),\text{zero}}\right\Vert _{2}
	\leq  \sigma_{\max}\left\Vert \bm{U}^{\star}\right\Vert _{2,\infty} ,
\]
which combined with (\ref{eq:Ul-Uzero-L2}) and the assumption  $\sigma_{\max} / \sigma_{\min} \asymp 1$ yields
\[
	\left\Vert \bm{U}^{0,\left(l\right)}\text{sgn}\left(\bm{U}^{0,\left(l\right)\top}\bm{U}^{\left(l\right),\text{zero}}\right)-\bm{U}^{\left(l\right),\text{zero}}\right\Vert _{\mathrm{F}}  
	\lesssim \left\Vert \bm{U}^{\star}\right\Vert _{2,\infty}
\]
The claim (\ref{eq:claim-U0l}) then follows by combining the above estimates:
\begin{align*}
	\left\Vert \bm{U}^{0,\left(l\right)}\right\Vert _{2,\infty} &= \left\Vert \bm{U}^{0,\left(l\right)}\text{sgn}\left(\bm{U}^{0,\left(l\right)\top}\bm{U}^{\left(l\right),\text{zero}}\right)\right\Vert _{2,\infty} \\
	&\leq \|\bm{U}^{(l),\text{zero}}\|_{2,\infty} + 
	\left\Vert \bm{U}^{0,\left(l\right)}\text{sgn}\left(\bm{U}^{0,\left(l\right)\top}\bm{U}^{\left(l\right),\text{zero}}\right)-\bm{U}^{\left(l\right),\text{zero}}\right\Vert _{\mathrm{F}}  \lesssim \|\bm{U}^{\star}\|_{2,\infty},
\end{align*}
where we have utilized the unitary invariance of $\left\Vert\cdot\right\Vert_{2,\infty}$. 
\end{proof}
\end{enumerate}


\section{Proofs for blind deconvolution}

Before proceeding to the proofs, we make note of the following concentration
results. The standard Gaussian concentration inequality and the union
bound give 
\begin{equation}
\max_{1\leq l\leq m}\left|\bm{a}_{l}^{\conj}\bm{x}^{\star}\right|\leq5\sqrt{\log m}
\label{eq:max_gaussian-1}
\end{equation}
with probability at least $1-O(m^{-10})$. In addition, with probability
exceeding $1-Cm\exp(-cK)$ for some constants $c,C>0$, 
\begin{equation}
\max_{1\leq l\leq m}\left\Vert \bm{a}_{l}\right\Vert _{2}\leq3\sqrt{K}.\label{eq:max_gaussian-2}
\end{equation}
In addition, the population/expected Wirtinger Hessian at the truth
$\bm{z}^{\star}$ is given by 
\begin{equation}
\nabla^{2}F\big(\bm{z}^{\star}\big)=\left[\begin{array}{cccc}
\bm{I}_{K} & \bm{0} & \bm{0} & \bm{h}^{\star}\bm{x}^{\star\top}\\
\bm{0} & \bm{I}_{K} & \bm{x}^{\star}\bm{h}^{\star\top} & \bm{0}\\
\bm{0} & \left(\bm{x}^{\star}\bm{h}^{\star\top}\right)^{\conj} & \bm{I}_{K} & \bm{0}\\
\left(\bm{h}^{\star}\bm{x}^{\star\top}\right)^{\conj} & \bm{0} & \bm{0} & \bm{I}_{K}
\end{array}\right].\label{eq:hessian-population-BD}
\end{equation}

\subsection{Proof of Lemma \ref{lemma:hessian-bd}\label{subsec:Proof-of-Lemma-hessian}}

First, we find it convenient to decompose the Wirtinger Hessian (cf.~(\ref{eq:hessian-BD})) into the expected Wirtinger Hessian at the truth (cf. \eqref{eq:hessian-population-BD}) and the perturbation part as follows: 
\begin{equation}
\nabla^{2}f\left(\bm{z}\right)=\nabla^{2}F\big(\bm{z}^{\star}\big)+\left(\nabla^{2}f\left(\bm{z}\right)-\nabla^{2}F\big(\bm{z}^{\star}\big)\right).\label{eq:hessian-decomposition}
\end{equation}
The proof then proceeds by showing that (i) the population Hessian
$\nabla^{2}F\big(\bm{z}^{\star}\big)$ satisfies the restricted strong convexity and smoothness
properties as advertised, and (ii) the perturbation $\nabla^{2}f\left(\bm{z}\right)-\nabla^{2}F\big(\bm{z}^{\star}\big)$
is well-controlled under our assumptions. We start by controlling the population Hessian in the following lemma.

\begin{lemma}\label{lemma:hessian-exp} Instate the notation and the conditions of Lemma \ref{lemma:hessian-bd}. We have 
\[
\left\Vert \nabla^{2}F\big(\bm{z}^{\star}\big)\right\Vert =2\qquad\text{and}\qquad\bm{u}^{\conj}\left[\bm{D}\nabla^{2}F\big(\bm{z}^{\star}\big)+\nabla^{2}F\big(\bm{z}^{\star}\big)\bm{D}\right]\bm{u}\geq\left\Vert \bm{u}\right\Vert _{2}^{2}.
\]
\end{lemma}

The next step is to bound the perturbation. To this end,
we define the set $$\mathcal{S} :=\left\{\bm{z} :\text{ } \bm{z}\; \text{satisfies \eqref{eq:condition_z}}\right\},$$ and derive the following lemma.

\begin{lemma}
\label{lemma:hessian-perturbation-BD} 
Suppose the sample complexity satisfies $m\gg\mu^{2}K\log^{9}m$, $c>0$ is a sufficiently small constant, and $\delta=c/\log^2 m$. Then with probability at least $1-O\left(m^{-10} + e^{-K} \log m \right)$,
one has 
\[
\sup_{\bm{z}\in\mathcal{S}}\left\Vert \nabla^{2}f\left(\bm{z}\right)-\nabla^{2}F\left(\bm{z}^{\star}\right)\right\Vert \leq {1}/{4}.
\]
 \end{lemma}

Combining the two lemmas, we can easily see that for $\bm{z}\in\mathcal{S}$,
\[
\left\Vert \nabla^{2}f\left(\bm{z}\right)\right\Vert \leq\left\Vert \nabla^{2}F\left(\bm{z}^{\star}\right)\right\Vert +\left\Vert \nabla^{2}f\left(\bm{z}\right)-\nabla^{2}F\left(\bm{z}^{\star}\right)\right\Vert \leq2+ {1}/{4}\leq3,
\]
which verifies the smoothness upper bound. In addition,
\begin{align*}
 & \bm{u}^{\conj}\left[\bm{D}\nabla^{2}f\left(\bm{z}\right)+\nabla^{2}f\left(\bm{z}\right)\bm{D}\right]\bm{u}\\
 & \quad=\bm{u}^{\conj}\left[\bm{D}\nabla^{2}F\left(\bm{z}^{\star}\right)+\nabla^{2}F\left(\bm{z}^{\star}\right)\bm{D}\right]\bm{u} + \bm{u}^{\conj}\bm{D}\left[\nabla^{2}f\left(\bm{z}\right)-\nabla^{2}F\left(\bm{z}^{\star}\right)\right]\bm{u}+\bm{u}^{\conj}\left[\nabla^{2}f\left(\bm{z}\right)-\nabla^{2}F\left(\bm{z}^{\star}\right)\right]\bm{D}\bm{u}\\
 & \quad\overset{\text{(i)}}{\geq}\bm{u}^{\conj}\left[\bm{D}\nabla^{2}F\left(\bm{z}^{\star}\right)+\nabla^{2}F\left(\bm{z}^{\star}\right)\bm{D}\right]\bm{u}-2\left\Vert \bm{D}\right\Vert \left\Vert \nabla^{2}f\left(\bm{z}\right)-\nabla^{2}F\left(\bm{z}^{\star}\right)\right\Vert \left\Vert \bm{u}\right\Vert _{2}^{2}\\
 & \quad\overset{\text{(ii)}}{\geq}\left\Vert \bm{u}\right\Vert _{2}^{2}-2\left(1+\delta\right)\cdot\frac{1}{4}\left\Vert \bm{u}\right\Vert _{2}^{2}\\
 & \quad\overset{\text{(iii)}}{\geq}\frac{1}{4}\left\Vert \bm{u}\right\Vert _{2}^{2},
\end{align*}
where (i) uses the triangle inequality, (ii) holds because of Lemma~\ref{lemma:hessian-perturbation-BD} and the fact that $\|\bm{D}\|\leq 1+\delta$, and (iii) follows if $\delta\leq1/2$. This establishes the claim on the restricted strong convexity.


\subsubsection{Proof of Lemma \ref{lemma:hessian-exp}}
We start by proving the identity
$\left\Vert \nabla^{2}F\left(\bm{z}^{\star}\right)\right\Vert =2$.
Let 
\[
\bm{u}_{1}=\frac{1}{\sqrt{2}}\left[\begin{array}{c}
\bm{h}^{\star}\\
\bm{0}\\
\bm{0}\\
\overline{\bm{x}^{\star}}
\end{array}\right],\qquad\bm{u}_{2}=\frac{1}{\sqrt{2}}\left[\begin{array}{c}
\bm{0}\\
\bm{x}^{\star}\\
\overline{\bm{h}^{\star}}\\
\bm{0}
\end{array}\right],\qquad\bm{u}_{3}=\frac{1}{\sqrt{2}}\left[\begin{array}{c}
\bm{h}^{\star}\\
\bm{0}\\
\bm{0}\\
-\overline{\bm{x}^{\star}}
\end{array}\right],\qquad\bm{u}_{4}=\frac{1}{\sqrt{2}}\left[\begin{array}{c}
\bm{0}\\
\bm{x}^{\star}\\
-\overline{\bm{h}^{\star}}\\
\bm{0}
\end{array}\right].
\]
Recalling that $\|\bm{h}^{\star}\|_{2}=\|\bm{x}^{\star}\|_{2}=1$,
we can easily check that these four vectors form an orthonormal set.
A little algebra reveals that 
\[
\nabla^{2}F\left(\bm{z}^{\star}\right)=\bm{I}_{4K}+\bm{u}_{1}\bm{u}_{1}^{\conj}+\bm{u}_{2}\bm{u}_{2}^{\conj}-\bm{u}_{3}\bm{u}_{3}^{\conj}-\bm{u}_{4}\bm{u}_{4}^{\conj},
\]
which immediately implies 
\[
\left\Vert \nabla^{2}F\left(\bm{z}^{\star}\right)\right\Vert =2.
\]

We now turn attention to the restricted strong convexity. Since $\bm{u}^{\conj}\bm{D}\nabla^{2}F\left(\bm{z}^{\star}\right)\bm{u}$ is the complex conjugate of $\bm{u}^{\conj}\nabla^{2}F\left(\bm{z}^{\star}\right)\bm{D}\bm{u}$ as both $\nabla^{2}F(\bm{z}^{\star})$ and $\bm{D}$
are Hermitian, we will focus on the first term $\bm{u}^{\conj}\bm{D}\nabla^{2}F\left(\bm{z}^{\star}\right)\bm{u}$. This term can be rewritten as
\begin{align}
 & \bm{u}^{\conj}\bm{D}\nabla^{2}F\left(\bm{z}^{\star}\right)\bm{u} \nonumber \\
 &\overset{\left(\text{i}\right)}{=} \left[\left(\bm{h}_{1}-\bm{h}_{2}\right)^{\conj},\left(\bm{x}_{1}-\bm{x}_{2}\right)^{\conj},\left(\overline{\bm{h}_{1}-\bm{h}_{2}}\right)^{\conj},\left(\overline{\bm{x}_{1}-\bm{x}_{2}}\right)^{\conj}\right]\bm{D} \left[\begin{array}{cccc}
\bm{I}_{K} & \bm{0} & \bm{0} & \bm{h}^{\star}\bm{x}^{\star\top}\\
\bm{0} & \bm{I}_{K} & \bm{x}^{\star}\bm{h}^{\star\top} & \bm{0}\\
\bm{0} & \left(\bm{x}^{\star}\bm{h}^{\star\top}\right)^{\conj} & \bm{I}_{K} & \bm{0}\\
\left(\bm{h}^{\star}\bm{x}^{\star\top}\right)^{\conj} & \bm{0} & \bm{0} & \bm{I}_{K}
\end{array}\right]\left[\begin{array}{c}
\bm{h}_{1}-\bm{h}_{2}\\
\bm{x}_{1}-\bm{x}_{2}\\
\overline{\bm{h}_{1}-\bm{h}_{2}}\\
\overline{\bm{x}_{1}-\bm{x}_{2}}
\end{array}\right]\nonumber \\
 &  \overset{\left(\text{ii}\right)}{=}\left[\gamma_{1}\left(\bm{h}_{1}-\bm{h}_{2}\right)^{\conj},\gamma_{2}\left(\bm{x}_{1}-\bm{x}_{2}\right)^{\conj},\gamma_{1}\left(\overline{\bm{h}_{1}-\bm{h}_{2}}\right)^{\conj},\gamma_{2}\left(\overline{\bm{x}_{1}-\bm{x}_{2}}\right)^{\conj}\right]\left[\begin{array}{c}
\bm{h}_{1}-\bm{h}_{2}+\bm{h}^{\star}\bm{x}^{\star\top}\overline{(\bm{x}_{1}-\bm{x}_{2})}\\
\bm{x}_{1}-\bm{x}_{2}+\bm{x}^{\star}\bm{h}^{\star\top}\overline{(\bm{h}_{1}-\bm{h}_{2})}\\
\left(\bm{x}^{\star}\bm{h}^{\star\top}\right)^{\conj}(\bm{x}_{1}-\bm{x}_{2})+\overline{(\bm{h}_{1}-\bm{h}_{2})}\\
\left(\bm{h}^{\star}\bm{x}^{\star\top}\right)^{\conj}(\bm{h}_{1}-\bm{h}_{2})+\overline{(\bm{x}_{1}-\bm{x}_{2})}
\end{array}\right]\nonumber \\
 &  =\left[\gamma_{1}\left(\bm{h}_{1}-\bm{h}_{2}\right)^{\conj},\gamma_{2}\left(\bm{x}_{1}-\bm{x}_{2}\right)^{\conj},\gamma_{1}\left(\overline{\bm{h}_{1}-\bm{h}_{2}}\right)^{\conj},\gamma_{2}\left(\overline{\bm{x}_{1}-\bm{x}_{2}}\right)^{\conj}\right]\left[\begin{array}{c}
\bm{h}_{1}-\bm{h}_{2}+\bm{h}^{\star}\left(\bm{x}_{1}-\bm{x}_{2}\right)^{\conj}\bm{x}^{\star}\\
\bm{x}_{1}-\bm{x}_{2}+\bm{x}^{\star}\left(\bm{h}_{1}-\bm{h}_{2}\right)^{\conj}\bm{h}^{\star}\\
\overline{\bm{h}_{1}-\bm{h}_{2}}+\overline{\bm{h}^{\star}\left(\bm{x}_{1}-\bm{x}_{2}\right)^{\conj}\bm{x}^{\star}}\\
\overline{\bm{x}_{1}-\bm{x}_{2}}+\overline{\bm{x}^{\star}\left(\bm{h}_{1}-\bm{h}_{2}\right)^{\conj}\bm{h}^{\star}}
\end{array}\right]\nonumber \\
 &  =2\gamma_{1}\left\Vert \bm{h}_{1}-\bm{h}_{2}\right\Vert _{2}^{2}+2\gamma_{2}\left\Vert \bm{x}_{1}-\bm{x}_{2}\right\Vert _{2}^{2}\nonumber \\
 & \quad\quad+\left(\gamma_{1}+\gamma_{2}\right)\underbrace{\left(\bm{h}_{1}-\bm{h}_{2}\right)^{\conj}\bm{h}^{\star}\left(\bm{x}_{1}-\bm{x}_{2}\right)^{\conj}\bm{x}^{\star}}_{:=\beta}+\left(\gamma_{1}+\gamma_{2}\right)\underset{=\overline{\beta}}{\underbrace{\overline{\left(\bm{h}_{1}-\bm{h}_{2}\right)^{\conj}\bm{h}^{\star}\left(\bm{x}_{1}-\bm{x}_{2}\right)^{\conj}\bm{x}^{\star}}}},\label{eq:uFu-LB-BD}
\end{align}
where (i) uses the definitions of $\bm{u}$ and $\nabla^{2}F\left(\bm{z}^{\star}\right)$,
and (ii) follows from the definition of $\bm{D}$. In view of the
assumption \eqref{eq:condition_D}, we can obtain 
\begin{align*}
2\gamma_{1}\left\Vert \bm{h}_{1}-\bm{h}_{2}\right\Vert _{2}^{2}+2\gamma_{2}\left\Vert \bm{x}_{1}-\bm{x}_{2}\right\Vert _{2}^{2} & \geq 2\min\left\{ \gamma_{1},\gamma_{2}\right\} \left( \left\Vert \bm{h}_{1}-\bm{h}_{2}\right\Vert _{2}^{2}+ \left\Vert \bm{x}_{1}-\bm{x}_{2}\right\Vert _{2}^{2}\right)\geq\left(1-\delta\right)\left\Vert \bm{u}\right\Vert _{2}^{2},
\end{align*}
where the last inequality utilizes the identity 
\begin{equation}
2\left\Vert \bm{h}_{1}-\bm{h}_{2}\right\Vert _{2}^{2}+2\left\Vert \bm{x}_{1}-\bm{x}_{2}\right\Vert _{2}^{2}=\left\Vert \bm{u}\right\Vert _{2}^{2}.\label{eq:u-norm-BD}
\end{equation}

It then boils down to controlling $\beta$. Toward this goal, we decompose
$\beta$ into the following four terms 
\begin{align*}
\beta & =\underbrace{\left(\bm{h}_{1}-\bm{h}_{2}\right)^{\conj}\bm{h}_{2}\left(\bm{x}_{1}-\bm{x}_{2}\right)^{\conj}\bm{x}_{2}}_{:=\beta_{1}}+\underbrace{\left(\bm{h}_{1}-\bm{h}_{2}\right)^{\conj}\left(\bm{h}^{\star}-\bm{h}_{2}\right)\left(\bm{x}_{1}-\bm{x}_{2}\right)^{\conj}\left(\bm{x}^{\star}-\bm{x}_{2}\right)}_{:=\beta_{2}}\\
 & \quad+\underbrace{\left(\bm{h}_{1}-\bm{h}_{2}\right)^{\conj}\left(\bm{h}^{\star}-\bm{h}_{2}\right)\left(\bm{x}_{1}-\bm{x}_{2}\right)^{\conj}\bm{x}_{2}}_{:=\beta_{3}}+\underbrace{\left(\bm{h}_{1}-\bm{h}_{2}\right)^{\conj}\bm{h}_{2}\left(\bm{x}_{1}-\bm{x}_{2}\right)^{\conj}\left(\bm{x}^{\star}-\bm{x}_{2}\right)}_{:=\beta_{4}}.
\end{align*}
Since $\left\Vert \bm{h}_{2}-\bm{h}^{\star}\right\Vert_2 $ and $\left\Vert \bm{x}_{2}-\bm{x}^{\star}\right\Vert _{2}$
are both small by \eqref{eq:condition_u}, $\beta_{2},\beta_{3}$ and $\beta_{4}$ are well-bounded. Specifically, regarding $\beta_{2}$, we discover that 
\begin{align*}
\left|\beta_{2}\right| & \leq\left\Vert \bm{h}^{\star}-\bm{h}_{2}\right\Vert _{2}\left\Vert \bm{x}^{\star}-\bm{x}_{2}\right\Vert _{2}\left\Vert \bm{h}_{1}-\bm{h}_{2}\right\Vert _{2}\left\Vert \bm{x}_{1}-\bm{x}_{2}\right\Vert _{2}\leq\delta^{2}\left\Vert \bm{h}_{1}-\bm{h}_{2}\right\Vert _{2}\left\Vert \bm{x}_{1}-\bm{x}_{2}\right\Vert _{2}\leq\delta\left\Vert \bm{h}_{1}-\bm{h}_{2}\right\Vert _{2}\left\Vert \bm{x}_{1}-\bm{x}_{2}\right\Vert _{2},
\end{align*}
where the second inequality is due to \eqref{eq:condition_u} and the last one holds since $\delta < 1$. Similarly, we can obtain 
\begin{align*}
\left|\beta_{3}\right| & \leq\delta\left\Vert \bm{x}_{2}\right\Vert _{2}\left\Vert \bm{h}_{1}-\bm{h}_{2}\right\Vert _{2}\left\Vert \bm{x}_{1}-\bm{x}_{2}\right\Vert _{2}\leq2\delta\left\Vert \bm{h}_{1}-\bm{h}_{2}\right\Vert _{2}\left\Vert \bm{x}_{1}-\bm{x}_{2}\right\Vert _{2},\\
\text{and}\qquad\left|\beta_{4}\right| & \leq\delta\left\Vert \bm{h}_{2}\right\Vert _{2}\left\Vert \bm{h}_{1}-\bm{h}_{2}\right\Vert _{2}\left\Vert \bm{x}_{1}-\bm{x}_{2}\right\Vert _{2}\leq2\delta\left\Vert \bm{h}_{1}-\bm{h}_{2}\right\Vert _{2}\left\Vert \bm{x}_{1}-\bm{x}_{2}\right\Vert _{2},
\end{align*}
where both lines make use of the facts that
\begin{equation}
\left\Vert \bm{x}_{2}\right\Vert _{2}\leq\left\Vert \bm{x}_{2}-\bm{x}^{\star}\right\Vert _{2}+\left\Vert \bm{x}^{\star}\right\Vert _{2}\leq1+\delta\leq2\qquad\text{and}\qquad\left\Vert \bm{h}_{2}\right\Vert _{2}\leq\left\Vert \bm{h}_{2}-\bm{h}^{\star}\right\Vert _{2}+\left\Vert \bm{h}^{\star}\right\Vert _{2}\leq1+\delta\leq2.\label{eq:x2-UB-BD}
\end{equation}
Combine the previous three bounds to reach 
\[
\left|\beta_{2}\right|+\left|\beta_{3}\right|+\left|\beta_{4}\right|\leq5\delta\left\Vert \bm{h}_{1}-\bm{h}_{2}\right\Vert _{2}\left\Vert \bm{x}_{1}-\bm{x}_{2}\right\Vert _{2}\leq5\delta\frac{\left\Vert \bm{h}_{1}-\bm{h}_{2}\right\Vert _{2}^{2}+\left\Vert \bm{x}_{1}-\bm{x}_{2}\right\Vert _{2}^{2}}{2}=\frac{5}{4}\delta\left\Vert \bm{u}\right\Vert _{2}^{2},
\]
where we utilize the elementary inequality $ab\leq (a^2+b^2)/2$ and the identity (\ref{eq:u-norm-BD}).

The only remaining term is thus $\beta_{1}$. Recalling that $(\bm{h}_{1},\bm{x}_{1})$
and $(\bm{h}_{2},\bm{x}_{2})$ are aligned by our assumption, we can invoke
Lemma \ref{lemma:bd-alignment} to obtain 
\begin{align*}
\left(\bm{h}_{1}-\bm{h}_{2}\right)^{\conj}\bm{h}_{2} & =\left\Vert \bm{x}_{1}-\bm{x}_{2}\right\Vert _{2}^{2}+\bm{x}_{2}^{\conj}\left(\bm{x}_{1}-\bm{x}_{2}\right)-\left\Vert \bm{h}_{1}-\bm{h}_{2}\right\Vert _{2}^{2},
\end{align*}
which allows one to rewrite $\beta_{1}$ as 
\begin{align*}
\beta_{1} & =\left\{ \left\Vert \bm{x}_{1}-\bm{x}_{2}\right\Vert _{2}^{2}+\bm{x}_{2}^{\conj}\left(\bm{x}_{1}-\bm{x}_{2}\right)-\left\Vert \bm{h}_{1}-\bm{h}_{2}\right\Vert _{2}^{2}\right\} \cdot\left(\bm{x}_{1}-\bm{x}_{2}\right)^{\conj}\bm{x}_{2}\\
 & =\left(\bm{x}_{1}-\bm{x}_{2}\right)^{\conj}\bm{x}_{2}\left(\left\Vert \bm{x}_{1}-\bm{x}_{2}\right\Vert _{2}^{2}-\left\Vert \bm{h}_{1}-\bm{h}_{2}\right\Vert _{2}^{2}\right)+\left|\left(\bm{x}_{1}-\bm{x}_{2}\right)^{\conj}\bm{x}_{2}\right|^{2}.
\end{align*}
Consequently, 
\begin{align*}
\beta_{1}+\overline{\beta_{1}} & =2\left|\left(\bm{x}_{1}-\bm{x}_{2}\right)^{\conj}\bm{x}_{2}\right|_{2}^{2}+2\text{Re}\left[\left(\bm{x}_{1}-\bm{x}_{2}\right)^{\conj}\bm{x}_{2}\right]\left(\left\Vert \bm{x}_{1}-\bm{x}_{2}\right\Vert _{2}^{2}-\left\Vert \bm{h}_{1}-\bm{h}_{2}\right\Vert _{2}^{2}\right)\\
 & \geq2\text{Re}\left[\left(\bm{x}_{1}-\bm{x}_{2}\right)^{\conj}\bm{x}_{2}\right]\left(\left\Vert \bm{x}_{1}-\bm{x}_{2}\right\Vert _{2}^{2}-\left\Vert \bm{h}_{1}-\bm{h}_{2}\right\Vert _{2}^{2}\right)\\
 & \overset{\left(\text{i}\right)}{\geq}-\left|\left(\bm{x}_{1}-\bm{x}_{2}\right)^{\conj}\bm{x}_{2}\right|\left\Vert \bm{u}\right\Vert _{2}^{2}\\
 & \overset{\left(\text{ii}\right)}{\geq}-4\delta\left\Vert \bm{u}\right\Vert _{2}^{2}.
\end{align*}
Here, (i) arises from the triangle inequality that 
\[
\left|\left\Vert \bm{x}_{1}-\bm{x}_{2}\right\Vert _{2}^{2}-\left\Vert \bm{h}_{1}-\bm{h}_{2}\right\Vert _{2}^{2}\right|\leq\left\Vert \bm{x}_{1}-\bm{x}_{2}\right\Vert _{2}^{2}+\left\Vert \bm{h}_{1}-\bm{h}_{2}\right\Vert _{2}^{2}=\frac{1}{2}\left\Vert \bm{u}\right\Vert _{2}^{2},
\]
and (ii) occurs since $\|\bm{x}_{1}-\bm{x}_{2}\|_{2}\leq\|\bm{x}_{1}-\bm{x}^{\star}\|_{2}+\|\bm{x}_{2}-\bm{x}^{\star}\|_{2}\leq2\delta$
and $\|\bm{x}_{2}\|_{2}\leq2$ (see~(\ref{eq:x2-UB-BD})).

To finish up, note that $\gamma_{1}+\gamma_{2}\leq2(1+\delta)\leq3$
for $\delta<1/2$. Substitute these bounds into (\ref{eq:uFu-LB-BD})
to obtain 
\begin{align*}
\bm{u}^{\conj}\bm{D}\nabla^{2}F\left(\bm{z}^{\star}\right)\bm{u} & \geq\left(1-\delta\right)\left\Vert \bm{u}\right\Vert _{2}^{2}+\left(\gamma_{1}+\gamma_{2}\right)\left(\beta+\overline{\beta}\right)\\
 & \geq\left(1-\delta\right)\left\Vert \bm{u}\right\Vert _{2}^{2}+\left(\gamma_{1}+\gamma_{2}\right)\left(\beta_{1}+\overline{\beta_{1}}\right)-2\left(\gamma_{1}+\gamma_{2}\right)\left(\left|\beta_{2}\right|+\left|\beta_{3}\right|+\left|\beta_{4}\right|\right)\\
 & \geq\left(1-\delta\right)\left\Vert \bm{u}\right\Vert _{2}^{2}-12\delta\left\Vert \bm{u}\right\Vert _{2}^{2}-6\cdot\frac{5}{4}\delta\left\Vert \bm{u}\right\Vert _{2}^{2}\\
 & \geq\left(1-20.5\delta\right)\left\Vert \bm{u}\right\Vert _{2}^{2}\\
 & \geq\frac{1}{2}\left\Vert \bm{u}\right\Vert _{2}^{2}
\end{align*}
as long as $\delta$ is small enough.  

\subsubsection{Proof of Lemma \ref{lemma:hessian-perturbation-BD}}

In view of the expressions of $\nabla^{2}f\left(\bm{z}\right)$ and $\nabla^{2}F\left(\bm{z}^{\star}\right)$
(cf. (\ref{eq:hessian-BD}) and (\ref{eq:hessian-population-BD}))
and the triangle inequality, we get
\begin{align}\label{ineq-hessian-decomposition}
\left\Vert \nabla^{2}f\left(\bm{z}\right)-\nabla^{2}F\left(\bm{z}^{\star}\right)\right\Vert \leq2\alpha_{1}+2\alpha_{2}+4\alpha_{3}+4\alpha_{4},
\end{align}
where the four terms on the right-hand side are defined as follows
\begin{align*}
\alpha_{1} & =\left\Vert \sum_{j=1}^{m}\left|\bm{a}_{j}^{\conj}\bm{x}\right|^{2}\bm{b}_{j}\bm{b}_{j}^{\conj}-\bm{I}_{K}\right\Vert ,\qquad\alpha_{2}=\left\Vert \sum_{j=1}^{m}\left|\bm{b}_{j}^{\conj}\bm{h}\right|^{2}\bm{a}_{j}\bm{a}_{j}^{\conj}-\bm{I}_{K}\right\Vert ,\\
\alpha_{3} & =\left\Vert \sum_{j=1}^{m}\left(\bm{b}_{j}^{\conj}\bm{h}\bm{x}^{\conj}\bm{a}_{j}-y_{j}\right)\bm{b}_{j}\bm{a}_{j}^{\conj}\right\Vert ,\qquad\alpha_{4}=\left\Vert \sum_{j=1}^{m}\bm{b}_{j}\bm{b}_{j}^{\conj}\bm{h}\left(\bm{a}_{j}\bm{a}_{j}^{\conj}\bm{x}\right)^{\top}-\bm{h}^{\star}\bm{x}^{\star\top}\right\Vert .
\end{align*}
In what follows, we shall control $\sup_{\bm{z}\in\mathcal{S}}\alpha_{j}$
for $j=1,2,3,4$ separately. 
\begin{enumerate}
\item Regarding the first term $\alpha_{1}$, the triangle inequality gives 
\begin{align*}
\alpha_{1} & \leq\underbrace{\left\Vert \sum_{j=1}^{m}\left|\bm{a}_{j}^{\conj}\bm{x}\right|^{2}\bm{b}_{j}\bm{b}_{j}^{\conj}-\sum_{j=1}^{m}\left|\bm{a}_{j}^{\conj}\bm{x}^{\star}\right|^{2}\bm{b}_{j}\bm{b}_{j}^{\conj}\right\Vert }_{:=\beta_{1}}+\underbrace{\left\Vert \sum_{j=1}^{m}\left|\bm{a}_{j}^{\conj}\bm{x}^{\star}\right|^{2}\bm{b}_{j}\bm{b}_{j}^{\conj}-\bm{I}_{K}\right\Vert }_{:=\beta_{2}}.
\end{align*}

\begin{itemize}
\item To control $\beta_{1}$, the key observation is that $\bm{a}_{j}^{\conj}\bm{x}$
and $\bm{a}_{j}^{\conj}\bm{x}^{\star}$ are extremely close. We can
rewrite $\beta_{1}$ as 
\begin{align}
\beta_{1} & =\left\Vert \sum_{j=1}^{m}\left(\left|\bm{a}_{j}^{\conj}\bm{x}\right|^{2}-\left|\bm{a}_{j}^{\conj}\bm{x}^{\star}\right|^{2}\right)\bm{b}_{j}\bm{b}_{j}^{\conj}\right\Vert \leq\left\Vert \sum_{j=1}^{m}\left|\left|\bm{a}_{j}^{\conj}\bm{x}\right|^{2}-\left|\bm{a}_{j}^{\conj}\bm{x}^{\star}\right|^{2}\right|\bm{b}_{j}\bm{b}_{j}^{\conj}\right\Vert ,\label{eq:hessian-beta-1-BD}
\end{align}
where 
\begin{align*}
\left|\left|\bm{a}_{j}^{\conj}\bm{x}\right|^{2}-\left|\bm{a}_{j}^{\conj}\bm{x}^{\star}\right|^{2}\right| & \overset{\text{(i)}}{=}\left|\left[\bm{a}_{j}^{\conj}\left(\bm{x}-\bm{x}^{\star}\right)\right]^{\conj}\bm{a}_{j}^{\conj}\left(\bm{x}-\bm{x}^{\star}\right)+\left[\bm{a}_{j}^{\conj}\left(\bm{x}-\bm{x}^{\star}\right)\right]^{\conj}\bm{a}_{j}^{\conj}\bm{x}^{\star}+\left(\bm{a}_{j}^{\conj}\bm{x}^{\star}\right)^{\conj}\bm{a}_{j}^{\conj}\left(\bm{x}-\bm{x}^{\star}\right)\right|\\
 & \overset{\text{(ii)}}{\leq}\left|\bm{a}_{j}^{\conj}\left(\bm{x}-\bm{x}^{\star}\right)\right|^{2}+2\left|\bm{a}_{j}^{\conj}\left(\bm{x}-\bm{x}^{\star}\right)\right|\left|\bm{a}_{j}^{\conj}\bm{x}^{\star}\right|\\
 & \overset{\text{(iii)}}{\leq}4C_{3}^{2}\frac{1}{\log^{3}m}+4C_{3}\frac{1}{\log^{3/2}m}\cdot5\sqrt{\log m}\\
 & \lesssim C_{3}\frac{1}{\log m}.
\end{align*}
Here, the first line (i) uses the identity for $u,v\in\mathbb{C}$,
\[
|u|^2 - |v|^2 = u^{\conj}u-v^{\conj}v=(u-v)^{\conj}(u-v)+(u-v)^{\conj}v+v^{\conj}(u-v),
\]
the second relation (ii) comes from the triangle inequality, and the third line (iii) follows
from (\ref{eq:max_gaussian-1}) and the assumption \eqref{eq:condition_incoherence_a}.
Substitution into (\ref{eq:hessian-beta-1-BD}) gives 
\begin{align*}
\beta_{1} & \leq\max_{1\leq j\leq m}\left|\left|\bm{a}_{j}^{\conj}\bm{x}\right|^{2}-\left|\bm{a}_{j}^{\conj}\bm{x}^{\star}\right|^{2}\right|\left\Vert \sum_{j=1}^{m}\bm{b}_{j}\bm{b}_{j}^{\conj}\right\Vert \lesssim C_{3}\frac{1}{\log m},
\end{align*}
where the last inequality comes from the fact that $\sum_{j=1}^{m}\bm{b}_{j}\bm{b}_{j}^{\conj}=\bm{I}_{K}$. 
\item The other term $\beta_{2}$ can be bounded through Lemma \ref{lemma:concentration-identity-BD},
which reveals that with probability $1-O\left(m^{-10}\right)$, 
\[
\beta_{2}\lesssim\sqrt{\frac{K}{m}\log m}.
\]
\end{itemize}
Taken collectively, the preceding two bounds give 
\[
\sup_{\bm{z}\in\mathcal{S}}\alpha_{1}\lesssim\sqrt{\frac{K}{m}\log {m}}+C_{3}\frac{1}{\log m}.
\]
Hence $\PP ( \sup_{\bm{z}\in\mathcal{S}}\alpha_{1} \leq 1/32 ) = 1 - O(m^{-10})$.

\item We are going to prove that $\PP( \sup_{\bm{z} \in \mathcal{S} } \alpha_2 \leq 1/32 ) = 1 - O(m^{-10})$.
The triangle inequality allows us to bound $\alpha_{2}$ as 
\begin{align*}
	\alpha_{2} & \leq\underbrace{\left\Vert \sum_{j=1}^{m}\left|\bm{b}_{j}^{\conj}\bm{h}\right|^{2}\bm{a}_{j}\bm{a}_{j}^{\conj}-\left\Vert \bm{h}\right\Vert _{2}^{2}\bm{I}_{K}\right\Vert }_{:=\theta_{1}(\bm{h})}+\underbrace{\left\Vert \left\Vert \bm{h}\right\Vert _{2}^{2}\bm{I}_{K}-\bm{I}_{K}\right\Vert \vphantom{\left\Vert \sum_{j=1}^{m}\left|\bm{b}_{j}^{\conj}\bm{h}\right|^{2}\bm{a}_{j}\bm{a}_{j}^{\conj}-\left\Vert \bm{h}\right\Vert _{2}^{2}\bm{I}_{K}\right\Vert }}_{:=\theta_{2}(\bm{h})}.
\end{align*}
The second term $\theta_{2}(\bm{h})$ is easy to control. To see this, we
have 
\[
\theta_{2}(\bm{h}) =\left|\left\Vert \bm{h}\right\Vert _{2}^{2}-1\right|=\big|\left\Vert \bm{h}\right\Vert _{2}-1\big|\left(\left\Vert \bm{h}\right\Vert _{2}+1\right)\leq3\delta < 1/64,
\]
where the penultimate relation uses the assumption that $\left\Vert \bm{h}-\bm{h}^{\star}\right\Vert _{2}\leq\delta$
and hence 
\[
\big|\left\Vert \bm{h}\right\Vert _{2}-1\big| \leq\delta,\qquad\left\Vert \bm{h}\right\Vert _{2}\leq 1+ \delta \leq 2.
\]
For the first term $\theta_{1}(\bm{h})$, we define a new set 
\begin{align*}
\cH := & \left\{
\bm{h}\in\CC^{K} :\text{ } \|\bm{h}-\bm{h}^{\star}\|_{2} \leq\delta\quad
\text{and}\quad
\max_{1\leq j\leq m}\left|\bm{b}_{j}^{\conj}\bm{h}\right|\leq \frac{2C_{4}\mu\log^{2}m}{\sqrt{m}}\right\}.
\end{align*}
It is easily seen that 
$\sup_{\bm{z}\in\cS}\theta_{1}\leq
\sup_{\bm{h}\in\cH}\theta_{1}$.
We plan to use the standard covering argument to show that
\begin{equation}
\PP\left(
\sup_{\bm{h}\in\cH }\theta_{1}(\bm{h}) \leq 1/64
\right) = 1 - O(m^{-10})
.\label{eq:hessian-alpha-2-BD}
\end{equation}
To this end, we define $c_j(\bm{h}) = | \bm{b}_j^\conj \bm{h} |^2$ for every $1\leq j \leq m$. It is straightforward to check that 
\begin{align}
\theta_1(\bm{h}) = \left\| \sum_{j=1}^{m} c_j(\bm{h}) \left( \bm{a}_j \bm{a}_j^\conj - \bm{I}_K \right) \right\|, 
\qquad 
\max_{1\leq j\leq m} |c_j|  \leq   \left(  \frac{ 2C_{4}\mu\log^{2}m }{ \sqrt{m} } \right)^{2} , \qquad\qquad  \label{eq:UB-BD-14} \\
%
\sum_{j=1}^{m}c_j^2  =\sum_{j=1}^{m}|\bm{b}_{j}^{\conj}\bm{h}|^{4}\leq \left\{ \max_{1\leq j\leq m}|\bm{b}_{j}^{\conj}\bm{h}|^{2}\right\} \sum_{j=1}^{m}|\bm{b}_{j}^{\conj}\bm{h}|^{2} = \left\{ \max_{1\leq j\leq m}|\bm{b}_{j}^{\conj}\bm{h}|^{2}\right\} \|\bm{h}\|_{2}^{2}
	\leq 4\left(\frac{2C_{4}\mu\log^{2}m}{\sqrt{m}}\right)^{2}  \label{eq:UB-BD-15}
\end{align}
for $\bm{h}\in\mathcal{H}$. 
In the above argument, we have used the facts that $\sum_{j=1}^{m}\bm{b}_{j}\bm{b}_{j}^{\conj}=\bm{I}_{K}$ and
\[
\sum_{j=1}^{m}|\bm{b}_{j}^{\conj}\bm{h}|^{2}=\bm{h}^{\conj}\left(\sum_{j=1}^{m}\bm{b}_{j}\bm{b}_{j}^{\conj}\right)\bm{h}=\|\bm{h}\|_{2}^{2} \leq (1+\delta)^2\leq 4,
\]
together with the definition of $\cH$. Lemma \ref{lemma-covariance-1205} combined with \eqref{eq:UB-BD-14} and \eqref{eq:UB-BD-15} readily yields that for any fixed $\bm{h} \in \cH$ and any $t \geq 0$,
\begin{align}
\PP( \theta_1(\bm{h}) \geq t ) &\leq 2 \exp \left(
\tilde{C}_1 K - \tilde{C}_2 \min\left\{
\frac{t}{\max_{ 1\leq j\leq m} |c_j| }, \frac{t^2}{\sum_{j=1}^{m}c_j^2}
\right\}
\right) \nonumber\\
& \leq 2 \exp \left(
\tilde{C}_1 K - \tilde{C}_2 \frac{m t \min\left\{
	1, t/4
	\right\}}{4 C_4^2 \mu^2 \log^4 m} 
\right),\label{ineq-hessian-norm-fixed}
\end{align}
where $\tilde{C}_1, \tilde{C}_2 >0$ are some universal constants.

Now we are in a position to strengthen this bound to obtain uniform control of $\theta_1$ over $\cH$. Note that for any $\bm{h}_1,\bm{h}_2\in\cH$,
\begin{align*}
|\theta_1(\bm{h}_1) - \theta_1(\bm{h}_2)| &\leq\left\| 
\sum_{j=1}^{m} \left(
|\bm{b}_j^\conj \bm{h}_1|^2 - |\bm{b}_j^\conj \bm{h}_2|^2
\right) \bm{a}_j \bm{a}_j^\conj
\right\| + \left| \|\bm{h}_1\|_2^2 -\|\bm{h}_2\|_2^2 \right|\\
&=
\max_{1\leq j\leq m}
\left|
|\bm{b}_j^\conj \bm{h}_1|^2 - |\bm{b}_j^\conj \bm{h}_2|^2
\right|
\left\|
\sum_{j=1}^{m} \bm{a}_j \bm{a}_j^\conj
\right\| + \left| \|\bm{h}_1\|_2^2 -\|\bm{h}_2\|_2^2 \right|,
\end{align*}
where
\begin{align*}
\left||\bm{b}_{j}^{\conj}\bm{h}_{2}|^{2}-|\bm{b}_{j}^{\conj}\bm{h}_{1}|^{2}\right|&=\left|(\bm{h}_{2}-\bm{h}_{1})^{\conj}\bm{b}_{j}\bm{b}_{j}^{\conj}\bm{h}_{2}+\bm{h}_{1}^{\conj}\bm{b}_{j}\bm{b}_{j}^{\conj}(\bm{h}_{2}-\bm{h}_{1})\right| \\
 &\leq 2\max\{\|\bm{h}_{1}\|_{2},\|\bm{h}_{2}\|_{2}\}\|\bm{h}_{2}-\bm{h}_{1}\|_{2}\|\bm{b}_{j}\|_{2}^{2}\\
& \leq 4\|\bm{h}_{2}-\bm{h}_{1}\|_{2}\|\bm{b}_{j}\|_{2}^{2}
 \leq\frac{4K}{m}\|\bm{h}_{2}-\bm{h}_{1}\|_{2}
\end{align*}
and 
\begin{align*}
\left| \|\bm{h}_1\|_2^2 -\|\bm{h}_2\|_2^2 \right|
=\left| \bm{h}_1^\conj ( \bm{h}_1 - \bm{h}_2 ) - ( \bm{h}_1 - \bm{h}_2 )^\conj \bm{h}_2 \right| 
\leq 2\max\{\|\bm{h}_{1}\|_{2},\|\bm{h}_{2}\|_{2}\}\|\bm{h}_{2}-\bm{h}_{1}\|_{2}
\leq 4 \| \bm{h}_1 - \bm{h}_2\|_2
.
\end{align*}
Define an event $\cE_0 = \left\{
\left\|
\sum_{j=1}^{m} \bm{a}_j \bm{a}_j^\conj
\right\| \leq 2m
\right\} $. When $\cE_0$ happens, the previous estimates give
\[
|\theta_1(\bm{h}_1) - \theta_1(\bm{h}_2)| \leq (8K+4) \| \bm{h}_1 - \bm{h}_2 \|_2 \leq 10 K \| \bm{h}_1 - \bm{h}_2 \|_2,
\qquad \forall \bm{h}_1,\bm{h}_2 \in \cH.
\]
Let $\varepsilon={1} / ({1280 K})$, and $\tilde{\cH}$ be an $\varepsilon$-net covering $\cH$ (see \cite[Definition 5.1]{Vershynin2012}). We have
\begin{align*}
\left(
\left\{\sup_{\bm{h} \in \tilde{\cH} } \theta_1(\bm{h}) \leq \frac{1}{128} \right\} \cap \cE_0
\right)
\subseteq \left\{
	\sup_{\bm{h} \in \cH } \theta_1 \leq  \frac{1}{64}
\right\}
\end{align*}
and as a result,
\begin{align*}
\PP\left(
\sup_{\bm{h} \in \cH } \theta_1(\bm{h}) \geq \frac{1}{64}
\right)
\leq \PP\left(
	\sup_{\bm{h} \in \tilde{\cH} } \theta_1(\bm{h}) \geq \frac{1}{128}
\right) +  \PP( \cE_0^c) \leq |\tilde{\cH}|\cdot \max_{\bm{h} \in \tilde{\cH}} \PP\left( \theta_1(\bm{h}) \geq \frac{1}{128} \right) +\PP( \cE_0^c) .
\end{align*}
Lemma \ref{lemma-covariance-1205} forces that $\PP ( \cE_0^c )=O(m^{-10})$. Additionally,  
we have $\log |\tilde\cH| \leq \tilde{C}_3 K \log K$ for some absolute constant $\tilde{C}_3>0$ according to \cite[Lemma 5.2]{Vershynin2012}. Hence (\ref{ineq-hessian-norm-fixed}) leads to
\begin{align*}
|\tilde{\cH}|\cdot \max_{\bm{h} \in \tilde{\cH}} \PP\left( \theta_1(\bm{h}) \geq \frac{1}{128} \right)
&\leq 2 \exp \left(
\tilde{C}_3 K \log K + \tilde{C}_1 K - \tilde{C}_2 \frac{m (1/128) \min\left\{
	1, (1/128)/4
	\right\}}{4 C_4^2 \mu^2 \log^4 m} 
\right)\\
& \leq 2 \exp \left(
2\tilde{C}_3 K \log m -  \frac{\tilde{C}_4 m}{\mu^2 \log^4 m}
\right)
\end{align*}
for some constant $\tilde{C}_4 > 0 $. Under the sample complexity $m\gg \mu^2 K \log^5 m$, the right-hand side of the above display is at most $O\left(m^{-10}\right)$. Combine the estimates above to establish the desired high-probability bound for $\sup_{\bm{z}\in \cS } \alpha_2 $.

\item Next, we will demonstrate that 
	$$\PP( \sup_{\bm{z}\,\in\,\cS }\alpha_{3}\leq 1/96 ) = 1-O\left(m^{-10} + e^{-K} \log m \right).$$
To this end, we let
\[
\bm{A}=
\left[\begin{array}{c}
\bm{a}^\conj_1\\
\vdots\\
\bm{a}^\conj_m
\end{array}\right] \in \CC^{m\times K}
,\quad
\bm{B}=
\left[\begin{array}{c}
\bm{b}^\conj_1\\
\vdots\\
\bm{b}^\conj_m
\end{array}\right] \in \CC^{m \times K}
,\quad
\bm{C}=\left[\begin{array}{cccc}
c_{1}\left(\bm{z}\right)\\
 & c_{2}\left(\bm{z}\right)\\
 &  & \cdots\\
 &  &  & c_{m}\left(\bm{z}\right)
\end{array}\right] \in \CC^{m \times m},
\]

where for each $1\leq j \leq m$,  
$$c_j(\bm{z}) := \bm{b}_{j}^{\conj}\bm{h}\bm{x}^{\conj}\bm{a}_{j}-y_{j} = \bm{b}_{j}^{\conj} ( \bm{h}\bm{x}^{\conj}-\bm{h}^{\star} \bm{x}^{\star \conj} )\bm{a}_{j}.$$ 
As a consequence, we can write $\alpha_3 = \| \bm{B}^\conj \bm{C} \bm{A} \|$.

The key observation is that both the $\ell_{\infty}$ norm and the Frobenius norm of $\bm{C}$ are well-controlled. Specifically, we claim for the moment that with probability at least $1-O\left(m^{-10}\right)$,
\begin{subequations}\label{eq:bd-norm-constraint}
\begin{align}
\left\Vert\bm{C}\right\Vert_{\infty}=\max_{1\leq j \leq m }\left|c_{j}\right| &\leq C \frac{\mu \log^{5/2} m }{\sqrt{m}};\label{eq:bd-C-inf-constraint}\\
\left\Vert\bm{C}\right\Vert_{\mathrm{F}}^{2} = \sum_{j=1}^{m} \left|c_{j}\right|^2 &\leq 12\delta^{2},\label{eq:bd-C-F-constraint}
\end{align}
\end{subequations}
where $C > 0$ is some absolute constant. This motivates us to divide the entries in $\bm{C}$ into multiple groups based on their magnitudes. 

To be precise,  introduce $R :=1+ \lceil \log_2 ( C \mu \log^{7/2} m ) \rceil$ sets $\{\mathcal{I}_{r}\}_{1\leq r \leq R}$, where 
\[
\mathcal{I}_r = \left\{
j\in[m]:\text{ } \frac{C \mu \log^{5/2} m }{2^r \sqrt{m} } < |c_j| \leq \frac{C \mu \log^{5/2} m }{2^{r-1} \sqrt{m} }
\right\}, \qquad 1\leq r \leq R-1
\]
and $\mathcal{I}_R = \{1,\cdots, m\} \,\backslash\, \big( \bigcup_{r=1}^{R-1} \mathcal{I}_r \big)$. An immediate consequence of the definition of $\mathcal{I}_{r}$ and the norm constraints in (\ref{eq:bd-norm-constraint}) is the following cardinality bound
\begin{align}
 |\mathcal{I}_r| \leq \frac{ \left\| \bm{C} \right\|_{\mathrm{F}}^2 }{ \min_{j \in \mathcal{I}_{r}}\left|c_{j}\right|^2}
\leq \frac{ 12\delta^2 }{\left( \frac{C \mu \log^{5/2} m }{2^{r} \sqrt{m} }\right)^2}
 = \underbrace{  \frac{ 12 \delta^2 4^r }{ C^2 \mu^2 \log^5 m } }_{\delta_r}  m
\label{eq:bd-cardinality}
\end{align}
for $1\leq r \leq R-1$.
Since $\{\mathcal{I}_{r}\}_{1\leq r \leq R}$ form a partition of the index set $\{1,\cdots,m\}$, it is easy to see that
\[
\bm{B}^\conj \bm{C} \bm{A}= \sum_{r=1}^{R} (\bm{B}_{\mathcal{I}_r,\cdot})^\conj \bm{C}_{\mathcal{I}_r,\mathcal{I}_r} \bm{A}_{\mathcal{I}_r,\cdot},
\]
where $\bm{D}_{\mathcal{I},\mathcal{J}}$ denotes the submatrix of $\bm{D}$ induced by the rows and columns of $\bm{D}$  having indices from $\mathcal{I}$ and $\mathcal{J}$, respectively, and $\bm{D}_{\mathcal{I},\cdot}$ refers to the submatrix formed by the rows from the index set $\mathcal{I}$. As a result, one can invoke the triangle inequality to derive
\begin{align}
 \alpha_3 \leq \sum_{r=1}^{R-1} \left\| \bm{B}_{\mathcal{I}_r,\cdot} \right\| \cdot \left\| \bm{C}_{\mathcal{I}_r,\mathcal{I}_r} \right\| \cdot \left\| \bm{A}_{\mathcal{I}_r,\cdot}\right\| +\left\| \bm{B}_{\mathcal{I}_R,\cdot} \right\| \cdot  \left\| \bm{C}_{\mathcal{I}_R,\mathcal{I}_R} \right\| \cdot \left\| \bm{A}_{\mathcal{I}_R,\cdot}\right\|.\label{ineq-hessian-decomposition-1206}
\end{align}
Recognizing that $\bm{B}^\conj \bm{B} = \bm{I}_K$, we obtain 
$$\left\| \bm{B}_{\mathcal{I}_r,\cdot} \right\| \leq \left\| \bm{B} \right\| =1$$ 
for every $1\leq r \leq R$. In addition, by  construction of $\mathcal{I}_{r}$, we have
\begin{align}
\left\| \bm{C}_{\mathcal{I}_r,\mathcal{I}_r} \right\|
= \max_{j\in \mathcal{I}_r} |c_j|
\leq \frac{C \mu \log^{5/2} m }{2^{r-1} \sqrt{m} }\notag
\end{align}
for $1\leq r \leq R$, and specifically for $R$, one has
\begin{align*}
\left\| \bm{C}_{\mathcal{I}_R,\mathcal{I}_R} \right\|
=   \max_{j\in \mathcal{I}_R} |c_j| \leq  \frac{C \mu \log^{5/2} m }{2^{R-1} \sqrt{m} } \leq \frac{1}{\sqrt{m} \log m},
\end{align*}
which follows from the definition of $R$, i.e.~$R=1+ \lceil \log_2 ( C \mu \log^{7/2} m ) \rceil$. 
Regarding $\left\| \bm{A}_{\mathcal{I}_r,\cdot} \right\|$, we discover that $\left\| \bm{A}_{I_R,\cdot} \right\| \leq \left\| \bm{A} \right\|$ and in view of (\ref{eq:bd-cardinality}), 
\[
	\left\| \bm{A}_{\mathcal{I}_r,\cdot} \right\| \leq \sup_{\mathcal{I}:|\mathcal{I}|\leq \delta_{r}m}\left\| \bm{A}_{\mathcal{I},\cdot} \right\|, \qquad 1\leq r \leq R-1.
\]
Substitute the above estimates  into (\ref{ineq-hessian-decomposition-1206}) to get
\begin{align}
\alpha_{3} \leq \sum_{r=1}^{R-1} \frac{C \mu \log^{5/2} m }{2^{r-1} \sqrt{m} }  \sup_{\mathcal{I}:|\mathcal{I}|\leq \delta_{r}m}\left\| \bm{A}_{\mathcal{I},\cdot} \right\| + \frac{\left\| \bm{A} \right\|}{\sqrt{m} \log m}\label{eq:bd-alpha3-upper-bound}.
\end{align}
It remains to upper bound $\left\| \bm{A} \right\|$ and $\sup_{\mathcal{I}: |\mathcal{I}|\leq \delta_{r}m}\left\| \bm{A}_{\mathcal{I},\cdot} \right\|$. Lemma~\ref{lemma-covariance-1205} tells us that $\left\| \bm{A} \right\| \leq 2\sqrt{m}$ with probability at least $1-O\left(m^{-10}\right)$. Furthermore, we can invoke Lemma~\ref{lemma-covariance-union-1206} to bound $\sup_{\mathcal{I}: |\mathcal{I}|\leq \delta_{r}m}\left\| \bm{A}_{\mathcal{I},\cdot} \right\|$ for each $1\leq r \leq R-1$. It is easily seen from our assumptions $m \gg \mu^2 K \log^9 m$ and $\delta = c/\log^2 m$ that $\delta_r \gg K/m$. In addition,
\[
\delta_r \leq \frac{ 12 \delta^2 4^{R-1} }{C^2 \mu^2 \log^5 m }
\leq \frac{ 12 \delta^2 4^{1+ \log_2 ( C \mu \log^{7/2} m )} }{C^2 \mu^2 \log^5 m }
=48\delta^2\log^2 m = \frac{48c}{\log^2 m}\ll 1.
\]
By Lemma \ref{lemma-covariance-union-1206} we obtain that for some constants $\tilde{C}_{2}, \tilde{C}_{3}>0$
\begin{align*}
	\PP\left(\sup_{\mathcal{I}: |\mathcal{I}|\leq \delta_{r}m}\left\| \bm{A}_{\mathcal{I},\cdot} \right\|  \geq \sqrt{4 \tilde{C}_3 \delta_r m \log(e/\delta_r) }\right) &\leq 2\exp \left( - \frac{\tilde{C}_2  \tilde{C}_3}{3} \delta_r m \log(e/ \delta_r ) \right) \\
	&\leq 2\exp \left( - \frac{\tilde{C}_2  \tilde{C}_3}{3} \delta_r m\right) \leq 2 e^{-K}.
\end{align*}
Taking the union bound and substituting the estimates above into (\ref{eq:bd-alpha3-upper-bound}), we see  that with probability at least $1-O\left(m^{-10}\right)- O\left((R-1)e^{-K}\right)$,  
\begin{align*}
\alpha_{3} &\leq \sum_{r=1}^{R-1} \frac{C \mu \log^{5/2} m }{2^{r-1} \sqrt{m} } \cdot \sqrt{4 \tilde{C}_3 \delta_r m \log(e/\delta_r) }+ \frac{2\sqrt{m}}{\sqrt{m} \log m} \\
&\leq
\sum_{r=1}^{R-1} 4\delta \sqrt{12 \tilde{C}_3 \log(e/\delta_r) } + \frac{2}{ \log m} \\
&\lesssim (R-1) \delta  \sqrt{\log(e/\delta_1)} + \frac{1}{\log m}.
\end{align*}
Note that $\mu\leq \sqrt{m}$, $R-1=\lceil \log_2 ( C \mu \log^{7/2} m ) \rceil \lesssim \log m$, and
\[
	\sqrt{\log \frac{e}{\delta_1}} = \sqrt{\log\left(
\frac{ eC^2 \mu^2 \log^5 m }{ 48 \delta^2 }
	\right)} \lesssim \log m.
\]
Therefore,  with probability exceeding $1-O\left(m^{-10}\right)- O\left(e^{-K}\log m\right)$, 
\begin{align*}
\sup_{\bm{z} \in \cS } \alpha_3 \lesssim \delta \log^2 m + \frac{1}{\log m}.
\end{align*}
By taking $c$ to be small enough in $\delta = c/ \log^2 m$, we get 
$$\PP \left(
\sup_{\bm{z} \in \cS } \alpha_3 \geq 1/96 \right) \leq O\left(m^{-10}\right)+ O\left(e^{-K}\log m\right)$$ as claimed.  

Finally, it remains to  justify (\ref{eq:bd-norm-constraint}).
For all $\bm{z}\in\cS$, the triangle inequality tells us that 
\begin{align*}
|c_j| &\leq \left| \bm{b}_j^\conj \bm{h} ( \bm{x} - \bm{x}^{\star} )^{\conj} \bm{a}_j \right| 
+ \left| \bm{b}_j^\conj ( \bm{h}- \bm{h}^{\star} ) \bm{x}^{\star \conj} \bm{a}_j \right| \\
&\leq
\left| \bm{b}_j^\conj \bm{h}\right| \cdot \left|\bm{a}_j^\conj ( \bm{x} - \bm{x}^{\star} )\right| 
+ \left( \left| \bm{b}_j^\conj \bm{h}\right| + \left| \bm{b}_j^\conj \bm{h}^{\star}  \right| \right) \cdot \left| \bm{a}_j^\conj \bm{x}^{\star} \right| \\
&\leq  \frac{2C_4 \mu \log^2 m}{\sqrt{m}} \cdot \frac{2C_3}{\log^{3/2} m } + \left(
\frac{2C_4 \mu \log^2 m}{\sqrt{m}} + \frac{\mu}{\sqrt{m}}
\right)5 \sqrt{\log m} \\
&\leq C \frac{\mu \log^{5/2} m }{\sqrt{m}},
\end{align*}
for some large constant $C>0$, where we have used the definition of $\cS$ and the fact (\ref{eq:max_gaussian-1}). The claim (\ref{eq:bd-C-F-constraint}) follows directly from \cite[Lemma 5.14]{DBLP:journals/corr/LiLSW16}. To avoid confusion, we use $\mu_1$ to refer to the parameter $\mu$ therein. Let $L=m$, $N=K$, $d_0=1$, $\mu_1=C_4 \mu \log^2 m / 2$, and $\varepsilon=1/15$. Then 
$$\cS \subseteq \mathcal{N}_{d_0} \cap \mathcal{N}_{\mu_1} \cap \mathcal{N}_{\varepsilon},$$ 
and the sample complexity condition $L \gg \mu_1^2 (K+N) \log^2 L $ is satisfied because we have assumed $m\gg \mu^2 K \log^6 m$. Therefore with probability exceeding $1-O\left(m^{-10}+e^{-K}\right)$, we obtain that for all $\bm{z}\in\cS $,
\begin{align*}
\left\|\bm{C}\right\|_{\mathrm{F}}^{2} \leq \frac{5}{4}
\left\Vert \bm{h}\bm{x}^{\conj}-\bm{h}^{\star}\bm{x}^{\star\conj}\right\Vert_{\mathrm{F}}^2.
\end{align*}
The claim (\ref{eq:bd-C-F-constraint}) can then be justified by observing that 
\begin{align*}
\left\Vert \bm{h}\bm{x}^{\conj}-\bm{h}^{\star}\bm{x}^{\star\conj}\right\Vert_{\mathrm{F}} =\left\Vert \bm{h}\left(\bm{x}-\bm{x}^{\star}\right)^{\conj}+\left(\bm{h}-\bm{h}^{\star}\right)\bm{x}^{\star\conj}\right\Vert _{\mathrm{F}}
\leq\left\Vert \bm{h}\right\Vert _{2}\left\Vert \bm{x}-\bm{x}^{\star}\right\Vert _{2}+\left\Vert \bm{h}-\bm{h}^{\star}\right\Vert _{2}\left\Vert \bm{x}^{\star}\right\Vert _{2}\leq 3\delta.
\end{align*}
%

\item It remains to control $\alpha_{4}$, for which we make note of the
following inequality 
\[
\alpha_{4}\leq
\underbrace{
\left\Vert \sum_{j=1}^{m}\bm{b}_{j}\bm{b}_{j}^{\conj}(\bm{h}\bm{x}^{\top}-\bm{h}^{\star}\bm{x}^{\star\top})\overline{\bm{a}_{j}}\,\overline{\bm{a}_{j}}^{\conj}\right\Vert
}_{\theta_3}
+
\underbrace{
\left\Vert \sum_{j=1}^{m}\bm{b}_{j}\bm{b}_{j}^{\conj}\bm{h}^{\star}\bm{x}^{\star\top}(\overline{\bm{a}_{j}}\,\overline{\bm{a}_{j}}^{\conj}-\bm{I}_{K})\right\Vert
}_{\theta_4}
\]
with $\overline{\bm{a}_{j}}$ denoting the entrywise conjugate of
$\bm{a}_{j}$. Since $\{\overline{\bm{a}_{j}}\}$ has the same joint
distribution as $\{\bm{a}_{j}\}$, by the same argument used for bounding
$\alpha_{3}$ we obtain control of the first term, namely, 
\[
\PP\left(\sup_{\bm{z}\in\cS}
\theta_3 \geq 1/96 \right) = O(m^{-10}+e^{-K}\log m).
\]
Note that $m \gg \mu^2 K \log m / \delta^2$ and $\delta\ll 1$. According to \cite[Lemma 5.20]{DBLP:journals/corr/LiLSW16},
\[
\PP\left(\sup_{\bm{z}\in\cS}
\theta_4 \geq 1/96 \right) \leq 
\PP\left(\sup_{\bm{z}\in\cS}
\theta_4 \geq \delta \right) = O(m^{-10}).
\]
Putting together the above bounds, we reach $\PP( \sup_{\bm{z} \in\cS } \alpha_{4}\leq 1/48) = 1-O(m^{-10}+e^{-K}\log m)$. 

\item Combining all the previous bounds for $\sup_{\bm{z}\in\cS} \alpha_j$ and (\ref{ineq-hessian-decomposition}), we deduce that with probability
$1-O(m^{-10}+e^{-K}\log m)$, 
\[
\left\Vert \nabla^{2}f\left(\bm{z}\right)-\nabla^{2}F\left(\bm{z}^{\star}\right)\right\Vert \leq
2\cdot\frac{1}{32}+2\cdot\frac{1}{32}+4\cdot\frac{1}{96}+4\cdot\frac{1}{48}= \frac{1}{4}.
\]
\end{enumerate}
 
\subsection{Proofs of Lemma \ref{lemma:inductive-ell-2} and Lemma \ref{lemma:alpha-t+1-BD}\label{subsec:Proof-of-Lemma-inductive-ell-2}}

\begin{proof}[Proof of Lemma \ref{lemma:inductive-ell-2}]In view of the
definition of $\alpha^{t+1}$ (see (\ref{eq:defn-alphat})), one has 
\[
 \text{dist}\left(\bm{z}^{t+1},\bm{z}^{\star}\right)^{2}=\left\Vert \frac{1}{\overline{\alpha^{t+1}}}\bm{h}^{t+1}-\bm{h}^{\star}\right\Vert _{2}^{2}+\left\Vert \alpha^{t+1}\bm{x}^{t+1}-\bm{x}^{\star}\right\Vert _{2}^{2}\leq\left\Vert \frac{1}{\overline{\alpha^{t}}}\bm{h}^{t+1}-\bm{h}^{\star}\right\Vert _{2}^{2}+\left\Vert \alpha^{t}\bm{x}^{t+1}-\bm{x}^{\star}\right\Vert _{2}^{2}.
\]
The gradient update rules (\ref{eq:gradient-update-Bd-explicit}) imply that 
\begin{align*}
\frac{1}{\overline{\alpha^{t}}}\bm{h}^{t+1} & =\frac{1}{\overline{\alpha^{t}}}\left(\bm{h}^{t}-\frac{\eta}{\|\bm{x}^{t}\|_{2}^{2}}\nabla_{\bm{h}}f\left(\bm{z}^{t}\right)\right)=\tilde{\bm{h}}^{t}-\frac{\eta}{\left\|\tilde{\bm{x}}^{t}\right\|_{2}^{2}}\nabla_{\bm{h}}f\left(\tilde{\bm{z}}^{t}\right),\\
\alpha^{t}\bm{x}^{t+1} & =\alpha^{t}\left(\bm{x}^{t}-\frac{\eta}{\|\bm{h}^{t}\|_{2}^{2}}\nabla_{\bm{x}}f\left(\bm{z}^{t}\right)\right)=\tilde{\bm{x}}^{t}-\frac{\eta}{\|\tilde{\bm{h}}^{t}\|_{2}^{2}}\nabla_{\bm{x}}f\left(\tilde{\bm{z}}^{t}\right),
\end{align*}
where we denote $\tilde{\bm{h}}^{t}=\frac{1}{\overline{\alpha^{t}}}\bm{h}^{t}$
and $\tilde{\bm{x}}^{t}=\alpha^{t}\bm{x}^{t}$ as in \eqref{eq:tilde-notaion-BD}. Let $\hat{\bm{h}}^{t+1} = \frac{1}{\overline{\alpha^{t}}}\bm{h}^{t+1}$ and $\hat{\bm{x}}^{t+1} = \alpha^{t}\bm{x}^{t+1}$. We further get 
\begin{align}
\left[\begin{array}{c}
 \hat{\bm{h}}^{t+1} -\bm{h}^{\star}\\
 \hat{\bm{x}}^{t+1} -\bm{x}^{\star}\\
\overline{ \hat{\bm{h}}^{t+1} -\bm{h}^{\star}}\\
\overline{ \hat{\bm{x}}^{t+1} -\bm{x}^{\star}}
\end{array}\right] & =\left[\begin{array}{c}
\tilde{\bm{h}}^{t}-\bm{h}^{\star}\\
\tilde{\bm{x}}^{t}-\bm{x}^{\star}\\
\overline{\tilde{\bm{h}}^{t}-\bm{h}^{\star}}\\
\overline{\tilde{\bm{x}}^{t}-\bm{x}^{\star}}
\end{array}\right]-\eta\underbrace{\left[\begin{array}{cccc}
\left\Vert \tilde{\bm{x}}^{t}\right\Vert _{2}^{-2}\bm{I}_{K}\\
 & \big\|\tilde{\bm{h}}^{t}\big\|_{2}^{-2}\bm{I}_{K}\\
 &  & \left\Vert \tilde{\bm{x}}^{t}\right\Vert _{2}^{-2}\bm{I}_{K}\\
 &  &  & \big\|\tilde{\bm{h}}^{t}\big\|_{2}^{-2}\bm{I}_{K}
\end{array}\right]}_{:=\bm{D}}\left[\begin{array}{c}
\nabla_{\bm{h}}f\left(\tilde{\bm{z}}^{t}\right)\\
\nabla_{\bm{x}}f\left(\tilde{\bm{z}}^{t}\right)\\
\overline{\nabla_{\bm{h}}f\left(\tilde{\bm{z}}^{t}\right)}\\
\overline{\nabla_{\bm{x}}f\left(\tilde{\bm{z}}^{t}\right)}
\end{array}\right].\label{eq:ell-2-gradient-equality-BD}
\end{align}
The fundamental theorem of calculus (see Appendix \ref{sec:Wirtinger-calculus})
together with the fact that $\nabla f\left(\bm{z}^{\star}\right)=\bm{0}$
tells us 
\begin{equation}
\left[\begin{array}{c}
\nabla_{\bm{h}}f\left(\tilde{\bm{z}}^{t}\right)\\
\nabla_{\bm{x}}f\left(\tilde{\bm{z}}^{t}\right)\\
\overline{\nabla_{\bm{h}}f\left(\tilde{\bm{z}}^{t}\right)}\\
\overline{\nabla_{\bm{x}}f\left(\tilde{\bm{z}}^{t}\right)}
\end{array}\right]=\left[\begin{array}{c}
\nabla_{\bm{h}}f\left(\tilde{\bm{z}}^{t}\right)-\nabla_{\bm{h}}f\left(\bm{z}^{\star}\right)\\
\nabla_{\bm{x}}f\left(\tilde{\bm{z}}^{t}\right)-\nabla_{\bm{x}}f\left(\bm{z}^{\star}\right)\\
\overline{\nabla_{\bm{h}}f\left(\tilde{\bm{z}}^{t}\right)-\nabla_{\bm{h}}f\left(\bm{z}^{\star}\right)}\\
\overline{\nabla_{\bm{x}}f\left(\tilde{\bm{z}}^{t}\right)-\nabla_{\bm{x}}f\left(\bm{z}^{\star}\right)}
\end{array}\right]=\underbrace{\int_{0}^{1}\nabla^{2}f\left({\bm{z}}\left(\tau\right)\right)\mathrm{d}\tau}_{:=\bm{A}}\left[\begin{array}{c}
\tilde{\bm{h}}^{t}-\bm{h}^{\star}\\
\tilde{\bm{x}}^{t}-\bm{x}^{\star}\\
\overline{\tilde{\bm{h}}^{t}-\bm{h}^{\star}}\\
\overline{\tilde{\bm{x}}^{t}-\bm{x}^{\star}}
\end{array}\right],\label{eq:ell-2-mean-value-BD}
\end{equation}
where we denote ${\bm{z}}\left(\tau\right):=\bm{z}^{\star}+\tau\left(\tilde{\bm{z}}^{t}-\bm{z}^{\star}\right)$
and $\nabla^{2}f$ is the Wirtinger Hessian. To further simplify notation, denote $\hat{\bm{z}}^{t+1}=\begin{bmatrix}
 \hat{\bm{h}}^{t+1}\\
 \hat{\bm{x}}^{t+1}
 \end{bmatrix}$. The identity (\ref{eq:ell-2-mean-value-BD}) allows us to rewrite (\ref{eq:ell-2-gradient-equality-BD}) as 
\begin{equation}
\left[\begin{array}{c}
\hat{\bm{z}}^{t+1}-\bm{z}^{\star}\\
\overline{\hat{\bm{z}}^{t+1} -\bm{z}^{\star}}
\end{array}\right]=\left(\bm{I}-\eta\bm{D}\bm{A}\right)\left[\begin{array}{c}
\tilde{\bm{z}}^{t}-\bm{z}^{\star}\\
\overline{\tilde{\bm{z}}^{t}-\bm{z}^{\star}}
\end{array}\right].\label{eq:ell_2_gradient_equality-BD-1}
\end{equation}

Take the squared
Euclidean norm of both sides of (\ref{eq:ell_2_gradient_equality-BD-1})
to reach 
\begin{align}
 & \left\Vert \hat{\bm{z}}^{t+1}-\bm{z}^{\star}\right\Vert _{2}^{2}=\frac{1}{2}\left[\begin{array}{c}
\tilde{\bm{z}}^{t}-\bm{z}^{\star}\\
\overline{\tilde{\bm{z}}^{t}-\bm{z}^{\star}}
\end{array}\right]^{\conj}\left(\bm{I}-\eta\bm{D}\bm{A}\right)^{\conj}\left(\bm{I}-\eta\bm{D}\bm{A}\right)\left[\begin{array}{c}
\tilde{\bm{z}}^{t}-\bm{z}^{\star}\\
\overline{\tilde{\bm{z}}^{t}-\bm{z}^{\star}}
\end{array}\right]\nonumber \\
 & \quad=\frac{1}{2}\left[\begin{array}{c}
\tilde{\bm{z}}^{t}-\bm{z}^{\star}\\
\overline{\tilde{\bm{z}}^{t}-\bm{z}^{\star}}
\end{array}\right]^{\conj}\left(\bm{I}+\eta^{2}\bm{A}\bm{D}^{2}\bm{A}-\eta\left(\bm{D}\bm{A}+\bm{A}\bm{D}\right)\right)\left[\begin{array}{c}
\tilde{\bm{z}}^{t}-\bm{z}^{\star}\\
\overline{\tilde{\bm{z}}^{t}-\bm{z}^{\star}}
\end{array}\right]\nonumber \\
 & \quad\leq(1+\eta^{2}\|\bm{A}\|^{2}\|\bm{D}\|^{2})\left\Vert \tilde{\bm{z}}^{t}-\bm{z}^{\star}\right\Vert _{2}^{2}-\frac{\eta}{2}\left[\begin{array}{c}
\tilde{\bm{z}}^{t}-\bm{z}^{\star}\\
\overline{\tilde{\bm{z}}^{t}-\bm{z}^{\star}}
\end{array}\right]^{\conj}(\bm{D}\bm{A}+\bm{A}\bm{D})\left[\begin{array}{c}
\tilde{\bm{z}}^{t}-\bm{z}^{\star}\\
\overline{\tilde{\bm{z}}^{t}-\bm{z}^{\star}}
\end{array}\right].\label{eq:zt-contraction-BD}
\end{align}
Since $\bm{z}\left(\tau\right)$ lies between $\tilde{\bm{z}}^{t}$
and $\bm{z}^{\star}$, we conclude from the assumptions (\ref{eq:contraction-BD}) that for all $0\leq \tau \leq 1$, 
\begin{align*}
\max\left\{ \left\Vert \bm{h}\left(\tau\right)-\bm{h}^{\star}\right\Vert _{2},\left\Vert \bm{x}\left(\tau\right)-\bm{x}^{\star}\right\Vert _{2}\right\}  & \leq\text{dist}\left(\bm{z}^{t},\bm{z}^{\star}\right)\leq\xi\leq\delta;\\
\max_{1\leq j\leq m}\left|\bm{a}_{j}^{\conj}\left(\bm{x}\left(\tau\right)-\bm{x}^{\star}\right)\right| & \leq C_{3}\frac{1}{\log^{3/2}m};\\
\max_{1\leq j\leq m}\left|\bm{b}_{j}^{\conj}\bm{h}\left(\tau\right)\right| & \leq C_{4}\frac{\mu}{\sqrt{m}}\log^{2}m
\end{align*}
for $\xi>0$ sufficiently small. Moreover, it is straightforward to
see that 
\[
\gamma_{1}:=\big\Vert \tilde{\bm{x}}^{t}\big\Vert _{2}^{-2}\qquad\text{and}\qquad\gamma_{2}:=\big\|\tilde{\bm{h}}^{t}\big\|_{2}^{-2}
\]
satisfy 
\[
\max\left\{ \left|\gamma_{1}-1\right|,\left|\gamma_{2}-1\right|\right\} \lesssim\max\left\{ \big\|\tilde{\bm{h}}^{t}-\bm{h}^{\star}\big\|_{2},\big\|\tilde{\bm{x}}^{t}-\bm{x}^{\star}\big\|_{2}\right\} \leq\delta
\]
as long as $\xi>0$ is sufficiently small. We can now readily invoke
Lemma \ref{lemma:hessian-bd} to arrive at 
\[
\left\Vert \bm{A}\right\Vert \left\Vert \bm{D}\right\Vert \leq 3 (1+\delta) \leq 4\qquad\text{and}
\]
\[
\left[\begin{array}{c}
\tilde{\bm{z}}^{t}-\bm{z}^{\star}\\
\overline{\tilde{\bm{z}}^{t}-\bm{z}^{\star}}
\end{array}\right]^{\conj}(\bm{D}\bm{A}+\bm{A}\bm{D})\left[\begin{array}{c}
\tilde{\bm{z}}^{t}-\bm{z}^{\star}\\
\overline{\tilde{\bm{z}}^{t}-\bm{z}^{\star}}
\end{array}\right]\geq\frac{1}{4}\left\Vert \left[\begin{array}{c}
\tilde{\bm{z}}^{t}-\bm{z}^{\star}\\
\overline{\tilde{\bm{z}}^{t}-\bm{z}^{\star}}
\end{array}\right]\right\Vert _{2}^{2}=\frac{1}{2}\left\Vert \tilde{\bm{z}}^{t}-\bm{z}^{\star}\right\Vert _{2}^{2}.
\]

Substitution into (\ref{eq:zt-contraction-BD}) indicates that 
\begin{align*}
\left\Vert \hat{\bm{z}}^{t+1}-\bm{z}^{\star}\right\Vert _{2}^{2}\leq\left(1+16\eta^{2}-\eta / 4\right)\left\Vert \tilde{\bm{z}}^{t}-\bm{z}^{\star}\right\Vert _{2}^{2}.
\end{align*}
When $0<\eta\leq {1}/{128}$, this implies that 
\[
\left\Vert \hat{\bm{z}}^{t}-\bm{z}^{\star}\right\Vert _{2}^{2}\leq(1-\eta/8)\left\Vert \tilde{\bm{z}}^{t}-\bm{z}^{\star}\right\Vert _{2}^{2},
\]
and hence 
\begin{equation}
\left\Vert \tilde{\bm{z}}^{t+1}-\bm{z}^{\star}\right\Vert _{2}\leq\left\Vert \hat{\bm{z}}^{t+1}-\bm{z}^{\star}\right\Vert _{2}\leq\left(1-\eta/8\right)^{1/2}\left\Vert \tilde{\bm{z}}^{t}-\bm{z}^{\star}\right\Vert _{2}\leq\left(1-\eta/16\right)\mathrm{dist}(\bm{z}^{t},\bm{z}^{\star}).\label{eq:ell_2_contraction_last_equation-BD}
\end{equation}
This completes the proof of Lemma \ref{lemma:inductive-ell-2}.\end{proof}

\begin{proof}[Proof of Lemma \ref{lemma:alpha-t+1-BD}]Reuse the
notation in this subsection, namely, $\hat{\bm{z}}^{t+1}=\left[\begin{array}{c}
\hat{\bm{h}}^{t+1}\\
\hat{\bm{x}}^{t+1}
\end{array}\right]$ with $\hat{\bm{h}}^{t+1}=\frac{1}{\overline{\alpha^{t}}}\bm{h}^{t+1}$
and $\hat{\bm{x}}^{t+1}=\alpha^{t}\bm{x}^{t+1}$. From (\ref{eq:ell_2_contraction_last_equation-BD}), one can tell that
\begin{align*}
\left\Vert \tilde{\bm{z}}^{t+1}-\bm{z}^{\star}\right\Vert _{2}\leq\left\Vert \hat{\bm{z}}^{t+1}-\bm{z}^{\star}\right\Vert _{2}\leq\mathrm{dist}(\bm{z}^{t},\bm{z}^{\star}).
\end{align*}
Invoke Lemma \ref{lemma:alpha-close-to-one} with $\beta=\alpha^{t}$
to get 
\[
\left|\alpha^{t+1}-\alpha^{t}\right|\lesssim\left\Vert \hat{\bm{z}}^{t+1}-\bm{z}^{\star}\right\Vert _{2}\leq \mathrm{dist}(\bm{z}^{t},\bm{z}^{\star}).
\]
This combined with the assumption $||\alpha^{t}|-1|\leq 1/2$ implies
that 
\[
\left|\alpha^{t}\right|\geq\frac{1}{2}\qquad\text{and}\qquad\left|\frac{\alpha^{t+1}}{\alpha^{t}}-1\right|=\left|\frac{\alpha^{t+1}-\alpha^{t}}{\alpha^{t}}\right|\lesssim\mathrm{dist}(\bm{z}^{t},\bm{z}^{\star})\lesssim C_{1}\frac{1}{\log^{2}m}.
\]
This finishes the proof of the first claim.

The second claim can be proved by induction. Suppose that $\big||\alpha^{s}|-1\big|\leq1/2$
and $\mathrm{dist}(\bm{z}^{s},\bm{z}^{\star})\leq C_{1}(1-\eta/16)^{s}/\log^{2}m$
hold for all $0\leq s\leq\tau\leq t$ , then using our result in the
first part gives 
\begin{align*}
\big||\alpha^{\tau+1}|-1\big| & \leq\big||\alpha^{0}|-1\big|+\sum_{s=0}^{\tau}\big|\alpha^{s+1}-\alpha^{s}\big|\leq\frac{1}{4}+c\sum_{s=0}^{\tau}\mathrm{dist}(\bm{z}^{s},\bm{z}^{\star})\\
 & \leq\frac{1}{4}+\frac{cC_{1}}{\frac{\eta}{16}\log^{2}m}\leq \frac{1}{2}
\end{align*}
for $m$ sufficiently large. The proof is then complete by induction.
\end{proof}

\subsection{Proof of Lemma \ref{lemma:inductive-loop}\label{subsec:Proof-of-Lemma-inductive-loop}}

Define the alignment parameter between $\bm{z}^{t,\left(l\right)}$
and $\tilde{\bm{z}}^{t}$ as
\[
\alpha_{\text{mutual}}^{t,\left(l\right)}:=\argmin_{\alpha\in\mathbb{C}}  \left\|\frac{1}{\overline{\alpha}}\bm{h}^{t,\left(l\right)} - \frac{1}{\overline{\alpha^{t}}}\bm{h}^{t} \right\|_2^2 + \left\| \alpha\bm{x}^{t,\left(l\right)} - \alpha^{t}\bm{x}^{t} \right\|_2^2. \]
Further denote, for simplicity of presentation, ${\hat{\bm{z}}^{t,(l)}=\begin{bmatrix}
\hat{\bm{h}}^{t,\left(l\right)}\\
\hat{\bm{x}}^{t,\left(l\right)}
\end{bmatrix}}$  with
\[
\hat{\bm{h}}^{t,\left(l\right)}:=\frac{1}{\overline{\alpha_{\text{mutual}}^{t,\left(l\right)}}}\bm{h}^{t,\left(l\right)}\qquad\text{and}\qquad\hat{\bm{x}}^{t,\left(l\right)}:=\alpha_{\text{mutual}}^{t,\left(l\right)}\bm{x}^{t,\left(l\right)}.
\]
Clearly, $\hat{\bm{z}}^{t,(l)}$ is aligned with $\tilde{\bm{z}}^{t}$.

Armed with the above notation, we have
\begin{align}
& \text{dist}\big(\bm{z}^{t+1,\left(l\right)},\tilde{\bm{z}}^{t+1}\big)  =\min_{\alpha}\sqrt{ \left\Vert 
\frac{1}{\overline{\alpha}} \bm{h}^{t+1,\left(l\right)} - \frac{1}{\overline{\alpha^{t+1}}}\bm{h}^{t+1} \right\Vert_2^2 +\left\Vert 
 \alpha \bm{x}^{t+1,\left(l\right)} -  \alpha^{t+1} \bm{x}^{t+1} \right\Vert_2^2   }\nonumber \\
 & \quad =\min_{\alpha}\sqrt{ \left\Vert  \left(\frac{\overline{\alpha^{t}}}{\overline{\alpha^{t+1}}} \right) \left(
\frac{1}{\overline{\alpha}} \frac{\overline{\alpha^{t+1}}}{\overline{\alpha^{t}}} \bm{h}^{t+1,\left(l\right)} - \frac{1}{\overline{\alpha^{t}}}\bm{h}^{t+1} \right) \right\Vert_2^2 +\left\Vert  \left(\frac{ \alpha^{t+1}}{ \alpha^{t}} \right)
\left( \alpha \frac{ \alpha^{t}}{ \alpha^{t+1}}  \bm{x}^{t+1,\left(l\right)} -  \alpha^{t} \bm{x}^{t+1}\right) \right\Vert_2^2   }\nonumber  \\
 & \quad \leq \sqrt{ \left\Vert  \left(\frac{\overline{\alpha^{t}}}{\overline{\alpha^{t+1}}} \right) \left(
\frac{1}{\overline{\alpha_{\text{mutual}}^{t,\left(l\right)}}}  \bm{h}^{t+1,\left(l\right)} - \frac{1}{\overline{\alpha^{t}}}\bm{h}^{t+1} \right) \right\Vert_2^2 +\left\Vert  \left(\frac{ \alpha^{t+1}}{ \alpha^{t}} \right)
\left( \alpha_{\text{mutual}}^{t,\left(l\right)} \bm{x}^{t+1,\left(l\right)} -  \alpha^{t} \bm{x}^{t+1}\right) \right\Vert_2^2   }\label{eq:loop-equality-1-BD-1}   \\
 & \quad\leq\max\left\{ \left|\frac{\alpha^{t+1}}{\alpha^{t}}\right|,\left|\frac{\alpha^{t}}{\alpha^{t+1}}\right|\right\} \left\Vert \left[\begin{array}{c}
\frac{1}{\overline{\alpha_{\text{mutual}}^{t,\left(l\right)}}}\bm{h}^{t+1,\left(l\right)}-\frac{1}{\overline{\alpha^{t}}}\bm{h}^{t+1}\\
\alpha_{\text{mutual}}^{t,\left(l\right)}\bm{x}^{t+1,\left(l\right)}-\alpha^{t}\bm{x}^{t+1}
\end{array}\right]\right\Vert _{2},\label{eq:loop-UB-1-BD-1}
\end{align}
where \eqref{eq:loop-equality-1-BD-1} follows by taking $\alpha=\frac{\alpha^{t+1}}{\alpha^{t}}\alpha_{\text{mutual}}^{t,\left(l\right)}$. The latter bound is more convenient to work with when
controlling the gap between $\bm{z}^{t,(l)}$ and $\bm{z}^{t}$.

We can then apply the gradient update rules (\ref{eq:gradient-update-Bd-explicit}) and (\ref{eq:loo-gradient_update}) to get 
\begin{align*}
 & \left[\begin{array}{c}
\frac{1}{\overline{\alpha_{\text{mutual}}^{t,\left(l\right)}}}\bm{h}^{t+1,\left(l\right)}-\frac{1}{\overline{\alpha^{t}}}\bm{h}^{t+1}\\
\alpha_{\text{mutual}}^{t,\left(l\right)}\bm{x}^{t+1,\left(l\right)}-\alpha^{t}\bm{x}^{t+1}
\end{array}\right]\\
 & \quad=\left[\begin{array}{c}
\frac{1}{\overline{\alpha_{\text{mutual}}^{t,\left(l\right)}}}\left(\bm{h}^{t,\left(l\right)}-\frac{\eta}{\left\Vert \bm{x}^{t,\left(l\right)}\right\Vert _{2}^{2}}\nabla_{\bm{h}}f^{\left(l\right)}\left(\bm{h}^{t,\left(l\right)},\bm{x}^{t,\left(l\right)}\right)\right)-\frac{1}{\overline{\alpha^{t}}}\left(\bm{h}^{t}-\frac{\eta}{\left\Vert \bm{x}^{t}\right\Vert _{2}^{2}}\nabla_{\bm{h}}f\left(\bm{h}^{t},\bm{x}^{t}\right)\right)\\
\alpha_{\text{mutual}}^{t,\left(l\right)}\left(\bm{x}^{t,\left(l\right)}-\frac{\eta}{\left\Vert \bm{h}^{t,\left(l\right)}\right\Vert _{2}^{2}}\nabla_{\bm{x}}f^{\left(l\right)}\left(\bm{h}^{t,\left(l\right)},\bm{x}^{t,\left(l\right)}\right)\right)-\alpha^{t}\left(\bm{x}^{t}-\frac{\eta}{ \Vert \bm{h}^{t} \Vert _{2}^{2}}\nabla_{\bm{x}}f\left(\bm{h}^{t},\bm{x}^{t}\right)\right)
\end{array}\right]\\
 & \quad=\left[\begin{array}{c}
\hat{\bm{h}}^{t,\left(l\right)}-\frac{\eta}{\left\Vert \hat{\bm{x}}^{t,\left(l\right)}\right\Vert _{2}^{2}}\nabla_{\bm{h}}f^{\left(l\right)}\big(\hat{\bm{h}}^{t,\left(l\right)},\hat{\bm{x}}^{t,\left(l\right)}\big)-\left(\tilde{\bm{h}}^{t}-\frac{\eta}{\left\Vert \tilde{\bm{x}}^{t}\right\Vert _{2}^{2}}\nabla_{\bm{h}}f\big(\tilde{\bm{h}}^{t},\tilde{\bm{x}}^{t}\big)\right)\\
\hat{\bm{x}}^{t,\left(l\right)}-\frac{\eta}{ \Vert \hat{\bm{h}}^{t,\left(l\right)} \Vert _{2}^{2}}\nabla_{\bm{x}}f^{\left(l\right)}\big(\hat{\bm{h}}^{t,\left(l\right)},\hat{\bm{x}}^{t,\left(l\right)}\big)-\left(\tilde{\bm{x}}^{t}-\frac{\eta}{ \Vert \tilde{\bm{h}}^{t} \Vert _{2}^{2}}\nabla_{\bm{x}}f\big(\tilde{\bm{h}}^{t},\tilde{\bm{x}}^{t}\big)\right)
\end{array}\right].
\end{align*}
By construction, we can write the leave-one-out gradients as 
\begin{align*}
\nabla_{\bm{h}}f^{\left(l\right)}\left(\bm{h},\bm{x}\right) & =\nabla_{\bm{h}}f\left(\bm{h},\bm{x}\right)-\left(\bm{b}_{l}^{\conj}\bm{h}\bm{x}^{\conj}\bm{a}_{l}-y_{l}\right)\bm{b}_{l}\bm{a}_{l}^{\conj}\bm{x}\qquad\text{and}\\
\nabla_{\bm{x}}f^{\left(l\right)}\left(\bm{h},\bm{x}\right) & =\nabla_{\bm{h}}f\left(\bm{h},\bm{x}\right)-\overline{(\bm{b}_{l}^{\conj}\bm{h}\bm{x}^{\conj}\bm{a}_{l}-y_{l})}\bm{a}_{l}\bm{b}_{l}^{\conj}\bm{h},
\end{align*}
which allow us to continue the derivation and obtain 
\begin{align*}
\left[\begin{array}{c}
\frac{1}{\overline{\alpha_{\text{mutual}}^{t,\left(l\right)}}}\bm{h}^{t+1,\left(l\right)}-\frac{1}{\overline{\alpha^{t}}}\bm{h}^{t+1}\\
\alpha_{\text{mutual}}^{t,\left(l\right)}\bm{x}^{t+1,\left(l\right)}-\alpha^{t}\bm{x}^{t+1}
\end{array}\right] & =\left[\begin{array}{c}
\hat{\bm{h}}^{t,\left(l\right)}-\frac{\eta}{\left\Vert \hat{\bm{x}}^{t,\left(l\right)}\right\Vert _{2}^{2}}\nabla_{\bm{h}}f\big(\hat{\bm{h}}^{t,\left(l\right)},\hat{\bm{x}}^{t,\left(l\right)}\big)-\left(\tilde{\bm{h}}^{t}-\frac{\eta}{\left\Vert \tilde{\bm{x}}^{t}\right\Vert _{2}^{2}}\nabla_{\bm{h}}f\big(\tilde{\bm{h}}^{t},\tilde{\bm{x}}^{t}\big)\right)\\
\hat{\bm{x}}^{t,\left(l\right)}-\frac{\eta}{\Vert \hat{\bm{h}}^{t,\left(l\right)} \Vert _{2}^{2}}\nabla_{\bm{x}}f\big(\hat{\bm{h}}^{t,\left(l\right)},\hat{\bm{x}}^{t,\left(l\right)}\big)-\left(\tilde{\bm{x}}^{t}-\frac{\eta}{\Vert \tilde{\bm{h}}^{t}\Vert _{2}^{2}}\nabla_{\bm{x}}f\big(\tilde{\bm{h}}^{t},\tilde{\bm{x}}^{t}\big)\right)
\end{array}\right]\\
 & \quad-\eta\underbrace{\left[\begin{array}{c}
\frac{1}{\left\Vert \hat{\bm{x}}^{t,\left(l\right)}\right\Vert _{2}^{2}}\left(\bm{b}_{l}^{\conj}\hat{\bm{h}}^{t,\left(l\right)}\hat{\bm{x}}^{t,\left(l\right)\conj}\bm{a}_{l}-y_{l}\right)\bm{b}_{l}\bm{a}_{l}^{\conj}\hat{\bm{x}}^{t,\left(l\right)}\\
\frac{1}{ \Vert \hat{\bm{h}}^{t,\left(l\right)} \Vert _{2}^{2}}\overline{\left(\bm{b}_{l}^{\conj}\hat{\bm{h}}^{t,\left(l\right)}\hat{\bm{x}}^{t,\left(l\right)\conj}\bm{a}_{l}-y_{l}\right)}\bm{a}_{l}\bm{b}_{l}^{\conj}\hat{\bm{h}}^{t,\left(l\right)}
\end{array}\right]}_{:=\bm{J}_{3}}.
\end{align*}
This further gives 
\begin{align}
\left[\begin{array}{c}
\frac{1}{\overline{\alpha_{\text{mutual}}^{t,\left(l\right)}}}\bm{h}^{t+1,\left(l\right)}-\frac{1}{\overline{\alpha^{t}}}\bm{h}^{t+1}\\
\alpha_{\text{mutual}}^{t,\left(l\right)}\bm{x}^{t+1,\left(l\right)}-\alpha^{t}\bm{x}^{t+1}
\end{array}\right] & =\underbrace{\left[\begin{array}{c}
\hat{\bm{h}}^{t,\left(l\right)}-\frac{\eta}{\left\Vert \hat{\bm{x}}^{t,\left(l\right)}\right\Vert _{2}^{2}}\nabla_{\bm{h}}f\left(\hat{\bm{h}}^{t,\left(l\right)},\hat{\bm{x}}^{t,\left(l\right)}\right)-\left(\tilde{\bm{h}}^{t}-\frac{\eta}{\left\Vert \hat{\bm{x}}^{t,\left(l\right)}\right\Vert _{2}^{2}}\nabla_{\bm{h}}f\left(\tilde{\bm{h}}^{t},\tilde{\bm{x}}^{t}\right)\right)\\
\hat{\bm{x}}^{t,\left(l\right)}-\frac{\eta}{\left\Vert \hat{\bm{h}}^{t,\left(l\right)}\right\Vert _{2}^{2}}\nabla_{\bm{x}}f\left(\hat{\bm{h}}^{t,\left(l\right)},\hat{\bm{x}}^{t,\left(l\right)}\right)-\left(\tilde{\bm{x}}^{t}-\frac{\eta}{\left\Vert \hat{\bm{h}}^{t,\left(l\right)}\right\Vert _{2}^{2}}\nabla_{\bm{x}}f\left(\tilde{\bm{h}}^{t},\tilde{\bm{x}}^{t}\right)\right)
\end{array}\right]}_{:=\bm{\nu}_{1}}\nonumber \\
 & \quad+\eta\underbrace{\left[\begin{array}{c}
\left(\frac{1}{\left\Vert \tilde{\bm{x}}^{t}\right\Vert _{2}^{2}}-\frac{1}{\left\Vert \hat{\bm{x}}^{t,\left(l\right)}\right\Vert _{2}^{2}}\right)\nabla_{\bm{h}}f\left(\tilde{\bm{h}}^{t},\tilde{\bm{x}}^{t}\right)\\
\left(\frac{1}{\left\Vert \tilde{\bm{h}}^{t}\right\Vert _{2}^{2}}-\frac{1}{\left\Vert \hat{\bm{h}}^{t,\left(l\right)}\right\Vert _{2}^{2}}\right)\nabla_{\bm{x}}f\left(\tilde{\bm{h}}^{t},\tilde{\bm{x}}^{t}\right)
\end{array}\right]}_{:=\bm{\nu}_{2}}-\eta\bm{\nu}_{3}.\label{eq:hx-I123-BD}
\end{align}
In what follows, we bound the three terms $\bm{\nu}_1$, $\bm{\nu}_2$, and $\bm{\nu}_3$ separately. 
\begin{enumerate}
\item Regarding the first term $\bm{\nu}_{1}$, one can adopt the same strategy
as in Appendix \ref{subsec:Proof-of-Lemma-inductive-ell-2}. Specifically,
write 
\begin{align*}
 & \left[\begin{array}{c}
\hat{\bm{h}}^{t,\left(l\right)}-\frac{\eta}{\left\Vert \hat{\bm{x}}^{t,\left(l\right)}\right\Vert _{2}^{2}}\nabla_{\bm{h}}f\left(\hat{\bm{z}}^{t,\left(l\right)}\right)-\Big(\tilde{\bm{h}}^{t}-\frac{\eta}{\left\Vert \hat{\bm{x}}^{t,\left(l\right)}\right\Vert _{2}^{2}}\nabla_{\bm{h}}f\left(\tilde{\bm{z}}^{t}\right)\Big)\\
\hat{\bm{x}}^{t,\left(l\right)}-\frac{\eta}{\left\Vert \hat{\bm{h}}^{t,\left(l\right)}\right\Vert _{2}^{2}}\nabla_{\bm{x}}f\left(\hat{\bm{z}}^{t,\left(l\right)}\right)-\Big(\tilde{\bm{x}}^{t}-\frac{\eta}{\left\Vert \hat{\bm{h}}^{t,\left(l\right)}\right\Vert _{2}^{2}}\nabla_{\bm{x}}f\left(\tilde{\bm{z}}^{t}\right)\Big)\\
\overline{\hat{\bm{h}}^{t,\left(l\right)}-\frac{\eta}{\Vert \hat{\bm{x}}^{t,\left(l\right)} \Vert _{2}^{2}}\nabla_{\bm{h}}f\big(\hat{\bm{z}}^{t,\left(l\right)}\big)-\Big(\tilde{\bm{h}}^{t}-\frac{\eta}{\Vert \hat{\bm{x}}^{t,\left(l\right)}\Vert _{2}^{2}}\nabla_{\bm{h}}f\left(\tilde{\bm{z}}^{t}\right)\Big)}\\
\overline{\hat{\bm{x}}^{t,\left(l\right)}-\frac{\eta}{\Vert \hat{\bm{h}}^{t,\left(l\right)}\Vert _{2}^{2}}\nabla_{\bm{x}}f\left(\hat{\bm{z}}^{t,\left(l\right)}\right)-\Big(\tilde{\bm{x}}^{t}-\frac{\eta}{\Vert \hat{\bm{h}}^{t,\left(l\right)} \Vert _{2}^{2}}\nabla_{\bm{x}}f\left(\tilde{\bm{z}}^{t}\right)\Big)}
\end{array}\right]=\left[\begin{array}{c}
\hat{\bm{h}}^{t,\left(l\right)}-\tilde{\bm{h}}^{t}\\
\hat{\bm{x}}^{t,\left(l\right)}-\tilde{\bm{x}}^{t}\\
\overline{\hat{\bm{h}}^{t,\left(l\right)}-\tilde{\bm{h}}^{t}}\\
\overline{\hat{\bm{x}}^{t,\left(l\right)}-\tilde{\bm{x}}^{t}}
\end{array}\right]\\
 & \quad-\eta\underbrace{\left[\begin{array}{cccc}
\left\Vert \hat{\bm{x}}^{t,\left(l\right)}\right\Vert _{2}^{-2}\bm{I}_{K}\\
 & \left\Vert \hat{\bm{h}}^{t,\left(l\right)}\right\Vert _{2}^{-2}\bm{I}_{K}\\
 &  & \left\Vert \hat{\bm{x}}^{t,\left(l\right)}\right\Vert _{2}^{-2}\bm{I}_{K}\\
 &  &  & \left\Vert \hat{\bm{h}}^{t,\left(l\right)}\right\Vert _{2}^{-2}\bm{I}_{K}
\end{array}\right]}_{:=\bm{D}}\left[\begin{array}{c}
\nabla_{\bm{h}}f\left(\hat{\bm{z}}^{t,\left(l\right)}\right)-\nabla_{\bm{h}}f\left(\tilde{\bm{z}}^{t}\right)\\
\nabla_{\bm{x}}f\left(\hat{\bm{z}}^{t,\left(l\right)}\right)-\nabla_{\bm{x}}f\left(\tilde{\bm{z}}^{t}\right)\\
\overline{\nabla_{\bm{h}}f\left(\hat{\bm{z}}^{t,\left(l\right)}\right)-\nabla_{\bm{h}}f\left(\tilde{\bm{z}}^{t}\right)}\\
\overline{\nabla_{\bm{x}}f\left(\hat{\bm{z}}^{t,\left(l\right)}\right)-\nabla_{\bm{x}}f\left(\tilde{\bm{z}}^{t}\right)}
\end{array}\right].
\end{align*}
The fundamental theorem of calculus (see Appendix \ref{sec:Wirtinger-calculus})
reveals that 
\[
\left[\begin{array}{c}
\nabla_{\bm{h}}f\left(\hat{\bm{z}}^{t,\left(l\right)}\right)-\nabla_{\bm{h}}f\left(\tilde{\bm{z}}^{t}\right)\\
\nabla_{\bm{x}}f\left(\hat{\bm{z}}^{t,\left(l\right)}\right)-\nabla_{\bm{x}}f\left(\tilde{\bm{z}}^{t}\right)\\
\overline{\nabla_{\bm{h}}f\left(\hat{\bm{z}}^{t,\left(l\right)}\right)-\nabla_{\bm{h}}f\left(\tilde{\bm{z}}^{t}\right)}\\
\overline{\nabla_{\bm{x}}f\left(\hat{\bm{z}}^{t,\left(l\right)}\right)-\nabla_{\bm{x}}f\left(\tilde{\bm{z}}^{t}\right)}
\end{array}\right]=\underbrace{\int_{0}^{1}\nabla^{2}f\left(\bm{z}\left(\tau\right)\right)\mathrm{d}\tau}_{:=\bm{A}}\left[\begin{array}{c}
\hat{\bm{h}}^{t,\left(l\right)}-\tilde{\bm{h}}^{t}\\
\hat{\bm{x}}^{t,\left(l\right)}-\tilde{\bm{x}}^{t}\\
\overline{\hat{\bm{h}}^{t,\left(l\right)}-\tilde{\bm{h}}^{t}}\\
\overline{\hat{\bm{x}}^{t,\left(l\right)}-\tilde{\bm{x}}^{t}}
\end{array}\right],
\]
where we abuse the notation and denote $\bm{z}\left(\tau\right)=\tilde{\bm{z}}^{t}+\tau\left(\hat{\bm{z}}^{t,\left(l\right)}-\tilde{\bm{z}}^{t}\right)$.
In order to invoke Lemma \ref{lemma:hessian-bd}, we need to verify the conditions required therein. Recall
the induction hypothesis (\ref{eq:LOO-perturb-hypothesis-BD}) that
$$
\text{dist}\big(\bm{z}^{t,\left(l\right)},\tilde{\bm{z}}^{t}\big)=\big\|\hat{\bm{z}}^{t,\left(l\right)}-\tilde{\bm{z}}^{t}\big\|_{2}\leq C_{2}\frac{\mu}{\sqrt{m}}\sqrt{\frac{\mu^{2}K\log^{9}m}{m}},$$
and the fact that $\bm{z}\left(\tau\right)$ lies between $\hat{\bm{z}}^{t,\left(l\right)}$
and $\tilde{\bm{z}}^{t}$. For all $0\leq\tau\leq1$: 
\begin{enumerate}
\item If $m\gg\mu^{2}\sqrt{K}\log^{13/2}m$, then 
\begin{align*}
\left\Vert \bm{z}\left(\tau\right)-\bm{z}^{\star}\right\Vert _{2} & \leq\max\left\{ \big\|\hat{\bm{z}}^{t,\left(l\right)}-\bm{z}^{\star}\big\|_{2},\left\Vert \tilde{\bm{z}}^{t}-\bm{z}^{\star}\right\Vert _{2}\right\} \leq\left\Vert \tilde{\bm{z}}^{t}-\bm{z}^{\star}\right\Vert _{2}+\big\|\hat{\bm{z}}^{t,\left(l\right)}-\tilde{\bm{z}}^{t}\big\|_{2}\\
 & \leq C_{1}\frac{1}{\log^{2}m}+C_{2}\frac{\mu}{\sqrt{m}}\sqrt{\frac{\mu^{2}K\log^{9}m}{m}}\leq2C_{1}\frac{1}{\log^{2}m},
\end{align*}
where we have used the induction hypotheses (\ref{eq:L2-hypothesis-BD})
and (\ref{eq:LOO-perturb-hypothesis-BD}); 
\item If $m\gg\mu^{2}K\log^{6}m$, then 
\begin{align}
\max_{1\leq j\leq m}\left|\bm{a}_{j}^{\conj}\left(\bm{x}\left(\tau\right)-\bm{x}^{\star}\right)\right| 
& = \max_{1\leq j\leq m}\left| \tau \bm{a}_{j}^{\conj}\big(\hat{\bm{x}}^{t,\left(l\right)}-\tilde{\bm{x}}^{t}\big) + \bm{a}_{j}^{\conj}\left(\tilde{\bm{x}}^{t}-\bm{x}^{\star}\right)\right| \nonumber\\
& \leq\max_{1\leq j\leq m}\left|\bm{a}_{j}^{\conj}\big(\hat{\bm{x}}^{t,\left(l\right)}-\tilde{\bm{x}}^{t}\big)\right|+\max_{1\leq j\leq m}\left|\bm{a}_{j}^{\conj}\left(\tilde{\bm{x}}^{t}-\bm{x}^{\star}\right)\right| \nonumber\\
 & \leq\max_{1\leq j\leq m}\left\Vert \bm{a}_{j}\right\Vert _{2}\big\|\hat{\bm{z}}^{t,\left(l\right)}-\tilde{\bm{z}}^{t}\big\|_{2}+C_{3}\frac{1}{\log^{3/2}m}\nonumber\\
 & \leq3\sqrt{K}\cdot C_{2}\frac{\mu}{\sqrt{m}}\sqrt{\frac{\mu^{2}K\log^{9}m}{m}}+C_{3}\frac{1}{\log^{3/2}m}\leq2C_{3}\frac{1}{\log^{3/2}m}\label{eq:ax-UB-BD1-1},
\end{align}
which follows from the bound (\ref{eq:max_gaussian-2}) and
the induction hypotheses (\ref{eq:LOO-perturb-hypothesis-BD}) and
(\ref{eq:incoherence-hypothesis-ax-BD}); 
\item If $m\gg\mu K\log^{5/2}m$, then 
\begin{align}
\max_{1\leq j\leq m}\left|\bm{b}_{j}^{\conj}\bm{h}\left(\tau\right)\right| 
& = \max_{1\leq j\leq m}\big| \tau \bm{b}_{j}^{\conj}\big(\hat{\bm{h}}^{t,\left(l\right)}-\tilde{\bm{h}}^{t}\big) + \bm{b}_{j}^{\conj}\tilde{\bm{h}}^{t}\big|\nonumber\\
& \leq\max_{1\leq j\leq m}\left|\bm{b}_{j}^{\conj}\big(\hat{\bm{h}}^{t,\left(l\right)}-\tilde{\bm{h}}^{t}\big)\right|+\max_{1\leq j\leq m}\big|\bm{b}_{j}^{\conj}\tilde{\bm{h}}^{t}\big|\nonumber\\
 & \leq\max_{1\leq j\leq m}\|\bm{b}_{j}\|_{2}\big\|\hat{\bm{h}}^{t,\left(l\right)}-\tilde{\bm{h}}^{t}\big\|_{2}+\max_{1\leq j\leq m}\big|\bm{b}_{j}^{\conj}\tilde{\bm{h}}^{t}\big|\nonumber\\
 & \leq\sqrt{\frac{K}{m}}\cdot C_{2}\frac{\mu}{\sqrt{m}}\sqrt{\frac{\mu^{2}K\log^{9}m}{m}}+C_{4}\frac{\mu}{\sqrt{m}}\log^{2}m\leq2C_{4}\frac{\mu}{\sqrt{m}}\log^{2}m\label{eq:ax-UB-BD1-2},
\end{align}
which makes use of the fact $\|\bm{b}_{j}\|_{2}=\sqrt{K/m}$ as well
as the induction hypotheses (\ref{eq:LOO-perturb-hypothesis-BD})
and (\ref{eq:incoherence-hypothesis-bh-BD}).
\end{enumerate}
These properties satisfy the condition \eqref{eq:condition_z} required in Lemma \ref{lemma:hessian-bd}. The other two conditions \eqref{eq:condition_u} and \eqref{eq:condition_D} are also straightforward to check and hence we omit it. 
Thus, we can repeat the argument used in Appendix \ref{subsec:Proof-of-Lemma-inductive-ell-2}
to obtain 
\[
\left\Vert \bm{\nu}_{1}\right\Vert _{2}\leq\left(1-{\eta}/{16}\right)\cdot\big\|\hat{\bm{z}}^{t,\left(l\right)}-\tilde{\bm{z}}^{t}\big\|_{2}.
\]

\item In terms of the second term $\bm{\nu}_{2}$, it is easily seen that
\begin{align*}
\left\Vert \bm{\nu}_{2}\right\Vert _{2} & \leq\max\left\{ \left|\frac{1}{\big\|\tilde{\bm{x}}^{t}\big\|_{2}^{2}}-\frac{1}{\big\|\hat{\bm{x}}^{t,\left(l\right)}\big\|_{2}^{2}}\right|,\left|\frac{1}{\big\|\tilde{\bm{h}}^{t}\big\|_{2}^{2}}-\frac{1}{\big\|\hat{\bm{h}}^{t,\left(l\right)}\big\|_{2}^{2}}\right|\right\} \left\Vert \left[\begin{array}{c}
\nabla_{\bm{h}}f\left(\tilde{\bm{z}}^{t}\right)\\
\nabla_{\bm{x}}f\left(\tilde{\bm{z}}^{t}\right)
\end{array}\right]\right\Vert _{2}.
\end{align*}
We first note that the upper bound on $\|\nabla^{2}f\left(\cdot\right)\|$
(which essentially provides a Lipschitz constant on the gradient)
in Lemma \ref{lemma:hessian-bd} forces 
\[
\left\Vert \left[\begin{array}{c}
\nabla_{\bm{h}}f\left(\tilde{\bm{z}}^{t}\right)\\
\nabla_{\bm{x}}f\left(\tilde{\bm{z}}^{t}\right)
\end{array}\right]\right\Vert _{2}=\left\Vert \left[\begin{array}{c}
\nabla_{\bm{h}}f\left(\tilde{\bm{z}}^{t}\right)-\nabla_{\bm{h}}f\left(\bm{z}^{\star}\right)\\
\nabla_{\bm{x}}f\left(\tilde{\bm{z}}^{t}\right)-\nabla_{\bm{x}}f\left(\bm{z}^{\star}\right)
\end{array}\right]\right\Vert _{2}\lesssim\left\Vert \tilde{\bm{z}}^{t}-\bm{z}^{\star}\right\Vert _{2}\lesssim C_{1}\frac{1}{\log^{2}m},
\]
where the first identity follows since $\nabla_{\bm{h}}f\left(\bm{z}^{\star}\right)=\bm{0}$, and the last inequality comes from the induction hypothesis (\ref{eq:L2-hypothesis-BD}).
Additionally, recognizing that $\left\Vert \tilde{\bm{x}}^{t}\right\Vert _{2} \asymp \left\Vert \hat{\bm{x}}^{t,\left(l\right)}\right\Vert _{2} \asymp 1$, one can easily verify that 
\[
\left|\frac{1}{\left\Vert \tilde{\bm{x}}^{t}\right\Vert _{2}^{2}}-\frac{1}{\left\Vert \hat{\bm{x}}^{t,\left(l\right)}\right\Vert _{2}^{2}}\right|=\left|\frac{\left\Vert \hat{\bm{x}}^{t,\left(l\right)}\right\Vert _{2}^{2}-\left\Vert \tilde{\bm{x}}^{t}\right\Vert _{2}^{2}}{\left\Vert \tilde{\bm{x}}^{t}\right\Vert _{2}^{2}\cdot\left\Vert \hat{\bm{x}}^{t,\left(l\right)}\right\Vert _{2}^{2}}\right|
\lesssim \Big| \big\|\hat{\bm{x}}^{t,\left(l\right)} \big\|_2 - \big\| \tilde{\bm{x}}^{t}\big\|_{2} \Big|
\lesssim\big\|\hat{\bm{x}}^{t,\left(l\right)}-\tilde{\bm{x}}^{t}\big\|_{2}.
\]
A similar bound holds for the other term involving $\bm{h}$. Combining
the estimates above thus yields 
\[
\left\Vert \bm{\nu}_{2}\right\Vert _{2}\lesssim C_{1}\frac{1}{\log^{2}m}\big\|\hat{\bm{z}}^{t,(l)}-\tilde{\bm{z}}^{t}\big\|_{2}.
\]
\item When it comes to the last term $\bm{\nu}_{3}$, one first sees that
\begin{equation}
\left\Vert \left(\bm{b}_{l}^{\conj}\hat{\bm{h}}^{t,\left(l\right)}\hat{\bm{x}}^{t,\left(l\right)\conj}\bm{a}_{l}-y_{l}\right)\bm{b}_{l}\bm{a}_{l}^{\conj}\hat{\bm{x}}^{t,\left(l\right)}\right\Vert _{2}\leq\left|\bm{b}_{l}^{\conj}\hat{\bm{h}}^{t,\left(l\right)}\hat{\bm{x}}^{t,\left(l\right)\conj}\bm{a}_{l}-y_{l}\right|\left\Vert \bm{b}_{l}\right\Vert _{2}\big|\bm{a}_{l}^{\conj}\hat{\bm{x}}^{t,\left(l\right)}\big|.\label{eq:UB-bh-10}
\end{equation}
The bounds (\ref{eq:max_gaussian-1}) and (\ref{eq:ax-UB-BD1-1}) taken collectively yield
\begin{align*}
\left|\bm{a}_{l}^{\conj}\hat{\bm{x}}^{t,\left(l\right)}\right| \leq\left|\bm{a}_{l}^{\conj}\bm{x}^{\star}\right|+\left|\bm{a}_{l}^{\conj}\big(\hat{\bm{x}}^{t,\left(l\right)}-\bm{x}^{\star}\big)\right| \lesssim\sqrt{\log m}+C_{3}\frac{1}{\log^{3/2}m}\asymp\sqrt{\log m}.
\end{align*}
In addition, the same argument as in obtaining (\ref{eq:ax-UB-BD1-2}) tells us that 
\begin{align*}
\big|\bm{b}_{l}^{\conj}(\hat{\bm{h}}^{t,\left(l\right)}-\bm{h}^{\star})\big| \lesssim C_{4}\frac{\mu}{\sqrt{m}}\log^{2}m.
\end{align*}
Combine the previous two bounds to obtain 
\begin{align*}
 & \left|\bm{b}_{l}^{\conj}\hat{\bm{h}}^{t,\left(l\right)}\hat{\bm{x}}^{t,\left(l\right)\conj}\bm{a}_{l}-y_{l}\right|\leq\big|\bm{b}_{l}^{\conj}\hat{\bm{h}}^{t,\left(l\right)}(\hat{\bm{x}}^{t,(l)}-\bm{x}^{\star})^{\conj}\bm{a}_{l}\big|+\big|\bm{b}_{l}^{\conj}(\hat{\bm{h}}^{t,\left(l\right)}-\bm{h}^{\star})\bm{x}^{\star\conj}\bm{a}_{l}\big|\\
 & \quad\leq\big|\bm{b}_{l}^{\conj}\hat{\bm{h}}^{t,\left(l\right)}\big|\cdot\big|\bm{a}_{l}^{\conj}(\hat{\bm{x}}^{t,(l)}-\bm{x}^{\star})\big|+\big|\bm{b}_{l}^{\conj}(\hat{\bm{h}}^{t,\left(l\right)}-\bm{h}^{\star})\big|\cdot\big|\bm{a}_{l}^{\conj}\bm{x}^{\star}\big|\\
 & \quad\leq\left(\big|\bm{b}_{l}^{\conj}(\hat{\bm{h}}^{t,\left(l\right)}-\bm{h}^{\star})\big|+\big|\bm{b}_{l}^{\conj}\bm{h}^{\star}\big|\right)\cdot\big|\bm{a}_{l}^{\conj}(\hat{\bm{x}}^{t,(l)}-\bm{x}^{\star})\big|+\big|\bm{b}_{l}^{\conj}(\hat{\bm{h}}^{t,\left(l\right)}-\bm{h}^{\star})\big|\cdot\big|\bm{a}_{l}^{\conj}\bm{x}^{\star}\big|\\
 & \quad\lesssim\left(C_{4}\mu\frac{\log^{2}m}{\sqrt{m}}+\frac{\mu}{\sqrt{m}}\right)\cdot C_{3}\frac{1}{\log^{3/2}m}+C_{4}\mu\frac{\log^{2}m}{\sqrt{m}}\cdot\sqrt{\log m}\lesssim C_{4}\mu\frac{\log^{5/2}m}{\sqrt{m}}.
\end{align*}
Substitution into (\ref{eq:UB-bh-10}) gives 
\begin{align}
\left\Vert \left(\bm{b}_{l}^{\conj}\hat{\bm{h}}^{t,\left(l\right)}\hat{\bm{x}}^{t,\left(l\right)\conj}\bm{a}_{l}-y_{l}\right)\bm{b}_{l}\bm{a}_{l}^{\conj}\hat{\bm{x}}^{t,\left(l\right)}\right\Vert _{2} & \lesssim C_{4}\mu\frac{\log^{5/2}m}{\sqrt{m}}\cdot\sqrt{\frac{K}{m}}\cdot\sqrt{\log m}.
\end{align}
Similarly, we can also derive 
\begin{align*}
\left\Vert \overline{\left(\bm{b}_{l}^{\conj}\hat{\bm{h}}^{t,\left(l\right)}\hat{\bm{x}}^{t,\left(l\right)\conj}\bm{a}_{l}-y_{l}\right)}\bm{a}_{l}\bm{b}_{l}^{\conj}\hat{\bm{h}}^{t,\left(l\right)}\right\Vert  & \leq\left|\bm{b}_{l}^{\conj}\hat{\bm{h}}^{t,\left(l\right)}\hat{\bm{x}}^{t,\left(l\right)\conj}\bm{a}_{l}-y_{l}\right|\left\Vert \bm{a}_{l}\right\Vert _{2}\left|\bm{b}_{l}^{\conj}\hat{\bm{h}}^{t,\left(l\right)}\right|\\
 & \lesssim C_{4}\mu\frac{\log^{5/2}m}{\sqrt{m}}\cdot\sqrt{K}\cdot C_{4}\frac{\mu}{\sqrt{m}}\log^{2}m
\end{align*}
Putting these bounds together indicates that 
\[
\left\Vert \bm{\nu}_{3}\right\Vert _{2}\lesssim\left(C_{4}\right)^{2}\frac{\mu}{\sqrt{m}}\sqrt{\frac{\mu^{2}K\log^{9}m}{m}}.
\]
\end{enumerate}
The above bounds taken together with (\ref{eq:loop-UB-1-BD-1}) and
(\ref{eq:hx-I123-BD}) ensure the existence of a constant $C>0$ such
that 
\begin{align*}
\text{dist}\big(\bm{z}^{t+1,\left(l\right)},\tilde{\bm{z}}^{t+1}\big) & \leq\max\left\{ \left|\frac{\alpha^{t+1}}{\alpha^{t}}\right|,\left|\frac{\alpha^{t}}{\alpha^{t+1}}\right|\right\} \left\{ \left(1-\frac{\eta}{16}+CC_{1}\eta\frac{1}{\log^{2}m}\right)\big\|\hat{\bm{z}}^{t,\left(l\right)}-\tilde{\bm{z}}^{t}\big\|_{2}+C\left(C_{4}\right)^{2}\eta\frac{\mu}{\sqrt{m}}\sqrt{\frac{\mu^{2}K\log^{9}m}{m}}\right\} \\
 & \overset{(\text{i})}{\leq}\frac{1-\eta/21}{1-\eta/20}\left\{ \left(1-\frac{\eta}{20}\right)\big\|\hat{\bm{z}}^{t,\left(l\right)}-\tilde{\bm{z}}^{t}\big\|_{2}+C\left(C_{4}\right)^{2}\eta\frac{\mu}{\sqrt{m}}\sqrt{\frac{\mu^{2}K\log^{9}m}{m}}\right\} \\
 & \leq\left(1-\frac{\eta}{21}\right)\big\|\hat{\bm{z}}^{t,\left(l\right)}-\tilde{\bm{z}}^{t}\big\|_{2}+2C\left(C_{4}\right)^{2}\eta\frac{\mu}{\sqrt{m}}\sqrt{\frac{\mu^{2}K\log^{9}m}{m}}\\
 & =\left(1-\frac{\eta}{21}\right)\mathrm{dist}\big(\bm{z}^{t,\left(l\right)},\tilde{\bm{z}}^{t}\big)+2C\left(C_{4}\right)^{2}\eta\frac{\mu}{\sqrt{m}}\sqrt{\frac{\mu^{2}K\log^{9}m}{m}}\\
 &  \overset{(\text{ii})}{\leq}C_{2}\frac{\mu}{\sqrt{m}}\sqrt{\frac{\mu^{2}K\log^{9}m}{m}}.
\end{align*}
Here, (i) holds as long as $m$ is sufficiently large such that $CC_{1}{1} / {\log^{2}m}\ll 1$ and 
\begin{equation}
\max\left\{ \left|\frac{\alpha^{t+1}}{\alpha^{t}}\right|,\left|\frac{\alpha^{t}}{\alpha^{t+1}}\right|\right\} <\frac{1-\eta/21}{1-\eta/20},\label{eq:alphat-alphat1-ratio}
\end{equation}
which is guaranteed by Lemma
\ref{lemma:alpha-t+1-BD}. The inequality (ii) arises from the induction hypothesis (\ref{eq:LOO-perturb-hypothesis-BD})
and  taking $C_{2}>0$ is sufficiently large. 

Finally we establish the second inequality claimed in the lemma. Take
$(\bm{h}_{1},\bm{x}_{1})=(\tilde{\bm{h}}^{t+1},\tilde{\bm{x}}^{t+1})$
and $(\bm{h}_{2},\bm{x}_{2})=(\hat{\bm{h}}^{t+1,(l)},\hat{\bm{x}}^{t+1,(l)})$
in Lemma \ref{lemma:stability-BD}. Since both $(\bm{h}_{1},\bm{x}_{1})$
and $(\bm{h}_{2},\bm{x}_{2})$ are close enough to $(\bm{h}^{\star},\bm{x}^{\star})$,
we deduce that 
\[
\big\|\tilde{\bm{z}}^{t+1,(l)}-\tilde{\bm{z}}^{t+1}\big\|_{2}\lesssim\big\|\hat{\bm{z}}^{t+1,(l)}-\tilde{\bm{z}}^{t+1}\big\|_{2}\lesssim C_{2}\frac{\mu}{\sqrt{m}}\sqrt{\frac{\mu^{2}K\log^{9}m}{m}}
\]
as claimed.

\subsection{Proof of Lemma \ref{lemma:incoherence-b}\label{subsec:Proof-of-Lemma-incoherence-b}}
Before going forward, we make note of the following inequality 
\[
\max_{1\leq l\leq m}\left|\bm{b}_{l}^{\conj}\frac{1}{\overline{\alpha^{t+1}}}\bm{h}^{t+1}\right|\leq\left|\frac{\alpha^{t}}{\alpha^{t+1}}\right|\max_{1\leq l\leq m}\left|\bm{b}_{l}^{\conj}\frac{1}{\overline{\alpha^{t}}}\bm{h}^{t+1}\right|\leq\left(1+\delta\right)\max_{1\leq l\leq m}\left|\bm{b}_{l}^{\conj}\frac{1}{\overline{\alpha^{t}}}\bm{h}^{t+1}\right|
\]
for some small $\delta\asymp{\log^{-2}m}$, where the last
relation follows from Lemma \ref{lemma:alpha-t+1-BD} that 
\[
\left|\frac{\alpha^{t+1}}{\alpha^{t}}-1\right|\lesssim\frac{1}{\log^{2}m}\leq\delta
\]
for $m$ sufficiently large. In view of the above inequality, the
focus of our subsequent analysis will be to control $\max_{l}\left|\bm{b}_{l}^{\conj}\frac{1}{\overline{\alpha^{t}}}\bm{h}^{t+1}\right|$.

The gradient update rule for $\bm{h}^{t+1}$ (cf.~(\ref{eq:gradient-update-h-Bd}))
gives 
\[
\frac{1}{\overline{\alpha^{t}}}\bm{h}^{t+1}=\tilde{\bm{h}}^{t}-\eta\xi\sum_{j=1}^{m}\bm{b}_{j}\bm{b}_{j}^{\conj}\big(\tilde{\bm{h}}^{t}\tilde{\bm{x}}^{t\conj}-\bm{h}^{\star}\bm{x}^{\star\conj}\big)\bm{a}_{j}\bm{a}_{j}^{\conj}\tilde{\bm{x}}^{t},
\]
where $\tilde{\bm{h}}^{t}=\frac{1}{\overline{\alpha^t}}{\bm{h}}^{t}$ and $\tilde{\bm{x}}^{t}={{\alpha^t}}{\bm{x}}^{t}$. 
Here and below, we denote 
$
\xi={1} / {\|\tilde{\bm{x}}^{t} \|_{2}^{2}}
$
for notational convenience. The above formula can be further decomposed
into the following terms 
\begin{align*}
\frac{1}{\overline{\alpha^{t}}}\bm{h}^{t+1} & =\tilde{\bm{h}}^{t}-\eta\xi\sum_{j=1}^{m}\bm{b}_{j}\bm{b}_{j}^{\conj}\tilde{\bm{h}}^{t}\left|\bm{a}_{j}^{\conj}\tilde{\bm{x}}^{t}\right|^{2}+\eta\xi\sum_{j=1}^{m}\bm{b}_{j}\bm{b}_{j}^{\conj}\bm{h}^{\star}\bm{x}^{\star\conj}\bm{a}_{j}\bm{a}_{j}^{\conj}\tilde{\bm{x}}^{t}\\
 & =\left(1-\eta\xi\big\|\bm{x}^{\star}\big\|_{2}^{2}\right)\tilde{\bm{h}}^{t}-\underbrace{\eta\xi\sum_{j=1}^{m}\bm{b}_{j}\bm{b}_{j}^{\conj}\tilde{\bm{h}}^{t}\big(\big|\bm{a}_{j}^{\conj}\tilde{\bm{x}}^{t}\big|^{2}-\big|\bm{a}_{j}^{\conj}\bm{x}^{\star}\big|^{2}\big)}_{:=\bm{v}_{1}}\\
 & \quad-\underbrace{\eta\xi\sum_{j=1}^{m}\bm{b}_{j}\bm{b}_{j}^{\conj}\tilde{\bm{h}}^{t}\big(\big|\bm{a}_{j}^{\conj}\bm{x}^{\star}\big|^{2}-\big\|\bm{x}^{\star}\big\|_{2}^{2}\big)}_{:=\bm{v}_{2}}+\underbrace{\eta\xi\sum_{j=1}^{m}\bm{b}_{j}\bm{b}_{j}^{\conj}\bm{h}^{\star}\bm{x}^{\star\conj}\bm{a}_{j}\bm{a}_{j}^{\conj}\tilde{\bm{x}}^{t}}_{:=\bm{v}_{3}},
\end{align*}
where we use the fact that $\sum_{j=1}^{m}\bm{b}_{j}\bm{b}_{j}^{\conj}=\bm{I}_{K}$.
In the sequel, we shall control each term separately. 
\begin{enumerate}
\item We start with $|\bm{b}_{l}^{\conj}\bm{v}_{1}|$ by making the observation
that 
\begin{align}
\frac{1}{\eta\xi}\left|\bm{b}_{l}^{\conj}\bm{v}_{1}\right| & =\left|\sum_{j=1}^{m}\bm{b}_{l}^{\conj}\bm{b}_{j}\bm{b}_{j}^{\conj}\tilde{\bm{h}}^{t}\left[\bm{a}_{j}^{\conj}\left(\tilde{\bm{x}}^{t}-\bm{x}^{\star}\right)\left(\bm{a}_{j}^{\conj}\tilde{\bm{x}}^{t}\right)^{\conj}+\bm{a}_{j}^{\conj}\bm{x}^{\star}\left(\bm{a}_{j}^{\conj}\left(\tilde{\bm{x}}^{t}-\bm{x}^{\star}\right)\right)^{\conj}\right]\right|\nonumber \\
 & \leq\sum_{j=1}^{m}\left|\bm{b}_{l}^{\conj}\bm{b}_{j}\right|\left\{ \max_{1\leq j\leq m}\big|\bm{b}_{j}^{\conj}\tilde{\bm{h}}^{t}\big|\right\} \left\{ \max_{1\leq j\leq m}\left|\bm{a}_{j}^{\conj}\left(\tilde{\bm{x}}^{t}-\bm{x}^{\star}\right)\right|\left(\left|\bm{a}_{j}^{\conj}\tilde{\bm{x}}^{t}\right|+\left|\bm{a}_{j}^{\conj}\bm{x}^{\star}\right|\right)\right\} .\label{eq:b-v1-bound-BD}
\end{align}
Combining the induction hypothesis (\ref{eq:incoherence-hypothesis-ax-BD})
and the condition (\ref{eq:max_gaussian-1}) yields 
\[
\max_{1\leq j\leq m}\left|\bm{a}_{j}^{\conj}\tilde{\bm{x}}^{t}\right|\leq\max_{1\leq j\leq m}\left|\bm{a}_{j}^{\conj}\left(\tilde{\bm{x}}^{t}-\bm{x}^{\star}\right)\right|+\max_{1\leq j\leq m}\left|\bm{a}_{j}^{\conj}\bm{x}^{\star}\right|\leq C_{3}\frac{1}{\log^{3/2}m}+5\sqrt{\log m}\leq6\sqrt{\log m}
\]
as long as $m$ is sufficiently large. 
This further implies 
\[
\max_{1\leq j\leq m}\left|\bm{a}_{j}^{\conj}\left(\tilde{\bm{x}}^{t}-\bm{x}^{\star}\right)\right|\left(\left|\bm{a}_{j}^{\conj}\tilde{\bm{x}}^{t}\right|+\left|\bm{a}_{j}^{\conj}\bm{x}^{\star}\right|\right)\leq C_{3}\frac{1}{\log^{3/2}m}\cdot11\sqrt{\log m}\leq11C_{3}\frac{1}{\log m}.
\]
Substituting it into (\ref{eq:b-v1-bound-BD}) and taking Lemma \ref{lemma:fourier-sum-inner},
we arrive at 
\[
\frac{1}{\eta\xi}\left|\bm{b}_{l}^{\conj}\bm{v}_{1}\right|\lesssim\log m\cdot\left\{ \max_{1\leq j\leq m}\big|\bm{b}_{j}^{\conj}\tilde{\bm{h}}^{t}\big|\right\} \cdot C_{3}\frac{1}{\log m}\lesssim C_{3}\max_{1\leq j\leq m}\big|\bm{b}_{j}^{\conj}\tilde{\bm{h}}^{t}\big|\leq0.1\max_{1\leq j\leq m}\big|\bm{b}_{j}^{\conj}\tilde{\bm{h}}^{t}\big|,
\]
with the proviso that $C_{3}$ is sufficiently small. 
\item We then move on to $|\bm{b}_{l}^{\conj}\bm{v}_{3}|$, which obeys 
\begin{align}
\frac{1}{\eta\xi}\left|\bm{b}_{l}^{\conj}\bm{v}_{3}\right| & \leq\left|\sum_{j=1}^{m}\bm{b}_{l}^{\conj}\bm{b}_{j}\bm{b}_{j}^{\conj}\bm{h}^{\star}\bm{x}^{\star\conj}\bm{a}_{j}\bm{a}_{j}^{\conj}\bm{x}^{\star}\right|+\left|\sum_{j=1}^{m}\bm{b}_{l}^{\conj}\bm{b}_{j}\bm{b}_{j}^{\conj}\bm{h}^{\star}\bm{x}^{\star\conj}\bm{a}_{j}\bm{a}_{j}^{\conj}\left(\tilde{\bm{x}}^{t}-\bm{x}^{\star}\right)\right|.\label{eq:bl-v3-2terms-BD}
\end{align}
Regarding the first term, we have the following lemma, whose proof is given in Appendix~\ref{proof_lemma:incoherence-fix}. 
\begin{lemma}
	\label{lemma:incoherence-fix}
Suppose $m\geq C K\log^2 m$ for some sufficiently large constant $C>0$. Then with probability at least $1-O\left(m^{-10}\right)$,
one has 
\[
\left|\sum_{j=1}^{m}\bm{b}_{l}^{\conj}\bm{b}_{j}\bm{b}_{j}^{\conj}\bm{h}^{\star}\bm{x}^{\star\conj}\bm{a}_{j}\bm{a}_{j}^{\conj}\bm{x}^{\star}-\bm{b}_{l}^{\conj}\bm{h}^{\star}\right|\lesssim\frac{\mu}{\sqrt{m}}.
\]
\end{lemma}
For the remaining term, we apply the same strategy as
in bounding $|\bm{b}_{l}^{\conj}\bm{v}_{1}|$ to get 
\begin{align*}
\left|\sum_{j=1}^{m}\bm{b}_{l}^{\conj}\bm{b}_{j}\bm{b}_{j}^{\conj}\bm{h}^{\star}\bm{x}^{\star\conj}\bm{a}_{j}\bm{a}_{j}^{\conj}\left(\tilde{\bm{x}}^{t}-\bm{x}^{\star}\right)\right| & \leq\sum_{j=1}^{m}\left|\bm{b}_{l}^{\conj}\bm{b}_{j}\right|\left\{ \max_{1\leq j\leq m}\left|\bm{b}_{j}^{\conj}\bm{h}^{\star}\right|\right\} \left\{ \max_{1\leq j\leq m}\left|\bm{a}_{j}^{\conj}\left(\tilde{\bm{x}}^{t}-\bm{x}^{\star}\right)\right|\right\} \left\{ \max_{1\leq j\leq m}\left|\bm{a}_{j}^{\conj}\bm{x}^{\star}\right|\right\} \\
 & \leq4\log m\cdot\frac{\mu}{\sqrt{m}}\cdot C_{3}\frac{1}{\log^{3/2}m}\cdot5\sqrt{\log m}\\
 & \lesssim C_{3}\frac{\mu}{\sqrt{m}},
\end{align*}
where the second line follows from the incoherence (\ref{eq:incoherence-BD}),
the induction hypothesis (\ref{eq:incoherence-hypothesis-ax-BD}),
the condition (\ref{eq:max_gaussian-1}) and Lemma \ref{lemma:fourier-sum-inner}.
Combining the above three inequalities and the incoherence (\ref{eq:incoherence-BD})
yields 
\[
\frac{1}{\eta\xi}\left|\bm{b}_{l}^{\conj}\bm{v}_{3}\right|\lesssim\big|\bm{b}_{l}^{\conj}\bm{h}^{\star}\big|+\frac{\mu}{\sqrt{m}}+C_{3}\frac{\mu}{\sqrt{m}}\lesssim\left(1+C_{3}\right)\frac{\mu}{\sqrt{m}}.
\]
\item Finally, we need to control $\left|\bm{b}_{l}^{\conj}\bm{v}_{2}\right|$.
For convenience of presentation, we will only bound $\left|\bm{b}_{1}^{\conj}\bm{v}_{2}\right|$
in the sequel, but the argument easily extends to all other $\bm{b}_{l}$'s.
The idea is to group $\left\{ \bm{b}_{j}\right\}_{1\leq j\leq m} $
into bins each containing $\tau$ adjacent vectors, and to look at each bin separately.
Here, $\tau\asymp\mathrm{poly}\log(m)$ is some integer to
be specified later. For notational simplicity, we assume $m/\tau$ to be an integer, although all arguments continue to hold when $m/\tau$ is not an integer. For each $0\leq l\leq m-\tau$, the following summation over $\tau$ adjacent data obeys 
\begin{align}
 & \bm{b}_{1}^{\conj}\sum_{j=1}^{\tau}\bm{b}_{l+j}\bm{b}_{l+j}^{\conj}\tilde{\bm{h}}^{t}\left(\left|\bm{a}_{l+j}^{\conj}\bm{x}^{\star}\right|^{2}-\left\Vert \bm{x}^{\star}\right\Vert _{2}^{2}\right)\nonumber \\
 & =\bm{b}_{1}^{\conj}\sum_{j=1}^{\tau}\bm{b}_{l+1}\bm{b}_{l+1}^{\conj}\tilde{\bm{h}}^{t}\left(\left|\bm{a}_{l+j}^{\conj}\bm{x}^{\star}\right|^{2}-\left\Vert \bm{x}^{\star}\right\Vert _{2}^{2}\right)+\bm{b}_{1}^{\conj}\sum_{j=1}^{\tau}\left(\bm{b}_{l+j}\bm{b}_{l+j}^{\conj}-\bm{b}_{l+1}\bm{b}_{l+1}^{\conj}\right)\tilde{\bm{h}}^{t}\left(\left|\bm{a}_{l+j}^{\conj}\bm{x}^{\star}\right|^{2}-\left\Vert \bm{x}^{\star}\right\Vert _{2}^{2}\right)\nonumber \\
 & =\left\{ \sum_{j=1}^{\tau}\left(\left|\bm{a}_{l+j}^{\conj}\bm{x}^{\star}\right|^{2}-\left\Vert \bm{x}^{\star}\right\Vert _{2}^{2}\right)\right\} \bm{b}_{1}^{\conj}\bm{b}_{l+1}\bm{b}_{l+1}^{\conj}\tilde{\bm{h}}^{t}+\bm{b}_{1}^{\conj}\sum_{j=1}^{\tau}\left(\bm{b}_{l+j}-\bm{b}_{l+1}\right)\bm{b}_{l+j}^{\conj}\tilde{\bm{h}}^{t}\left(\left|\bm{a}_{l+j}^{\conj}\bm{x}^{\star}\right|^{2}-\left\Vert \bm{x}^{\star}\right\Vert _{2}^{2}\right)\nonumber \\
 & \quad\quad+\bm{b}_{1}^{\conj}\sum_{j=1}^{\tau}\bm{b}_{l+1}\left(\bm{b}_{l+j}-\bm{b}_{l+1}\right)^{\conj}\tilde{\bm{h}}^{t}\left(\left|\bm{a}_{l+j}^{\conj}\bm{x}^{\star}\right|^{2}-\left\Vert \bm{x}^{\star}\right\Vert _{2}^{2}\right).\label{eq:bz2-terms}
\end{align}
We will now bound each term in \eqref{eq:bz2-terms} separately.
\begin{itemize}
\item Before bounding the first term in (\ref{eq:bz2-terms}), we first
bound the pre-factor $\left|\sum_{j=1}^{\tau}\big(|\bm{a}_{l+j}^{\conj}\bm{x}^{\star}|^{2}-\|\bm{x}^{\star}\|_{2}^{2}\big)\right|$. Notably, the fluctuation of this quantity does not grow fast as it is the sum of i.i.d.~random variables over
a group of relatively large size, i.e.~$\tau$. Since $2\left|\bm{a}_{j}^{\conj}\bm{x}^{\star}\right|^{2}$
		follows the $\chi_{2}^{2}$ distribution, by standard  concentration results (e.g.~\cite[Theorem 1.1]{rudelson2013hanson}), with probability exceeding $1-O\left(m^{-10}\right)$,
\[
\left|\sum_{j=1}^{\tau}\big(\big|\bm{a}_{l+j}^{\conj}\bm{x}^{\star}\big|^{2}-\|\bm{x}^{\star}\|_{2}^{2}\big)\right|\lesssim \sqrt{\tau\log m}.
\]
With this result in place, we can bound the first term in
(\ref{eq:bz2-terms}) as 
\begin{align*}
\left|\left\{ \sum_{j=1}^{\tau}\big(\big|\bm{a}_{l+j}^{\conj}\bm{x}^{\star}\big|^{2}-\|\bm{x}^{\star}\|_{2}^{2}\big)\right\} \bm{b}_{1}^{\conj}\bm{b}_{l+1}\bm{b}_{l+1}^{\conj}\tilde{\bm{h}}^{t}\right| & \lesssim\sqrt{\tau\log m}\left|\bm{b}_{1}^{\conj}\bm{b}_{l+1}\right|\max_{1\leq l\leq m}\left|\bm{b}_{l}^{\conj}\tilde{\bm{h}}^{t}\right|.
\end{align*}
Taking the summation over all bins gives 
\begin{align}
\sum_{k=0}^{ \frac{m}{\tau}-1 }\left|\left\{ \sum_{j=1}^{\tau}\big(\big|\bm{a}_{k\tau+j}^{\conj}\bm{x}^{\star}\big|^{2}-\|\bm{x}^{\star}\|_{2}^{2}\big)\right\} \bm{b}_{1}^{\conj}\bm{b}_{k\tau+1}\bm{b}_{k\tau+1}^{\conj}\tilde{\bm{h}}^{t}\right| & \lesssim\sqrt{\tau\log m}\sum_{k=0}^{ \frac{m}{\tau}-1 }\left|\bm{b}_{1}^{\conj}\bm{b}_{k\tau+1}\right|\max_{1\leq l\leq m}\left|\bm{b}_{l}^{\conj}\tilde{\bm{h}}^{t}\right| .
	\label{eq:sum-fourier-gap-before}
\end{align}
It is straightforward to see from the proof of Lemma \ref{lemma:fourier-sum-inner}
that 
\begin{equation}
\sum_{k=0}^{ \frac{m}{\tau}-1 }\left|\bm{b}_{1}^{\conj}\bm{b}_{k\tau+1}\right|=\|\bm{b}_{1}\|_{2}^{2}+\sum_{k=1}^{ \frac{m}{\tau}-1 }\left|\bm{b}_{1}^{\conj}\bm{b}_{k\tau+1}\right|\leq\frac{K}{m}+O\left(\frac{\log m}{\tau}\right).\label{eq:sum-fourier-gap}
\end{equation}
Substituting (\ref{eq:sum-fourier-gap}) into the previous
inequality (\ref{eq:sum-fourier-gap-before}) gives 
\begin{align*}
\sum_{k=0}^{ \frac{m}{\tau}-1 }\left|\left\{ \sum_{j=1}^{\tau}\big(\big|\bm{a}_{k\tau+j}^{\conj}\bm{x}^{\star}\big|^{2}-\|\bm{x}^{\star}\|_{2}^{2}\big)\right\} \bm{b}_{1}^{\conj}\bm{b}_{k\tau+1}\bm{b}_{k\tau+1}^{\conj}\tilde{\bm{h}}^{t}\right| & \lesssim\left(\frac{K\sqrt{\tau\log m}}{m}+\sqrt{\frac{\log^{3}m}{\tau}}\right)\max_{1\leq l\leq m}\left|\bm{b}_{l}^{\conj}\tilde{\bm{h}}^{t}\right|\\
 & \leq0.1\max_{1\leq l\leq m}\left|\bm{b}_{l}^{\conj}\tilde{\bm{h}}^{t}\right|,
\end{align*}
as long as $m\gg K\sqrt{\tau\log m}$ and $\tau\gg\log^{3}m$. 
\item The second term of (\ref{eq:bz2-terms}) obeys 
\begin{align*}
 & \left|\bm{b}_{1}^{\conj}\sum_{j=1}^{\tau}\left(\bm{b}_{l+j}-\bm{b}_{l+1}\right)\bm{b}_{l+j}^{\conj}\tilde{\bm{h}}^{t}\left(\left|\bm{a}_{l+j}^{\conj}\bm{x}^{\star}\right|^{2}-\left\Vert \bm{x}^{\star}\right\Vert _{2}^{2}\right)\right|\\
 & \quad\leq\max_{1\leq l\leq m}\left|\bm{b}_{l}^{\conj}\tilde{\bm{h}}^{t}\right|\sqrt{\sum_{j=1}^{\tau}\left|\bm{b}_{1}^{\conj}\left(\bm{b}_{l+j}-\bm{b}_{l+1}\right)\right|^{2}}\sqrt{\sum_{j=1}^{\tau}\left(\big|\bm{a}_{l+j}^{\conj}\bm{x}^{\star}\big|^{2}-\left\Vert \bm{x}^{\star}\right\Vert _{2}^{2}\right)^{2}}\\
 & \quad\lesssim\sqrt{\tau}\max_{1\leq l\leq m}\left|\bm{b}_{l}^{\conj}\tilde{\bm{h}}^{t}\right|\sqrt{\sum_{j=1}^{\tau}\left|\bm{b}_{1}^{\conj}\left(\bm{b}_{l+j}-\bm{b}_{l+1}\right)\right|^{2}},
\end{align*}
where the first inequality is due to Cauchy-Schwarz, and the second
one holds because of the following lemma, whose proof can be found in Appendix~\ref{proof_lemma:gaussian-hypercontractivity}. 
\begin{lemma}
	\label{lemma:gaussian-hypercontractivity}
Suppose
$\tau\geq C\log^{4}m$ for some sufficiently large constant $C>0$. Then
 with probability
exceeding $1-O\left(m^{-10}\right)$, 
\[
\sum_{j=1}^{\tau}\left(\left|\bm{a}_{j}^{\conj}\bm{x}^{\star}\right|^{2}-\left\Vert \bm{x}^{\star}\right\Vert _{2}^{2}\right)^{2}\lesssim \tau.
\]
\end{lemma}
With the above bound in mind, we can sum over all bins
of size $\tau$ to obtain 
\begin{align*}
 & \left|\bm{b}_{1}^{\conj}\sum_{k=0}^{ \frac{m}{\tau} -1 }\sum_{j=1}^{\tau}\left(\bm{b}_{k\tau+j}-\bm{b}_{k\tau+1}\right)\bm{b}_{k\tau+j}^{\conj}\tilde{\bm{h}}^{t}\left\{ \left|\bm{a}_{l+j}^{\conj}\bm{x}^{\star}\right|^{2}-\left\Vert \bm{x}^{\star}\right\Vert _{2}^{2}\right\} \right|\\
 & \quad\lesssim\left\{ \sqrt{\tau}\sum_{k=0}^{ \frac{m}{\tau} -1 }\sqrt{\sum_{j=1}^{\tau}\left|\bm{b}_{1}^{\conj}\left(\bm{b}_{k\tau+j}-\bm{b}_{k\tau+1}\right)\right|^{2}}\right\} \max_{1\leq l\leq m}\left|\bm{b}_{l}^{\conj}\tilde{\bm{h}}^{t}\right|\\
 & \quad
   \leq0.1 \max_{1\leq l\leq m}\left|\bm{b}_{l}^{\conj}\tilde{\bm{h}}^{t}\right|.
\end{align*}
Here, the last line arises from Lemma \ref{lemma:fourier-sum-square},
which says that for any small constant $c>0$, as long as $m\gg\tau K\log m$
\[
\sum_{k=0}^{ \frac{m}{\tau} -1 }\sqrt{\sum_{j=1}^{\tau}\left|\bm{b}_{1}^{\conj}\left(\bm{b}_{k\tau+j}-\bm{b}_{k\tau+1}\right)\right|^{2}}\leq c\frac{1}{\sqrt{\tau}}.
\]
\item The third term of (\ref{eq:bz2-terms}) obeys 
\begin{align*}
 & \left|\bm{b}_{1}^{\conj}\sum_{j=1}^{\tau}\bm{b}_{l+1}\left(\bm{b}_{l+j}-\bm{b}_{l+1}\right)^{\conj}\tilde{\bm{h}}^{t}\left\{ \left|\bm{a}_{l+j}^{\conj}\bm{x}^{\star}\right|^{2}-\left\Vert \bm{x}^{\star}\right\Vert _{2}^{2}\right\} \right|\\
 & \quad\leq\left|\bm{b}_{1}^{\conj}\bm{b}_{l+1}\right|\left\{ \sum_{j=1}^{\tau}\left|\left|\bm{a}_{l+j}^{\conj}\bm{x}^{\star}\right|^{2}-\left\Vert \bm{x}^{\star}\right\Vert _{2}^{2}\right|\right\} \max_{0\leq l\leq m-\tau,\,1\leq j\leq \tau}\left|\left(\bm{b}_{l+j}-\bm{b}_{l+1}\right)^{\conj}\tilde{\bm{h}}^{t}\right|\\
 & \quad\lesssim\tau\left|\bm{b}_{1}^{\conj}\bm{b}_{l+1}\right|\max_{0\leq l\leq m-\tau,\,1\leq j\leq \tau}\left|\left(\bm{b}_{l+j}-\bm{b}_{l+1}\right)^{\conj}\tilde{\bm{h}}^{t}\right|,
\end{align*}
where the last line relies on the inequality 
\[
\sum_{j=1}^{\tau}\left|\left|\bm{a}_{l+j}^{\conj}\bm{x}^{\star}\right|^{2}-\left\Vert \bm{x}^{\star}\right\Vert _{2}^{2}\right|\leq\sqrt{\tau}\sqrt{\sum_{j=1}^{\tau}\left(\left|\bm{a}_{l+j}^{\conj}\bm{x}^{\star}\right|^{2}-\left\Vert \bm{x}^{\star}\right\Vert _{2}^{2}\right)^{2}}\lesssim\tau
\]
owing to Lemma \ref{lemma:gaussian-hypercontractivity} and the Cauchy-Schwarz
inequality. Summing over all bins gives 
\begin{align*}
 & \sum_{k=0}^{ \frac{m}{\tau} -1 }\left|\bm{b}_{1}^{\conj}\sum_{j=1}^{\tau}\bm{b}_{k\tau+1}\left(\bm{b}_{k\tau+j}-\bm{b}_{k\tau+1}\right)^{\conj}\tilde{\bm{h}}^{t}\left\{ \left|\bm{a}_{k\tau+j}^{\conj}\bm{x}^{\star}\right|^{2}-\left\Vert \bm{x}^{\star}\right\Vert _{2}^{2}\right\} \right|\\
 & \quad\lesssim\tau\sum_{k=0}^{ \frac{m}{\tau} -1 }\left|\bm{b}_{1}^{\conj}\bm{b}_{k\tau+1}\right|\max_{0\leq l\leq m-\tau,\,1\leq j\leq \tau}\left|\left(\bm{b}_{l+j}-\bm{b}_{l+1}\right)^{\conj}\tilde{\bm{h}}^{t}\right|\\
 & \quad\lesssim\log m\max_{0\leq l\leq m-\tau,\,1\leq j\leq \tau}\left|\left(\bm{b}_{l+j}-\bm{b}_{l+1}\right)^{\conj}\tilde{\bm{h}}^{t}\right|,
\end{align*}
where the last relation makes use of (\ref{eq:sum-fourier-gap}) with the proviso that $m \gg K\tau$.
It then boils down to bounding $\max_{0\leq l\leq m-\tau,\,1\leq j\leq \tau}\big|\left(\bm{b}_{l+j}-\bm{b}_{l+1}\right)^{\conj}\tilde{\bm{h}}^{t}\big|$.
Without loss of generality, it suffices to look at $\big|(\bm{b}_{j}-\bm{b}_{1})^{\conj}\tilde{\bm{h}}^{t}\big|$
for all $1\leq j\leq\tau$. Specifically, we claim for the moment that
\begin{equation}
\max_{1\leq j \leq \tau}\left|\left(\bm{b}_{j}-\bm{b}_{1}\right)^{\conj}\tilde{\bm{h}}^{t}\right|\leq cC_{4}\frac{\mu}{\sqrt{m}}\log m\label{eq:bj-b1-h-BD}
\end{equation}
for some sufficiently small constant $c>0$, provided that $m\gg\tau K\log^{4}m$.
As a result, 
\begin{align*}
 & \sum_{k=0}^{ \frac{m}{\tau} -1 }\left|\bm{b}_{1}^{\conj}\sum_{j=1}^{\tau}\bm{b}_{k\tau+1}\left(\bm{b}_{k\tau+j}-\bm{b}_{k\tau+1}\right)^{\conj}\tilde{\bm{h}}^{t}\left\{ \left|\bm{a}_{k\tau+j}^{\conj}\bm{x}^{\star}\right|^{2}-\left\Vert \bm{x}^{\star}\right\Vert _{2}^{2}\right\} \right|\lesssim cC_{4}\frac{\mu}{\sqrt{m}}\log^{2}m.
\end{align*}
\item Putting the above results together, we get 
\[
\frac{1}{\eta\xi}\left|\bm{b}_{1}^{\conj}\bm{v}_{2}\right|\leq \sum_{k=0}^{ \frac{m}{\tau} -1 }\left|\bm{b}_{1}^{\conj}\sum_{j=1}^{\tau}\bm{b}_{k\tau+j}\bm{b}_{k\tau+j}^{\conj}\tilde{\bm{h}}^{t}\left\{ \left|\bm{a}_{k\tau+j}^{\conj}\bm{x}^{\star}\right|^{2}-\left\Vert \bm{x}^{\star}\right\Vert _{2}^{2}\right\} \right|\leq0.2\max_{1\leq l\leq m}\left|\bm{b}_{l}^{\conj}\tilde{\bm{h}}^{t}\right|+O\left( cC_{4}\frac{\mu}{\sqrt{m}}\log^{2}m\right).
\]
\end{itemize}
\item Combining the preceding bounds guarantees the existence of some constant
$C_{8}>0$ such that 
\begin{align*}
\left|\bm{b}_{l}^{\conj}\tilde{\bm{h}}^{t+1}\right| & \leq\left(1+\delta\right)\left\{ \left(1-\eta\xi\right)\left|\bm{b}_{l}^{\conj}\tilde{\bm{h}}^{t}\right|+0.3\eta\xi\max_{1\leq l\leq m}\left|\bm{b}_{l}^{\conj}\tilde{\bm{h}}^{t}\right|+C_{8}(1+C_{3})\eta\xi\frac{\mu}{\sqrt{m}}+C_{8}\eta\xi cC_{4}\frac{\mu}{\sqrt{m}}\log^{2}m\right\} \\
 & \overset{(\text{i})}{\leq}\left(1+O\left(\frac{1}{\log^{2}m}\right)\right)\left\{ \left(1-0.7\eta\xi\right)C_{4}\frac{\mu}{\sqrt{m}}\log^{2}m+C_{8}(1+C_{3})\eta\xi\frac{\mu}{\sqrt{m}}+C_{8}\eta\xi cC_{4}\frac{\mu}{\sqrt{m}}\log^{2}m\right\} \\
 & \overset{(\text{ii})}{\leq}C_{4}\frac{\mu}{\sqrt{m}}\log^{2}m.
\end{align*}
Here, (i) uses the induction hypothesis (\ref{eq:incoherence-hypothesis-bh-BD}),
and (ii) holds as long as $c>0$ is sufficiently small (so that $(1+\delta)C_{8}\eta\xi c\ll1$)
and $\eta>0$ is some sufficiently small constant. In order for the
proof to go through, it suffices to pick 
\[
\tau=c_{10}\log^{4}m
\]
for some sufficiently large constant $c_{10}>0$. Accordingly, we
need the sample size to exceed 
\[
m\gg\mu^{2}\tau K\log^{4}m\asymp\mu^{2}K\log^{8}m.
\]

\end{enumerate}

Finally, it remains to verify the claim (\ref{eq:bj-b1-h-BD}), which
we accomplish in Appendix \ref{sec:proof-claim-bj-b1-h}. 

\subsubsection{Proof of Lemma \ref{lemma:incoherence-fix}}\label{proof_lemma:incoherence-fix}
Denote
\[
w_{j}=\bm{b}_{l}^{\conj}\bm{b}_{j}\bm{b}_{j}^{\conj}\bm{h}^{\star}\bm{x}^{\star\conj}\bm{a}_{j}\bm{a}_{j}^{\conj}\bm{x}^{\star}.
\]
Recognizing that $\mathbb{E}[\bm{a}_{j}\bm{a}_{j}^{\conj}]=\bm{I}_{K}$ and
$\sum_{j=1}^{m}\bm{b}_{j}\bm{b}_{j}^{\conj}=\bm{I}_{K}$, we can write the
quantity of interest as the sum of independent random variables, namely,
\[
\sum_{j=1}^{m}\bm{b}_{l}^{\conj}\bm{b}_{j}\bm{b}_{j}^{\conj}\bm{h}^{\star}\bm{x}^{\star\conj}\bm{a}_{j}\bm{a}_{j}^{\conj}\bm{x}^{\star}-\bm{b}_{l}^{\conj}\bm{h}^{\star}=\sum_{j=1}^{m}\left(w_{j}-\EE\left[w_{j}\right]\right).
\]
Further, the sub-exponential norm (see definition in \cite{Vershynin2012})
of $w_{j}-\EE\left[w_{j}\right]$ obeys 
\begin{align*}
\left\Vert w_{j}-\EE\left[w_{j}\right]\right\Vert _{\psi_{1}} & \overset{\text{(i)}}{\leq}2\left\Vert w_{j}\right\Vert _{\psi_{1}}\overset{\text{(ii)}}{\leq}4\left|\bm{b}_{l}^{\conj}\bm{b}_{j}\right|\left|\bm{b}_{j}^{\conj}\bm{h}^{\star}\right|\left\Vert \bm{a}_{j}^{\conj}\bm{x}^{\star}\right\Vert _{\psi_{2}}^{2}\overset{\text{(iii)}}{\lesssim} \left|\bm{b}_{l}^{\conj}\bm{b}_{j}\right|\frac{\mu}{\sqrt{m}}
\overset{\text{(iv)}}{\leq} \frac{\mu\sqrt{K}}{m},
\end{align*}
where (i) arises from the centering property of the sub-exponential norm (see \cite[Remark 5.18]{Vershynin2012}), (ii) utilizes the relationship between the sub-exponential norm and the sub-Gaussian norm~\cite[Lemma 5.14]{Vershynin2012} and (iii) is a consequence of the incoherence condition (\ref{eq:incoherence-BD})
and the fact that $\left\Vert \bm{a}_{j}^{\conj}\bm{x}^{\star}\right\Vert _{\psi_{2}}\lesssim1$, and (iv) follows from $\| \bm{b}_j\|_2 = \sqrt{K/m} $. Let $M = \max_{j\in[m]} \left\Vert w_{j}-\EE\left[w_{j}\right]\right\Vert _{\psi_{1}} $ and
\[
V^2=\sum_{j=1}^{m} \left\Vert w_{j}-\EE\left[w_{j}\right]\right\Vert _{\psi_{1}}^2\lesssim 
\sum_{j=1}^{m}
\left( \left|\bm{b}_{l}^{\conj}\bm{b}_{j}\right|\frac{\mu}{\sqrt{m}} \right)^2
=\frac{\mu^2}{m} \left\|\bm{b}_{l}\right\|_2^2 =\frac{\mu^2 K}{m^2},
\]
which follows since $\sum_{j=1}^{m} \left|\bm{b}_{l}^{\conj}\bm{b}_{j}\right|^2 = \bm{b}_l^\conj \left(
\sum_{j=1}^{m} \bm{b}_j \bm{b}_j^\conj \right) \bm{b}_l = \| \bm{b}_l\|_2^2 = K/m$. Let $a_j = \left\Vert w_{j}-\EE\left[w_{j}\right]\right\Vert _{\psi_{1}} $ and $X_j =(w_j - \EE [w_j]) / a_j$. Since $\| X_j \|_{\psi_1} =1$, $\sum_{j=1}^{m}a_j^2=V^2$ and $\max_{j\in[m]} |a_j| = M$, we can invoke \cite[Proposition 5.16]{Vershynin2012} to obtain that
\[
\PP\left(\left| \sum_{j=1}^{m}a_j X_{j}\right|\geq t\right)\leq 2\exp\left(-c\min\left\{ \frac{t}{M},\frac{t^2}{V^2 }\right\} \right),
\]
where $c>0$ is some universal constant. By taking $t = \mu/\sqrt{m}$, we
see there exists some constant $c'$ such that
\begin{align*}
\PP\left( \left|\sum_{j=1}^{m}\bm{b}_{l}^{\conj}\bm{b}_{j}\bm{b}_{j}^{\conj}\bm{h}^{\star}\bm{x}^{\star\conj}\bm{a}_{j}\bm{a}_{j}^{\conj}\bm{x}^{\star}-\bm{b}_{l}^{\conj}\bm{h}^{\star}\right| \geq \frac{ \mu}{ \sqrt{m} }
\right)
&\leq 2\exp\left(-c\min\left\{ \frac{\mu/\sqrt{m}}{M},\frac{ \mu^2 / m}{V^2 }\right\} \right)\\
& \leq 2\exp\left(-c'\min\left\{ \frac{\mu/\sqrt{m}}{\mu\sqrt{K}/m},\frac{\mu^2 / m}{\mu^2 K/m^2 }\right\} \right) \\
&=2\exp\left(-c'\min\left\{ \sqrt{m/K}, m/K \right\} \right).
\end{align*}
We conclude the proof by observing that $m\gg K\log^2 m$ as stated in the assumption.

\subsubsection{Proof of Lemma \ref{lemma:gaussian-hypercontractivity}}\label{proof_lemma:gaussian-hypercontractivity}
From the elementary inequality $\left(a-b\right)^{2}\leq2\left(a^{2}+b^{2}\right)$,
we see that 
\begin{equation}
\sum_{j=1}^{\tau}\left(\left|\bm{a}_{j}^{\conj}\bm{x}^{\star}\right|^{2}-\left\Vert \bm{x}^{\star}\right\Vert _{2}^{2}\right)^{2}\leq2\sum_{j=1}^{\tau}\left(\left|\bm{a}_{j}^{\conj}\bm{x}^{\star}\right|^{4}+\left\Vert \bm{x}^{\star}\right\Vert _{2}^{4}\right)=2\sum_{j=1}^{\tau}\left|\bm{a}_{j}^{\conj}\bm{x}^{\star}\right|^{4}+2\tau,\label{eq:UB35}
\end{equation}
where the last identity holds true since $\left\Vert \bm{x}^{\star}\right\Vert _{2}=1$.
It thus suffices to control $\sum_{j=1}^{\tau}\left|\bm{a}_{j}^{\conj}\bm{x}^{\star}\right|^{4}$.
Let $\xi_{i}=\bm{a}_{j}^{\conj}\bm{x}^{\star}$, which is a standard
complex Gaussian random variable. Since the $\xi_{i}$'s are statistically
independent, one has 
\[
\mathsf{Var}\left(\sum_{i=1}^{\tau}|\xi_{i}|^{4}\right)\leq C_{4}\tau
\]
for some constant $C_{4}>0$. It then follows from the hypercontractivity
concentration result for Gaussian polynomials that \cite[Theorem 1.9]{schudy2012concentration}
\begin{align*}
\mathbb{P}\left\{ \sum_{i=1}^{\tau}\left(|\xi_{i}|^{4}-\mathbb{E}\left[|\xi_{i}|^{4}\right]\right)\geq c\tau\right\}  & \leq C\exp\left(-c_{2}\left(\frac{c^{2}\tau^{2}}{\mathsf{Var}\left(\sum_{i=1}^{\tau}|\xi_{i}|^{4}\right)}\right)^{1/4}\right)\\
 & \leq C\exp\left(-c_{2}\left(\frac{c^{2}\tau^{2}}{C_{4}\tau}\right)^{1/4}\right)=C\exp\left(-c_{2}\left(\frac{c^{2}}{C_{4}}\right)^{1/4}\tau^{1/4}\right)\\
 & \leq O(m^{-10}),
\end{align*}
for some constants $c,c_{2},C>0$, with the proviso that $\tau\gg\log^{4}m$.
As a consequence, with probability at least $1-O(m^{-10})$, 
\[
\sum_{j=1}^{\tau}\left|\bm{a}_{j}^{\conj}\bm{x}^{\star}\right|^{4}\lesssim\tau+\sum_{j=1}^{\tau}\mathbb{E}\left[\left|\bm{a}_{j}^{\conj}\bm{x}^{\star}\right|^{4}\right]\asymp\tau,
\]
which together with (\ref{eq:UB35}) concludes the proof.

\subsubsection{Proof of Claim (\ref{eq:bj-b1-h-BD})}
\label{sec:proof-claim-bj-b1-h}
We will prove
the claim by induction. Again, observe that 
\[
\left|\left(\bm{b}_{j}-\bm{b}_{1}\right)^{\conj}\tilde{\bm{h}}^{t}\right|=\left|\left(\bm{b}_{j}-\bm{b}_{1}\right)^{\conj}\frac{1}{\overline{\alpha^{t}}}\bm{h}^{t}\right|=\left|\frac{\alpha^{t-1}}{\alpha^{t}}\right|\left|\left(\bm{b}_{j}-\bm{b}_{1}\right)^{\conj}\frac{1}{\overline{\alpha^{t-1}}}\bm{h}^{t}\right|\leq\left(1+\delta\right)\left|\left(\bm{b}_{j}-\bm{b}_{1}\right)^{\conj}\frac{1}{\overline{\alpha^{t-1}}}\bm{h}^{t}\right|
\]
for some $\delta\asymp {\log^{-2}m}$, which allows us to look
at $\left(\bm{b}_{j}-\bm{b}_{1}\right)^{\conj}\frac{1}{\overline{\alpha^{t-1}}}\bm{h}^{t}$
instead.

Use the gradient update rule for $\bm{h}^{t}$ (cf. (\ref{eq:gradient-update-h-Bd}))
once again to get 
\begin{align*}
\frac{1}{\overline{\alpha^{t-1}}}\bm{h}^{t} & =\frac{1}{\overline{\alpha^{t-1}}}\left(\bm{h}^{t-1}-\frac{\eta}{\left\Vert \bm{x}^{t-1}\right\Vert _{2}^{2}}\sum_{l=1}^{m}\bm{b}_{l}\bm{b}_{l}^{\conj}\left(\bm{h}^{t-1}\bm{x}^{t-1\conj}-\bm{h}^{\star}\bm{x}^{\star\conj}\right)\bm{a}_{l}\bm{a}_{l}^{\conj}\bm{x}^{t-1}\right)\\
 & =\tilde{\bm{h}}^{t-1}-\eta\theta\sum_{l=1}^{m}\bm{b}_{l}\bm{b}_{l}^{\conj}\left(\tilde{\bm{h}}^{t-1}\tilde{\bm{x}}^{t-1\conj}-\bm{h}^{\star}\bm{x}^{\star\conj}\right)\bm{a}_{l}\bm{a}_{l}^{\conj}\tilde{\bm{x}}^{t-1},
\end{align*}
where we denote 
$
\theta:={1}/{\left\Vert \tilde{\bm{x}}^{t-1}\right\Vert _{2}^{2}}.
$
This further gives rise to 
\begin{align*}
\left(\bm{b}_{j}-\bm{b}_{1}\right)^{\conj}\frac{1}{\overline{\alpha^{t-1}}}\bm{h}^{t} & =\left(\bm{b}_{j}-\bm{b}_{1}\right)^{\conj}\tilde{\bm{h}}^{t-1}-\eta\theta\left(\bm{b}_{j}-\bm{b}_{1}\right)^{\conj}\sum_{l=1}^{m}\bm{b}_{l}\bm{b}_{l}^{\conj}\left(\tilde{\bm{h}}^{t-1}\tilde{\bm{x}}^{t-1\conj}-\bm{h}^{\star}\bm{x}^{\star\conj}\right)\bm{a}_{l}\bm{a}_{l}^{\conj}\tilde{\bm{x}}^{t-1}\\
 & =\left(\bm{b}_{j}-\bm{b}_{1}\right)^{\conj}\tilde{\bm{h}}^{t-1}-\eta\theta\left(\bm{b}_{j}-\bm{b}_{1}\right)^{\conj}\sum_{l=1}^{m}\bm{b}_{l}\bm{b}_{l}^{\conj}\left(\tilde{\bm{h}}^{t-1}\tilde{\bm{x}}^{t-1\conj}-\bm{h}^{\star}\bm{x}^{\star\conj}\right)\tilde{\bm{x}}^{t-1}\\
 & \qquad-\eta\theta\left(\bm{b}_{j}-\bm{b}_{1}\right)^{\conj}\sum_{l=1}^{m}\bm{b}_{l}\bm{b}_{l}^{\conj}\left(\tilde{\bm{h}}^{t-1}\tilde{\bm{x}}^{t-1\conj}-\bm{h}^{\star}\bm{x}^{\star\conj}\right)\left(\bm{a}_{l}\bm{a}_{l}^{\conj}-\bm{I}_{K}\right)\tilde{\bm{x}}^{t-1}\\
 & =\left(1-\eta\theta\|\tilde{\bm{x}}^{t-1}\|_{2}^{2}\right)\left(\bm{b}_{j}-\bm{b}_{1}\right)^{\conj}\tilde{\bm{h}}^{t-1}+\underbrace{\eta\theta\left(\bm{b}_{j}-\bm{b}_{1}\right)^{\conj}\bm{h}^{\star}\left(\bm{x}^{\star\conj}\tilde{\bm{x}}^{t-1}\right)}_{:=\beta_{1}}\\
 & \qquad-\underbrace{\eta\theta\left(\bm{b}_{j}-\bm{b}_{1}\right)^{\conj}\sum_{l=1}^{m}\bm{b}_{l}\bm{b}_{l}^{\conj}\left(\tilde{\bm{h}}^{t-1}\tilde{\bm{x}}^{t-1\conj}-\bm{h}^{\star}\bm{x}^{\star\conj}\right)\left(\bm{a}_{l}\bm{a}_{l}^{\conj}-\bm{I}_{K}\right)\tilde{\bm{x}}^{t-1}}_{:=\beta_{2}},
\end{align*}
where the last identity makes use of the fact that $\sum_{l=1}^{m}\bm{b}_{l}\bm{b}_{l}^{\conj} = \bm{I}_{K}$. 
For $\beta_{1}$, one can get 
\[
\frac{1}{\eta\theta}\left|\beta_{1}\right|\leq\left|\left(\bm{b}_{j}-\bm{b}_{1}\right)^{\conj}\bm{h}^{\star}\right|\|\bm{x}^{\star}\|_{2}\|\tilde{\bm{x}}^{t-1}\|_{2}\leq4\frac{\mu}{\sqrt{m}},
\]
where we utilize the incoherence condition (\ref{eq:incoherence-BD})
and the fact that $\tilde{\bm{x}}^{t-1}$ and $\bm{x}^{\star}$
are extremely close, i.e. 
\[
\left\Vert \tilde{\bm{x}}^{t-1}-\bm{x}^{\star}\right\Vert _{2}\leq\text{dist}\left(\bm{z}^{t-1},\bm{z}^{\star}\right)\ll1\qquad\Longrightarrow\qquad\|\tilde{\bm{x}}^{t-1}\|_{2}\leq2.
\]
Regarding the second term $\beta_{2}$, we have 
\begin{align*}
\frac{1}{\eta\theta}\left|\beta_{2}\right| & \leq\left\{ \sum_{l=1}^{m}\left|\left(\bm{b}_{j}-\bm{b}_{1}\right)^{\conj}\bm{b}_{l}\right|\right\} \underbrace{\max_{1\leq l\leq m}\left|\bm{b}_{l}^{\conj}\left(\tilde{\bm{h}}^{t-1}\tilde{\bm{x}}^{t-1\conj}-\bm{h}^{\star}\bm{x}^{\star\conj}\right)\left(\bm{a}_{l}\bm{a}_{l}^{\conj}-\bm{I}_{K}\right)\tilde{\bm{x}}^{t-1}\right|}_{:=\psi}.
\end{align*}
The term $\psi$ can be bounded as follows 
\begin{align*}
\psi & \leq\max_{1\leq l\leq m}\left|\bm{b}_{l}^{\conj}\tilde{\bm{h}}^{t-1}\tilde{\bm{x}}^{t-1\conj}\left(\bm{a}_{l}\bm{a}_{l}^{\conj}-\bm{I}\right)\tilde{\bm{x}}^{t-1}\right|+\max_{1\leq l\leq m}\left|\bm{b}_{l}^{\conj}\bm{h}^{\star}\bm{x}^{\star\conj}\left(\bm{a}_{l}\bm{a}_{l}^{\conj}-\bm{I}_{K}\right)\tilde{\bm{x}}^{t-1}\right|\\
 & \leq\max_{1\leq l\leq m}\left|\bm{b}_{l}^{\conj}\tilde{\bm{h}}^{t-1}\right|\max_{1\leq l\leq m}\left|\tilde{\bm{x}}^{t-1\conj}\left(\bm{a}_{l}\bm{a}_{l}^{\conj}-\bm{I}_{K}\right)\tilde{\bm{x}}^{t-1}\right|+\max_{1\leq l\leq m}\left|\bm{b}_{l}^{\conj}\bm{h}^{\star}\right|\max_{1\leq l\leq m}\left|\bm{x}^{\star\conj}\left(\bm{a}_{l}\bm{a}_{l}^{\conj}-\bm{I}_{K}\right)\tilde{\bm{x}}^{t-1}\right|\\
 & \lesssim\log m\left\{ \max_{1\leq l\leq m}\left|\bm{b}_{l}^{\conj}\tilde{\bm{h}}^{t-1}\right|+\frac{\mu}{\sqrt{m}}\right\} .
\end{align*}
Here, we have used the incoherence condition (\ref{eq:incoherence-BD}) and the facts that 
\begin{align*}
\left|(\tilde{\bm{x}}^{t-1})^{\conj}\left(\bm{a}_{l}\bm{a}_{l}^{\conj}-\bm{I}\right)\tilde{\bm{x}}^{t-1}\right| & \leq\big\|\bm{a}_{l}^{\conj}\tilde{\bm{x}}^{t-1}\big\|_{2}^{2}+\big\|\tilde{\bm{x}}^{t-1}\big\|_{2}^{2}\lesssim\log m,\\
\left|\bm{x}^{\star\conj}\left(\bm{a}_{l}\bm{a}_{l}^{\conj}-\bm{I}\right)\tilde{\bm{x}}^{t-1}\right| & \leq\big\|\bm{a}_{l}^{\conj}\tilde{\bm{x}}^{t-1}\big\|_{2}\big\|\bm{a}_{l}^{\conj}\bm{x}^{\star}\big\|_{2}+\big\|\tilde{\bm{x}}^{t-1}\big\|_{2}\|\bm{x}^{\star}\|_{2}\lesssim\log m,
\end{align*}
which are immediate consequences of (\ref{eq:incoherence-hypothesis-ax-BD})
and (\ref{eq:max_gaussian-1}). Combining this with Lemma \ref{lemma:fourier-sum-diff},
we see that for any small constant $c>0$
\[
\frac{1}{\eta\theta}\left|\beta_{2}\right|\leq c\frac{1}{\log m}\left\{ \max_{1\leq l\leq m}\left|\bm{b}_{l}^{\conj}\tilde{\bm{h}}^{t-1}\right|+\frac{\mu}{\sqrt{m}}\right\}
\]
holds as long as $m\gg \tau K\log^{4}m$.

To summarize, we arrive at 
\begin{align*}
\left|\left(\bm{b}_{j}-\bm{b}_{1}\right)^{\conj}\tilde{\bm{h}}^{t}\right| & \leq\left(1+\delta\right)\left\{ \left(1-\eta\theta\left\Vert \tilde{\bm{x}}^{t-1}\right\Vert _{2}^{2}\right)\left|\left(\bm{b}_{j}-\bm{b}_{1}\right)^{\conj}\tilde{\bm{h}}^{t-1}\right|+4\eta\theta\frac{\mu}{\sqrt{m}}+c\eta\theta\frac{1}{\log m}\left[ \max_{1\leq l\leq m}\left|\bm{b}_{l}^{\conj}\tilde{\bm{h}}^{t-1}\right|+\frac{\mu}{\sqrt{m}}\right]\right\} .
\end{align*}
Making use of the induction hypothesis (\ref{eq:induction-incoherence-b-BD})
and the fact that $\left\Vert \tilde{\bm{x}}^{t-1}\right\Vert _{2}^{2}\geq0.9$,
we reach 
\[
\left|\left(\bm{b}_{j}-\bm{b}_{1}\right)^{\conj}\tilde{\bm{h}}^{t}\right|\leq\left(1+\delta\right)\left\{ \left(1-0.9\eta\theta\right)\left|\left(\bm{b}_{j}-\bm{b}_{1}\right)^{\conj}\tilde{\bm{h}}^{t-1}\right|+cC_{4}\eta\theta\frac{\mu}{\sqrt{m}}\log m+\frac{c\mu\eta\theta}{\sqrt{m}\log m}\right\} .
\]
Recall that $\delta\asymp1/\log^{2}m$. As a result, if $\eta>0$
is some sufficiently small constant and if 
\[
\left|\left(\bm{b}_{j}-\bm{b}_{1}\right)^{\conj}\tilde{\bm{h}}^{t-1}\right|\leq10c\left(C_{4}\frac{\mu}{\sqrt{m}}\log m+\frac{\mu}{\eta\theta\sqrt{m}\log m}\right)\leq20cC_{4}\frac{\mu}{\sqrt{m}}\log m
\]
holds, then one has
\[
\left|\left(\bm{b}_{j}-\bm{b}_{1}\right)^{\conj}\tilde{\bm{h}}^{t}\right|\leq20cC_{4}\frac{\mu}{\sqrt{m}}\log m.
\]

Therefore, this concludes the proof of the claim (\ref{eq:bj-b1-h-BD})
by induction, provided that the base case is true, i.e.~for some
$c>0$ sufficiently small 
\begin{equation}
\left|\left(\bm{b}_{j}-\bm{b}_{1}\right)^{\conj}\tilde{\bm{h}}^{0}\right|\leq20cC_{4}\frac{\mu}{\sqrt{m}}\log m.\label{eq:incoherence-h-diff-base-BD}
\end{equation}
The claim (\ref{eq:incoherence-h-diff-base-BD}) is proved in Appendix
\ref{subsec:Proof-of-Lemma-perturbation-init-BD} (see Lemma \ref{lemma:diff-base-BD}).

\subsection{Proof of Lemma \ref{lemma:spectral-BD-L2}\label{sec:proof-of-lemma:spectral-BD-L2}}

Recall that $\check{\bm{h}}^{0}$ and $\check{\bm{x}}^{0}$ are the
leading left and right singular vectors of $\bm{M}$, respectively.
Applying a variant of Wedin's sin$\Theta$ theorem \cite[Theorem 2.1]{dopico2000note},
we derive that 
\begin{align}
 & \min_{\alpha\in\mathbb{C},|\alpha|=1}\left\{ \big\|\alpha\check{\bm{h}}^{0}-\bm{h}^{\star}\big\|_{2}+\big\|\alpha\check{\bm{x}}^{0}-\bm{x}^{\star}\big\|_{2}\right\} \leq\frac{c_{1}\left\Vert \bm{M}-\mathbb{E}\left[\bm{M}\right]\right\Vert }{\sigma_{1}\left(\mathbb{E}\left[\bm{M}\right]\right)-\sigma_{2}\left(\bm{M}\right)},\label{eq:UB20}
\end{align}
for some universal constant $c_{1}>0$. Regarding the numerator of
(\ref{eq:UB20}), it has been shown in \cite[Lemma 5.20]{DBLP:journals/corr/LiLSW16}
that for any $\xi>0$, 
\begin{equation}
\left\Vert \bm{M}-\mathbb{E}\left[\bm{M}\right]\right\Vert \leq\xi\label{eq:spectral-gap}
\end{equation}
with probability exceeding $1-O(m^{-10})$, provided that 
\[
m\geq\frac{c_{2}\mu^{2}K\log^{2}m}{\xi^{2}}
\]
for some universal constant $c_{2}>0$. For the denominator of (\ref{eq:UB20}),
we can take (\ref{eq:spectral-gap}) together with Weyl's inequality
to demonstrate that 
\begin{align*}
\sigma_{1}\left(\mathbb{E}\left[\bm{M}\right]\right)-\sigma_{2}\left(\bm{M}\right) & \geq\sigma_{1}\left(\mathbb{E}\left[\bm{M}\right]\right)-\sigma_{2}\left(\EE\left[\bm{M}\right]\right)-\left\Vert \bm{M}-\EE\left[\bm{M}\right]\right\Vert  \geq1-\xi,
\end{align*}
where the last inequality utilizes the facts that $\sigma_{1}\left(\mathbb{E}\left[\bm{M}\right]\right)=1$
and $\sigma_{2}\left(\mathbb{E}[\bm{M}]\right)=0$. These together
with (\ref{eq:UB20}) reveal that 
\begin{align}
\min_{\alpha\in\mathbb{C},|\alpha|=1}\left\{ \big\|\alpha\check{\bm{h}}^{0}-\bm{h}^{\star}\big\|_{2}+\big\|\alpha\check{\bm{x}}^{0}-\bm{x}^{\star}\big\|_{2}\right\}  & \leq\frac{c_{1}\xi}{1-\xi}\leq2c_{1}\xi\label{eq:UB-19}
\end{align}
as long as $\xi\leq 1/2$.

Now we connect the preceding bound \eqref{eq:UB-19} with the
scaled singular vectors $\bm{h}^{0}=\sqrt{\sigma_{1}\left(\bm{M}\right)}\;\check{\bm{h}}^{0}$
and $\bm{x}^{0}=\sqrt{\sigma_{1}\left(\bm{M}\right)}\,\check{\bm{x}}^{0}$. For any $\alpha\in\CC$ with $|\alpha|=1$, from the definition
of $\bm{h}^{0}$ and $\bm{x}^{0}$ we have 
\begin{align*}
\left\Vert \alpha\bm{h}^{0}-\bm{h}^{\star}\right\Vert _{2}+\left\Vert \alpha\bm{x}^{0}-\bm{x}^{\star}\right\Vert _{2} & =\left\Vert \sqrt{\sigma_{1}\left(\bm{M}\right)}\left(\alpha\check{\bm{h}}^{0}\right)-\bm{h}^{\star}\right\Vert _{2}+\left\Vert \sqrt{\sigma_{1}\left(\bm{M}\right)}\left(\alpha\check{\bm{x}}^{0}\right)-\bm{x}^{\star}\right\Vert _{2}.
\end{align*}
Since $\alpha\check{\bm{h}}^{0},\alpha\check{\bm{x}}^{0}$ are also
the leading left and right singular vectors of $\bm{M}$, we can invoke
Lemma \ref{lemma:singular-value-difference} to get 
\begin{align}
\left\Vert \alpha\bm{h}^{0}-\bm{h}^{\star}\right\Vert _{2}+\left\Vert \alpha\bm{x}^{0}-\bm{x}^{\star}\right\Vert _{2} & \leq\sqrt{\sigma_{1}(\mathbb{E}[\bm{M}])}\left(\big\|\alpha\check{\bm{h}}^{0}-\bm{h}^{\star}\big\|_{2}+\big\|\alpha\check{\bm{x}}^{0}-\bm{x}^{\star}\big\|_{2}\right)+\frac{2\left|\sigma_{1}(\bm{M})-\sigma_{1}\left(\EE\left[\bm{M}\right]\right)\right|}{\sqrt{\sigma_{1}(\bm{M})}+\sqrt{\sigma_{1}\left(\EE\left[\bm{M}\right]\right)}}\nonumber \\
 & =\left\Vert \alpha\check{\bm{h}}^{0}-\bm{h}^{\star}\right\Vert _{2}+\left\Vert \alpha\check{\bm{x}}^{0}-\bm{x}^{\star}\right\Vert _{2}+\frac{2\left|\sigma_{1}(\bm{M})-\sigma_{1}\left(\EE\left[\bm{M}\right]\right)\right|}{\sqrt{\sigma_{1}(\bm{M})}+1}.\label{eq:singular-vector-diff}
\end{align}
In addition, we can apply Weyl's inequality once again to deduce that 
\begin{align}
\left|\sigma_{1}(\bm{M})-\sigma_{1}(\mathbb{E}[\bm{M}])\right|\leq\left\Vert \bm{M}-\mathbb{E}[\bm{M}]\right\Vert \leq\xi,\label{eq:sigma1-M-bound}
\end{align}
where the last inequality comes from (\ref{eq:spectral-gap}). Substitute
(\ref{eq:sigma1-M-bound}) into (\ref{eq:singular-vector-diff})
to obtain 
\begin{equation}
\left\Vert \alpha\bm{h}^{0}-\bm{h}^{\star}\right\Vert _{2}+\left\Vert \alpha\bm{x}^{0}-\bm{x}^{\star}\right\Vert _{2}\leq\left\Vert \alpha\check{\bm{h}}^{0}-\bm{h}^{\star}\right\Vert _{2}+\left\Vert \alpha\check{\bm{x}}^{0}-\bm{x}^{\star}\right\Vert _{2}+2\xi.\label{eq:singular-vector-diff-1}
\end{equation}
Taking the minimum over $\alpha$, one can thus conclude that 
\begin{align*}
\min_{\alpha\in\mathbb{C},|\alpha|=1}\left\{ \left\Vert \alpha\bm{h}^{0}-\bm{h}^{\star}\right\Vert _{2}+\left\Vert \alpha\bm{x}^{0}-\bm{x}^{\star}\right\Vert _{2}\right\}  & \leq\min_{\alpha\in\mathbb{C},|\alpha|=1}\left\{ \big\|\alpha\check{\bm{h}}^{0}-\bm{h}^{\star}\big\|_{2}+\big\|\alpha\check{\bm{x}}^{0}-\bm{x}^{\star}\big\|_{2}\right\} +2\xi\leq2c_{1}\xi+2\xi,
\end{align*}
where the last inequality comes from (\ref{eq:UB-19}). Since $\xi$
is arbitrary, by taking $m/(\mu^2 K \log^2 m )$ to be large enough, we finish the proof for (\ref{eq:spectral-L2-h0-1}). Carrying out similar arguments (which we omit here), we can also establish
(\ref{eq:spectral-L2-h0-1l}). 

The last claim in Lemma \ref{lemma:spectral-BD-L2} that $\left||\alpha_0|-1\right|  \leq 1/4$ is a direct corollary of (\ref{eq:spectral-L2-h0-1}) and Lemma \ref{lemma:alpha-close-to-one}.

\subsection{Proof of Lemma \ref{lemma:perturbabtion-init-BD}\label{subsec:Proof-of-Lemma-perturbation-init-BD}}

The proof is composed of three steps:
\begin{itemize}
\item In the first step, we  show that the normalized singular vectors
of $\bm{M}$ and $\bm{M}^{\left(l\right)}$ are close enough; see
(\ref{eq:UB-25}).
\item We then proceed by passing this proximity result to the scaled singular
vectors; see (\ref{eq:UB-25-scaled-BD}).
\item Finally, we translate the usual $\ell_{2}$ distance metric to
the distance function we defined in (\ref{eq:defn-dist-BD}); see
(\ref{eq:UB-25-1}). Along the way, we also prove the incoherence
of $\bm{h}^{0}$ with respect to $\left\{ \bm{b}_{l}\right\} $. 
\end{itemize}
Here comes the formal proof. Recall that $\check{\bm{h}}^{0}$ and
$\check{\bm{x}}^{0}$ are respectively the leading left and right
singular vectors of $\bm{M}$, and $\check{\bm{h}}^{0,\left(l\right)}$
and $\check{\bm{x}}^{0,\left(l\right)}$ are respectively the leading
left and right singular vectors of $\bm{M}^{(l)}$. Invoke  Wedin's
sin$\Theta$ theorem \cite[Theorem 2.1]{dopico2000note} to obtain
\begin{align*}
\min_{\alpha\in\mathbb{C},|\alpha|=1}\left\{ \big\Vert \alpha\check{\bm{h}}^{0}-\check{\bm{h}}^{0,\left(l\right)}\big\Vert _{2}+\big\Vert \alpha\check{\bm{x}}^{0}-\check{\bm{x}}^{0,\left(l\right)}\big \Vert _{2}\right\}  & \leq c_{1}\frac{\left\Vert \left(\bm{M}-\bm{M}^{(l)}\right)\check{\bm{x}}^{0,\left(l\right)}\right\Vert _{2}+\left\Vert \check{\bm{h}}^{0,\left(l\right)\conj}\left(\bm{M}-\bm{M}^{(l)}\right)\right\Vert _{2}}{\sigma_{1}\big(\bm{M}^{(l)}\big)-\sigma_{2}\left(\bm{M}\right)}
\end{align*}
for some universal constant $c_{1}>0$. Using the Weyl's inequality we get
\begin{align*}
\sigma_{1}\big(\bm{M}^{(l)}\big)-\sigma_{2}\left(\bm{M}\right) & \geq\sigma_{1}\big(\mathbb{E}[\bm{M}^{(l)}]\big)-\|\bm{M}^{(l)}-\mathbb{E}[\bm{M}^{(l)}]\|-\sigma_{2}\left(\mathbb{E}[\bm{M}]\right)-\|\bm{M}-\mathbb{E}[\bm{M}]\|\\
 & \geq3/4-\|\bm{M}^{(l)}-\mathbb{E}[\bm{M}^{(l)}]\|-\|\bm{M}-\mathbb{E}[\bm{M}]\|\geq1/2,
\end{align*}
where the penultimate inequality follows from 
\[
\sigma_{1}\big(\mathbb{E}[\bm{M}^{(l)}]\big)\geq3/4
\]
for $m$ sufficiently large, and the last inequality comes from \cite[Lemma 5.20]{DBLP:journals/corr/LiLSW16},
provided that $m\geq c_{2}\mu^{2}K\log^{2}m$ for some sufficiently
large constant $c_{2}>0$. As a result, denoting 
\begin{equation}
	\beta^{0,\left(l\right)}:=\argmin_{\alpha\in\CC,|\alpha|=1} \Big\{ \big\Vert \alpha\check{\bm{h}}^{0}-\check{\bm{h}}^{0,\left(l\right)}\big\Vert _{2}+\big\Vert \alpha\check{\bm{x}}^{0}-\check{\bm{x}}^{0,\left(l\right)}\big\Vert _{2} \Big\}
	\label{eq:defn-alpha-0-l}
\end{equation}
allows us to obtain 
\begin{align}
\big\Vert \beta^{0,\left(l\right)}\check{\bm{h}}^{0}-\check{\bm{h}}^{0,\left(l\right)}\big\Vert _{2}+ \big\Vert \beta^{0,\left(l\right)}\check{\bm{x}}^{0}-\check{\bm{x}}^{0,\left(l\right)}\big\Vert _{2} & \leq2c_{1}\left\{ \big\Vert \big(\bm{M}-\bm{M}^{(l)}\big)\check{\bm{x}}^{0,\left(l\right)}\big\Vert _{2}+\big\Vert \check{\bm{h}}^{0,\left(l\right)\conj}\big(\bm{M}-\bm{M}^{(l)}\big)\big\Vert _{2}\right\} .\label{eq:Davis-Kahan-SVD}
\end{align}

It then boils down to controlling the two terms on the right-hand
side of (\ref{eq:Davis-Kahan-SVD}). By construction, 
\[
\bm{M}-\bm{M}^{(l)}=\bm{b}_{l}\bm{b}_{l}^{\conj}\bm{h}^{\star}\bm{x}^{\star\conj}\bm{a}_{l}\bm{a}_{l}^{\conj}.
\]

\begin{itemize}
\item To bound the first term, observe that
\begin{align}
\left\Vert \big(\bm{M}-\bm{M}^{(l)}\big)\check{\bm{x}}^{0,\left(l\right)}\right\Vert _{2} & =\left\Vert \bm{b}_{l}\bm{b}_{l}^{\conj}\bm{h}^{\star}\bm{x}^{\star\conj}\bm{a}_{l}\bm{a}_{l}^{\conj}\check{\bm{x}}^{0,\left(l\right)}\right\Vert _{2}=\left\Vert \bm{b}_{l}\right\Vert _{2}\left|\bm{b}_{l}^{\conj}\bm{h}^{\star}\right|\left|\bm{a}_{l}^{\conj}\bm{x}^{\star}\right|\cdot\big|\bm{a}_{l}^{\conj}\check{\bm{x}}^{0,\left(l\right)}\big|\nonumber \\
 & \leq30\frac{\mu}{\sqrt{m}}\cdot\sqrt{\frac{K\log^{2}m}{m}},\label{eq:spectral-loop-M-x}
\end{align}
where we use the fact that $\|\bm{b}_{l}\|_{2}=\sqrt{K/m}$, the incoherence
condition (\ref{eq:incoherence-BD}), the bound (\ref{eq:max_gaussian-1})
and the fact that with probability exceeding $1-O\left(m^{-10}\right)$,
\[
\max_{1\leq l\leq m}\big|\bm{a}_{l}^{\conj}\check{\bm{x}}^{0,\left(l\right)}\big|\leq5\sqrt{\log m},
\]
due to the independence between $\check{\bm{x}}^{0,\left(l\right)}$
and $\bm{a}_{l}$. 
\item To bound the second term, for any $\tilde{\alpha}$ obeying $|\tilde{\alpha}|=1$
one has 
\begin{align*}
\left\Vert \check{\bm{h}}^{0,\left(l\right)\conj}\big(\bm{M}-\bm{M}^{(l)}\big)\right\Vert _{2} & =\left\Vert \check{\bm{h}}^{0,\left(l\right)\conj}\bm{b}_{l}\bm{b}_{l}^{\conj}\bm{h}^{\star}\bm{x}^{\star\conj}\bm{a}_{l}\bm{a}_{l}^{\conj}\right\Vert _{2}=\left\Vert \bm{a}_{l}\right\Vert _{2}\left|\bm{b}_{l}^{\conj}\bm{h}^{\star}\right|\left|\bm{a}_{l}^{\conj}\bm{x}^{\star}\right|\cdot\big|\bm{b}_{l}^{\conj}\check{\bm{h}}^{0,\left(l\right)}\big|\\
 & \overset{\left(\text{i}\right)}{\leq}3\sqrt{K}\cdot\frac{\mu}{\sqrt{m}}\cdot5\sqrt{\log m}\cdot\big|\bm{b}_{l}^{\conj}\check{\bm{h}}^{0,\left(l\right)}\big|\\
 & \overset{\left(\text{ii}\right)}{\leq}15\sqrt{\frac{\mu^{2}K\log m}{m}}\big|\tilde{\alpha}\bm{b}_{l}^{\conj}\check{\bm{h}}^{0}\big|+15\sqrt{\frac{\mu^{2}K\log m}{m}}\left|\bm{b}_{l}^{\conj}\big(\tilde{\alpha}\check{\bm{h}}^{0}-\check{\bm{h}}^{0,\left(l\right)}\big)\right|\\
 & \overset{\left(\text{iii}\right)}{\leq}15\sqrt{\frac{\mu^{2}K\log m}{m}}\big|\bm{b}_{l}^{\conj}\check{\bm{h}}^{0}\big|+15\sqrt{\frac{\mu^{2}K\log m}{m}}\cdot\sqrt{\frac{K}{m}}\left\Vert \tilde{\alpha}\check{\bm{h}}^{0}-\check{\bm{h}}^{0,\left(l\right)}\right\Vert _{2}.
\end{align*}
Here, (i) arises from the incoherence condition (\ref{eq:incoherence-BD})
together with the bounds (\ref{eq:max_gaussian-1}) and (\ref{eq:max_gaussian-2}),
the inequality (ii) comes from the triangle inequality, and the last
line (iii) holds since $\|\bm{b}_{l}\|_{2}=\sqrt{K/m}$ and $|\tilde{\alpha}|=1$. 
\end{itemize}
Substitution of the above bounds into (\ref{eq:Davis-Kahan-SVD}) yields 
\begin{align*}
 & \left\Vert \beta^{0,\left(l\right)}\check{\bm{h}}^{0}-\check{\bm{h}}^{0,\left(l\right)}\right\Vert _{2}+\left\Vert \beta^{0,\left(l\right)}\check{\bm{x}}^{0}-\check{\bm{x}}^{0,\left(l\right)}\right\Vert _{2}\\
 & \quad\leq2c_{1}\left\{ 30\frac{\mu}{\sqrt{m}}\cdot\sqrt{\frac{K\log^{2}m}{m}}+15\sqrt{\frac{\mu^{2}K\log m}{m}}\left|\bm{b}_{l}^{\conj}\check{\bm{h}}^{0}\right|+15\sqrt{\frac{\mu^{2}K\log m}{m}}\cdot\sqrt{\frac{K}{m}}\left\Vert \tilde{\alpha}\check{\bm{h}}^{0}-\check{\bm{h}}^{0,\left(l\right)}\right\Vert _{2}\right\} .
\end{align*}
Since the previous inequality holds for all $\left|\tilde{\alpha}\right|=1$,
we can choose $\tilde{\alpha}=\beta^{0,\left(l\right)}$ and rearrange
terms to get 
\begin{align*}
 & \left(1-30c_{1}\sqrt{\frac{\mu^{2}K\log m}{m}}\sqrt{\frac{K}{m}}\right)\left(\left\Vert \beta^{0,\left(l\right)}\check{\bm{h}}^{0}-\check{\bm{h}}^{0,\left(l\right)}\right\Vert _{2}+\left\Vert \beta^{0,\left(l\right)}\check{\bm{x}}^{0}-\check{\bm{x}}^{0,\left(l\right)}\right\Vert _{2}\right)\\
 & \quad\leq60c_{1}\frac{\mu}{\sqrt{m}}\cdot\sqrt{\frac{K\log^{2}m}{m}}+30c_{1}\sqrt{\frac{\mu^{2}K\log m}{m}}\left|\bm{b}_{l}^{\conj}\check{\bm{h}}^{0}\right|.
\end{align*}
Under the condition that $m\gg\mu K\log^{1/2}m$, one has $1-30c_{1}\sqrt{{\mu^{2}K\log m}/{m}}\cdot \sqrt{{K}/{m}}\geq\frac{1}{2}$,
and therefore
\[
\left\Vert \beta^{0,\left(l\right)}\check{\bm{h}}^{0}-\check{\bm{h}}^{0,\left(l\right)}\right\Vert _{2}+\left\Vert \beta^{0,\left(l\right)}\check{\bm{x}}^{0}-\check{\bm{x}}^{0,\left(l\right)}\right\Vert _{2}\leq120c_{1}\frac{\mu}{\sqrt{m}}\cdot\sqrt{\frac{K\log^{2}m}{m}}+60c_{1}\sqrt{\frac{\mu^{2}K\log m}{m}}\left|\bm{b}_{l}^{\conj}\check{\bm{h}}^{0}\right|,
\]
which immediately implies that 
\begin{align}
 & \max_{1\leq l\leq m}\left\{ \left\Vert \beta^{0,\left(l\right)}\check{\bm{h}}^{0}-\check{\bm{h}}^{0,\left(l\right)}\right\Vert _{2}+\left\Vert \beta^{0,\left(l\right)}\check{\bm{x}}^{0}-\check{\bm{x}}^{0,\left(l\right)}\right\Vert _{2}\right\} \nonumber \\
 & \qquad\leq120c_{1}\frac{\mu}{\sqrt{m}}\cdot\sqrt{\frac{K\log^{2}m}{m}}+60c_{1}\sqrt{\frac{\mu^{2}K\log m}{m}}\max_{1\leq l\leq m}\left|\bm{b}_{l}^{\conj}\check{\bm{h}}^{0}\right|.\label{eq:LOO-perturbation}
\end{align}

We then move on to $\left|\bm{b}_{l}^{\conj}\check{\bm{h}}^{0}\right|$. The
aim is to show that $\max_{1\leq l\leq m}\left|\bm{b}_{l}^{\conj}\check{\bm{h}}^{0}\right|$
can also be upper bounded by the left-hand side of (\ref{eq:LOO-perturbation}).
By construction, we have $\bm{M}\check{\bm{x}}^{0}=\sigma_{1}\left(\bm{M}\right)\check{\bm{h}}^{0}$,
which further leads to 
\begin{align}
\left|\bm{b}_{l}^{\conj}\check{\bm{h}}^{0}\right| & =\frac{1}{\sigma_{1}\left(\bm{M}\right)}\left|\bm{b}_{l}^{\conj}\bm{M}\check{\bm{x}}^{0}\right|\nonumber \\
 & \overset{\left(\text{i}\right)}{\leq}2\left|\sum_{j=1}^{m}\left(\bm{b}_{l}^{\conj}\bm{b}_{j}\right)\bm{b}_{j}^{\conj}\bm{h}^{\star}\bm{x}^{\star\conj}\bm{a}_{j}\bm{a}_{j}^{\conj}\check{\bm{x}}^{0}\right|\nonumber \\
 & \leq2\left(\sum_{j=1}^{m}\left|\bm{b}_{l}^{\conj}\bm{b}_{j}\right|\right)\max_{1\leq j\leq m}\left\{ \left|\bm{b}_{j}^{\conj}\bm{h}^{\star}\right|\left|\bm{a}_{j}^{\conj}\bm{x}^{\star}\right|\left|\bm{a}_{j}^{\conj}\check{\bm{x}}^{0}\right|\right\} \nonumber \\
 & \overset{\left(\text{ii}\right)}{\leq}8\log m\cdot\frac{\mu}{\sqrt{m}}\cdot\left(5\sqrt{\log m}\right)\max_{1\leq j\leq m}\left\{ \left|\bm{a}_{j}^{\conj}\check{\bm{x}}^{0,\left(j\right)}\right|+\left\Vert \bm{a}_{j}\right\Vert _{2}\left\Vert \beta^{0,\left(j\right)}\check{\bm{x}}^{0}-\check{\bm{x}}^{0,\left(j\right)}\right\Vert _{2}\right\} \nonumber \\
 & \leq200\frac{\mu\log^{2}m}{\sqrt{m}}+120\sqrt{\frac{\mu^{2}K\log^{3}m}{m}}\max_{1\leq j\leq m}\left\Vert \beta^{0,\left(j\right)}\check{\bm{x}}^{0}-\check{\bm{x}}^{0,\left(j\right)}\right\Vert _{2},\label{eq:bh-UB}
\end{align}
where $\beta^{0,\left(j\right)}$ is as defined in (\ref{eq:defn-alpha-0-l}).
Here, (i) comes from the lower bound $\sigma_{1}\left(\bm{M}\right)\geq {1}/{2}$.
The bound (ii) follows by combining the incoherence condition (\ref{eq:incoherence-BD}),
the bound (\ref{eq:max_gaussian-1}), the triangle inequality, as
well as the estimate 
$\sum_{j=1}^{m}\left|\bm{b}_{l}^{\conj}\bm{b}_{j}\right|\leq4\log m$
from Lemma \ref{lemma:fourier-sum-inner}. The last line
uses the upper estimate $\max_{1\leq j\leq m}\left|\bm{a}_{j}^{\conj}\check{\bm{x}}^{0,\left(j\right)}\right|\leq5\sqrt{\log m}$
and (\ref{eq:max_gaussian-2}). Our bound (\ref{eq:bh-UB}) further implies 
\begin{equation}
\max_{1\leq l\leq m}\left|\bm{b}_{l}^{\conj}\check{\bm{h}}^{0}\right|\leq200\frac{\mu\log^{2}m}{\sqrt{m}}+120\sqrt{\frac{\mu^{2}K\log^{3}m}{m}}\max_{1\leq j\leq m}\left\Vert \beta^{0,\left(j\right)}\check{\bm{x}}^{0}-\check{\bm{x}}^{0,\left(j\right)}\right\Vert _{2}.\label{eq:bh-UB-1}
\end{equation}

The above bound (\ref{eq:bh-UB-1}) taken together with \eqref{eq:LOO-perturbation}
gives 
\begin{align}
 & \max_{1\leq l\leq m}\left\{ \left\Vert \beta^{0,\left(l\right)}\check{\bm{h}}^{0}-\check{\bm{h}}^{0,\left(l\right)}\right\Vert _{2}+\left\Vert \beta^{0,\left(l\right)}\check{\bm{x}}^{0}-\check{\bm{x}}^{0,\left(l\right)}\right\Vert _{2}\right\} \leq120c_{1}\frac{\mu}{\sqrt{m}}\cdot\sqrt{\frac{K\log^{2}m}{m}}\nonumber \\
 & \quad\quad+60c_{1}\sqrt{\frac{\mu^{2}K\log m}{m}}\left(200\frac{\mu\log^{2}m}{\sqrt{m}}+120\sqrt{\frac{\mu^{2}K\log^{3}m}{m}}\max_{1\leq j\leq m}\left\Vert \beta^{0,\left(j\right)}\check{\bm{x}}^{0}-\check{\bm{x}}^{0,\left(j\right)}\right\Vert _{2}\right).\label{eq:LOO-perturbation-1}
\end{align}
As long as $m\gg\mu^{2}K\log^{2}m$ we have $60c_{1}\sqrt{{\mu^{2}K\log m}/{m}}\cdot120\sqrt{{\mu^{2}K\log^{3}m}/{m}}\leq1/2$.
Rearranging terms, we are left with 
\begin{equation}
\max_{1\leq l\leq m}\left\{ \left\Vert \beta^{0,\left(l\right)}\check{\bm{h}}^{0}-\check{\bm{h}}^{0,\left(l\right)}\right\Vert _{2}+\left\Vert \beta^{0,\left(l\right)}\check{\bm{x}}^{0}-\check{\bm{x}}^{0,\left(l\right)}\right\Vert _{2}\right\} \leq c_{3}\frac{\mu}{\sqrt{m}}\sqrt{\frac{\mu^{2}K\log^{5}m}{m}}\label{eq:UB-25}
\end{equation}
for some constant $c_{3}>0$. Further, this bound combined with \eqref{eq:bh-UB-1}
yields 
\begin{align}
\max_{1\leq l\leq m}\left|\bm{b}_{l}^{\conj}\check{\bm{h}}^{0}\right| & \leq200\frac{\mu\log^{2}m}{\sqrt{m}}+120\sqrt{\frac{\mu^{2}K\log^{3}m}{m}}\cdot c_{3}\frac{\mu}{\sqrt{m}}\sqrt{\frac{\mu^{2}K\log^{5}m}{m}}\leq c_{2}\frac{\mu\log^{2}m}{\sqrt{m}}\label{eq:spectral-incoherence-b-normalized}
\end{align}
for some constant $c_{2}>0$, with the proviso that $m\gg\mu^{2}K\log^{2}m$.

We now translate the preceding bounds to the scaled version. Recall
from the bound (\ref{eq:sigma1-M-bound}) that 
\begin{equation}
{1}/{2}\leq1-\xi\leq\|\bm{M}\|=\sigma_{1}(\bm{M})\leq1+\xi\leq2,\label{eq:spectral-M-upper}
\end{equation}
as long as $\xi\leq1/2$. For any $\alpha\in\CC$ with $\left|\alpha\right|=1$,
$\alpha\check{\bm{h}}^{0},\alpha\check{\bm{x}}^{0}$ are still the
leading left and right singular vectors of $\bm{M}$. Hence, we can
use Lemma \ref{lemma:singular-value-difference} to derive that
\begin{align*}
 \left|\sigma_{1}\big(\bm{M}\big)-\sigma_{1}\big(\bm{M}^{(l)}\big)\right| &  \leq\left\Vert \big(\bm{M}-\bm{M}^{(l)}\big)\check{\bm{x}}^{0,\left(l\right)}\right\Vert _{2}+\left\{ \left\Vert \alpha\check{\bm{h}}^{0}-\check{\bm{h}}^{0,\left(l\right)}\right\Vert _{2}+\left\Vert \alpha\check{\bm{x}}^{0}-\check{\bm{x}}^{0,\left(l\right)}\right\Vert _{2}\right\} \left\Vert \bm{M}\right\Vert \\
 &  \leq\left\Vert \big(\bm{M}-\bm{M}^{(l)}\big)\check{\bm{x}}^{0,\left(l\right)}\right\Vert _{2}+2\left\{ \left\Vert \alpha\check{\bm{h}}^{0}-\check{\bm{h}}^{0,\left(l\right)}\right\Vert _{2}+\left\Vert \alpha\check{\bm{x}}^{0}-\check{\bm{x}}^{0,\left(l\right)}\right\Vert _{2}\right\} 
\end{align*}
and 
\begin{align*}
 & \left\Vert \alpha\bm{h}^{0}-\bm{h}^{0,(l)}\right\Vert _{2}+\left\Vert \alpha\bm{x}^{0}-\bm{x}^{0,(l)}\right\Vert _{2}\\
 & \quad=\left\Vert \sqrt{\sigma_{1}\left(\bm{M}\right)}\left(\alpha\check{\bm{h}}^{0}\right)-\sqrt{\sigma_{1}\left(\bm{M}^{\left(l\right)}\right)}\check{\bm{h}}^{0,\left(l\right)}\right\Vert _{2}+\left\Vert \sqrt{\sigma_{1}\left(\bm{M}\right)}\alpha\check{\bm{x}}^{0}-\sqrt{\sigma_{1}\left(\bm{M}^{\left(l\right)}\right)}\check{\bm{x}}^{0,\left(l\right)}\right\Vert _{2}\\
 & \quad\leq\sqrt{\sigma_{1}(\bm{M})}\left\{ \left\Vert \alpha\check{\bm{h}}^{0}-\check{\bm{h}}^{0,\left(l\right)}\right\Vert _{2}+\left\Vert \alpha\check{\bm{x}}^{0}-\check{\bm{x}}^{0,\left(l\right)}\right\Vert _{2}\right\} +\frac{2\left|\sigma_{1}(\bm{M})-\sigma_{1}(\bm{M}^{\left(l\right)})\right|}{\sqrt{\sigma_{1}(\bm{M})}+\sqrt{\sigma_{1}(\bm{M}^{(l)})}}\\
 & \quad\leq\sqrt{2}\left\{ \left\Vert \alpha\check{\bm{h}}^{0}-\check{\bm{h}}^{0,\left(l\right)}\right\Vert _{2}+\left\Vert \alpha\check{\bm{x}}^{0}-\check{\bm{x}}^{0,\left(l\right)}\right\Vert _{2}\right\} +\sqrt{2}\left|\sigma_{1}(\bm{M})-\sigma_{1}(\bm{M}^{(l)})\right|.
\end{align*}
Taking the previous two bounds collectively yields 
\[
\left\Vert \alpha\bm{h}^{0}-\bm{h}^{0,(l)}\right\Vert _{2}+\left\Vert \alpha\bm{x}^{0}-\bm{x}^{0,(l)}\right\Vert _{2}\leq\sqrt{2}\left\Vert \big(\bm{M}-\bm{M}^{(l)}\big)\check{\bm{x}}^{0,\left(l\right)}\right\Vert _{2}+6\left\{ \left\Vert \alpha\check{\bm{h}}^{0}-\check{\bm{h}}^{0,\left(l\right)}\right\Vert _{2}+\left\Vert \alpha\check{\bm{x}}^{0}-\check{\bm{x}}^{0,\left(l\right)}\right\Vert _{2}\right\} ,
\]
which together with (\ref{eq:spectral-loop-M-x}) and (\ref{eq:UB-25})
implies
\begin{align}
 & \min_{\alpha\in\mathbb{C},|\alpha|=1}\left\{ \left\Vert \alpha\bm{h}^{0}-\bm{h}^{0,(l)}\right\Vert _{2}+\left\Vert \alpha\bm{x}^{0}-\bm{x}^{0,(l)}\right\Vert _{2}\right\} \leq c_{5}\frac{\mu}{\sqrt{m}}\sqrt{\frac{\mu^{2}K\log^{5}m}{m}}\label{eq:UB-25-scaled-BD}
\end{align}
for some constant $c_{5}>0$, as long as $\xi$ is sufficiently small.
Moreover, we have 
\[
\Big\|\frac{1}{\overline{\alpha^{0}}}\bm{h}^{0}-\frac{\alpha}{\overline{\alpha^{0}}}\bm{h}^{0,\left(l\right)}\Big\|_{2}+\Big\|\alpha^{0}\bm{x}^{0}-\alpha\alpha^{0}\bm{x}^{0,\left(l\right)}\Big\|_{2}\leq2\left\{ \Big\|\bm{h}^{0}-\alpha\bm{h}^{0,\left(l\right)}\Big\|_{2}+\Big\|\bm{x}^{0}-\alpha\bm{x}^{0,\left(l\right)}\Big\|_{2}\right\} 
\]
for any $|\alpha|=1$, where $\alpha^{0}$ is defined in \eqref{eq:defn-alphat}
and, according to Lemma \ref{lemma:spectral-BD-L2}, satisfies 
\begin{equation}
1/2\leq |\alpha^{0}| \leq2.\label{eq:alpha0-bounds}
\end{equation}
Therefore, 
\begin{align*}
 & \min_{\alpha\in\mathbb{C},|\alpha|=1}\sqrt{\Big\|\frac{1}{\overline{\alpha^{0}}}\bm{h}^{0}-\frac{\alpha}{\overline{\alpha^{0}}}\bm{h}^{0,\left(l\right)}\Big\|_{2}^{2}+\Big\|\alpha^{0}\bm{x}^{0}-\alpha\alpha^{0}\bm{x}^{0,\left(l\right)}\Big\|_{2}^{2}}\\
 & \quad\leq\min_{\alpha\in\mathbb{C},|\alpha|=1}\left\{ \left\Vert \frac{1}{\overline{\alpha^{0}}}\bm{h}^{0}-\frac{\alpha}{\overline{\alpha^{0}}}\bm{h}^{0,\left(l\right)}\right\Vert _{2}+\left\Vert \alpha^{0}\bm{x}^{0}-\alpha\alpha^{0}\bm{x}^{0,\left(l\right)}\right\Vert _{2}\right\} \\
 & \quad\leq2\min_{\alpha\in\mathbb{C},|\alpha|=1}\left\{ \Big\|\bm{h}^{0}-\alpha\bm{h}^{0,\left(l\right)}\Big\|_{2}+\Big\|\bm{x}^{0}-\alpha\bm{x}^{0,\left(l\right)}\Big\|_{2}\right\} \\
 & \quad\leq2c_{5}\frac{\mu}{\sqrt{m}}\sqrt{\frac{\mu^{2}K\log^{5}m}{m}}.
\end{align*}
Furthermore, we have
\begin{align}
\text{dist}\big(\bm{z}^{0,\left(l\right)},\tilde{\bm{z}}^{0}\big)= &\min_{\alpha\in\mathbb{C}} \sqrt{ \Big\|\frac{1}{\overline{\alpha}}\bm{h}^{0,(l)}-\frac{1}{\overline{\alpha^{0}}}\bm{h}^{0}\Big\|{}_{2}^{2}+\big\|\alpha\bm{x}^{0,(l)}-\alpha^{0}\bm{x}^{0}\big\|_{2}^{2} } \nonumber \\
\leq & \min_{\alpha\in\mathbb{C},|\alpha|=1}\sqrt{\Big\|\frac{1}{\overline{\alpha^{0}}}\bm{h}^{0}-\frac{\alpha}{\overline{\alpha^{0}}}\bm{h}^{0,\left(l\right)}\Big\|_{2}^{2}+\Big\|\alpha^{0}\bm{x}^{0}-\alpha\alpha^{0}\bm{x}^{0,\left(l\right)}\Big\|_{2}^{2}}\nonumber \\
\leq & 2c_{5}\frac{\mu}{\sqrt{m}}\sqrt{\frac{\mu^{2}K\log^{5}m}{m}},\label{eq:UB-25-1}
\end{align}
where the second line follows since the latter is minimizing over a smaller feasible set. This completes the proof for the claim (\ref{eq:base-loop-BD}).

Regarding $\big|\bm{b}_{l}^{\conj}\tilde{\bm{h}}^{0}\big|$, one first sees that
\[
\left|\bm{b}_{l}^{\conj}\bm{h}^{0}\right|=\left|\sqrt{\sigma_{1}\left(\bm{M}\right)}\bm{b}_{l}^{\conj}\check{\bm{h}}^{0}\right|\leq\sqrt{2}c_{2}\frac{\mu\log^{2}m}{\sqrt{m}},
\]
where the last relation holds due to (\ref{eq:spectral-incoherence-b-normalized})
and (\ref{eq:spectral-M-upper}). Hence, using the property (\ref{eq:alpha0-bounds}),
we have 
\[
\big|\bm{b}_{l}^{\conj}\tilde{\bm{h}}^{0}\big|=\left|\bm{b}_{l}^{\conj}\frac{1}{\overline{\alpha^{0}}}\bm{h}^{0}\right|\leq\left|\frac{1}{\overline{\alpha^{0}}}\right|\left|\bm{b}_{l}^{\conj}\bm{h}^{0}\right|\leq2\sqrt{2}c_{2}\frac{\mu\log^{2}m}{\sqrt{m}},
\]
which finishes the proof of the claim (\ref{eq:base-incoherence-bh-BD}).

Before concluding this section, we note a byproduct of the proof.
Specifically, we can establish the claim required in (\ref{eq:incoherence-h-diff-base-BD})
using many results derived in this section. This is formally stated
in the following lemma. \begin{lemma}\label{lemma:diff-base-BD}Fix any small constant $c>0$. Suppose
the number of samples obeys $m\gg\tau K\log^{4}m$. Then with probability
at least $1-O\left(m^{-10}\right)$, we have 
\[
\max_{1\leq j\leq\tau}\left|\left(\bm{b}_{j}-\bm{b}_{1}\right)^{\conj}\tilde{\bm{h}}^{0}\right|\leq c\frac{\mu}{\sqrt{m}}\log m.
\]\end{lemma}\begin{proof}Instate
the notation and hypotheses in Appendix \ref{subsec:Proof-of-Lemma-perturbation-init-BD}.
Recognize that
\begin{align*}
\left|\left(\bm{b}_{j}-\bm{b}_{1}\right)^{\conj}\tilde{\bm{h}}^{0}\right| & =\left|\left(\bm{b}_{j}-\bm{b}_{1}\right)^{\conj}\frac{1}{\overline{\alpha^{0}}}\bm{h}^{0}\right|=\left|\left(\bm{b}_{j}-\bm{b}_{1}\right)^{\conj}\frac{1}{\overline{\alpha^{0}}}\sqrt{\sigma_{1}\left(\bm{M}\right)}\check{\bm{h}}^{0}\right|\\
 & \leq\left|\frac{1}{\overline{\alpha^{0}}}\right|\sqrt{\sigma_{1}\left(\bm{M}\right)}\left|\left(\bm{b}_{j}-\bm{b}_{1}\right)^{\conj}\check{\bm{h}}^{0}\right|\\
 & \leq4\left|\left(\bm{b}_{j}-\bm{b}_{1}\right)^{\conj}\check{\bm{h}}^{0}\right|,
\end{align*}
where the last inequality comes from (\ref{eq:spectral-M-upper})
and (\ref{eq:alpha0-bounds}). It thus suffices to prove that $\left|\left(\bm{b}_{j}-\bm{b}_{1}\right)^{\conj}\check{\bm{h}}^{0}\right|\leq c{\mu}\log m / {\sqrt{m}}$
for some $c>0$ small enough. To this end, it can be seen that
\begin{align}
\left|\left(\bm{b}_{j}-\bm{b}_{1}\right)^{\conj}\check{\bm{h}}^{0}\right| & =\frac{1}{\sigma_{1}\left(\bm{M}\right)}\left|\left(\bm{b}_{j}-\bm{b}_{1}\right)^{\conj}\bm{M}\check{\bm{x}}^{0}\right|\nonumber \\
 & \leq2\left|\sum_{k=1}^{m}\left(\bm{b}_{j}-\bm{b}_{1}\right)^{\conj}\bm{b}_{k}\bm{b}_{k}^{\conj}\bm{h}^{\star}\bm{x}^{\star\conj}\bm{a}_{k}\bm{a}_{k}^{\conj}\check{\bm{x}}^{0}\right|\nonumber \\
 & \leq2\left(\sum_{k=1}^{m}\left|\left(\bm{b}_{j}-\bm{b}_{1}\right)^{\conj}\bm{b}_{k}\right|\right)\max_{1\leq k\leq m}\left\{ \left|\bm{b}_{k}^{\conj}\bm{h}^{\star}\right|\left|\bm{a}_{k}^{\conj}\bm{x}^{\star}\right|\left|\bm{a}_{k}^{\conj}\check{\bm{x}}^{0}\right|\right\} \nonumber \\
 & \overset{\left(\text{i}\right)}{\leq}c\frac{1}{\log^{2}m}\cdot\frac{\mu}{\sqrt{m}}\cdot\left(5\sqrt{\log m}\right)\max_{1\leq j\leq m}\left\{ \left|\bm{a}_{j}^{\conj}\check{\bm{x}}^{0,\left(j\right)}\right|+\left\Vert \bm{a}_{j}\right\Vert _{2}\left\Vert \alpha^{0,\left(j\right)}\check{\bm{x}}^{0}-\check{\bm{x}}^{0,\left(j\right)}\right\Vert _{2}\right\} \nonumber \\
 & \overset{\left(\text{ii}\right)}{\lesssim}c\frac{\mu}{\sqrt{m}}\frac{1}{\log m}\leq c\frac{\mu}{\sqrt{m}}\log m,\label{eq:incoherence-diff-base-case-BD}
\end{align}
where (i) comes from Lemma \ref{lemma:fourier-sum-diff}, the incoherence
condition (\ref{eq:incoherence-BD}), and the estimate (\ref{eq:max_gaussian-1}). The last line (ii) holds
since we have already established (see (\ref{eq:bh-UB}) and (\ref{eq:UB-25}))
\[
\max_{1\leq j\leq m}\left\{ \left|\bm{a}_{j}^{\conj}\check{\bm{x}}^{0,\left(j\right)}\right|+\left\Vert \bm{a}_{j}\right\Vert _{2}\left\Vert \alpha^{0,\left(j\right)}\check{\bm{x}}^{0}-\check{\bm{x}}^{0,\left(j\right)}\right\Vert _{2}\right\} \lesssim\sqrt{\log m}.
\]
The proof is then complete. \end{proof}

\subsection{Proof of Lemma \ref{lemma:BD-init-al-x0}\label{subsec:Proof-of-Lemma-BD-init-al-x0}}

Recall that $\alpha^{0}$ and $\alpha^{0,\left(l\right)}$ are the
alignment parameters between $\bm{z}^{0}$ and $\bm{z}^{\star}$,
and between $\bm{z}^{0,\left(l\right)}$ and $\bm{z}^{\star}$,
respectively, that is, 
\begin{align*}
\alpha^{0} & :=\argmin_{\alpha\in\mathbb{C}}\left\{ \Big\|\frac{1}{\overline{\alpha}}\bm{h}^{0}-\bm{h}^{\star}\Big\|{}_{2}^{2}+\big\|\alpha\bm{x}^{0}-\bm{x}^{\star}\big\|_{2}^{2}\right\} ,\\
\alpha^{0,\left(l\right)} & :=\argmin_{\alpha\in\mathbb{C}}\left\{ \Big\|\frac{1}{\overline{\alpha}}\bm{h}^{0,(l)}-\bm{h}^{\star}\Big\|{}_{2}^{2}+\big\|\alpha\bm{x}^{0,(l)}-\bm{x}^{\star}\big\|_{2}^{2}\right\} .
\end{align*}
Also, we let 
\[
\alpha_{\text{mutual}}^{0,\left(l\right)}:=\argmin_{\alpha\in\mathbb{C}}\left\{ \Big\|\frac{1}{\overline{\alpha}}\bm{h}^{0,(l)}-\frac{1}{\overline{\alpha^{0}}}\bm{h}^{0}\Big\|{}_{2}^{2}+\big\|\alpha\bm{x}^{0,(l)}-\alpha^{0}\bm{x}^{0}\big\|_{2}^{2}\right\}.
\]
The triangle inequality together with (\ref{eq:spectral-BD-dist})
and (\ref{eq:UB-25-1}) then tells us that 
\begin{align*}
 & \sqrt{\Big\|\frac{1}{\overline{\alpha_{\text{mutual}}^{0,\left(l\right)}}}\bm{h}^{0,(l)}-\bm{h}^{\star}\Big\|_{2}^{2}+\big\|\alpha_{\text{mutual}}^{0,\left(l\right)}\bm{x}^{0,(l)}-\bm{x}^{\star}\big\|_{2}^{2}}\\
 & \quad\leq\sqrt{\Big\|\frac{1}{\overline{\alpha^{0}}}\bm{h}^{0}-\frac{1}{\overline{\alpha_{\text{mutual}}^{0,\left(l\right)}}}\bm{h}^{0,\left(l\right)}\Big\|_{2}^{2}+\Big\|\alpha^{0}\bm{x}^{0}-\alpha_{\text{mutual}}^{0,\left(l\right)}\bm{x}^{0,\left(l\right)}\Big\|_{2}^{2}}+\sqrt{\left\Vert \frac{1}{\overline{\alpha^{0}}}\bm{h}^{0}-\bm{h}^{\star}\right\Vert _{2}^{2}+\left\Vert \alpha^{0}\bm{x}^{0}-\bm{x}^{\star}\right\Vert _{2}^{2}}\\
 & \quad\leq2c_{5}\frac{\mu}{\sqrt{m}}\sqrt{\frac{\mu^{2}K\log^{5}m}{m}}+C_{1}\frac{1}{\log^{2}m}\\
 & \quad\leq2C_{1}\frac{1}{\log^{2}m},
\end{align*}
where the last relation holds as long as $m\gg\mu^{2}\sqrt{K}\log^{9/2}m$.

Let 
\[
\bm{x}_{1}=\alpha^{0}\bm{x}^{0},\quad\bm{h}_{1}=\frac{1}{\overline{\alpha^{0}}}\bm{h}_{0}\qquad\text{and}\qquad\bm{x}_{2}=\alpha_{\text{mutual}}^{0,\left(l\right)}\bm{x}^{0,(l)},\quad\bm{h}_{2}=\frac{1}{\overline{\alpha_{\text{mutual}}^{0,\left(l\right)}}}\bm{h}^{0,(l)}.
\]
It is easy to see that $\bm{x}_{1},\bm{h}_{1},\bm{x}_{2},\bm{h}_{2}$
satisfy the assumptions in Lemma \ref{lemma:stability-BD}, which
implies 
\begin{align}
 \sqrt{\Big\|\frac{1}{\overline{\alpha^{0,\left(l\right)}}}\bm{h}^{0,(l)}-\frac{1}{\overline{\alpha^{0}}}\bm{h}^{0}\Big\|_{2}^{2}+\big\|\alpha^{0,\left(l\right)}\bm{x}^{0,(l)}-\alpha^{0}\bm{x}^{0}\big\|_{2}^{2}} 
 & \lesssim\sqrt{\Big\|\frac{1}{\overline{\alpha^{0}}}\bm{h}^{0}-\frac{1}{\overline{\alpha_{\text{mutual}}^{0,\left(l\right)}}}\bm{h}^{0,\left(l\right)}\Big\|_{2}^{2}+\Big\|\alpha^{0}\bm{x}^{0}-\alpha_{\text{mutual}}^{0,\left(l\right)}\bm{x}^{0,\left(l\right)}\Big\|_{2}^{2}}\nonumber \\
 & \lesssim\frac{\mu}{\sqrt{m}}\sqrt{\frac{\mu^{2}K\log^{5}m}{m}},\label{eq:UB-37}
\end{align}
where the last line comes from (\ref{eq:UB-25-1}). With this upper
estimate at hand, we are now ready to show that with high probability,
\begin{align*}
\left|\bm{a}_{l}^{\conj}\big(\alpha^{0}\bm{x}^{0}-\bm{x}^{\star}\big)\right| & \overset{\left(\text{i}\right)}{\leq}\left|\bm{a}_{l}^{\conj}\big(\alpha^{0,\left(l\right)}\bm{x}^{0,(l)}-\bm{x}^{\star}\big)\right|+\left|\bm{a}_{l}^{\conj}\big(\alpha^{0}\bm{x}^{0}-\alpha^{0,\left(l\right)}\bm{x}^{0,(l)}\big)\right|\\
 & \overset{\left(\text{ii}\right)}{\leq}5\sqrt{\log m}\left\Vert \alpha^{0,\left(l\right)}\bm{x}^{0,(l)}-\bm{x}^{\star}\right\Vert _{2}+\|\bm{a}_{l}\|_{2}\left\Vert \alpha^{0}\bm{x}^{0}-\alpha^{0,\left(l\right)}\bm{x}^{0,(l)}\right\Vert _{2}\\
 & \overset{\left(\text{iii}\right)}{\lesssim}\sqrt{\log m}\cdot\frac{1}{\log^{2}m}+\sqrt{K}\frac{\mu}{\sqrt{m}}\sqrt{\frac{\mu^{2}K\log^{5}m}{m}}\\
 & \overset{\left(\text{iv}\right)}{\lesssim}\frac{1}{\log^{3/2}m},
\end{align*}
where (i) follows from the triangle inequality, (ii) uses Cauchy-Schwarz
 and the independence between $\bm{x}^{0,\left(l\right)}$
and $\bm{a}_{l}$, (iii) holds because of (\ref{eq:spectral-BD-dist-l})
and (\ref{eq:UB-37}) under the condition $m\gg\mu^{2}K\log^{6}m$, and (iv) holds true as long as $m\gg\mu^{2}K\log^{4}m$.

\section{Technical lemmas}

\subsection{Technical lemmas for phase retrieval}
\subsubsection{Matrix concentration inequalities} 

%

\begin{lemma}\label{lemma:identity_concentration-WF}Suppose that
$\bm{a}_{j}\overset{\mathrm{i.i.d.}}{\sim}\mathcal{N}\left(\bm{0},\bm{I}_{n}\right)$
for every $1\leq j\leq m$. Fix any small constant $\delta>0$. With
probability at least $1-C_{2}e^{-c_{2}m}$, one has 
\[
\left\Vert \frac{1}{m}\sum_{j=1}^{m}\bm{a}_{j}\bm{a}_{j}^{\top}-\bm{I}_{n}\right\Vert \leq\delta,
\]
as long as $m\geq c_{0}n$ for some sufficiently large constant $c_{0}>0$.
Here, $C_{2},c_{2}>0$ are some universal constants. \end{lemma}\begin{proof}This
is an immediate consequence of \cite[Corollary 5.35]{Vershynin2012}.\end{proof}

\begin{lemma}\label{lemma:spectral-fix-WF}Suppose that $\bm{a}_{j}\overset{\mathrm{i.i.d.}}{\sim}\mathcal{N}\left(\bm{0},\bm{I}_{n}\right)$,
for every $1\leq j\leq m$. Fix any small constant $\delta>0$. With probability at least $1-O(n^{-10})$,
we have 
\[
\left\Vert \frac{1}{m}\sum_{j=1}^{m}\left(\bm{a}_{j}^{\top}\bm{x}^{\star}\right)^{2}\bm{a}_{j}\bm{a}_{j}^{\top}-\left(\|\bm{x}^{\star}\|_{2}^{2}\bm{I}_{n}+2\bm{x}^{\star}\bm{x}^{\star\top}\right)\right\Vert \leq\delta\|\bm{x}^{\star}\|_{2}^{2},
\]
provided that $m\geq c_{0}n\log n$ for some sufficiently large constant
$c_{0}>0$. 

\end{lemma} \begin{proof}This is adapted from \cite[Lemma 7.4]{candes2014wirtinger}.
\end{proof}

\begin{lemma}\label{lemma:universal-chen-candes}Suppose that $\bm{a}_{j}\overset{\mathrm{i.i.d.}}{\sim}\mathcal{N}\left(\bm{0},\bm{I}_{n}\right)$,
for every $1\leq j\leq m$. Fix any small constant $\delta>0$ and
any constant $C>0$. Suppose $m\geq c_{0}n$ for some sufficiently
large constant $c_{0}>0$. Then with probability at least $1-C_{2}e^{-c_{2}m}$,
\[
\left\Vert \frac{1}{m}\sum_{j=1}^{m}\left(\bm{a}_{j}^{\top}\bm{x}\right)^{2}\ind_{\{|\bm{a}_{j}^{\top}\bm{x}|\leq C\}}\bm{a}_{j}\bm{a}_{j}^{\top}-\left(\beta_{1}\bm{x}\bm{x}^{\top}+\beta_{2}\|\bm{x}\|_{2}^{2}\bm{I}_{n}\right)\right\Vert \leq\delta\|\bm{x}\|_{2}^{2},\quad\forall\bm{x}\in\RR^{n}
\]
holds for some absolute constants $c_{2},C_{2}>0$, where 
\[
\beta_{1}:=\EE\left[\xi^{4}\ind_{\{|\xi|\leq C\}}\right]-\EE\left[\xi^{2}\ind_{|\xi|\leq C}\right]\quad\text{and}\quad\beta_{2}=\EE\left[\xi^{2}\ind_{|\xi|\leq C}\right]
\]
with $\xi$ being a standard Gaussian random variable.

\end{lemma}\begin{proof}This is supplied in \cite[supplementary material]{ChenCandes15solving}.
\end{proof}

\subsubsection{Matrix perturbation bounds}
\begin{lemma}\label{lemma:eigenvalue-difference}Let $\lambda_{1}(\bm{A})$,
$\bm{u}$ be the leading eigenvalue and eigenvector of a symmetric matrix $\bm{A}$,
respectively, and $\lambda_{1}(\tilde{\bm{A}})$, $\tilde{\bm{u}}$
be the leading eigenvalue and eigenvector of a symmetric matrix $\tilde{\bm{A}}$, respectively.
Suppose that $\lambda_{1}(\bm{A}),\lambda_{1}(\tilde{\bm{A}}),\|\bm{A}\|,\|\tilde{\bm{A}}\|\in[C_{1},C_{2}]$
for some $C_{1},C_{2}>0$. Then, 
\begin{align*}
\left\Vert \sqrt{\lambda_{1}(\bm{A})}\;\bm{u}-\sqrt{\lambda_{1}(\tilde{\bm{A}})}\;\tilde{\bm{u}}\right\Vert _{2} & \leq\frac{\big\|\big(\bm{A}-\tilde{\bm{A}}\big)\bm{u}\big\|_{2}}{2\sqrt{C_{1}}}+\left(\sqrt{C_{2}}+\frac{C_{2}}{\sqrt{C_{1}}}\right)\left\Vert \bm{u}-\tilde{\bm{u}}\right\Vert _{2}.
\end{align*}
\end{lemma}\begin{proof}Observe that

\begin{align}
\left\Vert \sqrt{\lambda_{1}(\bm{A})}\;\bm{u}-\sqrt{\lambda_{1}(\tilde{\bm{A}})}\;\tilde{\bm{u}}\right\Vert _{2} & \leq\left\Vert \sqrt{\lambda_{1}(\bm{A})}\;\bm{u}-\sqrt{\lambda_{1}(\tilde{\bm{A}})}\;\bm{u}\right\Vert _{2}+\left\Vert \sqrt{\lambda_{1}(\tilde{\bm{A}})}\;\bm{u}-\sqrt{\lambda_{1}(\tilde{\bm{A}})}\;\tilde{\bm{u}}\right\Vert _{2}\nonumber \\
 & \leq\left|\sqrt{\lambda_{1}\left(\bm{A}\right)}-\sqrt{\lambda_{1}(\tilde{\bm{A}})}\right|+\sqrt{\lambda_{1}(\tilde{\bm{A}})}\left\Vert \bm{u}-\tilde{\bm{u}}\right\Vert _{2},\label{eq:init-loop-PR-2}
\end{align}
where the last inequality follows since $\left\Vert \bm{u}\right\Vert _{2}=1$.
Using the identity $\sqrt{a}-\sqrt{b}=({a-b}) /({\sqrt{a}+\sqrt{b}})$,
we have 
\[
\left|\sqrt{\lambda_{1}\left(\bm{A}\right)}-\sqrt{\lambda_{1}(\tilde{\bm{A}})}\right|=\frac{\left|\lambda_{1}\big(\bm{A}\big)-\lambda_{1}(\tilde{\bm{A}})\right|}{\left|\sqrt{\lambda_{1}\left(\bm{A}\right)}+\sqrt{\lambda_{1}(\tilde{\bm{A}})}\right|}\leq\frac{\left|\lambda_{1}\big(\bm{A}\big)-\lambda_{1}(\tilde{\bm{A}})\right|}{2\sqrt{C_{1}}},
\]
where the last inequality comes from our assumptions on $\lambda_{1}(\bm{A})$
and $\lambda_{1}(\tilde{\bm{A}})$. This combined with (\ref{eq:init-loop-PR-2})
yields 
\begin{equation}
\left\Vert \sqrt{\lambda_{1}(\bm{A})}\;\bm{u}-\sqrt{\lambda_{1}(\tilde{\bm{A}})}\;\tilde{\bm{u}}\right\Vert _{2}\leq\frac{\left|\lambda_{1}\big(\bm{A}\big)-\lambda_{1}(\tilde{\bm{A}})\right|}{2\sqrt{C_{1}}}+\sqrt{C_{2}}\left\Vert \bm{u}-\tilde{\bm{u}}\right\Vert _{2}.\label{eq:init-loop-PR-1-1}
\end{equation}

To control $\left|\lambda_{1}\big(\bm{A}\big)-\lambda_{1}(\tilde{\bm{A}})\right|$,
use the relationship between the eigenvalue and the eigenvector to
obtain 
\begin{align*}
\left|\lambda_{1}(\bm{A})-\lambda_{1}(\tilde{\bm{A}})\right| & =\left|\bm{u}^{\top}\bm{A}\bm{u}-\tilde{\bm{u}}^{\top}\tilde{\bm{A}}\tilde{\bm{u}}\right|\\
 & \leq\left|\bm{u}^{\top}\big(\bm{A}-\tilde{\bm{A}}\big)\bm{u}\right|+\left|\bm{u}^{\top}\tilde{\bm{A}}\bm{u}-\tilde{\bm{u}}^{\top}\tilde{\bm{A}}\bm{u}\right|+\left|\tilde{\bm{u}}^{\top}\tilde{\bm{A}}\bm{u}-\tilde{\bm{u}}^{\top}\tilde{\bm{A}}\tilde{\bm{u}}\right|\\
 & \leq\big\|\big(\bm{A}-\tilde{\bm{A}}\big)\bm{u}\big\|_{2}+2\left\Vert \bm{u}-\tilde{\bm{u}}\right\Vert _{2}\big\|\tilde{\bm{A}}\big\|,
\end{align*}
which together with (\ref{eq:init-loop-PR-1-1}) gives 
\begin{align*}
\left\Vert \sqrt{\lambda_{1}(\bm{A})}\;\bm{u}-\sqrt{\lambda_{1}(\tilde{\bm{A}})}\;\tilde{\bm{u}}\right\Vert _{2} & \leq\frac{\big\|\big(\bm{A}-\tilde{\bm{A}}\big)\bm{u}\big\|_{2}+2\left\Vert \bm{u}-\tilde{\bm{u}}\right\Vert _{2}\big\|\tilde{\bm{A}}\big\|}{2\sqrt{C_{1}}}+\sqrt{C_{2}}\left\Vert \bm{u}-\tilde{\bm{u}}\right\Vert _{2}\\
 & \leq\frac{\big\|\big(\bm{A}-\tilde{\bm{A}}\big)\bm{u}\big\|_{2}}{2\sqrt{C_{1}}}+\left(\frac{C_{2}}{\sqrt{C_{1}}}+\sqrt{C_{2}}\right)\left\Vert \bm{u}-\tilde{\bm{u}}\right\Vert _{2}
\end{align*}
as claimed.
\end{proof}


\subsection{Technical lemmas for matrix completion}

\subsubsection{Orthogonal Procrustes problem}
The orthogonal Procrustes problem is a matrix approximation problem which seeks an orthogonal matrix $\bm{R}$ to best ``align'' two matrices $\bm{A}$ and $\bm{B}$. Specifically, for $\bm{A},\bm{B}\in\RR^{n\times r}$, define $\hat{\bm{R}}$ to be the minimizer of  
\begin{equation}
	\text{minimize}_{\bm{R}\in\cO^{r\times r}} ~~\left\Vert \bm{A}\bm{R}-\bm{B}\right\Vert _{\mathrm{F}}\label{eq:tech-opp}.
\end{equation}

The first lemma is concerned with the characterization of the minimizer $\hat{\bm{R}}$ of (\ref{eq:tech-opp}).
\begin{lemma}
	\label{lemma:opp}For $\bm{A},\bm{B}\in\RR^{n\times r}$, $\hat{\bm{R}}$ is the minimizer of (\ref{eq:tech-opp})
 if and only if $\hat{\bm{R}}^{\top}\bm{A}^{\top}\bm{B}$
is symmetric and positive semidefinite. \end{lemma}\begin{proof}This
is an immediate consequence of \cite[Theorem 2]{MR0501589}. \end{proof}

Let $\bm{A}^\top \bm{B} = \bm{U}\bm{\Sigma}\bm{V}^\top$ be the singular value decomposition of $\bm{A}^\top \bm{B} \in \RR^{r \times r}$. It is easy to check that $\hat{\bm{R}} := \bm{U} \bm{V}^\top$ satisfies the conditions that $\hat{\bm{R}}^{\top}\bm{A}^{\top}\bm{B}$ is both symmetric and positive semidefinite. In view of Lemma~\ref{lemma:opp}, $\hat{\bm{R}}=\bm{U} \bm{V}^\top$ is the minimizer of (\ref{eq:tech-opp}). 
In the special case when $\bm{C}:= \bm{A}^\top \bm{B}$ is invertible, $\hat{\bm{R}}$ enjoys the following equivalent form:
\begin{equation}
	\hat{\bm{R}} = \hat{\bm{H}}\left(\bm{C}\right):= \bm{C}\left(\bm{C}^{\top}\bm{C}\right)^{-1/2}\label{eq:oop-H-defn},
\end{equation}
where $\hat{\bm{H}}\left(\cdot\right)$ is an $\RR^{r \times r}$-valued function on $\RR^{r \times r}$. This motivates us to look at the perturbation bounds for the matrix-valued function $\hat{\bm{H}}\left(\cdot\right)$, which is formulated in the following lemma.
\begin{lemma}\label{lemma:rotation-diff}Let $\bm{C}\in\RR^{r\times r}$
be a nonsingular matrix. Then for any matrix $\bm{E}\in\RR^{r\times r}$
with $\left\Vert \bm{E}\right\Vert \leq\sigma_{\min}\left(\bm{C}\right)$
and any unitarily invariant norm $\vertiii{\cdot} $,
one has 
\[
\vertiii{ \hat{\bm{H}}\left(\bm{C}+\bm{E}\right)-\hat{\bm{H}}\left(\bm{C}\right)} \leq\frac{2}{\sigma_{r-1}\left(\bm{C}\right)+\sigma_{r}\left(\bm{C}\right)}\vertiii{ \bm{E}},
\]
where $\hat{\bm{H}}\left(\cdot\right)$ is defined above. 
\end{lemma}\begin{proof}This
is an immediate consequence of \cite[Theorem 2.3]{mathias1993perturbation}.
\end{proof}

With Lemma~\ref{lemma:rotation-diff} in place, we are ready to present the following bounds on two matrices after ``aligning'' them with $\bm{X}^{\star}$. 

\begin{lemma}\label{lemma:align-diff-MC}
Instate the notation in Section~\ref{sec:main-mc}. Suppose $\bm{X}_{1},\bm{X}_{2}\in\RR^{n\times r}$
are two matrices such that  \begin{subequations} \label{eq:align-diff-MC-condition}
\begin{align}
\left\Vert \bm{X}_{1}-\bm{X}^{\star}\right\Vert \left\Vert \bm{X}^{\star}\right\Vert  & \leq\sigma_{\min}/2,\label{eq:align-diff-assum-1}\\
\left\Vert \bm{X}_{1}-\bm{X}_{2}\right\Vert \left\Vert \bm{X}^{\star}\right\Vert  & \leq\sigma_{\min}/4.\label{eq:align-diff-assump-2}
\end{align}\end{subequations}
Denote 
\begin{align*}
\bm{R}_{1} :=\argmin_{\bm{R}\in\cO^{r\times r}}\left\Vert \bm{X}_{1}\bm{R}-\bm{X}^{\star}\right\Vert _{\mathrm{F}}\qquad\text{and}\qquad
\bm{R}_{2} :=\argmin_{\bm{R}\in\cO^{r\times r}}\left\Vert \bm{X}_{2}\bm{R}-\bm{X}^{\star}\right\Vert _{\mathrm{F}}.
\end{align*}
Then the following two inequalities hold true: 
\begin{align*}
\left\Vert \bm{X}_{1}\bm{R}_{1}-\bm{X}_{2}\bm{R}_{2}\right\Vert  \leq5\kappa\left\Vert \bm{X}_{1}-\bm{X}_{2}\right\Vert\qquad\text{and}\qquad
\left\Vert \bm{X}_{1}\bm{R}_{1}-\bm{X}_{2}\bm{R}_{2}\right\Vert _{\mathrm{F}}  \leq5\kappa\left\Vert \bm{X}_{1}-\bm{X}_{2}\right\Vert _{\mathrm{F}}.
\end{align*}
\end{lemma}\begin{proof}Before proving the claims, we first gather
some immediate consequences of the assumptions~(\ref{eq:align-diff-MC-condition}). Denote $\bm{C}=\bm{X}_{1}^{\top}\bm{X}^{\star}$
and $\bm{E}=\left(\bm{X}_{2}-\bm{X}_{1}\right)^{\top}\bm{X}^{\star}$.
It is easily seen that $\bm{C}$ is invertible since 
\begin{equation}
\left\Vert \bm{C}-\bm{X}^{\star\top}\bm{X}^{\star}\right\Vert \leq\left\Vert \bm{X}_{1}-\bm{X}^{\star}\right\Vert \left\Vert \bm{X}^{\star}\right\Vert \overset{\left(\text{i}\right)}{\leq}\sigma_{\min}/2\qquad\overset{\left(\text{ii}\right)}{\Longrightarrow}\qquad\sigma_{r}\left(\bm{C}\right)\geq\sigma_{\min}/2,\label{eq:align-min-singular}
\end{equation}
where (i) follows from the assumption (\ref{eq:align-diff-assum-1})
and (ii) is a direct application of Weyl's inequality. In addition,
$\bm{C}+\bm{E}=\bm{X}_{2}^{\top}\bm{X}^{\star}$ is also invertible
since 
\[
\left\Vert \bm{E}\right\Vert \leq\left\Vert \bm{X}_{1}-\bm{X}_{2}\right\Vert \left\Vert \bm{X}^{\star}\right\Vert \overset{\left(\text{i}\right)}{\leq}\sigma_{\min}/4\overset{\left(\text{ii}\right)}{<}\sigma_{r}\left(\bm{C}\right),
\]
where (i) arises from the assumption (\ref{eq:align-diff-assump-2})
and (ii) holds because of (\ref{eq:align-min-singular}). When both $\bm{C}$
and $\bm{C}+\bm{E}$ are invertible, the orthonormal matrices $\bm{R}_{1}$
and $\bm{R}_{2}$ admit closed-form expressions as follows 
\[
\bm{R}_{1}=\bm{C}\left(\bm{C}^{\top}\bm{C}\right)^{-1/2}\qquad\text{and}\qquad\bm{R}_{2}=\left(\bm{C}+\bm{E}\right)\left[\left(\bm{C}+\bm{E}\right)^{\top}\left(\bm{C}+\bm{E}\right)\right]^{-1/2}.
\]
Moreover, we have the following bound on $\left\Vert \bm{X}_{1}\right\Vert$: 
\begin{equation}
\left\Vert \bm{X}_{1}\right\Vert \overset{\text{(i)}}{\leq}\left\Vert \bm{X}_{1}-\bm{X}^{\star}\right\Vert +\left\Vert \bm{X}^{\star}\right\Vert \overset{\text{(ii)}}{\leq}\frac{\sigma_{\min}}{2\left\Vert \bm{X}^{\star}\right\Vert }+ \left\Vert \bm{X}^{\star}\right\Vert\leq\frac{\sigma_{\max}}{2\left\Vert \bm{X}^{\star}\right\Vert }+ \left\Vert \bm{X}^{\star}\right\Vert\overset{\text{(iii)}}{\leq}2\left\Vert \bm{X}^{\star}\right\Vert,\label{eq:X1-norm-UB}
\end{equation}
where (i) is the triangle inequality, (ii) uses the assumption~\eqref{eq:align-diff-assum-1} and (iii) arises from the fact that $\left\Vert \bm{X}^{\star}\right\Vert = \sqrt{\sigma_{\max}}$.

With these in place, we turn to establishing the claimed bounds. We
will focus on the upper bound on $\left\Vert \bm{X}_{1}\bm{R}_{1}-\bm{X}_{2}\bm{R}_{2}\right\Vert _{\mathrm{F}}$,
as the bound on $\left\Vert \bm{X}_{1}\bm{R}_{1}-\bm{X}_{2}\bm{R}_{2}\right\Vert $
can be easily obtained using the same argument. Simple algebra reveals that 
\begin{align}
\left\Vert \bm{X}_{1}\bm{R}_{1}-\bm{X}_{2}\bm{R}_{2}\right\Vert _{\mathrm{F}} & =\left\Vert \left(\bm{X}_{1}-\bm{X}_{2}\right)\bm{R}_{2}+\bm{X}_{1}\left(\bm{R}_{1}-\bm{R}_{2}\right)\right\Vert _{\mathrm{F}}\nonumber \\
 & \leq\left\Vert \bm{X}_{1}-\bm{X}_{2}\right\Vert _{\mathrm{F}}+\left\Vert \bm{X}_{1}\right\Vert \left\Vert \bm{R}_{1}-\bm{R}_{2}\right\Vert _{\mathrm{F}}\nonumber \\
 & \leq\left\Vert \bm{X}_{1}-\bm{X}_{2}\right\Vert _{\mathrm{F}}+2\left\Vert \bm{X}^{\star}\right\Vert \left\Vert \bm{R}_{1}-\bm{R}_{2}\right\Vert _{\mathrm{F}},\label{eq:align-triangle}
\end{align}
where the first inequality uses the fact that $\left\|\bm{R}_{2}\right\| = 1$ and the last inequality comes from (\ref{eq:X1-norm-UB}). An application
of Lemma \ref{lemma:rotation-diff} leads us to conclude that 
\begin{align}
\left\Vert \bm{R}_{1}-\bm{R}_{2}\right\Vert _{\mathrm{F}} & \leq\frac{2}{\sigma_{r}\left(\bm{C}\right)+\sigma_{r-1}\left(\bm{C}\right)}\left\Vert \bm{E}\right\Vert _{\mathrm{F}}\nonumber \\
 & \leq\frac{2}{\sigma_{\min}}\left\Vert \left(\bm{X}_{2}-\bm{X}_{1}\right)^{\top}\bm{X}^{\star}\right\Vert _{\mathrm{F}}\label{eq:align-min-singular-inequality}\\
 & \leq\frac{2}{\sigma_{\min}}\left\Vert \bm{X}_{2}-\bm{X}_{1}\right\Vert _{\mathrm{F}}\left\Vert \bm{X}^{\star}\right\Vert ,\label{eq:align-rotation-haha}
\end{align}
where (\ref{eq:align-min-singular-inequality}) utilizes (\ref{eq:align-min-singular}).
Combine (\ref{eq:align-triangle}) and (\ref{eq:align-rotation-haha})
to reach 
\begin{align*}
\left\Vert \bm{X}_{1}\bm{R}_{1}-\bm{X}_{2}\bm{R}_{2}\right\Vert _{\mathrm{F}} & \leq\left\Vert \bm{X}_{1}-\bm{X}_{2}\right\Vert _{\mathrm{F}}+\frac{4}{\sigma_{\min}}\left\Vert \bm{X}_{2}-\bm{X}_{1}\right\Vert _{\mathrm{F}}\left\Vert \bm{X}^{\star}\right\Vert ^{2}\\
 & \leq\left(1+4\kappa\right)\left\Vert \bm{X}_{1}-\bm{X}_{2}\right\Vert _{\mathrm{F}},
\end{align*}
which finishes the proof by noting that $\kappa\geq1$.
\end{proof}

\subsubsection{Matrix concentration inequalities}
This section collects various measure concentration results regarding the Bernoulli random variables $\{\delta_{j,k}\}_{1\leq j,k \leq n}$, which is ubiquitous in the analysis for matrix completion.
\begin{lemma}
\label{lemma:injectivity}
Fix any small constant $\delta>0$,
	and suppose that $m\gg \delta^{-2}\mu nr\log n$. Then with probability exceeding
$1-O\left(n^{-10}\right)$, one has
\[
(1-\delta)\|\bm{B}\|_{\mathrm{F}}\leq\frac{1}{\sqrt{p}}\|\mathcal{P}_{\Omega}(\bm{B})\|_{\mathrm{F}}\leq(1+\delta)\|\bm{B}\|_{\mathrm{F}}
\]
holds simultaneously for all $\bm{B}\in\mathbb{R}^{n\times n}$ lying
within the tangent space of $\bm{M}^{\star}$. \end{lemma}
\begin{proof}
This result has been established in \cite[Section 4.2]{ExactMC09} for
asymmetric sampling patterns (where each $(i,j)$, $i\neq j$ is
included in $\Omega$ independently). It is straightforward to extend
the proof and the result to symmetric sampling patterns (where each $(i,j)$,
$i\geq j$ is included in $\Omega$ independently). We omit the proof
for conciseness. 
\end{proof}

\begin{lemma}\label{lemma:montanari}Fix a matrix $\bm{M}\in\RR^{n\times n}$. Suppose $n^{2}p\geq c_{0}n\log n$
for some sufficiently large constant $c_{0}>0$. With probability
at least $1-O\left(n^{-10}\right)$, one has 
\[
\left\Vert \frac{1}{p}\mathcal{P}_{\Omega}\left(\bm{M}\right)-\bm{M}\right\Vert \leq C\sqrt{\frac{n}{p}}\left\Vert \bm{M}\right\Vert _{\infty},
\]
where $C>0$ is some absolute constant. \end{lemma}
\begin{proof}
See \cite[Lemma 3.2]{KesMonSew2010}. Similar to Lemma \ref{lemma:injectivity}, the result
therein was provided for the asymmetric sampling patterns but can
be easily extended to the symmetric case.
\end{proof}

\begin{lemma}\label{lemma:noise-spectral}Recall from Section~\ref{sec:main-mc} that $\bm{E}\in\RR^{n \times n}$ is the symmetric noise matrix. Suppose the sample size
obeys $n^{2}p\geq c_{0}n\log^{2}n$ for some sufficiently large constant
$c_{0}>0$. With probability at least $1-O\left(n^{-10}\right)$,
one has 
\[
\left\Vert \frac{1}{p}\cP_{\Omega}\left(\bm{E}\right)\right\Vert \leq C\sigma\sqrt{\frac{n}{p}},
\]
where $C>0$ is some universal constant. \end{lemma}\begin{proof}See
\cite[Lemma 11]{chen2015fast}.\end{proof}

\begin{lemma}\label{lemma:bernoulli-spectral-norm}Fix some matrix $\bm{A}\in\RR^{n\times r}$ with $n\geq 2r$ and some $1\leq l \leq n$. Suppose $\left\{ \delta_{l,j}\right\} _{1\leq j\leq n}$
are independent Bernoulli random variables with means $\left\{p_{j}\right\}_{1\leq j \leq n}$ no more than $p$. Define
\[
\bm{G}_{l}\left(\bm{A}\right):=\left[\delta_{l,1}\bm{A}_{1,\cdot}^{\top},\delta_{l,2}\bm{A}_{2,\cdot}^{\top},\cdots,\delta_{l,n}\bm{A}_{n,\cdot}^{\top}\right]\in\RR^{r\times n}.
\]
Then one has 
\[
\mathsf{Median}\left[\left\Vert \bm{G}_{l}\left(\bm{A}\right)\right\Vert \right]\leq\sqrt{p\left\Vert \bm{A}\right\Vert ^{2}+\sqrt{2p\left\Vert \bm{A}\right\Vert _{2,\infty}^{2}\left\Vert \bm{A}\right\Vert ^{2}\log\left(4r\right)}+\frac{2\left\Vert \bm{A}\right\Vert _{2,\infty}^{2}}{3}\log\left(4r\right)}
\]
and for any constant $C \geq 3$, with probability exceeding $1-n^{-\left(1.5C-1\right)}$
\[
\left\Vert \sum_{j=1}^{n}(\delta_{l,j}-p)\bm{A}_{j,\cdot}^{\top}\bm{A}_{j,\cdot}\right\Vert \leq C\left(\sqrt{p\left\Vert \bm{A}\right\Vert _{2,\infty}^{2}\left\Vert \bm{A}\right\Vert ^{2}\log n}+\left\Vert \bm{A}\right\Vert _{2,\infty}^{2}\log n\right),
\]
and
\[
\left\Vert \bm{G}_{l}\left(\bm{A}\right)\right\Vert \leq\sqrt{p\left\Vert \bm{A}\right\Vert ^{2}+C\left(\sqrt{p\left\Vert \bm{A}\right\Vert _{2,\infty}^{2}\left\Vert \bm{A}\right\Vert ^{2}\log n}+\left\Vert \bm{A}\right\Vert _{2,\infty}^{2}\log n\right)}.
\]
\end{lemma}\begin{proof}By the definition of $\bm{G}_{l}\left(\bm{A}\right)$ and the triangle inequality, one has  
\begin{align*}
\left\Vert \bm{G}_{l}\left(\bm{A}\right)\right\Vert ^{2} & =\left\Vert \bm{G}_{l}\left(\bm{A}\right)\bm{G}_{l}\left(\bm{A}\right)^{\top}\right\Vert =\left\Vert \sum_{j=1}^{n}\delta_{l,j}\bm{A}_{j,\cdot}^{\top}\bm{A}_{j,\cdot}\right\Vert \leq\left\Vert \sum_{j=1}^{n}\left(\delta_{l,j}-p_{j}\right)\bm{A}_{j,\cdot}^{\top}\bm{A}_{j,\cdot}\right\Vert +p\left\Vert \bm{A}\right\Vert ^{2}.
\end{align*}
Therefore, it suffices to control the first term. It can be seen that $\left\{\left(\delta_{l,j}-p_{j}\right)\bm{A}_{j,\cdot}^{\top}\bm{A}_{j,\cdot}\right\}_{1\leq j \leq n}$ are i.i.d. zero-mean random matrices. Letting 
\begin{align*}
L & :=\max_{1\leq j\leq n}\left\Vert \left(\delta_{l,j}-p_{j}\right)\bm{A}_{j,\cdot}^{\top}\bm{A}_{j,\cdot}\right\Vert \leq\left\Vert \bm{A}\right\Vert _{2,\infty}^{2}\\
\text{and}\quad V & :=\left\Vert \sum_{j=1}^{n}\EE\left[\left(\delta_{l,j}-p_{j}\right)^{2}\bm{A}_{j,\cdot}^{\top}\bm{A}_{j,\cdot}\bm{A}_{j,\cdot}^{\top}\bm{A}_{j,\cdot}\right]\right\Vert \leq\mathbb{E}\left[\left(\delta_{l,j}-p_{j}\right)^{2}\right]\left\Vert \bm{A}\right\Vert _{2,\infty}^{2}\left\Vert \sum_{j=1}^{n}\bm{A}_{j,\cdot}^{\top}\bm{A}_{j,\cdot}\right\Vert  \leq p\left\Vert \bm{A}\right\Vert _{2,\infty}^{2}\left\Vert \bm{A}\right\Vert ^{2}
\end{align*}
and invoking matrix Bernstein's inequality \cite[Theorem 6.1.1]{Tropp:2015:IMC:2802188.2802189},
one has for all $t\geq0$, 
\begin{equation}
\PP\left\{ \left\Vert \sum_{j=1}^{n}\left(\delta_{l,j}-p_{j}\right)\bm{A}_{j,\cdot}^{\top}\bm{A}_{j,\cdot}\right\Vert \geq t\right\} \leq2r\cdot\exp\left(\frac{-t^{2}/2}{p\left\Vert \bm{A}\right\Vert _{2,\infty}^{2}\left\Vert \bm{A}\right\Vert ^{2}+\left\Vert \bm{A}\right\Vert _{2,\infty}^{2}\cdot t/3}\right).\label{eq:UB-MC-10}
\end{equation}
We can thus find an upper bound on $\mathsf{Median}\left[\left\Vert \sum_{j=1}^{n}\left(\delta_{l,j}-p_{j}\right)\bm{A}_{j,\cdot}^{\top}\bm{A}_{j,\cdot}\right\Vert \right]$
by finding a value $t$ that ensures the right-hand side of (\ref{eq:UB-MC-10})
is smaller than $1/2$. Using this strategy and some simple calculations,
we get
\[
\mathsf{Median}\left[\left\Vert \sum_{j=1}^{n}\left(\delta_{l,j}-p_{j}\right)\bm{A}_{j,\cdot}^{\top}\bm{A}_{j,\cdot}\right\Vert \right]\leq\sqrt{2p\left\Vert \bm{A}\right\Vert _{2,\infty}^{2}\left\Vert \bm{A}\right\Vert ^{2}\log\left(4r\right)}+\frac{2\left\Vert \bm{A}\right\Vert _{2,\infty}^{2}}{3}\log\left(4r\right)
\]
and for any $C\geq3$, 
\[
\left\Vert \sum_{j=1}^{n}\left(\delta_{l,j}-p_{j}\right)\bm{A}_{j,\cdot}^{\top}\bm{A}_{j,\cdot}\right\Vert \leq C\left(\sqrt{p\left\Vert \bm{A}\right\Vert _{2,\infty}^{2}\left\Vert \bm{A}\right\Vert ^{2}\log n}+\left\Vert \bm{A}\right\Vert _{2,\infty}^{2}\log n\right)
\]
holds with probability at least $1-n^{-\left(1.5C-1\right)}$. As a 
consequence, we have 
\[
\mathsf{Median}\left[\left\Vert \bm{G}_{l}\left(\bm{A}\right)\right\Vert \right]\leq\sqrt{p\left\Vert \bm{A}\right\Vert ^{2}+\sqrt{2p\left\Vert \bm{A}\right\Vert _{2,\infty}^{2}\left\Vert \bm{A}\right\Vert ^{2}\log\left(4r\right)}+\frac{2\left\Vert \bm{A}\right\Vert _{2,\infty}^{2}}{3}\log\left(4r\right)},
\]
and with probability exceeding $1-n^{-\left(1.5C-1\right)}$, 
\[
\left\Vert \bm{G}_{l}\left(\bm{A}\right)\right\Vert ^{2}\leq p\left\Vert \bm{A}\right\Vert ^{2}+C\left(\sqrt{p\left\Vert \bm{A}\right\Vert _{2,\infty}^{2}\left\Vert \bm{A}\right\Vert ^{2}\log n}+\left\Vert \bm{A}\right\Vert _{2,\infty}^{2}\log n\right).
\]
This completes the proof. \end{proof}
\begin{lemma}\label{lemma:op-uniform-spectral-MC}Let $\left\{ \delta_{l,j}\right\} _{1\leq l\leq j\leq n}$
be i.i.d.~Bernoulli random variables with mean $p$ and $\delta_{l,j}=\delta_{j,l}$.
For any $\bm{\Delta}\in\RR^{n\times r}$, define 
\[
\bm{G}_{l}\left(\bm{\Delta}\right):=\left[\delta_{l,1}\bm{\Delta}_{1,\cdot}^{\top},\delta_{l,2}\bm{\Delta}_{2,\cdot}^{\top},\cdots,\delta_{l,n}\bm{\Delta}_{n,\cdot}^{\top}\right]\in\RR^{r\times n}.
\]
Suppose the sample size obeys $n^2 p\gg \kappa \mu r n \log^{2} n$. Then for any $k>0$ and $\alpha>0$ large enough, with probability at least $1-c_{1}e^{-{\alpha C}nr\log n / 2}$,
\[
\sum_{l=1}^{n}\ind_{\left\{ \left\Vert \bm{G}_{l}\left(\bm{\Delta}\right)\right\Vert \geq4\sqrt{p}\psi+2\sqrt{kr}\xi\right\} }\leq\frac{2\alpha n\log n}{k}
\]
holds simultaneously for all $\bm{\Delta}\in\RR^{n\times r}$ obeying
\begin{align*}
\left\Vert \bm{\Delta}\right\Vert _{2,\infty} & \leq C_{5}\rho^{t}\mu r\sqrt{\frac{\log n}{np}}\left\Vert \bm{X}^{\star}\right\Vert _{2,\infty}+C_{8}\sigma\sqrt{\frac{n\log n}{p}}\left\Vert \bm{X}^{\star}\right\Vert _{2,\infty}:=\xi\\
\text{and}\qquad\left\Vert \bm{\Delta}\right\Vert  & \leq C_{9}\rho^{t}\mu r\frac{1}{\sqrt{np}}\left\Vert \bm{X}^{\star}\right\Vert +C_{10}\sigma\sqrt{\frac{n}{p}}\left\Vert \bm{X}^{\star}\right\Vert :=\psi,
\end{align*}
where $c_{1},C_{5},C_{8}, C_{9}, C_{10}>0$ are some absolute constants. 
\end{lemma}
\begin{proof}
For simplicity of presentation, we will prove the claim for the asymmetric
case where $\left\{ \delta_{l,j}\right\} _{1\leq l,j\leq n}$ are
independent. The results immediately carry over to the symmetric case
as claimed in this lemma. To see this, note that we can always divide
$\bm{G}_{l}(\bm{\Delta})$ into
\[
\bm{G}_{l}(\bm{\Delta})=\bm{G}_{l}^{\mathrm{upper}}(\bm{\Delta})+\bm{G}_{l}^{\mathrm{lower}}(\bm{\Delta}),
\]
where all nonzero components of $\bm{G}_{l}^{\mathrm{upper}}(\bm{\Delta})$
come from the upper triangular part (those blocks with $l\leq j$
), while all nonzero components of $\bm{G}_{l}^{\mathrm{lower}}(\bm{\Delta})$
are from the lower triangular part (those blocks with $l>j$). We
can then look at $\left\{ \bm{G}_{l}^{\mathrm{upper}}(\bm{\Delta})\mid1\leq l\leq n\right\} $
and $\left\{ \bm{G}_{l}^{\mathrm{upper}}(\bm{\Delta})\mid1\leq l\leq n\right\} $
separately using the argument we develop for the asymmetric case.
From now on, we  assume that $\left\{ \delta_{l,j}\right\} _{1\leq l,j\leq n}$
are independent. 

Suppose for the moment that $\bm{\Delta}$ is statistically independent
of $\left\{ \delta_{l,j}\right\}$. Clearly, for
any $\bm{\Delta},\tilde{\bm{\Delta}}\in\RR^{n\times r}$, 
\begin{align*}
\left|\big\|\bm{G}_{l}\left(\bm{\Delta}\right)\big\|-\big\|\bm{G}_{l}(\tilde{\bm{\Delta}})\big\|\right| & \leq\left\Vert \bm{G}_{l}\left(\bm{\Delta}\right)-\bm{G}_{l}\big(\tilde{\bm{\Delta}}\big)\right\Vert \leq\left\Vert \bm{G}_{l}\left(\bm{\Delta}\right)-\bm{G}_{l}\big(\tilde{\bm{\Delta}}\big)\right\Vert _{\mathrm{F}}\\
 & \leq\sqrt{\sum_{j=1}^{n}\left\Vert \bm{\Delta}_{j,\cdot}-\tilde{\bm{\Delta}}_{j,\cdot}\right\Vert _{2}^{2}}\\
 & :=d\big(\bm{\Delta},\tilde{\bm{\Delta}}\big),
\end{align*}
which implies that $\left\Vert \bm{G}_{l}\left(\bm{\Delta}\right)\right\Vert $
is $1$-Lipschitz with respect to the metric $d\left(\cdot,\cdot\right)$.
Moreover, 
\[
\max_{1\leq j\leq n}\left\Vert \delta_{l,j}\bm{\Delta}_{j,\cdot}\right\Vert _{2}\leq\left\Vert \bm{\Delta}\right\Vert _{2,\infty}\leq\xi
\]
according to our assumption. Hence, Talagrand's inequality \cite[Proposition 1]{chen2016projected}
reveals the existence of some absolute constants $C,c>0$ such that for all $\lambda >0$
\begin{equation}
\PP\left\{ \left\Vert \bm{G}_{l}\left(\bm{\Delta}\right)\right\Vert -\mathsf{Median}\left[\left\Vert \bm{G}_{l}\left(\bm{\Delta}\right)\right\Vert \right]\geq\lambda\xi\right\} \leq C\exp\left(-c\lambda^{2}\right).\label{eq:talagrand}
\end{equation}

We then proceed to control $\mathrm{Median}\left[\left\Vert \bm{G}_{l}\left(\bm{\Delta}\right)\right\Vert \right]$.
A direct application of Lemma \ref{lemma:bernoulli-spectral-norm}
yields 
\begin{align*}
\mathsf{Median}\left[\left\Vert \bm{G}_{l}\left(\bm{\Delta}\right)\right\Vert \right] \leq\sqrt{2p\psi^{2}+\sqrt{p\log\left(4r\right)}\xi\psi+\frac{2\xi^{2}}{3}\log\left(4r\right)}\leq2\sqrt{p}\psi,
\end{align*}
where the last relation holds since $p\psi^{2}\gg\xi^{2}\log r$, which
follows by combining the definitions of $\psi$ and $\xi$, the sample
size condition $np\gg\kappa\mu r\log^{2}n$, and the incoherence condition
(\ref{eq:incoherence-X-MC}). Thus, substitution into (\ref{eq:talagrand})
and taking $\lambda=\sqrt{kr}$ give 
\begin{equation}
\PP\left\{ \left\Vert \bm{G}_{l}\left(\bm{\Delta}\right)\right\Vert \geq2\sqrt{p}\psi+\sqrt{kr}\xi\right\} \leq C\exp\left(-ckr\right)\label{eq:GlDelta-UB}
\end{equation}
for any $k\geq0$. Furthermore, invoking \cite[Corollary A.1.14]{Alon2008}
and using the bound (\ref{eq:GlDelta-UB}), one has 
\begin{align*}
 & \PP\left(\sum_{l=1}^{n}\ind_{\left\{ \left\Vert \bm{G}_{l}\left(\bm{\Delta}\right)\right\Vert \geq2\sqrt{p}\psi+\sqrt{kr}\xi\right\} }\geq tnC\exp\left(-ckr\right)\right)\leq2\exp\left(-\frac{t\log t}{2}nC\exp\left(-ckr\right)\right)
\end{align*}
for any $t\geq6$. Choose $t={\alpha\log n} / \left[{kC\exp\left(-ckr\right)}\right]\geq6$
to obtain
\begin{align}
 & \PP\left(\sum_{l=1}^{n}\ind_{\left\{ \left\Vert \bm{G}_{l}\left(\bm{\Delta}\right)\right\Vert \geq2\sqrt{p}\psi+\sqrt{kr}\xi\right\} }\geq\frac{\alpha n\log n}{k}\right)\leq2\exp\left(-\frac{\alpha C}{2}nr\log n\right).\label{eq:tail-bound-Gl}
\end{align}

So far we have demonstrated that for any fixed $\bm{\Delta}$ obeying
our assumptions, $\sum_{l=1}^{n}\ind_{\left\{ \left\Vert \bm{G}_{l}\left(\bm{\Delta}\right)\right\Vert \geq2\sqrt{p}\psi+\sqrt{kr}\xi\right\} }$
is well controlled with exponentially high probability. In order to
extend the results to all feasible $\bm{\Delta}$, we resort to the
standard $\epsilon$-net argument. Clearly, due to the homogeneity
property of $\left\Vert \bm{G}_{l}\left(\bm{\Delta}\right)\right\Vert $,
it suffices to restrict attention to the following set:

\begin{equation}
\mathcal{S}=\left\{ \bm{\Delta}\mid\min\left\{ \xi,\psi\right\} \leq\|\bm{\Delta}\|\leq\psi\right\} ,\label{eq:defn-set-S}
\end{equation}
where $\psi/\xi \lesssim \|\bm{X}^{\star}\| / \|\bm{X}^{\star}\|_{2,\infty} \lesssim \sqrt{n}$.
We then proceed with the following steps. 
\begin{enumerate}
\item Introduce the auxiliary function 
\[
\chi_{l}(\bm{\Delta})=\begin{cases}
1,\qquad & \text{if }\left\Vert \bm{G}_{l}\left(\bm{\Delta}\right)\right\Vert \geq4\sqrt{p}\psi+2\sqrt{kr}\xi,\\
\frac{\left\Vert \bm{G}_{l}\left(\bm{\Delta}\right)\right\Vert -2\sqrt{p}\psi-\sqrt{kr}\xi}{2\sqrt{p}\psi + \sqrt{kr}\xi}, & \text{if }\left\Vert \bm{G}_{l}\left(\bm{\Delta}\right)\right\Vert \in[2\sqrt{p}\psi+\sqrt{kr}\xi,\text{ }4\sqrt{p}\psi+2\sqrt{kr}\xi],\\
0, & \text{else}.
\end{cases}
\]
Clearly, this function is sandwiched between two indicator functions
\[
\ind_{\left\{ \left\Vert \bm{G}_{l}\left(\bm{\Delta}\right)\right\Vert \geq4\sqrt{p}\psi+2\sqrt{kr}\xi\right\} }\leq\chi_{l}(\bm{\Delta})\leq\ind_{\left\{ \left\Vert \bm{G}_{l}\left(\bm{\Delta}\right)\right\Vert \geq2\sqrt{p}\psi+\sqrt{kr}\xi\right\} }.
\]
Note that $\chi_{l}$ is more convenient to work with due to continuity. 
\item Consider an $\epsilon$-net $\mathcal{N}_{\epsilon}$ \cite[Section 2.3.1]{Tao2012RMT} of the set $\mathcal{S}$
as defined in (\ref{eq:defn-set-S}).
For any $\epsilon=1/n^{O(1)}$, one can find such a net with cardinality
$\log|\mathcal{N}_{\epsilon}|\lesssim nr\log n$. Apply the union
bound and (\ref{eq:tail-bound-Gl}) to yield 
\begin{align*}
\PP\left(\sum_{l=1}^{n}\chi_{l}(\bm{\Delta})\geq\frac{\alpha n\log n}{k},\text{ }\forall\bm{\Delta}\in\mathcal{N}_{\epsilon}\right) & \leq\PP\left(\sum_{l=1}^{n}\ind_{\left\{ \left\Vert \bm{G}_{l}\left(\bm{\Delta}\right)\right\Vert \geq2\sqrt{p}\psi+\sqrt{kr}\xi\right\} }\geq\frac{\alpha n\log n}{k},\text{ }\forall\bm{\Delta}\in\mathcal{N}_{\epsilon}\right)\\
 & \leq2|\mathcal{N}_{\epsilon}|\exp\left(-\frac{\alpha C}{2}nr\log n\right)\leq2\exp\left(-\frac{\alpha C}{4}nr\log n\right),
\end{align*}
as long as $\alpha$ is chosen to be sufficiently large. 
\item One can then use the continuity argument to extend the bound to all
$\bm{\Delta}$ outside the $\epsilon$-net, i.e.~with exponentially
high probability, 
\[
\sum_{l=1}^{n}\chi_{l}(\bm{\Delta})\leq\frac{2\alpha n\log n}{k},\quad \forall\bm{\Delta}\in\mathcal{S}
\]
\[
\Longrightarrow\qquad\sum_{l=1}^{n}\ind_{\left\{ \left\Vert \bm{G}_{l}\left(\bm{\Delta}\right)\right\Vert \geq4\sqrt{p}\psi+2\sqrt{kr}\xi\right\} }\leq\sum_{l=1}^{n}\chi_{l}(\bm{\Delta})\leq\frac{2\alpha n\log n}{k},\quad \forall\bm{\Delta}\in\mathcal{S}
\]
This is fairly standard (see, e.g.~\cite[Section 2.3.1]{Tao2012RMT})
and is thus omitted here. 
\end{enumerate}
We have thus concluded the proof. \end{proof}
\begin{lemma}\label{lemma:energy-MC}Suppose the sample size obeys
$n^{2}p\geq C\kappa\mu rn\log n$ for some sufficiently large constant
$C>0$. Then with probability at least $1-O\left(n^{-10}\right)$,
\begin{align*}
\left\Vert \frac{1}{p}\cP_{\Omega}\left(\bm{X}\bm{X}^{\top}-\bm{X}^{\star}\bm{X}^{\star\top}\right)\right\Vert  & \leq2n\epsilon^{2}\left\Vert \bm{X}^{\star}\right\Vert _{2,\infty}^{2}+4\epsilon\sqrt{n}\log n\left\Vert \bm{X}^{\star}\right\Vert _{2,\infty}\left\Vert \bm{X}^{\star}\right\Vert 
\end{align*}
holds simultaneously for all $\bm{X}\in\RR^{n\times r}$ satisfying
\begin{equation}
\left\Vert \bm{X}-\bm{X}^{\star}\right\Vert _{2,\infty}\leq\epsilon\left\Vert \bm{X}^{\star}\right\Vert _{2,\infty},\label{eq:Delta-2-infty}
\end{equation}
where $\epsilon>0$ is any fixed constant. \end{lemma}

\begin{proof}

To simplify the notations hereafter, we denote $\bm{\Delta}:=\bm{X}-\bm{X}^{\star}$. With this notation in place, one can decompose 
\[
\bm{X}\bm{X}^{\top}-\bm{X}^{\star}\bm{X}^{\star\top}=\bm{\Delta}\bm{X}^{\star\top}+\bm{X}^{\star}\bm{\Delta}^{\top}+\bm{\Delta}\bm{\Delta}^{\top},
\]
which together with the triangle inequality implies that 
\begin{align}
\left\Vert \frac{1}{p}\cP_{\Omega}\left(\bm{X}\bm{X}^{\top}-\bm{X}^{\star}\bm{X}^{\star\top}\right)\right\Vert  & \leq\left\Vert \frac{1}{p}\cP_{\Omega}\left(\bm{\Delta}\bm{X}^{\star\top}\right)\right\Vert +\left\Vert \frac{1}{p}\cP_{\Omega}\left(\bm{X}^{\star}\bm{\Delta}^{\top}\right)\right\Vert +\left\Vert \frac{1}{p}\cP_{\Omega}\left(\bm{\Delta}\bm{\Delta}^{\top}\right)\right\Vert \nonumber \\
 & =\underbrace{\left\Vert \frac{1}{p}\cP_{\Omega}\left(\bm{\Delta}\bm{\Delta}^{\top}\right)\right\Vert }_{:=\alpha_{1}}+2\underbrace{\left\Vert \frac{1}{p}\cP_{\Omega}\left(\bm{\Delta}\bm{X}^{\star\top}\right)\right\Vert }_{:=\alpha_{2}}.\label{eq:alpha1-alpha2-decomp}
\end{align}
In the sequel, we bound $\alpha_{1}$ and $\alpha_{2}$ separately. 
\begin{enumerate}
\item Recall from \cite[Theorem 2.5]{MR1071714} the elementary inequality
that 
\begin{equation}
\left\Vert \bm{C}\right\Vert \leq\big\||\bm{C}|\big\|,\label{eq:fact-abs}
\end{equation}
where $|\bm{C}|:=[|c_{i,j}|]_{1\leq i,j\leq n}$ for any matrix $\bm{C}=[c_{i,j}]_{1\leq i,j\leq n}$.
In addition, for any matrix $\bm{D}:=[d_{i,j}]_{1\leq i,j\leq n}$ such
that $|d_{i,j}|\geq|c_{i,j}|$ for all $i$ and $j$, one has $\big\||\bm{C}|\big\|\leq\big\||\bm{D}|\big\|$.
Therefore 
\[
\alpha_{1}\leq\left\Vert \frac{1}{p}\cP_{\Omega}\left(\left|\bm{\Delta}\bm{\Delta}^{\top}\right|\right)\right\Vert \leq\left\Vert \bm{\Delta}\right\Vert _{2,\infty}^{2}\left\Vert \frac{1}{p}\mathcal{P}_{\Omega}\left(\bm{1}\bm{1}^{\top}\right)\right\Vert .
\]
Lemma \ref{lemma:montanari} then tells us that with probability at
least $1-O(n^{-10})$, 
\begin{equation}
\left\Vert \frac{1}{p}\mathcal{P}_{\Omega}\left(\bm{1}\bm{1}^{\top}\right)-\bm{1}\bm{1}^{\top}\right\Vert \leq C\sqrt{\frac{n}{p}}\label{eq:gamma_2_perturbation-1-2}
\end{equation}
for some universal constant $C>0$, as long as $p\gg\log n/n$.
This together with the triangle inequality yields 
\begin{equation}
\left\Vert \frac{1}{p}\mathcal{P}_{\Omega}\left(\bm{1}\bm{1}^{\top}\right)\right\Vert \leq\left\Vert \frac{1}{p}\mathcal{P}_{\Omega}\left(\bm{1}\bm{1}^{\top}\right)-\bm{1}\bm{1}^{\top}\right\Vert +\left\Vert \bm{1}\bm{1}^{\top}\right\Vert \leq C\sqrt{\frac{n}{p}}+n\leq2n,\label{eq:gamma_2_perturbation-1-1-1}
\end{equation}
provided that $p\gg1/n$. Putting together the previous bounds, we
arrive at 
\begin{equation}
\alpha_{1}\leq2n\left\Vert \bm{\Delta}\right\Vert _{2,\infty}^{2}.\label{eq:alpha1-bound}
\end{equation}
\item Regarding the second term $\alpha_{2}$, apply the elementary inequality
(\ref{eq:fact-abs}) once again to get 
\[
\left\Vert \cP_{\Omega}\left(\bm{\Delta}\bm{X}^{\star\top}\right)\right\Vert \leq\left\Vert \cP_{\Omega}\left(\left|\bm{\Delta}\bm{X}^{\star\top}\right|\right)\right\Vert ,
\]
which motivates us to look at $\left\Vert \cP_{\Omega}\left(\left|\bm{\Delta}\bm{X}^{\star\top}\right|\right)\right\Vert $
instead. A key step of this part is to take advantage of the $\ell_{2,\infty}$
norm constraint of $\cP_{\Omega}\left(\left|\bm{\Delta}\bm{X}^{\star\top}\right|\right)$.
Specifically, we claim for the moment that with probability exceeding
$1-O(n^{-10})$, 
\begin{equation}
\left\Vert \cP_{\Omega}\left(\left|\bm{\Delta}\bm{X}^{\star\top}\right|\right)\right\Vert _{2,\infty}^{2}\leq2p\sigma_{\max}\left\Vert \bm{\Delta}\right\Vert _{2,\infty}^{2}:=\theta\label{eq:L2-inf-P}
\end{equation}
holds under our sample size condition. In addition, we also have the
following trivial $\ell_{\infty}$ norm bound 
\begin{equation}
\left\Vert \cP_{\Omega}\left(\left|\bm{\Delta}\bm{X}^{\star\top}\right|\right)\right\Vert _{\infty}\leq\left\Vert \bm{\Delta}\right\Vert _{2,\infty}\left\Vert \bm{X}^{\star}\right\Vert _{2,\infty}:=\gamma.\label{eq:defn-gamma}
\end{equation}
In what follows, for simplicity of presentation, we will denote 
		\begin{equation}
\bm{A}:=\cP_{\Omega}\left(\left|\bm{\Delta}\bm{X}^{\star\top}\right|\right).
		\end{equation} 
\begin{enumerate}
\item To facilitate the analysis of $\|\bm{A}\|$, we first introduce $k_{0}+1=\frac{1}{2}\log\left(\kappa\mu r\right)$
auxiliary matrices\footnote{For simplicity, we assume $\frac{1}{2}\log\left(\kappa\mu r\right)$ is an integer. The argument here can be easily adapted to the case when $\frac{1}{2}\log\left(\kappa\mu r\right)$ is not an integer.} $\bm{B}_{s}\in\RR^{n\times n}$ that satisfy
\begin{equation}
\left\Vert \bm{A}\right\Vert \leq\left\Vert \bm{B}_{k_{0}}\right\Vert +\sum_{s=0}^{k_{0}-1}\left\Vert \bm{B}_{s}\right\Vert .\label{eq:A-B}
\end{equation}
To be precise, each $\bm{B}_{s}$ is defined such that 
\begin{align*}
\left[\bm{B}_{s}\right]_{j,k} & =\begin{cases}
\frac{1}{2^{s}}\gamma, & \text{if}\quad A_{j,k}\in(\frac{1}{2^{s+1}}\gamma,\frac{1}{2^{s}}\gamma],\\
0, & \text{else},
\end{cases}\quad \text{for } 0\leq s\leq k_{0}-1\qquad\text{and}\\
\left[\bm{B}_{k_{0}}\right]_{j,k} & =\begin{cases}
\frac{1}{2^{k_{0}}}\gamma, & \text{if}\quad A_{j,k}\leq\frac{1}{2^{k_{0}}}\gamma,\\
0, & \text{else},
\end{cases}
\end{align*}
which clearly satisfy (\ref{eq:A-B}); in words, $\bm{B}_{s}$ is
constructed by rounding up those entries of $\bm{A}$ within a prescribed
magnitude interval. Thus, it suffices to bound $\|\bm{B}_{s}\|$ for
every $s$. To this end, we start with $s=k_{0}$ and use the definition
of $\bm{B}_{k_{0}}$ to get 
\[
\left\Vert \bm{B}_{k_{0}}\right\Vert \overset{\text{(i)}}{\leq}\left\Vert \bm{B}_{k_{0}}\right\Vert _{\infty}\sqrt{\left(2np\right)^{2}}\overset{\text{(ii)}}{\leq}4np\frac{1}{\sqrt{\kappa\mu r}}\left\Vert \bm{\Delta}\right\Vert _{2,\infty}\left\Vert \bm{X}^{\star}\right\Vert _{2,\infty}\overset{\text{(iii)}}{\leq}4\sqrt{n}p\left\Vert \bm{\Delta}\right\Vert _{2,\infty}\left\Vert \bm{X}^{\star}\right\Vert ,
\]
where (i) arises from Lemma \ref{lemma:sparse_norm},
with $2np$ being a crude upper bound on the number of nonzero entries
in each row and each column. This can be derived by applying the standard Chernoff bound on $\Omega$. The second inequality (ii) relies on the definitions of $\gamma$ and $k_0$. The last one (iii) follows from the
incoherence condition (\ref{eq:incoherence-X-MC}). Besides, for any $0\leq s\leq k_{0}-1$,
by construction one has 
\[
\left\Vert \bm{B}_{s}\right\Vert _{2,\infty}^{2}\leq4\theta=8p\sigma_{\max}\left\Vert \bm{\Delta}\right\Vert _{2,\infty}^{2}\qquad\text{and}\qquad\left\Vert \bm{B}_{s}\right\Vert _{\infty}=\frac{1}{2^{s}}\gamma,
\]
where $\theta$ is as defined in (\ref{eq:L2-inf-P}). Here, we have
used the fact that the magnitude of each entry of $\bm{B}_{s}$ is
at most 2 times that of $\bm{A}$. An immediate implication is that
there are at most 
\[
\frac{\left\Vert \bm{B}_{s}\right\Vert _{2,\infty}^{2}}{\left\Vert \bm{B}_{s}\right\Vert _{\infty}^{2}}\leq\frac{8p\sigma_{\max}\left\Vert \bm{\Delta}\right\Vert _{2,\infty}^{2}}{\left(\frac{1}{2^{s}}\gamma\right)^{2}}:=k_{\mathrm{r}}
\]
nonzero entries in each row of $\bm{B}_{s}$ and at most 
\[
k_{\mathrm{c}}=2np
\]
nonzero entries in each column of $\bm{B}_{s}$, where $k_{\mathrm{c}}$ is
derived from the standard Chernoff bound on $\Omega$. Utilizing Lemma
\ref{lemma:sparse_norm} once more, we discover that 
\[
\left\Vert \bm{B}_{s}\right\Vert \leq\left\Vert \bm{B}_{s}\right\Vert _{\infty}\sqrt{k_{\mathrm{r}}k_{\mathrm{c}}}=\frac{1}{2^{s}}\gamma\sqrt{k_{\mathrm{r}}k_{\mathrm{c}}}=\sqrt{16np^{2}\sigma_{\max}\left\Vert \bm{\Delta}\right\Vert _{2,\infty}^{2}}=4\sqrt{n}p\left\Vert \bm{\Delta}\right\Vert _{2,\infty}\left\Vert \bm{X}^{\star}\right\Vert 
\]
for each $0\leq s \leq k_{0}-1$. Combining all, we arrive at 
\begin{align*}
\|\bm{A}\| & \leq\sum_{s=0}^{k_{0}-1}\left\Vert \bm{B}_{s}\right\Vert +\left\Vert \bm{B}_{k_{0}}\right\Vert \leq\left(k_{0}+1\right)4\sqrt{n}p\left\Vert \bm{\Delta}\right\Vert _{2,\infty}\left\Vert \bm{X}^{\star}\right\Vert \\
 & \leq2\sqrt{n}p\log\left(\kappa\mu r\right)\left\Vert \bm{\Delta}\right\Vert _{2,\infty}\left\Vert \bm{X}^{\star}\right\Vert \\
 & \leq2\sqrt{n}p\log n\left\Vert \bm{\Delta}\right\Vert _{2,\infty}\left\Vert \bm{X}^{\star}\right\Vert ,
\end{align*}
where the last relation holds under the condition $n\geq\kappa\mu r$. This further
gives 
\begin{align}
\alpha_{2} & \leq\frac{1}{p}\left\Vert \bm{A}\right\Vert \leq2\sqrt{n}\log n\left\Vert \bm{\Delta}\right\Vert _{2,\infty}\left\Vert \bm{X}^{\star}\right\Vert .\label{eq:alpha2-bound}
\end{align}
\item In order to finish the proof of this part, we need to justify the
claim (\ref{eq:L2-inf-P}). Observe that 
\begin{align}
\left\Vert \left[\cP_{\Omega}\left(\left|\bm{\Delta}\bm{X}^{\star\top}\right|\right)\right]_{l,\cdot}\right\Vert _{2}^{2} & =\sum\nolimits _{j=1}^{n}\left(\bm{\Delta}_{l,\cdot}\bm{X}_{j,\cdot}^{\star\top}\delta_{l,j}\right)^{2}\nonumber \\
 & =\bm{\Delta}_{l,\cdot}\left(\sum\nolimits _{j=1}^{n}\delta_{l,j}\bm{X}_{j,\cdot}^{\star\top}\bm{X}_{j,\cdot}^{\star}\right)\bm{\Delta}_{l,\cdot}^{\top}\nonumber \\
 & \leq\left\Vert \bm{\Delta}\right\Vert _{2,\infty}^{2}\left\Vert \sum\nolimits _{j=1}^{n}\delta_{l,j}\bm{X}_{j,\cdot}^{\star\top}\bm{X}_{j,\cdot}^{\star}\right\Vert \label{eq:energy_row_norm}
\end{align}
for every $1\leq l\leq n$, where $\delta_{l,j}$ indicates whether the entry with the index $(l,j)$ is observed or not. Invoke Lemma \ref{lemma:bernoulli-spectral-norm}
to yield 
\begin{align}
\left\Vert \sum\nolimits _{j=1}^{n}\delta_{l,j}\bm{X}_{j,\cdot}^{\star\top}\bm{X}_{j,\cdot}^{\star}\right\Vert  & =\left\Vert \left[\delta_{l,1}\bm{X}_{1,\cdot}^{\star\top},\delta_{l,2}\bm{X}_{2,\cdot}^{\star\top},\cdots,\delta_{l,n}\bm{X}_{n,\cdot}^{\star\top}\right]\right\Vert ^{2}\nonumber \\
 & \leq p\sigma_{\max}+C\left(\sqrt{p\left\Vert \bm{X}^{\star}\right\Vert _{2,\infty}^{2}\left\Vert \bm{X}^{\star}\right\Vert ^{2}\log n}+\left\Vert \bm{X}^{\star}\right\Vert _{2,\infty}^{2}\log n\right)\nonumber\\
 & \leq\left(p+C\sqrt{\frac{p\kappa\mu r\log n}{n}}+C\frac{\kappa\mu r\log n}{n}\right)\sigma_{\max}\nonumber \\
 & \leq2p\sigma_{\max},\label{eq:sampling_op}
\end{align}
with high probability, as soon as $np\gg\kappa\mu r\log n$. Combining (\ref{eq:energy_row_norm})
and (\ref{eq:sampling_op}) yields 
\begin{align*}
\left\Vert \left[\cP_{\Omega}\left(\left|\bm{\Delta}\bm{X}^{\star\top}\right|\right)\right]_{l,\cdot}\right\Vert _{2}^{2} & \leq2p\sigma_{\max}\left\Vert \bm{\Delta}\right\Vert _{2,\infty}^{2},\qquad1\leq l\leq n
\end{align*}
as claimed in (\ref{eq:L2-inf-P}). 
\end{enumerate}
\item Taken together, the preceding bounds (\ref{eq:alpha1-alpha2-decomp}),
(\ref{eq:alpha1-bound}) and (\ref{eq:alpha2-bound}) yield 
\begin{align*}
\left\Vert \frac{1}{p}\cP_{\Omega}\left(\bm{X}\bm{X}^{\top}-\bm{X}^{\star}\bm{X}^{\star\top}\right)\right\Vert  & \leq  \alpha_1 + 2\alpha_2 \leq 
2n\left\Vert \bm{\Delta}\right\Vert _{2,\infty}^{2}+4\sqrt{n}\log n\left\Vert \bm{\Delta}\right\Vert _{2,\infty}\left\Vert \bm{X}^{\star}\right\Vert.
\end{align*}
\end{enumerate}
The proof is completed by substituting the assumption $\left\Vert \bm{\Delta}\right\Vert _{2,\infty} \leq \epsilon  \left\Vert \bm{X}^{\star}\right\Vert _{2,\infty}.$
\end{proof}
In the end of this subsection, we record a useful lemma to bound the spectral norm of a sparse Bernoulli matrix.
\begin{lemma}
	\label{lemma:sparse_norm}
Let $\bm{A}\in\left\{ 0,1\right\} ^{n_{1}\times n_{2}}$ be a binary matrix,
and suppose that there are at most $k_{\mathrm{r}}$ and $k_{\mathrm{c}}$
nonzero entries in each row and column of $\bm{A}$, respectively.
Then one has $\left\Vert \bm{A}\right\Vert \leq\sqrt{k_{\mathrm{c}}k_{\mathrm{r}}}$.
\end{lemma}
\begin{proof} This immediately follows from the elementary inequality $\| \bm{A} \|^2 \leq \| \bm{A} \|_{1\rightarrow 1} \| \bm{A} \|_{\infty \rightarrow \infty}$ (see \cite[equation~(1.11)]{higham1992estimating}), where $\|\bm{A} \|_{1\rightarrow 1} $  and $\|\bm{A} \|_{\infty\rightarrow \infty} $  are the induced 1-norm (or maximum absolute column sum norm) and the induced $\infty$-norm (or maximum absolute row sum norm), respectively. 
\end{proof}

\subsubsection{Matrix perturbation bounds}
\begin{lemma}\label{lemma:davis-kahan-operator-norm}Let $\bm{M}\in\RR^{n\times n}$
be a symmetric matrix with the top-$r$ eigendecomposition $\bm{U}\bm{\Sigma}\bm{U}^{\top}$.
Assume $\left\Vert \bm{M}-\bm{M}^{\star}\right\Vert \leq\sigma_{\min}/2$ and denote 
\[
\hat{\bm{Q}}:=\argmin_{\bm{R}\in\cO^{r\times r}}\left\Vert \bm{U}\bm{R}-\bm{U}^{\star}\right\Vert _{\mathrm{F}}.
\]
Then there is some numerical constant $c_{3}>0$ such that
\[
\big\Vert \bm{U}\hat{\bm{Q}}-\bm{U}^{\star}\big\Vert \,\leq\,\frac{c_{3}}{\sigma_{\min}}\left\Vert \bm{M}-\bm{M}^{\star}\right\Vert .
\]
\end{lemma}
\begin{proof}Define $\bm{Q}=\bm{U}^{\top}\bm{U}^{\star}$.
The triangle inequality gives 
\begin{equation}
\big\|\bm{U}\hat{\bm{Q}}-\bm{U}^{\star}\big\|\leq\big\|\bm{U}\big(\hat{\bm{Q}}-\bm{Q}\big)\big\|+\left\Vert \bm{U}\bm{Q}-\bm{U}^{\star}\right\Vert \leq\big\|\hat{\bm{Q}}-\bm{Q}\big\|+\left\Vert \bm{U}\bm{U}^{\top}\bm{U}^{\star}-\bm{U}^{\star}\right\Vert .\label{eq:rotation-operator-triangle}
\end{equation}
\cite[Lemma 3]{abbe2017entrywise} asserts that
\[
\big\|\hat{\bm{Q}}-\bm{Q}\big\|\leq4\left(\left\Vert \bm{M}-\bm{M}^{\star}\right\Vert /\sigma_{\min}\right)^{2}
\]
 as long as $\left\Vert \bm{M}-\bm{M}^{\star}\right\Vert \leq\sigma_{\min}/2$.
For the remaining term in (\ref{eq:rotation-operator-triangle}),
one can use $\bm{U}^{\star\top}\bm{U}^{\star}=\bm{I}_{r}$ to
obtain 
\[
\left\Vert \bm{U}\bm{U}^{\top}\bm{U}^{\star}-\bm{U}^{\star}\right\Vert =\left\Vert \bm{U}\bm{U}^{\top}\bm{U}^{\star}-\bm{U}^{\star}\bm{U}^{\star\top}\bm{U}^{\star}\right\Vert \leq\left\Vert \bm{U}\bm{U}^{\top}-\bm{U}^{\star}\bm{U}^{\star\top}\right\Vert ,
\]
which together with the Davis-Kahan sin$\Theta$ theorem \cite{davis1970rotation}
reveals that
\[
\left\Vert \bm{U}\bm{U}^{\top}\bm{U}^{\star}-\bm{U}^{\star}\right\Vert \leq\frac{c_{2}}{\sigma_{\min}}\left\Vert \bm{M}-\bm{M}^{\star}\right\Vert 
\]
for some constant $c_{2}>0$. Combine the estimates on $\big\|\hat{\bm{Q}}-\bm{Q}\big\|$,
$\left\Vert \bm{U}\bm{U}^{\top}\bm{U}^{\star}-\bm{U}^{\star}\right\Vert $
and (\ref{eq:rotation-operator-triangle}) to reach 
\[
\big\|\bm{U}\hat{\bm{Q}}-\bm{U}^{\star}\big\|\leq\left(\frac{4}{\sigma_{\min}}\left\Vert \bm{M}-\bm{M}^{\star}\right\Vert \right)^{2}+\frac{c_{2}}{\sigma_{\min}}\left\Vert \bm{M}-\bm{M}^{\star}\right\Vert \leq\frac{c_{3}}{\sigma_{\min}}\left\Vert \bm{M}-\bm{M}^{\star}\right\Vert 
\]
for some numerical constant $c_{3}>0$, where we have utilized the fact
that $\left\Vert \bm{M}-\bm{M}^{\star}\right\Vert /\sigma_{\min}\leq 1/2$.
\end{proof}

\begin{lemma}\label{lemma:interchange-MC}Let $\bm{M},\tilde{\bm{M}}\in\RR^{n\times n}$
be two symmetric matrices with top-$r$ eigendecompositions $\bm{U}\bm{\Sigma}\bm{U}^{\top}$ and $\tilde{\bm{U}}\tilde{\bm{\Sigma}}\tilde{\bm{U}}^{\top}$,
respectively. Assume $\left\Vert \bm{M}-\bm{M}^{\star}\right\Vert \leq\sigma_{\min}/4$
and $\big\|\tilde{\bm{M}}-\bm{M}^{\star}\big\|\leq\sigma_{\min}/4$,
and suppose $\sigma_{\max}/\sigma_{\min}$ is bounded by some constant
$c_{1}>0$, with $\sigma_{\max}$ and $\sigma_{\min}$ the largest
and the smallest singular values of $\bm{M}^{\star}$, respectively.
If we denote
\[
\bm{Q}:=\argmin_{\bm{R}\in\cO^{r\times r}}\big\|\bm{U}\bm{R}-\tilde{\bm{U}}\big\|_{\mathrm{F}},
\]
then there exists some numerical constant $c_{3}>0$ such that
\[
\left\Vert \bm{\Sigma}^{1/2}\bm{Q}-\bm{Q}\tilde{\bm{\Sigma}}^{1/2}\right\Vert \leq\frac{c_{3}}{\sqrt{\sigma_{\min}}}\big\|\tilde{\bm{M}}-\bm{M}\big\|\qquad\text{and}\qquad\left\Vert \bm{\Sigma}^{1/2}\bm{Q}-\bm{Q}\tilde{\bm{\Sigma}}^{1/2}\right\Vert _{\mathrm{F}}\leq\frac{c_{3}}{\sqrt{\sigma_{\min}}}\left\Vert \big(\tilde{\bm{M}}-\bm{M}\big)\bm{U}\right\Vert _{\mathrm{F}}.
\]
\end{lemma}\begin{proof}Here, we focus on the Frobenius norm; the
bound on the operator norm follows from the same argument, and hence
we omit the proof. Since $\left\Vert \cdot\right\Vert _{\mathrm{F}}$ is unitarily invariant,
we have
\[
\left\Vert \bm{\Sigma}^{1/2}\bm{Q}-\bm{Q}\tilde{\bm{\Sigma}}^{1/2}\right\Vert _{\mathrm{F}}=\left\Vert \bm{Q}^{\top}\bm{\Sigma}^{1/2}\bm{Q}-\tilde{\bm{\Sigma}}^{1/2}\right\Vert _{\mathrm{F}},
\]
where $\bm{Q}^{\top}\bm{\Sigma}^{1/2}\bm{Q}$ and $\tilde{\bm{\Sigma}}^{1/2}$
are the matrix square roots of $\bm{Q}^{\top}\bm{\Sigma}\bm{Q}$ and
$\tilde{\bm{\Sigma}}$, respectively. In view of the matrix square
root perturbation bound \cite[Lemma 2.1]{MR1176461}, 
\begin{equation}
\left\Vert \bm{\Sigma}^{1/2}\bm{Q}-\bm{Q}\tilde{\bm{\Sigma}}^{1/2}\right\Vert _{\mathrm{F}}\leq\frac{1}{\sigma_{\min}\big[\left(\bm{\Sigma}\right)^{1/2}\big]+\sigma_{\min}\big[(\tilde{\bm{\Sigma}})^{1/2}\big]}\left\Vert \bm{Q}^{\top}\bm{\Sigma}\bm{Q}-\tilde{\bm{\Sigma}}\right\Vert _{\mathrm{F}}\leq\frac{1}{\sqrt{\sigma_{\min}}}\left\Vert \bm{Q}^{\top}\bm{\Sigma}\bm{Q}-\tilde{\bm{\Sigma}}\right\Vert _{\mathrm{F}},\label{eq:interchange-sqrt-MC}
\end{equation}
where the last inequality follows from the lower estimates $$\sigma_{\min}\left(\bm{\Sigma}\right)\geq \sigma_{\min}\left(\bm{\Sigma}^{\star}\right) - \| \bm{M} - \bm{M}^{\star} \| \geq \sigma_{\min}/4 $$
and, similarly, $\sigma_{\min}(\tilde{\bm{\Sigma}})\geq \sigma_{\min} /4$.
Recognizing that $\bm{\Sigma}=\bm{U}^{\top}\bm{M}\bm{U}$ and $\tilde{\bm{\Sigma}}=\tilde{\bm{U}}^{\top}\tilde{\bm{M}}\tilde{\bm{U}}$,
one gets
\begin{align}
 & \left\Vert \bm{Q}^{\top}\bm{\Sigma}\bm{Q}-\tilde{\bm{\Sigma}}\right\Vert _{\mathrm{F}}=\left\Vert \big(\bm{U}\bm{Q}\big)^{\top}\bm{M}\big(\bm{U}\bm{Q}\big)-\tilde{\bm{U}}^{\top}\tilde{\bm{M}}\tilde{\bm{U}}\right\Vert _{\mathrm{F}}\nonumber \\
 & \quad\leq\left\Vert \big(\bm{U}\bm{Q}\big)^{\top}\bm{M}\big(\bm{U}\bm{Q}\big)-\big(\bm{U}\bm{Q}\big)^{\top}\tilde{\bm{M}}\big(\bm{U}\bm{Q}\big)\right\Vert _{\mathrm{F}}+\left\Vert \big(\bm{U}\bm{Q}\big)^{\top}\tilde{\bm{M}}\big(\bm{U}\bm{Q}\big)-\tilde{\bm{U}}^{\top}\tilde{\bm{M}}\big(\bm{U}\bm{Q}\big)\right\Vert _{\mathrm{F}}\nonumber \\
 & \qquad\qquad+\left\Vert \tilde{\bm{U}}^{\top}\tilde{\bm{M}}\big(\bm{U}\bm{Q}\big)-\tilde{\bm{U}}^{\top}\tilde{\bm{M}}\tilde{\bm{U}}\right\Vert _{\mathrm{F}}\nonumber \\
 & \quad\leq\left\Vert \big(\tilde{\bm{M}}-\bm{M}\big)\bm{U}\right\Vert _{\mathrm{F}}+2\big\|\bm{U}\bm{Q}-\tilde{\bm{U}}\big\|_{\mathrm{F}}\big\|\tilde{\bm{M}}\big\|\leq\left\Vert \big(\tilde{\bm{M}}-\bm{M}\big)\bm{U}\right\Vert _{\mathrm{F}}+4\sigma_{\max}\big\|\bm{U}\bm{Q}-\tilde{\bm{U}}\big\|_{\mathrm{F}},\label{eq:interchange-triangle-MC}
\end{align}
where the last relation holds due to the upper estimate
\[
	\big\|\tilde{\bm{M}}\big\|\leq \big\|{\bm{M}^{\star}}\big\|+\big\|\tilde{\bm{M}}-\bm{M}^{\star}\big\|\leq\sigma_{\max}+\sigma_{\min}/4 \leq 2\sigma_{\max}.
\]
Invoke the Davis-Kahan sin$\Theta$ theorem \cite{davis1970rotation}
to obtain 
\begin{equation}
\big\|\bm{U}\bm{Q}-\tilde{\bm{U}}\big\|_{\mathrm{F}}\leq\frac{c_{2}}{\sigma_{r}\left(\bm{M}\right)-\sigma_{r+1}(\tilde{\bm{M}})}\left\Vert \big(\tilde{\bm{M}}-\bm{M}\big)\bm{U}\right\Vert _{\mathrm{F}}\leq\frac{2c_{2}}{\sigma_{\min}}\left\Vert \big(\tilde{\bm{M}}-\bm{M}\big)\bm{U}\right\Vert _{\mathrm{F}},\label{eq:interchange-davis-MC}
\end{equation}
for some constant $c_{2}>0$, where the last inequality follows from
the bounds 
\begin{align*}
\sigma_{r}\left(\bm{M}\right) &\geq \sigma_{r}\left(\bm{M}^{\star}\right) - \| \bm{M} - \bm{M}^{\star} \| \geq 3\sigma_{\min}/4, \\
\sigma_{r+1}(\tilde{\bm{M}})  &\leq  \sigma_{r+1}\left(\bm{M}^{\star}\right) + \| \tilde{\bm{M}} - \bm{M}^{\star} \| \leq  \sigma_{\min}/4.
\end{align*}
Combine
(\ref{eq:interchange-sqrt-MC}), (\ref{eq:interchange-triangle-MC}),
(\ref{eq:interchange-davis-MC}) and the fact $\sigma_{\max}/\sigma_{\min}\leq c_{1}$
to reach 
\[
\left\Vert \bm{\Sigma}^{1/2}\bm{Q}-\bm{Q}\tilde{\bm{\Sigma}}^{1/2}\right\Vert _{\mathrm{F}}\leq\frac{c_{3}}{\sqrt{\sigma_{\min}}}\left\Vert \big(\tilde{\bm{M}}-\bm{M}\big)\bm{U}\right\Vert _{\mathrm{F}}
\]
for some constant $c_{3}>0$. \end{proof}


\begin{lemma}\label{lemma:rotation-U-R-diff-MC}Let $\bm{M}\in\RR^{n\times n}$
be a symmetric matrix with the top-$r$ eigendecomposition $\bm{U}\bm{\Sigma}\bm{U}^{\top}$.
Denote $\bm{X}=\bm{U}\bm{\Sigma}^{1/2}$ and $\bm{X}^{\star}=\bm{U}^{\star}(\bm{\Sigma}^{\star})^{1/2}$,
and define
\[
\hat{\bm{Q}}:=\argmin_{\bm{R}\in\cO^{r\times r}}\left\Vert \bm{U}\bm{R}-\bm{U}^{\star}\right\Vert _{\mathrm{F}}\qquad\text{and}\qquad\hat{\bm{H}}:=\argmin_{\bm{R}\in\cO^{r\times r}}\left\Vert \bm{X}\bm{R}-\bm{X}^{\star}\right\Vert _{\mathrm{F}}.
\]
Assume $\left\Vert \bm{M}-\bm{M}^{\star}\right\Vert \leq\sigma_{\min}/2$,
and suppose $\sigma_{\max}/\sigma_{\min}$ is bounded by some constant
$c_{1}>0$. Then there exists a numerical constant $c_{3}>0$ such
that
\[
\big\Vert \hat{\bm{Q}}-\hat{\bm{H}}\big\Vert \,\leq\,\frac{c_{3}}{\sigma_{\min}}\left\Vert \bm{M}-\bm{M}^{\star}\right\Vert .
\]
\end{lemma}\begin{proof}We first collect several useful facts about
the spectrum of $\bm{\Sigma}$. Weyl's inequality tells us that $\left\Vert \bm{\Sigma}-\bm{\Sigma}^{\star}\right\Vert \leq\left\Vert \bm{M}-\bm{M}^{\star}\right\Vert \leq\sigma_{\min}/2$,
which further implies that 
\[
\sigma_{r}\left(\bm{\Sigma}\right)\geq\sigma_{r}\left(\bm{\Sigma}^{\star}\right)-\left\Vert \bm{\Sigma}-\bm{\Sigma}^{\star}\right\Vert \geq\sigma_{\min}/2\qquad\text{and}\qquad\left\Vert \bm{\Sigma}\right\Vert \leq\left\Vert \bm{\Sigma}^{\star}\right\Vert +\left\Vert \bm{\Sigma}-\bm{\Sigma}^{\star}\right\Vert \leq2\sigma_{\max}.
\]

Denote 
\[
\bm{Q}=\bm{U}^{\top}\bm{U}^{\star}\qquad\text{and}\qquad\bm{H}=\bm{X}^{\top}\bm{X}^{\star}.
\]
Simple algebra yields 
\[
\bm{H}=\bm{\Sigma}^{1/2}\bm{Q}\left(\bm{\Sigma}^{\star}\right)^{1/2}=\underset{:=\bm{E}}{\underbrace{\bm{\Sigma}^{1/2}\big(\bm{Q}-\hat{\bm{Q}}\big)\left(\bm{\Sigma}^{\star}\right)^{1/2}+\big(\bm{\Sigma}^{1/2}\hat{\bm{Q}}-\hat{\bm{Q}}\bm{\Sigma}^{1/2}\big)\left(\bm{\Sigma}^{\star}\right)^{1/2}}}+\underset{:=\bm{A}}{\underbrace{\hat{\bm{Q}}\left(\bm{\Sigma}\bm{\Sigma}^{\star}\right)^{1/2}}}.
\]
It can be easily seen that $\sigma_{r-1}\left(\bm{A}\right)\geq\sigma_{r}\left(\bm{A}\right)\geq\sigma_{\min}/2$,
and 
\begin{align*}
\left\Vert \bm{E}\right\Vert  & \leq\big\|\bm{\Sigma}^{1/2}\big\|\cdot\big\|\bm{Q}-\hat{\bm{Q}}\big\|\cdot\big\|\big(\bm{\Sigma}^{\star}\big)^{1/2}\big\|+\left\Vert \bm{\Sigma}^{1/2}\hat{\bm{Q}}-\hat{\bm{Q}}\bm{\Sigma}^{1/2}\right\Vert \cdot\big\|\big(\bm{\Sigma}^{\star}\big)^{1/2}\big\|\\
 & \leq2\sigma_{\max}\underbrace{\big\|\bm{Q}-\hat{\bm{Q}}\big\|}_{:=\alpha}+\sqrt{\sigma_{\max}}\underbrace{\left\Vert \bm{\Sigma}^{1/2}\hat{\bm{Q}}-\hat{\bm{Q}}\bm{\Sigma}^{1/2}\right\Vert }_{:=\beta},
\end{align*}
which can be controlled as follows.
\begin{itemize}
\item Regarding $\alpha$, use \cite[Lemma 3]{abbe2017entrywise} to reach
\begin{equation*}
\alpha=\big\|\bm{Q}-\hat{\bm{Q}}\big\|\leq4\left\Vert \bm{M}-\bm{M}^{\star}\right\Vert ^{2}/\sigma_{\min}^{2}.
\end{equation*}
\item For $\beta$, one has 
\begin{align*}
\beta & \overset{\left(\text{i}\right)}{=}\left\Vert \hat{\bm{Q}}^{\top}\bm{\Sigma}^{1/2}\hat{\bm{Q}}-\bm{\Sigma}^{1/2}\right\Vert \overset{\left(\text{ii}\right)}{\leq}\frac{1}{2\sigma_{r}\left(\bm{\Sigma}^{1/2}\right)}\left\Vert \hat{\bm{Q}}^{\top}\bm{\Sigma}\hat{\bm{Q}}-\bm{\Sigma}\right\Vert \overset{\left(\text{iii}\right)}{=}\frac{1}{2\sigma_{r}\left(\bm{\Sigma}^{1/2}\right)}\left\Vert \bm{\Sigma}\hat{\bm{Q}}-\hat{\bm{Q}}\bm{\Sigma}\right\Vert ,
\end{align*}
 where (i) and (iii) come from the unitary invariance of $\left\Vert \cdot\right\Vert $,
and (ii) follows from the matrix square root perturbation bound \cite[Lemma 2.1]{MR1176461}.
We can further take the triangle inequality to obtain 
\begin{align*}
\left\Vert \bm{\Sigma}\hat{\bm{Q}}-\hat{\bm{Q}}\bm{\Sigma}\right\Vert  
& = \left\Vert \bm{\Sigma}\bm{Q}-\bm{Q}\bm{\Sigma} + \bm{\Sigma}(\hat{\bm{Q}} - \bm{Q}) - (\hat{\bm{Q}}-\bm{Q})\bm{\Sigma} \right\Vert \\
& \leq\left\Vert \bm{\Sigma}\bm{Q}-\bm{Q}\bm{\Sigma}\right\Vert +2\left\Vert \bm{\Sigma}\right\Vert \big\|\bm{Q}-\hat{\bm{Q}}\big\|\\
 & =\left\Vert \bm{U}\left(\bm{M}-\bm{M}^{\star}\right)\bm{U}^{\star\top}+\bm{Q}\left(\bm{\Sigma}^{\star}-\bm{\Sigma}\right)\right\Vert +2\left\Vert \bm{\Sigma}\right\Vert \big\|\bm{Q}-\hat{\bm{Q}}\big\|\\
 & \leq\left\Vert \bm{U}\left(\bm{M}-\bm{M}^{\star}\right)\bm{U}^{\star\top}\big\|+\big\|\bm{Q}\left(\bm{\Sigma}^{\star}-\bm{\Sigma}\right)\right\Vert +2\left\Vert \bm{\Sigma}\right\Vert \big\|\bm{Q}-\hat{\bm{Q}}\big\|\\
 & \leq2\left\Vert \bm{M}-\bm{M}^{\star}\right\Vert +4\sigma_{\max}\alpha,
\end{align*}
where the last inequality uses the Weyl's inequality $\|\bm{\Sigma}^{\star}-\bm{\Sigma}\|\leq\|\bm{M}-\bm{M}^{\star}\|$
and the fact that $\left\Vert \bm{\Sigma}\right\Vert \leq2\sigma_{\max}$.
\item Rearrange the previous bounds to arrive at 
\[
\left\Vert \bm{E}\right\Vert \leq2\sigma_{\max}\alpha+\sqrt{\sigma_{\max}}\frac{1}{\sqrt{\sigma_{\min}}}\left(2\left\Vert \bm{M}-\bm{M}^{\star}\right\Vert +4\sigma_{\max}\alpha\right)\leq c_{2}\left\Vert \bm{M}-\bm{M}^{\star}\right\Vert 
\]
for some numerical constant $c_{2}>0$, where we have used the assumption that $\sigma_{\max} / \sigma_{\min}$ is bounded. 
\end{itemize}
Recognizing that $\hat{\bm{Q}}=\text{sgn}\left(\bm{A}\right)$ (see
definition in (\ref{eq:sign-matrix-defn})), we are ready to invoke
Lemma \ref{lemma:rotation-diff} to deduce that 
\begin{align*}
\left\Vert \hat{\bm{Q}}-\hat{\bm{H}}\right\Vert  & \leq\frac{2}{\sigma_{r-1}\left(\bm{A}\right)+\sigma_{r}\left(\bm{A}\right)}\left\Vert \bm{E}\right\Vert \leq\frac{c_{3}}{\sigma_{\min}}\left\Vert \bm{M}-\bm{M}^{\star}\right\Vert 
\end{align*}
for some constant $c_{3}>0$. \end{proof}


\subsection{Technical lemmas for blind deconvolution}

\subsubsection{Wirtinger calculus\label{sec:Wirtinger-calculus}}
In this section, we formally prove the fundamental theorem of calculus and the mean-value form of Taylor's theorem under the Wirtinger calculus; see (\ref{eq:fundamental-wirtinger}) and (\ref{eq:mean-value-wirtinger}), respectively.

Let $f:\CC^{n}\to\RR$ be a real-valued function. Denote $\bm{z}=\bm{x}+i\bm{y}\in\CC^{n}$,
then $f\left(\cdot\right)$ can alternatively be viewed as a function
$\RR^{2n}\to\RR$. There is a one-to-one mapping connecting the Wirtinger
derivatives and the conventional derivatives \cite{kreutz2009complex}:
\begin{subequations}\label{eq:relation-Wirtinger-usual} 
\begin{align}
\left[\begin{array}{c}
\bm{x}\\
\bm{y}
\end{array}\right] & =\bm{J}^{-1}\left[\begin{array}{c}
\bm{z}\\
\overline{\bm{z}}
\end{array}\right],\\
\nabla_{\RR}f\left(\left[\begin{array}{c}
\bm{x}\\
\bm{y}
\end{array}\right]\right) & =\bm{J}^{\conj}\nabla_{\CC}f\left(\left[\begin{array}{c}
\bm{z}\\
\overline{\bm{z}}
\end{array}\right]\right),\\
\nabla_{\RR}^{2}f\left(\left[\begin{array}{c}
\bm{x}\\
\bm{y}
\end{array}\right]\right) & =\bm{J}^{\conj}\nabla_{\CC}^{2}f\left(\left[\begin{array}{c}
\bm{z}\\
\overline{\bm{z}}
\end{array}\right]\right)\bm{J},
\end{align}
\end{subequations}where the subscripts $\RR$ and $\CC$
represent calculus in the real (conventional) sense and in the complex
(Wirtinger) sense, respectively, and 
\[
\bm{J}=\left[\begin{array}{cc}
\bm{I}_{n} & i\bm{I}_{n}\\
\bm{I}_{n} & -i\bm{I}_{n}
\end{array}\right].
\]

With these relationships in place, we are ready to verify the fundamental theorem of calculus using the Wirtinger derivatives. Recall from \cite[Chapter XIII, Theorem 4.2]{lang1993real} that 
\begin{equation}
\nabla_{\RR}f\left(\left[\begin{array}{c}
\bm{x}_{1}\\
\bm{y}_{1}
\end{array}\right]\right)-\nabla_{\RR}f\left(\left[\begin{array}{c}
\bm{x}_{2}\\
\bm{y}_{2}
\end{array}\right]\right)=\left[\int_{0}^{1}\nabla_{\RR}^{2}f\left(\left[\begin{array}{c}
\bm{x}\left(\tau\right)\\
\bm{y}\left(\tau\right)
\end{array}\right]\right)\mathrm{d}\tau\right]\left(\left[\begin{array}{c}
\bm{x}_{1}\\
\bm{y}_{1}
\end{array}\right]-\left[\begin{array}{c}
\bm{x}_{2}\\
\bm{y}_{2}
\end{array}\right]\right),\label{eq:MVT-usual}
\end{equation}
where 
\[
\left[\begin{array}{c}
\bm{x}\left(\tau\right)\\
\bm{y}\left(\tau\right)
\end{array}\right] :=\left[\begin{array}{c}
\bm{x}_{2}\\
\bm{y}_{2}
\end{array}\right]+\tau\left(\left[\begin{array}{c}
\bm{x}_{1}\\
\bm{y}_{1}
\end{array}\right]-\left[\begin{array}{c}
\bm{x}_{2}\\
\bm{y}_{2}
\end{array}\right]\right).
\]
Substitute the identities (\ref{eq:relation-Wirtinger-usual}) into
(\ref{eq:MVT-usual}) to arrive at 
\begin{align}
\bm{J}^{\conj}\nabla_{\CC}f\left(\left[\begin{array}{c}
\bm{z}_{1}\\
\overline{\bm{z}_{1}}
\end{array}\right]\right)-\bm{J}^{\conj}\nabla_{\CC}f\left(\left[\begin{array}{c}
\bm{z}_{2}\\
\overline{\bm{z}_{2}}
\end{array}\right]\right) & =\bm{J}^{\conj}\left[\int_{0}^{1}\nabla_{\CC}^{2}f\left(\left[\begin{array}{c}
\bm{z}\left(\tau\right)\\
\overline{\bm{z}\left(\tau\right)}
\end{array}\right]\right)\mathrm{d}\tau\right]\bm{J}\bm{J}^{-1}\left(\left[\begin{array}{c}
\bm{z}_{1}\\
\overline{\bm{z}_{1}}
\end{array}\right]-\left[\begin{array}{c}
\bm{z}_{2}\\
\overline{\bm{z}_{2}}
\end{array}\right]\right)\nonumber \\
 & =\bm{J}^{\conj}\left[\int_{0}^{1}\nabla_{\CC}^{2}f\left(\left[\begin{array}{c}
\bm{z}\left(\tau\right)\\
\overline{\bm{z}\left(\tau\right)}
\end{array}\right]\right)\mathrm{d}\tau\right]\left(\left[\begin{array}{c}
\bm{z}_{1}\\
\overline{\bm{z}_{1}}
\end{array}\right]-\left[\begin{array}{c}
\bm{z}_{2}\\
\overline{\bm{z}_{2}}
\end{array}\right]\right),\label{eq:hessian-real-wirtinger}
\end{align}
where $\bm{z}_{1}=\bm{x}_{1}+i\bm{y}_{1}$, $\bm{z}_{2}=\bm{x}_{2}+i\bm{y}_{2}$
and 
\[
\left[\begin{array}{c}
\bm{z}\left(\tau\right)\\
\overline{\bm{z}\left(\tau\right)}
\end{array}\right] :=\left[\begin{array}{c}
\bm{z}_{2}\\
\overline{\bm{z}_{2}}
\end{array}\right]+\tau\left(\left[\begin{array}{c}
\bm{z}_{1}\\
\overline{\bm{z}_{1}}
\end{array}\right]-\left[\begin{array}{c}
\bm{z}_{2}\\
\overline{\bm{z}_{2}}
\end{array}\right]\right).
\]
Simplification of (\ref{eq:hessian-real-wirtinger}) gives 
\begin{align}
\nabla_{\CC}f\left(\left[\begin{array}{c}
\bm{z}_{1}\\
\overline{\bm{z}_{1}}
\end{array}\right]\right)-\nabla_{\CC}f\left(\left[\begin{array}{c}
\bm{z}_{2}\\
\overline{\bm{z}_{2}}
\end{array}\right]\right)=\left[\int_{0}^{1}\nabla_{\CC}^{2}f\left(\left[\begin{array}{c}
\bm{z}\left(\tau\right)\\
\overline{\bm{z}\left(\tau\right)}
\end{array}\right]\right)\mathrm{d}\tau\right]\left(\left[\begin{array}{c}
\bm{z}_{1}\\
\overline{\bm{z}_{1}}
\end{array}\right]-\left[\begin{array}{c}
\bm{z}_{2}\\
\overline{\bm{z}_{2}}
\end{array}\right]\right).\label{eq:fundamental-wirtinger}
\end{align}

Repeating the above arguments, one can also show that 
\begin{align}
f\left(\bm{z}_{1}\right)-f\left(\bm{z}_{2}\right)=\nabla_{\CC}f\left(\bm{z}_{2}\right)^{\conj}\left[\begin{array}{c}
\bm{z}_{1}-\bm{z}_{2}\\
\overline{\bm{z}_{1}-\bm{z}_{2}}
\end{array}\right]+\frac{1}{2}\left[\begin{array}{c}
\bm{z}_{1}-\bm{z}_{2}\\
\overline{\bm{z}_{1}-\bm{z}_{2}}
\end{array}\right]^{\conj}\nabla_{\CC}^{2}f\left(\tilde{\bm{z}}\right)\left[\begin{array}{c}
\bm{z}_{1}-\bm{z}_{2}\\
\overline{\bm{z}_{1}-\bm{z}_{2}}
\end{array}\right],\label{eq:mean-value-wirtinger}
\end{align}
where $\tilde{\bm{z}}$ is some point lying on the vector connecting $\bm{z}_{1}$ and $\bm{z}_{2}$. This is the mean-value form of Taylor's theorem under the Wirtinger calculus.

\subsubsection{Discrete Fourier transform matrices}\label{sec:fourier}

Let $\bm{B}\in\CC^{m\times K}$ be the first $K$ columns of a discrete Fourier transform (DFT)
matrix $\bm{F}\in\mathbb{C}^{m\times m}$, and denote by $\bm{b}_{l}$
the $l$th column of the matrix $\bm{B}^{\conj}$. By definition, 
\[
\bm{b}_{l}=\frac{1}{\sqrt{m}}\left(1,\omega^{\left(l-1\right)},\omega^{2\left(l-1\right)},\cdots,\omega^{\left(K-1\right)\left(l-1\right)}\right)^{\conj},
\]
where $\omega:=e^{-i\frac{2\pi}{m}}$ with $i$ representing the imaginary unit. It is seen that for any $j\neq l$,
\begin{equation}
\bm{b}_{l}^{\conj}\bm{b}_{j}=\frac{1}{m}\sum_{k=0}^{K-1}\omega^{k\left(l-1\right)}\cdot\overline{\omega^{k\left(j-1\right)}}\overset{\left(\text{i}\right)}{=}\frac{1}{m}\sum_{k=0}^{K-1}\omega^{k\left(l-1\right)}\cdot\omega^{k\left(1-j\right)}=\frac{1}{m}\sum_{k=0}^{K-1}\left(\omega^{l-j}\right)^{k}\overset{\left(\text{ii}\right)}{=}\frac{1}{m}\frac{1-\omega^{K\left(l-j\right)}}{1-\omega^{l-j}}.\label{eq:fourier-inner-prod}
\end{equation}
Here, (i) uses $\overline{\omega^{\alpha}}=\omega^{-\alpha}$ for
all $\alpha \in \RR$, while the last identity (ii) follows from the formula
for the sum of a finite geometric series when $\omega^{l-j} \neq 1$. This leads to the following
lemma.

\begin{lemma}\label{lemma:fourier-sum-inner}For any $m\geq3$ and
any $1\leq l\leq m$, we have 
\[
\sum_{j=1}^{m}\left|\bm{b}_{l}^{\conj}\bm{b}_{j}\right|\leq4\log m.
\]
\end{lemma}\begin{proof}We first make use of the identity (\ref{eq:fourier-inner-prod})
to obtain 
\[
\sum_{j=1}^{m}\left|\bm{b}_{l}^{\conj}\bm{b}_{j}\right|=\left\Vert \bm{b}_{l}\right\Vert _{2}^{2}+\frac{1}{m}\sum_{j:j\neq l}^{m}\left|\frac{1-\omega^{K\left(l-j\right)}}{1-\omega^{l-j}}\right|=\frac{K}{m}+\frac{1}{m}\sum_{j:j\neq l}^{m}\left|\frac{\sin\left[K\left(l-j\right)\frac{\pi}{m}\right]}{\sin\left[\left(l-j\right)\frac{\pi}{m}\right]}\right|,
\]
where the last identity follows since $\left\Vert \bm{b}_{l}\right\Vert _{2}^{2}={K} / {m}$
and, for all $\alpha \in \RR$, 
\begin{align}
\left|1-\omega^{\alpha}\right| & =\left|1-e^{-i\frac{2\pi}{m}\alpha}\right|=\left|e^{-i\frac{\pi}{m}\alpha}\left(e^{i\frac{\pi}{m}\alpha}-e^{-i\frac{\pi}{m}\alpha}\right)\right|=2\left|\sin\left(\alpha\frac{\pi}{m}\right)\right|.\label{eq:omega-1-minus}
\end{align}
Without loss of generality, we focus on the case when $l=1$ in the
sequel. Recall that for $c>0$, we denote by $\left\lfloor c\right\rfloor $ 
the largest integer that does not exceed $c$. We can continue the derivation to get 
\begin{align*}
\sum_{j=1}^{m}\left|\bm{b}_{1}^{\conj}\bm{b}_{j}\right| & =\frac{K}{m}+\frac{1}{m}\sum_{j=2}^{m}\left|\frac{\sin\left[K\left(1-j\right)\frac{\pi}{m}\right]}{\sin\left[\left(1-j\right)\frac{\pi}{m}\right]}\right|\overset{\left(\text{i}\right)}{\leq}\frac{1}{m}\sum_{j=2}^{m}\left|\frac{1}{\sin\left[\left(j-1\right)\frac{\pi}{m}\right]}\right|+\frac{K}{m}\\
 & =\frac{1}{m}\left(\sum_{j=2}^{\left\lfloor \frac{m}{2}\right\rfloor +1}\left|\frac{1}{\sin\left[\left(j-1\right)\frac{\pi}{m}\right]}\right|+\sum_{j=\left\lfloor \frac{m}{2}\right\rfloor +2}^{m}\left|\frac{1}{\sin\left[\left(j-1\right)\frac{\pi}{m}\right]}\right|\right)+\frac{K}{m}\\
 & \overset{\left(\text{ii}\right)}{=}\frac{1}{m}\left(\sum_{j=2}^{\left\lfloor \frac{m}{2}\right\rfloor +1}\left|\frac{1}{\sin\left[\left(j-1\right)\frac{\pi}{m}\right]}\right|+\sum_{j=\left\lfloor \frac{m}{2}\right\rfloor +2}^{m}\left|\frac{1}{\sin\left[\left(m+1-j\right)\frac{\pi}{m}\right]}\right|\right)+\frac{K}{m},
\end{align*}
where (i) follows from $\left|\sin\left(K\left(1-j\right)\frac{\pi}{m}\right)\right|\leq1$
and $\left|\sin\left(x\right)\right|=\left|\sin\left(-x\right)\right|$,
and (ii) relies on the fact that $\sin\left(x\right)=\sin\left(\pi-x\right)$.
The property that $\sin\left(x\right)\geq x/2$ for any $x\in\left[0, {\pi}/2\right]$
allows one to further derive 
\begin{align*}
\sum_{j=1}^{m}\left|\bm{b}_{1}^{\conj}\bm{b}_{j}\right| & \leq\frac{1}{m}\left(\sum_{j=2}^{\left\lfloor \frac{m}{2}\right\rfloor +1}\frac{2m}{\left(j-1\right)\pi}+\sum_{j=\left\lfloor \frac{m}{2}\right\rfloor +2}^{m}\frac{2m}{\left(m+1-j\right)\pi}\right)+\frac{K}{m}=\frac{2}{\pi}\left(\sum_{k=1}^{\left\lfloor \frac{m}{2}\right\rfloor }\frac{1}{k}+\sum_{k=1}^{\left\lfloor \frac{m+1}{2}\right\rfloor -1}\frac{1}{k}\right)+\frac{K}{m}\\
 & \overset{\left(\text{i}\right)}{\leq}\frac{4}{\pi}\sum_{k=1}^{m}\frac{1}{k}+\frac{K}{m}\overset{\left(\text{ii}\right)}{\leq}\frac{4}{\pi}\left(1+\log m\right)+1 \overset{\left(\text{iii}\right)}{\leq}4\log m,
\end{align*}
where in (i) we extend the range of the summation, (ii) uses the elementary
inequality $\sum_{k=1}^{m}k^{-1}\leq1+\log m$ and (iii) holds
true as long as $m\geq3$. \end{proof}

The next lemma considers the difference of two inner products, namely, $\left(\bm{b}_{l}-\bm{b}_{1}\right)^{\conj}\bm{b}_{j}$.
\begin{lemma}\label{lemma:fourier-diff-detail}For all $0\leq l-1\leq\tau\leq\left\lfloor \frac{m}{10}\right\rfloor $,
we have 
\[
\left|\left(\bm{b}_{l}-\bm{b}_{1}\right)^{\conj}\bm{b}_{j}\right|\leq\begin{cases}
\frac{4\tau}{\left(j-l\right)}\frac{K}{m}+\frac{8\tau/\pi}{\left(j-l\right)^{2}} & \text{for}\quad l+\tau\leq j\leq\left\lfloor \frac{m}{2}\right\rfloor +1,\\
\frac{4\tau}{m-\left(j-l\right)}\frac{K}{m}+\frac{8\tau/\pi}{\left[m-\left(j-1\right)\right]^{2}} & \text{for}\quad\left\lfloor \frac{m}{2}\right\rfloor +l\leq j\leq m-\tau.
\end{cases}
\]
In addition, for any $j$ and $l$, the following uniform upper bound
holds 
\[
\left|\left(\bm{b}_{l}-\bm{b}_{1}\right)^{\conj}\bm{b}_{j}\right|\leq2\frac{K}{m}.
\]
\end{lemma} 

\begin{proof}Given (\ref{eq:fourier-inner-prod}),
we can obtain for $j\neq l$ and $j\neq1$, 
\begin{align*}
\left|\left(\bm{b}_{l}-\bm{b}_{1}\right)^{\conj}\bm{b}_{j}\right| & =\frac{1}{m}\left|\frac{1-\omega^{K\left(l-j\right)}}{1-\omega^{l-j}}-\frac{1-\omega^{K\left(1-j\right)}}{1-\omega^{1-j}}\right|\\
 & =\frac{1}{m}\left|\frac{1-\omega^{K\left(l-j\right)}}{1-\omega^{l-j}}-\frac{1-\omega^{K\left(1-j\right)}}{1-\omega^{l-j}}+\frac{1-\omega^{K\left(1-j\right)}}{1-\omega^{l-j}}-\frac{1-\omega^{K\left(1-j\right)}}{1-\omega^{1-j}}\right|\\
 & =\frac{1}{m}\left|\frac{\omega^{K\left(1-j\right)}-\omega^{K\left(l-j\right)}}{1-\omega^{l-j}}+\left(\omega^{l-j}-\omega^{1-j}\right)\frac{1-\omega^{K\left(1-j\right)}}{\left(1-\omega^{l-j}\right)\left(1-\omega^{1-j}\right)}\right|\\
 & \leq\frac{1}{m}\left|\frac{1-\omega^{K\left(l-1\right)}}{1-\omega^{l-j}}\right|+\frac{2}{m}\left|\left(1-\omega^{1-l}\right)\frac{1}{\left(1-\omega^{l-j}\right)\left(1-\omega^{1-j}\right)}\right|,
\end{align*}
where the last line is due to the triangle inequality and 
$\left|\omega^{\alpha}\right|=1$ for all $\alpha \in \RR$. The identity (\ref{eq:omega-1-minus}) allows us to rewrite this bound
as 
\begin{equation}
\left|\left(\bm{b}_{l}-\bm{b}_{1}\right)^{\conj}\bm{b}_{j}\right|\leq\frac{1}{m}\left|\frac{1}{\sin\left[\left(l-j\right)\frac{\pi}{m}\right]}\right|\left\{ \left|\sin\left[K\left(l-1\right)\frac{\pi}{m}\right]\right|+\left|\frac{\sin\left[\left(1-l\right)\frac{\pi}{m}\right]}{\sin\left[\left(1-j\right)\frac{\pi}{m}\right]}\right|\right\} .\label{eq:fourier-inner-diff}
\end{equation}

Combined with the fact that $\left|\sin x\right|\leq2\left|x\right|$
for all $x \in \RR$, we can upper bound (\ref{eq:fourier-inner-diff})
as 
\[
\left|\left(\bm{b}_{l}-\bm{b}_{1}\right)^{\conj}\bm{b}_{j}\right|\leq\frac{1}{m}\left|\frac{1}{\sin\left[\left(l-j\right)\frac{\pi}{m}\right]}\right|\left\{ 2K\tau\frac{\pi}{m}+\left|\frac{2\tau\frac{\pi}{m}}{\sin\left[\left(1-j\right)\frac{\pi}{m}\right]}\right|\right\} ,
\]
where we also utilize the assumption $0\leq l-1\leq\tau$. Then
for $l+\tau\leq j\leq\left\lfloor {m}/{2}\right\rfloor +1$, one
has 
\[
\left|\left(l-j\right)\frac{\pi}{m}\right|\leq\frac{\pi}{2}\qquad\text{and}\qquad\left|\left(1-j\right)\frac{\pi}{m}\right|\leq\frac{\pi}{2}.
\]
Therefore, utilizing the property $\sin\left(x\right)\geq x/2$
for any $x\in\left[0, \pi/2\right]$, we arrive at 
\[
\left|\left(\bm{b}_{l}-\bm{b}_{1}\right)^{\conj}\bm{b}_{j}\right|\leq\frac{2}{\left(j-l\right)\pi}\left(2K\tau\frac{\pi}{m}+\frac{4\tau}{j-1}\right)\leq\frac{4\tau}{(j-l)}\frac{K}{m}+\frac{8\tau/\pi}{\left(j-l\right)^{2}},
\]
where the last inequality holds since $j-1 > j-l$. 
Similarly we can obtain the upper bound for $\left\lfloor {m}/{2}\right\rfloor +l\leq j\leq m-\tau$
using nearly identical argument (which is omitted for brevity).

The uniform upper bound can be justified as follows 
\[
\left|\left(\bm{b}_{l}-\bm{b}_{1}\right)^{\conj}\bm{b}_{j}\right|\leq\left(\left\Vert \bm{b}_{l}\right\Vert _{2}+\left\Vert \bm{b}_{1}\right\Vert _{2}\right)\left\Vert \bm{b}_{j}\right\Vert _{2}\leq2{K}/{m}.
\]
The last relation holds since $\left\Vert \bm{b}_{l}\right\Vert _{2}^{2}= K/m$
for all $1\leq l\leq m$. \end{proof}

Next, we list 
two consequences of the above estimates in Lemma~\ref{lemma:fourier-sum-diff} and Lemma~\ref{lemma:fourier-sum-square}.

\begin{lemma}
\label{lemma:fourier-sum-diff}
Fix any constant $c>0$ that is independent of $m$ and $K$. Suppose $m\geq C\tau K\log^{4}m$
for some sufficiently large constant $C>0$, which solely depends on $c$. If $0\leq l-1\leq\tau$,
then one has 
\[
\sum_{j=1}^{m}\left|\left(\bm{b}_{l}-\bm{b}_{1}\right)^{\conj}\bm{b}_{j}\right|\leq \frac{c}{\log^{2}m}.
\]
\end{lemma}\begin{proof}For some constant $c_{0}>0$, we can split
the index set $\left[m\right]$ into the following three disjoint sets
\begin{align*}
\cA_{1} & =\left\{ j:l+c_{0}\tau\log^{2}m\leq j\leq\left\lfloor \frac{m}{2}\right\rfloor \right\} ,\\
\cA_{2} & =\left\{ j:\left\lfloor \frac{m}{2}\right\rfloor +l\leq j\leq m-c_{0}\tau\log^{2}m\right\} ,\\
\text{and}\qquad\cA_{3} & =\left[m\right]\backslash\left(\cA_{1}\cup\cA_{2}\right).
\end{align*}
With this decomposition in place, we can write 
\[
\sum_{j=1}^{m}\left|\left(\bm{b}_{l}-\bm{b}_{1}\right)^{\conj}\bm{b}_{j}\right|=\sum_{j\in\cA_{1}}\left|\left(\bm{b}_{l}-\bm{b}_{1}\right)^{\conj}\bm{b}_{j}\right|+\sum_{j\in\cA_{2}}\left|\left(\bm{b}_{l}-\bm{b}_{1}\right)^{\conj}\bm{b}_{j}\right|+\sum_{j\in\cA_{3}}\left|\left(\bm{b}_{l}-\bm{b}_{1}\right)^{\conj}\bm{b}_{j}\right|.
\]

We first look at $\mathcal{A}_{1}$. By Lemma \ref{lemma:fourier-diff-detail},
one has for any $j\in\cA_{1}$, 
\[
\left|\left(\bm{b}_{l}-\bm{b}_{1}\right)^{\conj}\bm{b}_{j}\right|\leq\frac{4\tau}{j-l}\frac{K}{m}+\frac{8\tau/\pi}{\left(j-l\right)^{2}},
\]
and hence 
\begin{align*}
\sum_{j\in\cA_{1}}\left|\left(\bm{b}_{l}-\bm{b}_{1}\right)^{\conj}\bm{b}_{j}\right| & \leq\sum_{j=l+c_{0}\tau\log^{2}m}^{\left\lfloor \frac{m}{2}\right\rfloor +1}\left(\frac{4\tau}{j-l}\frac{K}{m}+\frac{8\tau/\pi}{\left(j-l\right)^{2}}\right)\leq\frac{4\tau K}{m}\sum_{k=1}^{m}\frac{1}{k}+\frac{8\tau}{\pi}\sum_{k=c_{0}\tau\log^{2}m}^{m}\frac{1}{k^{2}}\\
 & \leq8\tau\frac{K}{m}\log m+\frac{16\tau}{\pi}\frac{1}{c_{0}\tau\log^{2}m},
\end{align*}
where the last inequality arises from $\sum_{k=1}^{m} {k}^{-1} \leq 1+\log m \leq 2\log m$ and $\sum_{k=c}^{m} {k^{-2}}\leq 2/{c}$.

Similarly, for $j\in\cA_{2}$, we have 
\[
\left|\left(\bm{b}_{l}-\bm{b}_{1}\right)^{\conj}\bm{b}_{j}\right|\leq\frac{4\tau}{m-\left(j-l\right)}\frac{K}{m}+\frac{8\tau/\pi}{\left[m-\left(j-1\right)\right]^{2}},
\]
which in turn implies 
\begin{align*}
\sum_{j\in\cA_{2}}\left|\left(\bm{b}_{l}-\bm{b}_{1}\right)^{\conj}\bm{b}_{j}\right| & \leq8\tau\frac{K}{m}\log m+\frac{16\tau}{\pi}\frac{1}{c_{0}\tau\log^{2}m}.
\end{align*}

Regarding $j\in\cA_{3}$, we observe that 
\[
\left|\cA_{3}\right|\leq2\left(c_{0}\tau\log^{2}m+l\right)\leq2\left(c_{0}\tau\log^{2}m+\tau +1 \right)\leq4c_{0}\tau\log^{2}m.
\]
This together with the simple bound $\left|\left(\bm{b}_{l}-\bm{b}_{1}\right)^{\conj}\bm{b}_{j}\right|\le 2{K}/{m}$
gives 
\[
\sum_{j\in\cA_{3}}\left|\left(\bm{b}_{l}-\bm{b}_{1}\right)^{\conj}\bm{b}_{j}\right|\leq2\frac{K}{m}\left|\cA_{3}\right|\leq\frac{8c_{0}\tau K\log^{2}m}{m}.
\]

The previous three estimates taken collectively yield 
\begin{align*}
\sum_{j=1}^{m}\left|\left(\bm{b}_{l}-\bm{b}_{1}\right)^{\conj}\bm{b}_{j}\right| & \leq\frac{16\tau K\log m}{m}+\frac{32\tau}{\pi}\frac{1}{c_{0}\tau\log^{2}m}+\frac{8c_{0}\tau K\log^{2}m}{m}\leq c\frac{1}{\log^{2}m}
\end{align*}
as long as $c_{0}\geq ({32}/ {\pi})\cdot ({1}/ {c})$ and $m\geq {8 c_0} \tau K\log^{4}m/c$.\end{proof} 

\begin{lemma}
\label{lemma:fourier-sum-square}
Fix any constant $c>0$ that is independent of $m$
and $K$. Consider an integer $\tau >0$, and suppose that $m\geq C\tau K\log m$
for some large constant $C>0$, which depends solely  on $c$. Then we have
\[
\sum_{k=0}^{\left\lfloor {m}/{\tau}\right\rfloor }\sqrt{\sum_{j=1}^{\tau}\left|\bm{b}_{1}^{\conj}\left(\bm{b}_{k\tau+j}-\bm{b}_{k\tau+1}\right)\right|^{2}}\leq \frac{c}{\sqrt{\tau}}.
\]
\end{lemma} \begin{proof}The proof strategy is similar
to the one used in Lemma \ref{lemma:fourier-sum-diff}. First notice
that 
\[
\left|\bm{b}_{1}^{\conj}\left(\bm{b}_{k\tau+j}-\bm{b}_{k\tau+1}\right)\right|=\left|\left(\bm{b}_{m}-\bm{b}_{m+1-j}\right)^{\conj}\bm{b}_{k\tau}\right|.
\]
As before, for some $c_{1}>0$, we can split the index set $\left\{ 1,\cdots,\left\lfloor {m}/{\tau}\right\rfloor \right\} $
into three disjoint sets
\begin{align*}
\cB_{1} & =\left\{ k:c_{1}\leq k\leq\left\lfloor \left(\left\lfloor \frac{m}{2}\right\rfloor +1-j\right)/\tau\right\rfloor \right\} ,\\
\cB_{2} & =\left\{ k:\left\lfloor \left(\left\lfloor \frac{m}{2}\right\rfloor +1-j\right)/\tau\right\rfloor +1\leq k\leq\left\lfloor \left(m+1-j\right)/\tau\right\rfloor -c_{1}\right\} ,\\
\text{and}\qquad\cB_{3} & =\left\{ 1,\cdots,\left\lfloor \frac{m}{\tau}\right\rfloor \right\} \backslash\left(\cB_{1}\cup\cB_{2}\right),
\end{align*}
where $1\leq j\leq \tau$. 

By Lemma \ref{lemma:fourier-diff-detail}, one has 
\[
\left|\left(\bm{b}_{m}-\bm{b}_{m+1-j}\right)^{\conj}\bm{b}_{k\tau}\right|\leq\frac{4\tau}{k\tau}\frac{K}{m}+\frac{8\tau/\pi}{\left(k\tau\right)^{2}},\qquad k\in\mathcal{B}_{1}.
\]
Hence for any $k\in\cB_{1}$, 
\[
\sqrt{\sum_{j=1}^{\tau}\left|\bm{b}_{1}^{\conj}\left(\bm{b}_{k\tau+j}-\bm{b}_{k\tau+1}\right)\right|^{2}}
\leq\sqrt{\tau}\left(\frac{4\tau}{k\tau}\frac{K}{m}+\frac{8\tau/\pi}{\left(k\tau\right)^{2}}\right)
= \sqrt{\tau}\left(\frac{4}{k}\frac{K}{m}+\frac{8/\pi}{k^2\tau}\right) ,
\]
which further implies that 
\begin{align*}
\sum_{k\in\cB_{1}}\sqrt{\sum_{j=1}^{\tau}\left|\bm{b}_{1}^{\conj}\left(\bm{b}_{k\tau+j}-\bm{b}_{k\tau+1}\right)\right|^{2}} & \leq\sqrt{\tau}\sum_{k=c_{1}}^{m}\left(\frac{4}{k}\frac{K}{m}+\frac{8/\pi}{k^2\tau}\right)\leq8\sqrt{\tau}\frac{K\log m}{m}+\frac{16}{\pi}\frac{1}{\sqrt{\tau}}\frac{1}{c_{1}},
\end{align*}
where the last inequality follows since $\sum_{k=1}^m k^{-1}\leq 2\log m$ and $\sum_{k=c_1}^{m}k^{-2}\leq 2/c_1$. A similar bound can be obtained for $k\in\cB_{2}$.

For the remaining set $\cB_{3}$, observe that 
\[
\left|\cB_{3}\right|\leq2c_{1}.
\]
This together with the crude upper bound $\left|\left(\bm{b}_{l}-\bm{b}_{1}\right)^{\conj}\bm{b}_{j}\right|\le 2{K}/{m}$
gives 
\[
\sum_{k\in\cB_{3}}\sqrt{\sum_{j=1}^{\tau}\left|\bm{b}_{1}^{\conj}\left(\bm{b}_{k\tau+j}-\bm{b}_{k\tau+1}\right)\right|^{2}}\leq\left|\cB_{3}\right|\sqrt{\tau\max_{j}\left|\bm{b}_{1}^{\conj}\left(\bm{b}_{k\tau+j}-\bm{b}_{k\tau+1}\right)\right|^{2}}\leq\left|\cB_{3}\right|\sqrt{\tau}\cdot\frac{2K}{m}\leq\frac{4c_{1}\sqrt{\tau}K}{m}.
\]

The previous  estimates taken collectively yield 
\begin{align*}
\sum_{k=0}^{\left\lfloor {m}/{\tau}\right\rfloor }\sqrt{\sum_{j=1}^{\tau}\left|\bm{b}_{1}^{\conj}\left(\bm{b}_{k\tau+j}-\bm{b}_{k\tau+1}\right)\right|^{2}} & \leq2\left(8\sqrt{\tau}\frac{K\log m}{m}+\frac{16}{\pi}\frac{1}{\sqrt{\tau}}\frac{1}{c_{1}}\right)+\frac{4c_{1}\sqrt{\tau}K}{m}\leq c\frac{1}{\sqrt{\tau}},
\end{align*}
as long as $c_{1} \gg {1}/{c}$ and $m/(c_{1}\tau K\log m) \gg {1}/{c}$. \end{proof}

\subsubsection{Complex-valued alignment}

Let $g_{\bm{h},\bm{x}}\left(\cdot\right):\CC\to\RR$ be a real-valued
function defined as 
\[
g_{\bm{h},\bm{x}}\left(\alpha\right):=\left\Vert \frac{1}{\overline{\alpha}}\bm{h}-\bm{h}^{\star}\right\Vert _{2}^{2}+\left\Vert \alpha\bm{x}-\bm{x}^{\star}\right\Vert _{2}^{2},
\]
which is the key function in the definition (\ref{eq:defn-dist-BD}). Therefore, the alignment parameter of $(\bm{h},\bm{x})$ to $(\bm{h}^{\star},\bm{x}^{\star})$ is the minimizer of $g_{\bm{h},\bm{x}}\left(\alpha\right)$. This section is devoted to studying various properties of $g_{\bm{h},\bm{x}}\left(\cdot\right)$. To begin with, the Wirtinger gradient and Hessian of $g_{\bm{h},\bm{x}}\left(\cdot\right)$
can be calculated as 
\begin{equation}
\nabla g_{\bm{h},\bm{x}}\left(\alpha\right)=\left[\begin{array}{c}
\frac{\partial g_{\bm{h},\bm{x}}(\alpha,\overline{\alpha})}{\partial\alpha}\\
\frac{\partial g_{\bm{h},\bm{x}}(\alpha,\overline{\alpha})}{\partial\overline{\alpha}}
\end{array}\right]=\left[\begin{array}{c}
\alpha\left\Vert \bm{x}\right\Vert _{2}^{2}-\bm{x}^{\conj}\bm{x}^{\star}-\alpha^{-1}\left(\overline{\alpha}\right)^{-2}\left\Vert \bm{h}\right\Vert _{2}^{2}+\left(\overline{\alpha}\right)^{-2}\bm{h}^{\star\conj}\bm{h}\\
\overline{\alpha}\left\Vert \bm{x}\right\Vert _{2}^{2}-\bm{x}^{\star\conj}\bm{x}-\left(\overline{\alpha}\right)^{-1}\alpha^{-2}\left\Vert \bm{h}\right\Vert _{2}^{2}+\alpha^{-2}\bm{h}^{\conj}\bm{h}^{\star}
\end{array}\right];\label{eq:complex-alignment-gradient}
\end{equation}
\begin{equation}\label{eq:complex-alignment-hessian}
\nabla^{2}g_{\bm{h},\bm{x}}\left(\alpha\right)=\left[\begin{array}{cc}
\left\Vert \bm{x}\right\Vert _{2}^{2}+\left|\alpha\right|^{-4}\left\Vert \bm{h}\right\Vert _{2}^{2} & 2\alpha^{-1}\left(\overline{\alpha}\right)^{-3}\left\Vert \bm{h}\right\Vert _{2}^{2}-2\left(\overline{\alpha}\right)^{-3}\bm{h}^{\star\conj}\bm{h}\\
2\left(\overline{\alpha}\right)^{-1}\alpha^{-3}\left\Vert \bm{h}\right\Vert _{2}^{2}-2\alpha^{-3}\bm{h}^{\conj}\bm{h}^{\star} & \left\Vert \bm{x}\right\Vert _{2}^{2}+\left|\alpha\right|^{-4}\left\Vert \bm{h}\right\Vert _{2}^{2}
\end{array}\right].
\end{equation}

The first lemma reveals that, as long as $\left(\frac{1}{\overline{\beta}}\bm{h},\beta\bm{x}\right)$
is sufficiently close to $(\bm{h}^{\star},\bm{x}^{\star})$, the
minimizer of $g_{\bm{h},\bm{x}}\left(\alpha\right)$ cannot be far
away from $\beta$. 
\begin{lemma}
\label{lemma:alpha-close-to-one}
Assume theres exists
$\beta\in\CC$ with ${1}/ {2}\leq\left|\beta\right|\leq{3}/{2}$
such that $\max\left\{ \left\Vert \frac{1}{\overline{\beta}}\bm{h}-\bm{h}^{\star}\right\Vert _{2},\left\Vert \beta\bm{x}-\bm{x}^{\star}\right\Vert _{2}\right\} \leq\delta\leq {1}/{4}$.
Denote by $\hat{\alpha}$ the minimizer of $g_{\bm{h},\bm{x}}\left(\alpha\right)$,
then we necessarily have 
\[
\big||\hat{\alpha}|-|\beta|\big|\leq\left|\hat{\alpha}-\beta\right|\leq18\delta.
\]
\end{lemma}
\begin{proof}The first inequality is a direct consequence
of the triangle inequality. Hence we concentrate on the second one.
Notice that by assumption,
\begin{equation}
g_{\bm{h},\bm{x}}\left(\beta\right)=\left\Vert \frac{1}{\overline{\beta}}\bm{h}-\bm{h}^{\star}\right\Vert _{2}^{2}+\left\Vert \beta\bm{x}-\bm{x}^{\star}\right\Vert _{2}^{2}\leq2\delta^{2},\label{eq:UB-32}
\end{equation}
which immediately implies that $g_{\bm{h},\bm{x}}\left(\hat{\alpha}\right)\leq2\delta^{2}$.
It thus suffices to show that for any $\alpha$ obeying $\left|\alpha-\beta\right|>18\delta$,
one has $g_{\bm{h},\bm{x}}\left(\alpha\right)>2\delta^{2}$, and hence
it cannot be the minimizer. To this end, we lower bound $g_{\bm{h},\bm{x}}\left(\alpha\right)$
as follows: 
\begin{align*}
g_{\bm{h},\bm{x}}\left(\alpha\right) & \geq\left\Vert \alpha\bm{x}-\bm{x}^{\star}\right\Vert _{2}^{2}=\left\Vert \left(\alpha-\beta\right)\bm{x}+\left(\beta\bm{x}-\bm{x}^{\star}\right)\right\Vert _{2}^{2}\\
 & =\left|\alpha-\beta\right|^{2}\left\Vert \bm{x}\right\Vert _{2}^{2}+\left\Vert \beta\bm{x}-\bm{x}^{\star}\right\Vert _{2}^{2}+2\mathrm{Re}\left[\left(\alpha-\beta\right)\left(\beta\bm{x}-\bm{x}^{\star}\right)^{\conj}\bm{x}\right]\\
 & \geq\left|\alpha-\beta\right|^{2}\left\Vert \bm{x}\right\Vert _{2}^{2}-2\left|\alpha-\beta\right|\left|\left(\beta\bm{x}-\bm{x}^{\star}\right)^{\conj}\bm{x}\right|.
\end{align*}
Given that $\left\Vert \beta\bm{x}-\bm{x}^{\star}\right\Vert _{2}\leq\delta\leq {1}/{4}$
and $\left\Vert \bm{x}^{\star}\right\Vert _{2}=1$, we have 
\[
\left\Vert \beta\bm{x}\right\Vert _{2}\geq\|\bm{x}^{\star}\|_{2}-\left\Vert \beta\bm{x}-\bm{x}^{\star}\right\Vert _{2}\geq1-\delta\geq {3}/{4},
\]
which together with the fact that ${1}/{2}\leq\left|\beta\right|\leq {3}/{2}$
implies 
\[
\left\Vert \bm{x}\right\Vert _{2}\geq1/2\qquad\text{and}\qquad\left\Vert \bm{x}\right\Vert _{2}\leq2
\]
and 
\[
\big|\left(\beta\bm{x}-\bm{x}^{\star}\right)^{\conj}\bm{x}\big|\leq\left\Vert \beta\bm{x}-\bm{x}^{\star}\right\Vert _{2}\left\Vert \bm{x}\right\Vert _{2}\leq2\delta.
\]
Taking the previous estimates collectively yields
\[
g_{\bm{h},\bm{x}}\left(\alpha\right)\geq\frac{1}{4}\left|\alpha-\beta\right|^{2}-4\delta\left|\alpha-\beta\right|.
\]
It is self-evident that once $\left|\alpha-\beta\right|>18\delta,$
one gets 
$
g_{\bm{h},\bm{x}}\left(\alpha\right)>2\delta^{2},
$
and hence $\alpha$ cannot be the minimizer as $g_{\bm{h},\bm{x}}\left(\alpha\right)>g_{\bm{h},\bm{x}}\left(\beta\right)$
according to (\ref{eq:UB-32}). This concludes the proof. \end{proof}

The next lemma reveals the local strong convexity of $g_{\bm{h},\bm{x}}\left(\alpha\right)$
when $\alpha$ is close to one.

\begin{lemma}\label{lemma:alpha-strongly-convex}Assume that $\max\left\{ \left\Vert \bm{h}-\bm{h}^{\star}\right\Vert _{2},\left\Vert \bm{x}-\bm{x}^{\star}\right\Vert _{2}\right\} \leq\delta $
for some sufficiently small constant $\delta>0$. Then, for any $\alpha$
satisfying $\left|\alpha-1\right|\leq18\delta$ and any $u,v\in\CC$,
one has 
\[
\left[u^{\conj},v^{\conj}\right]\nabla^{2}g_{\bm{h},\bm{x}}\left(\alpha\right)\left[\begin{array}{c}
u\\
v
\end{array}\right]\geq\frac{1}{2}\left(\left|u\right|^{2}+\left|v\right|^{2}\right),
\]
where $\nabla^{2}g_{\bm{h},\bm{x}}\left(\cdot\right)$ stands for
the Wirtinger Hessian of $g_{\bm{h},\bm{x}}(\cdot)$. 
\end{lemma}
\begin{proof}
For simplicity of presentation, we use $g_{\bm{h},\bm{x}}\left(\alpha,\overline{\alpha}\right)$ and $g_{\bm{h},\bm{x}}\left(\alpha\right)$
interchangeably. By \eqref{eq:complex-alignment-hessian}, for any $u,v\in\CC$ , one has 
\[
\left[u^{\conj},v^{\conj}\right]\nabla^{2}g_{\bm{h},\bm{x}}\left(\alpha\right)\left[\begin{array}{c}
u\\
v
\end{array}\right]=\underbrace{\left(\left\Vert \bm{x}\right\Vert _{2}^{2}+\left|\alpha\right|^{-4}\left\Vert \bm{h}\right\Vert _{2}^{2}\right)}_{:=\beta_{1}}\left(\left|u\right|^{2}+\left|v\right|^{2}\right)+2\underbrace{\mathrm{Re}\left[u^{\conj}v\left(2\alpha^{-1}\left(\overline{\alpha}\right)^{-3}\left\Vert \bm{h}\right\Vert _{2}^{2}-2\left(\overline{\alpha}\right)^{-3}\bm{h}^{\star\conj}\bm{h}\right)\right]}_{:=\beta_{2}}.
\]

We would like to demonstrate that this is at least on the order of
$\left|u\right|^{2}+\left|v\right|^{2}$. We first develop a lower bound on $\beta_{1}$. Given the assumption
that 
$\max\left\{ \left\Vert \bm{h}-\bm{h}^{\star}\right\Vert _{2},\left\Vert \bm{x}-\bm{x}^{\star}\right\Vert _{2}\right\} \leq\delta$,
one necessarily has 
\[
1-\delta\leq\left\Vert \bm{x}\right\Vert _{2}\leq1+\delta\qquad\text{and}\qquad1-\delta\leq\left\Vert \bm{h}\right\Vert _{2}\leq1+\delta.
\]
Thus, for any $\alpha$ obeying $|\alpha-1|\leq18\delta$, one has
\[
\beta_{1}\geq\left(1+\left|\alpha\right|^{-4}\right)\left(1-\delta\right)^{2}\geq\left(1+\left(1+18\delta\right)^{-4}\right)\left(1-\delta\right)^{2}\geq1
\]
as long as $\delta>0$ is sufficiently small. Regarding the second
term $\beta_{2}$, we utilizes the conditions $\left|\alpha-1\right|\leq18\delta$,
$\|\bm{x}\|_{2}\leq1+\delta$ and $\|\bm{h}\|_{2}\leq1+\delta$ to
get 
\begin{align*}
\left|\beta_{2}\right| & 
\leq2\left|u\right|\left|v\right|\left|\alpha\right|^{-3}\left| \alpha^{-1} \left\Vert \bm{h}\right\Vert _{2}^{2}-\bm{h}^{\star \conj}\bm{h}\big|\right|\\
& = 2\left|u\right|\left|v\right|\left|\alpha\right|^{-3}\left|\left(\alpha^{-1}-1\right)\left\Vert \bm{h}\right\Vert _{2}^{2}-(\bm{h}^{\star}-\bm{h})^{\conj}\bm{h}\big|\right|\\
 & \leq2\left|u\right|\left|v\right|\left|\alpha\right|^{-3}\left(\left|\alpha^{-1}-1\right|\left\Vert \bm{h}\right\Vert _{2}^{2}+\left\Vert \bm{h}-\bm{h}^{\star}\right\Vert _{2}\left\Vert \bm{h}\right\Vert _{2}\right)\\
 & \leq2\left|u\right|\left|v\right|\left(1-18\delta\right)^{-3}\left(\frac{18\delta}{1-18\delta}\left(1+\delta\right)^{2}+\delta\left(1+\delta\right)\right)\\
 & \lesssim\delta\big(\left|u\right|^{2}+\left|v\right|^{2}\big),
\end{align*}
where the last relation holds since $2\left|u\right|\left|v\right| \leq \left|u\right|^{2}+\left|v\right|^{2}$ and $\delta > 0$ is sufficiently small. Combining the previous bounds
on $\beta_{1}$ and $\beta_{2}$, we arrive at 
\[
\left[u^{\conj},v^{\conj}\right]\nabla^{2}g_{\bm{h},\bm{x}}\left(\alpha\right)\left[\begin{array}{c}
u\\
v
\end{array}\right]\geq\left(1-O(\delta)\right)\left(\left|u\right|^{2}+\left|v\right|^{2}\right)\geq\frac{1}{2}\left(\left|u\right|^{2}+\left|v\right|^{2}\right)
\]
as long as $\delta$ is sufficiently small. This completes the proof.
\end{proof}

Additionally, in a local region surrounding the optimizer, the alignment parameter is Lipschitz continuous, namely, the
 difference of the alignment parameters associated with two distinct vector pairs is at most proportional to the $\ell_2$ distance between the two vector pairs involved, as demonstrated below. 
\begin{lemma}
\label{lemma:alpha-perturbation}
Suppose that the vectors
$\bm{x}_{1},\bm{x}_{2},\bm{h}_{1},\bm{h}_{2}\in\CC^{K}$ satisfy 
\begin{equation}
\max\left\{ \left\Vert \bm{x}_{1}-\bm{x}^{\star}\right\Vert _{2},\left\Vert \bm{h}_{1}-\bm{h}^{\star}\right\Vert _{2},\left\Vert \bm{x}_{2}-\bm{x}^{\star}\right\Vert _{2},\left\Vert \bm{h}_{2}-\bm{h}^{\star}\right\Vert _{2}\right\} \leq\delta \leq {1}/{4} 
\label{eq:x1-x2-h1-h2-norm}
\end{equation}
for some sufficiently small constant $\delta>0$. Denote by $\alpha_{1}$
and $\alpha_{2}$ the minimizers of $g_{\bm{h}_{1},\bm{x}_{1}}\left(\alpha\right)$
and $g_{\bm{h}_{2},\bm{x}_{2}}\left(\alpha\right)$, respectively.
Then we have 
\[
\left|\alpha_{1}-\alpha_{2}\right|\lesssim\left\Vert \bm{x}_{1}-\bm{x}_{2}\right\Vert _{2}+\left\Vert \bm{h}_{1}-\bm{h}_{2}\right\Vert _{2}.
\]
\end{lemma}
\begin{proof}
Since $\alpha_{1}$ minimizes $g_{\bm{h}_{1},\bm{x}_{1}}\left(\alpha\right)$, the 
mean-value form of Taylor's theorem (see Appendix \ref{sec:Wirtinger-calculus}) gives
\begin{align*}
g_{\bm{h}_{1},\bm{x}_{1}}\left(\alpha_{2}\right) & \geq g_{\bm{h}_{1},\bm{x}_{1}}\left(\alpha_{1}\right)\\
 & =g_{\bm{h}_{1},\bm{x}_{1}}\left(\alpha_{2}\right)+\nabla g_{\bm{h}_{1},\bm{x}_{1}}\left(\alpha_{2}\right)^{\conj}\left[\begin{array}{c}
\alpha_{1}-\alpha_{2}\\
\overline{\alpha_{1}-\alpha_{2}}
\end{array}\right]+\frac{1}{2}\left(\overline{\alpha_{1}-\alpha_{2}},\alpha_{1}-\alpha_{2}\right)\nabla^{2}g_{\bm{h}_{1},\bm{x}_{1}}\left(\tilde{\alpha}\right)\left[\begin{array}{c}
\alpha_{1}-\alpha_{2}\\
\overline{\alpha_{1}-\alpha_{2}}
\end{array}\right],
\end{align*}
where $\tilde{\alpha}$ is some complex number lying between $\alpha_{1}$
and $\alpha_{2}$, and $\nabla g_{\bm{h}_{1},\bm{x}_{1}}$ and $\nabla^{2}g_{\bm{h}_{1},\bm{x}_{1}}$
are the Wirtinger gradient and Hessian of $g_{\bm{h}_{1},\bm{x}_{1}}\left(\cdot\right)$,
respectively. Rearrange the previous inequality to obtain 
\begin{equation} 
\label{eq:lemma57}
\left|\alpha_{1}-\alpha_{2}\right|
\lesssim
\frac{\left\Vert \nabla g_{\bm{h}_{1},\bm{x}_{1}}\left(\alpha_{2}\right)\right\Vert _{2}}{\lambda_{\min}\left(\nabla^{2}g_{\bm{h}_{1},\bm{x}_{1}}\left(\tilde{\alpha}\right)\right)}
\end{equation}
as long as $\lambda_{\min}\left(\nabla^{2}g_{\bm{h}_{1},\bm{x}_{1}}\left(\tilde{\alpha}\right)\right) >0$.  
This calls for evaluation of the Wirtinger gradient and Hessian of
$g_{\bm{h}_{1},\bm{x}_{1}}\left(\cdot\right)$. 

Regarding the Wirtinger Hessian, by the assumption (\ref{eq:x1-x2-h1-h2-norm}),
we can invoke Lemma \ref{lemma:alpha-close-to-one} with $\beta=1$
to reach $\max\left\{ \left|\alpha_{1}-1\right|,\left|\alpha_{2}-1\right|\right\} \leq18\delta$.
This together with Lemma \ref{lemma:alpha-strongly-convex} implies
\[
\lambda_{\min}\left(\nabla^{2}g_{\bm{h}_{1},\bm{x}_{1}}\left(\tilde{\alpha}\right)\right)\geq1/2,
\]
since $\tilde{\alpha}$ lies between $\alpha_{1}$ and $\alpha_{2}$. 

 For the Wirtinger gradient, since $\alpha_{2}$ is the minimizer of
$g_{\bm{h}_{2},\bm{x}_{2}}\left(\alpha\right)$, the first-order optimality condition \cite[equation (38)]{kreutz2009complex} requires $\nabla g_{\bm{h}_{2},\bm{x}_{2}}\left(\alpha_{2}\right)=\bm{0}$ , which gives 
\begin{align*}
 & \left\Vert \nabla g_{\bm{h}_{1},\bm{x}_{1}}\left(\alpha_{2}\right)\right\Vert _{2}=\left\Vert \nabla g_{\bm{h}_{1},\bm{x}_{1}}\left(\alpha_{2}\right)-\nabla g_{\bm{h}_{2},\bm{x}_{2}}\left(\alpha_{2}\right)\right\Vert _{2}.
\end{align*}
Plug in the gradient expression \eqref{eq:complex-alignment-gradient}
to reach 
\begin{align*}
 & \left\Vert \nabla g_{\bm{h}_{1},\bm{x}_{1}}\left(\alpha_{2}\right)-\nabla g_{\bm{h}_{2},\bm{x}_{2}}\left(\alpha_{2}\right)\right\Vert _{2}\\
 & \quad=\sqrt{2}\;\Big|\left[\alpha_{2}\left\Vert \bm{x}_{1}\right\Vert _{2}^{2}-\bm{x}_{1}^{\conj}\bm{x}^{\star}-\alpha_{2}^{-1}\left(\overline{\alpha_{2}}\right)^{-2}\left\Vert \bm{h}_{1}\right\Vert _{2}^{2}+\left(\overline{\alpha_{2}}\right)^{-2}\bm{h}^{\star\conj}\bm{h}_{1}\right]\\
 & \qquad\qquad-\left[\alpha_{2}\left\Vert \bm{x}_{2}\right\Vert _{2}^{2}-\bm{x}_{2}^{\conj}\bm{x}^{\star}-\alpha_{2}^{-1}\left(\overline{\alpha_{2}}\right)^{-2}\left\Vert \bm{h}_{2}\right\Vert _{2}^{2}+\left(\overline{\alpha_{2}}\right)^{-2}\bm{h}^{\star\conj}\bm{h}_{2}\right]\Big|\\
 & \quad\lesssim\left|\alpha_{2}\right|\left|\left\Vert \bm{x}_{1}\right\Vert _{2}^{2}-\left\Vert \bm{x}_{2}\right\Vert _{2}^{2}\right|+\left|\bm{x}_{1}^{\conj}\bm{x}^{\star}-\bm{x}_{2}^{\conj}\bm{x}^{\star}\right|+\frac{1}{\left|\alpha_{2}\right|^{3}}\left|\left\Vert \bm{h}_{1}\right\Vert _{2}^{2}-\left\Vert \bm{h}_{2}\right\Vert _{2}^{2}\right|+\frac{1}{\left|\alpha_{2}\right|^{2}}\left|\bm{h}^{\star\conj}\bm{h}_{1}-\bm{h}^{\star\conj}\bm{h}_{2}\right|\\
 & \quad\lesssim\left|\alpha_{2}\right|\left|\left\Vert \bm{x}_{1}\right\Vert _{2}^{2}-\left\Vert \bm{x}_{2}\right\Vert _{2}^{2}\right|+\left\Vert \bm{x}_{1}-\bm{x}_{2}\right\Vert _{2}+\frac{1}{\left|\alpha_{2}\right|^{3}}\left|\left\Vert \bm{h}_{1}\right\Vert _{2}^{2}-\left\Vert \bm{h}_{2}\right\Vert _{2}^{2}\right|+\frac{1}{\left|\alpha_{2}\right|^{2}}\left\Vert \bm{h}_{1}-\bm{h}_{2}\right\Vert _{2},
\end{align*}
where the last line follows from the triangle inequality. It is straightforward
to see that 
\[
{1}/{2}\leq\left|\alpha_{2}\right|\leq2,\qquad\left|\left\Vert \bm{x}_{1}\right\Vert _{2}^{2}-\left\Vert \bm{x}_{2}\right\Vert _{2}^{2}\right|\lesssim\left\Vert \bm{x}_{1}-\bm{x}_{2}\right\Vert _{2},\qquad\left|\left\Vert \bm{h}_{1}\right\Vert _{2}^{2}-\left\Vert \bm{h}_{2}\right\Vert _{2}^{2}\right|\lesssim\left\Vert \bm{h}_{1}-\bm{h}_{2}\right\Vert _{2}
\]
under the condition \eqref{eq:x1-x2-h1-h2-norm} and the assumption $\|\bm{x}^{\star}\|_2=\|\bm{h}^{\star}\|_2=1$, where the first inequality follows from Lemma \ref{lemma:alpha-close-to-one}. Taking these estimates together reveals that
\[
\left\Vert \nabla g_{\bm{h}_{1},\bm{x}_{1}}\left(\alpha_{2}\right)-\nabla g_{\bm{h}_{2},\bm{x}_{2}}\left(\alpha_{2}\right)\right\Vert _{2}\lesssim\left\Vert \bm{x}_{1}-\bm{x}_{2}\right\Vert _{2}+\left\Vert \bm{h}_{1}-\bm{h}_{2}\right\Vert _{2}.
\]
The proof is accomplished by substituting the two bounds on the gradient and the Hessian into \eqref{eq:lemma57}.
\end{proof}

Further, if two  vector pairs are both close to the optimizer, then their distance after alignement (w.r.t.~the optimizer) cannot be much larger than their distance without alignment, as revealed by the following lemma.
\begin{lemma}
\label{lemma:stability-BD}
Suppose that the vectors
$\bm{x}_{1},\bm{x}_{2},\bm{h}_{1},\bm{h}_{2}\in\CC^{K}$ satisfy 
\begin{equation}
\max\left\{ \left\Vert \bm{x}_{1}-\bm{x}^{\star}\right\Vert _{2},\left\Vert \bm{h}_{1}-\bm{h}^{\star}\right\Vert _{2},\left\Vert \bm{x}_{2}-\bm{x}^{\star}\right\Vert _{2},\left\Vert \bm{h}_{2}-\bm{h}^{\star}\right\Vert _{2}\right\} \leq \delta \leq {1}/{4}
\label{eq:x-h-norm-assumption}
\end{equation}
for some sufficiently small constant $\delta>0$. Denote by $\alpha_{1}$
and $\alpha_{2}$ the minimizers of $g_{\bm{h}_{1},\bm{x}_{1}}\left(\alpha\right)$
and $g_{\bm{h}_{2},\bm{x}_{2}}\left(\alpha\right)$, respectively.
Then we have 
\[
\left\Vert \alpha_{1}\bm{x}_{1}-\alpha_{2}\bm{x}_{2}\right\Vert _{2}^{2}+\left\Vert \frac{1}{\overline{\alpha_{1}}}\bm{h}_{1}-\frac{1}{\overline{\alpha_{2}}}\bm{h}_{2}\right\Vert _{2}^{2}\lesssim\left\Vert \bm{x}_{1}-\bm{x}_{2}\right\Vert _{2}^{2}+\left\Vert \bm{h}_{1}-\bm{h}_{2}\right\Vert _{2}^{2}.
\]
\end{lemma}\begin{proof}To start with, we control the magnitudes
of $\alpha_{1}$ and $\alpha_{2}$. Lemma \ref{lemma:alpha-close-to-one}
together with the assumption (\ref{eq:x-h-norm-assumption}) guarantees
that 
\[
{1}/{2}\leq\left|\alpha_{1}\right|\leq2\qquad\text{and}\qquad
{1}/{2}\leq\left|\alpha_{2}\right|\leq2.
\]

Now we can prove the lemma. The triangle inequality gives 
\begin{align*}
\left\Vert \alpha_{1}\bm{x}_{1}-\alpha_{2}\bm{x}_{2}\right\Vert _{2} & =\left\Vert \alpha_{1}\left(\bm{x}_{1}-\bm{x}_{2}\right)+\left(\alpha_{1}-\alpha_{2}\right)\bm{x}_{2}\right\Vert _{2}\\
 & \leq\left|\alpha_{1}\right|\left\Vert \bm{x}_{1}-\bm{x}_{2}\right\Vert _{2}+\left|\alpha_{1}-\alpha_{2}\right|\left\Vert \bm{x}_{2}\right\Vert _{2}\\
 & \overset{\left(\text{i}\right)}{\leq}2\left\Vert \bm{x}_{1}-\bm{x}_{2}\right\Vert _{2}+2\left|\alpha_{1}-\alpha_{2}\right|\\
 & \overset{\left(\text{ii}\right)}{\lesssim}\left\Vert \bm{x}_{1}-\bm{x}_{2}\right\Vert _{2}+\left\Vert \bm{h}_{1}-\bm{h}_{2}\right\Vert _{2},
\end{align*}
where (i) holds since $\left|\alpha_{1}\right|\leq2$ and $\|\bm{x}_{2}\|_{2}\leq1+\delta\leq2$,
and (ii) arises from Lemma \ref{lemma:alpha-perturbation} that $\left|\alpha_{1}-\alpha_{2}\right|\lesssim\left\Vert \bm{x}_{1}-\bm{x}_{2}\right\Vert _{2}+\left\Vert \bm{h}_{1}-\bm{h}_{2}\right\Vert _{2}$.
Similarly, 
\begin{align*}
\left\Vert \frac{1}{\overline{\alpha_{1}}}\bm{h}_{1}-\frac{1}{\overline{\alpha_{2}}}\bm{h}_{2}\right\Vert _{2} & =\left\Vert \frac{1}{\overline{\alpha_{1}}}\left(\bm{h}_{1}-\bm{h}_{2}\right)+\left(\frac{1}{\overline{\alpha_{1}}}-\frac{1}{\overline{\alpha_{2}}}\right)\bm{h}_{2}\right\Vert _{2}\\
 & \leq\left|\frac{1}{\overline{\alpha_{1}}}\right|\left\Vert \bm{h}_{1}-\bm{h}_{2}\right\Vert _{2}+\left|\frac{1}{\overline{\alpha_{1}}}-\frac{1}{\overline{\alpha_{2}}}\right|\left\Vert \bm{h}_{2}\right\Vert _{2}\\
 & \leq2\left\Vert \bm{h}_{1}-\bm{h}_{2}\right\Vert _{2}+2\frac{\left|\alpha_{1}-\alpha_{2}\right|}{\left|\alpha_{1}\alpha_{2}\right|}\\
 & \lesssim\left\Vert \bm{x}_{1}-\bm{x}_{2}\right\Vert _{2}+\left\Vert \bm{h}_{1}-\bm{h}_{2}\right\Vert _{2},
\end{align*}
where the last inequality comes from Lemma \ref{lemma:alpha-perturbation}
as well as the facts that $|\alpha_{1}|\geq1/2$ and $|\alpha_{2}|\geq1/2$
as shown above. Combining all of the above bounds and recognizing
that $\left\Vert \bm{x}_{1}-\bm{x}_{2}\right\Vert _{2}+\left\Vert \bm{h}_{1}-\bm{h}_{2}\right\Vert _{2}\leq\sqrt{2\left\Vert \bm{x}_{1}-\bm{x}_{2}\right\Vert _{2}^{2}+2\left\Vert \bm{h}_{1}-\bm{h}_{2}\right\Vert _{2}^{2}}$,
we conclude the proof. \end{proof}


Finally, there is a useful identity associated with the minimizer of $\tilde{g}(\alpha)$ as defined below. 
\begin{lemma}
\label{lemma:bd-alignment}
For any $\bm{h}_{1},\bm{h}_{2},\bm{x}_{1},\bm{x}_{2}\in\CC^{K}$,
denote 
\[
\alpha^{\sharp}:=\arg\min_{\alpha}\tilde{g}(\alpha), \quad\text{where}\quad \tilde{g}\left(\alpha\right):=\left\Vert \frac{1}{\overline{\alpha}}\bm{h}_{1}-\bm{h}_{2}\right\Vert _{2}^{2}+\left\Vert \alpha\bm{x}_{1}-\bm{x}_{2}\right\Vert _{2}^{2}. \]
Let $\tilde{\bm{x}}_{1}=\alpha^{\sharp}\bm{x}_{1}$ and $\tilde{\bm{h}}_{1}=\frac{1}{\overline{\alpha^{\sharp}}}\bm{h}_{1}$, then we have
\begin{align*}
\big\|\tilde{\bm{x}}_{1}-\bm{x}_{2}\big\|_{2}^{2}+\bm{x}_{2}^{\conj}\left(\tilde{\bm{x}}_{1}-\bm{x}_{2}\right) & =\big\|\tilde{\bm{h}}_{1}-\bm{h}_{2}\big\|_{2}^{2}+\big(\tilde{\bm{h}}_{1}-\bm{h}_{2}\big)^{\conj}\bm{h}_{2}.
\end{align*}
\end{lemma}
\begin{proof} 
We can rewrite the function $\tilde{g}\left(\alpha\right)$ as 
\begin{align*}
\tilde{g}\left(\alpha\right) & =\left|\alpha\right|^{2}\left\Vert \bm{x}_{1}\right\Vert _{2}^{2}+\left\Vert \bm{x}_{2}\right\Vert _{2}^{2}-\left(\alpha\bm{x}_{1}\right)^{\conj}\bm{x}_{2}-\bm{x}_{2}^{\conj}\left(\alpha\bm{x}_{1}\right)+\left|\frac{1}{\overline{\alpha}}\right|^{2}\left\Vert \bm{h}_{1}\right\Vert _{2}^{2}+\left\Vert \bm{h}_{2}\right\Vert _{2}^{2}-\left(\frac{1}{\overline{\alpha}}\bm{h}_{1}\right)^{\conj}\bm{h}_{2}-\bm{h}_{2}^{\conj}\left(\frac{1}{\overline{\alpha}}\bm{h}_{1}\right)\\
 & =\overline{\alpha}\alpha\left\Vert \bm{x}_{1}\right\Vert _{2}^{2}+\left\Vert \bm{x}_{2}\right\Vert _{2}^{2}-\overline{\alpha}\bm{x}_{1}^{\conj}\bm{x}_{2}-\alpha\bm{x}_{2}^{\conj}\bm{x}_{1}+\frac{1}{\overline{\alpha}\alpha}\left\Vert \bm{h}_{1}\right\Vert _{2}^{2}+\left\Vert \bm{h}_{2}\right\Vert _{2}^{2}-\frac{1}{\alpha}\bm{h}_{1}^{\conj}\bm{h}_{2}-\frac{1}{\overline{\alpha}}\bm{h}_{2}^{\conj}\bm{h}_{1}.
\end{align*}
The first-order optimality condition \cite[equation (38)]{kreutz2009complex}
requires 
\[
\left.\frac{\partial\tilde{g}}{\partial\overline{\alpha}}\right|_{\alpha=\alpha^{\sharp}}=\alpha^{\sharp}\left\Vert \bm{x}_{1}\right\Vert _{2}^{2}-\bm{x}_{1}^{\conj}\bm{x}_{2}+\frac{1}{\alpha^{\sharp}}\left(-\frac{1}{\overline{\alpha^{\sharp}}^{2}}\right)\left\Vert \bm{h}_{1}\right\Vert _{2}^{2}-\left(-\frac{1}{\overline{\alpha^{\sharp}}^{2}}\right)\bm{h}_{2}^{\conj}\bm{h}_{1}=0,
\]
which further simplifies to 
\[
\left\Vert \tilde{\bm{x}}_{1}\right\Vert _{2}^{2}-\tilde{\bm{x}}_{1}^{\conj}\bm{x}_{2}=\big\|\tilde{\bm{h}}_{1}\big\|_{2}^{2}-\bm{h}_{2}^{\conj}\tilde{\bm{h}}_{1}
\]
since $\tilde{\bm{x}}_{1}=\alpha^{\sharp}\bm{x}_{1}$,
$\tilde{\bm{h}}_{1}=\frac{1}{\overline{\alpha^{\sharp}}}\bm{h}_{1}$, and $\alpha^{\sharp}\neq0$ (otherwise $\tilde{g}(\alpha^{\sharp})=\infty$
and cannot be the minimizer).
Furthermore, this condition is equivalent to 
\begin{align*}
\tilde{\bm{x}}_{1}^{\conj}\left(\tilde{\bm{x}}_{1}-\bm{x}_{2}\right) & =\big(\tilde{\bm{h}}_{1}-\bm{h}_{2}\big)^{\conj}\tilde{\bm{h}}_{1}.
\end{align*}
Recognizing that 
\begin{align*}
\tilde{\bm{x}}_{1}^{\conj}\left(\tilde{\bm{x}}_{1}-\bm{x}_{2}\right) & =\bm{x}_{2}^{\conj}\left(\tilde{\bm{x}}_{1}-\bm{x}_{2}\right)+\big(\tilde{\bm{x}}_{1}-\bm{x}_{2}\big)^{\conj}\left(\tilde{\bm{x}}_{1}-\bm{x}_{2}\right)=\bm{x}_{2}^{\conj}\left(\tilde{\bm{x}}_{1}-\bm{x}_{2}\right)+\|\tilde{\bm{x}}_{1}-\bm{x}_{2}\|_{2}^{2},\\
\tilde{\bm{h}}_{1}^{\conj}\big(\tilde{\bm{h}}_{1}-\bm{h}_{2}\big) & =\bm{h}_{2}^{\conj}\big(\tilde{\bm{h}}_{1}-\bm{h}_{2}\big)+\big(\tilde{\bm{h}}_{1}-\bm{h}_{2}\big)^{\conj}\big(\tilde{\bm{h}}_{1}-\bm{h}_{2}\big)=\bm{h}_{2}^{\conj}\big(\tilde{\bm{h}}_{1}-\bm{h}_{2}\big)+\|\tilde{\bm{h}}_{1}-\bm{h}_{2}\|_{2}^{2},
\end{align*}
we arrive at the desired identity.
\end{proof}

\subsubsection{Matrix concentration inequalities}

The proof for blind deconvolution is largely built upon the concentration of random matrices that are functions of $\left\{\bm{a}_j\bm{a}_j^\conj\right\}$. In this subsection, we collect the measure concentration results for various forms of random matrices that we encounter in the analysis. 
\begin{lemma}\label{lemma-covariance-1205}
Suppose $\bm{a}_{j}\overset{\text{i.i.d.}}{\sim}\mathcal{N}\left(\bm{0},\frac{1}{2}\bm{I}_{K}\right)+i\mathcal{N}\left(\bm{0},\frac{1}{2}\bm{I}_{K}\right)$ for every $1\leq j \leq m$, and $\{c_j\}_{1 \leq j \leq m}$ are a set of fixed numbers. 
Then there exist some universal constants $\tilde{C}_1,\tilde{C}_2 >0 $ such that for all $t\geq 0$
\[
\PP\left(
\Bigg\| \sum_{j=1}^m c_j ( \bm{a}_j \bm{a}_j^\conj - \bm{I}_K ) \Bigg\| \geq t
\right) \leq 2 \exp \left(
\tilde{C}_1 K - \tilde{C}_2 \min\left\{
\frac{t}{\max_{j}|c_j| },
\frac{t^2}{\sum_{j=1}^{m} c_j^2 }
\right\}
\right).
\]
\end{lemma}
\begin{proof}This is a simple variant of \cite[Theorem 5.39]{Vershynin2012}, which uses the Bernstein inequality and the standard covering argument. Hence we omit its proof.
\end{proof}
\begin{lemma}\label{lemma-covariance-union-1206}
Suppose $\bm{a}_{j}\overset{\text{i.i.d.}}{\sim}\mathcal{N}\left(\bm{0},\frac{1}{2}\bm{I}_{K}\right)+i\mathcal{N}\left(\bm{0},\frac{1}{2}\bm{I}_{K}\right)$ for every $1\leq j \leq m$. Then there exist some absolute constants $\tilde{C}_1,\tilde{C}_2, \tilde{C}_{3}> 0 $ such that for all $\max\{ 1,3 \tilde{C}_1 K/\tilde{C}_2  \} / m  \leq  \varepsilon \leq 1$, one has 
\[
\PP\left(
\sup_{|J| \leq \varepsilon m }
\Bigg\| \sum_{j\in J} \bm{a}_j \bm{a}_j^\conj \Bigg\| \geq 4\tilde{C}_3 \varepsilon m \log \frac{e}{\varepsilon}
\right) 
\leq
2\exp \left( - \frac{\tilde{C}_2  \tilde{C}_3}{3} \varepsilon m \log \frac{e}{\varepsilon} \right),
\]
where $J\subseteq [m]$ and $|J|$ denotes its cardinality.
\end{lemma}
\begin{proof}The proof relies on Lemma~\ref{lemma-covariance-1205} and the union bound. First, invoke Lemma~\ref{lemma-covariance-1205} to see that for any fixed $J\subseteq [m]$ and for all $t\geq 0$, we have
\begin{align}
\PP\left(
\Bigg\| \sum_{j\in J} ( \bm{a}_j \bm{a}_j^\conj - \bm{I}_K ) \Bigg\| \geq |J|\, t
\right) \leq 2 \exp \left(
\tilde{C}_1 K - \tilde{C}_2 |J| \min\left\{
t,t^2
\right\}
\right)\label{eq:bd-RIP-prev},
\end{align}
for some constants $\tilde{C}_1,\tilde{C}_2 > 0$, 
and as a result,
\begin{align*}
\PP\left(
\sup_{|J| \leq \varepsilon m }
\Bigg\| \sum_{j\in J} \bm{a}_j \bm{a}_j^\conj \Bigg\| \geq \lceil \varepsilon m \rceil (1+ t)
\right) 
&\overset{\text{(i)}}{\leq} \PP\left(
\sup_{|J|= \lceil \varepsilon m \rceil }
\Bigg\| \sum_{j\in J}  \bm{a}_j \bm{a}_j^\conj  \Bigg\| \geq  \lceil \varepsilon m \rceil (1+t)
\right)  \\
&\leq 
\PP\left(
\sup_{|J|= \lceil \varepsilon m \rceil }
\Bigg\| \sum_{j\in J} ( \bm{a}_j \bm{a}_j^\conj - \bm{I}_K ) \Bigg\| \geq  \lceil \varepsilon m \rceil t
\right) \\
&\overset{\text{(ii)}}{\leq} {m \choose \lceil \varepsilon m \rceil }
\cdot 2 \exp \left(
\tilde{C}_1 K - \tilde{C}_2 \lceil \varepsilon m \rceil \min\left\{
t,t^2
\right\}
\right),
\end{align*}
where $\lceil c \rceil$ denotes the smallest integer that is no smaller than $c$. Here, (i) holds since we take the supremum over a larger set and (ii) results from (\ref{eq:bd-RIP-prev}) and the union bound. Apply the elementary inequality ${n \choose k} \leq (en/k)^k$ for any $0\leq k \leq n$ to obtain
\begin{align}
\PP\left(
\sup_{|J| \leq \varepsilon m }
\Bigg\Vert \sum_{j\in J} \bm{a}_j \bm{a}_j^\conj \Bigg\Vert \geq \lceil \varepsilon m \rceil (1+ t)
\right) &\leq 2\left( \frac{e m}{\lceil \varepsilon m \rceil}\right)^{\lceil \varepsilon m \rceil} \exp \left(
\tilde{C}_1 K - \tilde{C}_2 \lceil \varepsilon m \rceil \min\left\{
t,t^2
\right\}\right) \nonumber\\
&\leq 2\left( \frac{e}{ \varepsilon } \right)^{2 \varepsilon m } \exp \left(
\tilde{C}_1 K - \tilde{C}_2  \varepsilon m \min\left\{
t,t^2
\right\}\right) \notag \\
&=2\exp \left[
\tilde{C}_1 K -  \varepsilon m \left( \tilde{C}_2  \min\left\{
t,t^2
\right\} - 2 \log (e/\varepsilon) \right) \right].\label{ineq-cov-union-1206}
\end{align}
where the second inequality uses $\varepsilon m \leq \lceil \varepsilon m\rceil \leq 2 \varepsilon m$ whenever $1/m\leq \varepsilon\leq 1$. 

The proof is then completed by taking $\tilde{C}_3 \geq \max\{1,6/\tilde{C}_2\}$ and $t = \tilde{C}_3 \log(e/\varepsilon)$. To see this, it is easy to check that $\min\{ t,t^2\} = t$ since $t\geq 1$. In addition, one has $\tilde{C}_1 K \leq \tilde{C}_2 \varepsilon m / 3 \leq \tilde{C}_2 \varepsilon m t  / 3$, and $2\log(e/\varepsilon) \leq \tilde{C}_2 t / 3$. Combine the estimates above with (\ref{ineq-cov-union-1206}) to arrive at
\begin{align*}
\PP\left(
\sup_{|J| \leq \varepsilon m }
\Bigg\| \sum_{j\in J} \bm{a}_j \bm{a}_j^\conj \Bigg\| \geq 4\tilde{C}_3 \varepsilon m \log (e/\varepsilon)
\right) 
&\overset{(\text{i})}{\leq} \PP\left(
\sup_{|J| \leq \varepsilon m }
\Bigg\| \sum_{j\in J} \bm{a}_j \bm{a}_j^\conj \Bigg\| \geq \lceil \varepsilon m \rceil (1+ t)
\right) \notag \\
&\leq 2\exp \left[
\tilde{C}_1 K -  \varepsilon m \left( \tilde{C}_2  \min\left\{
t,t^2
\right\} - 2 \log (e/\varepsilon) \right) \right] \\
&\overset{(\text{ii})}{\leq}
2\exp \left( -  \varepsilon m \tilde{C}_2  t/3 \right)
=2\exp \left( - \frac{\tilde{C}_2  \tilde{C}_3}{3} \varepsilon m \log(e/\varepsilon) \right)
\end{align*}
as claimed. Here (i) holds due to the facts that $\lceil \varepsilon m\rceil \leq 2 \varepsilon m$ and $1+t\leq 2t \leq 2\tilde{C}_3 \log(e/\varepsilon)$. The inequality (ii) arises from the estimates listed above.
\end{proof}

\begin{lemma}
\label{lemma:concentration-identity-BD}
Suppose $m\gg K\log^{3}m$.
With probability exceeding $1-O\left(m^{-10}\right)$, we have 
\[
\Bigg\Vert \sum_{j=1}^{m}\left|\bm{a}_{j}^{\conj}\bm{x}^{\star}\right|^{2}\bm{b}_{j}\bm{b}_{j}^{\conj}-\bm{I}_{K}\Bigg\Vert \lesssim\sqrt{\frac{K}{m}\log m}.
\]
\end{lemma}\begin{proof}
The identity $\sum_{j=1}^{m}\bm{b}_{j}\bm{b}_{j}^{\conj}=\bm{I}_{K}$
allows us to rewrite the quantity on the left-hand side as 
\[
\Bigg\Vert \sum_{j=1}^{m}\left|\bm{a}_{j}^{\conj}\bm{x}^{\star}\right|^{2}\bm{b}_{j}\bm{b}_{j}^{\conj}-\bm{I}_{K}\Bigg\Vert 
	=\Bigg\Vert \sum_{j=1}^{m}\underbrace{\left(\left|\bm{a}_{j}^{\conj}\bm{x}^{\star}\right|^{2}-1\right)\bm{b}_{j}\bm{b}_{j}^{\conj}}_{:=\bm{Z}_{j}}\Bigg\Vert ,
\]
where the $\bm{Z}_{j}$'s are independent zero-mean random matrices.
To control the above spectral norm, we resort to the
matrix Bernstein inequality \cite[Theorem 2.7]{Koltchinskii2011oracle}.
To this end, we first need to upper bound the sub-exponential norm
$\|\cdot\|_{\psi_{1}}$ (see definition in \cite{Vershynin2012})
of each summand $\bm{Z}_{j}$, i.e. 
\begin{align*}
\big\|\|\bm{Z}_{j}\|\big\|_{\psi_{1}} & =\left\Vert \bm{b}_{j}\right\Vert _{2}^{2}\left\Vert \left|\left|\bm{a}_{j}^{\conj}\bm{x}^{\star}\right|^{2}-1\right|\right\Vert _{\psi_{1}}\lesssim\left\Vert \bm{b}_{j}\right\Vert _{2}^{2}\left\Vert \left|\bm{a}_{j}^{\conj}\bm{x}^{\star}\right|^{2}\right\Vert _{\psi_{1}}\lesssim\frac{K}{m},
\end{align*}
where we make use of the facts that 
\[
\left\Vert \bm{b}_{j}\right\Vert _{2}^{2}={K}/{m}\qquad\text{and}\qquad\left\Vert \left|\bm{a}_{j}^{\conj}\bm{x}^{\star}\right|^{2}\right\Vert _{\psi_{1}}\lesssim1.
\]
We further need to bound the variance parameter, that is, 
\begin{align*}
\sigma_{0}^{2} & :=\left\Vert \EE\left[\sum_{j=1}^{m}\bm{Z}_{j}\bm{Z}_{j}^{\conj}\right]\right\Vert =\left\Vert \EE\Bigg[\sum_{j=1}^{m}\left(\left|\bm{a}_{j}^{\conj}\bm{x}^{\star}\right|^{2}-1\right)^{2}\bm{b}_{j}\bm{b}_{j}^{\conj}\bm{b}_{j}\bm{b}_{j}^{\conj}\Bigg]\right\Vert \\
 & \lesssim\Bigg\Vert \sum_{j=1}^{m}\bm{b}_{j}\bm{b}_{j}^{\conj}\bm{b}_{j}\bm{b}_{j}^{\conj}\Bigg\Vert =\frac{K}{m}\Bigg\Vert \sum_{j=1}^{m}\bm{b}_{j}\bm{b}_{j}^{\conj}\Bigg\Vert =\frac{K}{m},
\end{align*}
where the second line arises since $\EE \big[ \big(|\bm{a}_{j}^{\conj}\bm{x}^{\star}|^{2}-1\big)^2\big]\asymp1$,
$\|\bm{b}_{j}\|_{2}^{2}=K/m$, and $\sum_{j=1}^{m}\bm{b}_{j}\bm{b}_{j}^{\conj}=\bm{I}_{K}$.
A direct application of the matrix Bernstein inequality \cite[Theorem 2.7]{Koltchinskii2011oracle}
leads us to conclude that with probability exceeding $1-O\left(m^{-10}\right)$,
\begin{align*}
\Big\Vert \sum\nolimits_{j=1}^{m}\bm{Z}_{j}\Big\Vert  & \lesssim\max\left\{ \sqrt{\frac{K}{m}\log m},\frac{K}{m}\log^{2}m\right\} \asymp \sqrt{\frac{K}{m}\log{m}},
\end{align*}
where the last relation holds under the assumption that $m\gg K\log^{3}m$.\end{proof}


\subsubsection{Matrix perturbation bounds}
We also need the following perturbation bound on the top singular vectors of a given matrix. The following lemma is parallel to Lemma~\ref{lemma:eigenvalue-difference}. 
\begin{lemma}\label{lemma:singular-value-difference}Let $\sigma_{1}(\bm{A})$,
$\bm{u}$ and $\bm{v}$ be the leading singular value, left and right
singular vectors of $\bm{A}$, respectively, and let $\sigma_{1}(\tilde{\bm{A}})$,
$\tilde{\bm{u}}$ and $\tilde{\bm{v}}$ be the leading singular value,
left and right singular vectors of $\tilde{\bm{A}}$, respectively. Suppose $\sigma_{1}(\bm{A})$ and $\sigma_{1}(\tilde{\bm{A}})$ are not identically zero, then one has 
\begin{align*}
\left|\sigma_{1}(\bm{A})-\sigma_{1}(\tilde{\bm{A}})\right| \leq\big\|\big(\bm{A}-\tilde{\bm{A}}\big)\bm{v}\big\|_{2}&+\left(\left\Vert \bm{u}-\tilde{\bm{u}}\right\Vert _{2}+\left\Vert \bm{v}-\tilde{\bm{v}}\right\Vert _{2}\right)\big\|\tilde{\bm{A}}\big\|; \\
\left\Vert \sqrt{\sigma_{1}(\bm{A})}\;\bm{u}-\sqrt{\sigma_{1}(\tilde{\bm{A}})}\;\tilde{\bm{u}}\right\Vert _{2}+\left\Vert \sqrt{\sigma_{1}(\bm{A})}\;\bm{v}-\sqrt{\sigma_{1}(\tilde{\bm{A}})}\;\tilde{\bm{v}}\right\Vert _{2} &\leq\sqrt{\sigma_{1}(\bm{A})}\left(\left\Vert \bm{u}-\tilde{\bm{u}}\right\Vert _{2}+\left\Vert \bm{v}-\tilde{\bm{v}}\right\Vert _{2}\right)+\frac{2\left|\sigma_{1}(\bm{A})-\sigma_{1}(\tilde{\bm{A}})\right|}{\sqrt{\sigma_{1}(\bm{A})}+\sqrt{\sigma_{1}(\tilde{\bm{A}})}}.
\end{align*}
\end{lemma}
\begin{proof}The
first claim follows since 
\begin{align*}
\left|\sigma_{1}(\bm{A})-\sigma_{1}(\tilde{\bm{A}})\right| & =\left|\bm{u}^{\conj}\bm{A}\bm{v}-\tilde{\bm{u}}^{\conj}\tilde{\bm{A}}\tilde{\bm{v}}\right|\nonumber \\
 & \leq\left|\bm{u}^{\conj}\big(\bm{A}-\tilde{\bm{A}}\big)\bm{v}\right|+\left|\bm{u}^{\conj}\tilde{\bm{A}}\bm{v}-\tilde{\bm{u}}^{\conj}\tilde{\bm{A}}\bm{v}\right|+\left|\tilde{\bm{u}}^{\conj}\tilde{\bm{A}}\bm{v}-\tilde{\bm{u}}^{\conj}\tilde{\bm{A}}\tilde{\bm{v}}\right|\nonumber \\
 & \leq\big\|\big(\bm{A}-\tilde{\bm{A}}\big)\bm{v}\big\|_{2}+\left\Vert \bm{u}-\tilde{\bm{u}}\right\Vert _{2}\big\|\tilde{\bm{A}}\big\|+\big\|\tilde{\bm{A}}\big\|\left\Vert \bm{v}-\tilde{\bm{v}}\right\Vert _{2}.
\end{align*}
With regards to the second claim, we see that 
\begin{align*}
\left\Vert \sqrt{\sigma_{1}(\bm{A})}\;\bm{u}-\sqrt{\sigma_{1}(\tilde{\bm{A}})}\;\tilde{\bm{u}}\right\Vert _{2} & \leq\left\Vert \sqrt{\sigma_{1}(\bm{A})}\;\bm{u}-\sqrt{\sigma_{1}(\bm{A})}\;\tilde{\bm{u}}\right\Vert _{2}+\left\Vert \sqrt{\sigma_{1}(\bm{A})}\;\tilde{\bm{u}}-\sqrt{\sigma_{1}(\tilde{\bm{A}})}\;\tilde{\bm{u}}\right\Vert _{2}\\
 & =\sqrt{\sigma_{1}(\bm{A})}\left\Vert \bm{u}-\tilde{\bm{u}}\right\Vert _{2}+\left|\sqrt{\sigma_{1}(\bm{A})}-\sqrt{\sigma_{1}(\tilde{\bm{A}})}\right|\\
 & =\sqrt{\sigma_{1}(\bm{A})}\left\Vert \bm{u}-\tilde{\bm{u}}\right\Vert _{2}+\frac{\left|\sigma_{1}(\bm{A})-\sigma_{1}(\tilde{\bm{A}})\right|}{\sqrt{\sigma_{1}(\bm{A})}+\sqrt{\sigma_{1}(\tilde{\bm{A}})}}.
\end{align*}
Similarly, one can obtain 
\begin{align*}
\left\Vert \sqrt{\sigma_{1}(\bm{A})}\;\bm{v}-\sqrt{\sigma_{1}(\tilde{\bm{A}})}\;\tilde{\bm{v}}\right\Vert _{2} & \leq\sqrt{\sigma_{1}(\bm{A})}\left\Vert \bm{v}-\tilde{\bm{v}}\right\Vert _{2}+\frac{\left|\sigma_{1}(\bm{A})-\sigma_{1}(\tilde{\bm{A}})\right|}{\sqrt{\sigma_{1}(\bm{A})}+\sqrt{\sigma_{1}(\tilde{\bm{A}})}}.
\end{align*}
Add these two inequalities to complete the proof. \end{proof}

\end{document}